%% file: review_sparse_arxiv2.tex
\documentclass[openany]{now_arxiv} 

\usepackage{etex}
\usepackage{amssymb}
\usepackage{amsmath}
\usepackage{subfigure}
\usepackage{pstricks}
\usepackage{pgf}
\usepackage{tikz}
\usepackage{pst-plot}
\usepackage{algorithm}
\usepackage{algorithmic}

\input{content_arxiv/macros.tex}

\title{Sparse Modeling for Image and \\ Vision Processing}

\author{
Julien Mairal \\
Inria\footnote{LEAR team, laboratoire Jean Kuntzmann, CNRS, Univ. Grenoble Alpes, France.} \\
julien.mairal@inria.fr
\and
Francis Bach \\
Inria\footnote{SIERRA team, d\'epartement d'informatique de l'Ecole Normale Sup\'erieure, ENS/CNRS/Inria UMR 8548, France.} \\
francis.bach@inria.fr
\and
Jean Ponce \\
Ecole Normale Sup\'erieure\footnote{WILLOW team, d\'epartement d'informatique de l'Ecole Normale Sup\'erieure, ENS/CNRS/Inria UMR 8548, France.}  \\
jean.ponce@ens.fr
}

\begin{document}

\copyrightowner{J.~Mairal, F.~Bach and J.~Ponce}

\frontmatter  

\maketitle

\tableofcontents

\mainmatter

\begin{abstract}
\input{content_arxiv/abstract.tex}
\end{abstract}

\chapter{A Short Introduction to Parsimony}
   \input{content_arxiv/intro_intro.tex}
   \section{Early concepts of parsimony in statistics}\label{sec:early}
   \input{content_arxiv/intro_early.tex}

   \section{Wavelets in signal processing}\label{sec:wavelets}
   \input{content_arxiv/intro_signal.tex}

   \section{Modern parsimony: the $\ell_1$-norm and other variants}\label{sec:l1}
   \input{content_arxiv/intro_modern.tex}

   \section{Dictionary learning}\label{sec:sparsecoding}
   \input{content_arxiv/intro_dict.tex}

   \section{Compressed sensing and sparse recovery}\label{sec:compressedsensing}
   \input{content_arxiv/intro_cs.tex}

   \section{Theoretical results about dictionary learning}\label{subsec:theory}
   \input{content_arxiv/intro_thdict.tex}

\chapter{Discovering the Structure of Natural Images} \label{chapter:patches}
   \input{content_arxiv/natstat_intro.tex}

   \section{Pre-processing}\label{subsec:preprocess}
   \input{content_arxiv/natstat_preprocessing.tex}
   \section{Principal component analysis}\label{subsec:pca}
   \input{content_arxiv/natstat_pca.tex}
   \section{Clustering or vector quantization}\label{subsec:kmeans}
   \input{content_arxiv/natstat_clustering.tex}

   \section{Dictionary learning}\label{subsec:dict}
   \input{content_arxiv/natstat_dictlearning.tex}

   \section{Structured dictionary learning}\label{subsec:struct_dict}
   \input{content_arxiv/natstat_structured.tex}

   \section{Other matrix factorization methods}\label{subsec:matrix_other}
   \input{content_arxiv/natstat_other.tex}

   \section{Discussion}\label{subsec:discussion}
   \input{content_arxiv/natstat_discussion.tex}

\chapter{Sparse Models for Image Processing}\label{chapter:image}
   \input{content_arxiv/image_intro.tex}

   \section{Image denoising} \label{subsec:denoising}
   \input{content_arxiv/image_denoising.tex}

   \section{Image inpainting} \label{subsec:inpainting}
   \input{content_arxiv/image_inpainting.tex}
   \section{Image demosaicking}\label{subsec:demosaicking}
   \input{content_arxiv/image_demosaicking.tex}

   \section{Image up-scaling}\label{subsec:upscaling}
   \input{content_arxiv/image_superres.tex}
   \section{Inverting nonlinear local transformations}\label{subsec:invert}
   \input{content_arxiv/image_invert.tex}
   \section{Video processing}
   \input{content_arxiv/image_video.tex}

   \section{Face compression}
   \input{content_arxiv/image_face.tex}

   \section{Other patch modeling approaches}\label{sec:image_other}
   \input{content_arxiv/image_other.tex}

\chapter{Sparse Coding for Visual Recognition}\label{chapter:vision}
   \input{content_arxiv/visual_intro.tex}

   \section{A coding and pooling approach to image modeling}\label{sec:visual_features}
   \input{content_arxiv/visual_features.tex}

   \section{The botany of sparse feature coding}\label{sec:visual_botany}
   \input{content_arxiv/visual_botany.tex}
   \section{Face recognition}\label{sec:visual_faces}
   \input{content_arxiv/visual_faces.tex}

   \section{Patch classification and edge detection}\label{sec:visual_edges}
   \input{content_arxiv/visual_patches.tex}

   \section{Connections with neural networks}\label{subsec:backprop}
   \input{content_arxiv/visual_backprop.tex}
   \input{content_arxiv/visual_fast.tex}

   \section{Other applications} \label{sec:visual_other}
   \input{content_arxiv/visual_other.tex}

\chapter{Optimization Algorithms}\label{chapter:optim}
   \input{content_arxiv/optim_intro.tex}
   \section{Sparse reconstruction with the $\ell_0$-penalty}\label{sec:optiml0}
   \input{content_arxiv/optim_l0.tex}

   \section{Sparse reconstruction with the $\ell_1$-norm}\label{sec:optiml1}
   \input{content_arxiv/optim_l1.tex}

   \section{Iterative reweighted-$\ell_1$ methods} \label{sec:reweighted}
   \input{content_arxiv/optim_reweighted.tex}

   \section{Iterative reweighted-$\ell_2$ methods}
   \input{content_arxiv/optim_group.tex}

   \section{Optimization for dictionary learning}\label{sec:optimdict}
   \input{content_arxiv/optim_dict.tex}
   \section{Other optimization techniques}
   \input{content_arxiv/optim_other.tex}

\chapter{Conclusions}\label{chapter:ccl}
\input{content_arxiv/ccl.tex}

\begin{acknowledgements}
\addcontentsline{toc}{chapter}{Acknowledgments} 
This monograph was partially supported by the European Research Council (SIERRA
and VideoWorld projects), the project Gargantua funded by the program
Mastodons-CNRS, the Microsoft Research-Inria joint centre, and a French grant
from the Agence Nationale de la Recherche (MACARON project, ANR-14-CE23-0003-01).

The authors would like to thank Ori Bryt and Michael Elad for providing the
output of their face compression algorithm (see Figure~\ref{fig:faces}).
Julien Mairal would also like to thank Zaid Harchaoui, Florent Perronnin and
Adrien Gaidon for useful discussions leading to improvements of this monograph.

\end{acknowledgements}

\backmatter  

\bibliographystyle{plainnat}
\bibliography{review_sparse}

\end{document}

%% file: content_arxiv/macros.tex
\newcommand{\ie}{\emph{i.e.}}
\newcommand{\hatj}{{\hat{\jmath}}}
\newcommand{\hati}{{\hat{\imath}}}
\newcommand{\eg}{\emph{e.g.}}
\newcommand{\as}{~~\text{a.s.}}

\newcommand{\fro}{\text{F}}
\newcommand{\x}{{\mathbf x}}
\newcommand{\hatx}{{\mathbf {\hat x}}}
\newcommand{\s}{{\mathbf s}}
\renewcommand{\c}{{\mathbf c}}
\renewcommand{\b}{{\mathbf b}}
\newcommand{\e}{{\mathbf e}}
\newcommand{\z}{{\mathbf z}}
\renewcommand{\r}{{\mathbf r}}
\newcommand{\y}{{\mathbf y}}

\renewcommand{\d}{{\mathbf d}}
\newcommand{\q}{{\mathbf q}}
\newcommand{\X}{{\mathbf X}}
\newcommand{\Y}{{\mathbf Y}}
\newcommand{\C}{{\mathbf C}}
\newcommand{\Q}{{\mathbf Q}}
\newcommand{\D}{{\mathbf D}}
\newcommand{\R}{{\mathbf R}}
\newcommand{\W}{{\mathbf W}}
\newcommand{\w}{{\mathbf w}}
\newcommand{\CC}{{\mathcal C}}

\renewcommand{\SS}{{\mathbb S}}
\newcommand{\WW}{{\mathcal W}}
\renewcommand{\AA}{{\mathcal A}}
\newcommand{\PP}{{\mathcal P}}
\newcommand{\PPP}{{\mathbb P}}
\newcommand{\KK}{{\mathcal K}}
\newcommand{\NN}{{\mathcal N}}
\newcommand{\MM}{{\mathcal M}}
\newcommand{\GG}{{\mathcal G}}
\newcommand{\LL}{{\mathcal L}}
\newcommand{\U}{{\mathbf U}}
\renewcommand{\u}{{\mathbf u}}
\renewcommand{\v}{{\mathbf v}}
\newcommand{\V}{{\mathbf V}}
\newcommand{\st}{~~\text{s.t.}~~}
\renewcommand{\S}{{\mathbf S}}
\newcommand{\SSS}{{\mathcal S}}
\newcommand{\EE}{{\mathbb E}}
\newcommand{\A}{{\mathbf A}}
\newcommand{\B}{{\mathbf B}}
\newcommand{\I}{{\mathbf I}}
\newcommand{\M}{{\mathbf M}}
\newcommand{\alphab}{{\boldsymbol{\alpha}}}
\newcommand{\varepsilonb}{{\boldsymbol{\varepsilon}}}
\newcommand{\Gammab}{{\boldsymbol{\Gamma}}}
\newcommand{\gammab}{{\boldsymbol{\gamma}}}
\newcommand{\hatalphab}{{\boldsymbol{\hat{\alpha}}}}
\newcommand{\kappab}{{\boldsymbol{\kappa}}}
\newcommand{\etab}{{\boldsymbol{\eta}}}
\newcommand{\betab}{{\boldsymbol{\beta}}}
\newcommand{\hatbetab}{{\boldsymbol{\hat{\beta}}}}
\newcommand{\thetab}{{\boldsymbol{\theta}}}
\newcommand{\Sigmab}{{\boldsymbol{\Sigma}}}
\newcommand{\mub}{{\boldsymbol{\mu}}}
\newcommand{\Real}{{\mathbb R}}
\newcommand{\ones}{{\mathbf 1}}
\newcommand\defin{\triangleq}
\newcommand\diag{\text{diag}}
\newcommand\rank{\text{rank}}

\usetikzlibrary{arrows,shapes,snakes,automata,backgrounds,petri,positioning}
\tikzstyle{varempty}=[circle,thick,draw=black!75,fill=blue!0,minimum size=10mm]
\tikzstyle{varfull}=[circle,thick,draw=black!75,fill=blue!20,minimum size=10mm]
\tikzstyle{arrow}=[->,very thick]
\tikzstyle{group}=[rectangle,thick,draw=red,rounded corners]
\tikzstyle{varempty2}=[rectangle,thick,draw=black,minimum size=4mm]
\tikzstyle{varfull2}=[rectangle,thick,draw=black,fill=black!60,minimum size=4mm]
\tikzstyle{sep}=[rectangle,draw=black,line width=2pt,minimum width=3cm,scale=0.2]
\def\distnode{1.5cm}
\def\distnodeb{0.4cm}
\DeclareMathOperator{\sign}{sign}
\definecolor{darkgreen}{rgb}{0,0.3,0.0}
\DeclareMathOperator*{\argmin}{arg\,min}
\DeclareMathOperator*{\argmax}{arg\,max}
\DeclareMathOperator{\trace}{trace}
\DeclareMathOperator{\mat}{Mat}
\renewcommand\complement{\textrm{c}}

%% file: content_arxiv/abstract.tex
In recent years, a large amount of multi-disciplinary research has been
conducted on sparse models and their applications. In statistics and machine
learning, the sparsity principle is used to perform model selection---that is,
automatically selecting a simple model among a large collection of them. In
signal processing, sparse coding consists of representing data with linear
combinations of a few dictionary elements. Subsequently, the corresponding tools
have been widely adopted by several scientific communities such as
neuroscience, bioinformatics, or computer vision. The goal of
this monograph is to offer a self-contained view of sparse modeling for visual
recognition and image processing. More specifically, we focus on applications
where the dictionary is learned and adapted to data, yielding a compact
representation that has been successful in various contexts.

%% file: content_arxiv/intro_intro.tex
In its most general definition, the principle of sparsity, or parsimony, consists
of representing some phenomenon with as few variables as possible. It appears
to be central to many research fields and is often considered to be inspired
from an early doctrine formulated by the philosopher and theologian William of
Ockham in the 14th century, which essentially favors simple theories over more
complex ones. Of course, the link between Ockham and the tools presented in
this monograph is rather thin, and more modern views seem to
appear later in the beginning of the 20th century. Discussing the scientific
method, \citet{wrinch1921} introduce indeed a simplicity principle related to
parsimony as follows:
\begin{quote}
  \emph{The existence of simple laws is, then, apparently, to be regarded as a quality
of nature; and accordingly we may infer that it is justifiable to prefer a
simple law to a more complex one that fits our observations slightly better.}
\end{quote}
Remarkably,~\citet{wrinch1921} further discuss statistical modeling of physical
observations and relate the concept of ``simplicity'' to the number of learning
parameters; as a matter of fact, this concept is relatively close to the
contemporary view of parsimony.  

Subsequently, numerous tools have been developed by statisticians to build models of physical phenomena with good
predictive power. Models are usually learned from observed data, and their
generalization performance is evaluated on test data. Among a collection of
plausible models, the simplest one is often preferred, and the number of
underlying parameters is used as a criterion to perform model
selection~\citep[][]{mallows1964,mallows1966,akaike,hocking1976,barron,rissanen,schwarz,tibshirani}.

In signal processing, similar problems as in statistics arise, but a different
terminology is used. Observations, or data vectors, are called ``signals'', and
data modeling appears to be a crucial step for performing various
operations such as restoration, compression, or for solving inverse problems.
Here also, the sparsity principle plays an important role and has been
successful~\citep{mallat,pati1993,donoho1994,cotter,chen,donoho2006,candes2006}.
Each signal is approximated by a sparse linear combination of prototypes called
dictionary elements, resulting in simple and compact models. 

However, statistics and signal processing remain two distinct fields with
different objectives and methodology; specifically, signals often come from the
same data source, \eg, natural images, whereas problems considered in
statistics are unrelated to each other in general. 
Then, a long series of works has been devoted to
finding appropriate dictionaries for signal classes of interest, leading to various
sorts of
wavelets~\citep{freeman1991,simoncelli1992,donoho1998,candes2002,do2005,lepennec2005,mallat2008}.
Even though statistics and signal processing have devised most of the
methodology of sparse modeling, the parsimony principle was also discovered
independently in other fields. To some extent, it appears indeed in the work
of~\citet{markowitz} about portfolio selection in finance, and also in
geophysics~\citep{claerbout1973,taylor1979}.

In neuroscience,~\citet{field1996,olshausen1997} proposed a significantly
different approach to sparse modeling than previously established practices.
Whereas classical techniques in signal processing were using fixed
off-the-shelf dictionaries, the method of~\citet{field1996,olshausen1997}
consists of learning it from training data. In a pioneer exploratory experiment,
they demonstrated that dictionary learning could easily discover underlying
structures in natural image patches; later, their approach found numerous
applications in many fields, notably in image and audio
processing~\citep{lewicki2002,elad2006,mairal2009,yang2010} and computer
vision \citep{raina2007,yang2009,zeiler2011,mairal2012,song2012,castrodad2012,elhamifar2012,pokrass2013}.

The goal of this monograph is to present basic tools of sparse modeling
and their applications to visual recognition and image processing.
We aim at offering a self-contained view combining pluri-disciplinary 
methodology, practical advice, and a large review of the literature.
Most of the figures in the paper are produced with the software
SPAMS\footnote{available here
\url{http://spams-devel.gforge.inria.fr/}.},
and the corresponding Matlab code will be provided on the first author's webpage.

The monograph is organized as follows: the current introductory section is
divided into several parts providing a simple historical view of sparse
estimation.  In Section~\ref{sec:early}, we start with early concepts of parsimony
in statistics and information theory from the 70's and 80's. We present the use of
sparse estimation within the wavelet framework in Section~\ref{sec:wavelets},
which was essentially developed in the 90's. Section~\ref{sec:l1} introduces
the era of ``modern parsimony''---that is, the~$\ell_1$-norm and its variants,
which have been heavily used during the last two decades.
Section~\ref{sec:sparsecoding} is devoted to the dictionary learning
formulation originally introduced by~\citet{field1996,olshausen1997}, which is
a key component of most applications presented later in this monograph. In
Sections~\ref{sec:compressedsensing} and~\ref{subsec:theory}, we conclude our
introductory tour with some theoretical aspects, such as the concept of
``compressed sensing'' and sparse recovery that has attracted a large attention
in recent years.

With all these parsimonious tools in hand, we discuss  the use of sparse coding and related sparse
matrix factorization techniques for discovering the underlying structure of
natural image patches in Section~\ref{chapter:patches} . Even though the task here is subjective and exploratory,
it is the first successful instance of dictionary learning; the insight gained
from these early experiments forms the basis of concrete applications presented
in subsequent sections.

Section~\ref{chapter:image} covers numerous applications of sparse models of
natural image patches in image processing, such as image denoising,
super-resolution, inpainting, or demosaicking. This section is concluded with
other related patch-modeling approaches.

Section~\ref{chapter:vision} presents recent success of sparse models for
visual recognition, such as codebook learning of visual descriptors, face
recognition, or more low-level tasks such as edge detection and classification
of textures and digits. We conclude the section with other computer vision 
applications such as visual tracking and data visualization.

Section~\ref{chapter:optim} is devoted to optimization algorithms. It presents
in a concise way efficient algorithms for solving sparse decomposition and
dictionary learning problems.

We see our monograph as a good complement of other books and monographs about
sparse estimation, which offer different perspectives, such
as~\citet{mallat2008,elad2010} in signal and image processing,
or~\citet{bach2012} in optimization and machine learning.
We also put the emphasis on the structure of natural image patches learned with
dictionary learning, and thus present an alternative view to the book
of~\citet{hyvarinen2009}, which is focused on independent component analysis.

\paragraph{Notation.}
In this monograph, vectors are denoted by bold lower-case letters and matrices by upper-case ones.
For instance, we consider in the rest of this paragraph a vector $\x$
in~$\Real^n$ and a matrix~$\X$ in~$\Real^{m \times n}$. The columns 
of~$\X$ are represented by indexed vectors~$\x_1,\ldots,\x_n$ such that we
can write~$\X=[\x_1,\ldots,\x_n]$. The~$i$-th entry of~$\x$ is denoted
by~$\x[i]$, and the $i$-th entry of the~$j$-th column of~$\X$ is represented
by~$\X[i,j]$.  For any subset $g$ of~$\{1,\ldots,n\}$, we denote by~$\x[g]$ the
vector in~$\Real^{|g|}$ that records the entries of~$\x$ corresponding to indices in~$g$.
For~$q \geq 1$, we define the~$\ell_q$-norm of~$\x$ as $\|\x\|_q \defin \left( \sum_{i=1}^n |\x[i]|^q \right)^{1/q}$,
and the~$\ell_\infty$-norm as~$\|\x\|_\infty \defin \lim_{q \to +\infty} \|\x\|_q = \max_{i=1,\ldots,n} |\x[i]|$.
For~$q < 1$, we define the~$\ell_q$-penalty as $\|\x\|_q \defin \sum_{i=1}^n |\x[i]|^q$, which, with an abuse of
terminology, is often referred to as~$\ell_q$-norm. The~$\ell_0$-penalty simply counts the number of non-zero entries
in a vector:~$\|\x\|_0 \defin \sharp\{ i \st  \x[i] \neq 0 \}$.
For a matrix~$\X$, we define the Frobenius norm~$\|\X\|_\fro=\left(\sum_{i=1}^m \sum_{j=1}^n \X[i,j]^2 \right)^{1/2}$.
When dealing with a random variable $X$ defined on a probability space, we denote its expectation by~$\EE[X]$,
assuming that there is no measurability or integrability issue.

%% file: content_arxiv/intro_early.tex
A large number of statistical procedures can be formulated as maximum
likelihood estimation. Given a statistical model with parameters~$\thetab$, it
consists of minimizing with respect to~$\thetab$ an objective function
representing the negative log-likelihood of observed data.
Assuming for instance that we observe independent samples~$\z_1,\ldots,\z_n$ of
the (unknown) data distribution, we need to solve
\begin{equation}
   \min_{ \thetab \in \Real^p} \left[ \LL(\thetab) \defin - \sum_{i=1}^n \log
   P_{\thetab}(\z_i) \right],\label{eq:mle}
\end{equation}
where~$P$ is some probability distribution parameterized by~$\thetab$. 

Simple methods such as ordinary least squares can be written as~(\ref{eq:mle}).
Consider for instance data points~$\z_i$ that are pairs~$(y_i,\x_i)$, with~$y_i$
is an observation in~$\Real$ and~$\x_i$ is a vector in~$\Real^p$, and assume that there exists a linear
relation~$y_i = \x_i^\top \theta + \varepsilon_i$, where~$\varepsilon_i$ is an
approximation error for observation~$i$.  Under a model where
the~$\varepsilon_i$'s are independent and identically normally distributed with
zero-mean, Eq.~(\ref{eq:mle}) is equivalent to a least square problem:\footnote{Note that the Gaussian noise assumption is not necessary to justify the ordinary least square formulation. It is only sufficient to interpret it as maximum likelihood estimation. In fact, as long as the conditional expectation~$\EE[y|\x]$ is linear, the ordinary least square estimator is statistically consistent under mild assumptions.}
\begin{displaymath}
   \min_{\thetab \in \Real^p} \sum_{i=1}^n \frac{1}{2}\left(y_i - \x_i^\top \thetab\right)^2.
\end{displaymath}
To prevent overfitting and to improve the interpretability of the learned
model, it was suggested in early work that a solution involving only a few
model variables could be more appropriate than an exact solution
of~(\ref{eq:mle}); in other words, a sparse solution involving only---let
us say---$k$ variables might be desirable in some situations.  Unfortunately,
such a strategy yields two difficulties: first, it is not clear a priori how to
choose $k$; second, finding the best subset of~$k$ variables is NP-hard in
general~\citep{natarajan}. The first issue was addressed with several
criterions for controlling the trade-off between the sparsity of the
solution~$\thetab$ and the adequacy of the fit to training data.  For the
second issue, approximate computational techniques have been proposed.

\paragraph{Mallows's $C_p$, AIC, and BIC.} For the ordinary least squares problem,
\citet{mallows1964,mallows1966} introduced the~$C_p$-statistics, later
generalized by~\citet{akaike} with the Akaike information
criterion (AIC), and then by~\citet{schwarz} with the Bayesian information
criterion (BIC). Using $C_p$, AIC, or BIC is equivalent to solving the
penalized $\ell_0$-maximum likelihood estimation problem
\begin{equation}
   \min_{ \thetab \in \Real^p} \LL(\thetab) + \lambda\|\thetab\|_0, \label{eq:l0}
\end{equation}
where~$\lambda$ depends on the chosen criterion~\citep[see][]{hastie2009},
and~$\|\thetab\|_0$ is the~$\ell_0$-penalty. Similar formulations have also
been derived by using the minimum description length (MDL) principle for model
selection~\citep{rissanen,barron}. As shown by~\citet{natarajan}, the
problem~(\ref{eq:l0}) is NP-hard, and approximate algorithms are necessary
unless~$p$ is very small, \eg, $p < 30$.

\paragraph{Forward selection and best subset selection for least squares.}
To obtain an approximate solution of~(\ref{eq:l0}), a classical approach is the
forward selection technique, which is a greedy algorithm that solves a sequence
of maximum likelihood estimation problems computed on a subset of variables.
After every iteration, a new variable is added to the subset according to the
chosen sparsity criterion in a greedy manner. Some variants allow backward
steps---that is, a variable can possibly exit the active subset after an
iteration. The algorithm is presented in more details in
Section~\ref{sec:optiml0} and seems to be due
to~\citet{efroymson1960}, according to~\citet{hocking1976}. Other approaches
considered in the 70's include also the~\emph{leaps and bounds} technique
of~\citet{furnival1974}, a branch-and-bound algorithm providing the exact
solution of~(\ref{eq:l0}) with exponential worst-case complexity.

%% file: content_arxiv/intro_signal.tex
In signal processing, similar problems as in statistics have been studied in
the context of wavelets. In a nutshell, a wavelet basis represents a set of
functions~$\phi_1,\phi_2,\ldots$ that are essentially dilated and shifted
versions of each other. Unlike Fourier basis, wavelets have the interesting
properties to be localized both in the space and frequency domains, and to be
suitable to multi-resolution analysis of
signals~\citep{mallat1989}. 

The concept of parsimony is central to wavelets. When a signal~$f$ is
``smooth'' in a particular sense~\citep[see][]{mallat2008}, it can be well
approximated by a linear combination of a few wavelets. Specifically, $f$ is
close to an expansion $\sum_{i} \alpha_i \phi_i$ where only a few coefficients
$\alpha_i$ are non-zero, and the resulting compact representation has effective
applications in estimation and compression. The wavelet theory is well
developed for continuous signals, \eg, $f$ is chosen in the Hilbert
space~$L^2(\Real)$, but also for discrete signals~$f$ in~$\Real^m$, making it
suitable to modern digital image processing.

Since the first wavelet was introduced by~\citet{haar1910}, much research has
been devoted to designing a wavelet set that is adapted to particular
signals such as natural images. After a long quest for finding good
orthogonal basis such as the one proposed by \citet{daubechies1988}, a series
of works has focused on wavelet sets whose elements are not linearly
independent. It resulted a large number of variants, such as steerable
wavelets~\citep{simoncelli1992}, curvelets~\citep{candes2002},
contourlets~\citep{do2005}, or bandlets~\citep{lepennec2005}.
For the purpose of our monograph, one concept related to sparse estimation is
particularly important; it is called \emph{wavelet thresholding}.

\paragraph{Sparse estimation and wavelet thresholding.}
Let us consider a discrete signal represented by a vector~$\x$ in~$\Real^p$ and
an orthogonal wavelet basis set~$\D=[\d_1,\ldots,\d_p]$---that is, satisfying $\D^\top\D=\I$ where~$\I$ is the identity matrix. Approximating~$\x$
by a sparse linear combination of wavelet elements can be formulated as
finding a sparse vector~$\alphab$ in~$\Real^p$, say with~$k$ non-zero coefficients, that minimizes
\begin{equation}
   \min_{\alphab \in \Real^p} \frac{1}{2}\|\x-\D\alphab\|_2^2  \st \|\alphab\|_0 \leq k. \label{eq:hardthresholding}
\end{equation}
The sparse decomposition problem~(\ref{eq:hardthresholding}) is
an instance of the best subset selection formulation presented in
Section~\ref{sec:early} where~$\alphab$ represents model parameters,
demonstrating that similar topics arise in statistics and signal processing.
However, whereas~(\ref{eq:hardthresholding}) is NP-hard for general
matrices~$\D$~\citep{natarajan}, we have assumed~$\D$ to be orthogonal in the
context of wavelets. As such,~(\ref{eq:hardthresholding}) is equivalent to
\begin{displaymath}
   \min_{\alphab \in \Real^p} \frac{1}{2}\left\|\D^\top\x-\alphab\right\|_2^2  \st \|\alphab\|_0 \leq k,
\end{displaymath}
and admits a closed form. Let us indeed define the vector $\betab \defin
\D^\top\x$ in~$\Real^p$, corresponding to the exact non-sparse decomposition
of~$\x$ onto~$\D$---that is, we have~$\x = \D\betab$ since~$\D$ is orthogonal.
To obtain the best~$k$-sparse approximation, we denote by~$\mu$ the $k$-th largest value among the set $\{|\betab[1]|,\ldots,|\betab[p]|\}$, 
and the solution~$\alphab^\text{ht}$ of~(\ref{eq:hardthresholding}) is obtained by applying to~$\betab$ an operator called ``hard-thresholding'' and defined as
\begin{equation}
   \alphab^\text{ht}[i] =  {\ones}_{|\betab[i]|\geq \mu}\betab[i] = \left\{ \begin{array}{lr}   
      \betab[i]  & \text{if}~~ |\betab[i]| \geq \mu, \\ 
      0  & \text{otherwise},
      \end{array}  \right. \label{eq:hardthrs}
\end{equation}
where~${\ones}_{|\betab[i]|\geq \mu}$ is the indicator function, which is equal to~$1$ if $|\betab[i]|\geq \mu$ and $0$ otherwise.
In other words, the hard-thresholding operator simply sets to zero coefficients
from~$\betab$ whose magnitude is below the threshold~$\mu$.  The corresponding
procedure, called ``wavelet thresholding'', is simple and effective for image
denoising, even though it does not perform as well as recent state-of-the-art 
techniques presented in Section~\ref{chapter:image}. When an image~$\x$ is
noisy, \eg, corrupted by white Gaussian noise, and~$\mu$ is well chosen, the
estimate~$\D\alphab^\text{ht}$ is a good estimate of the clean original image.
The terminology ``hard'' is defined in contrast to an important variant called
the ``soft-thresholding operator'', which was introduced by~\citet{donoho1994} in the context of wavelets:\footnote{Note that the soft-thresholding operator appears in fact earlier in the statistics literature~\citep[see][]{efron1971,bickel1984}, but it was used there for a different purpose.}
\begin{equation}
   \alphab^{\text{st}}[i] \defin  \text{sign}(\betab[i])\max(|\betab[i]|-\lambda,0) = \left\{ \begin{array}{ll}   
      \betab[i]-\lambda  & \text{if}~~ \betab[i] \geq \lambda, \\ 
      \betab[i]+\lambda  & \text{if}~~ \betab[i] \leq -\lambda, \\ 
      0  & \text{otherwise},
      \end{array}  \right. \label{eq:softth}
\end{equation}
where~$\lambda$ is a parameter playing the same role as~$\mu$ in~(\ref{eq:hardthrs}).
Not only does the operator set small coefficients of~$\betab$ to zero,
but it also reduces the magnitude of the non-zero ones.  Both operators are
illustrated and compared to each other in Figure~\ref{fig:soft}. Interestingly, whereas~$\alphab^\text{ht}$ is the
solution of~(\ref{eq:hardthresholding}) when~$\mu$ corresponds to the entry
of~$\betab=\D^\top \x$ with $k$-th largest magnitude, $\alphab^{\text{st}}$ is in fact the solution
of the following sparse reconstruction problem with the orthogonal matrix~$\D$:
\begin{equation}
   \min_{\alphab \in \Real^p} \frac{1}{2}\|\x-\D\alphab\|_2^2  + \lambda\|\alphab\|_1. \label{eq:softhresholding}
\end{equation}
This formulation will be the topic of the next section for general non-orthogonal
matrices.  Similar to statistics where choosing the parameter~$k$ of the
best subset selection was difficult, automatically selecting the best
thresholds~$\mu$ or~$\lambda$ has been a major research
topic~\citep[see, \eg][]{donoho1994,donoho,chang2000,chang2000b}.

\begin{figure}
\centering
\subfigure[Soft-thresholding operator, \newline $\alpha^{\text{st}}=\text{sign}(\beta)\max(|\beta|-\lambda,0)$.\label{subfig:soft}]{ 
   \includegraphics[page=1]{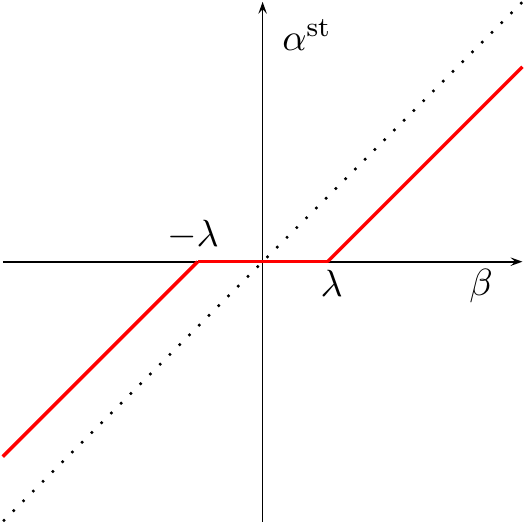}
} \hfill
\subfigure[Hard-thresholding operator \newline $\alpha^\text{ht}={\ones}_{|\beta|\geq \mu}\beta$.\label{subfig:hard}]{ 
   \includegraphics[page=2]{review_sparse_arxiv2-pics.pdf}
} 
\caption{Soft- and hard-thresholding operators, which are commonly used for signal estimation with orthogonal wavelet basis.} \label{fig:soft}
\end{figure}

\paragraph{Structured group thresholding.}
Wavelets coefficients have a particular structure since the basis
elements~$\d_i$ are dilated and shifted versions of each other.  It
is for instance possible to define neighborhood relationships for wavelets
whose spatial supports are close to each other, or hierarchical relationships
between wavelets with same or similar localization but with different scales.
For one-dimensional signals, we present in Figure~\ref{fig:treewavelets} a
typical organization of wavelet coefficients on a tree with arrows representing
such relations.
For two-dimensional images, the structure is slightly more involved and the
coefficients are usually organized as a collection of quadtrees~\citep[see][for
more details]{mallat2008}; we present such a configuration in
Figure~\ref{fig:waveletsfig}.

\begin{figure}
\begin{center}
 \begin{tikzpicture}[node distance=\distnode,>=stealth',bend angle=45,auto]
 \begin{scope}
     \node [varfull]    (v1)                       {$\alpha_1$};
     \node [varfull]    (v2)  [below of=v1,xshift=-3cm] {$\alpha_2$};
     \node [varfull]    (v3)  [below of=v1,xshift=3cm] {$\alpha_3$};
     \node [varempty]    (v4)  [below of=v2,xshift=-1.5cm] {$\alpha_4$};
     \node [varfull]    (v5)  [below of=v2,xshift=1.5cm] {$\alpha_5$};
     \node [varfull]    (v6)  [below of=v3,xshift=-1.5cm] {$\alpha_6$};
     \node [varempty]    (v7)  [below of=v3,xshift=1.5cm] {$\alpha_7$};
     \node [varempty]    (v8)  [below of=v4,xshift=-0.7cm] {$\alpha_8$};
     \node [varempty]    (v9)  [below of=v4,xshift=0.7cm] {$\alpha_9$};
     \node [varfull]    (v10)  [below of=v5,xshift=-0.7cm] {$\alpha_{10}$};
     \node [varfull]    (v11)  [below of=v5,xshift=0.7cm] {$\alpha_{11}$};
     \node [varfull]    (v12)  [below of=v6,xshift=-0.7cm] {$\alpha_{12}$};
     \node [varempty]    (v13)  [below of=v6,xshift=0.7cm] {$\alpha_{13}$};
     \node [varempty]    (v14)  [below of=v7,xshift=-0.7cm] {$\alpha_{14}$};
     \node [varempty]    (v15)  [below of=v7,xshift=0.7cm] {$\alpha_{15}$};
     \draw [arrow] (v1) -- (v2);
     \draw [arrow] (v1) -- (v3);
     \draw [arrow] (v2) -- (v4);
     \draw [arrow] (v2) -- (v5);
     \draw [arrow] (v3) -- (v6);
     \draw [arrow] (v3) -- (v7);
     \draw [arrow] (v4) -- (v8);
     \draw [arrow] (v4) -- (v9);
     \draw [arrow] (v5) -- (v10);
     \draw [arrow] (v5) -- (v11);
     \draw [arrow] (v6) -- (v12);
     \draw [arrow] (v6) -- (v13);
     \draw [arrow] (v7) -- (v14);
     \draw [arrow] (v7) -- (v15);
  \end{scope}
  \end{tikzpicture}
  \end{center}
  \caption{Illustration of a wavelet tree with four scales for one-dimensional
     signals. Nodes represent wavelet coefficients and their depth in the tree
     correspond to the scale parameter of the wavelet.
     We also illustrate the zero-tree coding scheme~\citep{shapiro1993} in this figure.
     Empty nodes correspond to zero coefficient: according to the zero-tree
     coding scheme, their descendants in the tree are also zero.
  }\label{fig:treewavelets}
\end{figure}
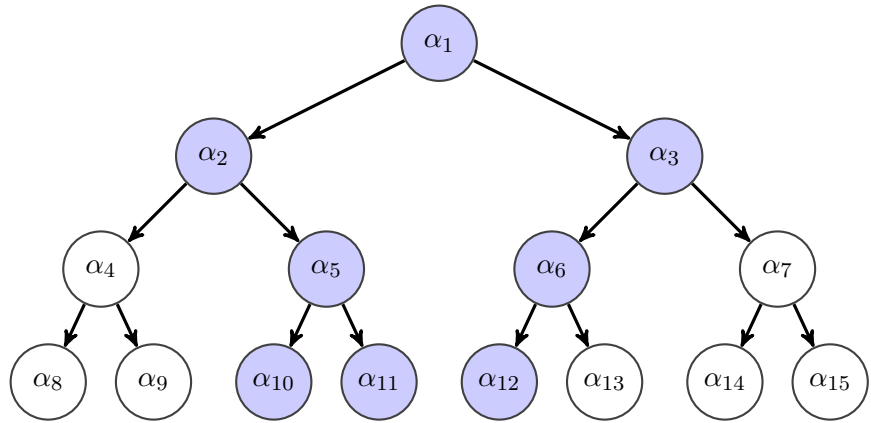

\begin{figure}
   \centering
   \includegraphics[width=0.9\linewidth]{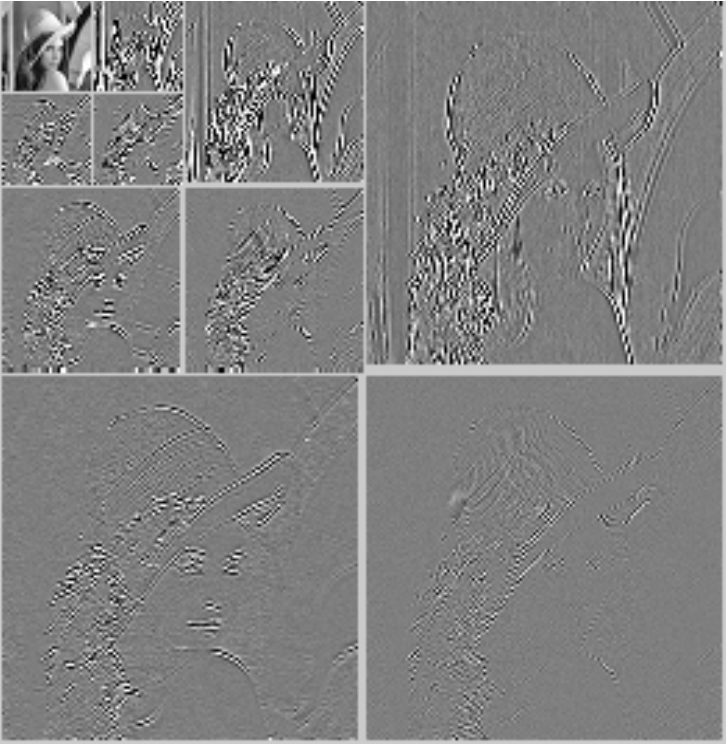}
   \caption{Wavelet coefficients displayed for the image~\textsf{lena} using
      the orthogonal basis of~\citet{daubechies1988}. A few coefficients
      representing a low-resolution version of the image are displayed on the top-left
      corner. Wavelets corresponding to this low-resolution image
      are obtained by filtering the original image with shifted versions of a low-pass filter called ``scaling
      function'' or ``father wavelet''. The rest of the coefficients are
      organized into three quadtrees (on the right, on the left, and on the
      diagonal). Each quadtree is obtained by filtering the original image
      with a wavelet at three different scales and at different positions.
      The value zero is represented by the grey color; negative values appear
      in black, and positive values in white.  The wavelet decomposition and
      this figure have been produced with the software package
      \textsf{matlabPyrTools} developed by Eero Simoncelli and available here:
   \url{http://www.cns.nyu.edu/~lcv/software.php}.}\label{fig:waveletsfig}
\end{figure}

A natural idea has inspired the recent concept of~\emph{group sparsity}
that will be presented in the next section; it consists in exploiting the wavelet
structure to improve thresholding estimators. Specifically, it is possible to
use neighborhood relations between wavelet basis elements to define groups of
coefficients that form a partition~$\GG$ of~$\{1,\ldots,p\}$, and use a
group-thresholding operator~\citep{hall1999,cai1999} defined for every
group~$g$ in~$\GG$ as
\begin{equation}
   \alphab^{\text{gt}}[g] \defin  \left\{ \begin{array}{ll}   
         \left(1-\frac{\lambda}{\left\|\betab[g]\right\|_2}\right)\betab[g]  & \text{if}~~ \left\|\betab[g]\right\|_2 \geq \lambda, \\ 
      0  & \text{otherwise},
      \end{array}  \right. \label{eq:groupthresholding}
\end{equation}
where~$\betab[g]$ is the vector of size~$|g|$ recording the entries of~$\betab$ whose indices are in~$g$.
By using such an estimator, groups of neighbor coefficients are set to zero
together when their joint~$\ell_2$-norm falls below the threshold~$\lambda$.
Interestingly, even though the next interpretation does not appear in early
work about group-thresholding~\citep{hall1999,cai1999}, it is possible to
view~$\alphab^\text{gt}$ with~$\betab=\D^\top \x$ as the solution of the following penalized problem
\begin{equation}
   \min_{ \alphab \in \Real^p} \frac{1}{2}\|\x-\D\alphab\|_2^2 + \lambda \sum_{g \in \GG} \|\alphab[g]\|_2, \label{eq:grouplassoa}
\end{equation}
where the closed-form solution~(\ref{eq:groupthresholding}) holds because~$\D$
is orthogonal~\citep[see][]{bach2012}. Such a formulation will be studied in the
next section for general matrices.

Finally, other ideas for exploiting both structure and wavelet parsimony have been proposed. One is a
coding scheme called ``zero-tree'' wavelet coding~\citep{shapiro1993}, which  uses
the tree structure of wavelets to force all descendants of zero
coefficients to be zero as well. Equivalently, a coefficient can be non-zero
only if its parent in the tree is non-zero, as illustrated in
Figure~\ref{fig:treewavelets}. This idea has been revisited later in a more general
context by~\citet{zhao}.  Other complex models have been used as well for
modeling interactions between coefficients: we can mention the application of hidden Markov models (HMM)
to wavelets by~\citet{crouse1998} and the Gaussian scale mixture model
of~\citet{portilla2003}.

%% file: content_arxiv/intro_modern.tex
The era of ``modern'' parsimony corresponds probably to the use of convex
optimization techniques for solving feature selection or sparse decomposition
problems.  Even though the~$\ell_1$-norm was introduced for that purpose in
geophysics~\citep{claerbout1973,taylor1979}, it was popularized in statistics
with the Lasso estimator of~\citet{tibshirani} and independently in signal
processing with the basis pursuit formulation of~\citet{chen}. Given
observations~$\x$ in~$\Real^n$ and a matrix of predictors~$\D$ in~$\Real^{n
\times p}$, the Lasso consists of learning a linear model~$\x \approx
\D\alphab$ by solving the following quadratic program:
\begin{equation}
   \min_{\alphab \in \Real^p} \frac{1}{2}\|\x-\D\alphab\|_2^2  \st \|\alphab\|_1 \leq \mu. \label{eq:lasso}
\end{equation}
As detailed in the sequel, the~$\ell_1$-norm encourages the solution~$\alphab$ to
be sparse and the parameter~$\mu$ is used to control the trade-off between data
fitting and the sparsity of~$\alphab$. In practice, reducing the value of~$\mu$
leads indeed to sparser solution in general, \ie, with more zeroes, even
though there is no formal relation between the sparsity of~$\alphab$ and
its~$\ell_1$-norm for general matrices~$\D$.

The basis pursuit denoising formulation of~\citet{chen} is relatively similar
but the~$\ell_1$-norm is used as a penalty instead of a constraint.
It can be written as
\begin{equation}
   \min_{\alphab \in \Real^p} \frac{1}{2}\|\x-\D\alphab\|_2^2  +\lambda\|\alphab\|_1, \label{eq:lasso2}
\end{equation}
which is essentially equivalent to~(\ref{eq:lasso}) from a convex optimization
perspective, and in fact~(\ref{eq:lasso2}) is also often called ``Lasso'' in
the literature.
Given some data~$\x$, matrix~$\D$, and parameter~$\mu > 0$, we indeed know
from Lagrange multiplier theory~\citep[see, \eg,][]{borwein,boyd.convex}
that for all solution~$\alphab^\star$ of~(\ref{eq:lasso}), there exists a
parameter~$\lambda \geq 0$ such that~$\alphab^\star$ is also a solution
of~(\ref{eq:lasso2}).  We note, however, that there is no direct mapping
between~$\lambda$ and~$\mu$,
and thus the choice of formulation~(\ref{eq:lasso}) or~(\ref{eq:lasso2}) should
be made according to how easy it is to select the parameters~$\lambda$ or~$\mu$.
For instance, one may prefer~(\ref{eq:lasso}) when a priori information
about the~$\ell_1$-norm of the solution is available.
In Figure~\ref{fig:path}, we illustrate the effect of changing the value of the
regularization parameter~$\lambda$ on the solution of~(\ref{eq:lasso2}) for
two datasets. When~$\lambda=0$, the solution is dense; in general,
increasing~$\lambda$ sets more and more variables to zero. However, the
relation between~$\lambda$ and the sparsity of the solution is not exactly
monotonic. In a few cases, increasing~$\lambda$ yields a denser solution.

Another ``equivalent'' formulation consists of finding a sparse decomposition
under a reconstruction constraint:
\begin{equation}
   \min_{\alphab \in \Real^p} \|\alphab\|_1 \st \|\x-\D\alphab\|_2^2   \leq \varepsilon. \label{eq:lasso4}
\end{equation} This formulation can be useful when we have a priori knowledge about
the noise level and the parameter~$\varepsilon$ is easy to choose.  The link
between~(\ref{eq:lasso2}) and~(\ref{eq:lasso4}) is similar to the link
between~(\ref{eq:lasso2}) and~(\ref{eq:lasso}).

\begin{figure}
\centering
\subfigure[Path for dataset 1]{
\includegraphics[width=0.47\linewidth]{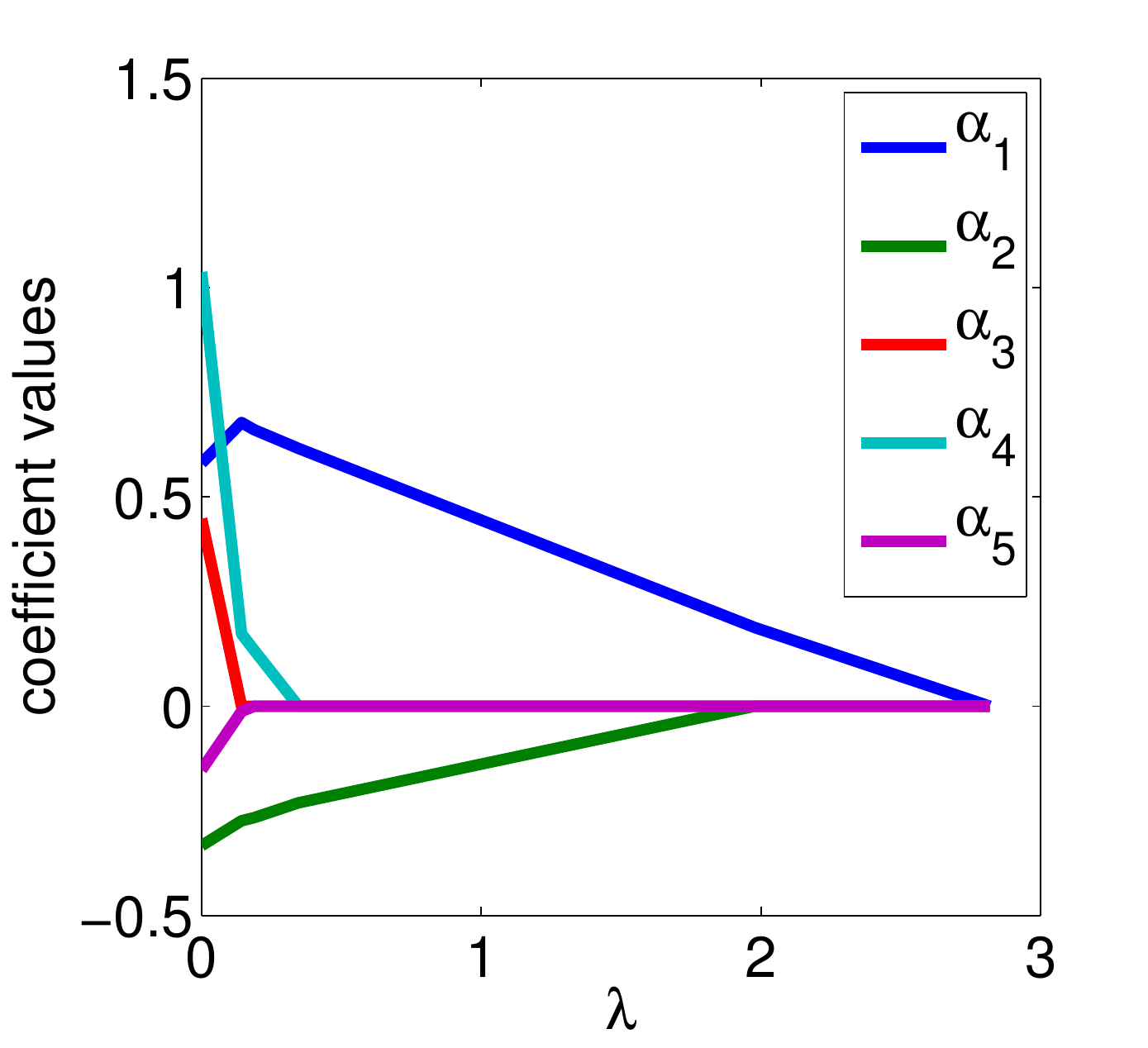}} \hfill
\subfigure[Path for dataset 2]{
\includegraphics[width=0.47\linewidth]{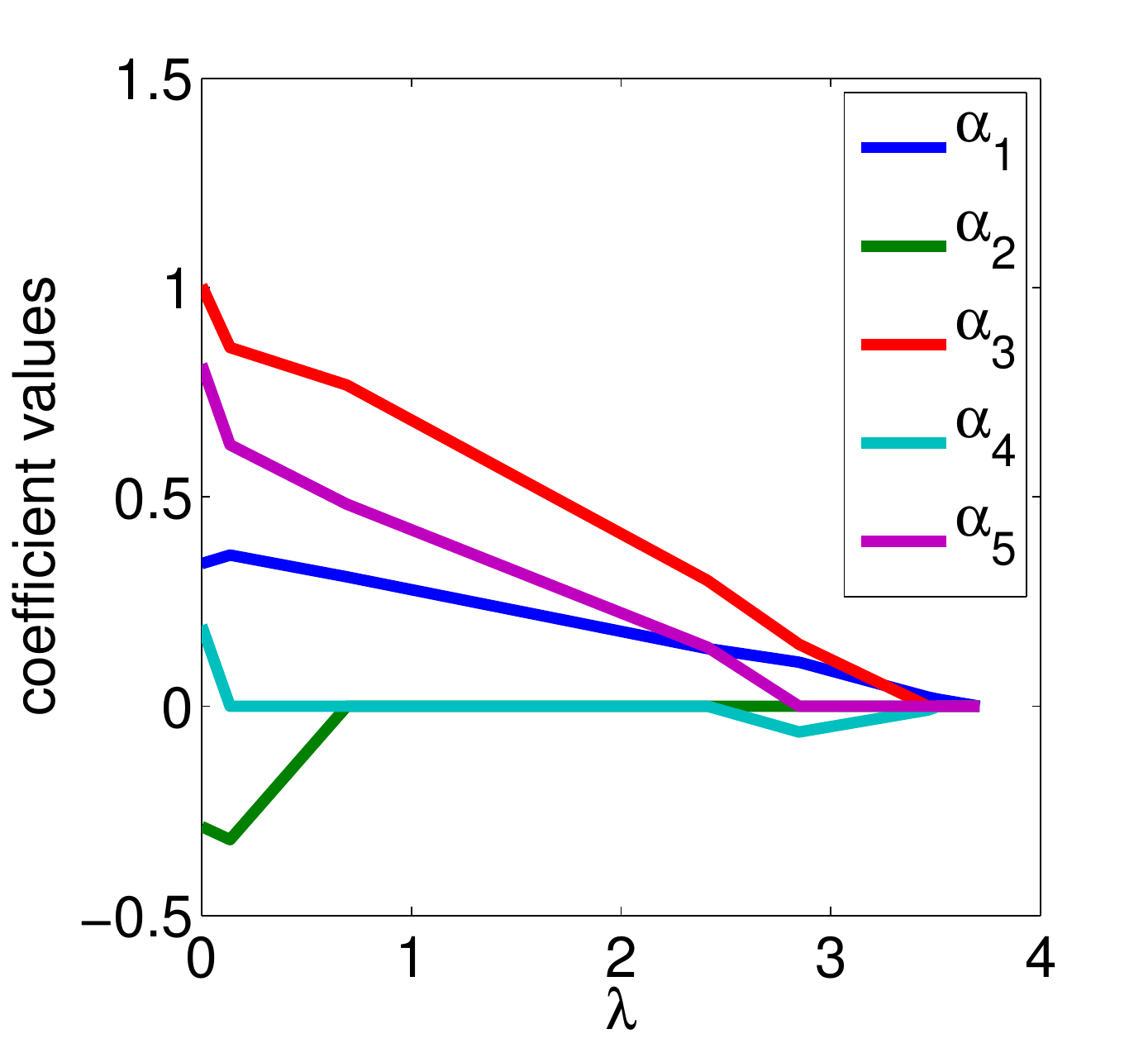}}
\caption{Two examples of regularization paths for the Lasso/Basis Pursuit. The curves represent the values of the $p=5$ entries of the solutions of~(\ref{eq:lasso2}) when varying the parameter~$\lambda$ for two datasets. On the left, the relation between~$\lambda$ and the sparsity of the solution is monotonic; On the right, this is not the case. Note that the paths are piecewise linear, see Section~\ref{sec:optiml1} for more details.} \label{fig:path}
\end{figure}

For noiseless problems, \citet{chen} have also introduced a formulation simply
called ``basis pursuit'' (without the terminology ``denoising''), defined as
\begin{equation}
   \min_{\alphab \in \Real^p} \|\alphab\|_1 \st \x=\D\alphab, \label{eq:lasso3}
\end{equation}
which is related to~(\ref{eq:lasso2}) in the sense that the set of solutions
of~(\ref{eq:lasso2}) converges to the solutions of~(\ref{eq:lasso3})
when~$\lambda$ converges to~$0^+$, whenever the linear system $\x=\D\alphab$ is
feasible.  These four formulations~(\ref{eq:lasso}-\ref{eq:lasso3}) have
gained a large success beyond the statistics and signal processing communities.
More generally, the~$\ell_1$-norm has been used as a regularization function
beyond the least-square context, leading to problems of the form
\begin{equation}
   \min_{\alphab \in \Real^p} f(\alphab) + \lambda\|\alphab\|_1, \label{eq:generall1}
\end{equation}
where~$f: \Real^p \to \Real$ is a loss function.  In the rest of this section,
we will present several variants of the~$\ell_1$-norm, but before that, we will
try to understand why such a penalty is appropriate for sparse estimation.

\paragraph{Why does the~$\ell_1$-norm induce sparsity?}
Even though we have claimed that there is no rigorous relation between the
sparsity of~$\alphab$ and its~$\ell_1$-norm in general, intuition about the
sparsity-inducing effect of the~$\ell_1$-norm may be obtained from
several viewpoints.

\subparagraph{Analytical point of view.}
In the previous section about wavelets, we have seen that when~$\D$ is orthogonal, the
$\ell_1$-decomposition problem~(\ref{eq:lasso2}) admits an analytic closed
form solution~(\ref{eq:softth}) obtained by soft-thresholding. As a result, whenever the
magnitude of the inner product $\d_i^\top \x$ is smaller than~$\lambda$ for
an index~$i$, the corresponding variable~$\alphab^\star[i]$ is equal to zero.
Thus, the number of zeroes of the solution~$\alphab^\star$ monotonically
increases with~$\lambda$.  

For non-orthogonal matrices~$\D$, such a monotonic relation does not formally
hold anymore; in practice, the sparsity-inducing property of
the~$\ell_1$-penalty remains effective, as illustrated in
Figure~\ref{fig:path}. Some intuition about this fact can be gained by studying optimality
conditions for the general~$\ell_1$-regularized problem~(\ref{eq:generall1})
where~$f$ is a differentiable function. The following lemma details these conditions.
\begin{lemma}[Optimality conditions for~$\ell_1$-regularized problems]~\label{lemma:opt}\\
A vector $\alphab^\star$ in~$\Real^p$ is a solution of~(\ref{eq:generall1}) if and only if
\begin{equation}
\forall i=1,\ldots,p~   \left\{ \begin{array}{rcll}
         -\nabla f(\alphab^\star)[i] &= &\lambda \sign(\alphab^\star[i]) & \text{if}~\alphab^\star[i] \neq 0, \\
         |\nabla f(\alphab^\star)[i]| &\leq&  \lambda  & \text{otherwise}. \\
   \end{array} \right. \label{eq:kktlasso}
\end{equation}
\end{lemma}
\begin{proof}
A proof using the classical concept of subdifferential from convex optimization
can be found in~\citep[see,\eg,][]{bach2012}. Here, we provide instead an elementary proof
using the simpler concept of directional derivative for nonsmooth functions,
defined as, when the limit exists,
\begin{displaymath}
   \nabla g(\alphab,\kappab) \defin \lim_{t \to 0^+} \frac{g(\alphab + t\kappab) - g(\alphab)}{t},
\end{displaymath}
for a function~$g: \Real^p \to \Real$ at a point~$\alphab$ in~$\Real^p$ and a
direction~$\kappab$ in~$\Real^p$. For convex functions~$g$, directional
derivatives always exist and a classical optimality condition
for~$\alphab^\star$ to be a minimum of~$g$ is to have $\nabla
g(\alphab^\star,\kappab)$ non-negative for all directions~$\kappab$~\citep{borwein}.
Intuitively, this means that one cannot find any direction~$\kappab$ such that
an infinitesimal move along~$\kappab$ from~$\alphab^\star$ decreases the
value of the objective. When~$g$ is differentiable, the condition is equivalent
to the classical optimality condition~$\nabla g(\alphab^\star)=0$.

We can now apply the directional derivative condition to the function $g: \alphab \mapsto f(\alphab) + \lambda\|\alphab\|_1$, which is equivalent to
\begin{equation}
   \forall \kappab \in \Real^p,~~~  \nabla f(\alphab^\star)^\top \kappab  + \lambda \sum_{i=1}^p \left\{ \begin{array}{lr}
         \sign(\alphab^\star[i])\kappab[i] & \text{if}~\alphab^\star[i] \neq 0, \\
         |\kappab[i]|                       & \text{otherwise}
   \end{array} \right\} \geq 0. \label{eq:kktlasso2}
\end{equation}
It is then easy to show that~(\ref{eq:kktlasso2}) holds for all~$\kappab$ if
and only the inequality holds for the specific values $\kappab=\e_i$
and~$\kappab=-\e_i$ for all~$i$, where $\e_i$ is the vector in~$\Real^p$ with
zeroes everywhere except for the~$i$-th entry that is equal to one.
This immediately provides an equivalence between~(\ref{eq:kktlasso2}) and~(\ref{eq:kktlasso}).
\end{proof}
Lemma~\ref{lemma:opt} is interesting from a computational point of view (see
Section~\ref{sec:optiml1}), but it also tells us that when~$\lambda \geq
\|\nabla f(0)\|_\infty$, the conditions~(\ref{eq:kktlasso}) are satisfied
for~$\alphab^\star=0$, the sparsest solution possible.

\subparagraph{Physical point of view.}
In image processing or computer vision, the word ``energy'' often
denotes the objective function of a minimization problem; it is indeed common in
physics to have complex systems that stabilize at a configuration
of minimum potential energy. The negative of the energy's gradient represents a force,
a terminology we will borrow in this paragraph. Consider for instance a
one-dimensional $\ell_1$-regularized estimation problem
\begin{equation}
   \min_{\alpha \in \Real} \frac{1}{2}(\beta-\alpha)^2 + \lambda|\alpha|,\label{eq:forcel1}
\end{equation}
where~$\beta$ is a positive constant.  Whenever~$\alpha$ is non-zero, the $\ell_1$-penalty is
differentiable with derivative~$\lambda\sign(\alpha)$.  When interpreting this
objective as an energy minimization problem, the~$\ell_1$-penalty can be seen
as applying \emph{a force driving~$\alpha$ towards the origin with constant
intensity~$\lambda$}.  Consider now instead the squared $\ell_2$-penalty, also
called regularization of \citet{tikhonov1963}, or ridge regression
regularization~\citep{hoerl1970}:
\begin{equation}
   \min_{\alpha \in \Real} \frac{1}{2}(\beta-\alpha)^2 + \frac{\lambda}{2}\alpha^2. \label{eq:forcel2}
\end{equation}
The derivative of the quadratic energy $({\lambda}/{2})\alpha^2$ is 
$\lambda \alpha$. It can be interpreted as \emph{a force that also points
to the origin but with linear intensity $\lambda|\alpha|$}.  Therefore, the force
corresponding to the ridge regularization can be arbitrarily strong
when~$\alpha$ is large, but if fades away when~$\alpha$ gets close to zero. 
As a result, the squared $\ell_2$-regularization does not have a
sparsity-inducing effect.  From an analytical point of view, we have
seen  that the solution of~(\ref{eq:forcel1}) is zero when~$|\beta|$ is smaller
than~$\lambda$.  In contrast, the solution of~(\ref{eq:forcel2}) admits a
closed form~$\alpha^\star = \beta/(1+\lambda)$. And thus, regardless of the
parameter~$\lambda$, the solution is never zero.

We present a physical example illustrating this phenomenon in
Figure~\ref{fig:ressort}. We use springs whose potential energy
is known to be quadratic, and objects with a gravitational potential energy
that is approximately linear on the Earth's surface.

\begin{figure}[hbtp]
\centering
\subfigure[Small regularization]{
   \begin{tikzpicture}[scale=0.75]

      \draw[line width=0.3mm,arrows=<->] (2.5,0) -- (2.5,6);
      \draw[snake=zigzag,segment amplitude=2mm,segment length=2mm] (0,6) -- (0,4);
      \draw[snake=zigzag,segment amplitude=2mm,segment length=4mm] (0,0) -- (0,4);
      \fill[fill=red] (0,6) circle (5pt);
      \fill[fill=blue] (0,4) circle (4pt);
      \draw[line width=0.3mm,arrows=<->] (0.5,0) -- (0.5,4);

      \draw[snake=zigzag,segment amplitude=2mm,segment length=3mm] (5.5,6) -- (5.5,3);
      \fill[fill=red] (5.5,6) circle (5pt);
      \fill[fill=blue] (4.5,2.9) rectangle (6.5,3.1);
      \draw[line width=0.3mm,arrows=<->] (4,0) -- (4,3);

      \path (-0.5,5) node[left] {$E_1=\frac{k_1}{2}(\beta-\alpha)^2$} 
            (-0.5,2) node[left] {$E_2=\frac{k_2}{2}\alpha^2$}
            (0.5,2) node[right] {$\alpha$}
            (2.5,2) node[right] {$\beta$}
            (4,1.5) node[left] {$\alpha$}
            (6,4.5) node[right] {$E_1=\frac{k_1}{2}(\beta-\alpha)^2$}
            (6,2.5) node[right] {$E_2=m g |\alpha|$, $\alpha \geq 0$};
      \draw[line width=1mm] (-2.5,0) -- (8,0);
      \fill[fill=red] (0,0) circle (5pt);
   \end{tikzpicture}\label{intro:fig:ressort:b}
} \hfill
\subfigure[High regularization]{
   \begin{tikzpicture}[scale=0.75]
      \draw[snake=zigzag,segment amplitude=2mm,segment length=4mm] (0,6) -- (0,2);
      \draw[snake=zigzag,segment amplitude=2mm,segment length=2mm] (0,0) -- (0,2);
      \fill[fill=red] (0,6) circle (5pt);
      \fill[fill=blue] (0,2) circle (4pt);
      \draw[line width=0.3mm,arrows=<->] (0.5,0) -- (0.5,2);

      \draw[snake=zigzag,segment amplitude=2mm,segment length=6mm] (5.5,6) -- (5.5,0.1);
      \fill[fill=red] (5.5,6) circle (5pt);
      \fill[fill=blue] (4.5,0) rectangle (6.5,0.2);

      \path (-0.5,5) node[left] {$E_1=\frac{k_1}{2}(\beta-\alpha)^2$} 
            (-0.5,1) node[left] {$E_2=\frac{k_2}{2}\alpha^2$}
            (0.5,1) node[right] {$\alpha$}
            (5,0.5) node[left] {\color{red} {$\alpha^\star=0$}}
            (6,4.5) node[right] {$E_1=\frac{k_1}{2}(\beta-\alpha)^2$}
            (6,0.5) node[right] {$E_2=m g |\alpha|, \alpha \geq 0$};
      \draw[line width=1mm] (-2.5,0) -- (8,0);
      \fill[fill=red] (0,0) circle (5pt);
   \end{tikzpicture} \label{intro:fig:ressort:c}
} \caption{A physical system illustrating the sparsity-inducing effect of the
   $\ell_1$-norm (on the right) in contrast to the Tikhonov-ridge regularization
   (on the left). Three springs are represented in each figure, two on the left, one
   on the right. Red points are fixed and cannot move. On the left, two
   springs are linked to each other by a blue point whose position can vary. On
   the right, a blue object of mass~$m$ is attached to the spring. Right and
   left configurations define two different dynamical systems with
   energies~$E_1+E_2$; on the left, $E_1$ and~$E_2$ are elastic potential energies;
   on the right, $E_1$ is the same as on the left, whereas~$E_2$ is a
   gravitational potential energy, where~$g$ is the gravitational constant on the Earth's surface.
   Both system can evolve according to their initial positions, and stabilize
   for the value of $\alpha^\star$ that minimizes the energy $E_1+E_2$, assuming that some energy can be dissipated by friction forces. 
   On the left, it is possible to show that~$\alpha^\star=\beta k_1/(k_1+k_2)$ and thus, the solution~$\alpha^\star$ is
   never equal to zero, regardless of the strength~$k_2$ of the bottom spring.
   On the right, the solution is obtained by soft-thresholding: $\alpha^\star=\max(\beta-mg/k_1,0)$. As shown on 
   Figure~\ref{intro:fig:ressort:c}, when the mass~$m$ is large enough, the blue object touches the ground and~$\alpha^\star=0$.
Figure adapted from~\citep{mairal_thesis}.
}\label{fig:ressort}
\end{figure}
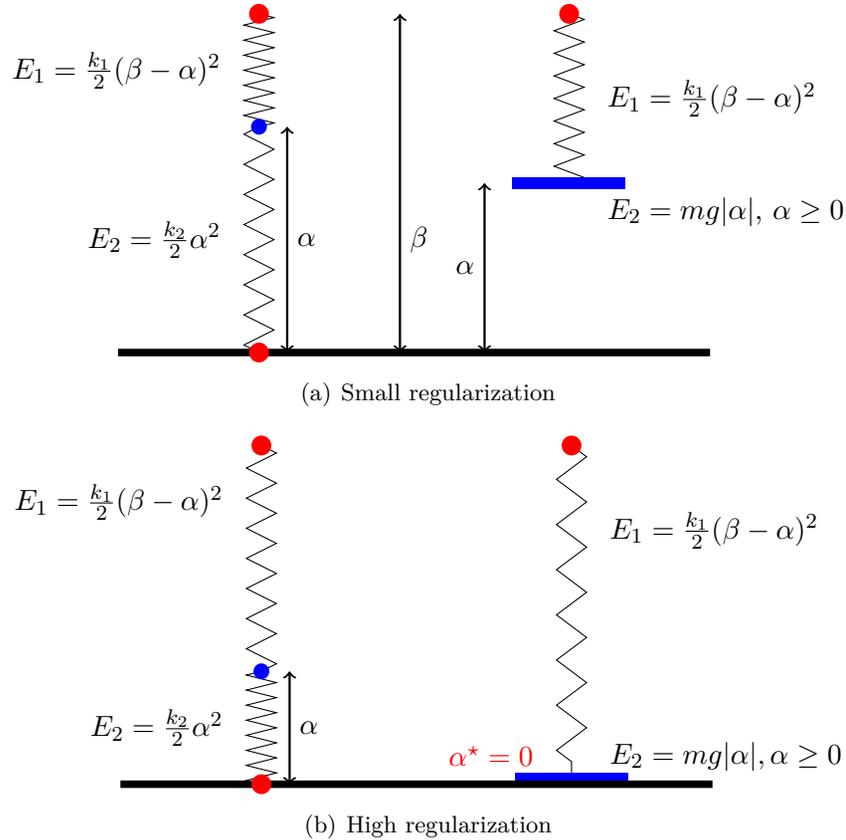

\subparagraph{Geometrical point of view.} 
The sparsity-inducing effect of the~$\ell_1$-norm can also be interpreted by
studying the geometry of the~$\ell_1$-ball $\{ \alphab \in \Real^p :
\|\alphab\|_1 \leq \mu\}$.
More precisely, understanding the effect of the Euclidean projection onto this
set is important: in simple cases where the design matrix~$\D$ is orthogonal,
the solution of~(\ref{eq:lasso}) can indeed be obtained by the 
projection
\begin{equation}
   \min_{\alphab \in \Real^p} \frac{1}{2}\|\betab- \alphab\|_2^2 \st \|\alphab\|_1 \leq \mu,\label{eq:projl1}
\end{equation}
where~$\betab = \D^\top \x$. When~$\D$ is not orthogonal, a classical algorithm
for solving~(\ref{eq:lasso}) is the projected gradient
method~(see Section~\ref{sec:optiml1}), which performs a sequence of 
projections~(\ref{eq:projl1}) for different values of~$\betab$.
Note that how to solve~(\ref{eq:projl1}) efficiently is well studied; it can be
achieved in $O(p)$ operations with a divide-and-conquer
strategy~\citep{brucker1984,duchi2008efficient}.

In Figure~\ref{fig:l1l2}, we illustrate the effect of the~$\ell_1$-norm
projection and compare it to the case of the~$\ell_2$-norm. The
corners of the~$\ell_1$-ball are on the main axes and correspond to sparse
solutions.  Two of them are represented by red and green dots, with respective
coordinates $(\mu,0)$ and~$(0,\mu)$. Most strikingly, a large part of the space in
the figure, represented by red and green regions, ends up on these corners
after projection. In contrast, the set of points that is projected onto the blue
dot, is simply the blue line. The blue dot corresponds in fact to a dense
solution with coordinates~$(\mu/2,\mu/2)$.  Therefore, the figure 
illustrates that the~$\ell_1$-ball in two dimensions encourages
solutions to be on its corners. In the case of the~$\ell_2$-norm, the ball is
isotropic, and treats every direction equally.  In Figure~\ref{fig:l1l2b}, we
represent these two balls in three dimensions, where we can make similar observations.

More formally, we can mathematically characterize our remarks 
about Figure~\ref{fig:l1l2}. Consider a point~$\y$ in~$\Real^p$ on the surface of
the~$\ell_1$-ball of radius~$\mu=1$, and define the set $\NN \defin \{ \z \in \Real^p :
\pi(\z)=\y\}$, where~$\pi$ is the projection operator onto the~$\ell_1$-ball.
Examples of pairs~$(\y,\NN)$ have been presented in Figure~\ref{fig:l1l2};
for instance, when~$\y$ is the red or green dot,~$\NN$ is respectively the red
or green region. It is particularly informative to study how~$\NN$ varies
with~$\y$, which is the focus of the next proposition.
\begin{proposition}[Characterization of the set~$\NN$]~\\
   For a non-zero vector~$\y$ in~$\Real^p$, the set~$\NN$ defined in the
   previous paragraph can be written as~$\NN = \y + \KK$, where
   $\KK$ is a polyhedral cone of dimension~$p-\|\y\|_0+1$.
\end{proposition}
\begin{proof}
A classical theorem~\citep[see][Proposition B.11]{bertsekas} allows
us to rewrite~$\NN$ as
\begin{displaymath}
      \NN = \{ \z \in \Real^p ~:~ \forall~\|\x\|_1 \leq 1,~~(\z-\y)^\top(\x-\y) \leq 0\}
          = \y + \KK,
\end{displaymath}
where~$\y+\KK$ denotes the Minkowski sum $\{\y+\z : \z \in \KK\}$ between the set~$\{\y\}$ and the cone~$\KK$ defined as
\begin{displaymath}
   \KK \defin \{ \d \in \Real^p ~:~ \forall~\|\x\|_1 \leq 1,~~\d^\top(\x-\y) \leq 0\}.
\end{displaymath}
Note that in the optimization literature,~$\KK$ is often called the ``normal cone'' to the unit~$\ell_1$-ball at the point~$\y$~\citep{borwein}.
Equivalently, we have
\begin{equation}
   \begin{split}
      \KK &= \{ \d \in \Real^p ~:~ \max_{\|\x\|_1 \leq 1} \d^\top\x \leq \d^\top\y\} \\
          &= \{ \d \in \Real^p ~:~ \|\d\|_\infty \leq \d^\top\y\},
   \end{split} \label{eq:coneK}
\end{equation}
where we have used the fact that quantity $\max_{\|\x\|_1 \leq 1} \d^\top\x$,
called the dual-norm of the~$\ell_1$-norm, is equal to
$\|\d\|_\infty$~\citep[see][]{bach2012}. Note now that according to H\"older's
inequality, we also have~$\d^\top\y \leq \|\d\|_\infty\|\y\|_1 \leq
\|\d\|_\infty$ in Eq.~(\ref{eq:coneK}).  Therefore, the inequalities are in
fact equalities. It is then easy to characterize vectors~$\d$ such 
that $\d^\top\y = \|\d\|_\infty\|\y\|_1$ and it is possible to show that~$\KK$
is simply the set of vectors~$\d$ satisfying~$\d[i] = \sign(\y[i])\|\d\|_\infty$ for
all~$i$ such that~$\y[i] \neq 0$. 

This would be sufficient to conclude the proposition, but it is also possible
to pursue the analysis and exactly characterize~$\KK$ by finding a set of generators.\footnote{A collection of vectors~$\z_1,\z_2,\ldots,\z_l$ are called
generators for a cone~$\KK$ when~$\KK$ consists of all positive combinations of
the vectors~$\z_i$. In other words, $\KK = \{ \sum_{i=1}^l \alpha_i \z_i :
\alpha_i \geq 0  \}$. In that case, we use the
notation~$\KK=\text{cone}(\z_1,\ldots,\z_l)$.\label{footnote:cone}} Let us
define the vector~$\s$ in~$\{-1,0,+1\}^p$ that carries the sparsity pattern
of~$\y$, more precisely, with~$\s[i]=\sign(\y[i])$ for all~$i$ such
that~$\y[i] \neq 0$, and $\s[i]=0$ otherwise. Let us also define the set of
indices~$\{i_1,\ldots,i_l\}$ corresponding to the~$l$ zero entries of~$\y$,
and~$\e_i$ in~$\Real^p$ the binary vector whose entries are all zero but
the~$i$-th one that is equal to $1$.
Then, after a short calculation, we can geometrically characterize 
the polyhedral cone~$\KK$:
\begin{displaymath}
   \KK = \text{cone}\left(\s,\s - \e_{i_1},\s + \e_{i_1},\s - \e_{i_2},\s + \e_{i_2}, \ldots, \s -\e_{i_{l}}, \s +\e_{i_{l}}\right),
\end{displaymath}
where the notation ``cone'' is defined in footnote~\ref{footnote:cone}. 
\end{proof}
It is now easy to see that the set~$\KK$ ``grows'' with the number~$l$ of zero
entries in~$\y$, and that~$\KK$ lives in a subspace of dimension~$l+1$ for all
non-zero vector~$\y$.  For example, when~$l=0$---that is,~$\y$ is a dense
vector (\eg, the blue point in Figure~\ref{subfig:l1}),~$\KK$ is simply a
half-line. 

To conclude, the geometrical intuition to gain from this section is that
\emph{the Euclidean projection onto a convex set encourages solutions on
singular points, such as edges or corners for polytopes}. Such a principle
indeed applies beyond the~$\ell_1$-norm. For instance, we illustrate the
regularization effect of the~$\ell_\infty$-norm in Figure~\ref{fig:linf},
whose corners coordinates have same magnitude.

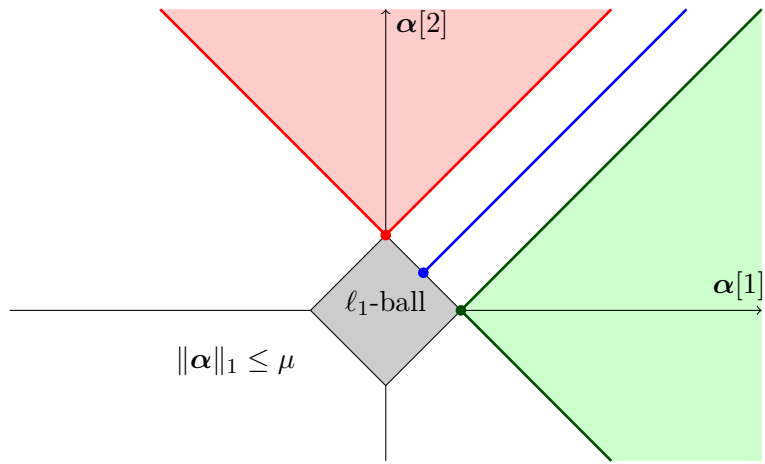
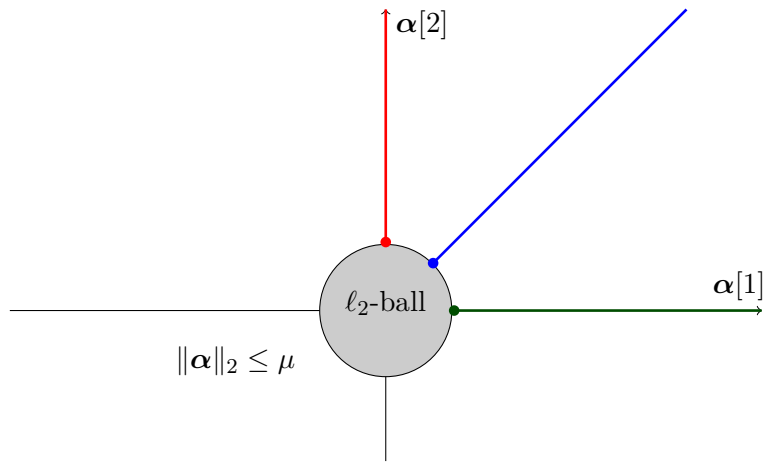
\begin{figure}
   \centering
   \subfigure[Effect of the Euclidean projection onto the~$\ell_1$-ball.]{
      \begin{tikzpicture}
         \fill[fill=green!20] (1,0) -- (5,4) -- (5,-2) --(3,-2);
         \fill[fill=red!20] (0,1) -- (3,4) -- (-3,4) --(0,1);
         \draw[arrows=-,line width=1pt,color=red](0,1)--(3,4);
         \draw[arrows=-,line width=1pt,color=red](0,1)--(-3,4);
         \draw[arrows=->,line width=.4pt](0,-2)--(0,4);
         \draw[arrows=->,line width=.4pt](-5,0)--(5,0);
         \fill[fill=black!20, draw=black] (-1,0) -- (0,1) -- (1,0) --(0,-1)--(-1,0);
         \fill[fill=red] (0,1) circle (2pt);
         \fill[fill=blue] (0.5,0.5) circle (2pt);
         \fill[fill=darkgreen] (1.0,0.0) circle (2pt);
         \draw[arrows=-,line width=1pt,color=darkgreen](1,0)--(5,4);
         \draw[arrows=-,line width=1pt,color=darkgreen](1,0)--(3,-2);
         \draw[arrows=-,line width=1pt,color=blue](0.5,0.5)--(4,4);
         \path (0,-0.2) node[above] {$\ell_1$-ball};
         \path (-2,-1) node[above] {$\|\alphab\|_1 \leq \mu$};
         \path (0,3.8) node[right] {$\alphab[2]$}
         (4.7,0) node[above] {$\alphab[1]$};
      \end{tikzpicture}\label{subfig:l1}
   }
   \subfigure[Effect of the Euclidean projection onto the~$\ell_2$-ball.]{
      \begin{tikzpicture}
         \draw[arrows=->,line width=.4pt](0,-2)--(0,4);
         \draw[arrows=->,line width=.4pt](-5,0)--(5,0);
         \draw[arrows=-,line width=1pt,color=red](0,0.91)--(0,4);
        \fill[fill=black!20,draw=black] (0,0) circle (25pt);
         \path (0,3.8) node[right] {$\alphab[2]$}
         (4.7,0) node[above] {$\alphab[1]$};
         \fill[fill=red] (0,0.91) circle (2pt);
         \fill[fill=blue] (0.63,0.63) circle (2pt);
         \fill[fill=darkgreen] (0.91,0.0) circle (2pt);
         \draw[arrows=-,line width=1pt,color=darkgreen](0.91,0)--(5,0);
         \draw[arrows=-,line width=1pt,color=blue](0.63,0.63)--(4,4);
         \path (0,-0.2) node[above] {$\ell_2$-ball};
         \path (-2,-1) node[above] {$\|\alphab\|_2 \leq \mu$};
      \end{tikzpicture}\label{subfig:l2}
   } \caption{Illustration in two dimensions of the projection operator onto
      the~$\ell_1$-ball in Figure~\subref{subfig:l1} and~$\ell_2$-ball in
   Figure~\subref{subfig:l2}. The balls are represented in gray. All points from
   the red regions are projected onto the point of coordinates~$(0,\mu)$ denoted
   by a red dot. Similarly, the green and blue regions are projected onto the green
   and blue dots, respectively.  For the~$\ell_1$-norm, a large part of the
   figure is filled by the red and green regions, whose points are projected to
   a sparse solution corresponding to a corner of the ball. For
   the~$\ell_2$-norm, this is not the case: any non-sparse point---say, for
   instance on the blue line---is projected onto a non-sparse solution.
}
   \label{fig:l1l2}
\end{figure}

\begin{figure}
    \centering
    \subfigure[$\ell_2$-ball in 3D]{\includegraphics[width=0.45\linewidth]{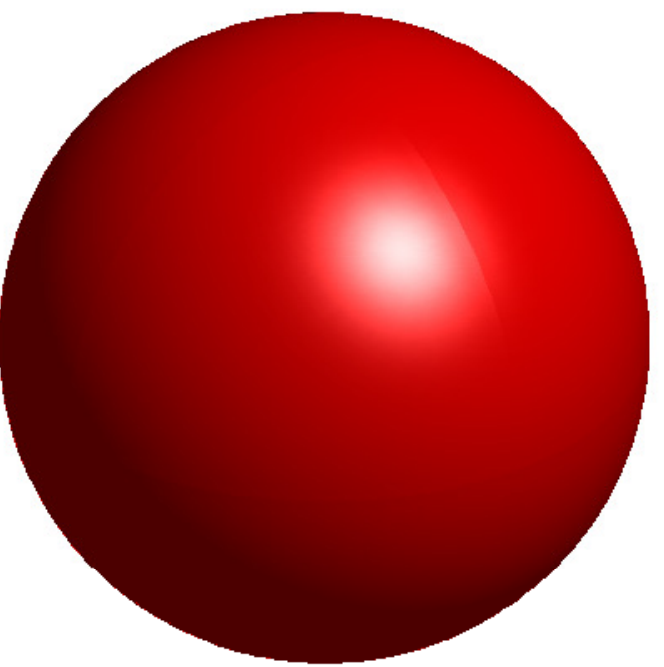}} \hfill
    \subfigure[$\ell_1$-ball in 3D]{\includegraphics[width=0.45\linewidth]{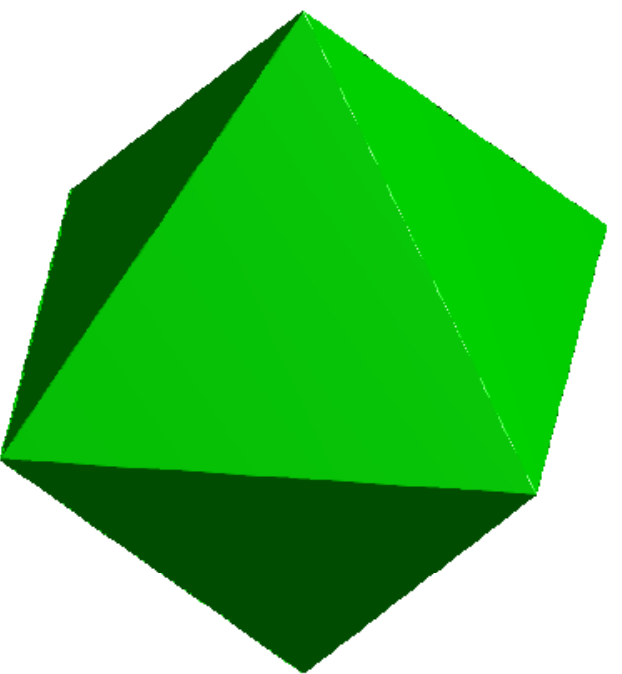}} 
    \caption{Representation in three dimensions of the~$\ell_1$- and~$\ell_2$-balls. Figure borrowed from~\citet{bach2012}, produced by Guillaume Obozinski.}
    \label{fig:l1l2b}
\end{figure}

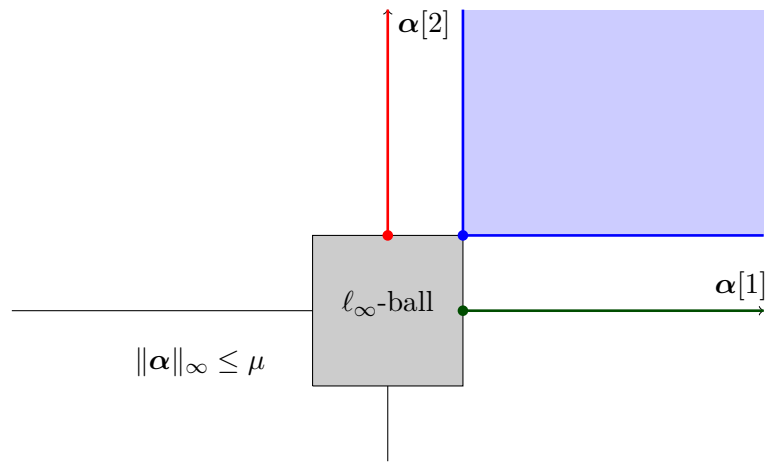
\begin{figure}
   \centering
      \begin{tikzpicture}
         \fill[fill=blue!20] (1,1) -- (1,4) -- (5,4) -- (5,1) --(1,1);
         \draw[arrows=-,line width=1pt,color=blue](1,1)--(1,4);
         \draw[arrows=-,line width=1pt,color=blue](1,1)--(5,1);
         \draw[arrows=->,line width=.4pt](0,-2)--(0,4);
         \draw[arrows=->,line width=.4pt](-5,0)--(5,0);
         \fill[fill=black!20, draw=black] (-1,-1) -- (-1,1) -- (1,1) --(1,-1)--(-1,-1);
         \path (0,3.8) node[right] {$\alphab[2]$}
         (4.7,0) node[above] {$\alphab[1]$};
         \fill[fill=red] (0,1) circle (2pt);
         \fill[fill=blue] (1,1) circle (2pt);
         \fill[fill=darkgreen] (1,0.0) circle (2pt);
         \draw[arrows=-,line width=1pt,color=darkgreen](1,0)--(5,0);
         \draw[arrows=-,line width=1pt,color=red](0,1)--(0,4);
         \path (0,-0.2) node[above] {$\ell_\infty$-ball};
         \path (-2.5,-1) node[above] {$\|\alphab\|_\infty \leq \mu$};
      \end{tikzpicture}
      \caption{Similar illustration as Figure~\ref{fig:l1l2} for the~$\ell_\infty$-norm. The regularization effect encourages solution to be on the corners of the ball, corresponding to points with the same magnitude $|\alphab[1]|=|\alphab[2]|=\mu$.}
         \label{fig:linf}
\end{figure}

\paragraph{Non-convex regularization.}
Even though it is well established that the~$\ell_1$-norm encourages sparse
solutions, it remains only a convex proxy of the~$\ell_0$-penalty. Both in
statistics and signal processing, other sparsity-inducing regularization
functions have been proposed, in particular continuous relaxations of~$\ell_0$
that are non-convex \citep{frank1993,fan2001,ingrid2010iteratively,gasso2009}.
These functions are using a non-decreasing concave function~$\varphi: \Real^+ \mapsto \Real$, 
and the sparsity-inducing penalty is defined as
\begin{displaymath}
   \psi(\alphab) \defin \sum_{i=1}^p \varphi\left(|\alphab[i]|\right).
\end{displaymath}
For example, the~$\ell_q$-penalty uses~$\varphi: x \mapsto
x^q$~\citep{frank1993}, or an approximation~$\varphi: x \mapsto (x+\varepsilon)^q$; the
reweighted-$\ell_1$ algorithm of~\citet{fazel2002,fazel2003,candes4} implicitly uses~$\varphi: x
\mapsto \log(x+\varepsilon)$. 
These penalties typically lead to intractable estimation problems, but
approximate solutions can be obtained with continuous optimization
techniques (see Section~\ref{sec:reweighted}). 

The sparsity-inducing effect
of the penalties~$\psi$ is known to be stronger than~$\ell_1$. As shown in
Figure~\ref{subfig:nonconvexa}, the magnitude of the derivative of~$\varphi$ grows
when one approaches zero because of its concavity. Thus, in the one-dimensional case,
~$\psi$ can be interpreted as \emph{a force driving~$\alpha$ towards the origin
with increasing intensity when~$\alpha$ gets closer to zero.}
In terms of geometry, we also display the~$\ell_q$-ball in
Figure~\ref{subfig:nonconvexb}, with the same red, blue, and green dots as in
Figure~\ref{fig:l1l2}. The part of the space that is projected onto the corners of
the~$\ell_q$-ball is larger than that for~$\ell_1$.
Interestingly, the geometrical structure of the red and green regions are also
more complex. Their combinatorial nature makes the projection problem onto
the~$\ell_q$-ball more involved when~$q < 1$.


\begin{figure}
   \centering
   \subfigure[Illustration of a non-convex sparsity-inducing penalty.]{
      \includegraphics[page=3]{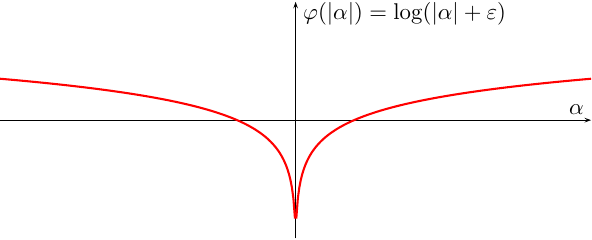}
      \label{subfig:nonconvexa}
   }
   \subfigure[$\ell_q$-ball with~$q < 1$.]{\begin{tikzpicture}
       \fill[fill=red!20] (0,1) -- (1.58,1.58) -- (4,4) -- (-4,4) -- (-1.58,1.58);
       \fill[fill=green!20] (1,0) -- (1.58,1.58) -- (4,4) -- (5,4)  -- (5,-2) -- (2,-2) -- (1.58,-1.58);
       \draw[arrows=-,line width=1pt,color=red](0,1)-- (-1.58,1.58) -- (-4,4);
         \draw[arrows=->,line width=.4pt](0,-2)--(0,4);
         \draw[arrows=->,line width=.4pt](-5,0)--(5,0);
         \fill[fill=black!20, draw=black] (-1,0) to [out=20,in=250] (0,1) to [out=290,in=160] (1,0) to [out=200,in=70] (0,-1) to [out=110,in=340](-1,0);
         \path (0,3.8) node[right] {$\alphab[2]$}
         (4.7,0) node[above] {$\alphab[1]$};
         \fill[fill=red] (0,1) circle (2pt);
         \fill[fill=blue] (0.38,0.38) circle (2pt);
         \fill[fill=darkgreen] (1,0) circle (2pt);
         \draw[arrows=-,line width=1pt,color=red](0,1) -- (1.58,1.60) -- (3.98,4);
         \draw[arrows=-,line width=1pt,color=darkgreen](1,0)--(1.58,1.56) -- (4,3.98);
         \draw[arrows=-,line width=1pt,color=darkgreen](1,0)--(1.58,-1.58) -- (2,-2);
         \draw[arrows=-,line width=1pt,color=blue](0.38,0.38)--(1.58,1.58);
         \path (0,-0.3) node[above] {$\ell_q$-ball};
         \path (-2.5,-1) node[above] {$\|\alphab\|_q \leq \mu$ with $q < 1$};
   \end{tikzpicture}\label{subfig:nonconvexb}}
   \caption{Illustration of the sparsity-inducing effect of a non-convex penalty. In~\subref{subfig:nonconvexa}, we plot the non-convex penalty~$\alpha \mapsto \log(|\alpha|+\varepsilon)$, and in~\subref{subfig:nonconvexb}, we present a similar figure as~\ref{fig:l1l2} for the~$\ell_q$-penalty, when choosing~$q < 1$.} \label{fig:nonconvex}
\end{figure}

\paragraph{The elastic-net.} To cope with some instability issues of the
estimators obtained with the~$\ell_1$-regularization,~\citet{zou} have proposed to
combine the~$\ell_1$- and~$\ell_2$-norms with a penalty called elastic-net:
\begin{displaymath}
   \psi(\alphab) \defin \|\alphab\|_1 + \gamma \|\alphab\|_2^2. 
\end{displaymath}
The effect of this penalty is illustrated in Figure~\ref{fig:elasticnet}.
Compared to Figure~\ref{fig:l1l2}, we observe that the red and green regions
are smaller for the elastic-net penalty than for~$\ell_1$.  The
sparsity-inducing effect is thus less aggressive than the one obtained
with~$\ell_1$. 

\begin{figure}
   \centering
   \begin{tikzpicture}
      \fill[fill=green!20] (1,0) -- (5,1.87) -- (5,-1.87);
      \fill[fill=red!20] (0,1) -- (1.40,4) -- (-1.40,4) --(0,1);
      \draw[arrows=-,line width=1pt,color=red](0,1)--(1.40,4);
      \draw[arrows=-,line width=1pt,color=red](0,1)--(-1.40,4);
      \draw[arrows=->,line width=.4pt](0,-2)--(0,4);
      \draw[arrows=->,line width=.4pt](-5,0)--(5,0);
      \fill[fill=black!20, draw=black] (-1,0) to [out=65,in=205] (0,1) to [out=-25,in=115] (1,0) to [out=245,in=25] (0,-1) to [out=155,in=-65](-1,0);
      \fill[fill=red] (0,1) circle (2pt);
      \fill[fill=blue] (0.6,0.6) circle (2pt);
      \fill[fill=darkgreen] (1.0,0.0) circle (2pt);
      \draw[arrows=-,line width=1pt,color=darkgreen](1,0)--(5,1.87);
      \draw[arrows=-,line width=1pt,color=darkgreen](1,0)--(5,-1.87);
      \draw[arrows=-,line width=1pt,color=blue](0.6,0.6)--(4,4);
      \path (0,-0.1) node[above] {elastic-net};
      \path (0,-0.05) node[below] {ball};
      \path (-2.5,-1.3) node[above] {$(1-\gamma)\|\alphab\|_1+\gamma\|\alphab\|_2^2 \leq \mu$};
      \path (0,3.8) node[right] {$\alphab[2]$}
      (4.7,0) node[above] {$\alphab[1]$};
   \end{tikzpicture}
   \caption{Similar figure as~\ref{fig:l1l2} for the elastic-net penalty.}\label{fig:elasticnet}
\end{figure}
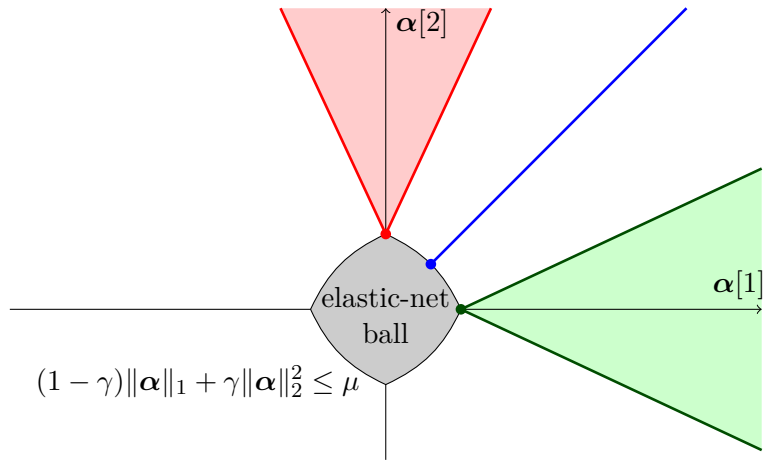

\paragraph{Total variation.} The anisotropic total variation penalty~\citep{rudin} for one dimensional 
signals is simply the~$\ell_1$-norm of finite differences
\begin{displaymath}
   \psi(\alphab) \defin \sum_{i=1}^{p-1} |\alphab[i+1]-\alphab[i]|,
\end{displaymath}
which encourages piecewise constant signals. It is also known in statistics under
the name of ``fused Lasso''~\citep{tibshirani2005sparsity}.  The penalty can
easily be extended to two-dimensional signals, and has been widely used for
regularizing inverse problems in image processing~\citep{chambolle2005total}.

\paragraph{Group sparsity.} 
In some cases, variables are organized into predefined groups forming a
partition~$\GG$ of~$\{1,\ldots,p\}$, and one is looking for a
solution~$\alphab^\star$ such that \emph{variables belonging to the same group
of~$\GG$ are set to zero together.} For example, such groups have appeared in
Section~\ref{sec:wavelets} about wavelets, where~$\GG$ could be defined
according to neighborhood relationships of wavelet coefficients.
Then, when it is known beforehand that a problem solution only requires a few groups
of variables to explain the data, a regularization function automatically 
selecting the relevant groups has been shown to improve the prediction performance
or to provide a solution with a better interpretation~\citep{turlach,yuan,obozinski-joint,huang2}.
The group sparsity principle is illustrated in Figure~\ref{subfig:groupspars}.

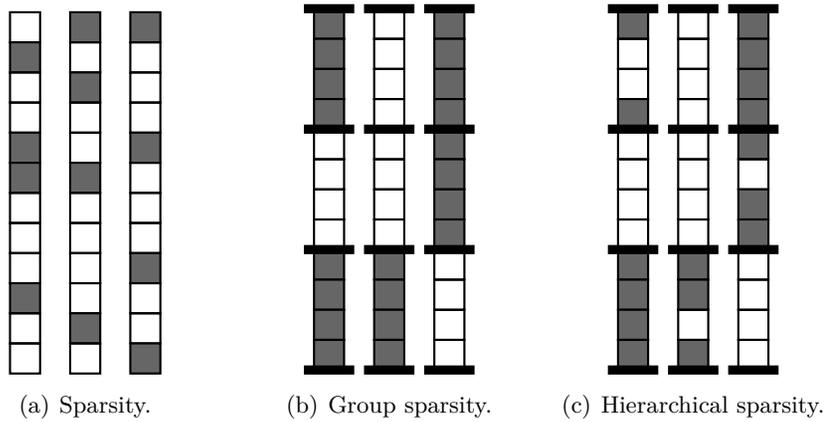
\begin{figure}[btp]
   \subfigure[Sparsity.]{
   \begin{tikzpicture}[node distance=\distnodeb]
      \node (s)  { };
      \node[varempty2] [right of=s,xshift=0.4cm] (a)  { };
      \node[varfull2] [below of=a] (b) { };
      \node[varempty2] [below of=b] (c) { };
      \node[varempty2] [below of=c] (d) { };
      \node[varfull2] [below of=d] (e) { };
      \node[varfull2] [below of=e] (f) { };
      \node[varempty2] [below of=f] (g) { };
      \node[varempty2] [below of=g] (h) { };
      \node[varempty2] [below of=h] (i) { };
      \node[varfull2] [below of=i] (j) { };
      \node[varempty2] [below of=j] (k) { };
      \node[varempty2] [below of=k] (l) { };
      \node[varfull2] [right of=a,xshift=\distnodeb] (a1)  { };
      \node[varempty2] [below of=a1] (b1) { };
      \node[varfull2] [below of=b1] (c1) { };
      \node[varempty2] [below of=c1] (d1) { };
      \node[varempty2] [below of=d1] (e1) { };
      \node[varfull2] [below of=e1] (f1) { };
      \node[varempty2] [below of=f1] (g1) { };
      \node[varempty2] [below of=g1] (h1) { };
      \node[varempty2] [below of=h1] (i1) { };
      \node[varempty2] [below of=i1] (j1) { };
      \node[varfull2] [below of=j1] (k1) { };
      \node[varempty2] [below of=k1] (l1) { };
      \node[varfull2] [right of=a1,xshift=\distnodeb] (a2)  { };
      \node[varempty2] [below of=a2] (b2) { };
      \node[varempty2] [below of=b2] (c2) { };
      \node[varempty2] [below of=c2] (d2) { };
      \node[varfull2] [below of=d2] (e2) { };
      \node[varempty2] [below of=e2] (f2) { };
      \node[varempty2] [below of=f2] (g2) { };
      \node[varempty2] [below of=g2] (h2) { };
      \node[varfull2] [below of=h2] (i2) { };
      \node[varempty2] [below of=i2] (j2) { };
      \node[varempty2] [below of=j2] (k2) { };
      \node[varfull2] [below of=k2] (l2) { };
      \node (t) [right of=l2,xshift=0.4cm] { };
   \end{tikzpicture}\label{subfig:spars}
}  \hfill
\subfigure[Group sparsity.]{
   \begin{tikzpicture}[node distance=\distnodeb]
      \node (s)  { };
      \node[varfull2] [right of=s,xshift=0.4cm] (a)  { };
      \node[varfull2] [below of=a] (b) { };
      \node[varfull2] [below of=b] (c) { };
      \node[varfull2] [below of=c] (d) { };
      \node[varempty2] [below of=d] (e) { };
      \node[varempty2] [below of=e] (f) { };
      \node[varempty2] [below of=f] (g) { };
      \node[varempty2] [below of=g] (h) { };
      \node[varfull2] [below of=h] (i) { };
      \node[varfull2] [below of=i] (j) { };
      \node[varfull2] [below of=j] (k) { };
      \node[varfull2] [below of=k] (l) { };
      \node[varempty2] [right of=a,xshift=\distnodeb] (a1)  { };
      \node[varempty2] [below of=a1] (b1) { };
      \node[varempty2] [below of=b1] (c1) { };
      \node[varempty2] [below of=c1] (d1) { };
      \node[varempty2] [below of=d1] (e1) { };
      \node[varempty2] [below of=e1] (f1) { };
      \node[varempty2] [below of=f1] (g1) { };
      \node[varempty2] [below of=g1] (h1) { };
      \node[varfull2] [below of=h1] (i1) { };
      \node[varfull2] [below of=i1] (j1) { };
      \node[varfull2] [below of=j1] (k1) { };
      \node[varfull2] [below of=k1] (l1) { };
      \node[varfull2] [right of=a1,xshift=\distnodeb] (a2)  { };
      \node[varfull2] [below of=a2] (b2) { };
      \node[varfull2] [below of=b2] (c2) { };
      \node[varfull2] [below of=c2] (d2) { };
      \node[varfull2] [below of=d2] (e2) { };
      \node[varfull2] [below of=e2] (f2) { };
      \node[varfull2] [below of=f2] (g2) { };
      \node[varfull2] [below of=g2] (h2) { };
      \node[varempty2] [below of=h2] (i2) { };
      \node[varempty2] [below of=i2] (j2) { };
      \node[varempty2] [below of=j2] (k2) { };
      \node[varempty2] [below of=k2] (l2) { };
      \node[sep] [below of=d,yshift=-0.6cm] { };
      \node[sep] [below of=h,yshift=-0.6cm] { };
      \node[sep] [below of=d2,yshift=-0.6cm] { };
      \node[sep] [below of=h2,yshift=-0.6cm] { };
      \node[sep] [below of=l2,yshift=-0.6cm] { };
      \node[sep] [below of=l,yshift=-0.6cm] { };
      \node[sep] [above of=a,yshift=0.6cm] { };
      \node[sep] [above of=a2,yshift=0.6cm] { };
      \node[sep] [below of=d1,yshift=-0.6cm] { };
      \node[sep] [below of=h1,yshift=-0.6cm] { };
      \node[sep] [below of=l1,yshift=-0.6cm] { };
      \node[sep] [above of=a1,yshift=0.6cm] { };
      \node (t) [right of=l2,xshift=0.4cm] { };
   \end{tikzpicture}\label{subfig:groupspars}
}   \hfill
\subfigure[Hierarchical sparsity.]{
   \begin{tikzpicture}[node distance=\distnodeb]
      \node (s)  { };
      \node[varfull2] [right of=s,xshift=0.4cm] (a)  { };
      \node[varempty2] [below of=a] (b) { };
      \node[varempty2] [below of=b] (c) { };
      \node[varfull2] [below of=c] (d) { };
      \node[varempty2] [below of=d] (e) { };
      \node[varempty2] [below of=e] (f) { };
      \node[varempty2] [below of=f] (g) { };
      \node[varempty2] [below of=g] (h) { };
      \node[varfull2] [below of=h] (i) { };
      \node[varfull2] [below of=i] (j) { };
      \node[varfull2] [below of=j] (k) { };
      \node[varfull2] [below of=k] (l) { };
      \node[varempty2] [right of=a,xshift=\distnodeb] (a1)  { };
      \node[varempty2] [below of=a1] (b1) { };
      \node[varempty2] [below of=b1] (c1) { };
      \node[varempty2] [below of=c1] (d1) { };
      \node[varempty2] [below of=d1] (e1) { };
      \node[varempty2] [below of=e1] (f1) { };
      \node[varempty2] [below of=f1] (g1) { };
      \node[varempty2] [below of=g1] (h1) { };
      \node[varfull2] [below of=h1] (i1) { };
      \node[varfull2] [below of=i1] (j1) { };
      \node[varempty2] [below of=j1] (k1) { };
      \node[varfull2] [below of=k1] (l1) { };
      \node[varfull2] [right of=a1,xshift=\distnodeb] (a2)  { };
      \node[varfull2] [below of=a2] (b2) { };
      \node[varfull2] [below of=b2] (c2) { };
      \node[varfull2] [below of=c2] (d2) { };
      \node[varfull2] [below of=d2] (e2) { };
      \node[varempty2] [below of=e2] (f2) { };
      \node[varfull2] [below of=f2] (g2) { };
      \node[varfull2] [below of=g2] (h2) { };
      \node[varempty2] [below of=h2] (i2) { };
      \node[varempty2] [below of=i2] (j2) { };
      \node[varempty2] [below of=j2] (k2) { };
      \node[varempty2] [below of=k2] (l2) { };
      \node[sep] [below of=d,yshift=-0.6cm] { };
      \node[sep] [below of=h,yshift=-0.6cm] { };
      \node[sep] [below of=d2,yshift=-0.6cm] { };
      \node[sep] [below of=h2,yshift=-0.6cm] { };
      \node[sep] [below of=l2,yshift=-0.6cm] { };
      \node[sep] [below of=l,yshift=-0.6cm] { };
      \node[sep] [above of=a,yshift=0.6cm] { };
      \node[sep] [above of=a2,yshift=0.6cm] { };
      \node[sep] [below of=d1,yshift=-0.6cm] { };
      \node[sep] [below of=h1,yshift=-0.6cm] { };
      \node[sep] [below of=l1,yshift=-0.6cm] { };
      \node[sep] [above of=a1,yshift=0.6cm] { };
      \node (t) [right of=l2,xshift=0.4cm] { };
   \end{tikzpicture}\label{subfig:hierspars}
} 
\caption{Illustration of the sparsity, group sparsity, and hierarchical
sparsity principles. Each column represents the sparsity pattern of a vector
with~$12$ variables. Non-zero coefficients are represented by gray squares. On
the left~\subref{subfig:spars}, the vectors are obtained with a simple
sparsity-inducing penalty, such as the~$\ell_1$-norm, and the non-zero
variables are scattered. In the middle figure~\subref{subfig:groupspars}, a
group sparsity-inducing penalty with three groups of variables is used. The
sparsity patterns respect the group structure. On the
right~\subref{subfig:hierspars}, we present the results obtained with a
hierarchical penalty consisting of the Group Lasso plus the~$\ell_1$-norm. The
group structure is globally respected, but some variables within a group can be
discarded.}\label{fig:groups}
\end{figure}

An appropriate regularization function to obtain a group-sparsity effect is
known as ``Group-Lasso'' penalty and is defined as
\begin{equation}
   \psi(\alphab) = \sum_{g \in \GG} \|\alphab[g]\|_q, \label{eq:grouplasso}
\end{equation}
where~$\|.\|_q$ is either the~$\ell_2$ or~$\ell_\infty$-norm.  To the best of
our knowledge, such a penalty appears in the early work
of~\citet{grandvalet1999} and~\citet{bakin1999} for~$q=2$, and~\citet{turlach}
for~$q=\infty$. It has been popularized later by~\citet{yuan}.

The function~$\psi$ in~(\ref{eq:grouplasso}) is a norm, thus convex, and can be
interpreted as the~$\ell_1$-norm of the vector~$[\|\alphab[g]\|_q]_{g \in \GG}$
of size~$|\GG|$. Consequently, the sparsity-inducing effect of the~$\ell_1$-norm
is applied at the group level. The penalty is highly related to the
group-thresholding approach for wavelets, since the group-thresholding
estimator~(\ref{eq:groupthresholding}) is linked to~$\psi$ through
Eq.~(\ref{eq:grouplassoa}).

In Figure~\ref{subfig:norma}, we visualize the unit ball of a Group-Lasso norm
obtained when~$\GG$ contains two groups $\GG=\{\{ 1,2\},\{3\}\}$.  The ball has
two singularities: the top and bottom corners, corresponding to solutions where
variables~$1$ and~$2$ are simultaneously set to zero, and the middle circle,
corresponding to solutions where variable~$3$ only is set to zero.
As expected, the geometry of the ball induces the group-sparsity
effect.

\paragraph{Structured sparsity.}
Group-sparsity is a first step towards the more general idea that a regularization
function can encourage sparse solutions with a particular structure. This
notion is called \emph{structured sparsity} and has been introduced under a
large number of different point of
views~\citep{zhao,jacob2009,jenatton,baraniuk,huang3}. To some extent, it
follows the concept of group-thresholding introduced in the wavelet
literature, which we have presented in Section~\ref{sec:wavelets}.
In this paragraph, we briefly review some of these works, but for a more
detailed review, we refer the reader to~\citep{bach2012b}.

Some penalties are non-convex. For instance, \citet{huang3}
and~\citet{baraniuk} propose two different combinatorial approaches based on a 
predefined set~$\GG$ of possibly overlapping groups of variables.
These penalties encourage solutions whose support is in the \emph{union of a few
number groups}, but they lead to NP-hard optimization problems.
Other penalties are convex. In particular, \citet{jacob2009} introduce a sparsity-inducing
norm that is exactly a convex relaxation of the penalty of~\citet{huang3},
even though these two approaches were independently developed at the same time. 
As a result, the convex penalty of~\citet{jacob2009} encourages a similar
structure as the one of~\citet{huang3}.

By following a different direction, the Group-Lasso
penalty~(\ref{eq:grouplasso}) has been considered when the groups are allowed
to overlap~\citep{zhao,jenatton}. As a consequence, variables belonging to the same groups are encouraged to be set to zero
together. It was proposed for hierarchical structures
by \citet{zhao} with the following rule:
whenever two groups~$g$ and~$h$ are in~$\GG$,
they should be either disjoint, or one should be included in another. 
Examples of such hierarchical group structures are given in
Figures~\ref{fig:groups} and~\ref{fig:treesparse}.
The effect of the penalty is to encourage sparsity patterns that are rooted subtrees.
Equivalently, a variable can be non-zero only if its parent in the tree is
non-zero, which is the main property of the zero-tree coding scheme introduced
in the wavelet literature~\citep{shapiro1993}, and already illustrated in
Figure~\ref{fig:treewavelets}.

Finally,~\citet{jenatton} has extended the hierarchical penalty of~\citet{zhao}
to more general group structures, for example when variable are organized on a
two-dimensional grid, encouraging neighbor variables to be simultaneously set
to zero.  We conclude this brief presentation of structured sparsity with
Figure~\ref{fig:structured}, where we present the unit balls of some
sparsity-inducing norms. Each of them exhibits singularities and
encourages particular sparsity patterns.

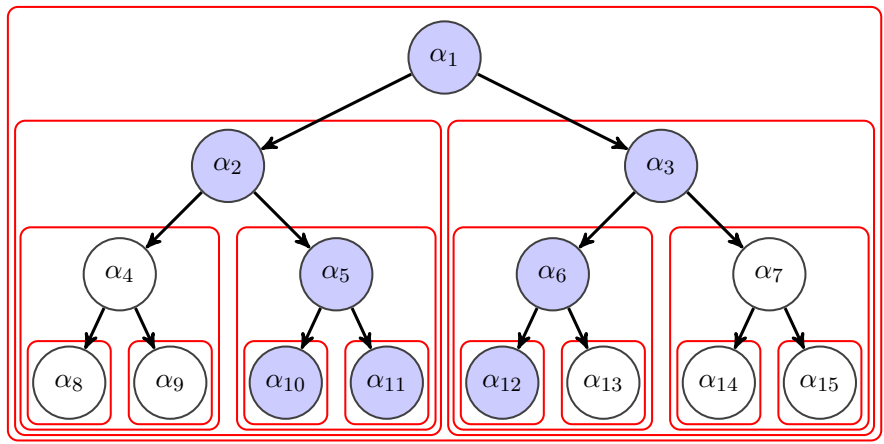
\begin{figure}[h]
\begin{center}
   \scalebox{0.96}{
 \begin{tikzpicture}[node distance=\distnode,>=stealth',bend angle=45,auto]
 \begin{scope}
     \node [varfull]    (v1)                       {$\alpha_1$};
     \node [varfull]    (v2)  [below of=v1,xshift=-3cm] {$\alpha_2$};
     \node [varfull]    (v3)  [below of=v1,xshift=3cm] {$\alpha_3$};
     \node [varempty]    (v4)  [below of=v2,xshift=-1.5cm] {$\alpha_4$};
     \node [varfull]    (v5)  [below of=v2,xshift=1.5cm] {$\alpha_5$};
     \node [varfull]    (v6)  [below of=v3,xshift=-1.5cm] {$\alpha_6$};
     \node [varempty]    (v7)  [below of=v3,xshift=1.5cm] {$\alpha_7$};
     \node [varempty]    (v8)  [below of=v4,xshift=-0.7cm] {$\alpha_8$};
     \node [varempty]    (v9)  [below of=v4,xshift=0.7cm] {$\alpha_9$};
     \node [varfull]    (v10)  [below of=v5,xshift=-0.7cm] {$\alpha_{10}$};
     \node [varfull]    (v11)  [below of=v5,xshift=0.7cm] {$\alpha_{11}$};
     \node [varfull]    (v12)  [below of=v6,xshift=-0.7cm] {$\alpha_{12}$};
     \node [varempty]    (v13)  [below of=v6,xshift=0.7cm] {$\alpha_{13}$};
     \node [varempty]    (v14)  [below of=v7,xshift=-0.7cm] {$\alpha_{14}$};
     \node [varempty]    (v15)  [below of=v7,xshift=0.7cm] {$\alpha_{15}$};
     \node [group] [left of=v15,xshift=\distnode,minimum size=1.15cm]  { };
     \node [group] [left of=v14,xshift=\distnode,minimum size=1.15cm]  { };
     \node [group] [left of=v13,xshift=\distnode,minimum size=1.15cm]  { };
     \node [group] [left of=v12,xshift=\distnode,minimum size=1.15cm]  { };
     \node [group] [left of=v11,xshift=\distnode,minimum size=1.15cm]  { };
     \node [group] [left of=v10,xshift=\distnode,minimum size=1.15cm]  { };
     \node [group] [left of=v9,xshift=\distnode,minimum size=1.15cm]  { };
     \node [group] [left of=v8,xshift=\distnode,minimum size=1.15cm]  { };
     \node [group] [below of=v7,yshift=0.75cm,minimum height=2.8cm,minimum width=2.75cm]  { };
     \node [group] [below of=v6,yshift=0.75cm,minimum height=2.8cm,minimum width=2.75cm]  { };
     \node [group] [below of=v5,yshift=0.75cm,minimum height=2.8cm,minimum width=2.75cm]  { };
     \node [group] [below of=v4,yshift=0.75cm,minimum height=2.8cm,minimum width=2.75cm]  { };
     \node [group] [below of=v2,yshift=-0.05cm,minimum height=4.35cm,minimum width=5.9cm]  { };
     \node [group] [below of=v3,yshift=-0.05cm,minimum height=4.35cm,minimum width=5.9cm]  { };
     \node [group] [below of=v1,yshift=-0.8cm,minimum height=6cm,minimum width=12.1cm]  { };
     \draw [arrow] (v1) -- (v2);
     \draw [arrow] (v1) -- (v3);
     \draw [arrow] (v2) -- (v4);
     \draw [arrow] (v2) -- (v5);
     \draw [arrow] (v3) -- (v6);
     \draw [arrow] (v3) -- (v7);
     \draw [arrow] (v4) -- (v8);
     \draw [arrow] (v4) -- (v9);
     \draw [arrow] (v5) -- (v10);
     \draw [arrow] (v5) -- (v11);
     \draw [arrow] (v6) -- (v12);
     \draw [arrow] (v6) -- (v13);
     \draw [arrow] (v7) -- (v14);
     \draw [arrow] (v7) -- (v15);
  \end{scope}
  \end{tikzpicture}
  }
  \end{center}
  \caption{Illustration of the hierarchical sparsity of~\citet{zhao}, which
     generalizes the zero-tree coding scheme of~\citet{shapiro1993}.  The
     groups of variables correspond to the red rectangles. The empty nodes
     represent variable that are set to zero. They are contained in three groups:~$\{4,8,9\},\{13\},\{7,14,15\}$.
  }\label{fig:treesparse}
\end{figure}

\begin{figure}
    \centering
    \subfigure[Group-Lasso penalty \newline $\psi(\alphab)=\|\alphab{[1,2]}\|_2+|{\alphab}{[3]}|$.]{\includegraphics[width=0.45\linewidth]{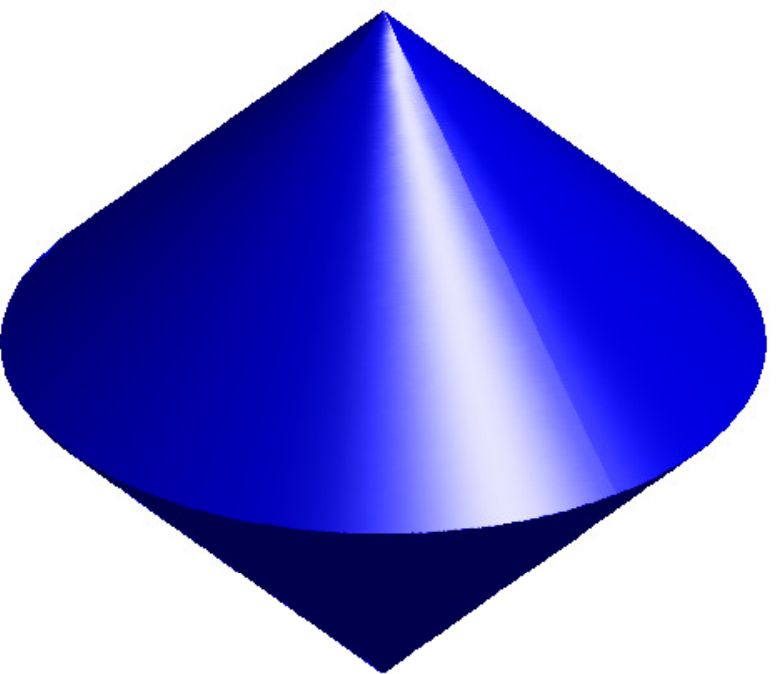}\label{subfig:norma}} \hfill
    \subfigure[Hierarchical penalty \newline $\psi(\alphab)=\|\alphab\|_2+|{\alphab}{[1]}|+|{\alphab}{[2]}|$.]{\includegraphics[width=0.40\linewidth]{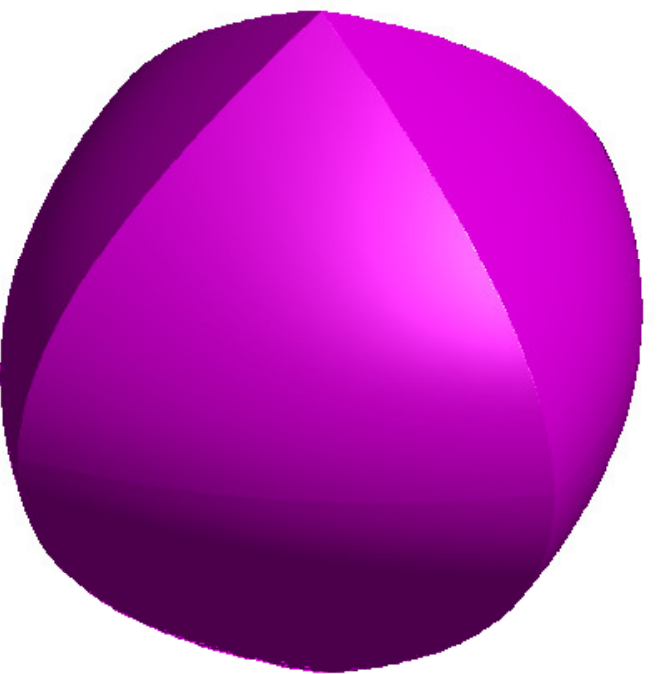}\label{subfig:normb}} \\
    \subfigure[Structured sparse penalty.]{\includegraphics[width=0.38\linewidth]{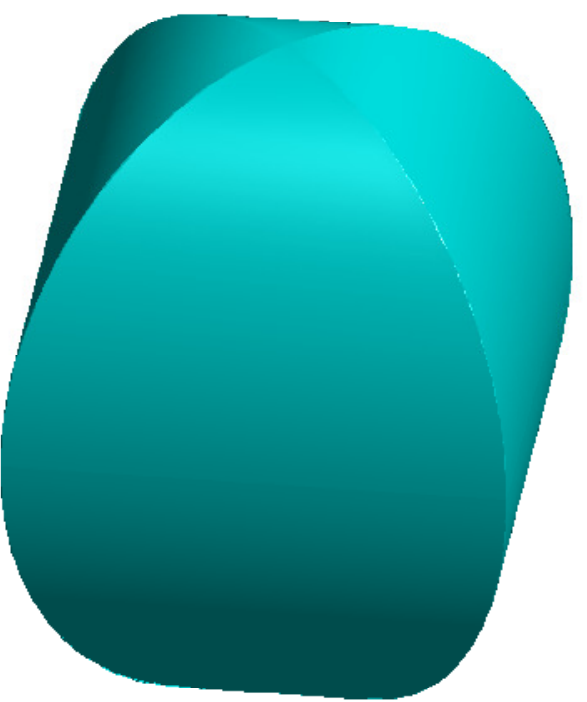}\label{subfig:normc}} \hfill
    \subfigure[Structured sparse penalty.]{\includegraphics[width=0.44\linewidth]{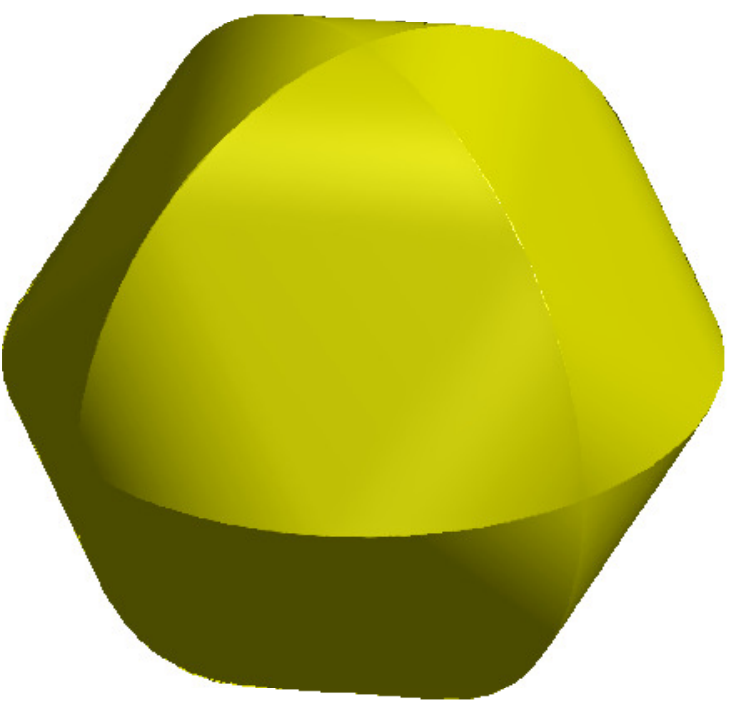}\label{subfig:normd}} 
    \caption{Visualization in three dimensions of unit balls corresponding to
    various sparsity-inducing norms. \subref{subfig:norma}: Group Lasso
 penalty; \subref{subfig:normb}: hierarchical penalty of~\citet{zhao};
 \subref{subfig:normc} and~\subref{subfig:normd}: examples of structured
 sparsity-inducing penalties of~\citet{jacob2009}. Figure borrowed
 from~\citet{bach2012}, produced by Guillaume Obozinski.}
    \label{fig:structured}
\end{figure}

\paragraph{Spectral sparsity.}
Another form of parsimony has been devised in the spectral
domain~\citep{fazel2001,Srebro2005}.  For estimation problems where model
parameters are matrices, the rank has been used as a natural regularization
function.  The rank of a matrix is equal to the number of non-zero singular
values, and thus, it can be interpreted as the~$\ell_0$-penalty of the matrix
spectrum.  Unfortunately, due to the combinatorial nature of~$\ell_0$, the rank
penalization typically leads to intractable optimization problems.

A natural convex relaxation has been introduced in the control theory
literature by~\citet{fazel2001} and consists of computing the~$\ell_1$-norm of
the spectrum---that is, simply the sum of the singular values.  The resulting
penalty appears under different names, the most common ones being the trace,
nuclear, or Schatten norm. It is defined for a matrix~$\A$ in~$\Real^{p \times
k}$ with~$k \geq p$ as
\begin{displaymath}
   \|\A\|_* \defin \sum_{i=1}^p s_i(\A),
\end{displaymath}
where~$s_i(\A)$ is the~$i$-th singular value of~$\A$.
Traditional applications of the trace norm in machine learning are matrix
completion or collaborative filtering~\citep{pontil,jake}.
These problems have become popular with the need of scalable recommender
systems for video streaming providers. The goal is to infer movie preferences
for each customer, based on their partial movie ratings. Typically, the matrix
is of size $p \times k$, where $p$ is the number of movies and~$k$ is the
number of users. Each user gives a score for a few movies, corresponding
to some entries of the matrix, and the recommender system tries to infer
the missing values. Similar techniques have also recently been used in other
fields, such as in genomics to infer missing genetic
information~\citep{chi2013}.

%% file: content_arxiv/intro_dict.tex
We have previously presented various formulations where a signal~$\x$
in~$\Real^m$ is approximated by a sparse linear combination of a few columns of
a matrix~$\D$ in~$\Real^{m \times p}$.  In the context of signal and image processing,
this matrix is often called \emph{dictionary} and its columns~\emph{atoms}.  As
seen in Section~\ref{sec:wavelets}, a large amount of work has been devoted in
the wavelet literature for designing a good dictionary adapted to natural
images.

In neuroscience,~\citet{field1996,olshausen1997} have proposed a significantly
different approach to sparse modeling consisting of adapting the dictionary to
training data.  Because the size of natural images is too large for learning a
full matrix~$\D$, they have chosen to learn the dictionary on natural image
patches, \eg, of size $m=16\times 16$ pixels, and have demonstrated that their
method could automatically discover interpretable structures. We discuss this
topic in more details in Section~\ref{chapter:patches}. 

The motivation of~\citet{field1996,olshausen1997} was to show that the
structure of natural images is related to classical theories of the mammalian
visual cortex. Later, dictionary learning found numerous applications in 
image restoration, and was shown to significantly
outperform off-the-shelf bases for signal reconstruction~\citep[see,
\eg,][]{elad2006,mairal2008,mairal2009,protter2009,yang2010}.

Concretely, given a dataset of $n$~training signals~$\X=[\x_1,\ldots,\x_n]$,
dictionary learning can be formulated as the following minimization problem
\begin{equation}
   \min_{\D \in \CC, \A \in \Real^{p \times n}} \sum_{i=1}^n
   \frac{1}{2}\|\x_i-\D\alphab_i\|_2^2 + \lambda \psi(\alphab_i),\label{eq:dictlearning}
\end{equation}
where~$\A = [\alphab_1,\ldots,\alphab_n]$ carries the decomposition coefficients of the signals~$\x_1,\ldots,\x_n$,
~$\psi$ is sparsity-inducing regularization function, and~$\CC$ is typically chosen
as the following set:
\begin{displaymath}
   \CC  \defin  \{ \D \in \Real^{m \times p} : \forall j~~ \|\d_j\|_2 \leq 1 \}.
\end{displaymath}
To be more precise, \citet{field1996} proposed
several choices for~$\psi$; their experiments were for instance conducted with the
$\ell_1$-norm, or with the smooth function $\psi(\alphab)\defin \sum_{j=1}^p
\log(\varepsilon + \alphab[j]^2)$, which has an approximate sparsity-inducing
effect.  The constraint $\D \in \CC$ was also not explicitly modeled in the
original dictionary learning formulation; instead, the algorithm
of~\citet{field1996} includes a mechanism to control and rescale the~$\ell_2$-norm of the
dictionary elements.
Indeed, without such a mechanism, the norm of~$\D$ would arbitrarily go to
infinity, leading to small values for the coefficients~$\alphab_i$ and making
the penalty~$\psi$ ineffective.

The number of samples~$n$ is typically large, whereas the signal dimension~$m$
is small.  The number of dictionary elements~$p$ is often chosen larger
than~$m$---in that case, the dictionary is said to be
\emph{overcomplete}---even though a choice~$p < m$ often leads to reasonable
results in many applications. For instance, a typical setting would be to have
$m=10 \times 10$ pixels for natural image patches, a dictionary of size~$p=256$,
and more than~$100\,000$ training patches.

A large part of this monograph is related to dictionary learning and thus 
we only briefly discuss this matter in this introduction. Section~\ref{chapter:patches}
is indeed devoted to unsupervised learning techniques for natural image
patches, including dictionary learning; Sections~\ref{chapter:image}
and~\ref{chapter:vision} present a large number of applications in image
processing and computer vision; how to solve~(\ref{eq:dictlearning}) is
explained in Section~\ref{sec:optimdict} about optimization.

\paragraph{Matrix factorization point of view.}
An equivalent representation of~(\ref{eq:dictlearning}) is the following
regularized matrix factorization problem
\begin{equation}
   \min_{\D \in \CC, \A \in \Real^{p \times n}} \frac{1}{2}\|\X-\D\A\|_{\fro}^2 + \lambda \Psi(\A),\label{eq:dictlearning2}
\end{equation}
where~$\Psi(\A) = \sum_{i=1}^n \psi(\alphab_i)$. Even though
reformulating~(\ref{eq:dictlearning}) as~(\ref{eq:dictlearning2}) is simply a
matter of using different notation, seeing dictionary learning as a matrix factorization
problem opens up interesting perspectives. In particular, it makes obvious 
some links with other unsupervised learning approaches such as non-negative
matrix factorization~\citep{paatero1994}, clustering techniques such as K-means, and
others~\citep[see][]{mairal2010}. These links will be further developed in
Section~\ref{chapter:patches}.

\paragraph{Risk minimization point of view.}
Dictionary learning can also be seen from a machine learning point of view.
Indeed, dictionary learning can be written as
\begin{equation*}
   \min_{\D \in \CC} \left\{ f_n(\D) \defin \frac{1}{n}\sum_{i=1}^n L(\x_i,\D) \right\},\label{eq:dictlearn3}
\end{equation*}
where~$L : \Real^m \times \Real^{m \times p}$ is a loss function defined as
\begin{displaymath}
   L(\x,\D) \defin \min_{\alphab \in \Real^p} \frac{1}{2}\|\x-\D\alphab\|_2^2 + \lambda\psi(\alphab).
\end{displaymath}
The quantity $L(\x,\D)$ should be small if $\D$ is ``good'' at representing the
signal $\x$ in a sparse fashion, and large otherwise. Then, $f_n(\D)$ is called
the \emph{empirical cost}. 

However, as pointed out by \citet{bottou2008tradeoffs}, one is usually not interested in the
exact minimization of the empirical cost $f_n(\D)$ for a fixed~$n$, which may
lead to overfitting on the training data, but instead in the minimization of
the {\em expected cost}, which measures the quality of the dictionary on new
unseen data:
\begin{equation*}
   f(\D) \defin \EE_{\x}[L(\x,\D)] = \lim_{n \to \infty} f_n(\D) \as,\label{eq:expect}
\end{equation*}
where the expectation is taken relative to
the (unknown) probability distribution of the
data.\footnote{We use ``a.s.'' to denote almost sure convergence.}

The expected risk minimization formulation is interesting since it paves the
way to stochastic optimization techniques when a large amount of data is available \citep{mairal2010} and to theoretical
analysis \citep{maurer2010dimensional,vainsencher2011,gribonval2013}, which are developed in
Sections~\ref{sec:optimdict} and~\ref{subsec:theory}, respectively.

\paragraph{Constrained variants.}
Following the original formulation of \citet{field1996,olshausen1997}, we have
chosen to present dictionary learning where the regularization function is used
as a penalty, even though it can also be used as a constraint as in~(\ref{eq:lasso}). Then, natural variants of~(\ref{eq:dictlearning}) are
\begin{equation}
   \min_{\D \in \CC, \A \in \Real^{p \times n}} \sum_{i=1}^n \frac{1}{2}\|\x_i-\D\alphab_i\|_2^2  \st \psi(\alphab_i) \leq \mu.\label{eq:dictlearning4}
\end{equation}
or 
\begin{equation}
   \min_{\D \in \CC, \A \in \Real^{p \times n}} \sum_{i=1}^n \psi(\alphab_i)  \st \|\x_i-\D\alphab_i\|_2^2 \leq \varepsilon.\label{eq:dictlearning3}
\end{equation}
Note that~(\ref{eq:dictlearning4}) and~(\ref{eq:dictlearning3}) are not
equivalent to~(\ref{eq:dictlearning}). For instance, problem~(\ref{eq:dictlearning4}) can be reformulated using a Lagrangian function~\citep{boyd.convex} as
\begin{displaymath}
   \min_{\D \in \CC}  \sum_{i=1}^n \left(\max_{\lambda_i \geq 0} \min_{\alphab_i \in \Real^p} \frac{1}{2}\|\x_i-\D\alphab_i\|_2^2  +\lambda_i (\psi(\alphab_i)-\mu)\right),
\end{displaymath}
where the optimal~$\lambda_i$'s are not necessarily equal to each other, and
their relation with the constraint parameter~$\mu$ is
unknown in advance. A similar discussion can be conducted
for~(\ref{eq:dictlearning3}) and it is thus important in practice to choose one of the
formulations~(\ref{eq:dictlearning}), (\ref{eq:dictlearning4}), or
(\ref{eq:dictlearning3}); the best one depends on the problem at hand and there is no general rule for preferring one instead of another.

%% file: content_arxiv/intro_cs.tex
Finally, we conclude our historical tour of parsimony with recent theoretical
results obtained in signal processing and statistics. We focus on methods based
on the $\ell_1$-norm, \ie, the basis pursuit formulation
of~(\ref{eq:lasso})---more results on structured sparsity-inducing norms
are presented by~\citet{bach2012b}.

Most analyses rely on particular assumptions regarding the problem. We start this section
with a cautionary note from \citet{hocking1976}:
\begin{quote}
   \emph{The problem of selecting a subset of independent or predictor variables is usually described in an idealized setting. That is, it is assumed that (a) the analyst has data on a large number of potential variables which include all relevant variables and appropriate functions of them plus, possibly, some other extraneous variables and variable functions and (b) the analyst has available ``good'' data on which to base the eventual conclusions. In practice, the lack of satisfaction of these assumptions may make a detailed subset selection analysis a meaningless exercise.}
\end{quote}
In this section, we present such theoretical results where the assumptions are often not met in practice, but also results that  either (1) can have an impact on the practice of sparse recovery or (2) do not need strong assumptions.

\paragraph{From support recovery to signal denoising.}
Given a signal $\x$ in $\Real^m$ and a dictionary $\D$ in $\Real^{m \times p}$ with $\ell_2$-normalized columns, throughout this section, we assume that~$\x$ is generated as  $\x = \D \alphab^\star + \boldsymbol{\varepsilon}$ with a sparse vector $\alphab^\star$ in $\Real^p$ and an additive noise $\boldsymbol{\varepsilon}$ in $\Real^m$. For simplicity, we consider $\alphab^\star$ and $\D$ as being deterministic while the noise is random, independent and identically distributed, with zero mean and finite variance $\sigma^2$.

The different formulations presented earlier in Section~\ref{sec:l1}, for
instance basis pursuit, provide estimators $\hatalphab$ of the ``true''
vector~$\alphab^\star$. Then, the three following goals have been studied in
sparse recovery, typically in decreasing order of hardness:
\begin{itemize}
\item \textbf{support recovery and sign consistency}: we want the support of $\hatalphab$ (\ie, the set of non-zero elements) to be the same or to be close to the one of $\alphab^\star$. The problem is often called ``model selection'' in statistics and ``support recovery'' in signal processing; it is often refined to the estimation of the full sign pattern---that is, among the non-zero elements, we also want the correct sign to be estimated.
\item \textbf{code estimation}: the distance $\| \hatalphab -
   \alphab^\star\|_2$ should be small. In statistical terms, this correspond to
   the ``estimation'' of $\alphab^\star$. 
\item \textbf{signal denoising}: regardless of code estimation, we simply want
   the distance $\| \D \hatalphab - \D \alphab^\star\|_2$ to be small; the goal
   is not to obtain exactly $\alphab^\star$, but simply to obtain a good
   denoised version of the signal $\x  = \D \alphab^\star + \boldsymbol{\varepsilon}$. 
\end{itemize}

A good code estimation performance does imply a good denoising performance but
the converse is not true in general. In most analyses, support recovery is
harder than code estimation. As detailed below, the sufficient conditions for
good support recovery lead indeed to good estimation.

\paragraph{High-dimensional phenomenon.} 
Without sparsity assumptions, even the simplest denoising task can only
achieve denoising errors of the order 
$\frac{1}{n}\| \D \hatalphab - \D \alphab^\star\|_2^2 \approx \frac{ \sigma^2 p}{m}$, which is attained for ordinary least-squares, and is the best possible~\citep{tsybakov2003optimal}.
Thus, in order to have at least a good denoising performance (prediction
performance in statistics), either the noise~$\sigma$ is small, or the signal
dimension $m$ (the number of samples) is much larger than the
number of atoms $p$ (the number of variables to select from).

When making the assumption that the true code~$\alphab^\star$ is sparse  
with at most $k$ non zeros, smaller denoising errors can be obtained.
In that case, it is possible indeed to replace the scaling $\frac{ \sigma^2 p}{m}$ 
by $\frac{ \sigma^2 k \log p}{m}$. Thus, even when $p$ is much larger
than $m$, as long as $\log p$ is much smaller than $m$, we may have good
prediction performance. However, this high-dimensional phenomenon
currently\footnote{Note that recent research suggests that this
fast rate of $\frac{ \sigma^2 k \log p}{m}$ cannot be achieved by
polynomial-time algorithms~\citep{zhang2014lower}.}  comes at a price: (1)
either an exhaustive search over the subsets of size $k$ needs to be
performed~\citep{massart-concentration,bunea2007aggregation,raskutti2011minimax}
or (2) some assumptions have to be made regarding the dictionary $\D$, which
we now describe.

\paragraph{Sufficient conditions for high-dimensional fast rates.}
Most sufficient conditions have the same flavor. A dictionary behaves well
if the off-diagonal elements of $\D^\top \D$ are small, in other words, if
there is little correlation between atoms. However, the notion of coherence
(the maximal possible correlation between two atoms) was the first to
emerge~\citep[see, \eg,][]{elad2002generalized,gribonval2003sparse}, but it is
not sufficient to obtain a high-dimensional phenomenon.

In the noiseless setting, \citet{candes2005decoding} and \citet{candes2006}
introduced the \emph{restricted isometry property} (RIP), which states that all
submatrices of size $k \times k$ of $\D^\top \D$ should be close to isometries,
that is, should have all of their eigenvalues sufficiently close to one. With
such an assumption, the Lasso behaves well: it recovers the true support and estimates the code $\alphab^\star$ and the signal~$\D\alphab^\star$ with an error of order $\frac{ \sigma^2 k \log p}{m}$.

The main advantage of the RIP assumption is that one may exhibit dictionaries
for which it is satisfied, usually obtained by normalizing a matrix $\D$
obtained from independent Gaussian entries, which may satisfy the condition that $ (k
\log p ) / m$  remains small. Thus, the sufficient conditions are not
vacuous. However, the RIP assumption has two main drawbacks: first,
it cannot be checked on a given dictionary $\D$ without
checking all $O(p^k)$ submatrices of size $k$; second, it may be weakened if
the goal is support recovery or simply estimation performance (code recovery).

There is therefore a need for sufficient conditions that can be checked in
polynomial time while ensuring sparse recovery. However, none currently exists
with the same scalings between $k$, $p$ and $m$ \citep[see,
\eg,][]{juditsky2011verifiable,d2011testing}. When refining to support
recovery,~\citet{fuchs,tropp,martin} provide sufficient and necessary conditions
of a similar flavor than requiring that all submatrices of size $k$ are
sufficiently close to orthogonal. For the tightest conditions, see, \eg,~\citet{buhlmann2011statistics}. Note that these conditions are also
typically sufficient for algorithms that are not based explicitly on convex
optimization~\citep{tropp}.

Finally, it is important to note that (a) most of the theoretical results advocate a value for the regularization parameter $\lambda$ proportional to
$\sigma \sqrt{ m \log p}$, which unfortunately depends on the noise level $\sigma$ (which is typically unknown in practice), and that (b) for orthogonal dictionaries, all of these assumptions are met; however, this imposes $p=m$.

\paragraph{Compressed sensing vs. statistics.}
Our earlier quote from \citet{hocking1976} applies to sparse estimation as used
in statistics for least-squares regression, where the dictionary $\D$ is simply
the input data and $\x$ the output data. In most situations, there are some
variables, represented by columns of $\D$, that are heavily correlated.
Therefore, in most practical situations, the assumptions do not apply. However,
it does not mean that the high-dimensional phenomenon does not apply in a
weaker sense (see the next paragraph for slow rates); moreover it is important to
remark that there are other scenarios, beyond statistical variable selection,
where the dictionary $\D$ may be chosen. 

In particular, in signal processing, the dictionary $\D$ may be seen as
\emph{measurements}---that is, we want to encode $\alphab^\star$ in $\Real^p$
using~$m$ linear measurements $ \D \alphab^\star$ in $\Real^m$ for $m$
much larger than $p$. What the result of \citet{candes2005decoding} alluded to
earlier shows is that for random measurements, one can recover a
$k$-sparse $\alphab^\star$ from (a potentially noisy version of) $\D \alphab^\star$,
with overwhelming probability, as long as $ (k \log p ) / m$  remains small.
This is  the core idea behind compressive sensing. See more details
from~\citet{donoho2006,candes2008introduction}.

\paragraph{High-dimensional slow rates.} 
While sufficient conditions presented earlier are often not met beyond random
dictionaries, for the basis pursuit/Lasso formulation from~(\ref{eq:lasso}), the
high-dimensional phenomenon may still be observed, but only for the
denoising situation and with a weaker result. Namely, as shown
by~\citet{greenshtein2006best} and~\citet[][Corollary
6.1]{buhlmann2011statistics},  \emph{without assumptions regarding
correlations}, we have $\frac{1}{n}\| \D \hatalphab - \D \alphab^\ast\|_2^2
\approx \sqrt{ \frac{  \sigma^2 k^2 \log p }{m} }$. 
Note that this slower rate does not readily extend to non-convex formulations.

\paragraph{Impact on dictionary learning.}
The dictionary learning framework which we describe in this monograph relies on
sparse estimation, that is, given the dictionary $\D$, the estimation of the
code $\alphab$ may be analyzed using the tools we have presented in this
section. However, the dictionaries that are learned do not exhibit low
correlations between atoms and thus theoretical results do not apply (see
dedicated results in the next section). However, they suggest that (a) the
codes $\alphab$ may not be unique in general and caution has to be observed
when representing a signal $\x$ by its code $\alphab$, (b) methods based on
$\ell_1$-penalization are more robust as they still provably perform denoising
in presence of strong correlations and (c) incoherence promoting may be used in order
to obtain better-behaved dictionaries~\citep[see, e.g.,][]{ramirez2009sparse}.

%% file: content_arxiv/intro_thdict.tex
Dictionary learning, as formulated in Eq.~(\ref{eq:dictlearning}), may be seen
from several perspectives, mainly as an unsupervised learning or a matrix
factorization problem. While the supervised learning problem from the previous
section (sparse estimation of a single signal given the dictionary) comes with
many theoretical analyses, there are still few theoretical results of the same
kind for dictionary learning. In this section, we present some
of them. For simplicity, we assume that we penalize with the $\ell_1$-norm and
consider the minimization of \begin{equation}
   \min_{\D \in \CC, \A \in \Real^{p \times n}} \sum_{i=1}^n
   \frac{1}{2}\|\x_i-\D\alphab_i\|_2^2 + \lambda \| \alphab_i \|_1,\label{eq:dictlearningth}
\end{equation}
where~$\A = [\alphab_1,\ldots,\alphab_n]$ carries the decomposition coefficients of the signals~$\x_1,\ldots,\x_n$,
 and~$\CC$ is   chosen
as the following set:
\begin{displaymath}
   \CC  \defin  \{ \D \in \Real^{m \times p} : \forall j~~ \|\d_j\|_2 \leq 1 \}.
\end{displaymath}

\paragraph{Non-convex optimization problem.}

After imposing parsimony through the $\ell_1$-norm, given $\D$ the objective
function is convex in $\alphab$, given $\alphab$ the objective and constraints
are convex in $\D$. However, the objective function is not jointly convex,
which is typical of unsupervised learning formulations. Hence, we consider an
optimization problem for which it is not possible in general to guarantee that
we are going to obtain the global minimum; the same applies to EM-based
approaches~\citep{dempster1977} or K-means~\citep[see, \eg,][]{bishop2006pattern}.

\paragraph{Symmetries.}
Worse, the problem in Equation~(\ref{eq:dictlearningth}) exhibits several
symmetries and admits multiple global optima, and the descent methods that are
described in Section~\ref{chapter:optim} will also have the same invariance
property.  For example, the columns of~$\D$ and rows of~$\A$ can be submitted
to~$p!$ arbitrary (but consistent) permutations. There are also sign
ambiguities: in fact, if~$(\D,\A)$ is solution
of~(\ref{eq:dictlearningth}), so is $(\D\diag(\varepsilonb),\diag(\varepsilonb)\A)$, where $\varepsilonb$ is a vector in~$\{-1,+1\}^p$ that
carries a sign pattern. Therefore, for every one of the $p!$ possible atom
orders, the dictionary learning problem admits $2^p$ equivalent solutions.  In
other words, for a solution $(\D,\A)$, the pair
$(\D\Gammab,\Gammab^{-1}\A)$ is also solution, where $\Gammab$ is a {\em generalized
permutation} formed by the product of a diagonal matrix with $+1$ and~$-1$'s on its
diagonal with a permutation matrix (in particular, $\Gammab$ is thus
orthogonal). 

The fact that there are no other transformations $\Gammab$ such that
$(\D\Gammab,\Gammab^{-1}\A)$ is also solution of Eq.~(\ref{eq:dictlearningth})
for {\em all} solutions $(\D,\A)$ of this problem follows from a general
property of isometries of the $\ell_q$ norm for finite values of $q$ such that
$q\ge 1$ and $q\neq 2$~\citep{LiSo94}.

\paragraph{A manifold interpretation of sparse coding with projective geometry.}
The interpretation of sparse coding as a {\em locally} linear representation of
a non-linear ``manifold'' is problematic because certain signals/features are
best thought of as ``points'' in some space rather than vectors. For example,
what does it mean to ``add'' two natural image patches? The simplest point
structure that one can think of is affine or projective, and we show below that
sparse coding indeed admits a natural interpretation in this setting, at least
for normalized signals.

Indeed, Let us restrict our attention from now on to unit-norm signals,
as is customary in image processing after the usual centering and
normalization steps, which will be studied in
Section~\ref{subsec:preprocess}.\footnote{Note that the fact that the
individual signals are centered does not imply that the dictionary elements
are.} Note that the dictionary elements $\d_j$ in a solution
$\D=[\d_1,\ldots,\d_p]$ of Eq.~(\ref{eq:dictlearningth}) also have unit norm by
construction.

Let us now consider the ``half sphere''
\begin{equation}
   \SS^{m-1}_+ \defin \left\{ \d \in \SS^{m-1} : \text{the first non-zero coefficient of~$\d$ is positive}\right\}, \label{eq:halfsphere}
 \end{equation}
where $\SS^{m-1}$ is the unit sphere of dimension $m-1$ formed
by the unit vectors of $\Real^m$.\footnote{Similarly, one may define the set~$\SS^{m-1}_-$ by replacing ``positive'' by ``non-negative'' in~(\ref{eq:halfsphere}). The two sets~$\SS^{m-1}_+$ and~$\SS^{m-1}_-$ form a partition of~$\SS^{m-1}$ with equal volume, and indeed, each one geometrically corresponds to a half sphere.}
A direct consequence of the sign ambiguities of dictionary
learning discussed in the previous paragraph is that, for any solution
$(\D,\A)$ of Eq.~(\ref{eq:dictlearningth}), there is an equivalent solution
$(\D',\A')$ with all columns of $\D'$ in $\SS^{m-1}_+$.
Indeed, suppose some column $\d_j$ is not in
$\SS^{m-1}_+$, and let $\d_j[i]$ be its first non-zero coefficient (which is
necessarily negative since~$\d_j \notin \SS^{m-1}_+$).  We can replace $\d_j$ by $-\d_j$ and the corresponding
row of the matrix $\A$ by its opposite to construct an equivalent minimum of
the dictionary learning problem in $\SS^{m-1}_+\times\Real^{p\times n}$.

Likewise, we can restrict the signals $\x_i$ to lie in
$\SS^{m-1}_+$ since replacing~$\x_i$ by its opposite for a given
dictionary simply amounts to replacing the code $\alphab_i$ by
its opposite. Note that this identifies a patch with its ``negative'',
but remember that the sign of its code elements is not uniquely
defined in the first place in conventional dictionary learning
settings (it is uniquely defined if we insist that the dictionary
elements belong to $\SS^{m-1}_+$).
This allows us to identify both the dictionary elements and the signals
with points in the projective space $\PPP^{m-1}=P(\Real^{m})$. Any~$k$ 
independent column vectors of $\D$ define a $(k-1)$-dimensional
projective subspace of $\PPP^{m-1}$ (see Figure~\ref{fig:proj}).

\begin{figure}[hbtp]
\centerline{%
\includegraphics[width=0.8\textwidth]{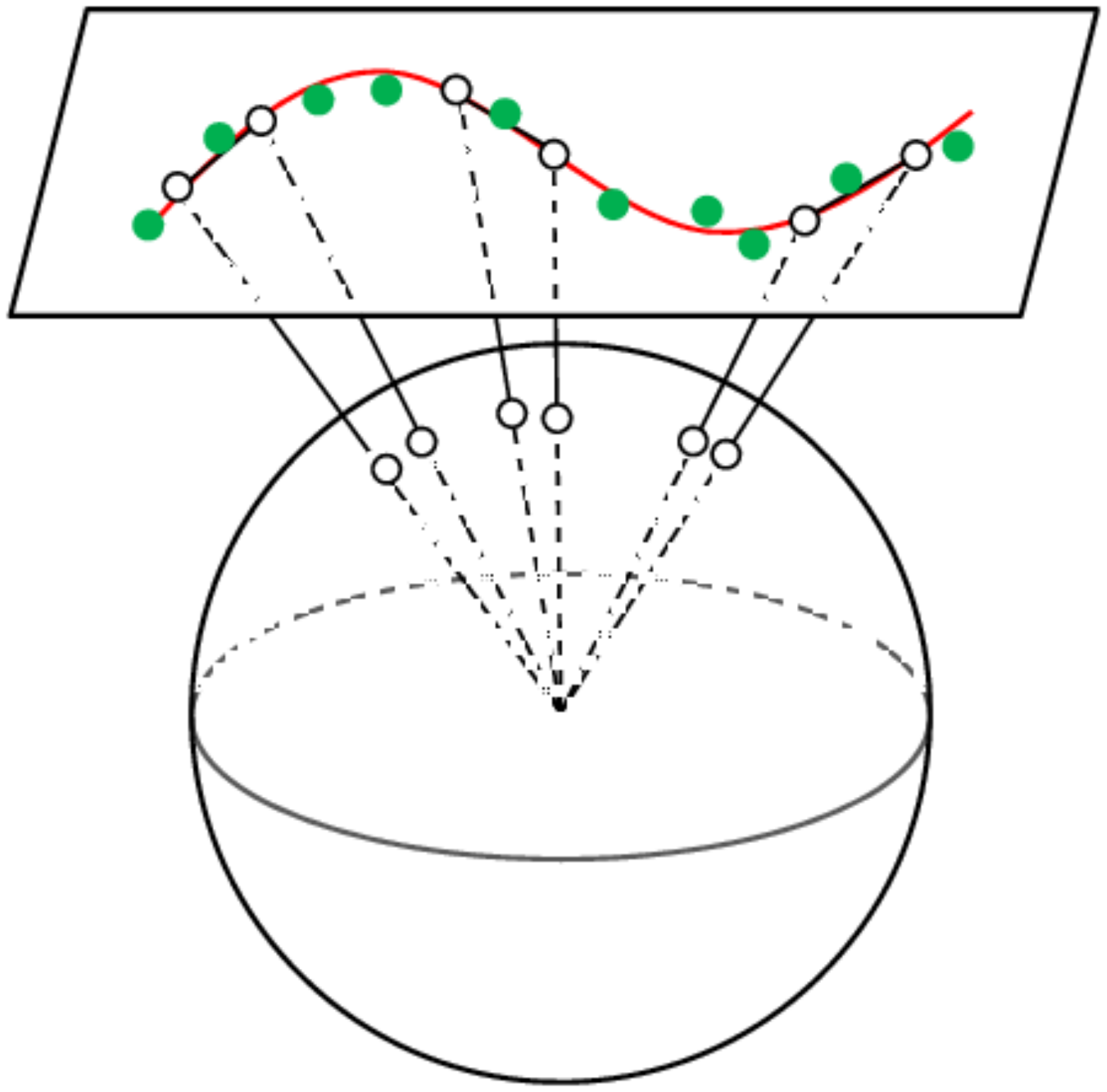}}
\caption{\small An illustration of the projective interpretation
of sparse coding.}
\label{fig:proj}
\end{figure}
 
In particular, if the data signals are assumed to be sampled from a
 ``noisy manifold'' of dimension $k-1$ embedded in $\PPP^{m-1}$, an
approximation of some sample $\x$ by a sparse linear combination
of $k$ elements of $\D$ can be thought of as lying in (or near)
the $k-1$ dimensional ``tangent plane'' there.

\paragraph{Consistency results.}
Given the dictionary learning problem from a finite number of signals, there
are several interesting theoretical questions to be answered. The first natural
question is to understand the properties of the cost function that is minimized
when the number of signals tends to infinity, and in particular how it
converges to the expectation under the signal generating
distribution~\citep{vainsencher2011,maurer2010dimensional}. Then, given
the non-convexity of the optimization problems, local consistency results may
be obtained, by showing that the cost function which is minimized has a local
minimum around the pairs $(\D^\star,\A^\star)$ that has generated the data. Given
RIP-based assumptions on the dictionary $\D^\star$ and number of non zero elements in the columns of
$\A^\star$, and the noise level, \citet{gribonval2014sparse} show that the cost
function defined in Eq.~(\ref{eq:dictlearningth}) has a local minimum around
$(\D^\star,\A^\star)$  with high probability, as long as the number of signals
$n$ is greater than a constant times $m p^3$.
In the noiseless case, earlier results have been also obtained~\citet{gribonval2010dictionary,geng2011local},
and recently it has been shown that under additional
assumptions, a good initializer could be found so that the previous type of
local consistency results can be applied~\citep{agarwal2013learning}. 

Finally, recent algorithms have emerged in the theoretical science community,
which are not explicitly based on optimization~\citep[see,
\eg,][]{spielman2013exact,recht2012factoring,arora2013new}. These come with
global convergence guarantees (with additional assumptions regarding the
signals), but their empirical performance on concrete signal and image processing
problems have not yet been demonstrated.

%% file: content_arxiv/natstat_intro.tex
Dictionary learning was first introduced by \citet{field1996,olshausen1997} as
an unsupervised learning technique for discovering and visualizing the
underlying structure of natural image patches. The results were found
impressive by the scientific community, and dictionary learning gained
early success before finding numerous applications in image
processing~\citep{elad2006,mairal2008,mairal2009,yang2010}. Without making any a priori
assumption about the data except a parsimony principle, the
method is able to produce dictionary elements that resemble Gabor
wavelets---that is, spatially localized oriented basis
functions~\citep{gabor1946,daugman1985}, as illustrated for example in
Figure~\ref{fig:gabors}.

Surprisingly, the goal of automatically learning local structures in natural
images was neither originally achieved in image processing, nor in
computer vision. The original motivation
of~\citet{field1996,olshausen1997} was in fact to establish a relation between
the statistical structure of natural images and the properties of neurons from
area V1 of the mammalian visual cortex.\footnote{Neuroscientists have identified several areas
in the human visual cortex. Area V1 is considered to be at the earliest
stage of visual processing.} The emergence of Gabor-like patterns from natural
images was
considered a significant result by neuroscientists. To better understand this
fact, we will briefly say a few words about the importance of Gabor models
within experimental studies of the visual cortex.

Since the pioneer work of \citet{hubel1968}, it is known that some visual
neurons are responding to particular image features, such as oriented edges.
Specifically, displaying oriented bars at a particular location in a visual
field's subject may elicit neuronal activity in some cells. Later,
\citet{daugman1985} demonstrated that fitting a linear model to neuronal
responses given a visual stimuli may produce filters that can be
well approximated by a two-dimensional Gabor function.
Models based on such filters have been subsequently widely used, and are still
present in state-of-the-art predictive models of the neuronal activity in
V1~\citep{kay2008,nishimoto2011}. By showing that Gabor-like visual patterns
could be automatically obtained from the statistics of natural images,
\citet{field1996,olshausen1997} provided support for the classical Gabor
model for V1 cells, even though such an intriguing phenomenon does not
constitute a strong evidence that the model matches real biological processes.
It is indeed commonly admitted that little is known about the early visual
cortex~\citep{olshausen2005,carandini2005} despite intensive studies and
impressive results achieved by predictive models.

In this chapter, we show how to use dictionary learning for discovering latent
structures in natural images, and we discuss other unsupervised learning
techniques. We present them under the unified point of view of \emph{matrix
factorization}, which makes explicit links between different learning methods.
Given a training set $\X=[\x_1,\ldots,\x_n]$ of signals---here, natural image
patches---represented by vectors~$\x_i$ in~$\Real^m$, we wish to find an
approximation~$\D\alphab_i=\sum_{j=1}^p \alphab_i[j]\d_j$ for each signal~$\x_i$, where
$\D=[\d_1,\ldots,\d_p]$ is a matrix whose columns are called ``dictionary elements'', and the
vectors~$\alphab_i$ are decomposition coefficients. In other words, we wish to
find a factorization
\begin{displaymath}
    \X \approx \D \A,
\end{displaymath}
where the matrix~$\A=[\alphab_1,\ldots,\alphab_n]$ in~$\Real^{p \times n}$
carries the coefficients~$\alphab_i$. Many unsupervised learning techniques can be
cast as matrix factorization problems; this includes dictionary learning,
principal component analysis (PCA), clustering or vector
quantization~\citep[see][]{nasrabadi1988,gersho1992}, non-negative matrix
factorization (NMF)~\citep{paatero1994,lee1999}, archetypal
analysis~\citep{cutler1994}, or independent component analysis
(ICA)~\citep{herault1985,bell1995,bell1997,hyvarinen2004}.  These methods
essentially differ in a priori assumptions that are made on~$\D$ and~$\A$
(sparse, low-rank, orthogonal, structured), and in the way the quality of the
approximation $\X \approx \D \A$ is measured. We will see in the rest of this
chapter how these criterions influence the results obtained on natural image
patches.

\begin{figure}
   \centering
   \subfigure[2D Gabor filter.]{\includegraphics[width=0.32\linewidth]{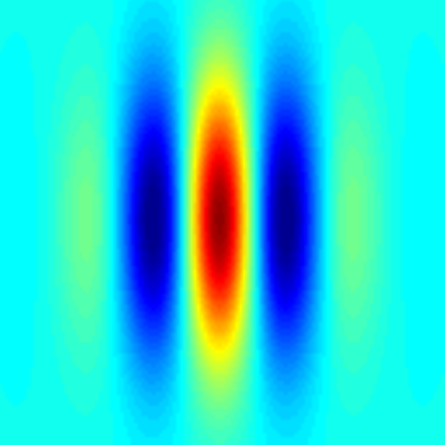}\label{subfig:gabora}} \hfill
   \subfigure[With shifted phase.]{\includegraphics[width=0.32\linewidth]{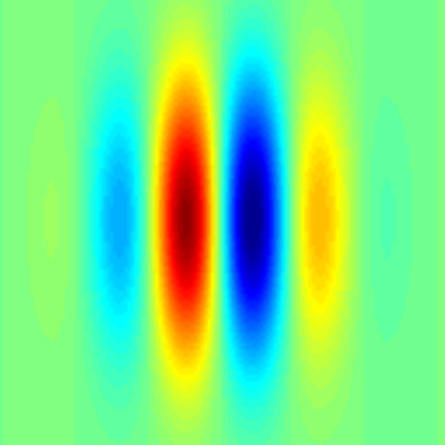}\label{subfig:gaborb}} \hfill
   \subfigure[With rotation.]{\includegraphics[width=0.32\linewidth]{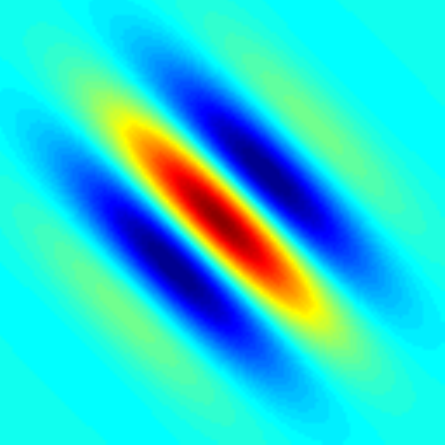}\label{subfig:gaborc}}  
   \caption{Three examples of two-dimensional Gabor filters. Red represents
      positive values and blue negative ones. The first two filters 
      correspond to the function $\displaystyle g(x,y)=
      e^{-x^2/(2\sigma_x^2)-y^2/(2\sigma_y^2)}\cos(\omega x + \varphi)$, with
      $\varphi=0$ for Figure~\subref{subfig:gabora} and $\varphi=\pi/2$ for
      Figure~\subref{subfig:gaborb}. Figure~\subref{subfig:gaborc} was obtained
      by rotating Figure~\subref{subfig:gabora}---that is, replacing $x$ and
      $y$ in the function~$g$ by~$x' = \cos(\theta)x + \sin(\theta)y$
      and~$y'=\cos(\theta)x - \sin(\theta)y$. The figure is best seen in color on a computer screen.
   }\label{fig:gabors}
\end{figure}

%% file: content_arxiv/natstat_preprocessing.tex
When manipulating data, it is often important to choose an appropriate
pre-processing scheme. There are several reasons for changing the way
data looks like before using an unsupervised learning algorithm. For
example, one may wish to reduce the amount of noise, make the data invariant to
some transformation, or remove a confounding (unwanted) factor. Unfortunately,
the literature is not clear about which pre-processing step should be applied
given a specific problem at hand. In this section, we describe common practices
when dealing with natural image patches, such as centering, contrast
normalization, or whitening.  We study the effect of these procedures on the
image itself, and provide visual interpretation that should help the practitioner
make his choice.

\paragraph{Centering.}
The most basic pre-processing scheme consists of removing the mean intensity
from every patch. Such an approach is called ``centering the data'' in statistics
and machine learning. Borrowing some terminology from electronics, it is also called ``removing the DC
component'' in the signal processing literature. Formally, it means applying
the following update to every signal~$\x_i$:
\begin{displaymath}
   \x_i  \leftarrow \x_i  -  \left(\frac{1}{m} \sum_{j=1}^m \x_i[j]\right)\ones_m, 
\end{displaymath}
where~$\ones_m$ is the vector of size $m$ whose entries are all ones.  It can
also be interpreted as performing a left multiplication on the matrix~$\X$:
\begin{displaymath}
   \X  \leftarrow \left(\I - \frac{1}{m}\ones_m \ones_m^\top\right) \X.
\end{displaymath}
When analyzing the structure of natural patches, one is often interested in 
retrieving geometrical visual patterns such as edges, which is a concept related
to intensity variations and thus invariant to the mean intensity.
Therefore, centering can be safely used; it makes the data invariant to the mean
intensity, and the learned structures are expected to have zero mean as well. 
The effect of such a pre-processing scheme can be visualized
in Figure~\ref{subfig:man_center} for the image~\textsf{man}. We extract all
overlapping patches from the image, remove their average pixel value, and recombine the
centered patches into a new image, which is displayed in
Figure~\ref{subfig:man_center}.  Because of the overlap, every pixel is
contained in several patches; the pixel values in the reconstructed image are
thus obtained by averaging their counterparts from the centered patches. The
procedure is in fact equivalent to applying a high-pass filter to the original
image. As a result, the geometrical structures are still present in the
centered image, but some low frequencies have disappeared.

In practice, centering the data has been found useful in concrete applications,
such as image denoising with dictionary learning \citep{elad2006}.
First removing the mean component before learning the patches structure, and
then adding the mean back after processing, leads to substantially better
results than working with the raw patches directly. 

\paragraph{Variance - contrast normalization.}
After centering, another common practice is to normalize the signals to
have unit $\ell_2$-norm, or equivalently unit variance. In computer vision,
the terminology of ``contrast normalization'' is also often used. One 
motivation for such a pre-processing step is to make different regions of an
image more homogeneous and to gain invariance to illumination changes.
Whereas this is probably not useful for image reconstruction problems, it is a
key component of many visual recognition
architectures~\citep[][]{pinto2008,jarrett2009}. 
However, contrast normalization should be applied with care: the naive update $\x_i \leftarrow \x_i
/ \|\x_i\|_2$ is likely to provide poor results for patches from uniform areas---\eg, from the sky in natural scenes.  After centering, such patches do not
contain any information about the scene anymore but carry residual noise,
which will be greatly amplified by the naive $\ell_2$-normalization.
A simple way of fixing that issue consists in providing a different treatment for patches
with a small norm---say, smaller than a parameter~$\eta$---and use the following update instead:
\begin{displaymath}
   \x_i \leftarrow  \frac{1}{\max(\|\x_i\|_2,\eta)}\x_i.
\end{displaymath}
We present in Figure~\ref{subfig:man_norm} the effect of contrast normalization
where~$\eta$ is chosen to be~$0.2$ times the mean value of $\|\x_i\|_2$ in the
image, following the same patch recombination scheme as in the previous
paragraph.  Compared to Figure~\ref{subfig:man_norm}, areas with small
intensity variations have been amplified, making the image more homogeneous.
Contrast normalization can also be applied after a whitening step, which we now
present.

\paragraph{Whitening and dimensionality reduction.}
Centering consists of removing from data the first-order statistics for
every patch. Some early work dealing with natural images patches
goes a step further by removing global second-order statistics,
a procedure called
\emph{whitening} \citep{field1996,olshausen1997,bell1997,hyvarinen2009}.

Let us consider a random variable~$\x$ representing centered patches uniformly
sampled at random from natural images.  Even though the centering step is
performed at the patch level, the pixel values across centered natural images
patches have zero mean~$\mub \defin \EE[\x]=0$. The different pixels of a patch
are indeed identically distributed, and thus the vector~$\mub$ is 
constant; it is then necessarily equal to the constant zero vector since it has zero mean after
centering. Therefore, the covariance matrix of~$\x$ is the quantity $\Sigmab \defin
\EE[\x\x^\top]$. Whitening consists of finding a data transformation such
that~$\Sigmab$ becomes close to the identity matrix. As a result, the pixel
values within a whitened patch are uncorrelated.

In practice, whitening is performed by computing the eigenvalue decomposition of
the sample covariance matrix~$(1/n)\sum_{i=1}^n \x_i \x_i^\top = \U \S^2
\U^\top$, where $\U$ is an orthogonal matrix in~$\Real^{m \times m}$, and $\S$
in~$\Real^{m \times m}$ is diagonal with non-negative entries. Since the sample
covariance matrix is positive semi-definite, the eigenvalues are
non-negative and~$\S=\diag(s_1,\ldots,s_m)$ carries in fact the \emph{singular
values} of the matrix~$(1/\sqrt{n})\X$.
Then, whitening consists of applying the following update to every patch~$\x_i$:
\begin{displaymath}
   \x_i \leftarrow \U \S^\dagger \U^\top \x_i,
\end{displaymath}
where $\S^\dagger=\diag(s_1^\dagger,\ldots,s_m^\dagger)$ with $s_j^\dagger = 1/s_j$
if~$|s_j| > \varepsilon$ and~$0$ otherwise, where~$\varepsilon$ is a small threshold
to prevent arbitrarily large entries in~$\S^\dagger$. It is then easy to show
that the sample covariance matrix after whitening is diagonal with
$k=\rank(\Sigmab)$ entries equal to~$1$. Choosing~$\varepsilon$
can be interpreted as controlling the dimensionality reduction effect of the
whitening procedure.  According to~\citet{hyvarinen2009}, such a step is
important for independent component analysis.  The experiments presented in
this chapter do not include dimensionality reduction for dictionary learning
with grayscale image patches---that is, we choose~$\varepsilon=0$, since the
resulting whitened images do not seem to suffer from any visual artifact. The
case of color patches is slightly different and is discussed in the next
paragraph.

Similarly as for centering and contrast normalization, we visualize the effect
of whitening on Figure~\ref{subfig:man_white}. Removing the spatial correlation
essentially results in sharper edges. As a matter of fact, whitening can be
shown to amplify high frequencies and reduce low ones~\citep{hyvarinen2009}.
The reason is related to the nature of the principal components of natural images
that form the columns of the matrix~$\U$, which will be discussed in
Section~\ref{subsec:pca}.  In practice, whitening is often not used in image
restoration applications based on dictionary learning since it modifies the
nature of the noise.

\begin{figure}
   \subfigure[Without pre-processing.]{\includegraphics[width=0.49\linewidth]{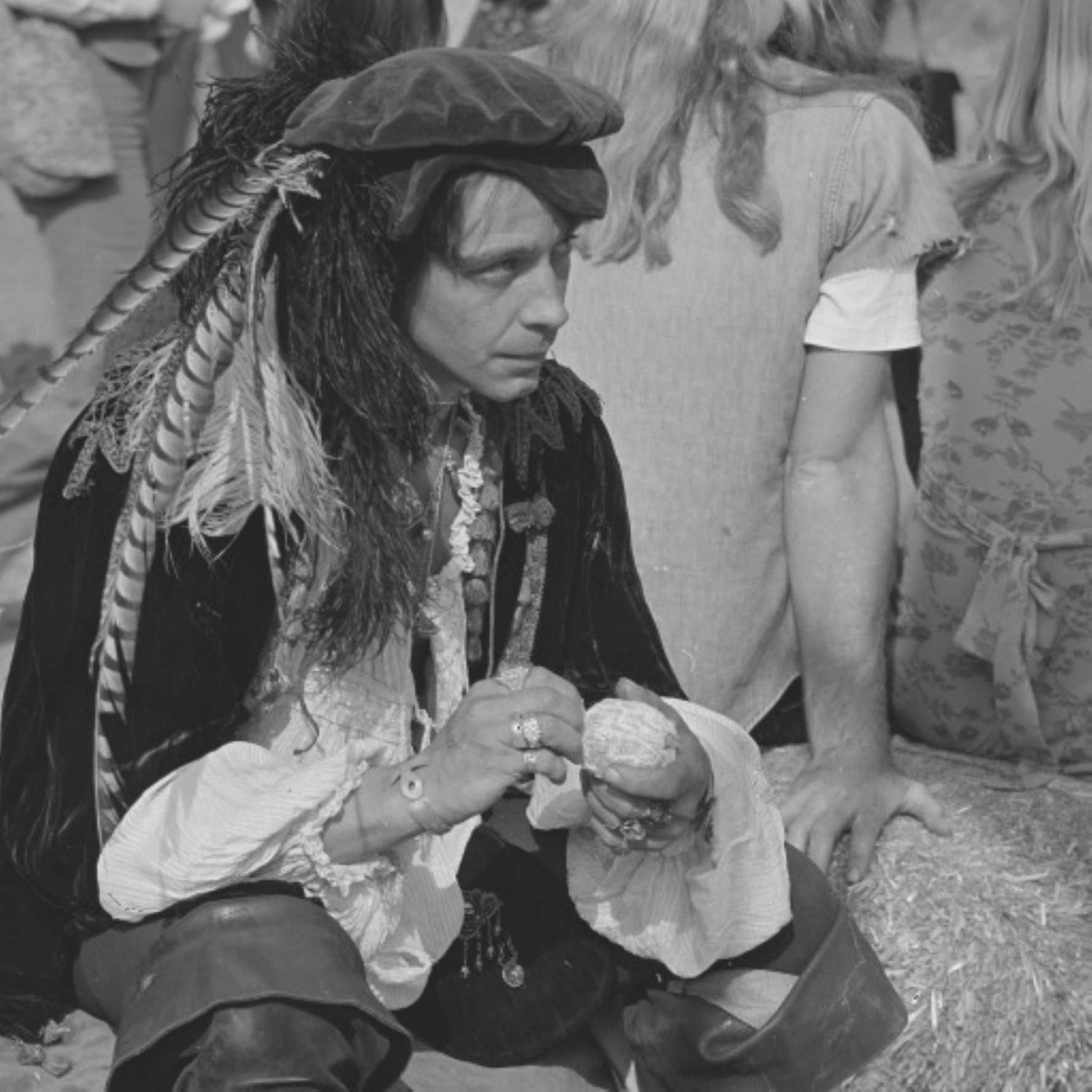}} \hfill
   \subfigure[After centering.]{\includegraphics[width=0.49\linewidth]{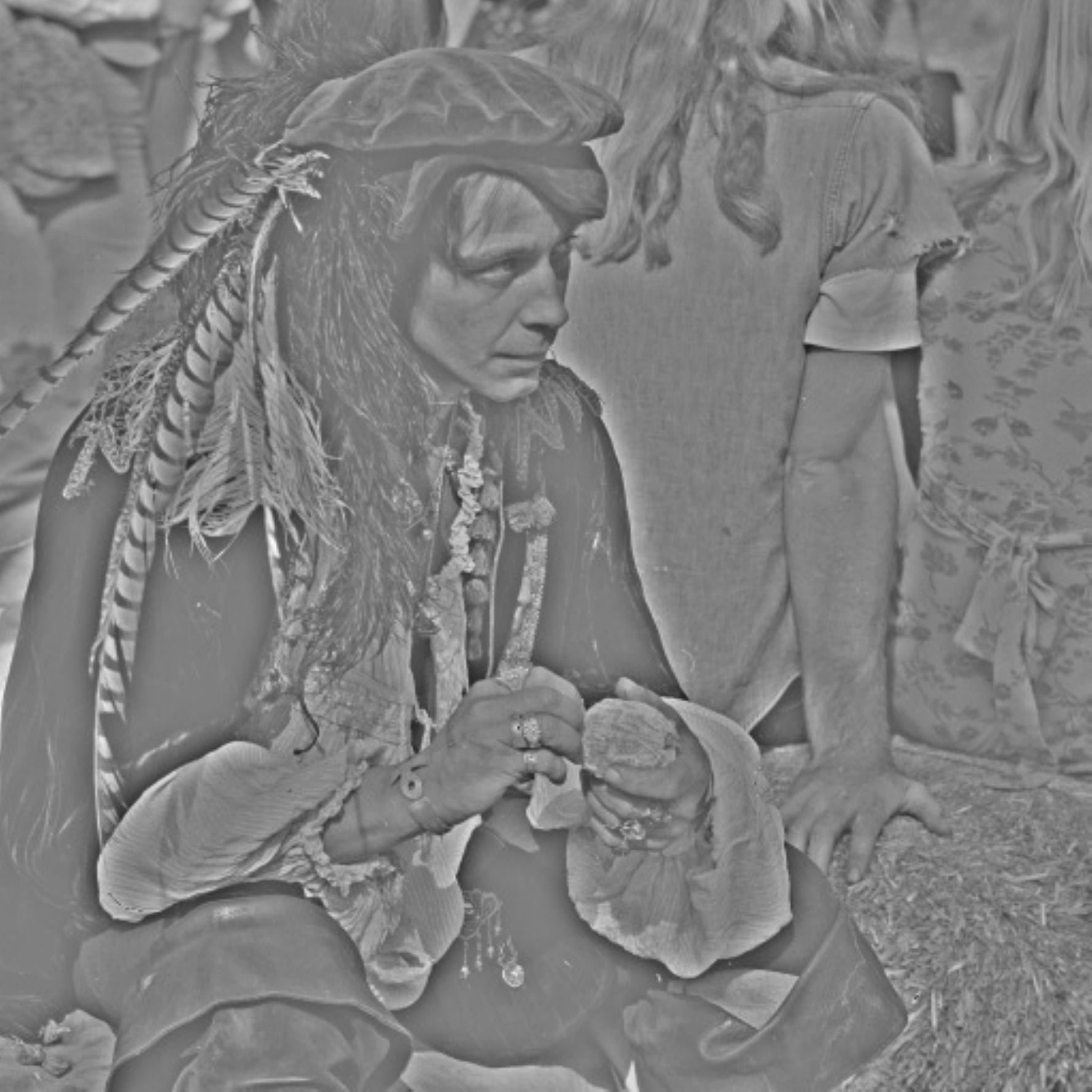}\label{subfig:man_center}} \\
   \subfigure[After centering and $\ell_2$-normalization.]{\includegraphics[width=0.49\linewidth]{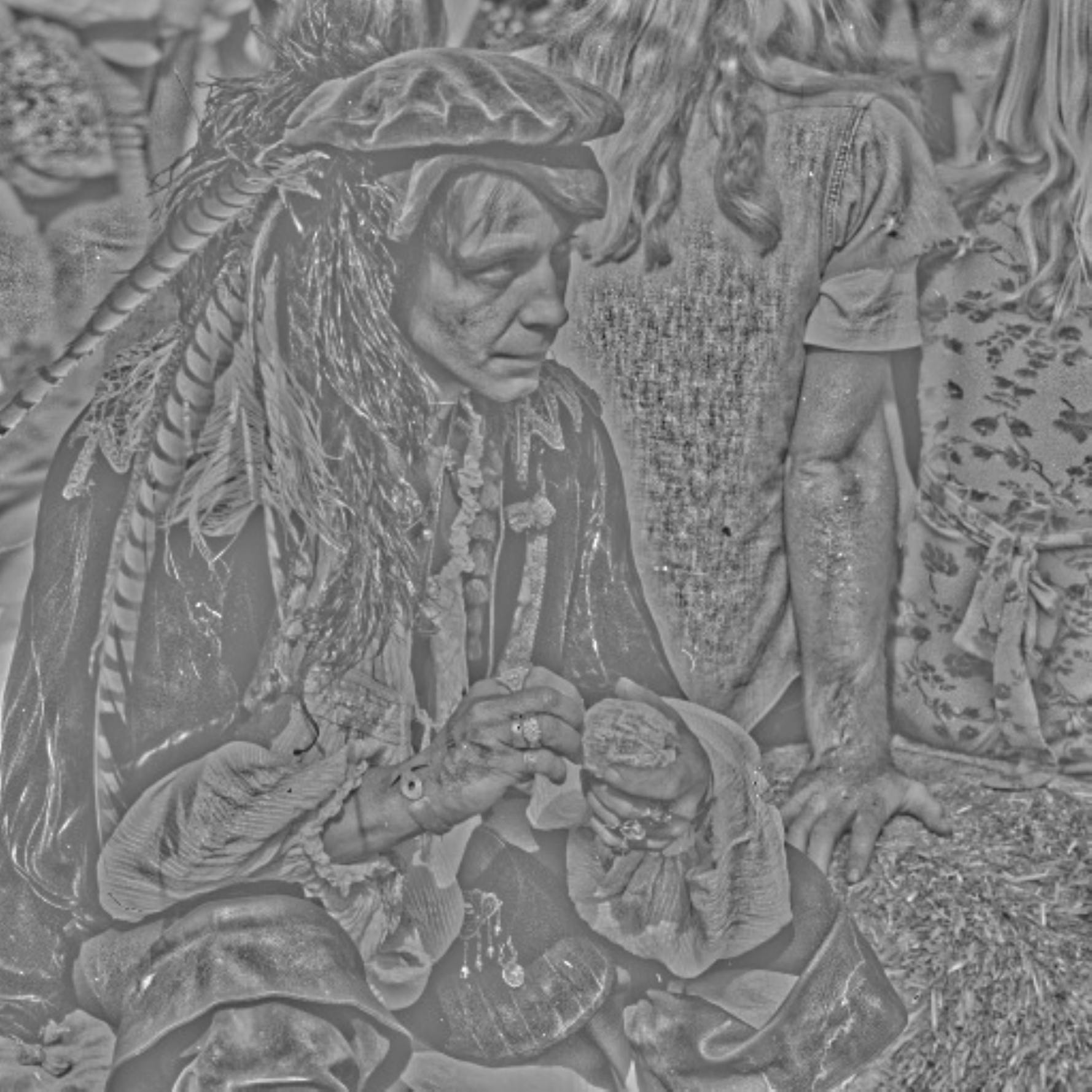}\label{subfig:man_norm}} \hfill
   \subfigure[After whitening.]{\includegraphics[width=0.49\linewidth]{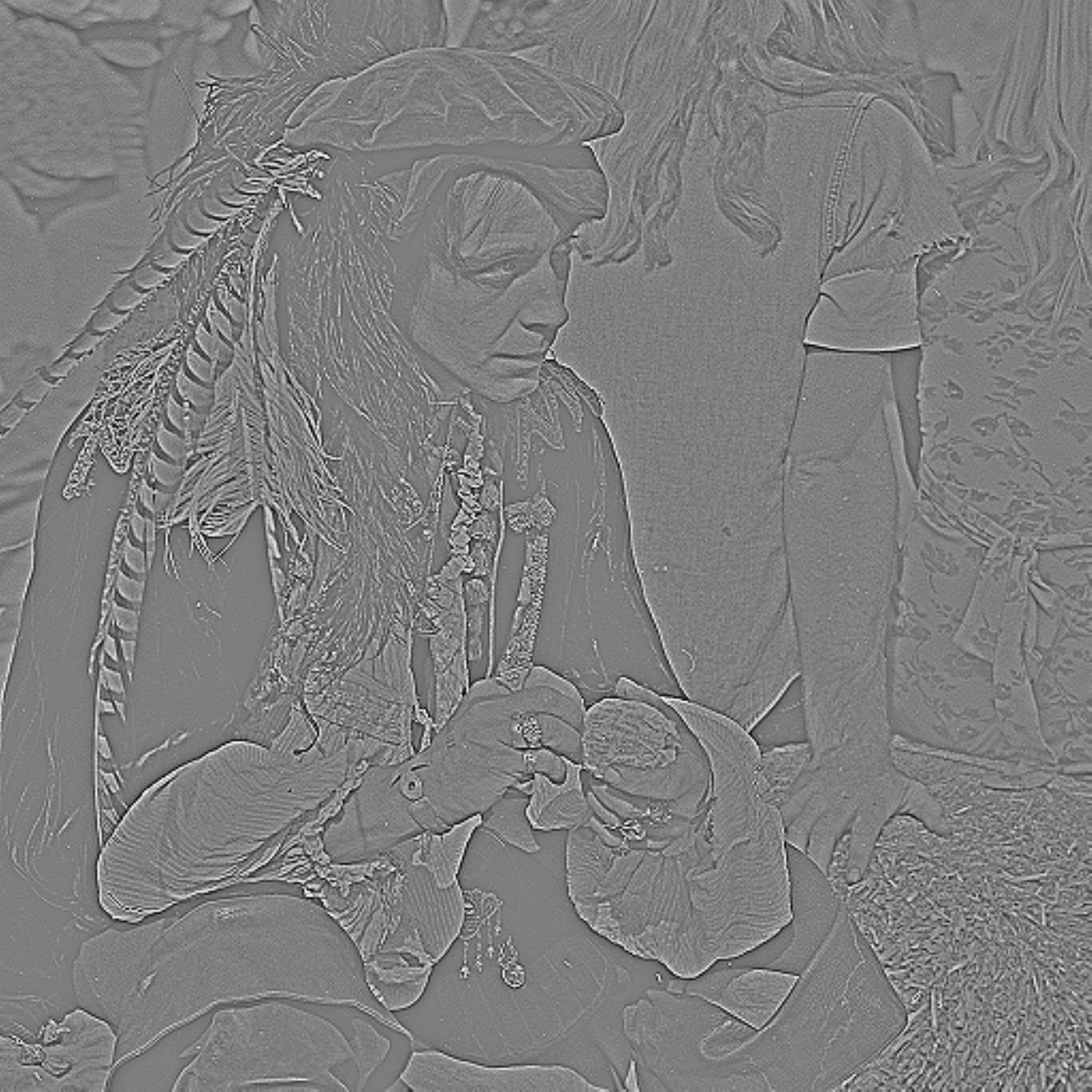}\label{subfig:man_white}} 
   \caption{The effect of various pre-processing procedures on the image
   \textsf{man}. For the figures~\subref{subfig:man_center},
\subref{subfig:man_norm} and~\subref{subfig:man_white}, the value~$0$ is
represented by a gray pixel. Since they contain negative values, pre-processed
images are shifted and rescaled for visualization purposes.}
\label{fig:preprocessing}
\end{figure}

\paragraph{Pre-processing color patches.}
Finally, we conclude this section on pre-processing by discussing the treatment
of color image patches. We assume that images are encoded in the RGB space,
and that patches are obtained by concatenating information from the three
color channels. In other words, $l \times l$ image patches are represented by
vectors~$\x_i$ of size $m = 3 \times l \times l$. 

The main question about color processing is probably whether or not the RGB
color space should be used~\citep{pratt1971,faugeras1979,sharma1997}. 
The choice of the three channels R, G, and B directly originates from our first
understanding of the nature of color. After discovering the existence of the
color spectrum,~\citet{newton1675} already suggests that the rays of light 
\emph{``impinging''} on the retina will
\begin{quote}
   \emph{``affect the sense with various colours, according
to their bigness and mixture; the biggest with the strongest colours, reds and yellow;
the least with the weakest, blue and violets; the middle with green''.}
\end{quote}
Building upon Newton's findings, it appears that~\citet{young1845} was the
first to introduce the concept of trichromatic vision. According to~\citet{maxwell1860}:
\begin{quote}
\emph{``Young appears to have originated the theory, that the three elements of colour are
determined as much by the constitution of the sense of sight as by anything external to us.
He conceives that three different sensations may be excited by light\ldots He conjectures
that these primary sensations correspond to red, green, and violet.''}
\end{quote}
Finally,~\citet{helmholtz1852}, \citet{grassmann1854}, and~\citet{maxwell1860}
proposed rigorous rules of color compositions, paving the way to the modern
treatment of color in vision processing~\citep[see][Chapter 3]{forsyth2012}.
More than a century after these early discoveries, the existence  in the eye of
three types of biological photoreceptors---respectively corresponding to red,
green, and blue wavelengths---is established~\citep[][]{nathans1986}, and the
RGB color space is widely used in electronic devices.

The question of changing the color space can be rephrased as follows: should we
apply a linear or non-linear transformation~$f$ to each RGB pixel value:
$(u,v,w)=f(r,g,b)$ and work in the resulting space? One of the early motivation
for studying the representation of color was to reduce the
bandwidth required by RGB over communication channels. RGB components are
indeed highly correlated; thus, removing the redundant information between the
different channels allows more efficient coding for transmitting
information~\citep{pratt1971}. Another motivation is that the Euclidean distance on
RGB is known to be a poor estimation of the perceptual distance by humans.
Therefore, a large variety of color spaces have been designed, such as CIELab,
CIEXYZ, YIQ, or YCrBr~\citep[see][]{sharma1997}, which are partially improving
upon RGB regarding the decorrelation of the transformed channels and human
perception. However, it remains unclear whether these spaces should be used or
not in practice, and the optimal choice depends on the task and algorithm. 
For instance, for color image denoising, \citet{takeda2007} and \citet{dabov2007} choose
the YCrBr space, whereas~\citet{buades2005,mairal2008}
and~\citet{chatterjee2012} successfully use RGB. Changing the color space
modifies indeed the noise structure as a side effect, which may be problematic
for some methods.

The other pre-processing steps that we have presented for grayscale images can
also be applied to color image patches. The first difference with the
monochromatic setting lies in the centering scheme when working with RGB. 
One motivation for centering grayscale image patches is to learn geometric
structures that are invariant to the mean intensity. 
For color images, the quantity~$\frac{1}{m} \sum_{j=1}^m \x_i[j]$ unfortunately
lacks of interpretation, and removing it from a patch does not seem to achieve
any interesting property. Instead, we may be looking for geometrical structures in images
that are invariant to color patterns, and thus, one may remove the mean color from the patch.
In other words, one should \emph{center every R,G,B channel independently}. As a result,
the centering scheme becomes also invariant to any linear transformation of the
color space, which can be a desirable property.
The effect of such a pre-processing can be visualized in
Figure~\ref{subfig:kodim_center} for the image \textsf{kodim07} from the Kodak PhotoCD dataset.\footnote{available here: \url{http://r0k.us/graphics/kodak/}.} 
Note that centered images have negative values and thus images are shifted and rescaled 
for visualization purposes.  Regions with relatively uniform colors appear
grayish, and thus mostly contain geometrical content; some other areas
contain transitions between a color---say, represented by RGB
values~$(r,g,b)$---and its opponent one---represented by~$(-r,-g,-b)$.  The
colored halos around the pink flowers and the green leaves are thus
due to this color/opponent color transitions. 

In the experiments of this paper about dictionary learning, whitening comprises
a light dimensionality reduction step, where we threshold to zero the smallest
singular values of~$\X$ (the entries on the diagonal of~$\S$). More precisely,
we keep the $k$ largest singular values $s_{1},\ldots,s_k$ such that
$\sum_{i=1}^k s_i^2 \geq 0.995 \sum_{i=1}^m s_i^2$; in other words, we keep
$99.5\%$ of the explained variance of the data. Contrast normalization can then
be applied without particular attention to the nature of the patch. The effects
of whitening and contrast normalization are displayed on
Figures~\ref{subfig:kodim_norm} and~\ref{subfig:kodim_white}, respectively.
Similar conclusions can be drawn for color image patches as in the
monochromatic case.

\begin{figure}
   \subfigure[Without pre-processing.]{\includegraphics[width=0.49\linewidth]{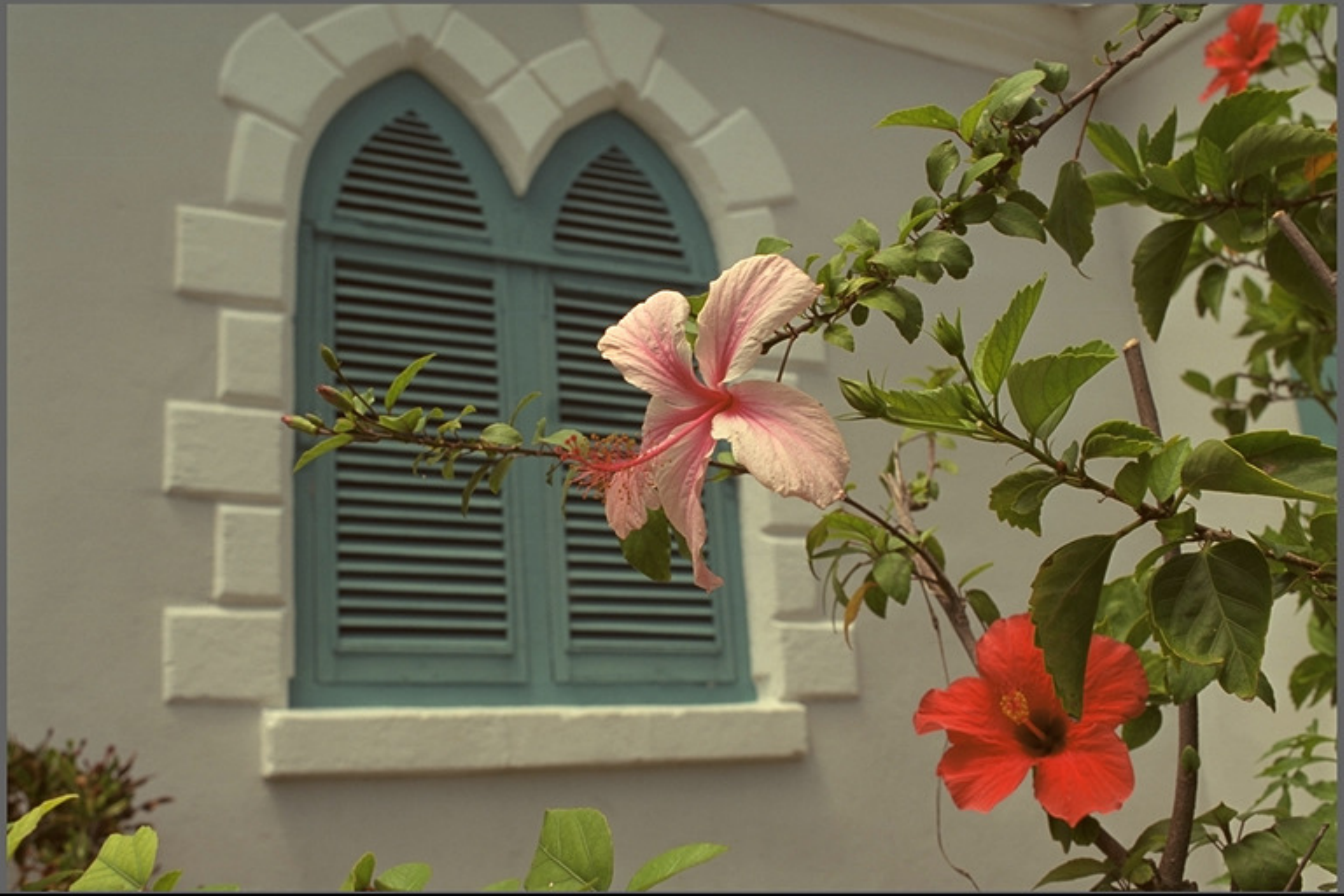}} \hfill
   \subfigure[After centering.]{\includegraphics[width=0.49\linewidth]{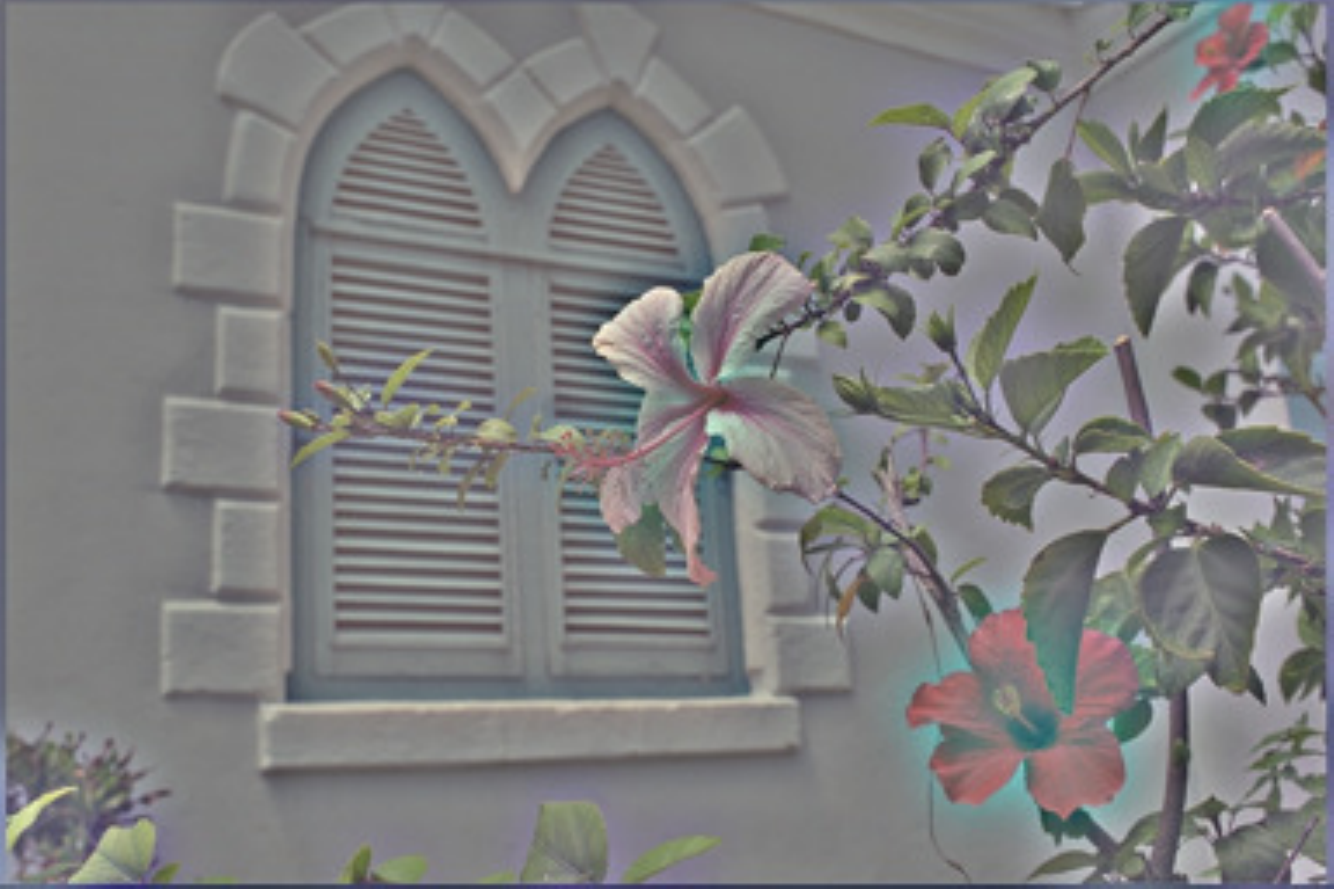}\label{subfig:kodim_center}} \\
   \subfigure[After centering and $\ell_2$-normalization.]{\includegraphics[width=0.49\linewidth]{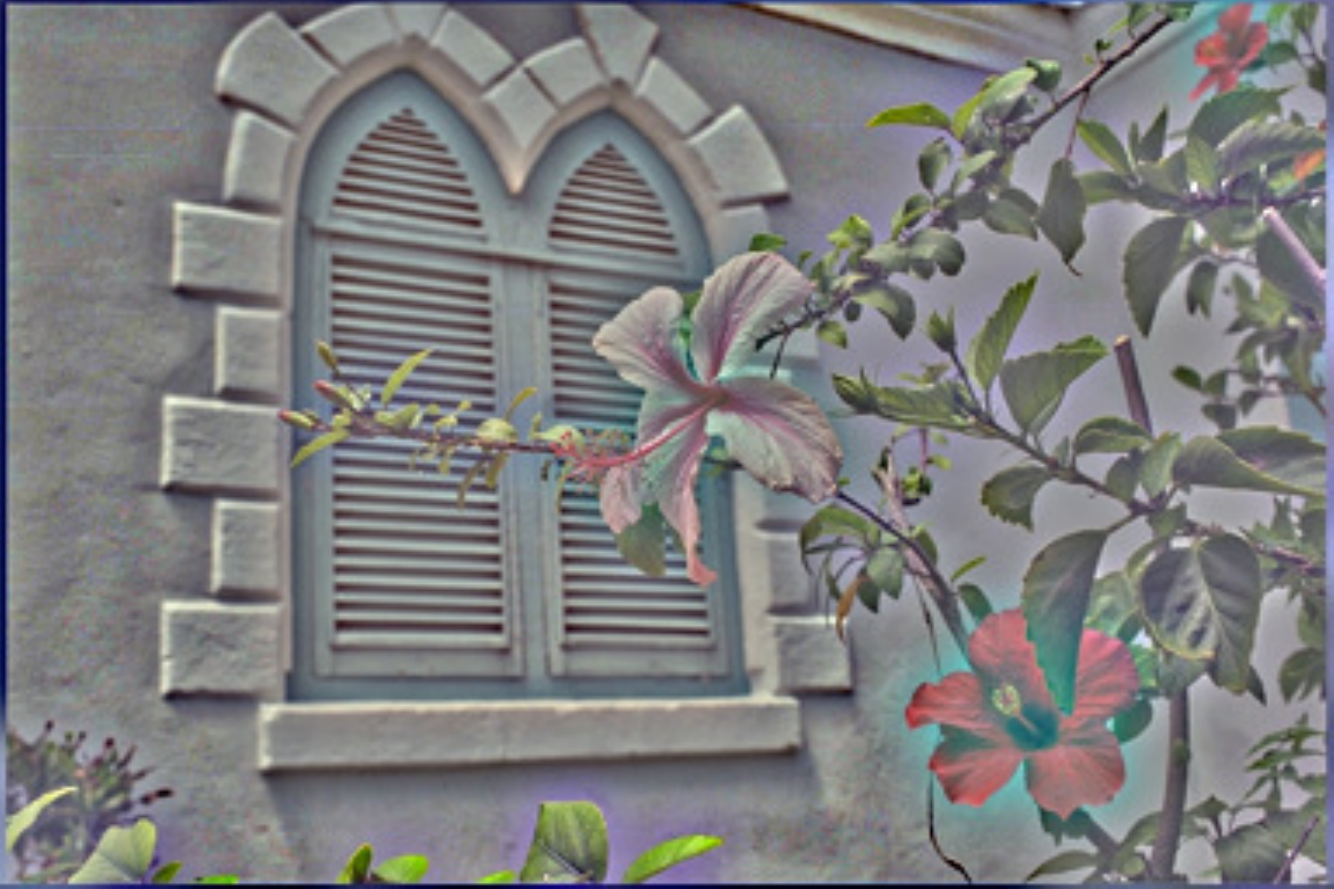}\label{subfig:kodim_norm}} \hfill
   \subfigure[After whitening.]{\includegraphics[width=0.49\linewidth]{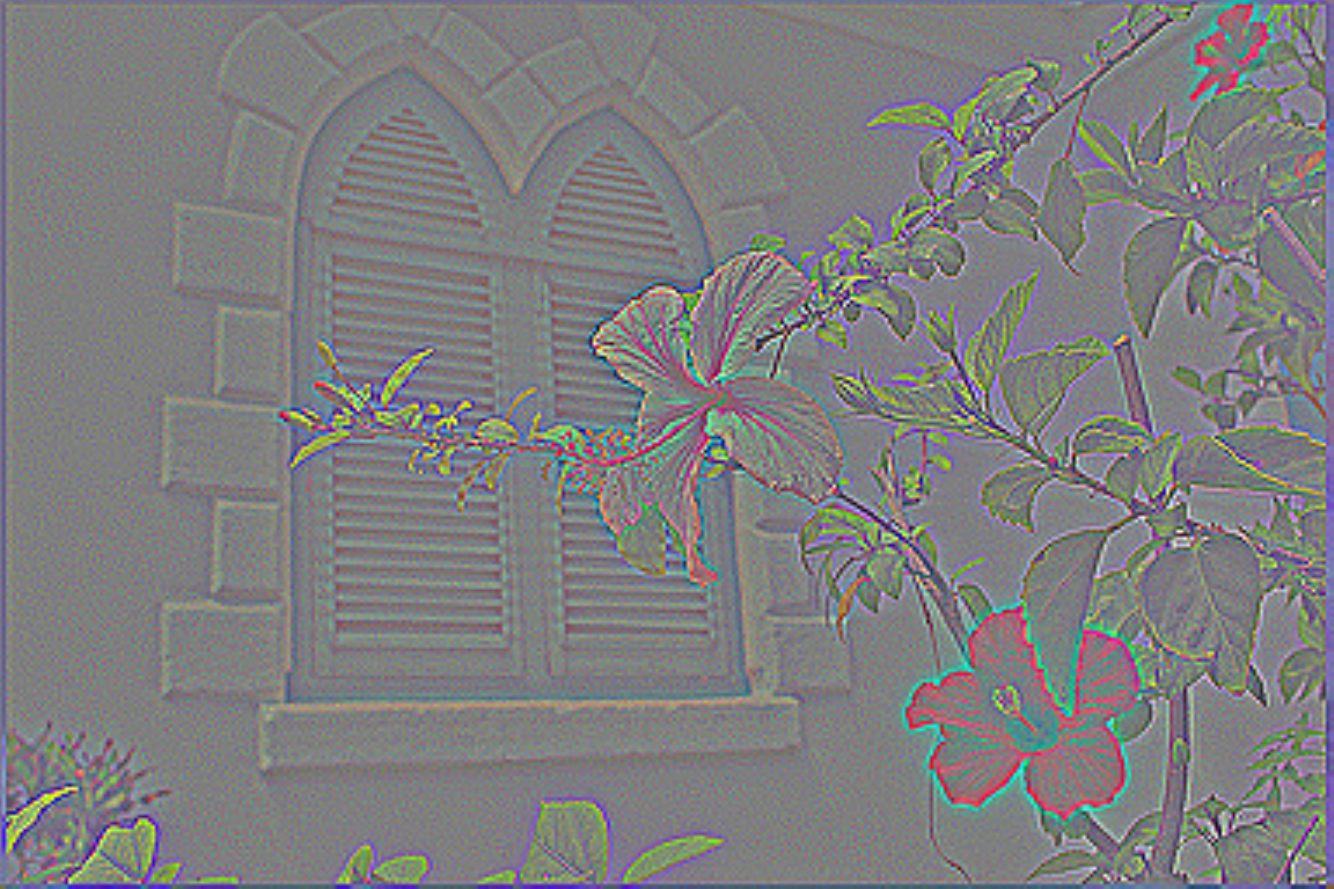}\label{subfig:kodim_white}} 
   \caption{The effect of various pre-processing procedures on the image
   \textsf{kodim07}. Pre-processed images are shifted and rescaled for
visualization purposes, since they contain negative values. The figure is best seen in color
on a computer screen.} \label{fig:preprocessingcolor}
\end{figure}

%% file: content_arxiv/natstat_pca.tex
Principal component analysis (PCA), also known as the Karhunen-Lo\`eve or
Hotelling transform~\citep{hotelling1933}, is probably the most widely used
unsupervised data analysis technique. Even though it is often presented as an
iterative process finding orthogonal directions maximizing variance in the
data, it can be cast as a low-rank matrix factorization problem:
\begin{displaymath}
   \min_{\U \in \Real^{m \times k}, \V \in \Real^{n \times k}} \left\|\X -
   \U\V^\top\right\|_{\text{F}}^2  \st \U^\top \U = \I_k,
\end{displaymath}
where~$k$ is the number of principal components we wish to obtain, and~$\I_k$
is the identity matrix in~$\Real^{k \times k}$. We also assume that the rows of the matrix~$\X$
have zero mean. As a consequence of the theorem of~\citet{eckart1936}, the
matrix~$\U$ contains the principal components of~$\X$ corresponding to the~$k$
largest singular values.  In Figure~\ref{fig:pca}, we visualize the principal
components of~$n=400\,000$ natural image patches of size~$16 \times 16$ pixels,
ordered by largest to smallest singular value (form left to right, then top to
bottom). The components resemble Fourier basis---that is, product of sinusoids
with different frequencies and phases, similarly to the discrete cosine transform
(DCT) dictionary~\citep{ahmed1974} presented in Figure~\ref{fig:dct}.

However, there is a good reason for obtaining sinusoids that is simply related
to a property of translation invariance, \emph{meaning that
the patterns observed in Figure~\ref{fig:pca} are unrelated to the
underlying structure of natural images.} 
This fact is well known and has been pointed earlier by others in the
literature \citep[see, \eg,][]{bossomaier1986,field1987,simoncelli2001,hyvarinen2009}.
Second-order statistics of natural images patches are commonly assumed
to be invariant by translation, such that the correlation of pixel values at
locations~$z_1=(k_1,l_1)$ and~$z_2=(k_2,l_2)$ only depends on the
displacement $z_1-z_2$ \citep{simoncelli2001}. This is particularly true for
patches that are extracted from larger images at any arbitrary position, and
whose distribution is exactly translation invariant.  Such signals are
called ``stationary'' and their principal components (equivalently the
eigenvectors of the covariance matrix) are often considered to be well
approximated by the Discrete Fourier Transform (DFT), as noted
by~\citet{pearl1973}, leading to sinusoidal principal components.

Nevertheless, the relation between principal components
and sinusoids only holds rigorously in an asymptotic regime. For instance, consider 
the case of an infinite one-dimensional signal with covariance
$\Sigma[k,l]=\sigma(k-l)$ for positions~$k$ and~$l$, where~$\sigma$ is an even
function.  Then, for all frequency~$\omega$ and phase~$\varphi$,
\begin{displaymath}
   \sum_{l} \Sigma(k,l) e^{i (\omega l + \varphi)} =    \sum_{l} \sigma(l-k) e^{i (\omega l + \varphi)} = \left( \sum_{l'} \sigma(l') e^{i \omega l'}  \right) e^{i (\omega k + \varphi)},
\end{displaymath}
where~$i$ denotes the imaginary unit, and the sums are over all integers. Since
the function~$\sigma$ is even, the infinite sum $\left( \sum_{l'} \sigma(l')
e^{i \omega l'}  \right)$ is real, and the signals $[\sin(\omega k +
\varphi)]_{k \in {\mathbb Z}}$ are all eigenvectors of the covariance
operator~$\Sigma$. Equivalently, they are principal components of the input
data. The case of two-dimensional signals can be treated similarly but it
involves heavier notation. Drawing conclusions about the eigenvectors of a
finite covariance matrix computed on natural image patches is nevertheless subject
to discussion, and understanding the quality of the approximation
of the principal components by the DFT in the finite regime is
non-trivial~\citep{pearl1973}.

From an experimental point of view, the fact that the structure of natural
images has nothing to do with the patterns displayed in Figure~\ref{fig:pca} is
easy to confirm. In Figure~\ref{fig:pca_tiger}, we present principal
components computed on all overlapping patches from the image~\textsf{tiger}.
Even though this image is sketched by hand---and thus, is not natural---we
recover sine waves as in Figure~\ref{fig:pca}.

\begin{figure}
   \centering
   \subfigure[DCT Dictionary.]{\includegraphics[width=0.49\linewidth]{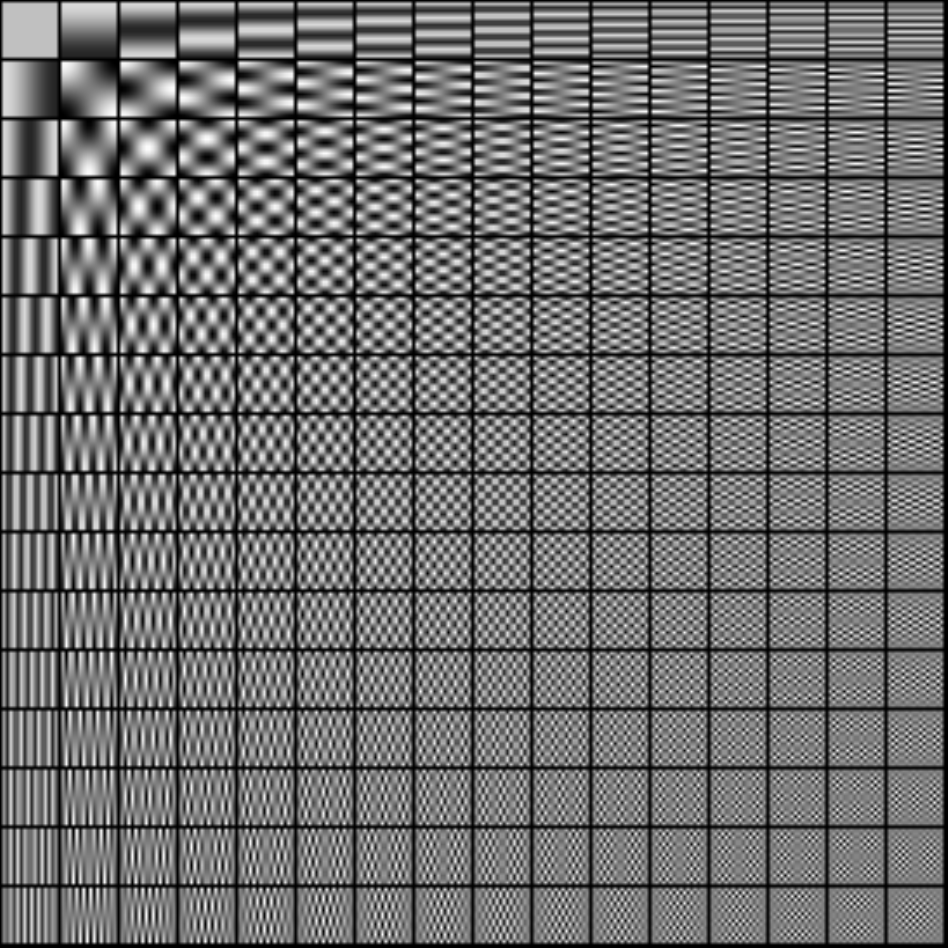}\label{fig:dct}} \hfill
   \subfigure[Principal
   components.]{\includegraphics[width=0.49\linewidth]{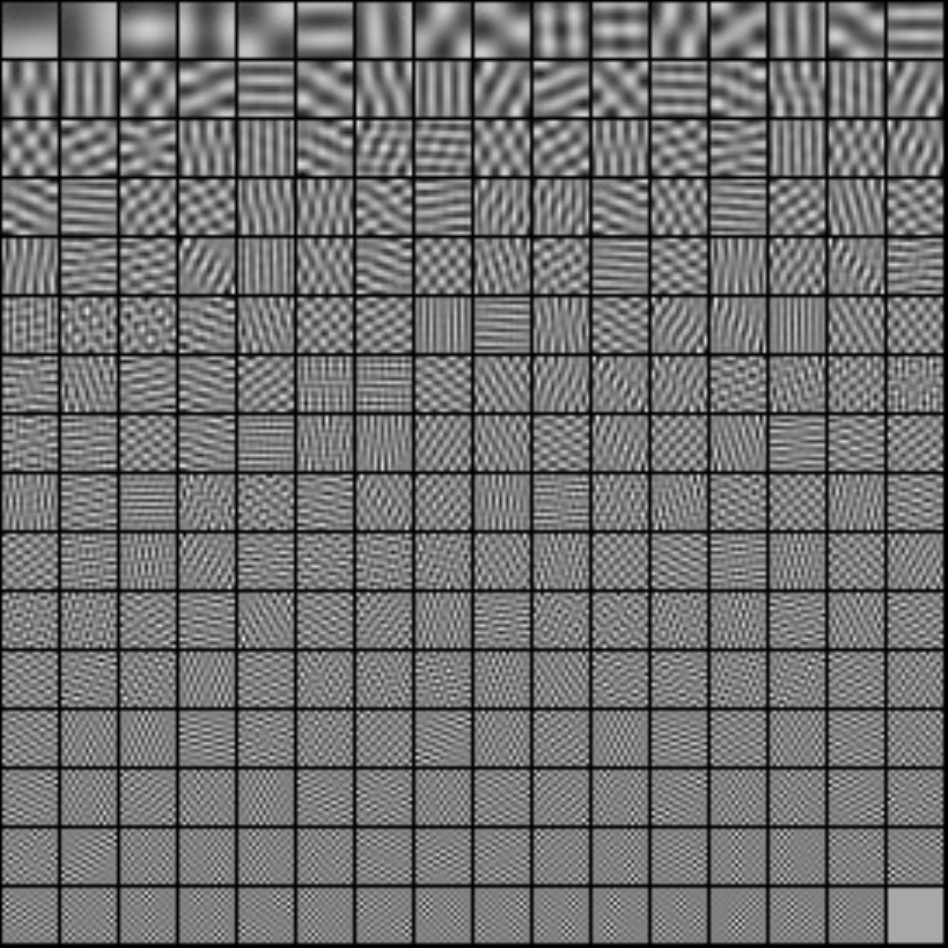}\label{fig:pca}}
   \caption{On the right, we visualize the principal components of $400\,000$
   randomly sampled natural image patches of size~$16 \times 16$, ordered by
   decreasing variance, from top to bottom and left to right. On the left, we
   display a discrete cosine transform (DCT) dictionary. Principal components
   resemble DCT dictionary elements.
} 
\end{figure}

\begin{figure}
   \centering
   \hfill \subfigure[Original Image.]{\includegraphics[width=0.28\linewidth]{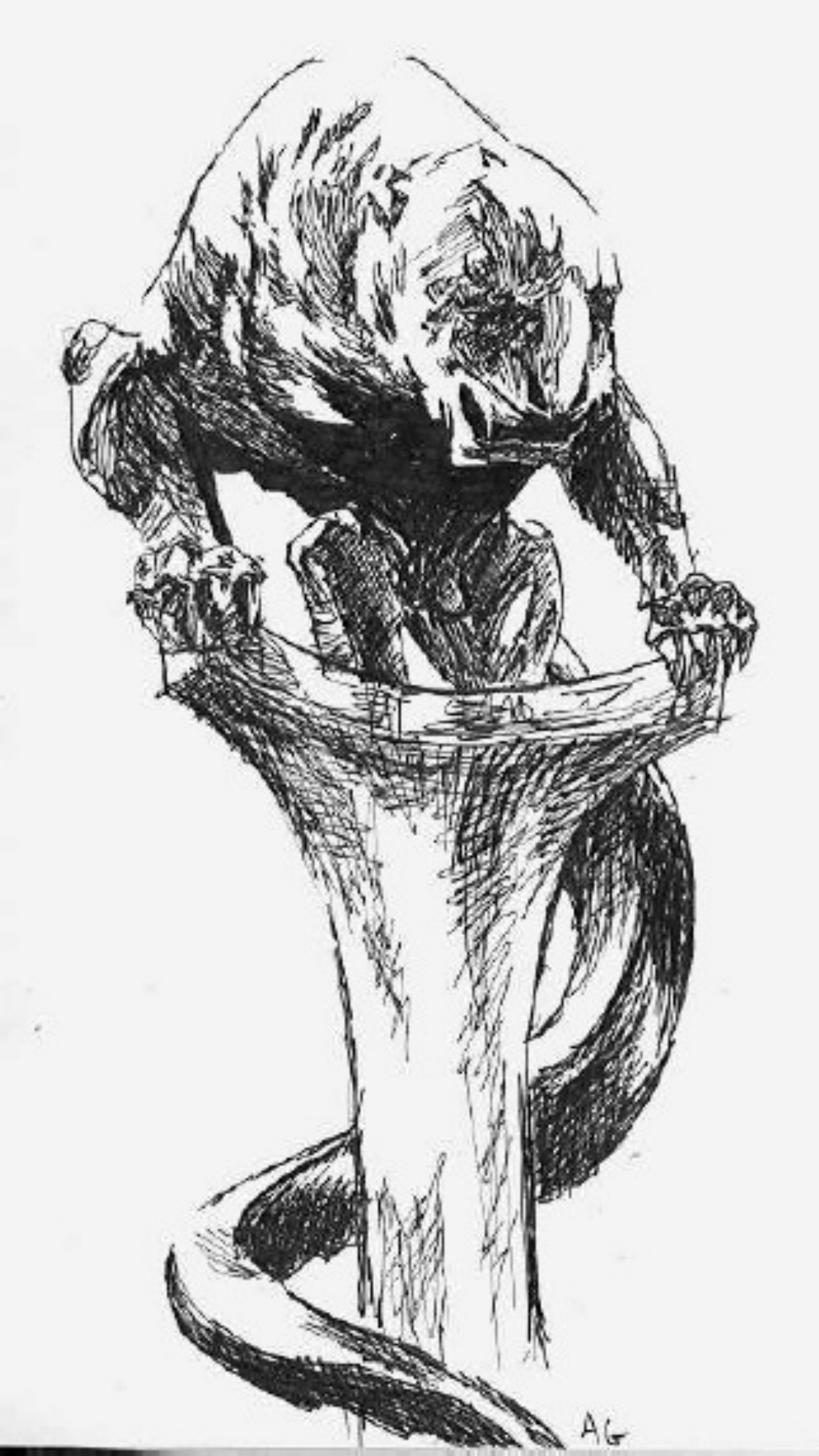}} \hfill
   \subfigure[Principal components.]{\includegraphics[width=0.49\linewidth]{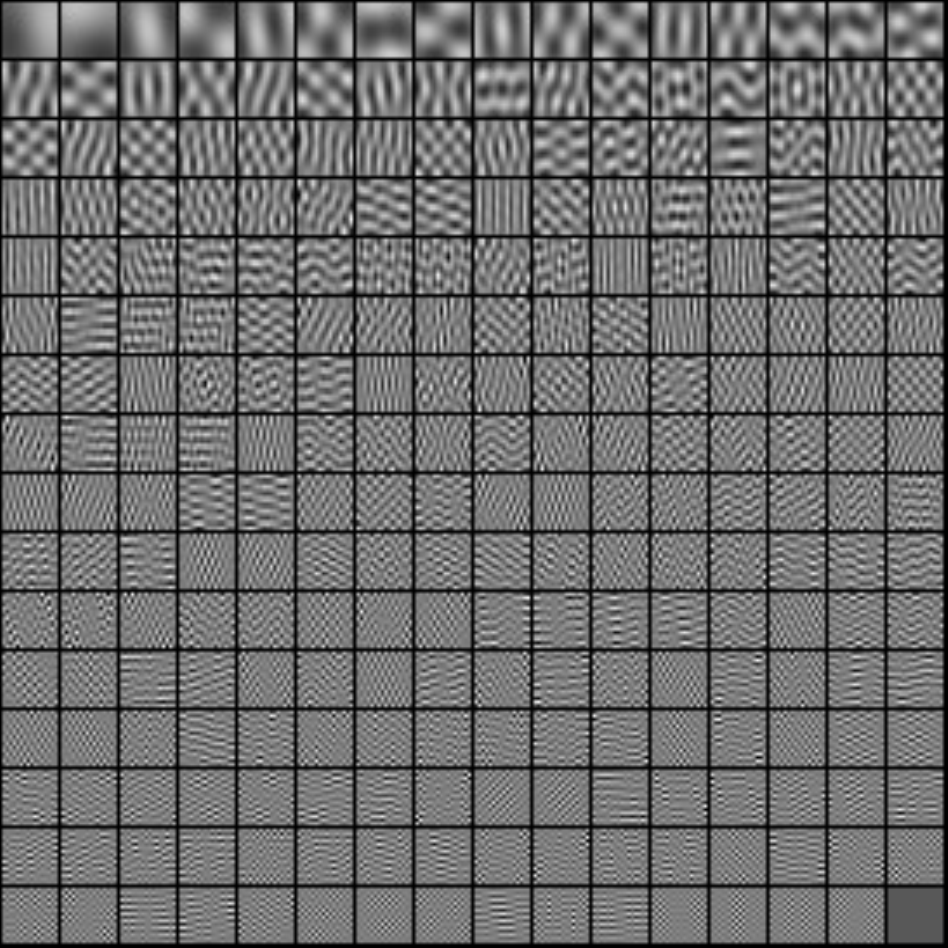}} 
   \caption{Visualization of the principal components of all overlapping patches from the image~\textsf{tiger}. Even though the image is not natural, its principal components are similar to the ones of Figure~\ref{fig:pca}.} \label{fig:pca_tiger}
\end{figure}

%% file: content_arxiv/natstat_clustering.tex
Clustering techniques have been used for a long time on natural image patches
for compression and communication purposes under the name of ``vector
quantization''~\citep[][]{nasrabadi1988,gersho1992}.  The goal is to find $p$
clusters in the data, by minimizing the following objective:
\begin{equation}
   \min_{\substack{\D \in \Real^{m \times p} \\ \forall i,~l_i \in \{1,\ldots,p\}}}  ~~\sum_{i=1}^n \|\x_i - \d_{l_i}\|_2^2, \label{eq:kmeans}
\end{equation}
where the columns of~$\D=[\d_1,\ldots,\d_p]$ are called ``centroids'' and~$l_i$
is the index of the cluster associated to the data point~$\x_i$.  The algorithm
K-means~\citep[see][]{hastie2009} approximately optimizes the non-convex
objective~(\ref{eq:kmeans}) by alternatively performing exact minimization with
respect to the labels~$l_i$ with~$\D$ fixed, and with respect to~$\D$ with the
labels fixed.

To make the link between clustering and matrix factorization, it is also possible to
reformulate~(\ref{eq:kmeans}) as follows
\begin{displaymath}
   \min_{\substack{\D \in \Real^{m \times p} \\ \A \in \{0,1\}^{p \times n}}}
   ~~\frac{1}{n}\sum_{i=1}^n \frac{1}{2}\|\x_i - \D \alphab_i \|_2^2 \st \forall i,~\sum_{j=1}^p
   \alphab_i[j]=1, 
\end{displaymath}
where~$\A=[\alphab_1,\ldots,\alphab_n]$ carries binary vectors that sum to one,
and clustering can be subsequently seen as a matrix factorization problem:
\begin{displaymath}
   \min_{\substack{\D \in \Real^{m \times p} \\ \A \in \{0,1\}^{p \times n}}}
   ~~\frac{1}{2n} \|\X - \D \A \|_{\text{F}}^2 \st \forall i,~\sum_{j=1}^p
   \alphab_i[j]=1.
\end{displaymath}
Then, the algorithm K-means is performing alternate minimization between~$\A$
and~$\D$, each step decreasing the value of the objective.

We visualize clustering results in Figure~\ref{fig:kmeans} after applying $250$
iterations of the K-means algorithm on~$n=400\,000$ natural image patches
and choosing~$p=256$ centroids. Without whitening, K-means produces mostly
low-frequency patterns, which is not surprising since most of the energy
($\ell_2$-norm) of images is concentrated in low frequencies. It is indeed
known that the spectral power of natural images obtained in the Fourier domain
typically decreases according to the power law $1/f^2$ for a
frequency~$f$~\citep[see][]{simoncelli2001}.  After whitening, high-frequency
``Gabor-like'' patterns and checkerboard patterns emerge from data. It is thus
interesting to see that with an appropriate pre-processing procedure, simple
unsupervised learning techniques such as K-means are able to discover Gabor
features with different orientations, frequencies, and positions within
patches. We note that the localized checkerboard patterns might be an artifact
of the whitening procedure without dimensionality reduction, which amplifies
such patterns appearing among the last principal components presented in
Figure~\ref{fig:pca}. 

\begin{figure}[hbtp]
   \centering
   \subfigure[With centering.]{\includegraphics[width=0.49\linewidth]{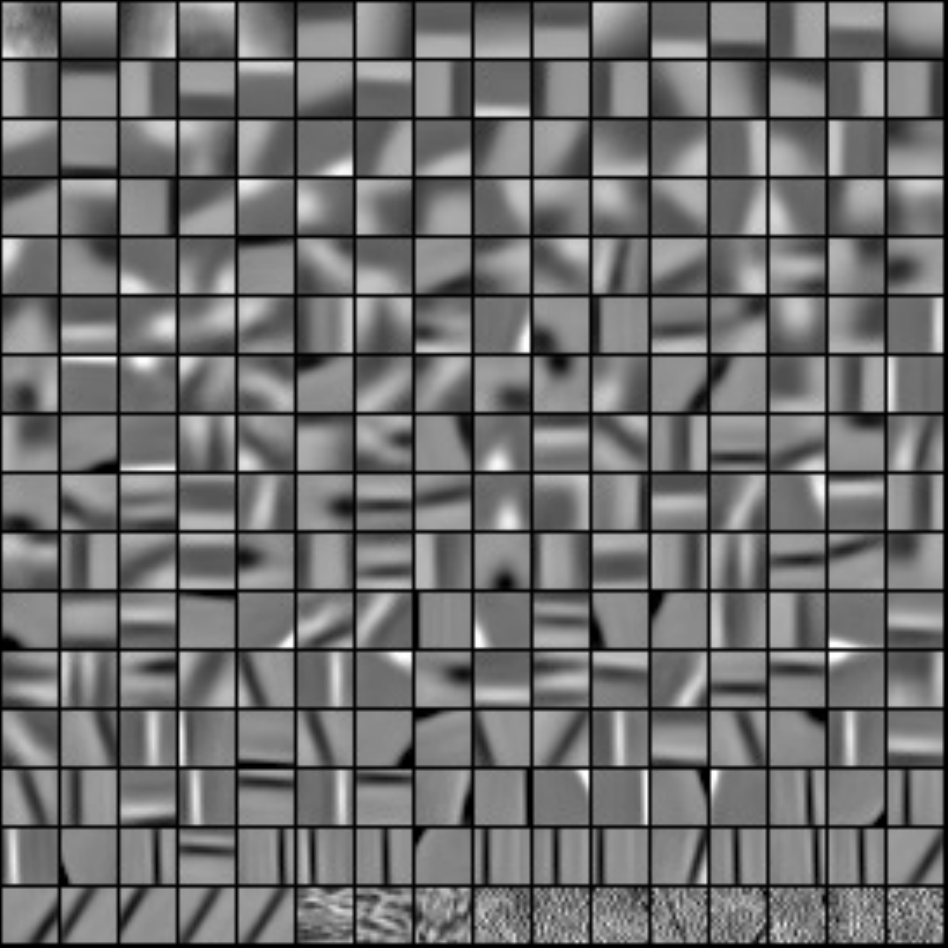}} \hfill
   \subfigure[With whitening.]{\includegraphics[width=0.49\linewidth]{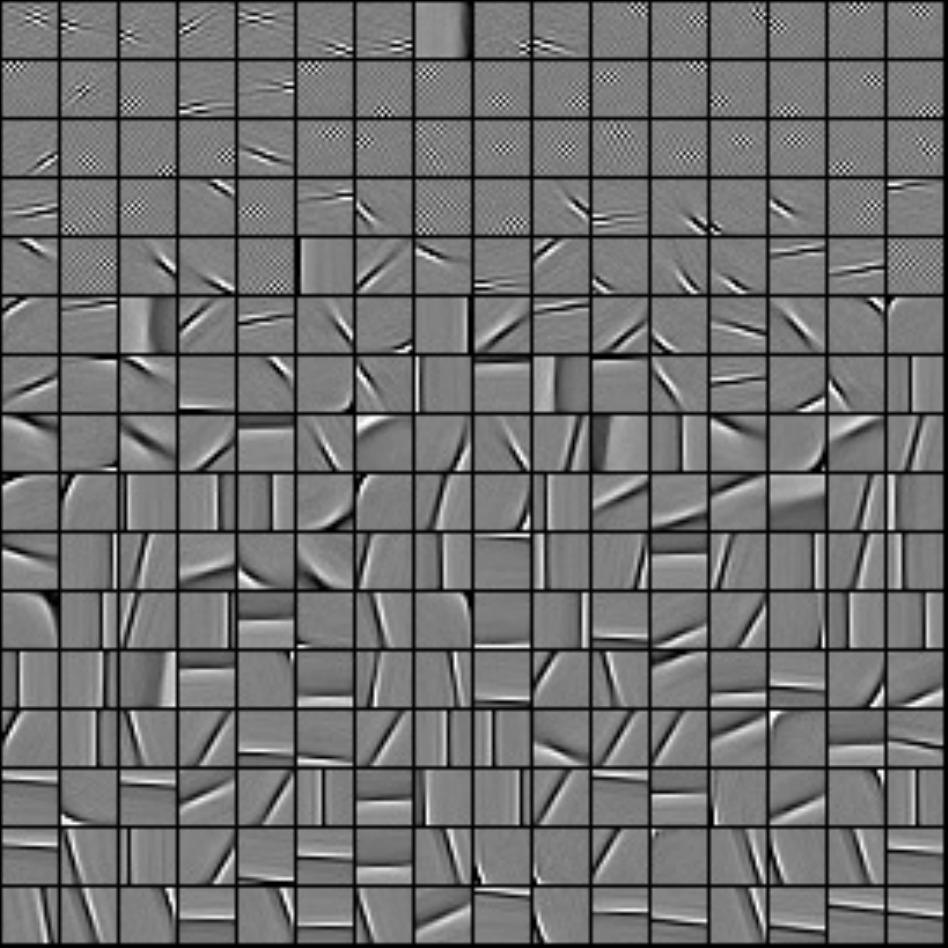}}
   \caption{Visualization of $p=256$ centroids computed with the algorithm
      K-means on~$n=400\,000$ image patches of size~$m=16 \times 16$ pixels. We
      compare the results with centering (left), and whitening (right). In both
      cases, we also apply a contrast normalization step. Centroids are ordered
   from most to least used (from left to right, then top to bottom).}
   \label{fig:kmeans} 
\end{figure}

%% file: content_arxiv/natstat_dictlearning.tex
We are now interested in the dictionary learning formulation originally
introduced by~\citet{field1996,olshausen1997} and its application to natural
image patches. We consider the corresponding matrix factorization
formulation~(\ref{eq:dictlearning2}), which we recall here
\begin{equation}
   \min_{\D \in \CC, \A \in \Real^{p \times n}} \frac{1}{2}\|\X-\D\A\|_{\fro}^2 + \lambda \Psi(\A),\label{eq:natstat_dict}
\end{equation}
where~$\Psi(\A) = \sum_{i=1}^n \psi(\alphab_i)$ and~$\CC$ is the set of matrix
whose columns  have a Euclidean norm smaller than one.

In Figure~\ref{fig:visudict}, we present visual results obtained with such a
formulation where~$\psi$ is the~$\ell_1$-norm. As in the previous section about
clustering, we use~$n=400\,000$ natural image patches, either gray, or
extracted from color images in RGB.  We center the patches, according to the
procedure presented in Section~\ref{subsec:preprocess}, and rescale the
matrix~$\X$ such that its columns have unit norm on average. Then, we learn a
dictionary~$\D$ with~$p=256$ dictionary elements and set the regularization
parameter~$\lambda$ to~$0.1$.  We use the online dictionary learning algorithm
of~\citet{mairal2010} available in the software SPAMS, making~$10$ passes over
the data. The dictionaries obtained for gray and RGB patches and with two
different preprocessing steps are displayed in
Figure~\ref{fig:visudict}.
Note that other dictionary learning approaches can be used for this task and
providing similar results. This is for instance the case of the K-SVD algorithm
of~\citet{aharon2006}, which we present in Section~\ref{sec:optimdict}.

When processing grayscale image patches, a simple centering step produces dictionary
elements with both low-frequency elements, and high-frequency, localized,
Gabor-like patterns, as shown in Figure~\ref{subfig:centergray}. The results after
a whitening step are displayed in Figure~\ref{subfig:whitegray}.
Whitening corresponds to applying a high-pass filter to the original images,
and we obtain exclusively Gabor-like features as in the early work
of~\citet{field1996,olshausen1997}.  The case of color image patches presented
in Figures~\ref{subfig:centercolor} and~\ref{subfig:whitecolor} is also
interesting. After centering, about $230$ dictionary elements out of~$256$ are
grayscale, or are almost grayscale, which is consistent with the greyish
appearance of images after centering (see Section~\ref{subsec:preprocess}).
An intriguing phenomenon concerns the remaining colored dictionary elements.
Those exhibit low-frequency patterns with two opponent colors, blue and yellow
for~$10$ of them, green and red/purple for~$16$ of them. Interestingly, such 
opponent color patterns are typical characteristics of neurons' receptive fields described
in the neuroscience literature~\citep{livingstone1984,ts1988}.

\begin{figure}
   \centering
   \subfigure[With centering - gray.]{\includegraphics[width=0.49\linewidth]{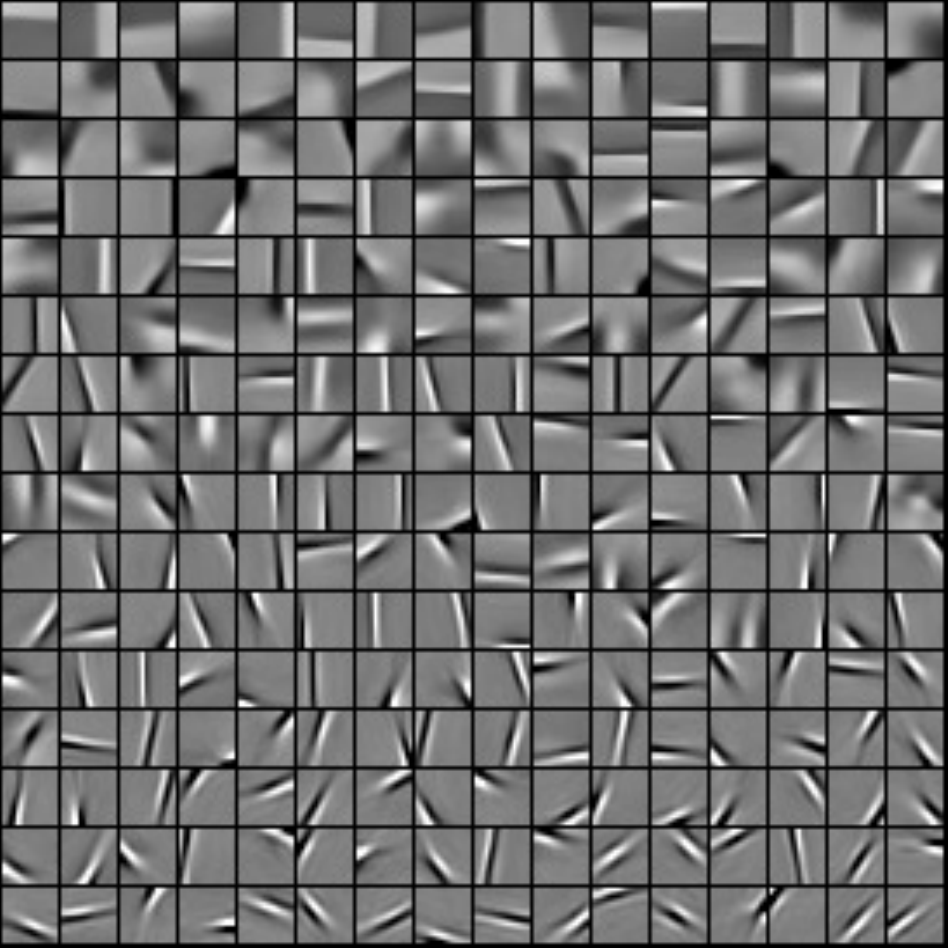}\label{subfig:centergray}} \hfill
   \subfigure[With centering - RGB.]{\includegraphics[width=0.49\linewidth]{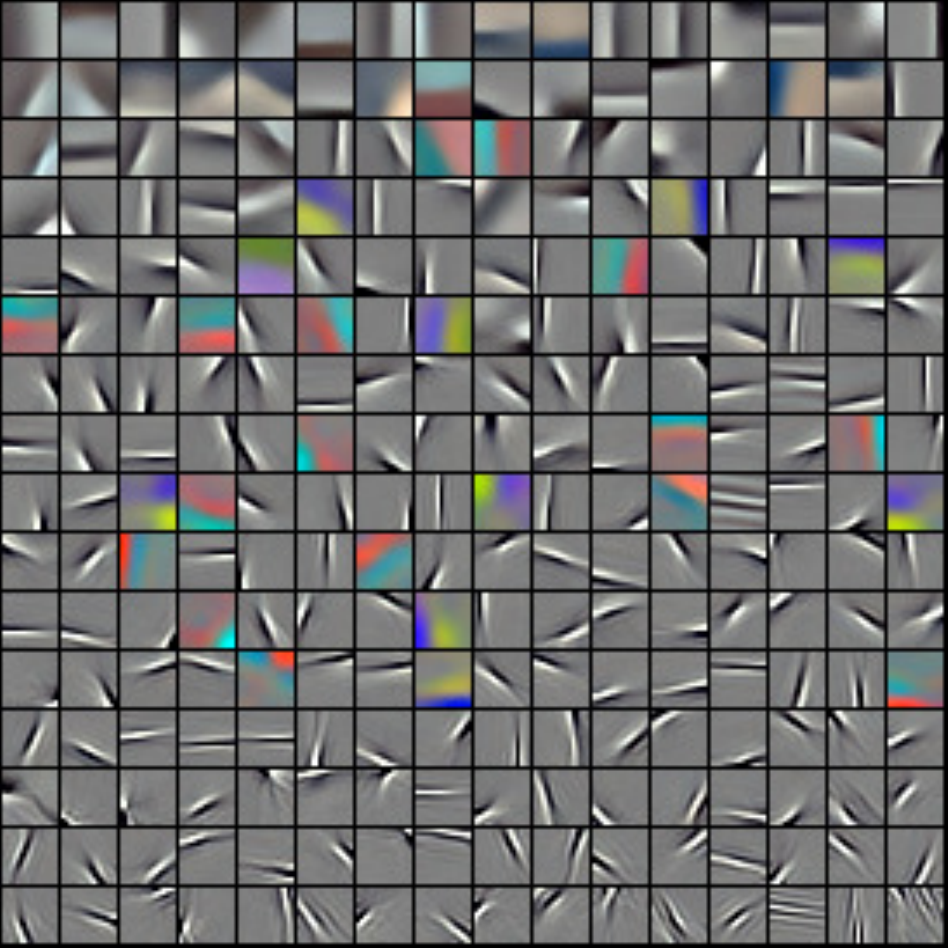}\label{subfig:centercolor}} \\
   \subfigure[With whitening - gray.]{\includegraphics[width=0.49\linewidth]{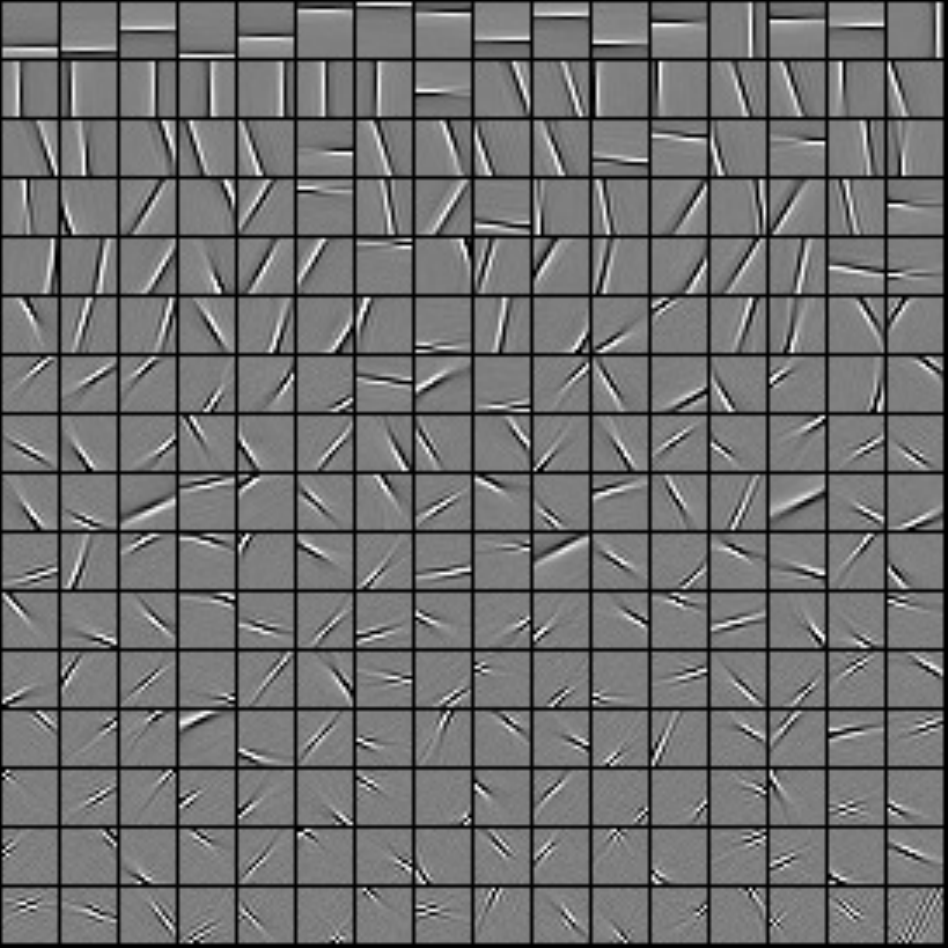}\label{subfig:whitegray}} \hfill
   \subfigure[With whitening - RGB.]{\includegraphics[width=0.49\linewidth]{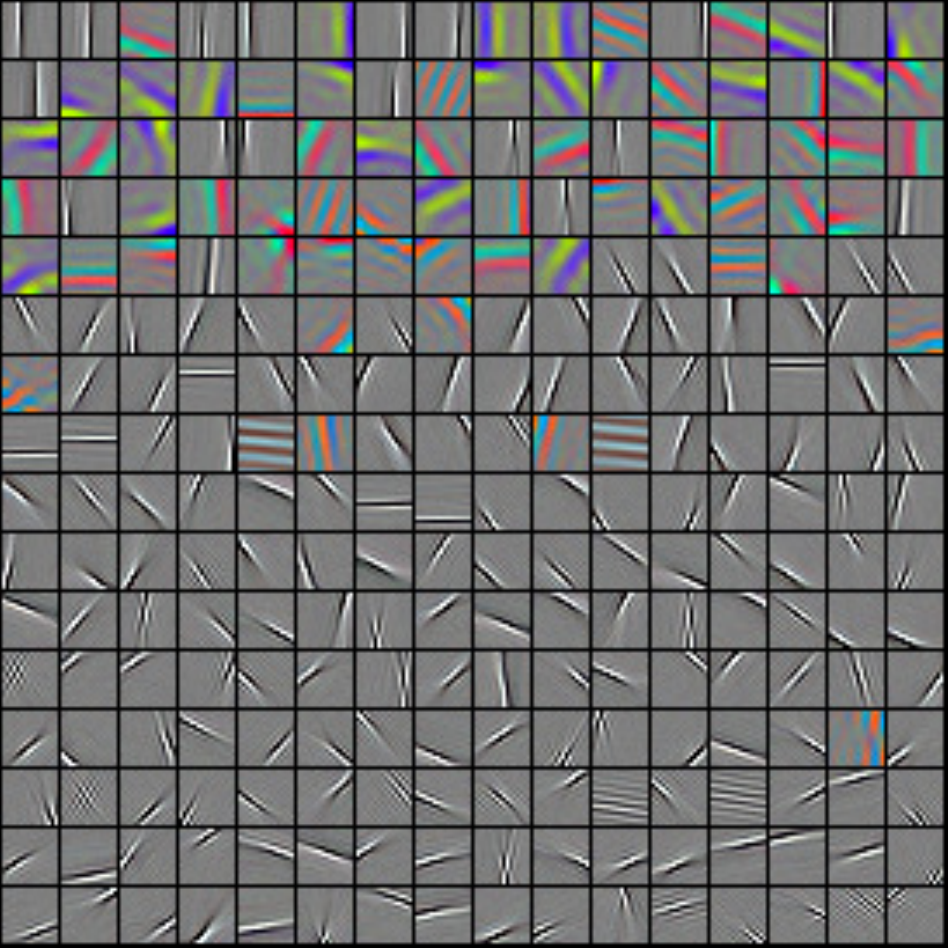}\label{subfig:whitecolor}} \\
   \caption{Dictionaries obtained with the formulation~(\ref{eq:natstat_dict}),
      with two different pre-processing procedures, and with gray or RGB
      patches.  The dictionary elements are ordered from most to least used
   (from left to right, then top to bottom). Best seen in color on a computer screen.}
   \label{fig:visudict} 
\end{figure}

%% file: content_arxiv/natstat_structured.tex
Because the~$\ell_1$-norm cannot model interactions between dictionary
elements, it is natural to consider extensions of dictionary learning where a
particular structure is taken into account. The concept of \emph{structured
sparsity} presented in Section~\ref{sec:l1} is a natural tool for this; it
consists of replacing the~$\ell_1$-norm by a more complex sparsity-inducing
penalty.

\paragraph{Hierarchical dictionary learning.} A first extension has been
investigated by~\citet{Jenatton2010a,Jenatton2010b}, when a pre-defined
hierarchical structure is assumed to exist among the~$p$ dictionary elements.
The formulation~(\ref{eq:natstat_dict}) is considered where the function~$\Psi$
is the hierarchical Group-Lasso penalty of~\citet{zhao}. A tree is given, \eg,
is defined by the user, and one dictionary element is associated to every node
of the tree. The penalty constructs a group structure~$\GG$ of subsets
of~$\{1,\ldots,p\}$, each group containing one node and all its descendants in the tree.
An example is illustrated in Figure~\ref{fig:treesparse}. The resulting penalty~$\Psi$
is then defined as
\begin{equation}
   \Psi(\A) = \sum_{i=1}^n \sum_{g \in \GG} \|\alphab_i[g]\|_q, \label{eq:dicttree}
\end{equation}
where $\|.\|_q$ can either represent the~$\ell_\infty$- or the~$\ell_2$-norm. As explained
in Section~\ref{sec:l1}, an effect of the penalty is that a dictionary
element can be used in the decomposition of a patch only if its parent in the
tree is also used. We present some visual results for~$q=\infty$ in
Figure~\ref{fig:visudicttree}, after centering the
natural image patches, for some manually tuned regularization parameter~$\lambda$,
and two different tree structures.  Dictionary elements naturally organize themselves in
the tree, often with low frequencies near the root of the tree, and high
frequencies near the leaves. We also observe strong correlations between each
parent node and their children in the tree, where children often look like
their parent, but with higher frequencies and with minor variations.
\begin{figure}
   \centering
   \subfigure[Tree structure~$1$.]{\includegraphics[width=0.49\linewidth]{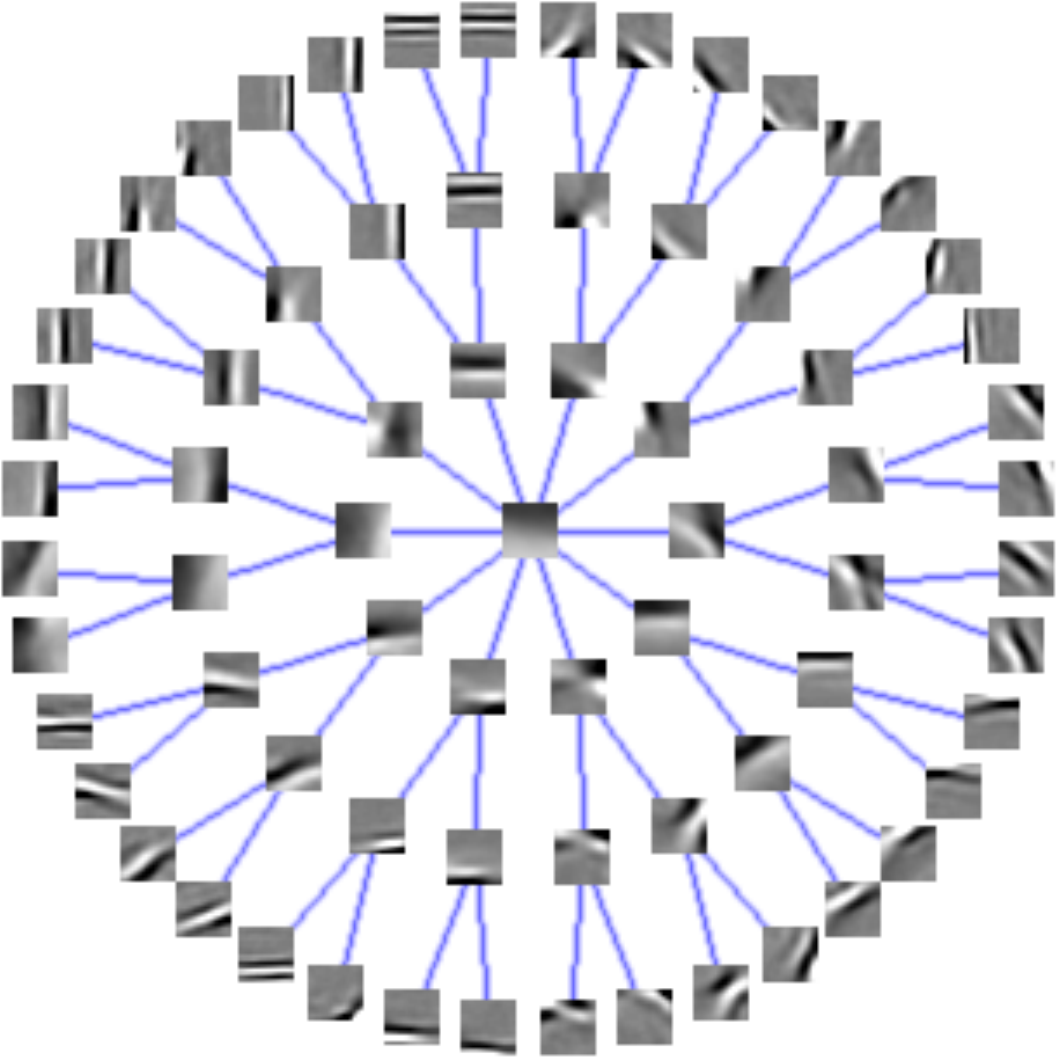}\label{subfig:tree2}} \hfill
   \subfigure[Tree structure~$2$.]{\includegraphics[width=0.49\linewidth]{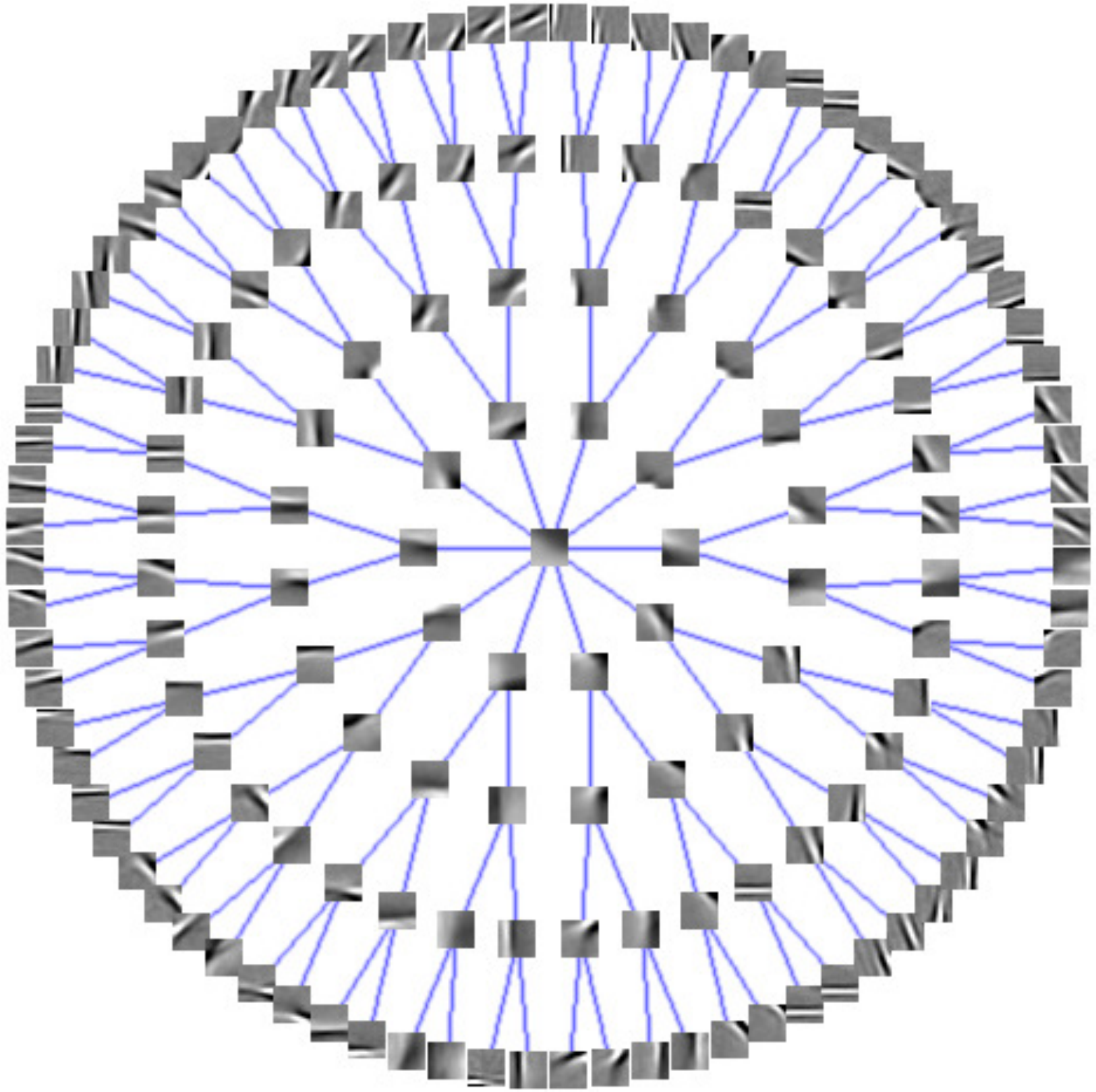}\label{subfig:tree}} \\
   \caption{Dictionaries obtained with the formulation~(\ref{eq:natstat_dict}) and a hierarchical structured-sparsity penalty~$\Psi$ for natural image patches of size~$16 \times 16$ pixels. The tree structures are pre-defined, and the dictionary elements naturally organize themselves in the tree during learning.
      For each tree, the root is represented in the middle of the figure.
      \subref{subfig:tree2}: the tree is of depth $4$ and the branching factors
      at depths $1,2,3$ are respectively $10$, $2$, $2$; \subref{subfig:tree}:
      the tree is slightly more complex; its depth is~$5$ and the branching
      factors at depths~$1,2,3,4$ are respectively $10$, $2$, $2$, $2$.
Figures borrowed from~\citet{Jenatton2010a,Jenatton2010b}. Best seen by zooming on a computer screen.}
   \label{fig:visudicttree} 
\end{figure}

\paragraph{Topographic dictionary learning.} A two-dimensional grid structure
has also been used for learning dictionaries~\citep{Kavukcuoglu2009,Mairal2011},
which was directly inspired from the topographic independent component analysis
formulation of \citet{hyvarinen2001}. The principle is similar to the
hierarchical case---that is, we can use a penalty~$\Psi$ as
in~(\ref{eq:dicttree}), but the set of groups~$\GG$ takes into account a
different structure.  We assume here that the dictionary elements are organized
on a grid, such that we can define neighborhood relations between them.

For instance, we can organize the $p$ dictionary elements on a $\sqrt{p} \times
\sqrt{p}$ grid, and consider $p$ overlapping groups that are $3 \times 3$ or $4
\times 4$ spatial neighborhoods on the grid (to avoid boundary effects, we
assume the grid to be cyclic). We represent some results obtained with such a
formulation in Figure~\ref{fig:visudicttopo}, where a function~$\Psi$ is
defined in~(\ref{eq:dicttree}) with~$q=2$. The regularization parameter~$\lambda$ 
is also manually tuned.  The dictionary elements automatically organize themselves
on the grid, and exhibit some intriguing spatial smoothness.  Another formulation
achieving a similar effect was also proposed by \citet{garrigues}, by mixing
sparse coding with a probabilistic model involving latent variables for groups
of neighbor variables. Another approach was also proposed by~\citet{gregor2011},
which models inhibition effects between dictionary elements.

\begin{figure}
   \centering
   \subfigure[With~$3 \times 3$ neighborhoods.]{\includegraphics[width=0.49\linewidth]{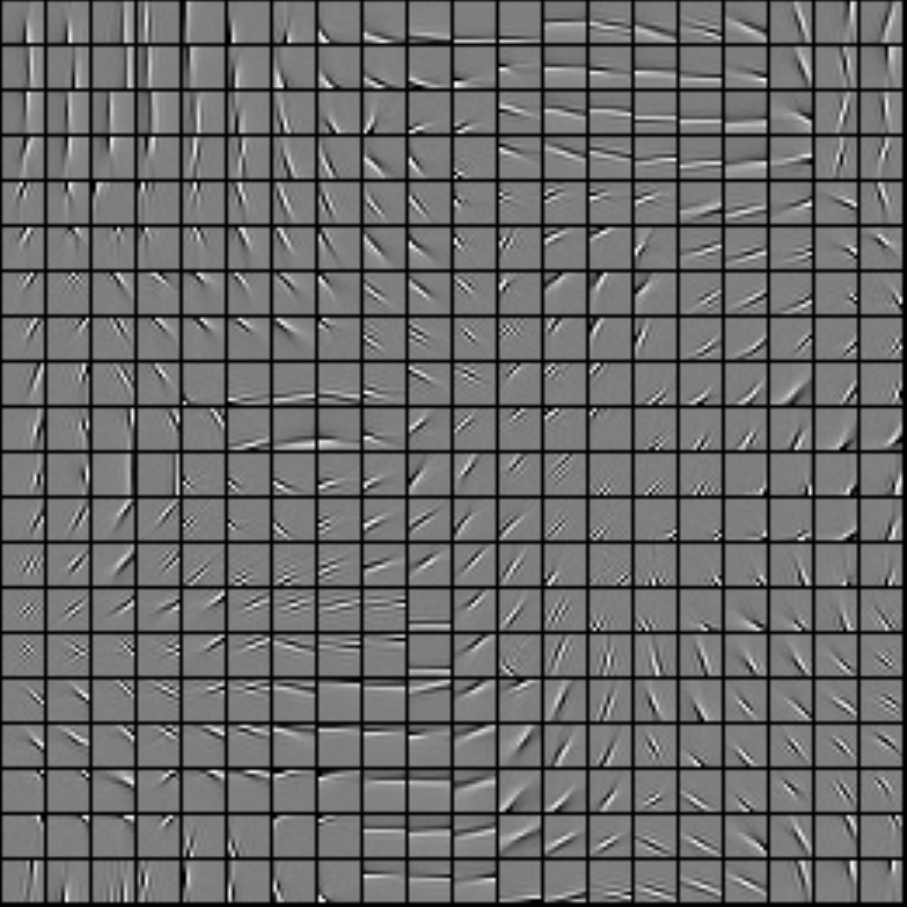}\label{subfig:topo}} \hfill
   \subfigure[With~$4 \times 4$ neighborhood.]{\includegraphics[width=0.49\linewidth]{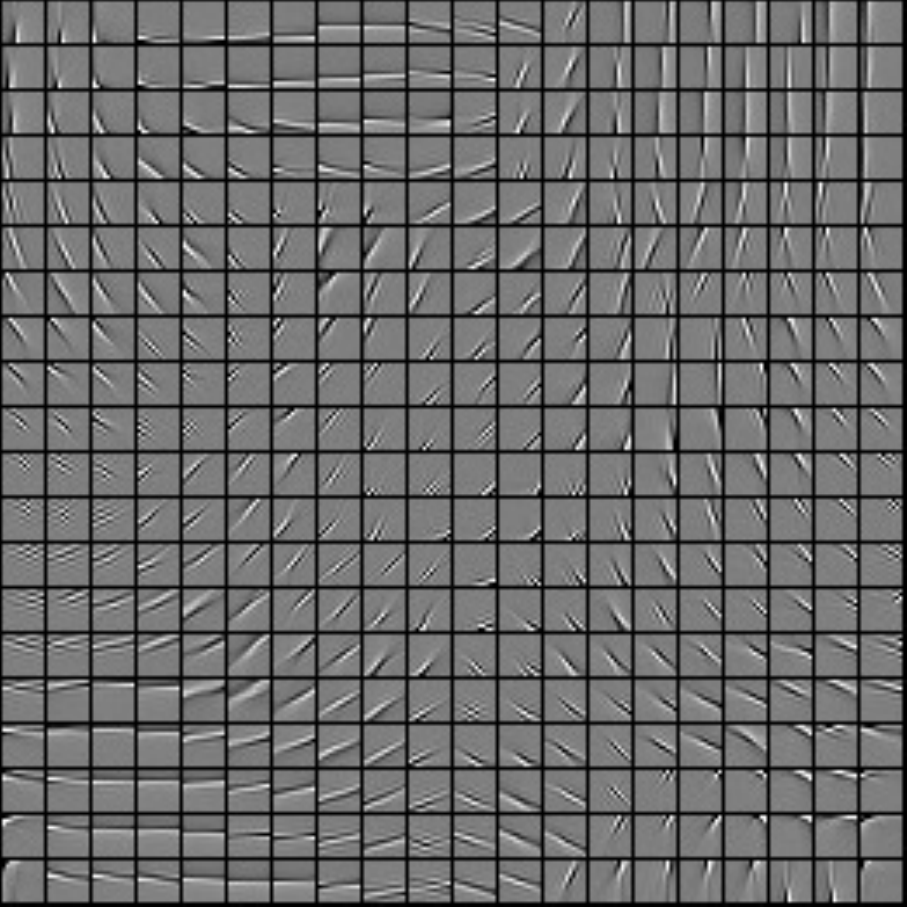}\label{subfig:topo2}} \\
   \caption{Dictionaries obtained with the formulation~(\ref{eq:natstat_dict}) and a structured sparsity penalty~$\Psi$ inducing a grid structure in the dictionary. The dictionaries are computed on whitened natural image patches of size~$12 \times 12$ pixels for two different group structures. Figure borrowed from~\citet{Mairal2011}. Best seen on a computer screen.}
   \label{fig:visudicttopo} 
\end{figure}

%% file: content_arxiv/natstat_other.tex
We have focused our study of learning methods for natural image patches on sparse
coding and clustering techniques so far.  Other matrix factorization
formulations have shown to be also effective for this task. In particular, 
independent component analysis (ICA) was applied successfully to image patches
around the same time as dictionary learning by~\citet{bell1997}, who also
obtained Gabor-like patterns from whitened data.

\paragraph{Independent component analysis.}
A significantly different point of view than dictionary learning consists of
factorizing \emph{whitened} data~$\X$ as a product~$\D\A$, where the sparsity
principle is replaced by an assumption of statistical independence. Therefore,
the ICA approach requires by nature a probabilistic interpretation of the data.
We assume that~$\x$ in~$\Real^m$ is a random variable representing a signal---here, a natural image patch---
and the columns of~$\X=[\x_1,\ldots,\x_n]$ are random realizations of the
variable~$\x$---that is, a set of natural image patches selected uniformly at
random. The principle of ICA assumes that there exists a
factorization~$\x=\D\alphab$, where $\D$ is an orthogonal matrix,
and~$\alphab= \D^\top \x$ is a random vector whose entries are statistically
independent~\citep[see][]{bell1997,hyvarinen2004,hyvarinen2009}.
Then, the columns of~$\D$ are called independent components.

Given data, the concept of ``statistical independence'' alone does not
immediately yield a precise cost function to optimize, and thus numerous
variants of ICA exist in the literature. Ultimately, most variants of ICA try to
approximate the statistical independence assumption, by finding an orthogonal
transformation~$\D^\top$ such that the probability density~$p(\alphab)$,
assuming it exists, factorizes into the product of its marginals $\prod_{j=1}^p
p(\alphab[j])$.  A popular cost function to compare probability distributions is
the Kullback-Leibler distance~\citep[see][]{cover2006}, denoted by~$KL$, and a natural goal of ICA
is to minimize over~$\D$ 
\begin{displaymath}
   KL\left(p(\alphab), \prod_{j=1}^p p(\alphab[j])\right) \defin \int_{\Real^p} p(\alphab) \log \left(\frac{p(\alphab)}{\prod_{j=1}^p p(\alphab[j])}\right) d\alphab,
\end{displaymath}
which is equal to zero if and only if the~$\alphab[j]$'s are independent.  The
quantity $KL$ above is also called the mutual information between the
variables~$\alphab[j]$, and can also be simply written in terms of the entropy
function~$H$:\footnote{The entropy is defined as $H(x) \defin \int p(x)\log p(x)
dx$~\citep[see][]{cover2006}. In general,~$H(x)$ is not computable exactly since the probability density~$p$ is unknown.}
\begin{displaymath}
   KL\left(p(\alphab), \prod_{j=1}^p p(\alphab[j])\right) = \sum_{j=1}^p H(\alphab[j]) - H(\alphab).
\end{displaymath}
Interestingly, the entropy $H(\alphab)$ can be shown to be independent of~$\D$
when~$\D$ is orthogonal and~$\x$ is whitened with identity covariance, such
that the goal of ICA can be expressed as the minimization of the sum of entropies
\begin{equation}
   \sum_{j=1}^p H(\alphab[j])  =   \sum_{j=1}^p H(\d_j^\top \x). \label{eq:ica}
\end{equation}
Unfortunately, entropy remains an abstract quantity that is not computable, and
thus~(\ref{eq:ica}) is not yet an objective function that is easy to minimize.
Turning~(\ref{eq:ica}) into a concrete formulation and algorithm can be
achieved by following different strategies. One of them consists of
parameterizing the densities~$p(\d_j^\top \x)$, assuming they admit a known parametric form,
and minimizing~(\ref{eq:ica}) becomes equivalent to performing maximum
likelihood estimation~\citep[see][]{hyvarinen2004}. Another approach consists
of using non-parametric estimators of the entropy~\citep{pham2004}. Finally,
a large class of methods consists of using approximations of the entropy cost
function~(\ref{eq:ica}) or encourage the distributions of the~$\alphab[j]$'s to
be ``non-Gaussian''~\citep{cardoso2003}. The non-gaussianity principle might
not seem intuitive at first sight, but it is in fact rather natural when
considering a minimization entropy problem.  Among all probability
distributions with same variance, the
Gaussian ones are known to maximize entropy~\citep{cover2006}. As such,
minimizing~(\ref{eq:ica}) implies that we are looking for some non-Gaussian
random variables~$\alphab[j]$.

Extensions where~$\D$ is not orthogonal and when the data~$\X$ is not whitened
exist, but we have omitted them in this brief presentation for
simplicity~\citep[see][for a better overview]{hyvarinen2004}.  A remarkable
fact is that ICA applied to natural image patches yields very similar results
to the sparse coding model of~\citet{field1996,olshausen1997}, as shown
by~\citet{bell1997}. In fact, an equivalence exists between the two
formulations under particular asymptotic conditions~\citep[see the appendix
of][]{olshausen1997}.

\paragraph{Non-negative matrix factorization.}
Another popular unsupervised learning technique is the non-negative matrix
factorization (NMF)~\citep{paatero1994}, which was shown to be able to automatically
discover some interpretable features on datasets of human faces~\citep{lee1999}.
When the data matrix~$\X$ is non-negative, the method consists of finding a
product~$\D\A$ that approximates~$\X$, and where each factor is also
non-negative. The corresponding formulation can be written as
\begin{displaymath}
   \min_{\D \in \Real^{m \times p}, \A \in \Real^{p \times n}} \|\X-\D\A\|_\fro^2 \st \D \geq 0~~\text{and}~~\A \geq 0.
\end{displaymath}
Even though this technique could be applied to natural image patches without
pre-processing, it has been empirically observed that NMF does not yield any
interpretable feature when applied to such data~\citep{mairal2010}.  Note that
NMF is often used with a different loss function than the square loss, in particular for audio applications~\citep{fevotte2009}.

\paragraph{Archetypal analysis.} 
Around the same time as dictionary learning and non-negative matrix
factorization, archetypal analysis was introduced by~\citet{cutler1994} for
discovering latent factors from high-dimensional data. The main motivation was
to propose an unsupervised learning technique that is more interpretable than
principal component analysis.  The method seeks a factorization~$\X=\D\A$,
with two symmetrical geometrical constraints. First, each column~$\d_j$ of~$\D$, called
``archetype'' is forced to be a convex combination of a few data points. In 
other words, for all~$j \in \{1,\ldots,p\}$, there exists a vector~$\betab_j$ such that
$\d_j = \X\betab_j$, and~$\betab_j$ is in the simplex~$\Delta_n$ defined as
\begin{equation}
   \Delta_n \defin \left\{  \betab \in \Real^n \st \betab \geq 0~~\text{and}~~ \sum_{i=1}^n \betab[i] = 1 \right\}. \label{eq:simplex}
\end{equation}
Second, each data point~$\x_i$ is encouraged to be close to the convex hull of
the archetypes, meaning that~$\x_i$ should be close to a product~$\D\alphab_i$,
where~$\alphab_i$ is in the simplex~$\Delta_p$, defined as
in~(\ref{eq:simplex}) when replacing~$n$ by~$p$.
The resulting formulation is the following
\begin{displaymath}
  \min_{\substack{
        \alphab_i \in \Delta_p~\text{for}~ 1\leq i \leq n \\
  \betab_j \in \Delta_n~\text{for}~ 1\leq j \leq p}}  \sum_{i=1}^n \|\x_i - \D \alphab_i\|_2^2  \st ~\forall j, \d_j = \X\betab_j, 
\end{displaymath}
which, in the matrix factorization form, is equivalent to
\begin{displaymath}
   \min_{\substack{
         \alphab_i \in \Delta_p~\text{for}~ 1\leq i \leq n \\
   \betab_j \in \Delta_n~\text{for}~ 1\leq j \leq p}}  \left\|\X-\X\B\A\right\|_{\text{F}}^2,
\end{displaymath}
where~$\A=[\alphab_1,\ldots,\alphab_n]$, $\B=[\betab_1,\ldots,\betab_p]$ and
the matrix of archetypes~$\D$ is equal to the product~$\X\B$.

Archetypal analysis is closely related to the dictionary learning formulation
of~\citet{field1996,olshausen1997} since the simplicial constraint $\alphab_i
\in \Delta_p$ can be equivalently rewritten as the~$\ell_1$-norm
constraint~$\|\alphab_i\|_1=1$ associated to a non-negativity one~$\alphab_i \geq 0$. As
a result, the coefficients~$\alphab_i$ are sparse in practice and archetypal analysis
provides a sparse decomposition of the data points~$\x_i$. 
The main difference with dictionary learning is the set of constraints~$\d_j =
\X\betab_j$. Because the vectors~$\betab_j$ are constrained to be in the
simplex~$\Delta_n$, they are encouraged to be sparse and each archetype is a convex
combination of a few data points only. Such a relation between latent factors~$\d_j$
and the data~$\X$ is useful whenever interpreting~$\D$ is important, \eg, in 
experimental sciences. For example, clustering techniques provide
such associations between data and centroids.
It is indeed common in genomics to cluster gene expression data
from several individuals, and to interpret each centroid by looking for some common
physiological traits among individuals of the same cluster~\citep{eisen98}.

Archetypal analysis is also related to non-negative matrix
factorization~\citep{paatero1994}. When the data~$\X$ is non-negative, it is
easy to see that both the matrix~$\A$ and the matrix~$\D$ of archetypes are
also non-negative. Unfortunately, archetypal analysis did not encountered as
much success as dictionary learning or NMF, despite the fact that it provides
an elegant methodology for interpreting its output. One of the reason for this
lack of popularity may be that no efficient software has been available for a
long time, which may have limited its application to important scientific
problems.  Based upon such observations, \citet{chen2014} have recently
revisited archetypal analysis for computer vision, with the goal of bringing
back this powerful unsupervised learning technique into favor. They made
publicly available an efficient implementation of archetypal analysis in the
SPAMS software, and demonstrated that it could perform as well as dictionary
learning for some classification tasks, while offering natural mechanisms for
visualizing large databases of images (see Section~\ref{sec:visual_other}).

\paragraph{Bayesian models for dictionary learning.}
Early models of dictionary learning were probabilistic~\citep[see, for
instance,][]{lewicki1999,lewicki2000}. Each image patch~$\x$ is considered to be a random variable
with a normal distribution given a dictionary~$\D$ and a latent variable~$\alphab$. More precisely,
we have the conditional law
\begin{equation}
   p(\x | \D, \alphab) \propto e^{-\frac{1}{2\sigma^2}\|\x-\D\alphab\|_2^2}, \label{eq:xDalpha}
\end{equation}
and the prior distribution on the coefficients $\alphab$ is Laplacian:\footnote{Note that choosing the Laplace distribution for the prior does not mean implicitly assuming that the true 
distribution of the latent variable~$\alphab$ is Laplace. In fact, \citet{gribonval2012} have shown that if this was the case, the choice~(\ref{eq:laplacian}) would be inappropriate in the context
of maximum a posteriori estimation.}
\begin{equation}
   p(\alphab) \propto e^{-\lambda\|\alphab\|_1}. \label{eq:laplacian}
\end{equation}
With this probabilistic model, maximizing the posterior $p(\alphab |
\x,\D)$ given some observation~$\x$ and a dictionary~$\D$, or equivalently
minimizing the negative log posterior, yields the Lasso formulation. Indeed, by using Bayes' rule, we have
\begin{equation}
   \begin{split}
      - \log p(\alphab | \x,\D) & = - \log  p( \x | \D,\alphab) - \log  p(\alphab) \\
         & = \frac{1}{2\sigma^2} \|\x-\D\alphab\|_2^2 + \lambda \|\alphab\|_1.
   \end{split} \label{eq:maplasso}
\end{equation}
When $n$ natural image patches~$\x_1,\ldots,\x_n$ are observed, learning a dictionary~$\D$ can
be achieved by (i) modeling each~$\x_i$ according to~(\ref{eq:xDalpha}) with a latent variable~$\alphab_i$ associated to~$\x_i$,
(ii) assuming the~$\alphab_i$'s to be statistically independent one from each other, (iii) choosing a prior
distribution~$p(\D)$, \eg, uniform on the set of matrices~$\CC$ defined
in~(\ref{eq:natstat_dict}), and (iv) computing the point estimate
\begin{displaymath}
      \min_{\D,\alphab_1,\ldots,\alphab_n} - \log p(\D,\alphab_1,\ldots,\alphab_n | \x_1,\ldots,\x_n).
\end{displaymath}
By using again Bayes' rule and by using the statistical independence assumption, we
recover the classical ``matrix factorization'' formulation presented earlier:
\begin{equation}
   \begin{split}
      - \log p(\D,\A | \X) & = - \log p(\alphab_1,\ldots,\alphab_n | \X, \D) - \log p(\D) \\ 
       & = \sum_{i=1}^n - \log p(\alphab_i | \x_i,\D) - \log p(\D) \\
       & = \sum_{i=1}^n \frac{1}{2\sigma^2} \|\x_i-\D\alphab_i\|_2^2 + \lambda \|\alphab_i\|_1 - \log p(\D),
   \end{split} \label{eq:dictprob}
\end{equation}
where~$\X=[\x_1,\ldots,\x_n]$ and~$\A=[\alphab_1,\ldots,\alphab_n]$.
Simply maximizing the posterior distribution is however not satisfactory 
from the point of view of a Bayesian statistician, who is usually not
interested in point estimates, but in the full posterior
distribution~\citep[see][]{bayarri2004}.\footnote{It seems that a common
misconception in the computer vision and image processing literature is to call
``Bayesian'' any probabilistic model that uses Bayes' rule such
as~(\ref{eq:maplasso}). Such a terminology for maximum a posteriori (MAP)
estimation is slightly misleading regarding the hundred-year-old debate among
frequentists and Bayesian statisticians. In general, the latter do not assume that
there exists a ``true'' parameter that can be obtained with MAP estimation;
instead, they treat all parameters as random variables and model their 
uncertainty~\citep[see][]{bayarri2004}.}
In fact, the formulation~(\ref{eq:dictprob}) discards any uncertainty regarding
the model parameters~$\D$ and~$\A$ and would be called a ``frequentist''
approach in statistics: assuming that one can repeatedly draw natural image
patches at random, one wishes to find a dictionary~$\D$ that is good on
average, which is nothing else than the (non-probabilistic) empirical risk
minimization point of view presented in Section~\ref{sec:sparsecoding}.

A first step towards a Bayesian dictionary learning formulation consists of
integrating instead of maximizing with respect to the latent
variables~$\alphab_i$, as typically done in Bayesian sparse linear
models~\citep{park2008,seeger2008bayesian}, and maximizing the
posterior~$p(\D|\X)$. As a result, we need to minimize
\begin{equation}
   \begin{split}
      - \log p(\D | \X) & = - \log p(\X | \D) - \log p(\D) \\ 
       & = \sum_{i=1}^n - \log p(\x_i | \D) - \log p(\D) \\
       & = \sum_{i=1}^n - \log \left( \int_{\alphab_i \in \Real^p} p(\x_i,\alphab_i | \D) d\alphab_i\right)  - \log p(\D) \\
       & = \sum_{i=1}^n - \log \left(\int_{\alphab_i \in \Real^p} p(\x_i |\alphab_i,\D)p(\alphab_i) d\alphab_i\right)  - \log p(\D). 
   \end{split} \label{eq:dictprob2}
\end{equation}
The formulation is again a pointwise estimator for the dictionary~$\D$ but it
can be argued to be more ``Bayesian'' due to the treatment of the latent
variables~$\alphab_i$. Another possibility is to maximize the data
likelihood~$p(\X|\D)$, resulting in the same formulation
as~(\ref{eq:dictprob2}) without the prior term~$-\log p(\D)$; this is the
strategy adopted for instance by \citet{lewicki1999}. Note that
optimizing~(\ref{eq:dictprob2}) is difficult since the integrals do not admit
an analytical closed form, and some approximations have to be
made~\citep{lewicki1999}.

Recently, a fully Bayesian dictionary learning formulation has been proposed
by~\citet{zhou2009,zhou2012}. By ``fully Bayesian'', we mean that all model
parameters such as dictionary~$\D$, coefficients~$\alphab$, but also
hyper-parameters such as~$\sigma$ are modeled with probability distributions,
and the full posterior density is estimated by using Gibbs sampling. The model
is in fact significantly different than the one described by~(\ref{eq:xDalpha})
and~(\ref{eq:laplacian}). In a nutshell, it involves a Beta-Bernoulli process
for selecting the non-zero coefficients in~$\alphab$, Gaussian priors with
Gamma hyperpriors for the value of these coefficients and for the dictionary
elements.  Compared to the traditional matrix factorization
formulation~(\ref{eq:natstat_dict}), a Bayesian treatment has both advantages
and drawbacks, which often appear in passionate discussions between
Bayesians and frequentists: on the one hand, the Bayesian formulation is robust
to model misspecification and can learn some hyper-parameters such as the
noise level~$\sigma$; on the other hand, Bayesian inference is significantly
more involved and computationally costly than matrix factorization.

\paragraph{Convolutional sparse coding.}
The dictionary learning formulation of~\citet{field1996,olshausen1997} is appropriate
for processing natural image patches. A natural extension to
\emph{full} images instead of patches is called ``convolutional sparse
coding''. It consists of linearly decomposing an image by using small
dictionary elements placed at all possible locations in the image.  Such a
principle was described early by \citet{zhu2005} without concrete application.
It has been revisited recently, and has been shown useful for various
recognition
tasks~\citep{zeiler2010,zeiler2011,rigamonti2011,rigamonti2013,kavukcuoglu2010}

Specifically, let us consider an image~$\x$ represented by a vector of size~$l$
and let the binary matrix~$\R_k$ in~$\{0,1\}^{m \times l}$ be the linear operator such
that~$\R_k \x$ is the patch of size~$\sqrt{m} \times \sqrt{m}$ from~$\x$
centered at the pixel indexed by~$k$. 
The purpose of $\R_k$ is to ``extract'' a patch at a specific location,
but it is also easy to show that its adjoint
operator~$\R_k^\top$ is such that~$\R_k^\top \z$,
for some vector~$\z$ in~$\Real^m$, is a vector of size~$l$ representing an image with zeroes
everywhere except the patch~$k$ that contains the~$m$ entries of~$\z$.
In other words, $\R_k^\top$ positions at pixel~$k$ a small patch in the larger image.  To simplify, we assume that there is no
boundary effect when the pixel~$k$ is
close to the image border, \eg, we use zero-padding to define patches that
overlap with the image boundary. Then, it becomes natural to decompose~$\x$ with the 
Lasso formulation:
\begin{equation}
   \min_{\A \in \Real^{p \times l}} \frac{1}{2}\left\|\x - \sum_{k=1}^l \R_k^\top \D \alphab_k\right\|_2^2 + \lambda\sum_{k=1}^l \|\alphab_k\|_1,\label{eq:convolutional}
\end{equation}
where~$\D$ is a dictionary in~$\Real^{m \times p}$, and, as usual,
$\A=[\alphab_1,\ldots,\alphab_l]$.  Solving~(\ref{eq:convolutional}) is easy
with any standard method adapted to large-scale $\ell_1$-regularized
convex problems~\citep[see][for a review]{bach2012}. For
instance,~\citet{rigamonti2011} uses a proximal gradient method,
whereas~\citet{kavukcuoglu2010} uses a coordinate descent scheme.  To conduct
the experiment of this paragraph, we have implemented the ISTA and FISTA
algorithms of~\citet{beck2009} for solving~(\ref{eq:convolutional}), and we have made it available in the SPAMS
toolbox.\footnote{\url{http://spams-devel.gforge.inria.fr/}.}

Given now a collection of~$n$ images~$\x_1,\ldots,\x_n$, which are assumed to be of the
same size for simplicity, the dictionary~$\D$ can be learned by minimizing 
\begin{equation}
   \min_{\substack{\{\A_i \in \Real^{p \times l}\}_{i=1\ldots n}\\ \D \in \CC}} \frac{1}{n}\sum_{i=1}^n \left(\frac{1}{2}\left\|\x_i - \sum_{k=1}^l \R_k^\top \D \alphab_{i,k}\right\|_2^2 + \lambda\sum_{k=1}^l \|\alphab_{i,k}\|_1\right),\label{eq:convolutionaldict}
\end{equation}
where~$\A_i=[\alphab_{i,1},\ldots,\alphab_{i,k}]$ for all~$i=1,\ldots,n$. A 
natural optimization scheme for addressing~(\ref{eq:convolutionaldict}) is
to alternate between the minimization of~$\D$ with~$\A$ fixed and vice versa,
as often done for the classical dictionary learning problem (see
Section~\ref{chapter:optim}). Even though finding the global optimum
of~(\ref{eq:convolutionaldict}) is not feasible because the objective function
is nonconvex, alternate minimization provides a stationary
point~\citep{bertsekas}. Updating the dictionary~$\D$ with fixed
coefficients~$\A$ can be achieved by projected gradient descent.

In Figure~\ref{fig:convolutional}, we visualize two dictionaries of
size~$p=100$ with elements of size~$m=16 \times 16$ pixels.
We learned them on~$30$ whitened natural images that are rescaled such
that each $\sqrt{m} \times \sqrt{m}$ patch has unit $\ell_2$-norm on average.
We perform $300$ steps of alternate minimization. For updating~$\A$ or~$\D$,
we always use $5$ iterations of the algorithm ISTA~\citep{beck2009}, yielding a
small decrease of the objective function at each step.
Interestingly, we observe that the learned features have the following properties:
(i) in general, they are well centered; (ii) some of them are more complex
than the ones obtained with the classical dictionary learning formulation.  In
other words, they go beyond simple Gabor features, with curvy patterns,
corners, and blobs.
\begin{figure}[hbtp]
   \centering
   \subfigure[With $\lambda=0.2$.]{\includegraphics[width=0.49\linewidth]{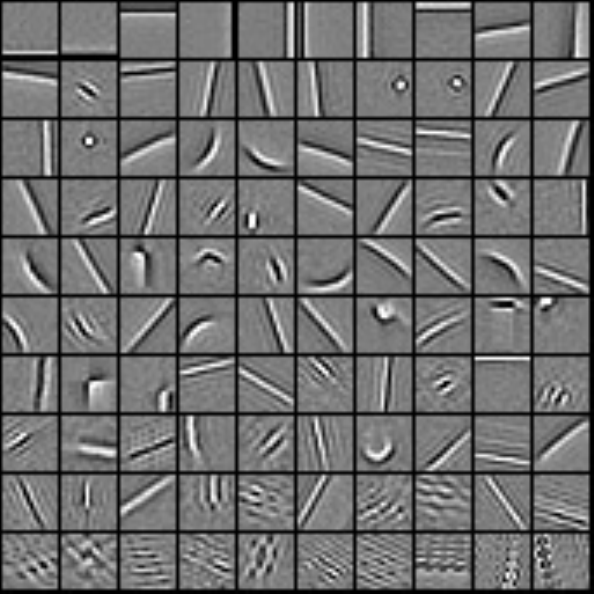}} \hfill
   \caption{Visualization of $p=100$ dictionary elements learned on~$30$
      whitened natural images. Dictionary elements are ordered from most to
      least used (from left to right, then top to bottom).}
   \label{fig:convolutional} 
\end{figure}

%% file: content_arxiv/natstat_discussion.tex
In this section, we have reviewed a large number of unsupervised learning
techniques that are able to discover underlying structures in natural images.
We have focused on matrix factorization and sparse coding, but other
techniques that are out of the scope of our study are known to achieve similar
results. We briefly mention a few of them.
 
A successful approach is the Gaussian mixture model (GMM). Even though it is
classical, it has been shown only recently to be well adapted to represent natural image
patches. GMMs for image patches were first investigated by~\citet{yu2012} and
then further developed by~\citet{zoran2011}, yielding very good results for several
image restoration tasks.  The probabilistic model consists of
representing the distribution of a patch~$\x$ as a convex combination of~$p$
Gaussian distributions with means~$\mub_k$ and (often diagonal) covariance
matrices~$\Sigmab_k$ for~$k=1,\ldots,p$. Given a database of
patches~$\x_1,\ldots,\x_n$, the EM-algorithm~\citep{dempster1977} is typically
used to learn means and covariances. When displaying the means, small
localized Gabor filters typically appear when the data is
whitened~\citep{coates2011}.

Finally, small oriented localized filters can also be learned from natural
image patches by using undirected graphical models such as restricted Boltzmann
machines (RBM)~\citep[see][]{hinton2002,bengio2009,coates2011}, and also by using
convolutional neural networks~\citep{lecun1998,zeiler2013}, which have recently
gained some popularity for solving large-scale visual recognition tasks.

%% file: content_arxiv/image_intro.tex
We have previously shown that natural image patches could be modeled by using
dictionary learning techniques and its variants, but we have not presented any
concrete application yet. We have indeed focused so far on the
\emph{interpretation} of the visual patterns obtained by various methods, but
one may wonder whether or not these interpretable structures can be useful for
prediction tasks. For quite a long time, the answer to this question was
uncertain, until the work of~\citet{elad2006} on image denoising that achieved
state-of-the-art results compared to other approaches at that time.

The current section is devoted to several applications of dictionary learning
in image processing. First, we present the simple and yet effective scheme for
image denoising introduced by~\citet{elad2006}, and then move to more complex
tasks. In general, we show that dictionary learning can be useful for
predicting missing visual information, leading to state-of-the-art results for
image inpainting and demosaicking~\citep{mairal2008,mairal2009}, 
super-resolution and
deblurring~\citep{yang2010,yang2012,couzinie2011,dong2011,zeyde2012,wang2012},
face compression~\citep{bryt2008}, or for inverting non-linear local
transformations~\citep{mairal2012}. We also present natural extensions of
dictionary learning to video processing~\citep{protter2009}, and finally, we
conclude this section with a presentation of other patch-modeling
approaches~\citep{buades2005,awate2006,dabov2,takeda2007,zoran2011,yu2012,chatterjee2012}.

A small part of the material of this section is borrowed from the PhD thesis of the
first author~\citep{mairal_thesis}.

%% file: content_arxiv/image_denoising.tex
Let us consider the classical problem consisting of
restoring a noisy image $\y$ in~$\Real^n$ that has been corrupted by white
Gaussian noise with known standard deviation~$\sigma$.\footnote{The
formulations presented in this section can be also modified to handle other
types of noise, such as Poisson noise~\citep[see][]{giryes2013}.} In many
cases, image denoising is formulated with an \emph{energy} to
minimize~\citep[\eg,][]{rudin}:
\begin{equation}
   \min_{\x \in \Real^n} \frac{1}{2}\|\y-\x\|_2^2 + \lambda \psi(\x), \label{eq:denoise}
\end{equation}
where~$\psi: \Real^n \to \Real$ is a regularization function, and the quadratic
``data-fitting'' term ensures that the estimate~$\x$ is close to the noisy
observation~$\y$. When~(\ref{eq:denoise}) is derived from a probabilistic image
model, it is often interpreted from the point of view of maximum a posteriori
estimation, where~$\lambda\psi(\x)$ is related to some negative log prior
distribution on~$\x$. Even though~(\ref{eq:denoise}) looks simple at first
sight, finding a good regularization function~$\psi$ is difficult, and in fact,
it is probably one of the most important research topic in image processing
nowadays. In early work, various smoothness assumptions about~$\x$ have led
to different functions~$\psi$. Specifically, natural images were assumed to
have small total variation~\citep{rudin}, or the smoothness between adjacent
pixel values was modeled with Markov random fields (MRF)~\citep{zhu1997}.

As described in Section~\ref{sec:wavelets}, sparse image models based on
wavelets have also been popular for the denoising
task~\citep[see][]{mallat2008}. This line of work has inspired the method we
present in this section but the latter differs in two aspects: (i) the approach
of~\citet{elad2006} is patch-based like other ``modern'' image processing
methods~\citep{buades2005,dabov2,roth2009}; (ii) it adapts to the image patches
with dictionary learning, and thus it does not use any pre-defined wavelet
basis.

\paragraph{Patch denoising given a fixed dictionary.}
\citet{elad2006} have proposed a simple denoising procedure that treats every
patch independently.
Let us consider the $\sqrt{m} \times \sqrt{m}$ patch~$\y_i$ of the noisy image~$\y$,
centered at the pixel indexed by~$i$, and assume that 
an appropriate dictionary~$\D$ in~$\Real^{m \times p}$ is given.\footnote{For simplicity, we always assume in this monograph that the patches are square, but all approaches can be easily extended
to deal with other patch shapes.} 
Then, the patch~$\y_i$ is denoised by the following steps:
\begin{enumerate}
   \item center $\y_i$ (see
      Section~\ref{subsec:preprocess}),
      \begin{displaymath}
         \y_i^c \defin \y_i - \mu_i \ones_m  ~~~\text{with}~~~ \mu_i \defin \frac{1}{n}\ones_m^\top \y_i;
      \end{displaymath}
   \item find a sparse linear combination of dictionary elements that
      approximates~$\y_i^c$ up to the noise level:
      \begin{equation}
         \min_{\alphab_i \in \Real^p} \|\alphab_i\|_0 \st \|\y_i^c-\D\alphab_i\|_2^2
         \leq \varepsilon,  \label{eq:reconstruct_patch}
      \end{equation}
      where~$\varepsilon$ is proportional to the noise variance~$\sigma^2$;
   \item add back the mean component to obtain the clean estimate~$\hatx_i$:
      \begin{displaymath}
         \hatx_i \defin \D\alphab_i^\star + \mu_i \ones_m,
      \end{displaymath}
      where~$\alphab_i^\star$ is the solution, or approximate solution, obtained when
      addressing~(\ref{eq:reconstruct_patch}).
\end{enumerate}
The approach we have just described yields two difficulties. First, it is
important to choose a good value for~$\varepsilon$. Second, 
problem~(\ref{eq:reconstruct_patch}) is unfortunately NP-hard.
For patches of size~$m = 8 \times 8 = 64$ pixels, \citet{elad2006} recommend
the value~$\varepsilon=m(1.15 \sigma)^2$. Another effective heuristic proposed
by~\citet{mairal2009} assumes that the targeted residual
$\y_i-\D\alphab_i$ behaves as the Gaussian noise of (supposedly known)
variance~$\sigma^2$, such that the quantity $\|\y_i-\D\alphab_i\|_2^2/\sigma^2$
follows a $\chi$-square distribution with~$m$ degrees of freedom. Then, the
heuristic sets~$\varepsilon = \sigma^2 F_m^{-1}(\tau)$ where $F_m^{-1}$ is the
inverse cumulative distribution function of the~$\chi_m^2$ distribution.
Selecting the value $\tau=0.9$ leads to appropriate values of~$\varepsilon$ in
practice for different patch sizes~$m$.  For dealing
with~(\ref{eq:reconstruct_patch}), greedy algorithms such as orthogonal
matching pursuit~\citep{pati1993} are known to provide approximate solutions
that are good enough for the denoising task~\citep{elad2006}. Such algorithms
are presented in details in Section~\ref{sec:optiml0}.

After having presented a simple ``patch'' denoising procedure given a fixed
dictionary~$\D$, we study three questions that will explain the different steps
needed for denoising a full image:
\begin{enumerate}
   \item \emph{how do we reconstruct the full image from the local patch models?}
   \item \emph{which dictionary should we choose?}
   \item \emph{how does $\ell_1$ compares with~$\ell_0$ for image denoising?}
\end{enumerate}

\paragraph{From patches to full image estimation.}
The denoising scheme proposed by~\citet{elad2006} processes
\emph{independently} every patch~$\y_i$ from the noisy image~$\y$, before constructing
the denoised image by \emph{averaging} the patch estimates~$\hatx_i$. More
precisely, since the patches~$\y_i$ overlap, each pixel belongs to~$m$ different patches
and thus admits~$m$ estimates.\footnote{For simplicity, we choose here to neglect
boundary effects for pixels that are close to the image border, which technically belong to fewer
patches.} Then, the denoised image~$\hatx$ in~$\Real^n$ is obtained as follows:
\begin{equation}
   \hatx = \frac{1}{m}\sum_{i=1}^n \R_i^\top \hatx_i,\label{eq:patch_averaging}
\end{equation}
where~$\R_i$ in~$\Real^{n \times n}$ is a binary matrix defined as in
Section~\ref{subsec:matrix_other}---that is, $\R_i^\top \hatx_i$ is a vector of
size~$n$ corresponding to an image with zeroes everywhere except for the patch
indexed by~$i$ that contains~$\hatx_i$. In other words,~$\R_i^\top$ is a linear
operator that ``positions'' a patch of size~$m$ at the pixel~$i$ in a larger
image with~$n > m$ pixels.\footnote{Note that the original averaging scheme
of~\citet{elad2006} slightly differs from~(\ref{eq:patch_averaging}) since
their estimate $\hatx$ is defined as $\hatx = (1-\lambda)\y + \lambda
({1}/{m})\sum_{i=1}^n \R_i^\top \hatx_i$,
where~$\y$ is the noisy image. We have empirically found that the
setting~$\lambda=1$ leads in fact to very good results and we have thus omitted
the parameter~$\lambda$.}
As a result, Eq.~(\ref{eq:patch_averaging}) is nothing else than an averaging
procedure, which is the simplest way of aggregating estimators of the same
quantity.

We remark that this averaging strategy is related to the translation-invariant
wavelet denoising technique of \citet{coifman1995}, where a clean
image is reconstructed by averaging estimates obtained by denoising several
shifted versions of an input image.  Even though aggregating
estimators by straight averaging might look suboptimal, we are not aware of any
other technique, in the context of local sparse linear models, leading to
better results for reconstructing the final image from the estimation of overlapping patches.

\paragraph{Use of generic versus global versus image-adaptive dictionary.}
The first baseline considered by~\citet{elad2006} uses a fixed overcomplete
discrete cosine transform (DCT) dictionary, which was presented earlier in
Figure~\ref{fig:dct}. We call such an approach ``generic'' since the dictionary
is pre-defined. The second strategy, dubbed ``global'', consists of learning a
dictionary on an external database of clean patches, simply following the
methodology of Section~\ref{chapter:patches}.  Finally, we also present the
``adaptive'' strategy of~\citet{elad2006}, which learns the dictionary from the
pool of noisy patches:
\begin{equation}
   \min_{\D \in \CC, \A \in \Real^{p \times n}} \frac{1}{n}
   \sum_{i=1}^n\psi(\alphab_i) \st  \|\y_i-\D\alphab_i\|_2^2 \leq
   \varepsilon,\label{eq:dict_noisy}
\end{equation}
where~$\CC$ is defined in Section~\ref{chapter:patches}, and~$\psi$ is a
sparsity-inducing penalty. \citet{elad2006} originally use the~$\ell_0$-penalty
in place of~$\psi$, and optimize~(\ref{eq:dict_noisy}) with the K-SVD
algorithm~\citep{aharon2006}. Later, additional experiments have shown that
learning the dictionary with~$\ell_1$ could bring some
benefits~\citep{mairal2009}. In fact, we experimentally demonstrate in the
sequel that $\ell_1$ consistently provides better results for the denoising
task, \emph{when used for learning the dictionary only}. Indeed, we observe
that regardless of the way the dictionary is learned, $\ell_0$ should always be
used for the final reconstruction step in~(\ref{eq:reconstruct_patch}).

This conclusion is drawn from the following experiment. We consider~$12$
classical images presented in Figure~\ref{fig:dataset}, which were used in
other benchmarks about image denoising~\citep[\eg,][]{mairal2009}.  We corrupt
every image with white Gaussian noise of standard deviation~$\sigma$
in~$\{5,10,15,20,25,50,100\}$ for pixel values in the range~$[0;255]$, and we
measure the reconstruction quality obtained by the three approaches by using
the PSNR criterion.\footnote{Denoting by MSE the mean squared-error for images
whose intensities are between~$0$ and~$255$, the PSNR is defined as
$\text{PSNR}=10\log_{10}(255^2/\text{MSE})$ and is measured in dB. A gain
of~$1$ dB reduces the MSE by approximately~$20\%$.} We consider dictionaries of
size~$p=256$ elements, different patch sizes~$m= l \times l$, with $l$
in~$\{6,8,10,12,14,16\}$, and we use the parameter~$\tau=0.9$ suggested before for choosing
the reconstruction threshold~$\varepsilon$.  For every noise level, the
parameter~$l$ is selected such that it maximizes the average PSNR obtained on
the last~$5$ images of the dataset.

\begin{figure}[hbtp]
   \subfigure[\textsf{house}]{\includegraphics[width=0.24\linewidth]{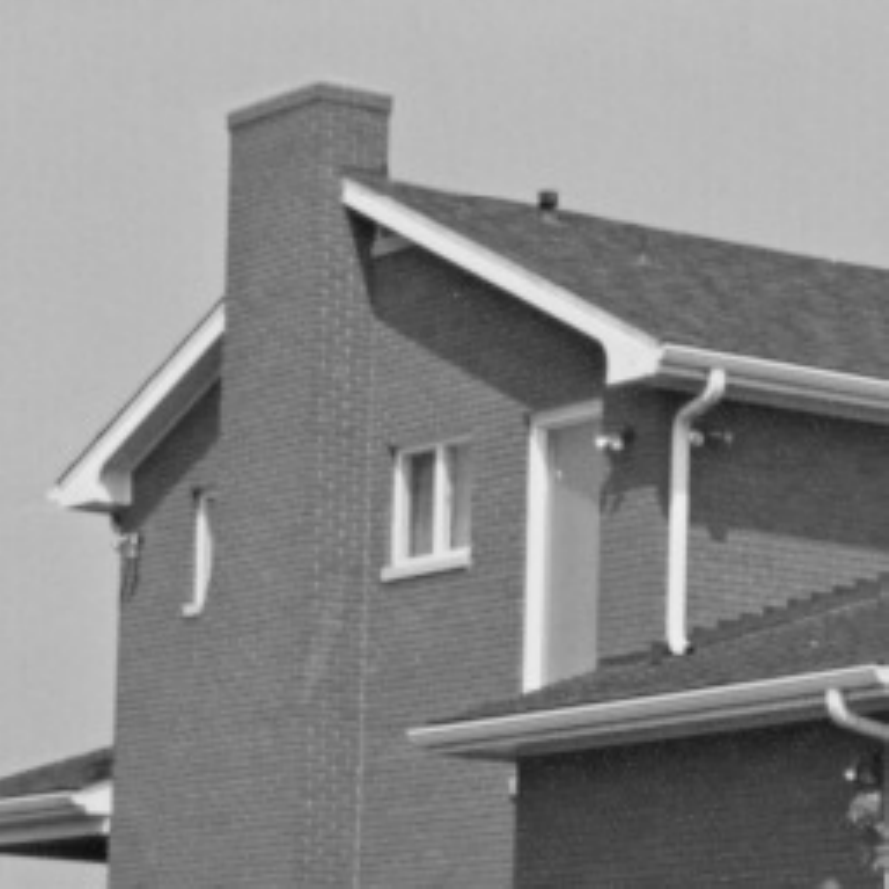}} \hfill
   \subfigure[\textsf{peppers}]{\includegraphics[width=0.24\linewidth]{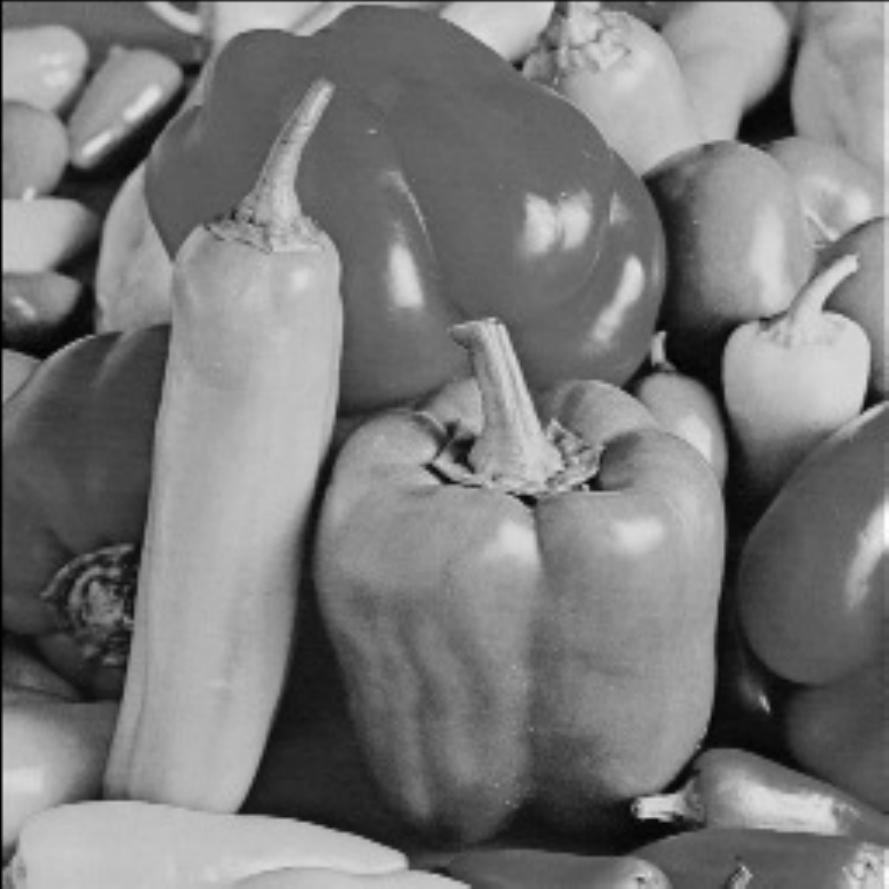}} \hfill
   \subfigure[\textsf{Cameraman}]{\includegraphics[width=0.24\linewidth]{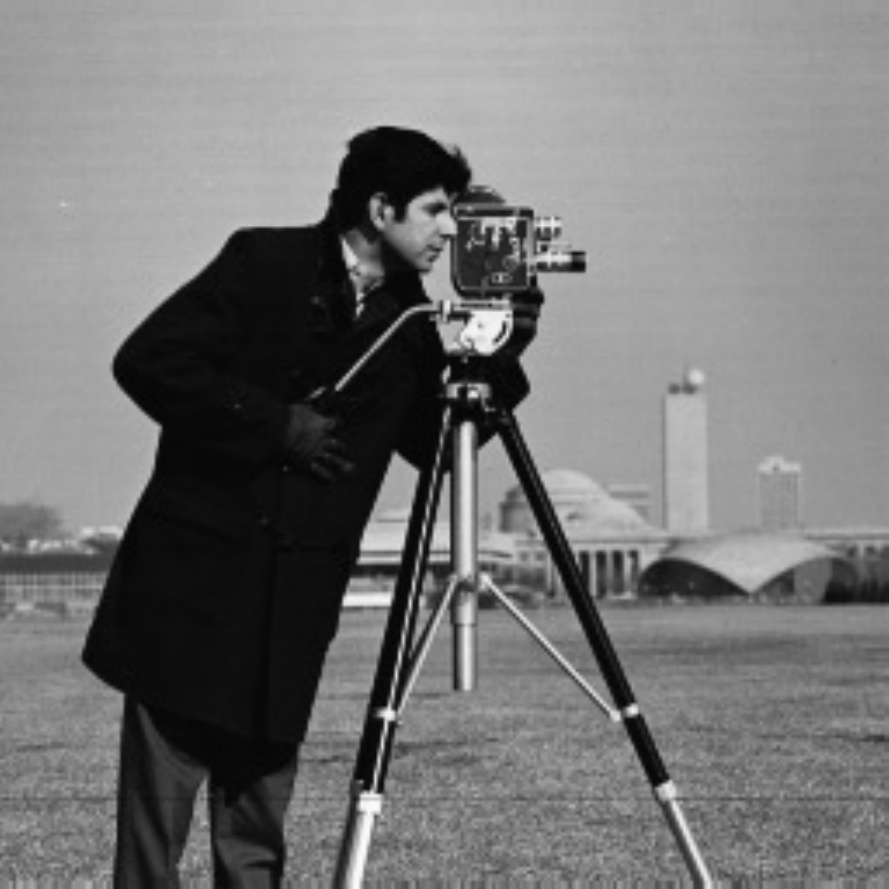}} \hfill
   \subfigure[\textsf{lena}]{\includegraphics[width=0.24\linewidth]{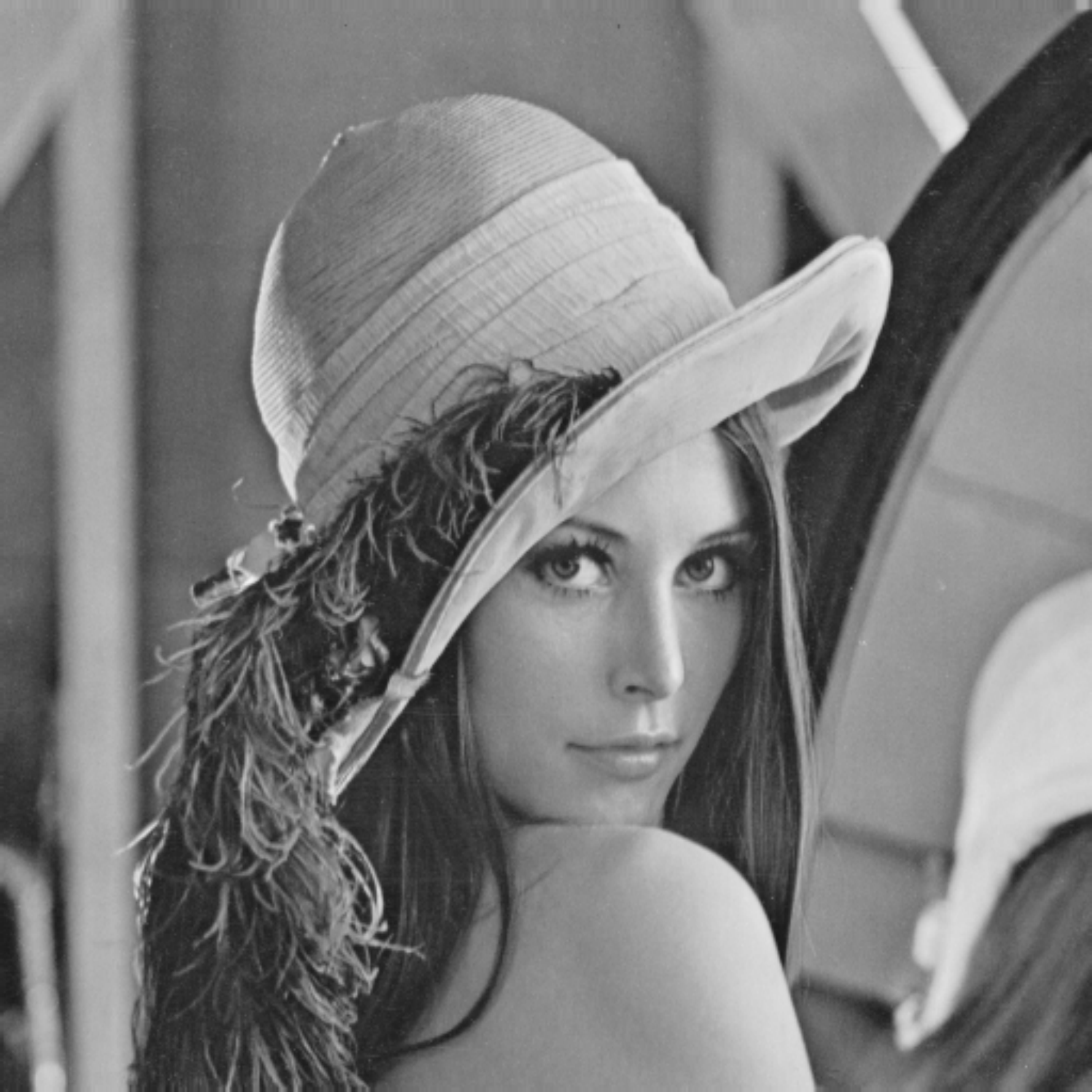}} \\
   \subfigure[\textsf{barbara}]{\includegraphics[width=0.24\linewidth]{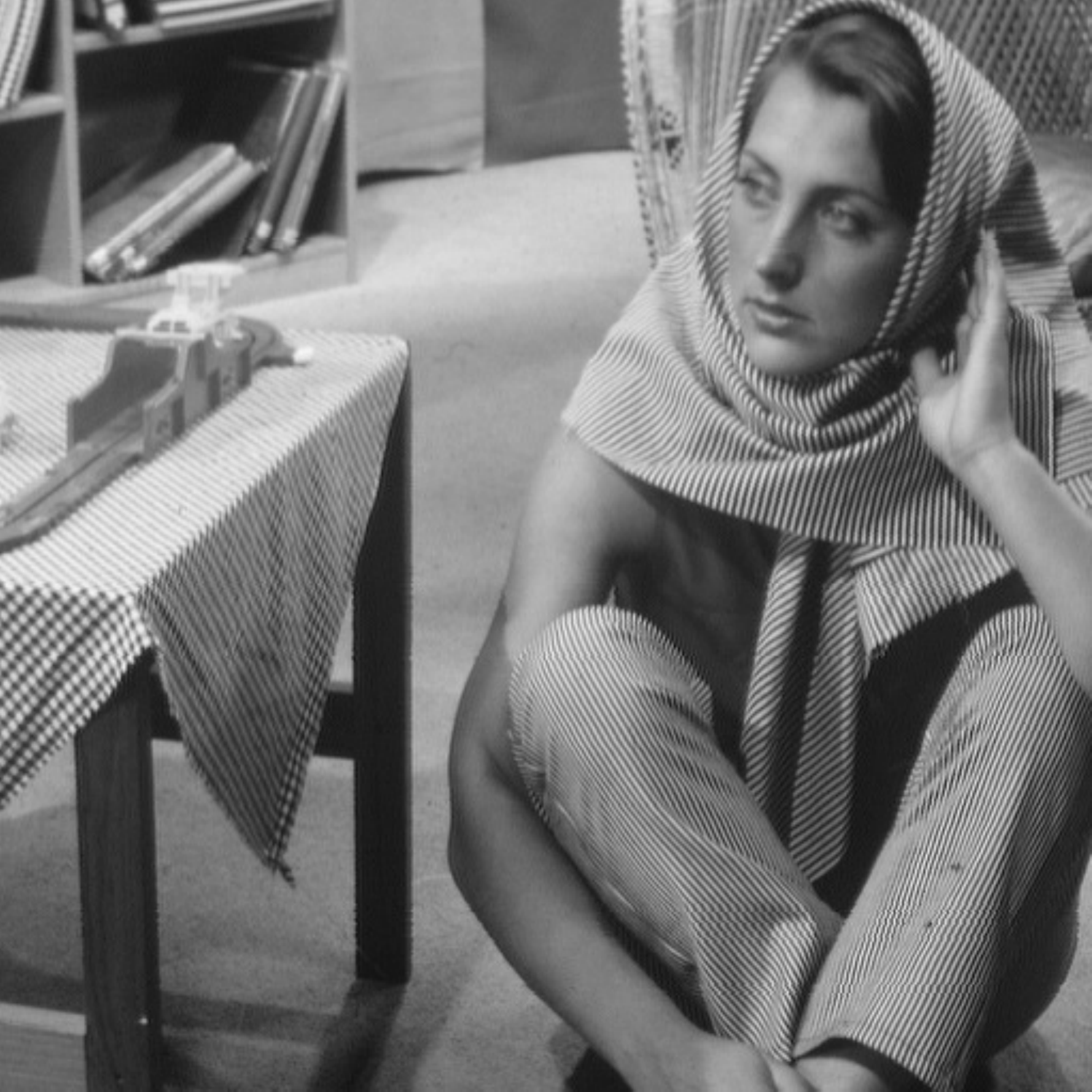}} \hfill
   \subfigure[\textsf{boat}]{\includegraphics[width=0.24\linewidth]{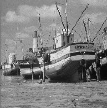}} \hfill
   \subfigure[\textsf{hill}]{\includegraphics[width=0.24\linewidth]{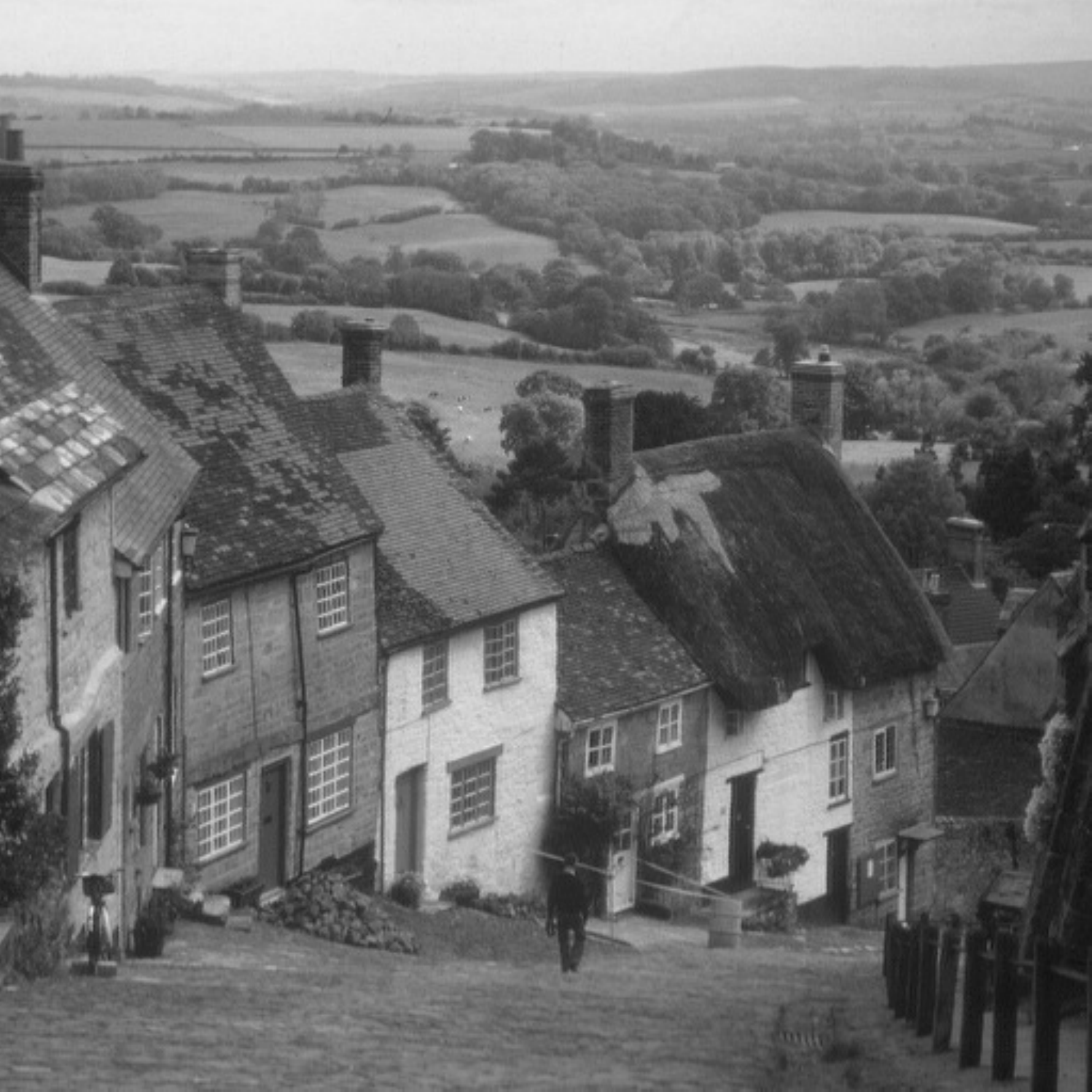}} \hfill
   \subfigure[\textsf{couple}]{\includegraphics[width=0.24\linewidth]{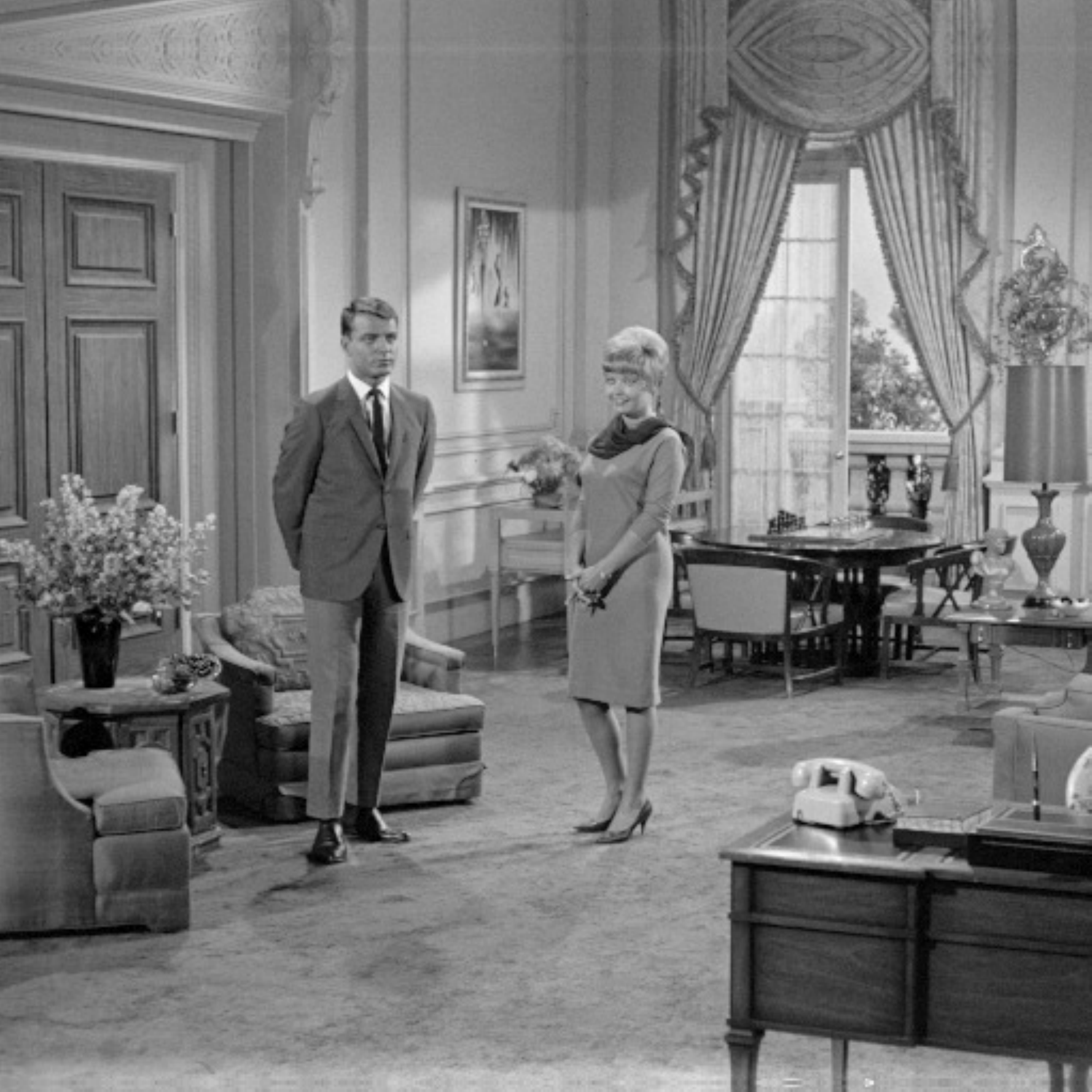}} \\
   \subfigure[\textsf{man}]{\includegraphics[width=0.24\linewidth]{images_arxiv/man.pdf}} \hfill
   \subfigure[\textsf{fingerprint}]{\includegraphics[width=0.24\linewidth]{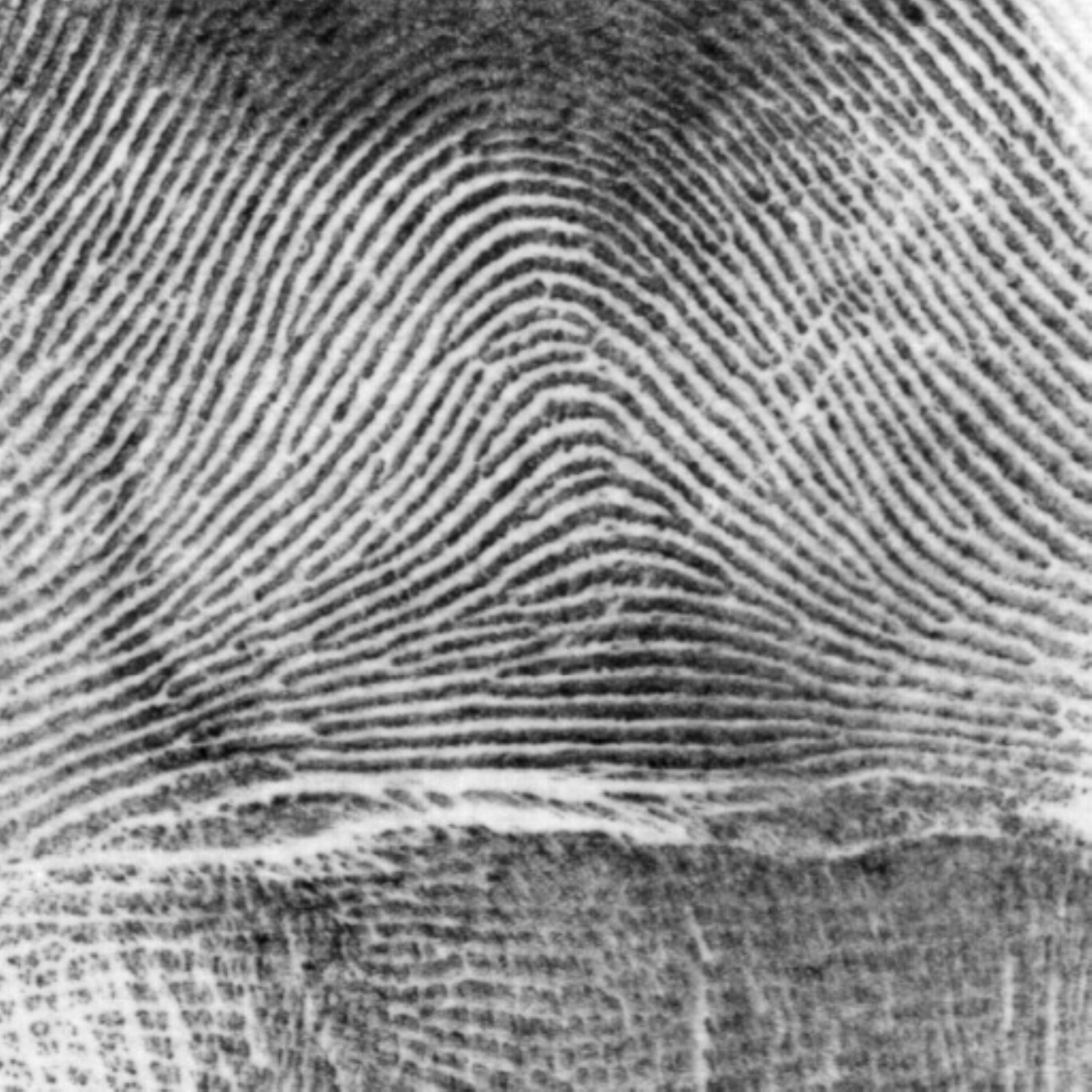}} \hfill
   \subfigure[\textsf{bridge}]{\includegraphics[width=0.24\linewidth]{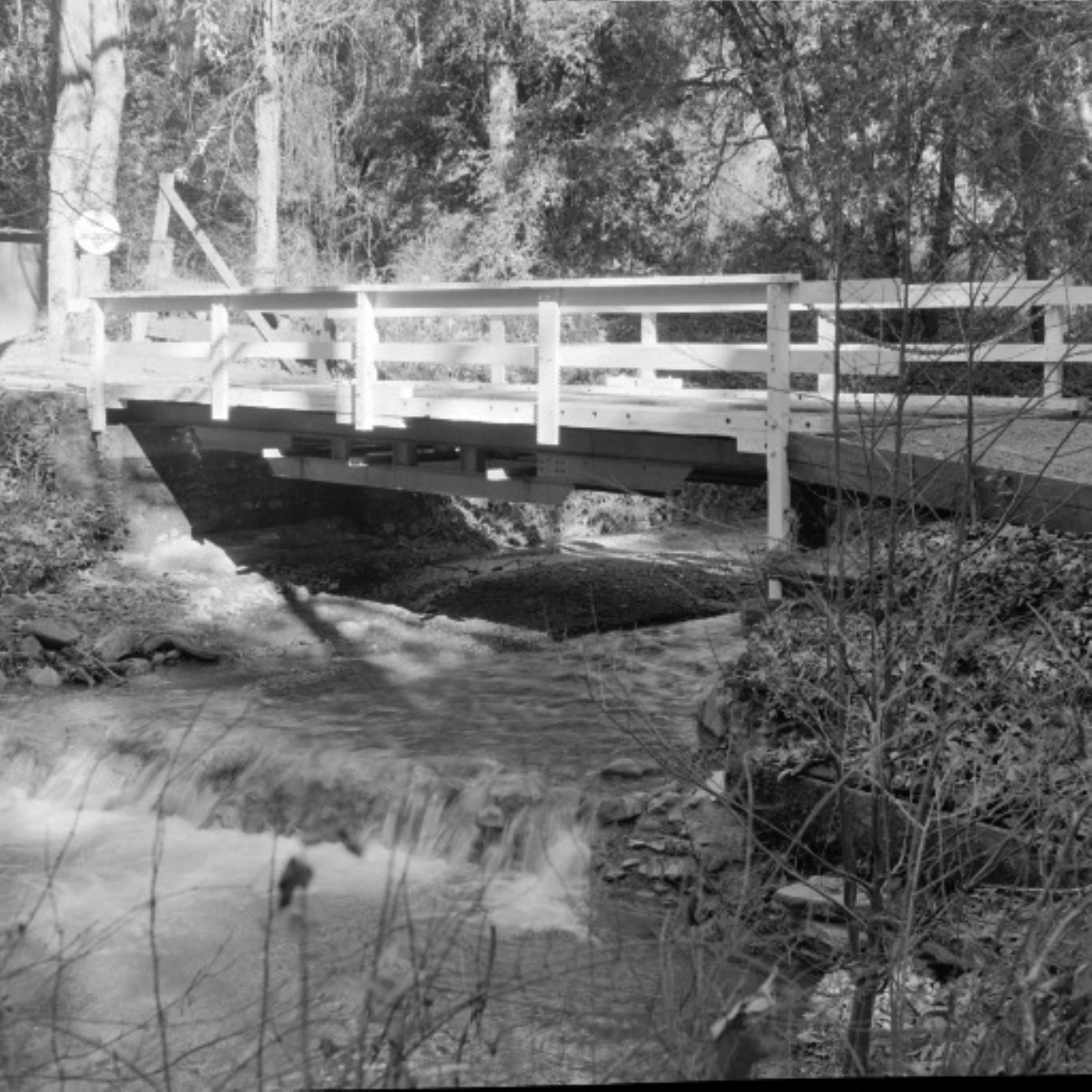}} \hfill
   \subfigure[\textsf{flintstones}]{\includegraphics[width=0.24\linewidth]{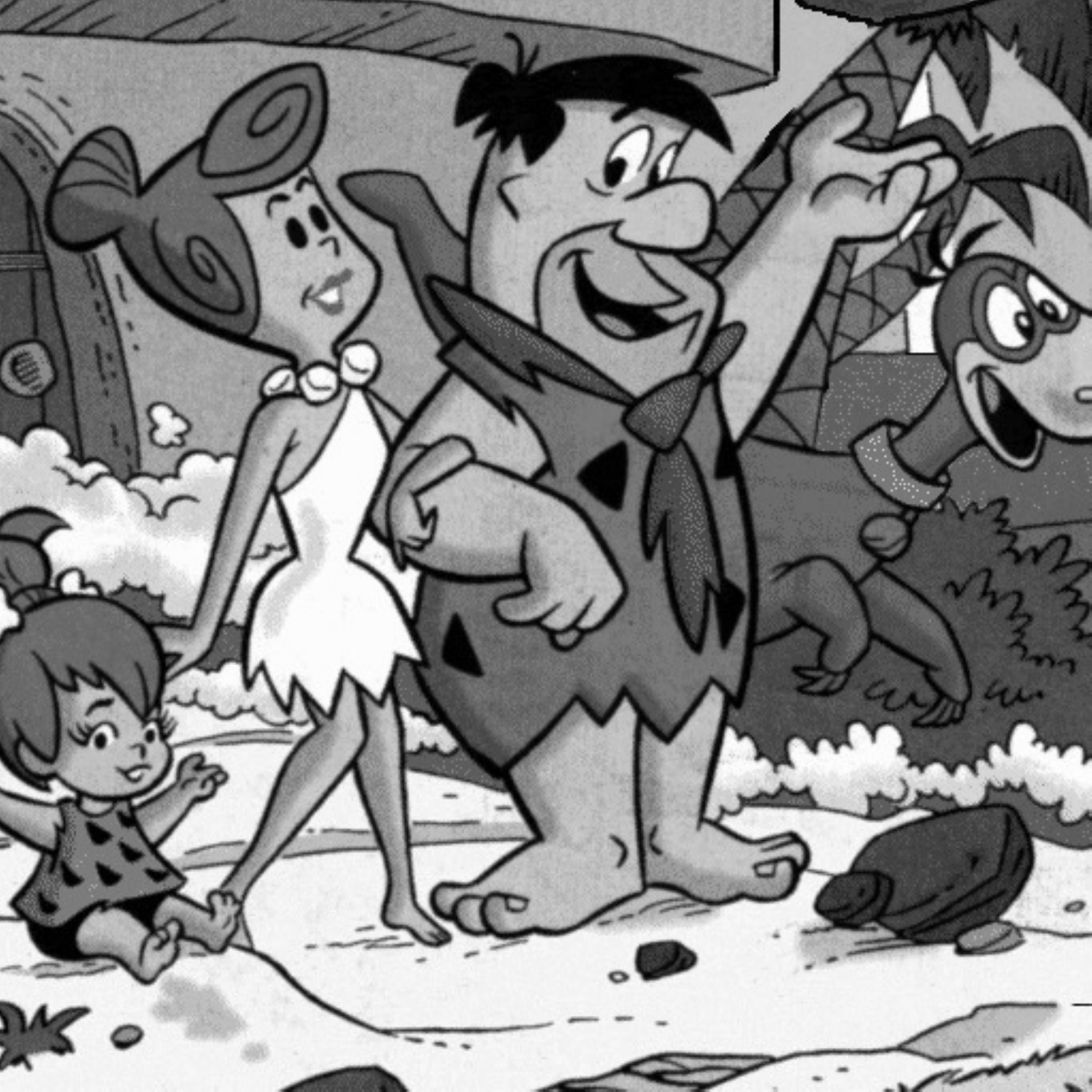}}
   \caption{Dataset of $12$ standard images used in the image denoising benchmarks.}
   \label{fig:dataset}
\end{figure}

We make a comparison between the six scenarios described below. In all cases,
the dictionaries are learned with the SPAMS software, by performing~$10$ passes
over the available training data.
\begin{enumerate}
   \item DCT: we use an overcomplete DCT dictionary; the patches are
      reconstructed by using the
      $\ell_0$-regularization~(\ref{eq:reconstruct_patch}).\label{item:dct}
   \item $\ell_0$-global/$\ell_0$: the dictionary is learned
      on~$400\,000$ natural image patches extracted from the Kodak PhotoCD
      images\footnote{The dataset is available here: \url{http://r0k.us/graphics/kodak/}.}.
      Denoting by~$\x_1,\ldots,\x_n$ these patches, we learn~$\D$ by minimizing
      \begin{displaymath}
         \min_{\D \in \CC, \A \in \Real^{p \times n}} \frac{1}{n} \sum_{i=1}^n \frac{1}{2}\|\x_i-\D\alphab_i\|_2^2 \st \|\alphab_i\|_0 \leq s,
      \end{displaymath}
      where~$s=10$ is known to approximate well clean natural image patches~\citep{elad2006}.\label{item:global}
   \item $\ell_0$-adapt/$\ell_0$: the dictionary is learned by
      minimizing~(\ref{eq:dict_noisy}) with~$\psi=\ell_0$; we initialize the
      learning procedure with the global dictionary obtained in the
      scenario~$\ell_0$-global/$\ell_0$.
   \item $\ell_1$-global/$\ell_0$: the scenario is the same as
      $\ell_0$-global/$\ell_0$, except that the dictionary is learned with
      the~$\ell_1$-penalty, following exactly the methodology of
      Section~\ref{subsec:dict}. The final reconstruction is obtained
      with~$\ell_0$ as in~(\ref{eq:reconstruct_patch}).\label{item:l1global}
   \item $\ell_1$-adapt/$\ell_0$: same as~$\ell_1$-adapt/$\ell_0$, except that the dictionary is
      learned by minimizing~(\ref{eq:dict_noisy}) with~$\psi=\ell_1$; we use
      the dictionary obtained in the scenario $\ell_1$-global for initializing the
      learning procedure. The final reconstruction is still obtained with the~$\ell_0$-penalty.\label{item:l1adaptive}
   \item $\ell_1$-adapt/$\ell_1$: same as~$\ell_1$-adapt/$\ell_0$, but the
      patches~$\y_i^c$ in the final reconstruction are denoised by replacing
      the~$\ell_0$-penalty by the $\ell_1$-norm
      in~(\ref{eq:reconstruct_patch}).
\end{enumerate}
We report the mean PSNR obtained by the different scenarios in
Table~\ref{table:denoising}, along with the performance achieved by other
state-of-the-art approaches of the literature. The conclusions are the
following:
\begin{itemize}
   \item the simple denoising scheme that we have presented performs slightly worse than more recent methods such
      as~\citet{dabov2,mairal2009,chatterjee2012} in terms of PSNR;
   \item adaptive dictionaries yield better results than global ones, as already
      observed by~\citet{elad2006};
   \item learning the dictionary with~$\ell_1$ consistently yields better
      results than~$\ell_0$, even though the $\ell_0$-penalty is used for the
      final image reconstruction step in~(\ref{eq:reconstruct_patch});
   \item $\ell_0$ yields better results than~$\ell_1$ in the final image
      reconstruction step, regardless of the way the dictionary was learned.
\end{itemize}

Even though the conclusion that one should ``first learn the dictionary
with~$\ell_1$, before using $\ell_0$ for denoising the image'' might seem
counterintuitive, we believe that the reason of the good performance of
$\ell_1$ for dictionary learning might be a better stability of the sparsity
patterns than the ones obtained with~$\ell_0$. This may subsequently yield a
better behavior in terms of optimization. Whereas the $\ell_1$-scheme
guarantees us to obtain a stationary point of the dictionary learning
formulation, the $\ell_0$-counterpart does not.  A similar experimented was
conducted by~\citet[][chapter 1.6]{mairal_thesis}, with similar conclusions.

\begin{table}
   \begin{tabular}{|c|c|c|c|c|c|c|c|}
      \hline
      & \multicolumn{7}{c|}{$\sigma$} \\
      \hline
       & 5 & 10& 15 & 20 & 25 & 50 & 100 \\
      \hline
      \hline
      DCT & 37.30 & 33.38 & 31.24 & 29.75 & 28.66 & 25.24 & 22.00 \\
      \hline
      $\ell_0$-global/$\ell_0$ & 37.21 & 33.37 & 31.41  & 30.01 & 28.95 & 25.65 & 22.44 \\
      \hline
      $\ell_1$-global/$\ell_0$ & 37.22 & 33.50 & 31.56 & 30.17 & 29.13 & 25.77  & 22.53 \\
      \hline
      $\ell_0$-adapt/$\ell_0$ & 37.49 & 33.75 & 31.70 & 30.40 & 29.33 & 26.04 & 22.64 \\
      \hline
      $\ell_1$-adapt/$\ell_0$ & 37.60 & 33.90 & 31.90  & 30.51 & 29.43 & 26.20 & 22.72 \\
      \hline
      $\ell_1$-adapt/$\ell_1$ & 36.58 & 32.85 & 30.77 & 29.29 & 28.14 & 24.47 & 21.83 \\
      \hline
      \hline
      GSM & 37.05  & 33.34 & 31.31 & 29.91 & 28.84 & 25.66 & 22.80 \\
      \hline
      K-SVD & 37.42 & 33.62 & 31.58 & 30.18 & 29.10 & 25.61 & 22.10 \\
      \hline
      BM3D & 37.62 & 34.00 & 32.05 & 30.73 & 29.72 & 26.38 & 23.25 \\
      \hline
      LSSC & 37.67 & 34.06 & 32.12 & 30.78 & 29.74 & 26.57 & 23.39 \\
      \hline
      Plow & 37.38 & 32.98 & 31.38 & 30.13 & 29.30 & 26.38 & 23.24 \\
      \hline
      EPLL & 37.36 & 33.64 & 31.67 & 30.32 & 29.29 & 26.12 & 23.03 \\
      \hline
      CSR & 37.61 & 34.00 & 32.05 & 30.72 & 29.70  & 26.53  & 23.45 \\
      \hline
      BM3D-PCA & 37.79 & 34.20 & 32.27 & 30.94 &  29.92 & 26.75 & 23.16 \\
      \hline
   \end{tabular}
   \caption{Denoising performance in PSNR for various methods on the $12$
      standard images of Figure~\ref{fig:dataset} for various levels of
      noise~$\sigma$. GSM refers to the Gaussian scale mixture model
      of~\citet{portilla2003}; K-SVD refers to the original method
      of~\citet{elad2006}; BM3D refers to the state-of-the-art denoising
      approach of~\citet{dabov2}; LSSC refers to~\citet{mairal2009}, Plow
      refers to~\citet{chatterjee2012}, EPLL to~\citet{zoran2011}, CSR
      to~\citet{dong2013}, and BM3D-SAPCA to an extension of BM3D with better
      results~\citep{dabov2009,katkovnik2010}. For all these approaches,
      publicly available software is available on the corresponding authors'
   web pages.}\label{table:denoising}
\end{table}

%% file: content_arxiv/image_inpainting.tex
Dictionary learning is well adapted to the presence of missing data---that is,
it is appropriate for \emph{inpainting}~\citep{bertalmio2000} when the missing
pixels form small holes that are smaller than the patch sizes.  To deal with
unobserved information, the dictionary learning formulation can be modified by
introducing
a binary mask~$\M_i$ for every patch indexed by~$i$~\citep{mairal2008,mairal2008b}.
Formally, we define $\M_i$ as a diagonal matrix in $\Real^{m \times m}$ whose
value on the $j$-th entry of the diagonal is $1$ if the pixel~$\y_i[j]$ is
observed and $0$ otherwise. Then, the dictionary learning formulation becomes
\begin{displaymath}
   \min_{\D \in \CC, \A \in \Real^{p \times n}} \frac{1}{n}\sum_{i=1}^n\frac{1}{2}\|\M_i(\y_i-\D\alphab_i)\|_2^2 + \lambda \psi(\alphab_i),
\end{displaymath}
where~$\psi$ is a sparsity-inducing penalty, $\M_i\y_i$ represents the
observed pixels from the $i$-th patch of the image~$\y$ and~$\D\alphab_i$ is the
estimate of the full patch~$i$.  In practice, the binary mask
does not drastically change the optimization procedure, and one can still use
classical optimization techniques for dictionary learning,
\eg, alternating between the optimization of $\D$ and $\A$, as described in
Section~\ref{chapter:optim}. When the image $\y$ is only corrupted by missing
pixels and not by other noise source, it is also possible to enforce hard
reconstruction constraints, leading to the formulation
\begin{equation}   
   \min_{\D \in \CC, \A \in \Real^{p \times n}} \sum_{i=1}^n \psi(\alphab_i) \st \M_i\y_i = \M_i\D\alphab_i. \label{eq:inpaint}
\end{equation}
Similarly to the denoising problem, once the dictionary is learned, we need to
sparsely encode all overlapping patches, \eg, by approximately
minimizing~(\ref{eq:inpaint}) with a greedy algorithm when~$\psi$ is
the~$\ell_0$-penalty. Then, the full image is obtained by averaging
using~(\ref{eq:patch_averaging}).

Before showing any inpainting result, we shall comment on \emph{when the
formulation described here is supposed to work}. Our first remark is that it can only
handle holes that are smaller than the patch size. Dealing with larger holes
might be possible, but with a different formulation that would be able to
synthesize new information in empty
patches~\citep[see, \eg][]{criminisi2004,peyre2009,roth2009}.  Second, one assumes that
\emph{the noise pattern is unstructured}, such that no noise patterns are
learned by the dictionary.  The demosaicking task from the next section is a
typical example with very structured missing patterns, and where an alternative
strategy needs to be used. 

Finally, we show inpainting results in Figure~\ref{fig:inpaint}. For both images,
the noise structure is relatively random when seen at the patch level, and the
dictionary learning approach performs well at recovering the missing pixels.
In particular, the brick texture in the image \textsf{house} is hardly visible
for the human eye, but it is well recovered by the algorithm.
\begin{figure}[hbtp]
   \centering
   \subfigure[Example A, Damaged]{\includegraphics[width=0.49\linewidth]{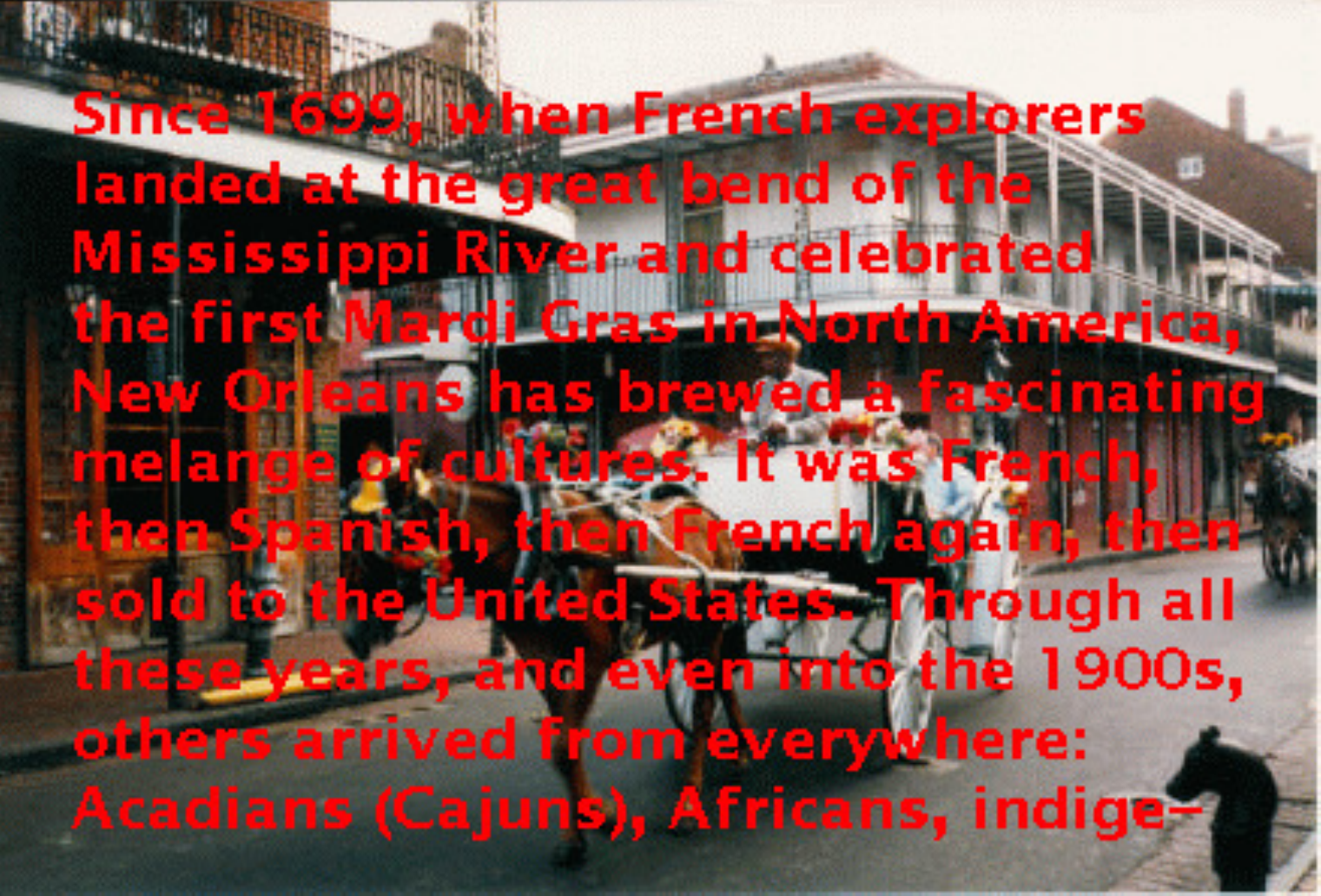}} \hfill
   \subfigure[Example A, Restored]{\includegraphics[width=0.49\linewidth]{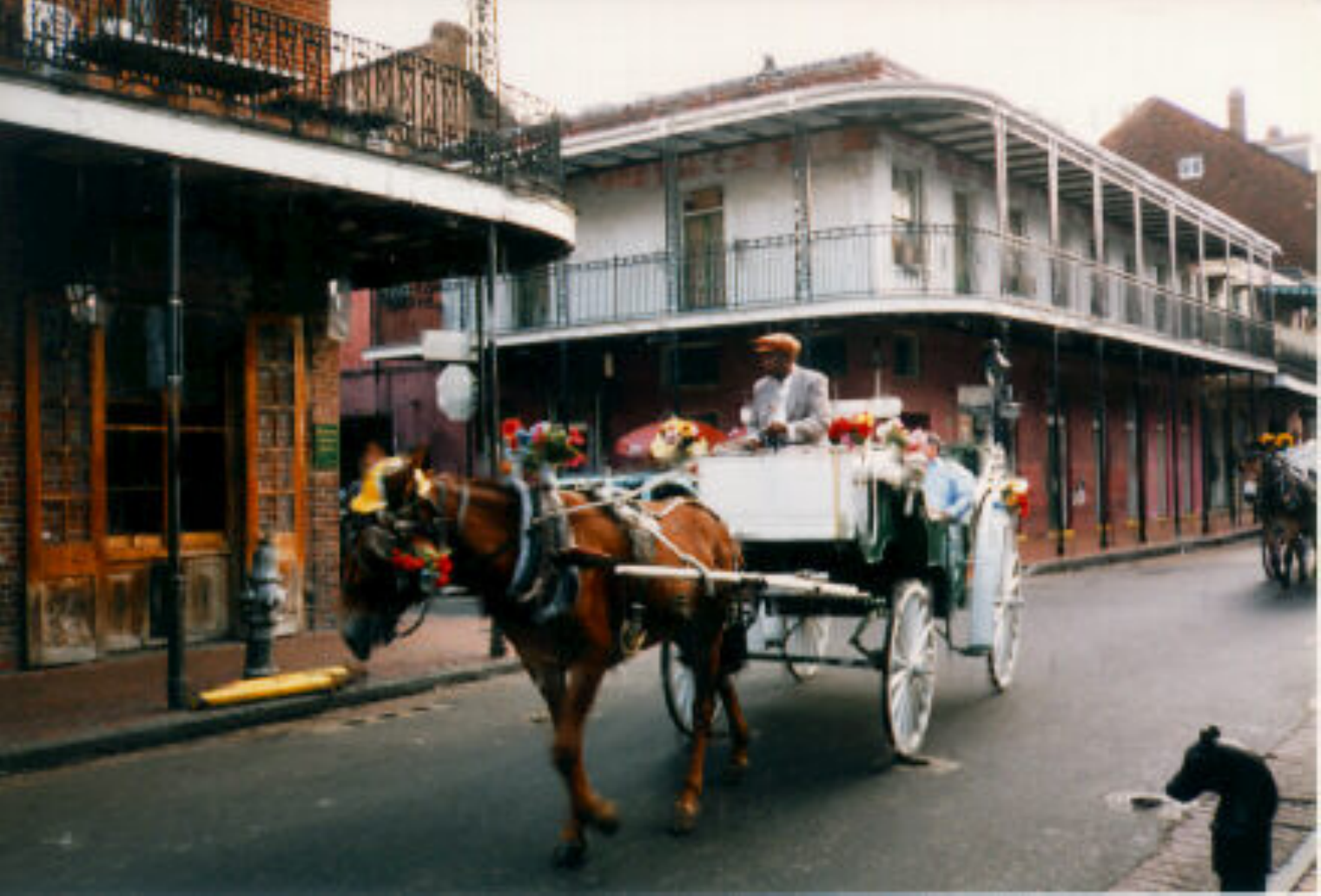}} \\
   \subfigure[Example B, Damaged]{\includegraphics[width=0.49\linewidth]{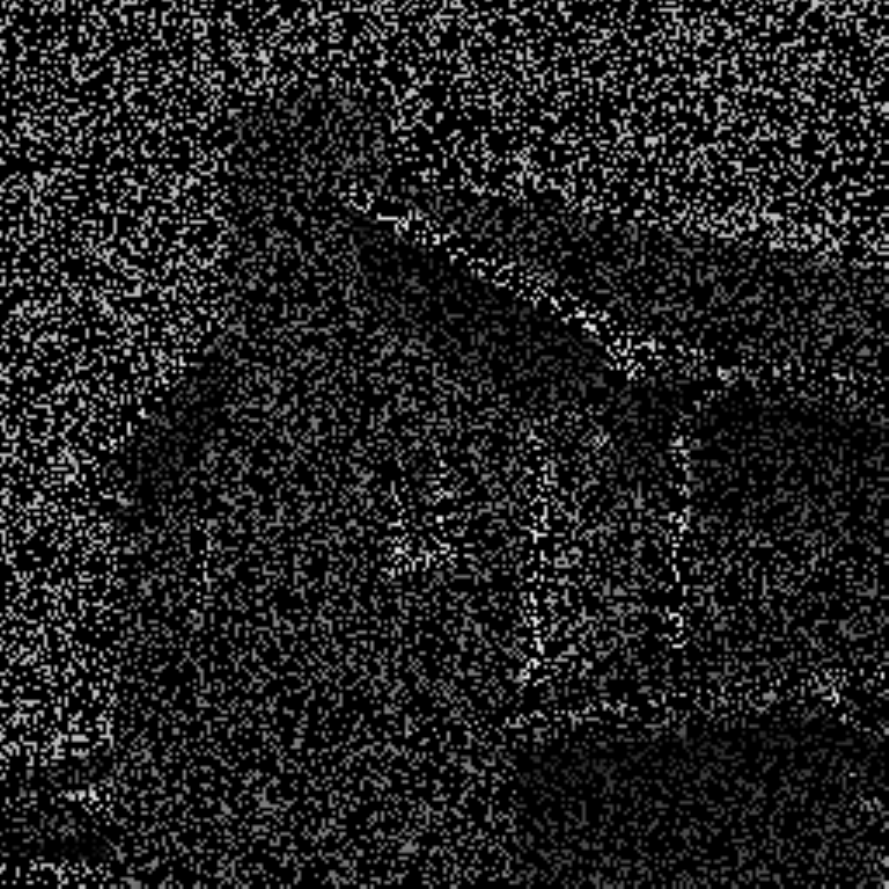}} \hfill
   \subfigure[Example B, Restored]{\includegraphics[width=0.49\linewidth]{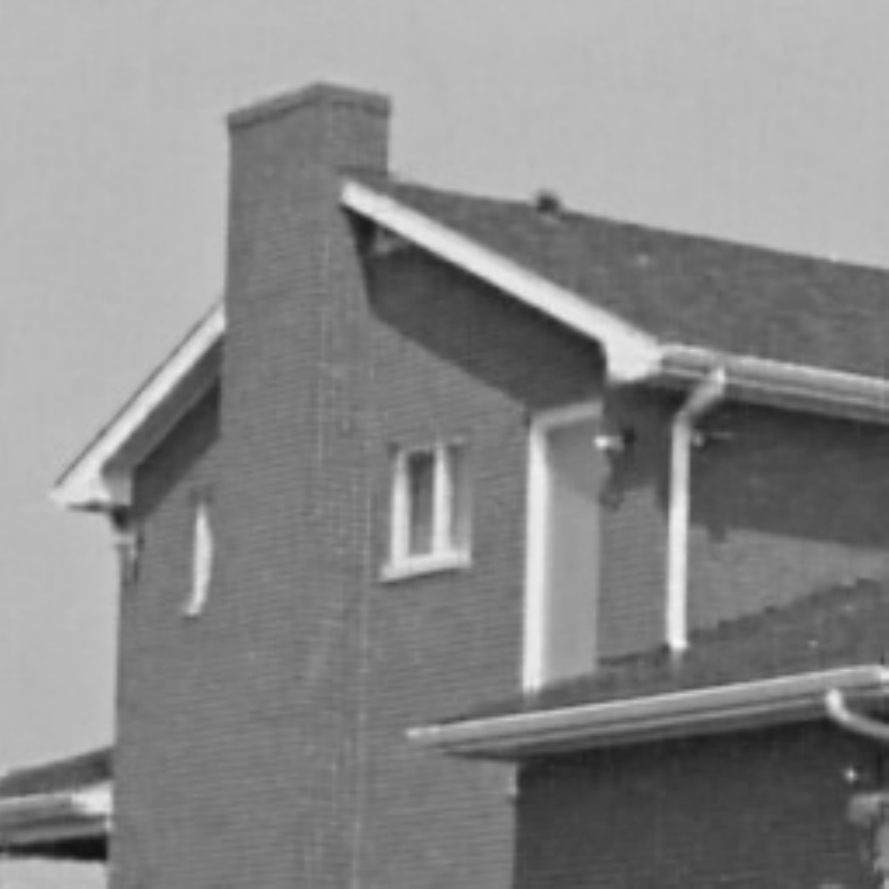}}
   \caption{Top: inpainting result obtained by using the source code
      of~\citet{mairal2008}, where the text is automatically removed on the
      restored image. Bottom: inpainting result from~\citet{mairal2008b}, where
      $80\%$ of the pixels are randomly missing from the original image. The
      algorithm is able to reconstruct the brick texture on the right, by adapting the dictionary
      to the available information displayed on the left. The bottom images are under ``copyright \copyright 2008 Society for Industrial and Applied Mathematics.  Reprinted with permission.  All rights reserved''. Best seen by zooming on a computer screen.}
\label{fig:inpaint}
\end{figure}

%% file: content_arxiv/image_demosaicking.tex
The problem of demosaicking consists of reconstructing a color image from the
raw information provided by colored-filtered sensors from digital
cameras~\citep{kimmel1999}. Most consumer-grade cameras use a grid of sensors;
for every pixel, one sensor records the amount of light it receives for a short
period of time. A color filter, red, green, or blue, is placed in front of
the sensor such that only one color channel is measured for each pixel.  Reconstructing
the missing channels for the full image is a task called \emph{demosaicking}. It is 
usually performed on-the-fly by digital cameras, but can also be achieved with
a better quality by professional software on a desktop computer.

A typical pattern of filters, called the ``Bayer'' pattern, is displayed
in Figure~\ref{fig:bayer}. Every row alternates between red and green, or green
and blue filters. As a result, information captured by digital sensors in
the green channel is twice more important than in the red or blue channels,
even though the final post-processed color image observed by the user does not
reflect this fact. An example of a mosaicked image with the Bayer pattern is
presented in Figure~\ref{fig:bayera}. Reconstructing the color image without
producing any visual artifact is difficult since it requires recovering
frequencies in ever channel that are greater than the sampling rate.
Specifically, the resolution in the red and blue channels needs to be doubled.
\begin{figure}[hbtp]
   \centering
   \includegraphics[width=0.4\linewidth]{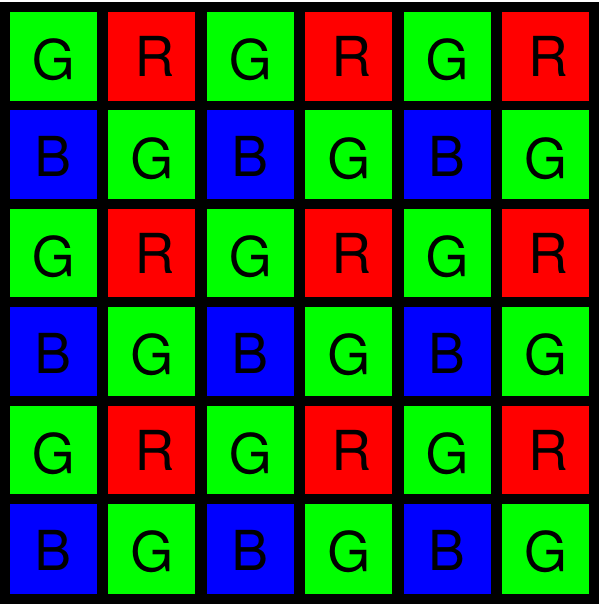}
   \caption{Common Bayer pattern used for digital camera sensors.} \label{fig:bayer}
\end{figure}

Since demosaicking consists of recovering missing information for every pixel,
it can be cast as an inpainting problem, but the inpainting procedure described
in the previous section cannot be directly applied because the noise pattern is
strongly structured. An effective two-step strategy proposed
by~\citet{mairal2008} to overcome the shortcoming of the original inpainting
formulation consists of the following steps:
\begin{enumerate}
   \item learn a dictionary~$\D_0$ on an external database of color
      patches; 
   \item use~$\D_0$ to obtain a first estimate of the color image, following
      the image reconstruction procedure of the previous
      section;\label{item:recon0}
   \item retrain the dictionary on the first image estimate, leading to a new dictionary~$\D_1$ adapted to the image content.
   \item use~$\D_1$ to obtain the final estimate of the color image as
      in~\ref{item:recon0}.
\end{enumerate}
In addition, the patches should also be centered according to
the recommendations of Section~\ref{subsec:preprocess}. The results achieved by
this method can be were shown by~\citet{mairal2008} to outperform the state of the art
in terms of PSNR on the~$24$ images from the Kodak PhotoCD benchmark, even
though some visual artifacts remain visible in areas that are particularly
difficult to reconstruct, as shown in Figure~\ref{fig:bayerb}.

The demosaicking approach based on dictionary learning has been improved later
by~\citet{mairal2009} by combining sparse estimation with image
self-similarities~\citep{buades2009}. The resulting method yields a higher
PSNR, but more importantly it seems to significantly reduce visual artifacts
(see Figure~\ref{fig:bayerc}), and achieves state-of-the-art results in terms
of PSNR~\citep{menon2011}.
\begin{figure}[hbtp]
   \centering
   \subfigure[Mosaicked image]{\label{fig:bayera}
   \includegraphics[width=0.31\linewidth]{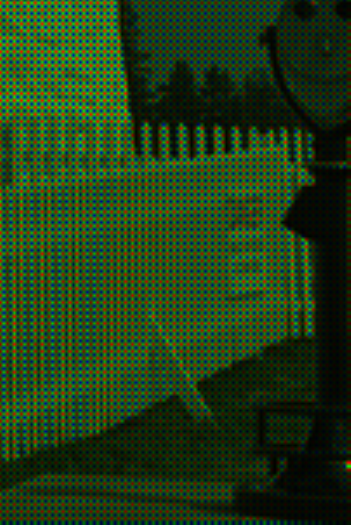}}
   \subfigure[Demosaicked image A]{\label{fig:bayerb}
   \includegraphics[width=0.31\linewidth]{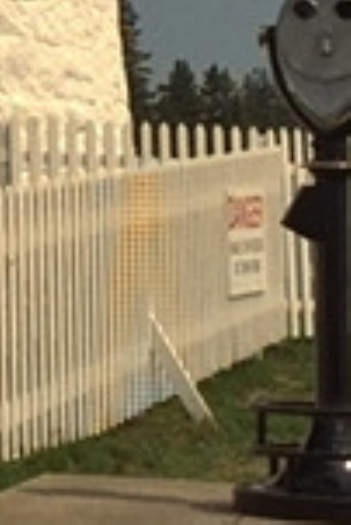}} 
   \subfigure[Demosaicked image B]{\label{fig:bayerc}
   \includegraphics[width=0.31\linewidth]{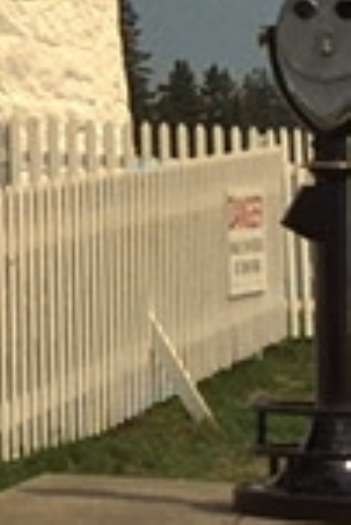}}\\
   \subfigure[Zoom on~\subref{fig:bayera}]{\label{fig:bayeraa}
   \includegraphics[width=0.31\linewidth]{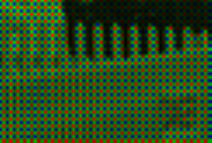}}
   \subfigure[Zoom on~\subref{fig:bayerb}]{\label{fig:bayerbb}
   \includegraphics[width=0.31\linewidth]{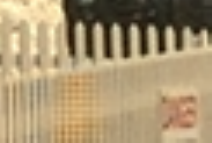}} 
   \subfigure[Zoom on~\subref{fig:bayerc}]{\label{fig:bayercc}
   \includegraphics[width=0.31\linewidth]{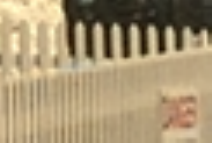}}\
   \caption{\subref{fig:bayera}: Mosaicked image; \subref{fig:bayerb}: image
   demosaicked with the approach of~\citet{mairal2008} presented in
Section~\ref{subsec:demosaicking}. \subref{fig:bayerc}: image demosaicked with
the non-local sparse image models of~\citet{mairal2009}. Note that the part of
the image displayed here in known to be particularly difficult to reconstruct.
Other reconstructions obtained from the database of~$24$ images is
artifact-free in general. Best seen in color by zooming on a computer screen.}
\label{fig:demos}
\end{figure}

%% file: content_arxiv/image_superres.tex
Dictionary learning has also gained important success for reconstructing
high-resolution images from low-resolution ones, a problem called \emph{image
up-scaling} or \emph{digital zooming}.\footnote{The terminology of
super-resolution is sometimes used but leads to some confusion.
``Super-resolution'' traditionally means synthesizing high-resolution images
from a sequence of low-resolution ones~\citep[see the discussion in][page
341]{elad2010}. Here, we are dealing with a single input image.} The technique
originally proposed by~\citet{yang2010} consists of learning a pair of
dictionaries~$(\D_l,\D_h)$ that can respectively reconstruct low- and
high-resolution patches with the same sparse decomposition coefficients. Then, 
the relation between~$\D_l$ and~$\D_h$ is exploited for processing new
low-resolution images and turn them into high-resolution ones.

Similar to the ``example-based'' approach of~\citet{freeman2002}, the method
requires a database of pairs of training patches~$(\x_i^l,\x_i^h)_{i=1}^n$,
where~$\x_i^l$ in~$\Real^{m_l}$ is a low-resolution version of the
patch~$\x_i^h$ in~$\Real^{m_h}$, and~$m_h > m_l$. The database is built from
training images that are generic enough to represent well the variety of
natural images. Then, the pairs of patches are used offline to train the
dictionaries~$\D_l$ in~$\Real^{m_l \times p}$ and~$\D_h$ in~$\Real^{m_h \times
p}$, which are subsequently used for image up-scaling in a way that will
be described in the sequel. The methodology is related to the \emph{global}
approach for image denoising from Section~\ref{subsec:denoising}, where the
dictionary is not adapted to the image at hand that needs to be processed, but
is adapted instead to a generic database of images.

The image upscaling method of~\citet{yang2010} has received a significant amount of
attention because of the quality of its results, and has been improved recently in
various ways \citep{couzinie2011,zeyde2012,dong2011,wang2012,yang2012}. In this
section, we briefly review some of the ideas that have appeared in this line of
work, but refer to the corresponding papers for more details.  We start with the
original approach of~\citet{yang2010}, before moving to two natural variants.

\paragraph{The original approach of~\citet{yang2010}.}
The dictionary learning formulation of~\citet{yang2010} for image up-scaling consists of
learning~$\D_l$ and~$\D_h$ with the following joint minimization problem:
\begin{equation}
   \min_{\substack{\D_l \in \CC_l \\ \D_h \in \Real^{m_h \times p} \\ \A \in \Real^{p \times n} }} \frac{1}{n} \sum_{i=1}^n \frac{1}{2 m_l} \left\|\x_i^l - \D_l \alphab_i\right\|_2^2 + \frac{1}{2 m_h} \left\|\x_i^h - \D_h \alphab_i\right\|_2^2  + \lambda\left\|\alphab_i\right\|_1, \label{eq:upscaling}
\end{equation}
where~$\A=[\alphab_1,\ldots,\alphab_n]$ is the matrix of sparse coefficients, and $\CC_l$ is the set
of matrices in~$\Real^{m_l \times p}$ 
whose columns have less than unit~$\ell_2$-norm. The formulation encourages the
dictionaries to provide a common sparse representation~$\alphab_i$ for the pair
of patches~$\x_i^l$ and~$\x_i^h$. Note that, to simplify, we have omitted
pre-processing steps for the patches~$\x_i^l$, which lead to substantially
improved results in practice~\citep[see][]{yang2010,zeyde2012}, and we have
preferred to focus on the main principles of the method.

Once the dictionaries are trained, they can be used for turning 
new low-resolution images into high-resolution ones. 
Specifically, let us consider an image~$\y^l$ of size~$\sqrt{n_l} \times
\sqrt{n_l}$, which is independent from the training database, and assume that
the goal is to synthesize a high-resolution image~$\y^h$ of larger
size~$\sqrt{n_h} \times \sqrt{n_h}$---say, for instance,
$\sqrt{n_h}=2\sqrt{n_l}$ when one wishes to double the resolution. 

According to~(\ref{eq:upscaling}), the main underlying assumption is that
high-resolution and low-resolution variants of the same patch admit a common
sparse decomposition onto~$\D_h$ and~$\D_l$, respectively.
Then, given a patch~$\y_i^l$ in~$\Real^{m_l}$ centered at pixel~$i$ in
the low-resolution image~$\y^l$, the method of~\citet{yang2010} computes a
sparse vector~$\betab_i$ in~$\Real^p$ such that~$\y_i^l \approx \D_l
\betab_i$, and the high-resolution estimate of the patch is simply~$\D_h
\betab_i$. Such an approach is however not fully
satisfactory regarding the formulation~(\ref{eq:upscaling}), where each sparse
decomposition is obtained by using both the low- and high-resolution patches.
In other words, $\betab_i$ should be ideally computed by minimizing
\begin{equation}
   \min_{\betab_i \in \Real^{p} } \frac{1}{2 m_l} \left\|\y_i^l - \D_l \betab_i\right\|_2^2 + \frac{1}{2 m_h} \left\|\y_i^h - \D_h \betab_i\right\|_2^2  + \lambda\left\|\betab_i\right\|_1, \label{eq:upscaling2}
\end{equation}
but unfortunately only~$\y_i^l$ is available at test time since~$\y_i^h$ is
the unknown quantity to estimate. Such a discrepancy is addressed by~\citet{yang2010} with a heuristic,
where~$\y_i^h$ in~(\ref{eq:upscaling2}) is replaced by an auxiliary variable
representing the current prediction of the high-resolution patches given
previously processed neighbors. Finally, reconstructing the full image from the
local patch estimates can be done as usual by straight averaging, even though
other possibilities yielding good results have also been
proposed~\citep[see][]{zeyde2012}.

\paragraph{First variation with a regression formulation.}
A strategy for overcoming the discrepancy between training and
testing in the previous approach has been independently proposed by several
authors~\citep{zeyde2012,couzinie2011,yang2012}. In all these approaches,
the sparse coefficients~$\alphab_i$ and~$\betab_i$ are always
computed by using \emph{low-resolution patches only}. 
This motivates the following two-step training approach:
\begin{enumerate}
   \item compute~$\D_l$ and~$\A$ with a classical dictionary learning formulation, for instance:
\begin{displaymath}
   \min_{\D_l \in \CC_l, \A \in \Real^{p \times n}}  \frac{1}{n} \sum_{i=1}^n \frac{1}{2}\left\|\x_i^l - \D_l \alphab_i\right\|_2^2 + \lambda \|\alphab_i\|_1.
\end{displaymath}
   \item obtain $\D_h$ by solving a multivariate \emph{regression} problem:
      \begin{displaymath}
         \min_{\D_h \in \Real^{m_h \times p}}  \frac{1}{n} \sum_{i=1}^n \frac{1}{2}\left\|\x_i^h - \D_h \alphab_i\right\|_2^2,
      \end{displaymath}
      where the~$\alphab_i$'s are fixed after the first step. See
      also~\citet{zeyde2012} for other variants.
\end{enumerate}
Once the dictionaries are learned, a new low-resolution image~$\y_l$ is
processed by decomposing its patches~$\y_i^l$ onto~$\D_l$, yielding sparse
coefficients~$\betab_i$, and the estimate of the corresponding high resolution
patches are again the quantities~$\D_h\betab_i$. 

\paragraph{Second variation with bilevel optimization.}
The previous variant addresses one shortcoming of the original approach
of~\citet{yang2010}, but unfortunately loses the ability of jointly learning~$\D_l$
and~$\D_h$.  The variant presented in this paragraph is motivated by that issue.
It was independently proposed by~\citet{couzinie2011} and~\citet{yang2012} and
allows learning~$\D_l$ and~$\D_h$ at the same time. As we shall see, it raises
a challenging bilevel optimization problem.

Let us first introduce a notation describing the solution of the Lasso for any vector~$\x$ and dictionary~$\D$:
\begin{displaymath}
\alphab^\star(\x,\D) \defin \argmin_{\alphab \in \Real^p} \left[\frac{1}{2}\|\x-\D\alphab\|_2^2 + \lambda \|\alphab\|_1\right],
\end{displaymath}
where we assume the solution of the Lasso problem to be unique.\footnote{When
the solution is not unique, the elastic-net regularization~\citep{zou} can be
used instead of the~$\ell_1$-norm. It consists of replacing the
penalty~$\|\alphab\|_1$ by~$\|\alphab\|_1+(\mu/2)\|\alphab\|_2^2$. The
resulting penalty still enjoy sparsity-inducing properties, but is strongly
convex and thus yields a unique solution.}
Then, the joint dictionary learning formulation consists of minimizing
\begin{equation}
   \min_{\substack{\D_l \in \CC_l \\ \D_h \in \Real^{m_h \times p}}} \frac{1}{n} \sum_{i=1}^n \frac{1}{2}\left\|\x_i^h - \D_h \alphab^\star(\x_i^l,\D_l)\right\|_2^2. \label{eq:taskdriven}
\end{equation}
The problem is challenging since it is non-convex, non-differentiable, and the
dependency of the objective with respect to~$\D_l$ goes through the argmin of
the Lasso formulation.  Such a setting is sometimes referred to as ``bilevel
optimization'' problems and is known to be difficult.
Eq.~(\ref{eq:taskdriven}) appears under the name of ``coupled dictionary
learning'' in~\citep{yang2012}, and is in fact a particular case of the more
general ``task-driven dictionary learning'' formulation of~\citet{mairal2012},
which we will present in more details in Section~\ref{subsec:backprop} about
visual recognition. It is possible to show
that in the asymptotic regime where~$n$ grows to infinity, the cost function to
optimize becomes differentiable and that its gradient can be estimated.
Based upon this observation, a stochastic gradient descent algorithm can be
used as an effective heuristic for finding an approximate solution
\citep[see][]{couzinie2011,mairal2012}.
We conclude this section with a few visual results, which we present in
Figure~\ref{fig:upscale}.
\begin{figure}[hbtp]
   \centering
   \subfigure[Original]{\includegraphics[width=0.32\linewidth]{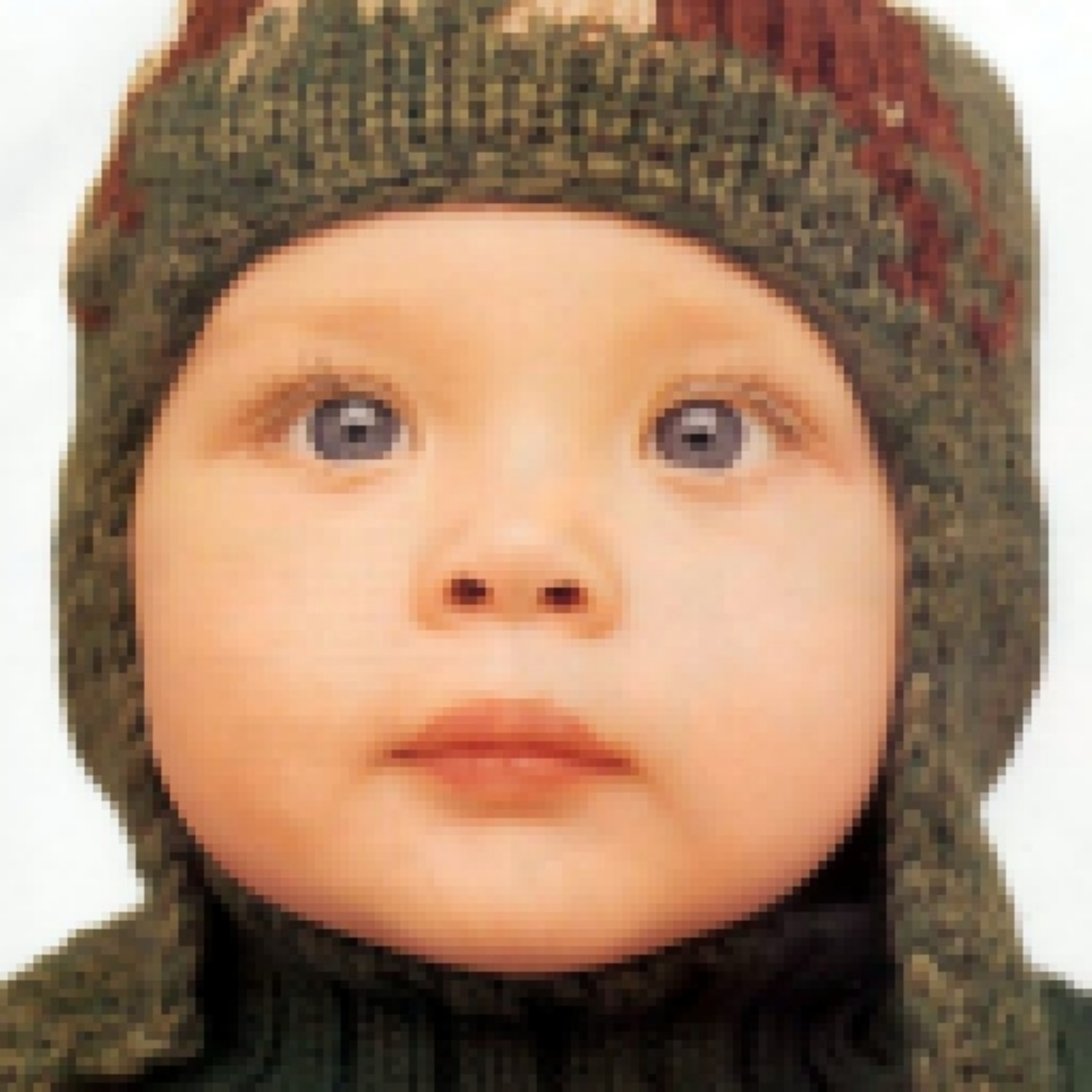}\label{fig:upscalea}} 
   \subfigure[Bicubic interpolation]{\includegraphics[width=0.32\linewidth]{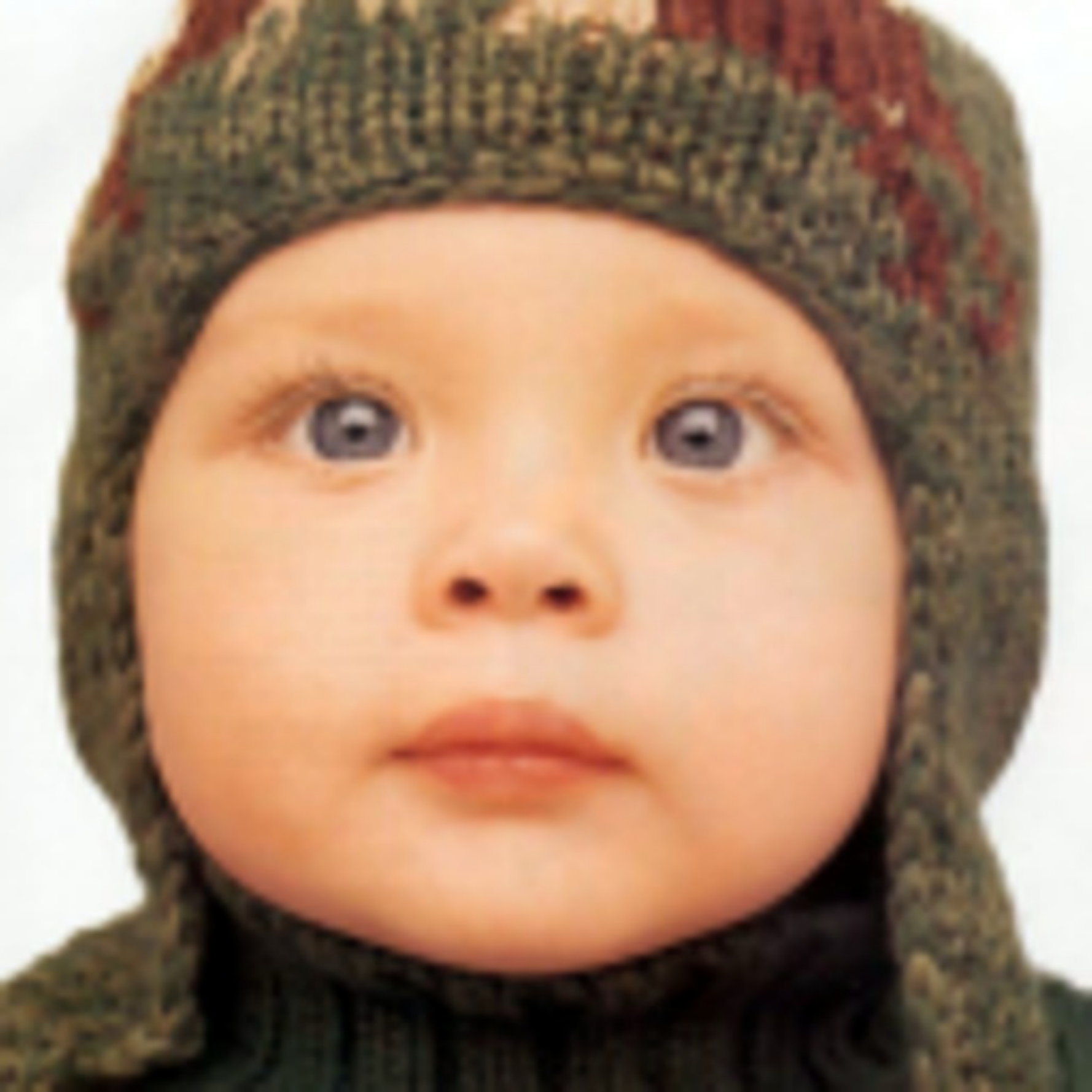}\label{fig:upscaleb}} 
   \subfigure[With sparse coding]{\includegraphics[width=0.32\linewidth]{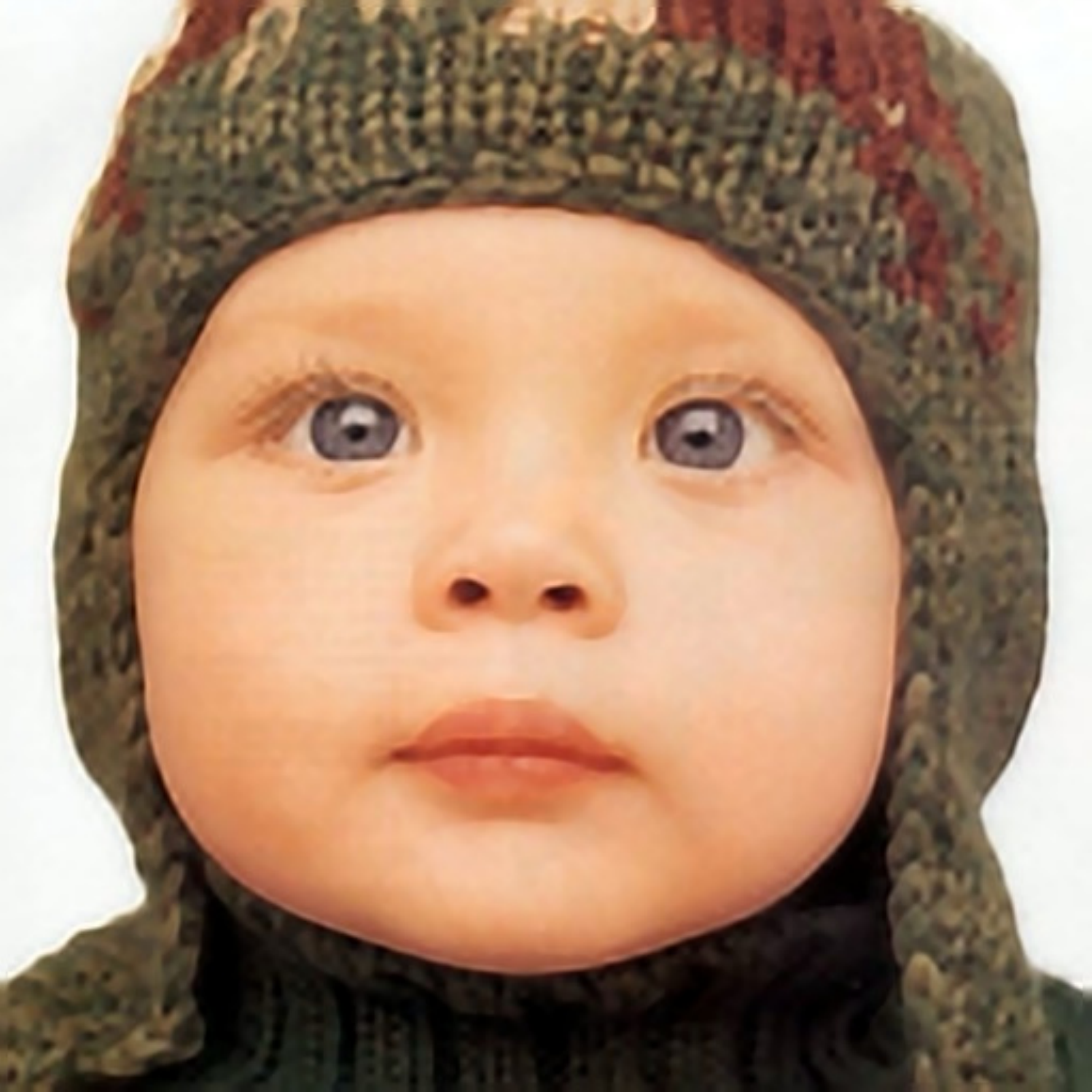}\label{fig:upscalec}} \\
   \subfigure[Zoom on~\subref{fig:upscalea}]{\includegraphics[width=0.32\linewidth,trim=0 20 0 5,clip=true]{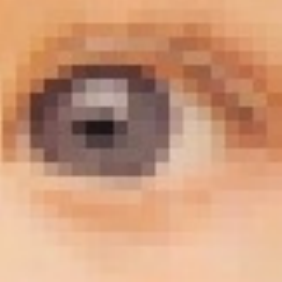}} 
   \subfigure[Zoom on~\subref{fig:upscaleb}]{\includegraphics[width=0.32\linewidth,trim=0 20 0 5,clip=true]{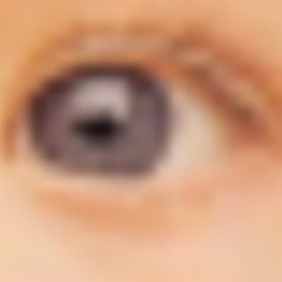}} 
   \subfigure[Zoom on~\subref{fig:upscalec}]{\includegraphics[width=0.32\linewidth,trim=0 20 0 5,clip=true]{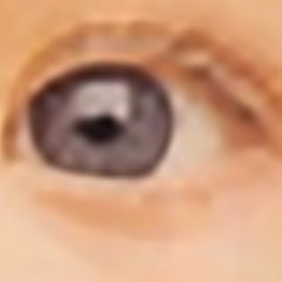}} \\
   \subfigure[Original]{\includegraphics[width=0.32\linewidth]{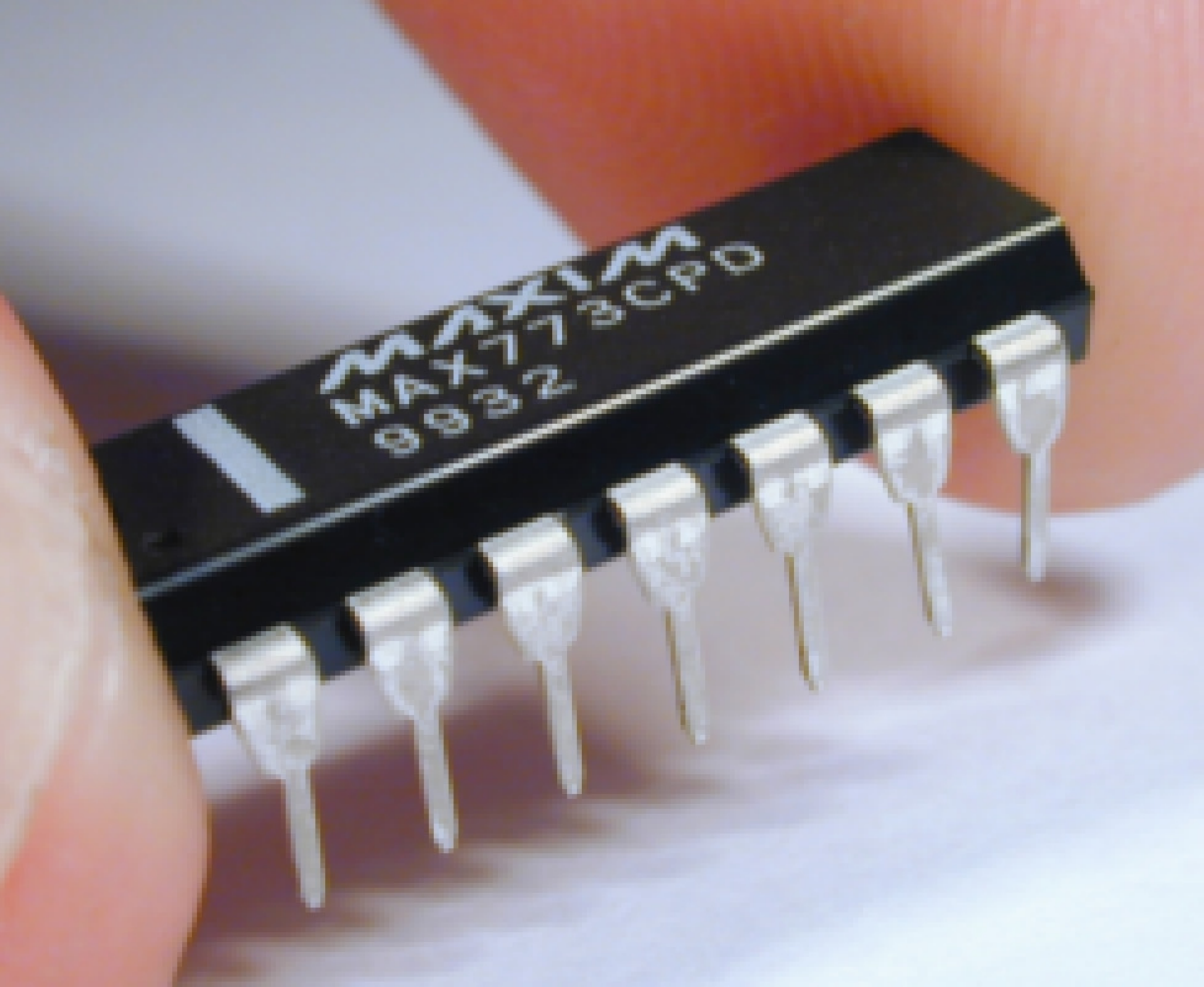}\label{fig:upscaled}} 
   \subfigure[Bicubic interpolation]{\includegraphics[width=0.32\linewidth]{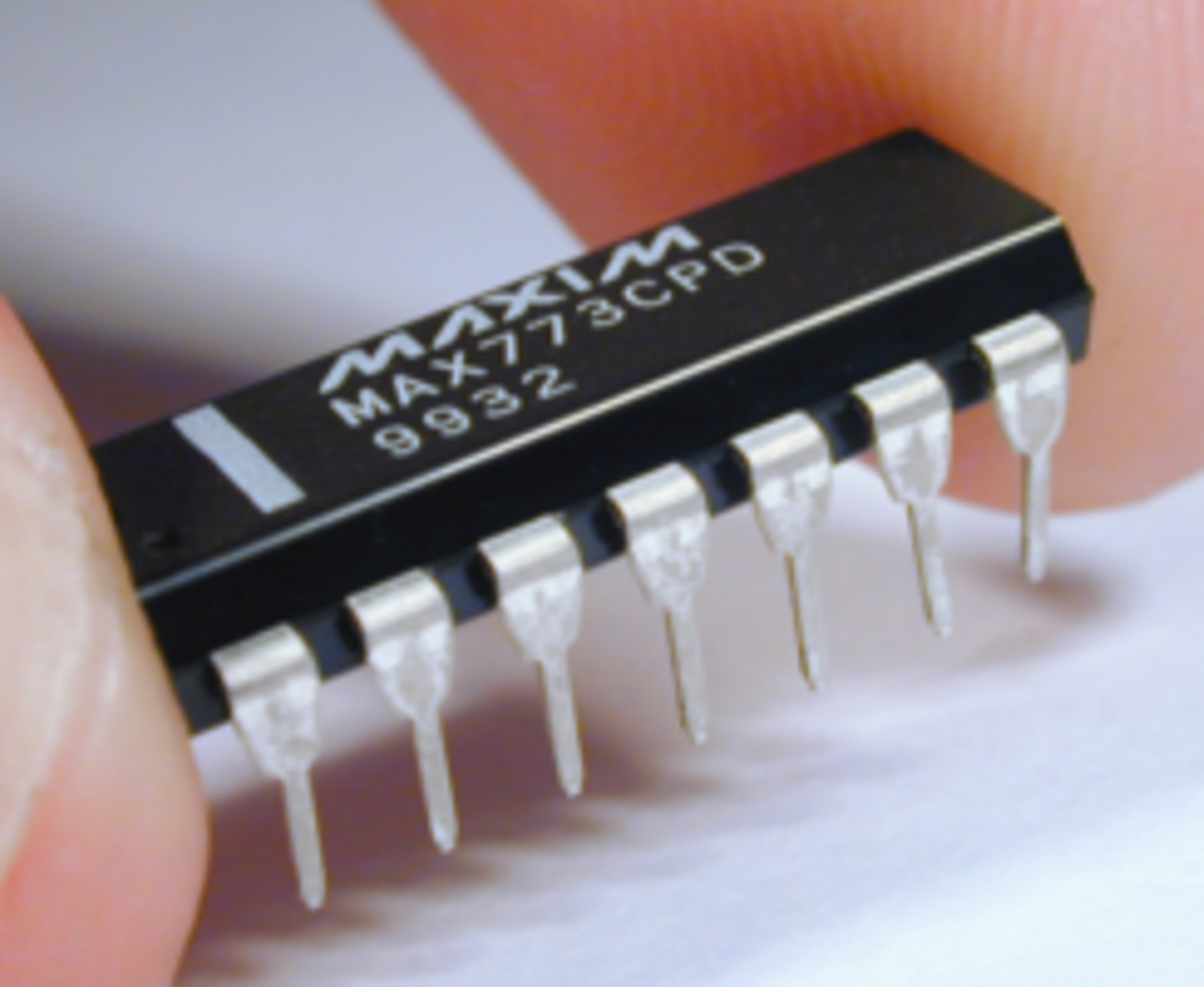}\label{fig:upscalee}} 
   \subfigure[With sparse coding]{\includegraphics[width=0.32\linewidth]{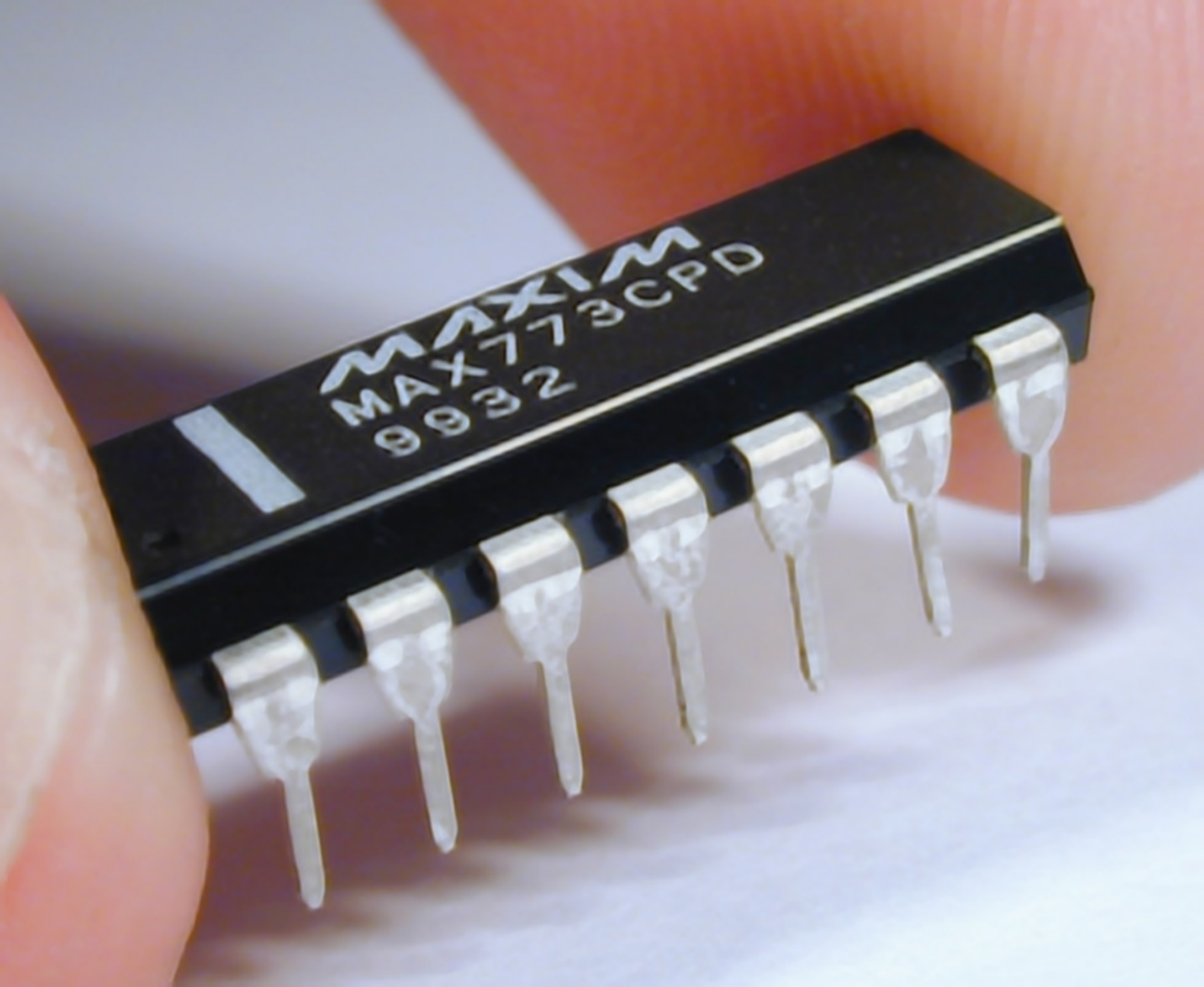}\label{fig:upscalef}} \\
   \subfigure[Zoom on~\subref{fig:upscaled}]{\includegraphics[width=0.32\linewidth,trim=0 20 0 5,clip=true]{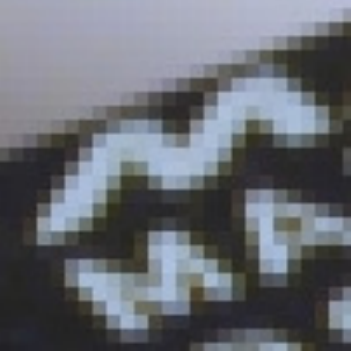}} 
   \subfigure[Zoom on~\subref{fig:upscalee}]{\includegraphics[width=0.32\linewidth,trim=0 20 0 5,clip=true]{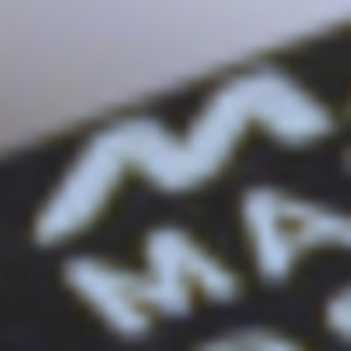}} 
   \subfigure[Zoom on~\subref{fig:upscalef}]{\includegraphics[width=0.32\linewidth,trim=0 20 0 5,clip=true]{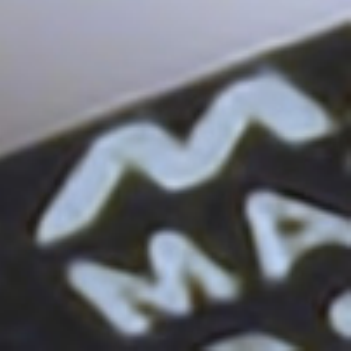}} 
   \caption{Image upscaling results obtained with the dictionary learning
   technique of~\citet{couzinie2011}. Images produced by Florent
Couzinie-Devy.}
   \label{fig:upscale}
\end{figure}

%% file: content_arxiv/image_invert.tex
Interestingly, the image up-scaling formulations that we have
presented previously do not explicitly take into account the type of
transformation between the patches~$\x_i^h$ and~$\x_i^l$. In other words, they
only assume that the pairs of patches~$\x_i^h$ and~$\x_i^l$ admit a common
sparse representation onto two different dictionaries, but they do not make any
a priori assumption on the relation between the patches. For this reason, it is
appealing to try similar formulations for other problems than image
upscaling.

Several extensions have been developed along this line of work.
\citet{couzinie2011} and~\citet{dong2011} have for instance developed image
deblurring techniques that are effective when the blur is local. Other
less-standard applications also appear in the literature. \citet{dong2011}
propose indeed to learn a mapping between face photographs and hand-drawn
sketches. They use a database of images called CUFS~\citep{wang2009} of
about~$600$ subjects. For each subject, a grayscale photograph is provided,
along with a sketch drawn by an artist. The dataset can be used to create
pairs of photograph-sketch patches and a nonlinear mapping is learned between
the two patch spaces, providing convincing visual results.

Finally, we illustrate the ability of the previous formulations for learning
nonlinear mappings with the problem of reconstructing grayscale images from
binary ones, a task called ``inverse halftoning''. With the use of binary
display and printing technologies in the 70's, converting a grayscale
continuous-tone image into a binary one that looks perceptually similar to the
original one (``halftoning'') was of utmost importance. A classical algorithm
was developed for that purpose by~\citet{floyd1976}. The goal of ``inverse
halftoning'' is to restore these binary images and estimate the original ones.
Since building a large database of pairs of grayscale-halftoned image patches
is easy, the methodology of the previous section can be used for learning a
mapping between binary and grayscale patches. In Figure~\ref{fig:ht}, we
display a few visual results obtained with the approach of~\citet{mairal2012}.
The dictionaries were trained on an external database build with the
Floyd-Steinberg algorithm.

\begin{figure}[hbtp]
   \centering
   \subfigure[Original]{\includegraphics[width=0.32\linewidth]{images_arxiv/hill.pdf}\label{fig:hta}} 
   \subfigure[Halftoned image]{\includegraphics[width=0.32\linewidth]{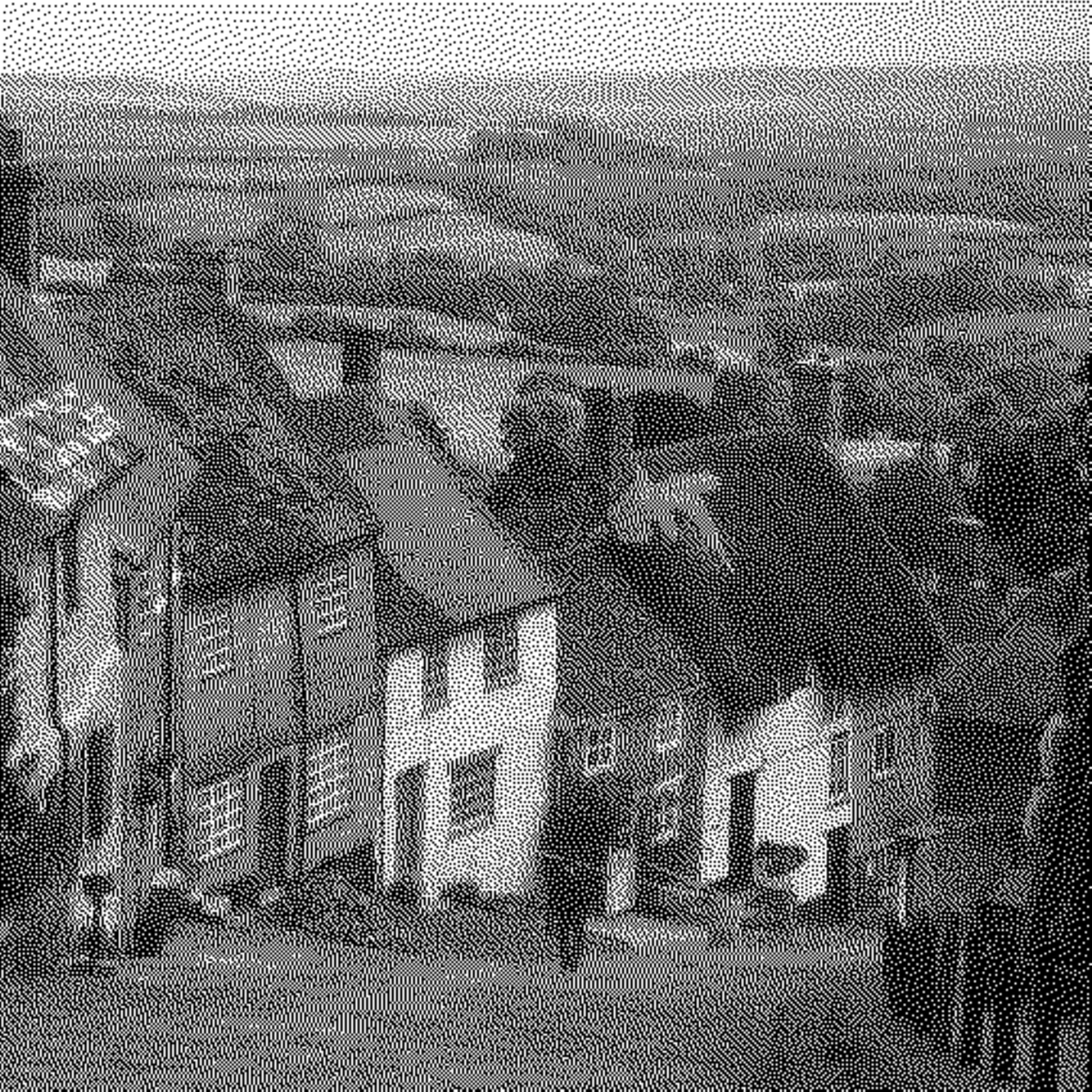}\label{fig:htb}} 
   \subfigure[Reconstructed image]{\includegraphics[width=0.32\linewidth]{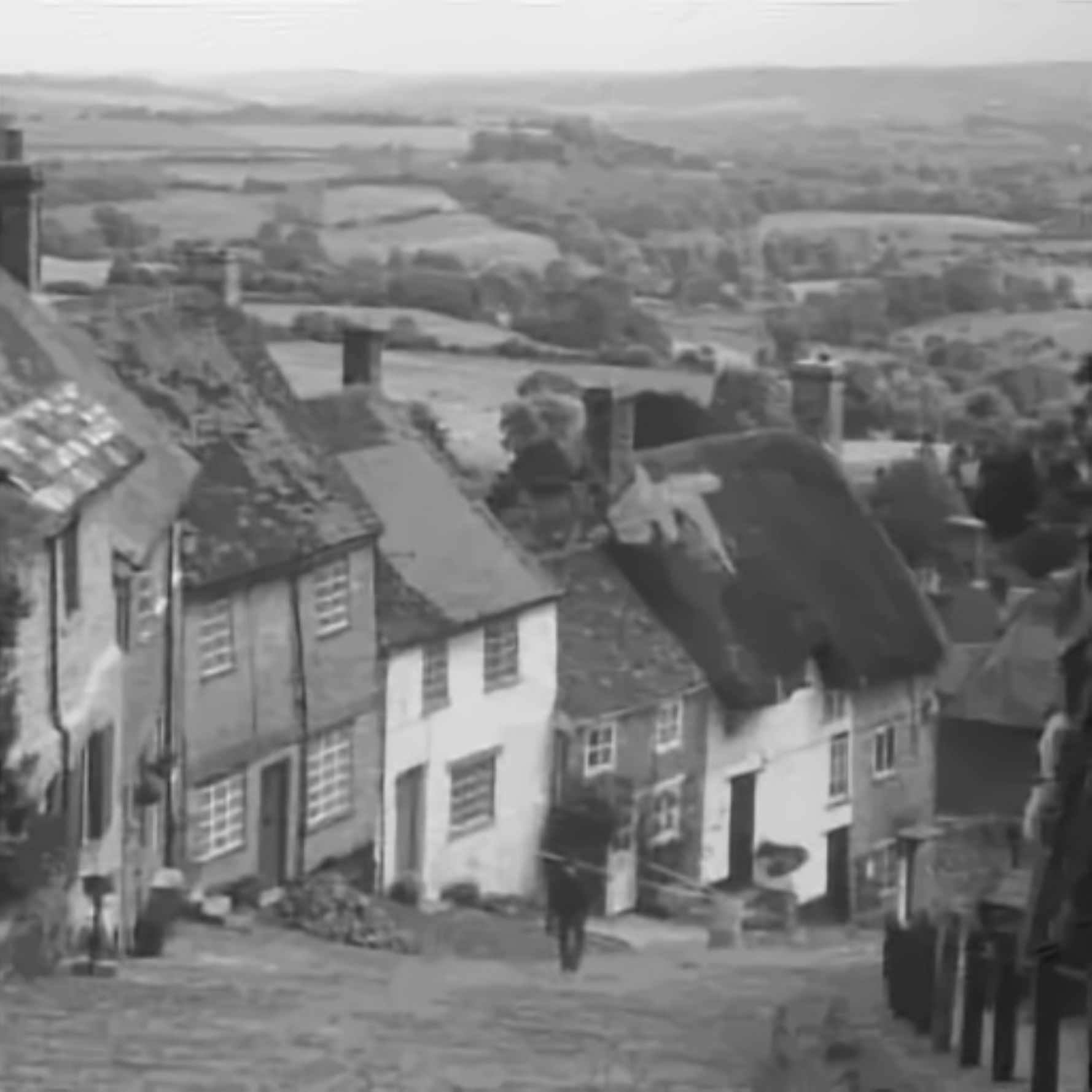}\label{fig:htc}} \\
   \subfigure[Zoom on~\subref{fig:hta}]{\includegraphics[width=0.32\linewidth]{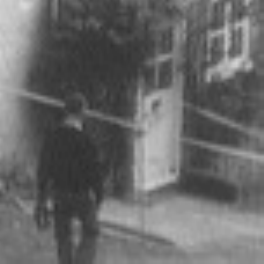}} 
   \subfigure[Zoom on~\subref{fig:htb}]{\includegraphics[width=0.32\linewidth]{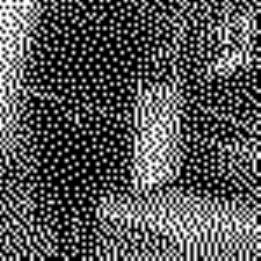}} 
   \subfigure[Zoom on~\subref{fig:htc}]{\includegraphics[width=0.32\linewidth]{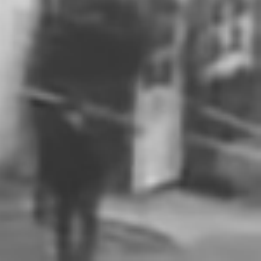}} \\
   \subfigure[Binary image]{\includegraphics[width=0.48\linewidth]{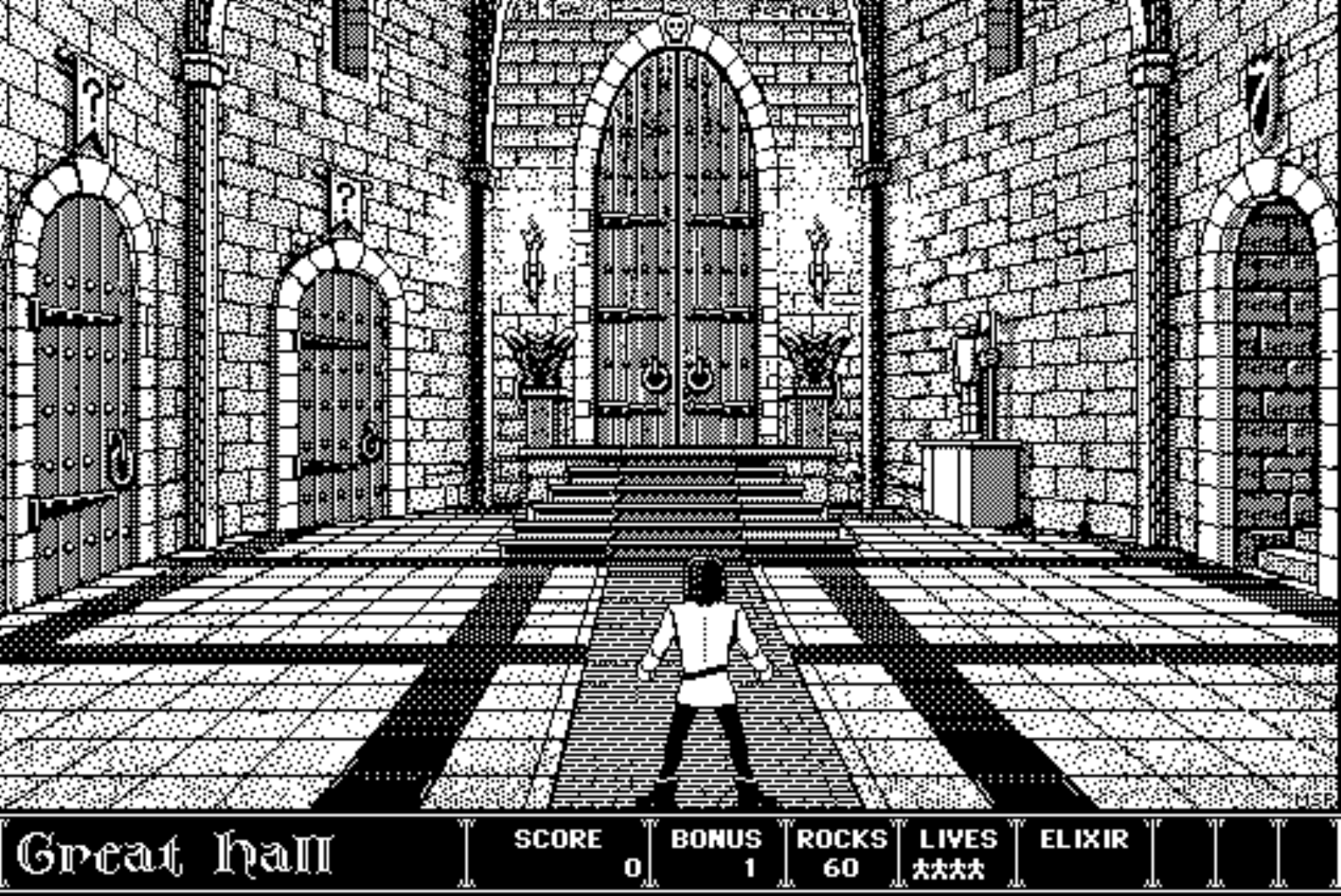}} 
   \subfigure[Restored image]{\includegraphics[width=0.48\linewidth]{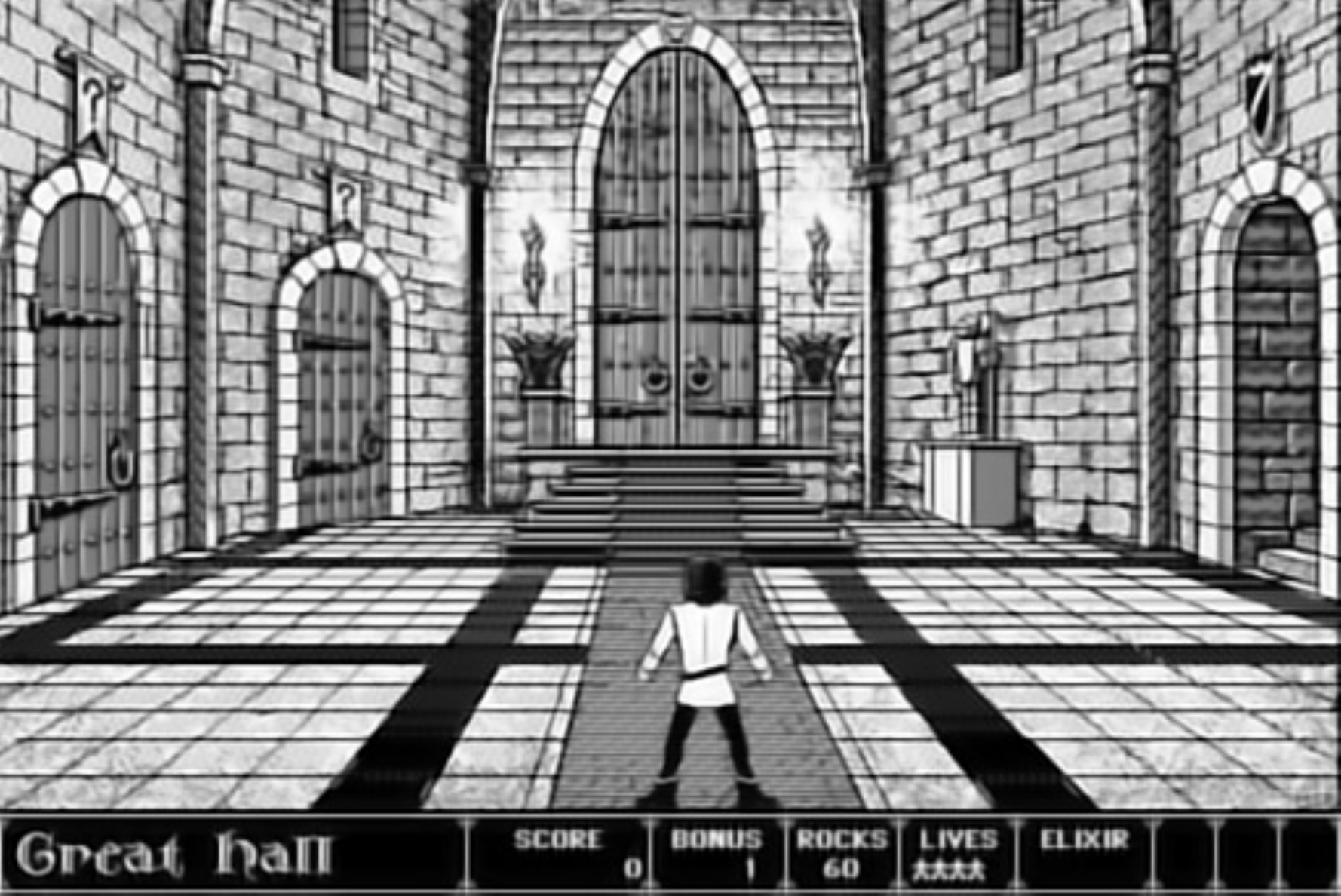}} 
   \caption{Inverse halftoning results obtained by~\citet{mairal2012}. The last
      row presents a result on a binary image publicly available on the
      Internet. No ground truth is available for this binary image from an old
   computer game. Despite the fact that the algorithm that has generated this
image is unknown and probably different than the Floyd-Steinberg algorithm, the
result is visually convincing. Best seen by zooming on a computer screen.}
\label{fig:ht}
\end{figure}

%% file: content_arxiv/image_video.tex
Extensions of dictionary learning techniques for dealing with videos have
been proposed by~\citet{protter2009}. Given a noisy video sequence, the most naive
approach consists of processing each frame independently. However,
significantly better results can be obtained by exploiting the temporal
consistency that exists among consecutive frames. To do so, a few modifications
to the classical denoising formulation are sufficient:
\begin{itemize}
   \item several frames can be processed at the same time by considering
      three-dimensional patches that involve one temporal dimension.
      For instance, a video patch can be of size $m=e \times e \times T$,
      where~$T$ is the number of frames in the patch, \eg, 
      $e=10$ pixels and $T=5$ frames.
   \item After processing $T$ frames starting at time~$t$, we obtain a
      dictionary~$\D_t$ adapted to the set of three-dimensional patches at
      time~$t$. When moving to the next block of $T$ frames at time~$t+1$
      (considering that the blocks overlap in time), we can learn a
      dictionary~$\D_{t+1}$ adapted to the information available a time~$t+1$,
      using~$\D_t$ as an initialization to the learning algorithm.
\end{itemize}
In Figures~\ref{fig:dict_videoA} and~\ref{fig:dict_videoB}, we present two
video processing results from~\citep{mairal2008b}, where the original video
extension of~\citet{protter2009} has been adapted to the inpainting and color
video denoising tasks. The results of the ``video processing'' approach are 
significantly better than those obtained by processing independently each frame.
\begin{figure}[hbtp]
  \subfigure{\includegraphics[width=0.24\linewidth]{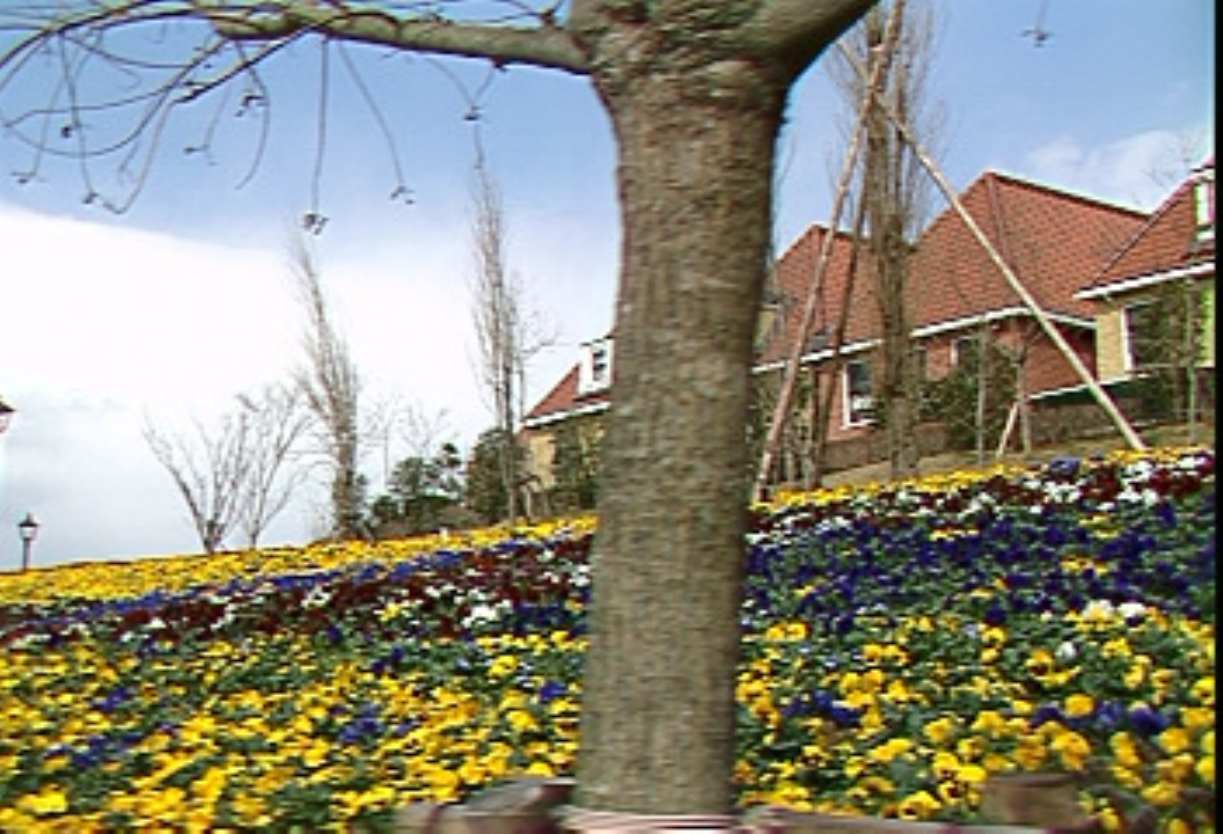}} \hfill
  \subfigure{\includegraphics[width=0.24\linewidth]{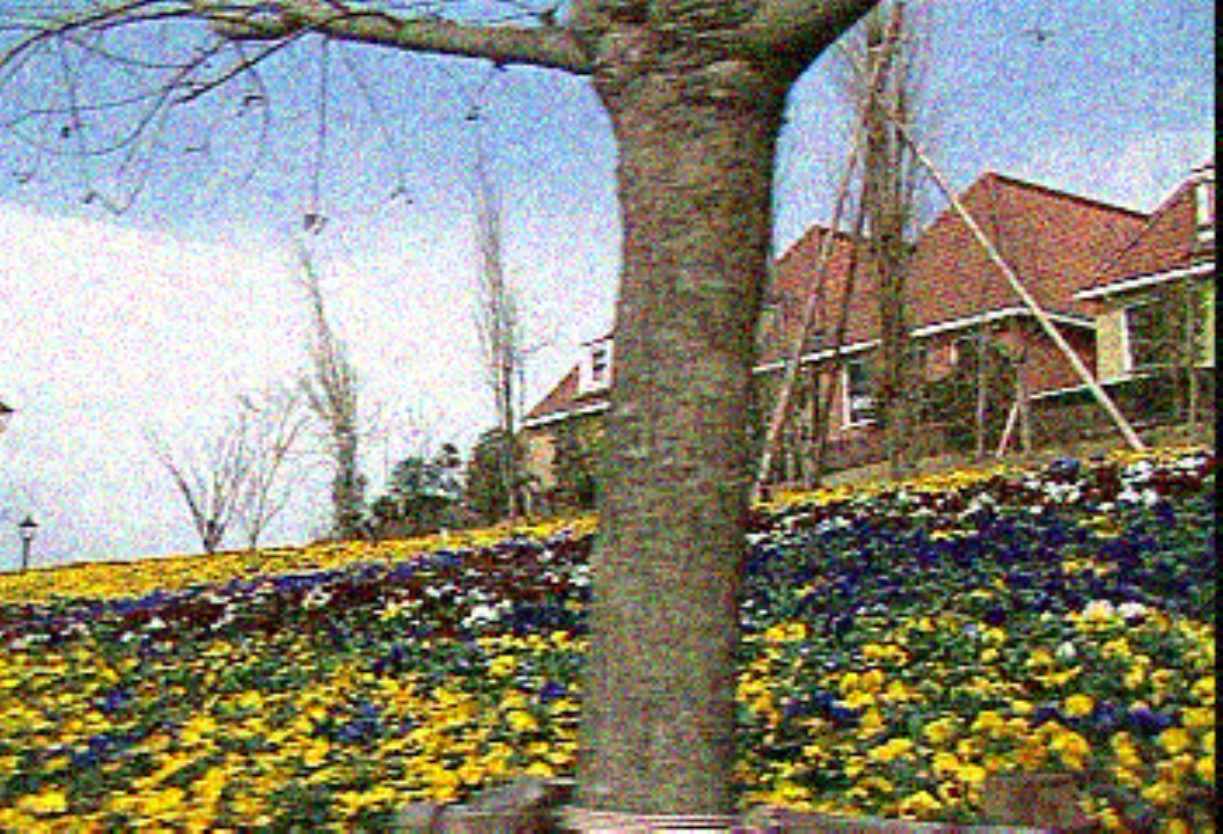}} \hfill
  \subfigure{\includegraphics[width=0.24\linewidth]{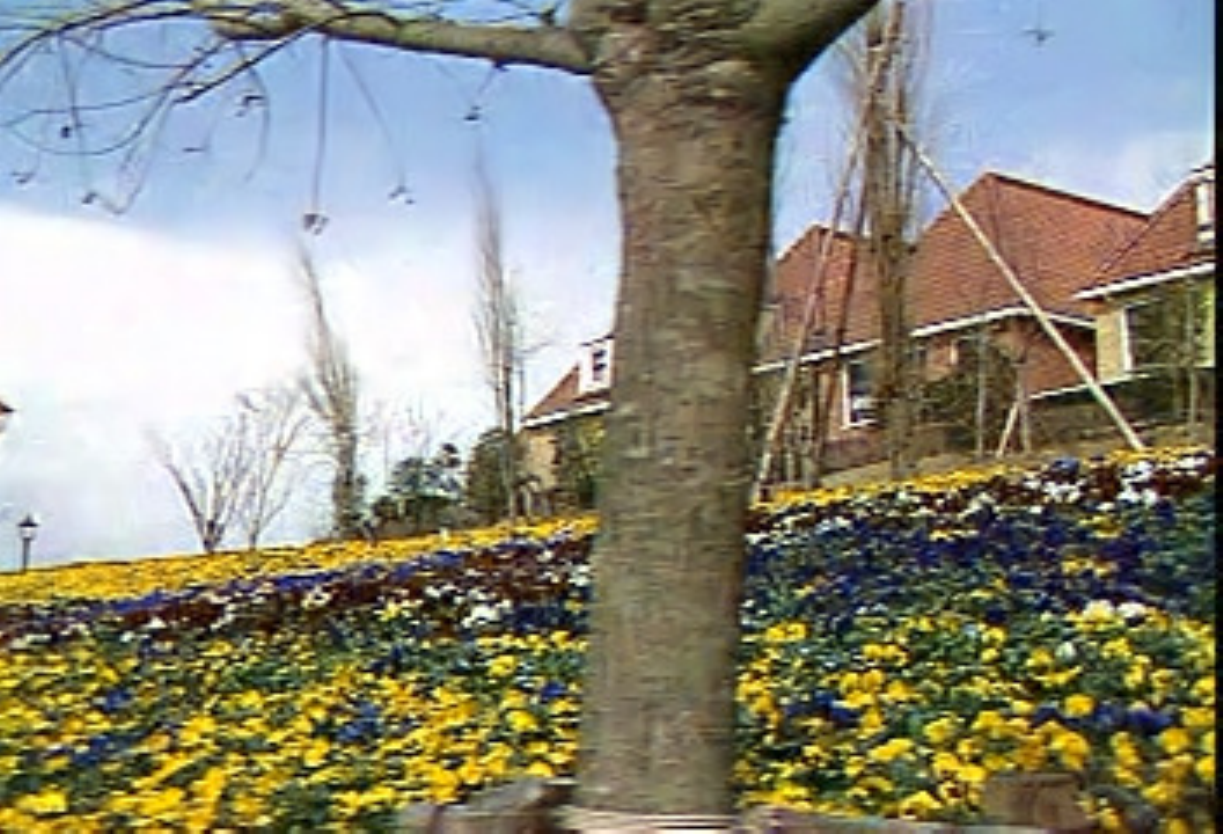}} \hfill
  \subfigure{\includegraphics[width=0.24\linewidth]{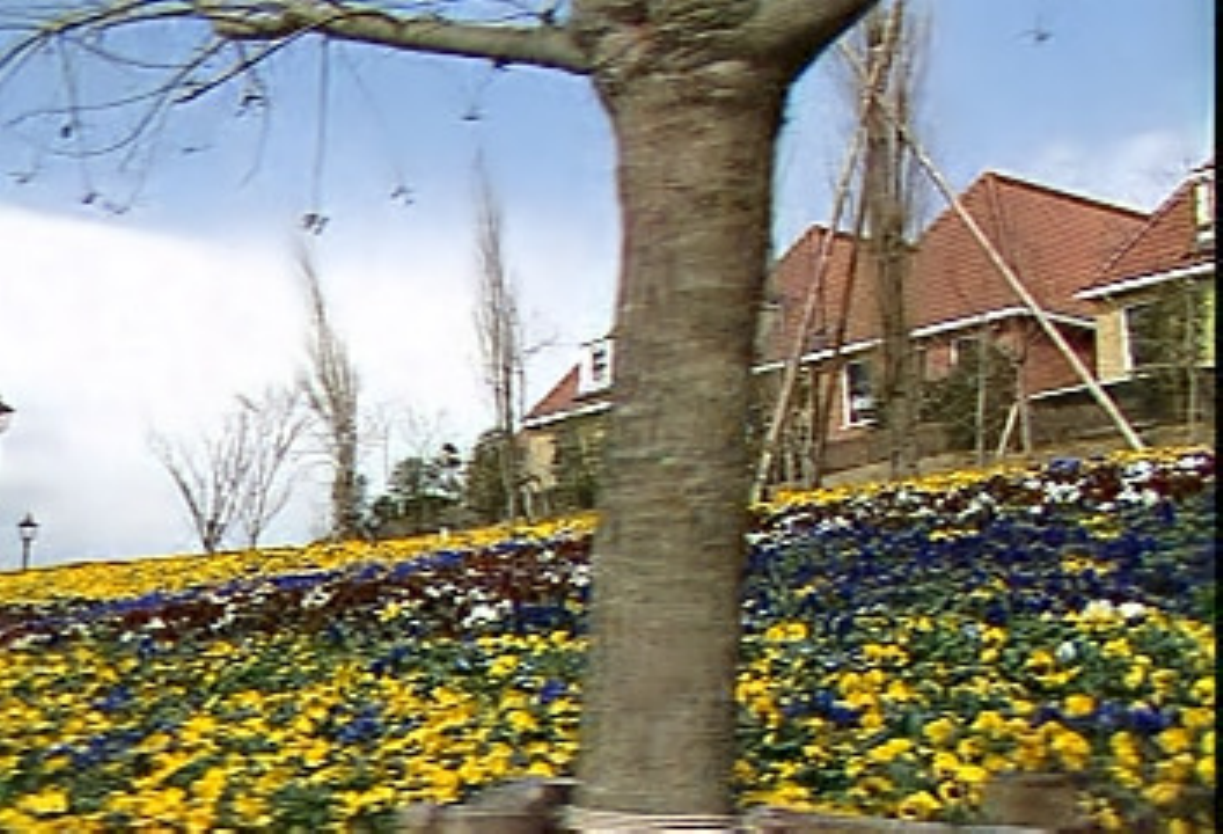}} \\
  \subfigure{\includegraphics[width=0.24\linewidth]{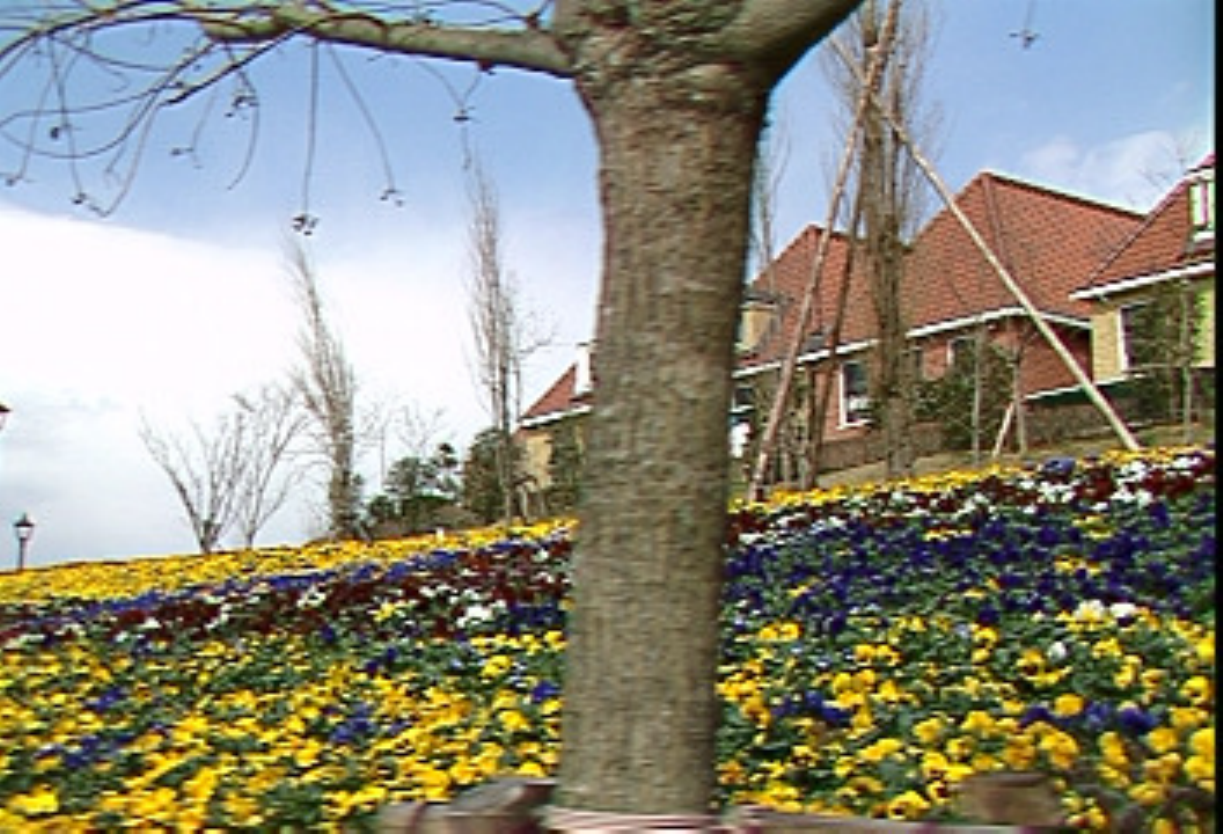}} \hfill
  \subfigure{\includegraphics[width=0.24\linewidth]{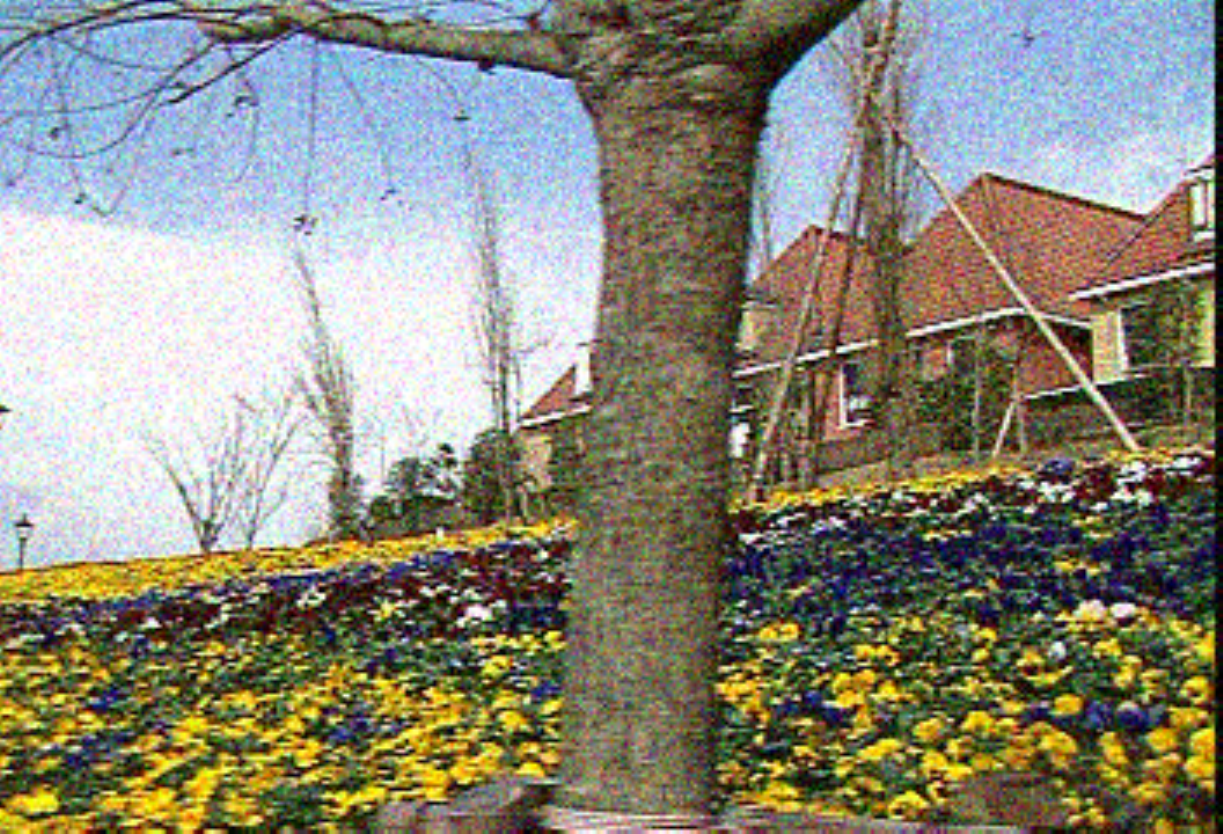}} \hfill
  \subfigure{\includegraphics[width=0.24\linewidth]{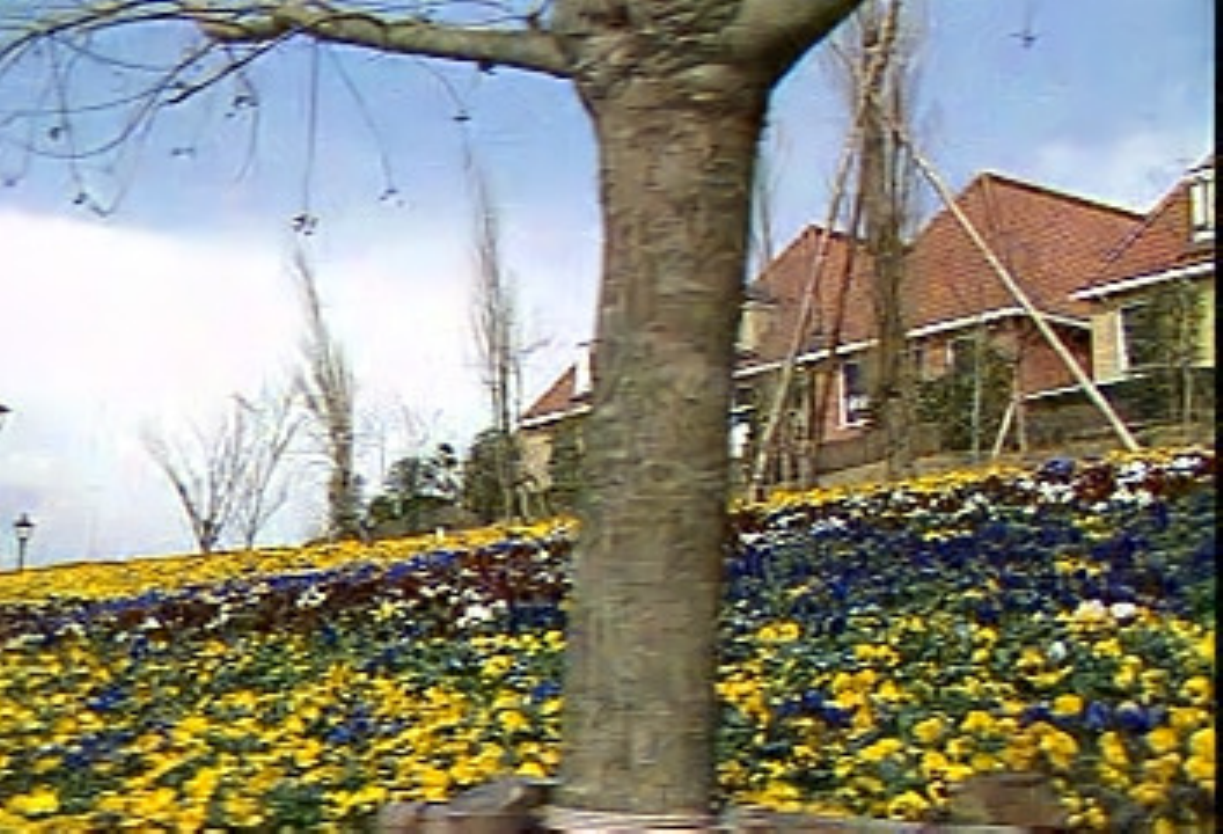}} \hfill
  \subfigure{\includegraphics[width=0.24\linewidth]{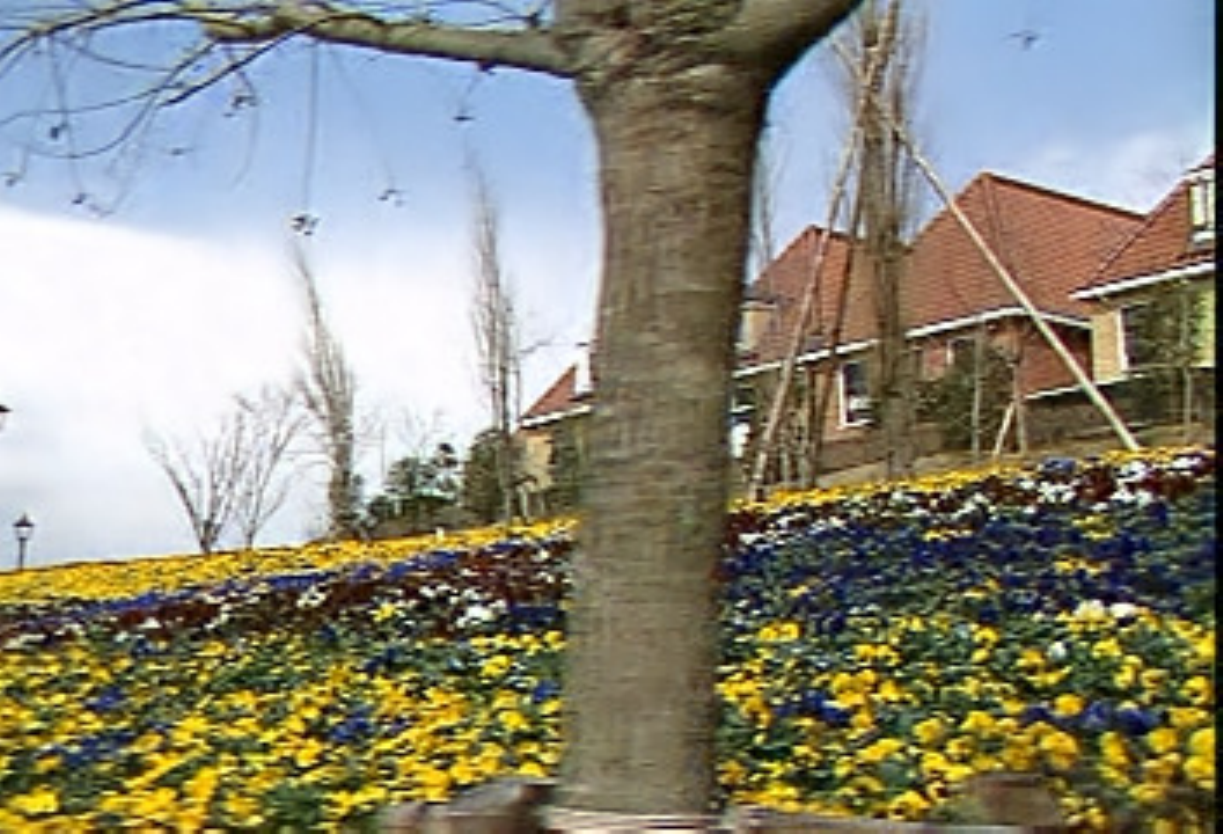}} \\
  \subfigure{\includegraphics[width=0.24\linewidth]{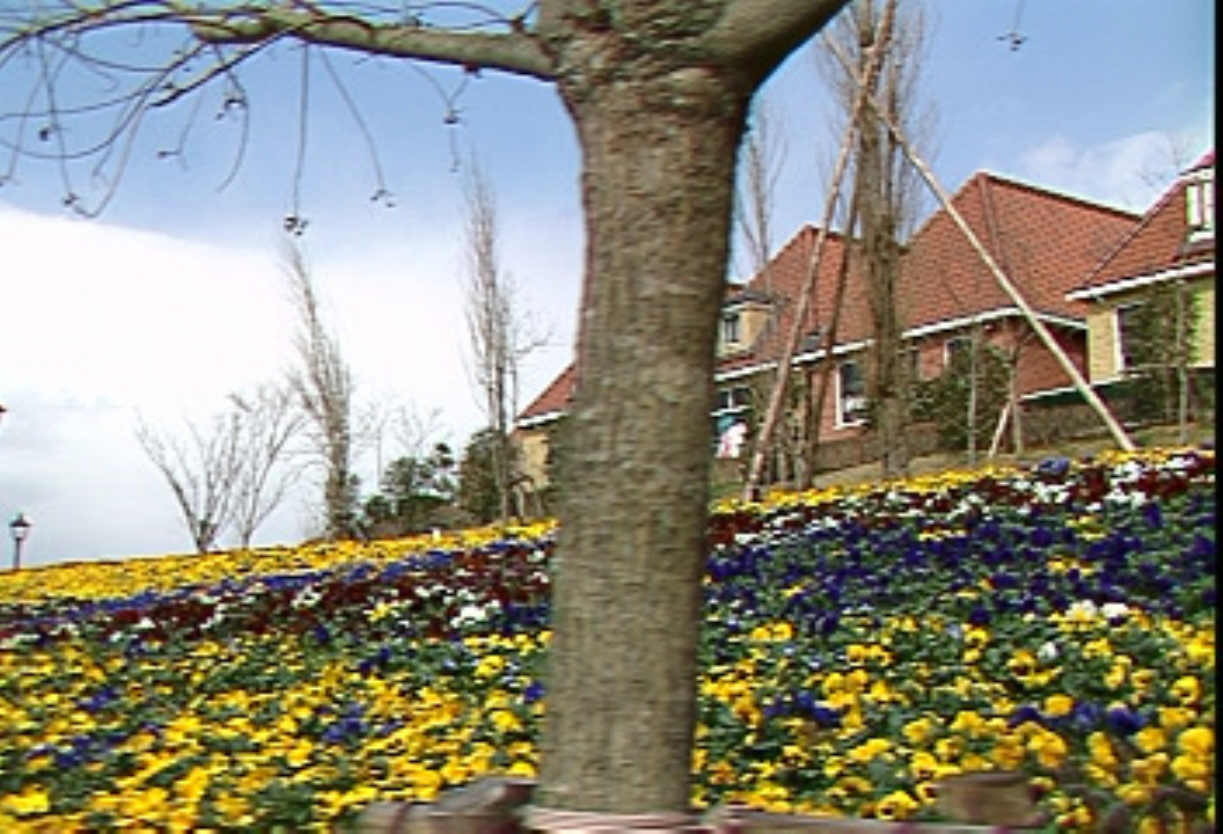}} \hfill
  \subfigure{\includegraphics[width=0.24\linewidth]{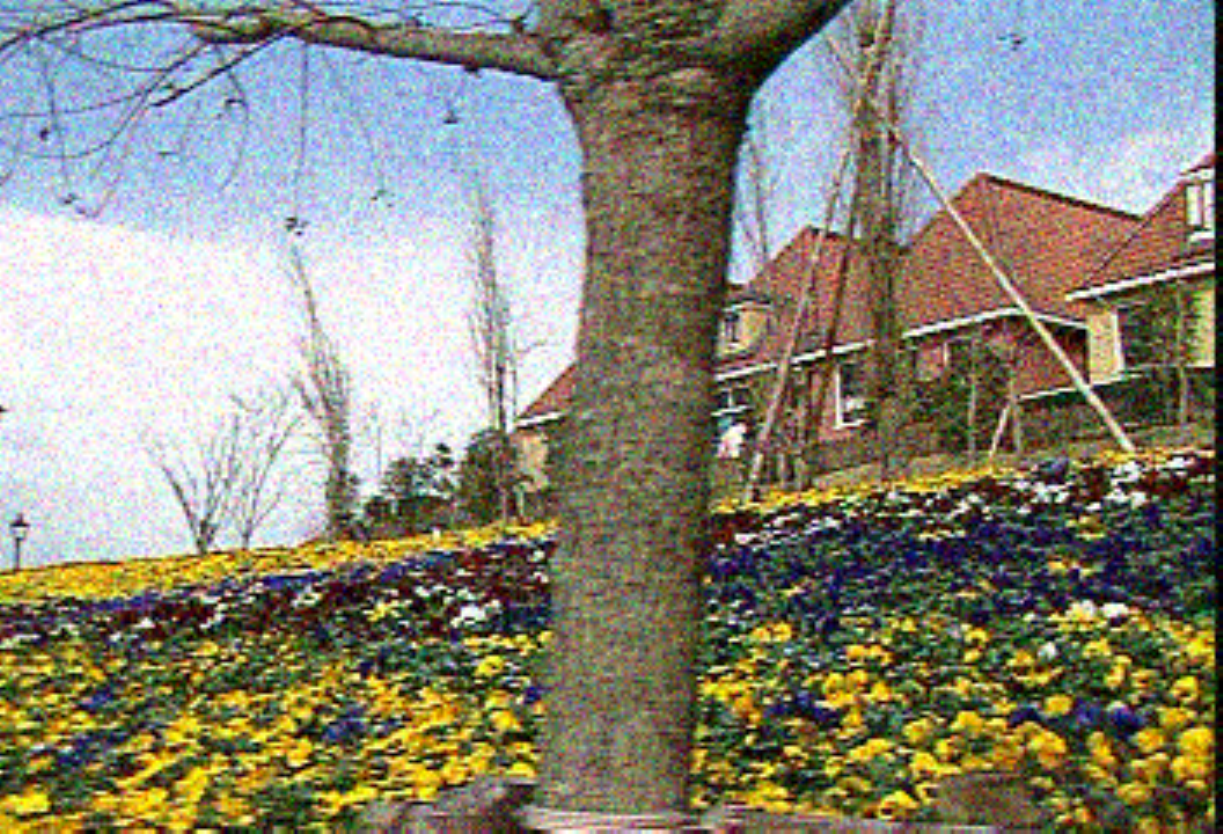}} \hfill
  \subfigure{\includegraphics[width=0.24\linewidth]{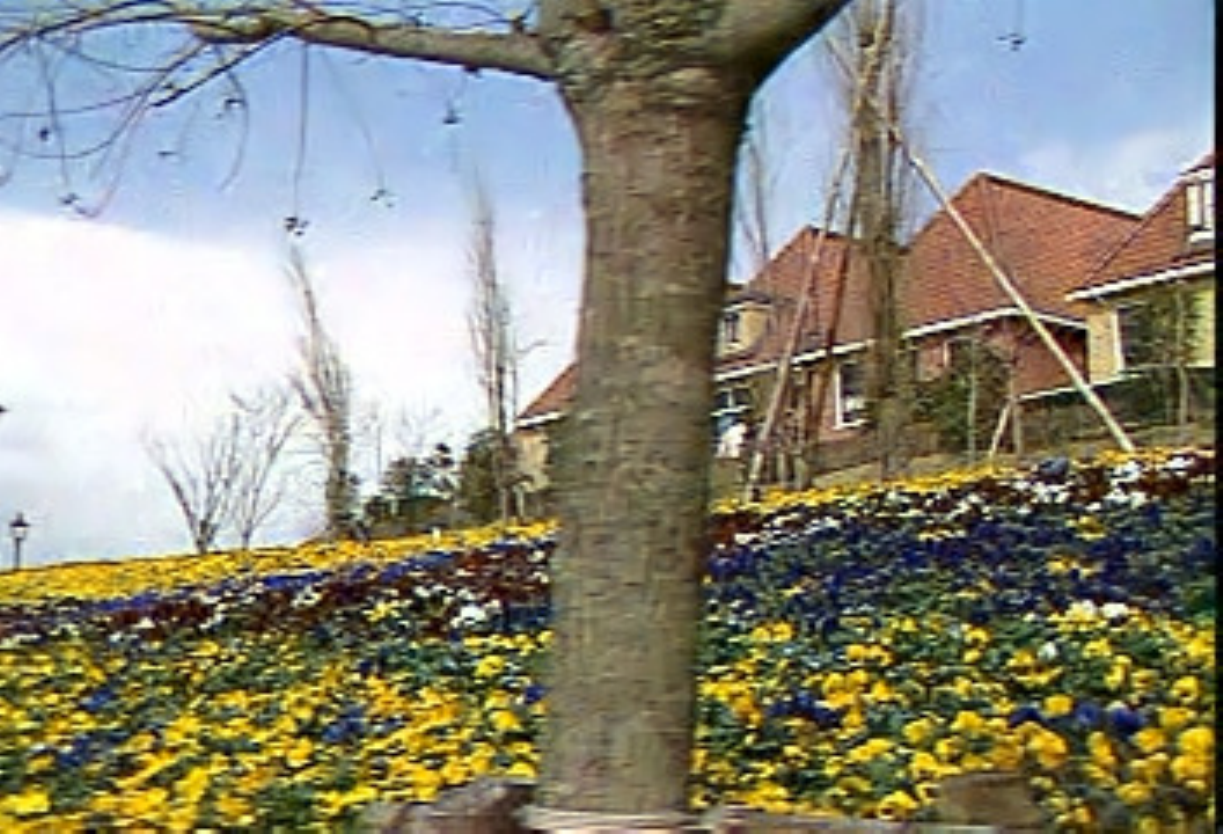}} \hfill
  \subfigure{\includegraphics[width=0.24\linewidth]{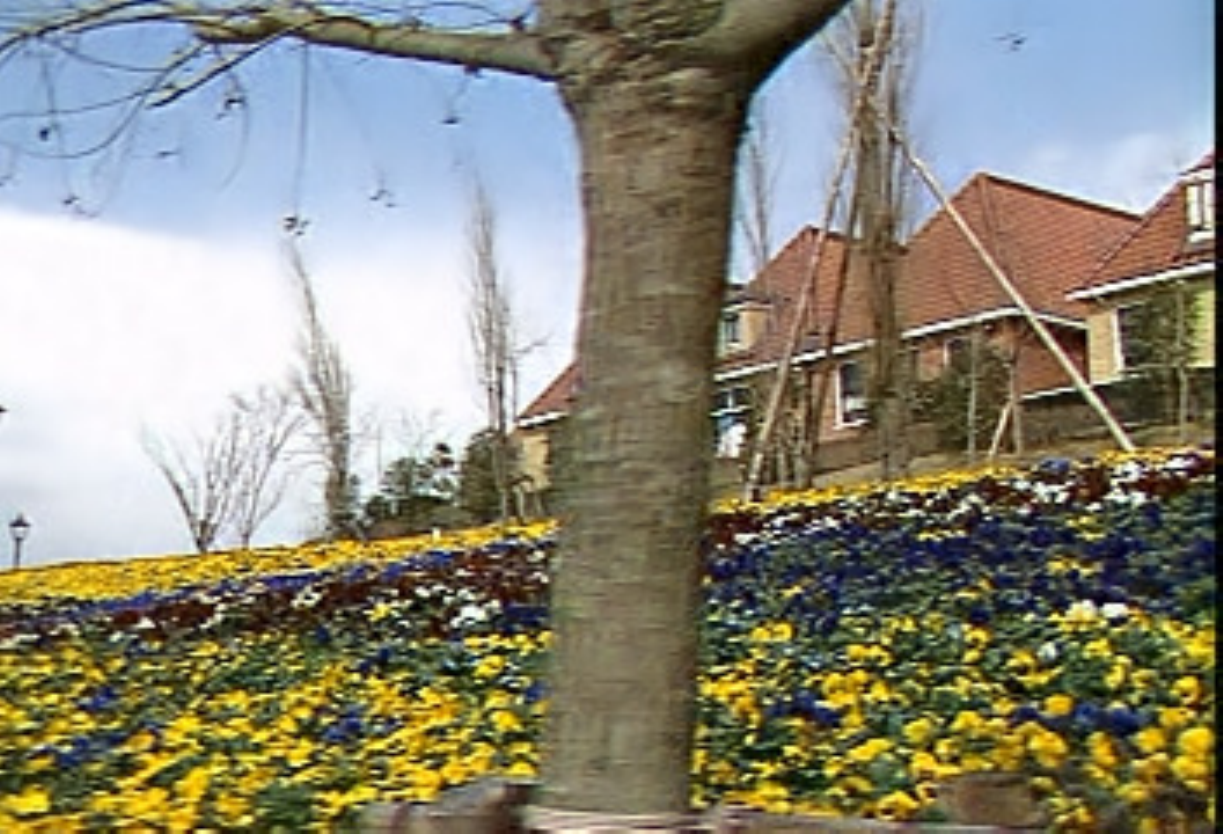}} \\
  \subfigure{\includegraphics[width=0.24\linewidth]{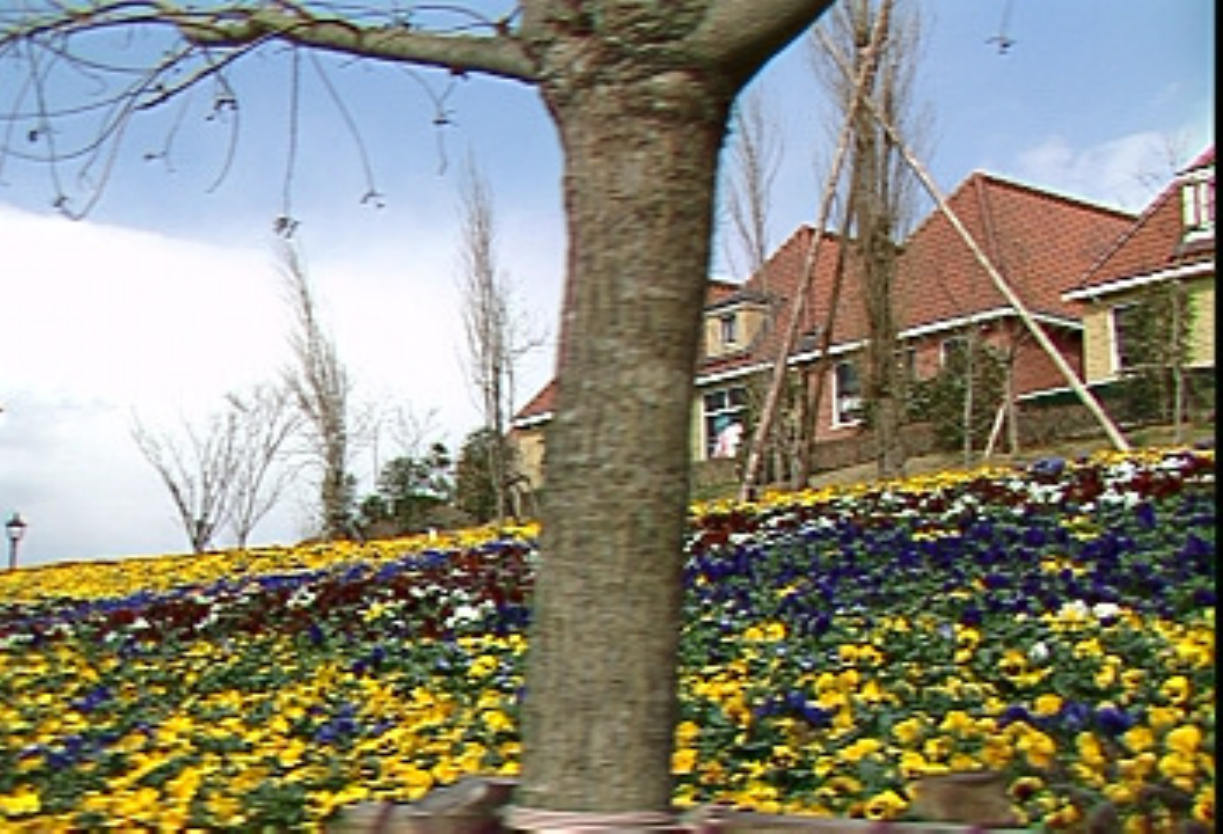}} \hfill
  \subfigure{\includegraphics[width=0.24\linewidth]{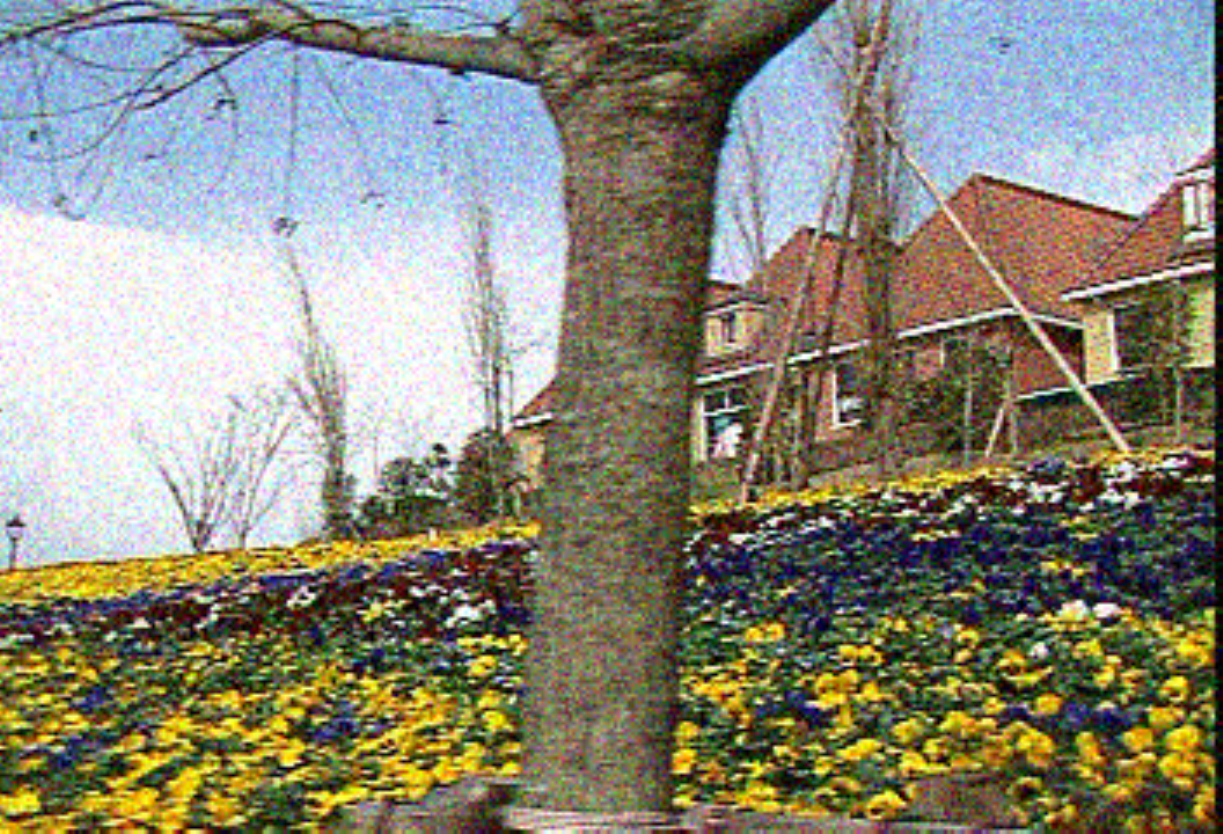}} \hfill
  \subfigure{\includegraphics[width=0.24\linewidth]{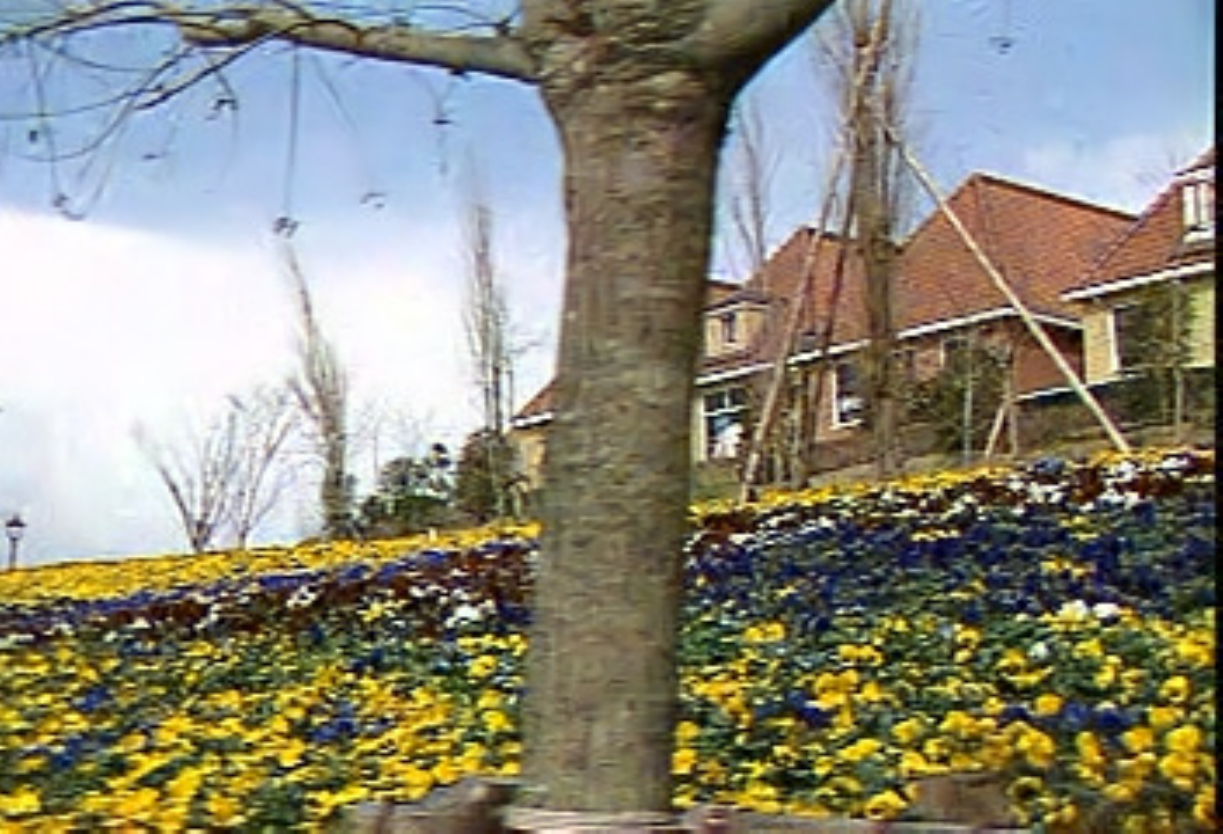}} \hfill
  \subfigure{\includegraphics[width=0.24\linewidth]{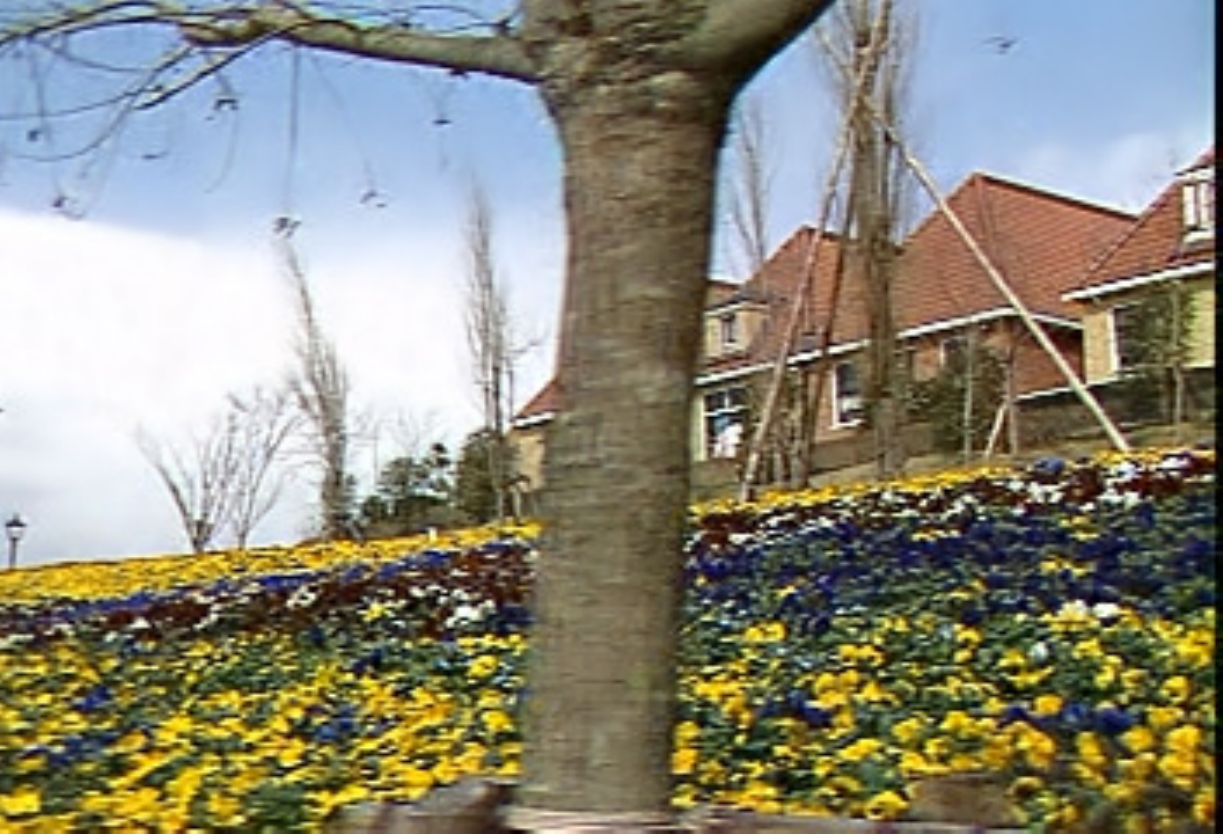}} \\
  \addtocounter{subfigure}{-16}
  \subfigure[{Original}]{\includegraphics[width=0.24\linewidth]{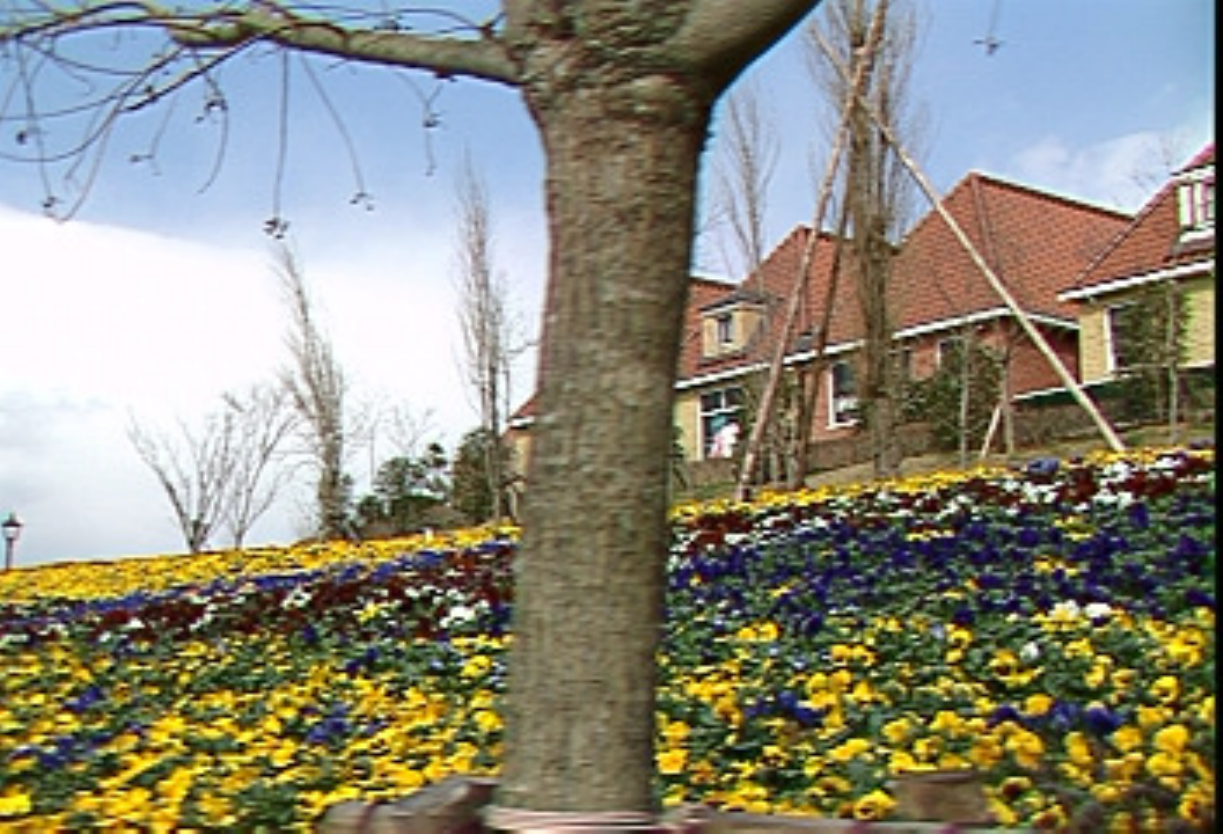}} \hfill
  \subfigure[{Damaged}]{\includegraphics[width=0.24\linewidth]{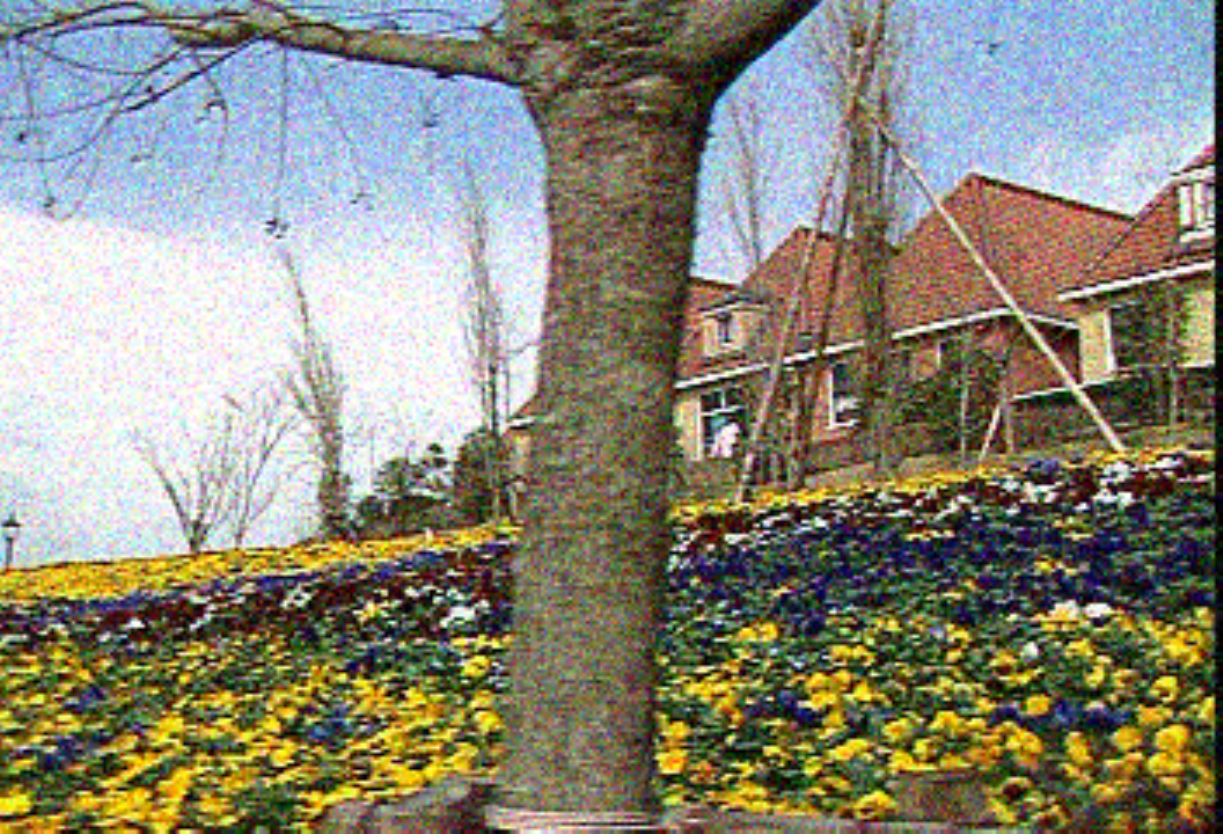}} \hfill
  \subfigure[{Image denoising}]{\includegraphics[width=0.24\linewidth]{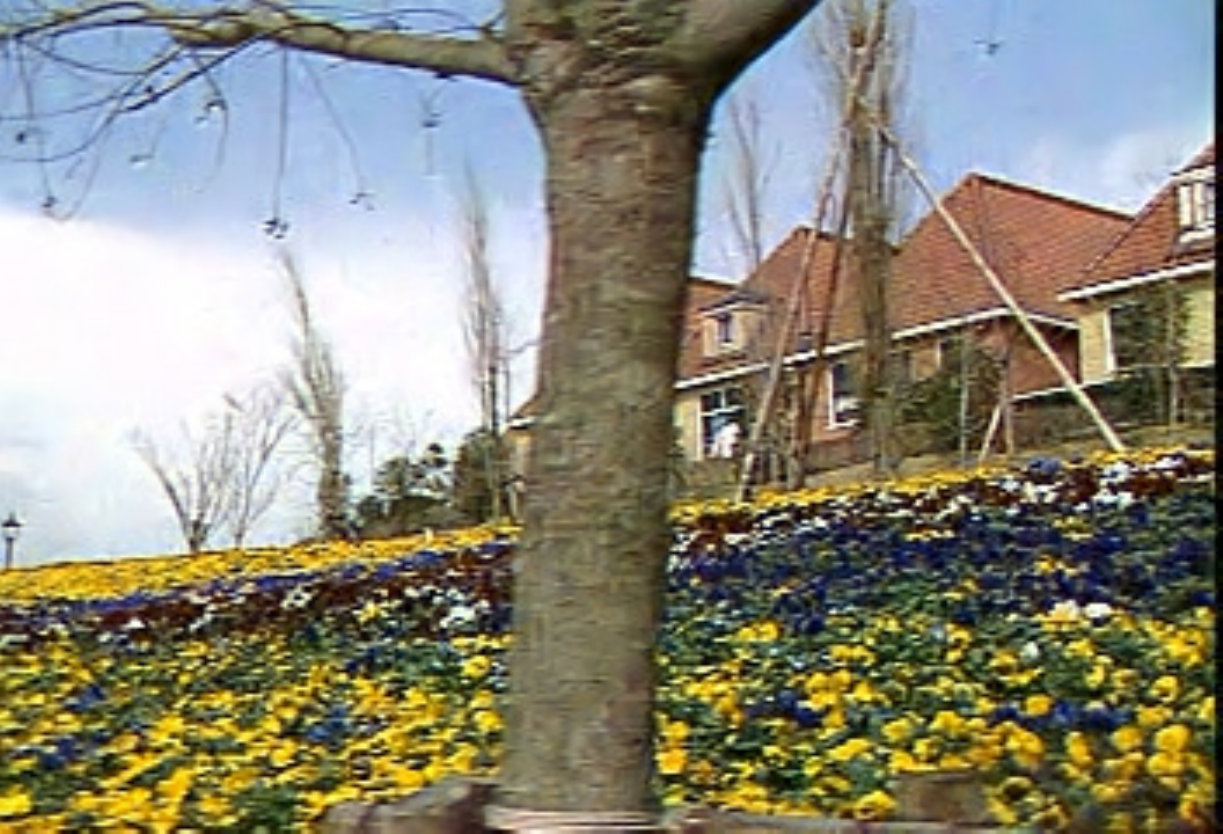}} \hfill
  \subfigure[{Video denoising}]{\includegraphics[width=0.24\linewidth]{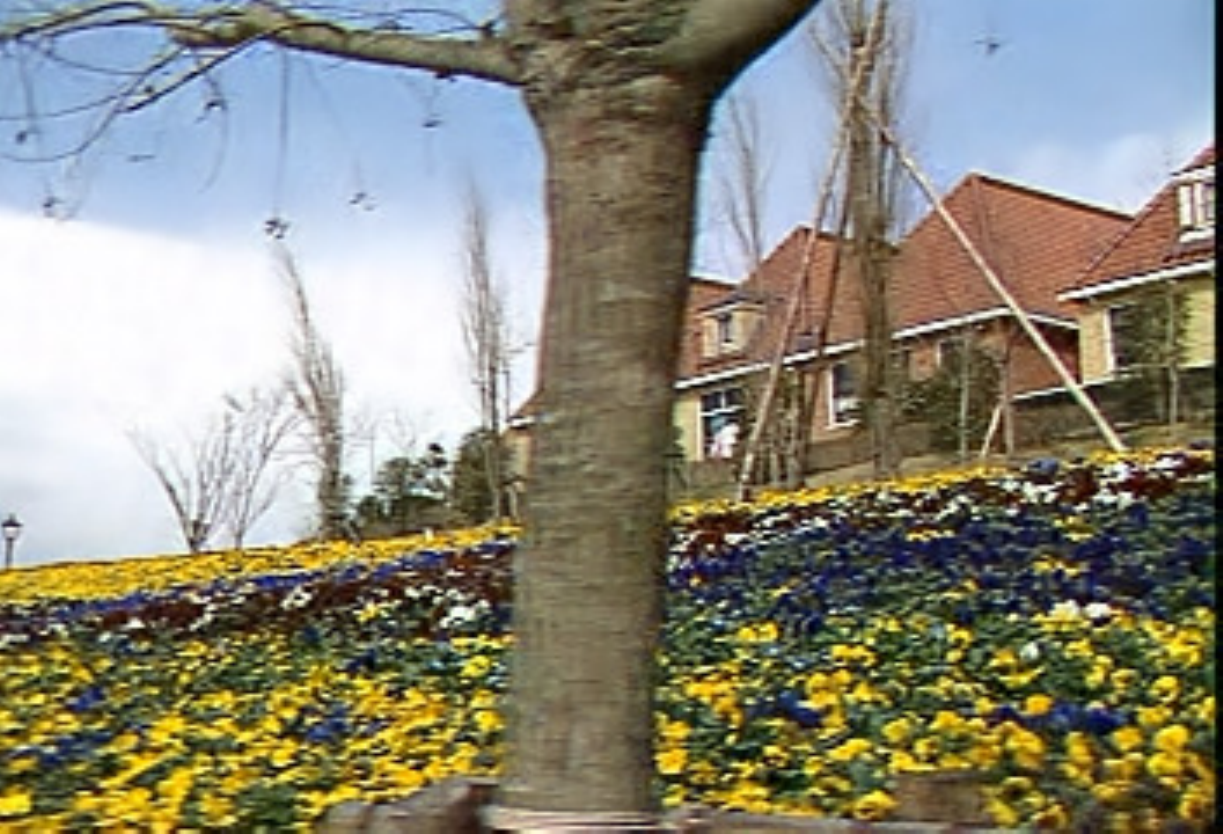}} \\
  \subfigure[{Zoom on (a)}]{\includegraphics[width=0.24\linewidth]{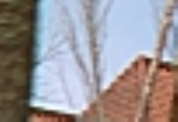}} \hfill
  \subfigure[{Zoom on (b)}]{\includegraphics[width=0.24\linewidth]{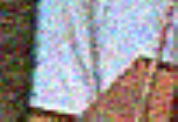}} \hfill
  \subfigure[{Zoom on (c)}]{\includegraphics[width=0.24\linewidth]{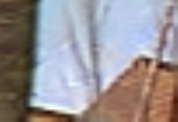}} \hfill
  \subfigure[{Zoom on (d)}]{\includegraphics[width=0.24\linewidth]{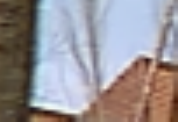}}
  \caption{Color video denoising result from~\citet{mairal2008b}. The third column shows the result when each frame is processed independently from the others. Last column shows the result of the video processing approach. Best seen by zooming on a computer screen. ``Copyright \copyright 2008 Society for Industrial and Applied Mathematics.  Reprinted with permission.  All rights reserved''.}
\label{fig:dict_videoA}
\end{figure}
\begin{figure}[hbtp]
  \subfigure{\includegraphics[width=0.24\linewidth]{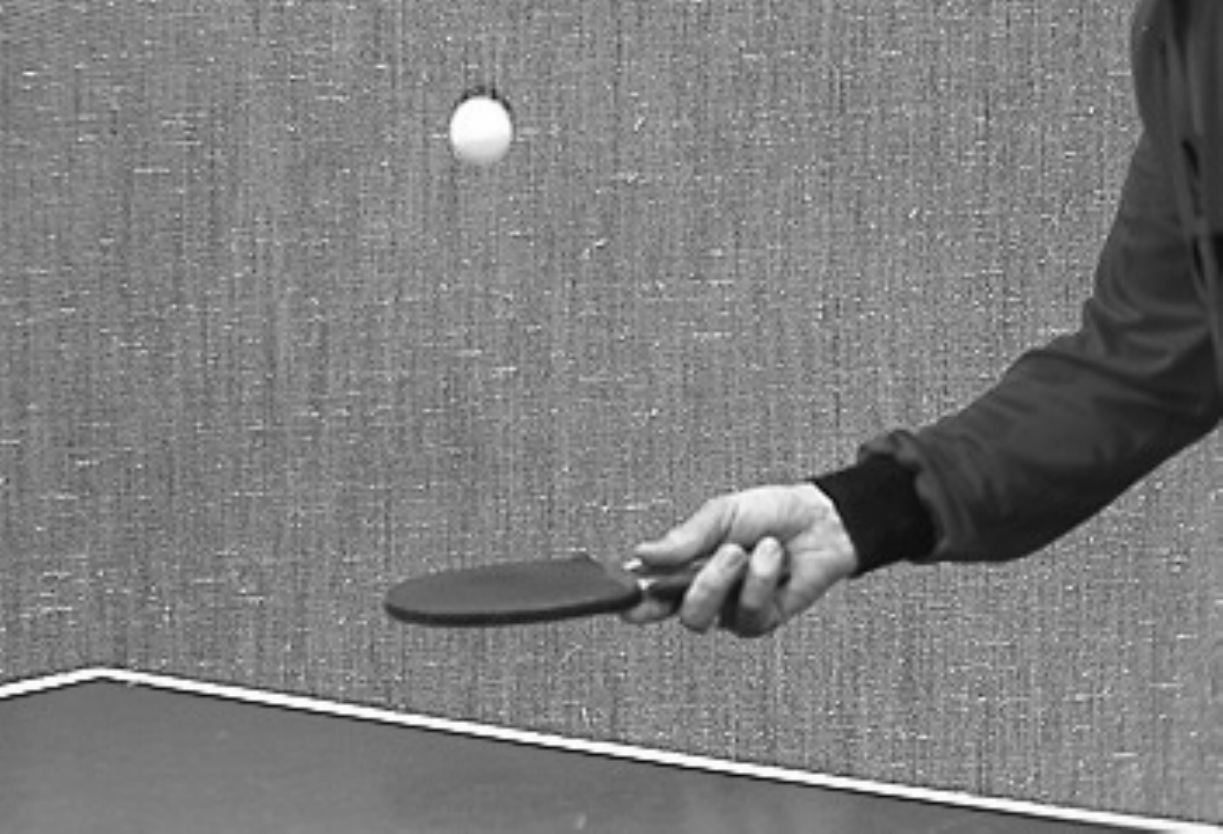}} \hfill
  \subfigure{\includegraphics[width=0.24\linewidth]{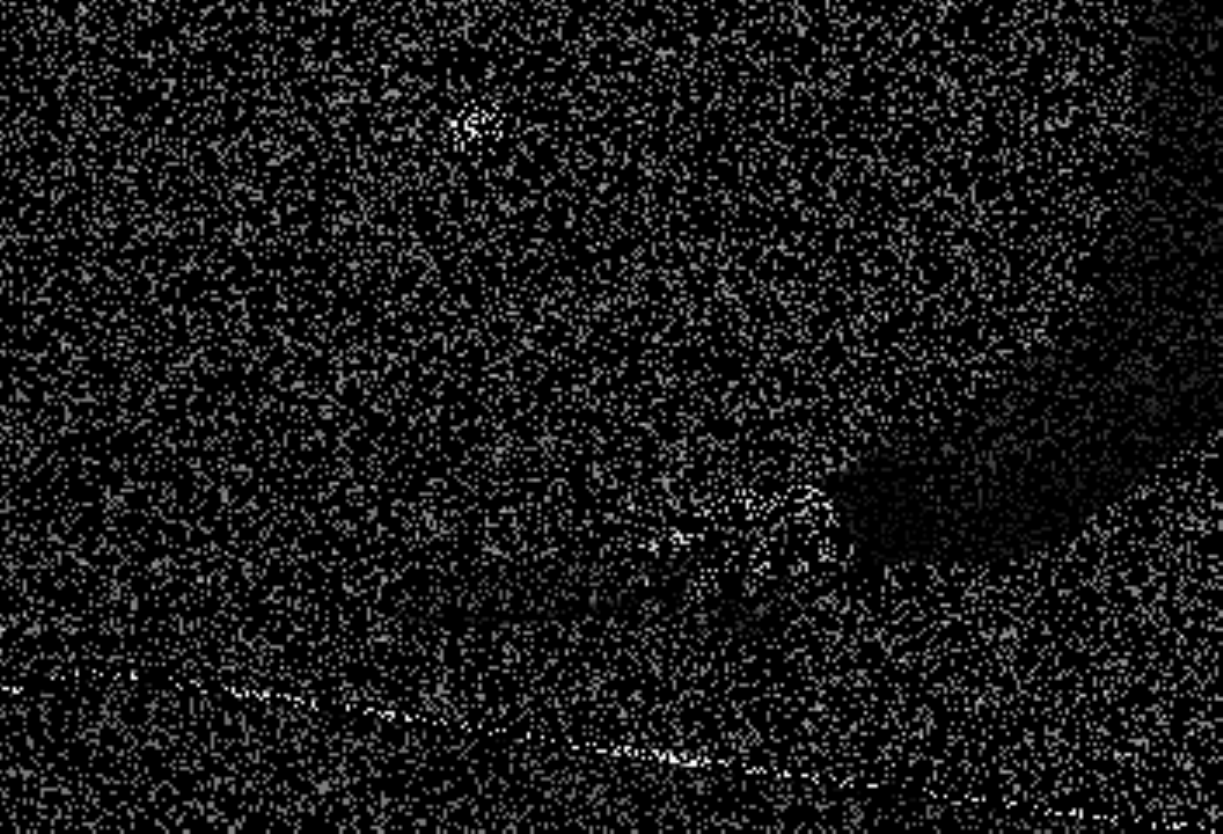}} \hfill
  \subfigure{\includegraphics[width=0.24\linewidth]{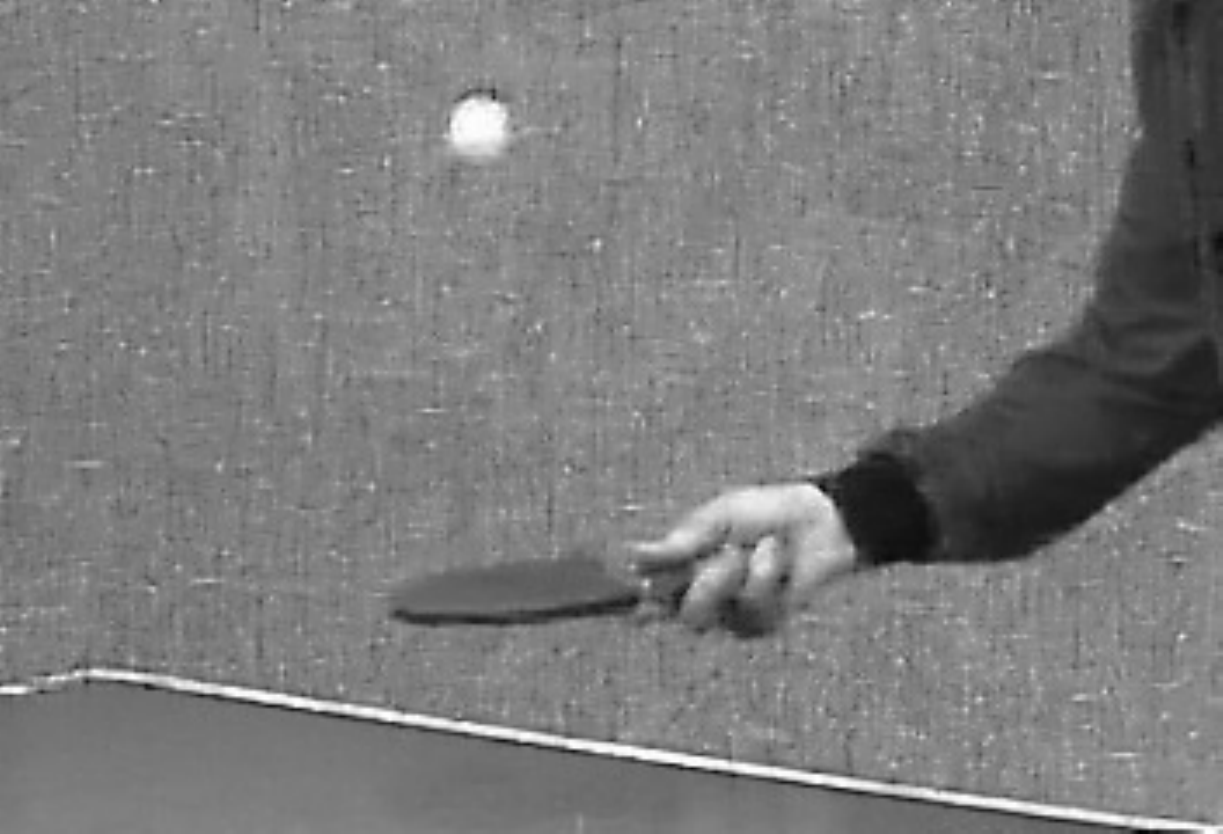}} \hfill
  \subfigure{\includegraphics[width=0.24\linewidth]{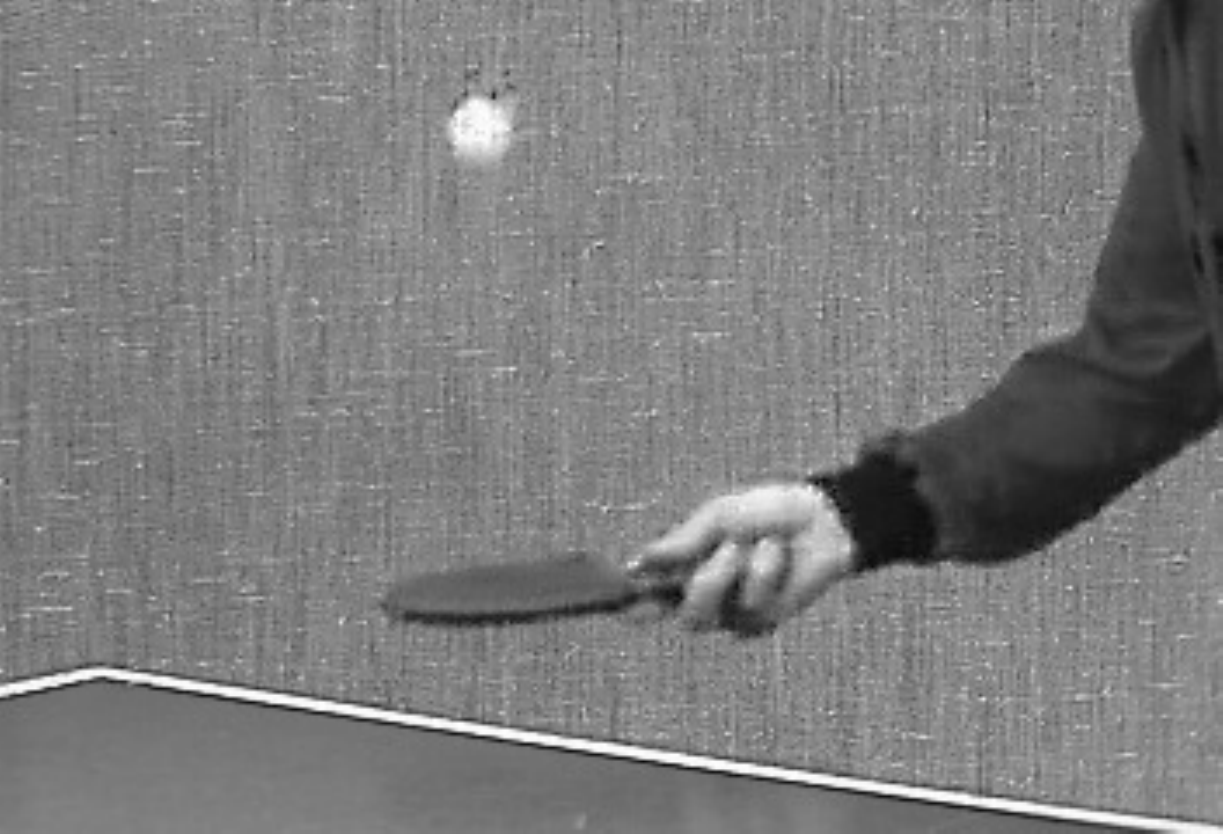}} \\
  \subfigure{\includegraphics[width=0.24\linewidth]{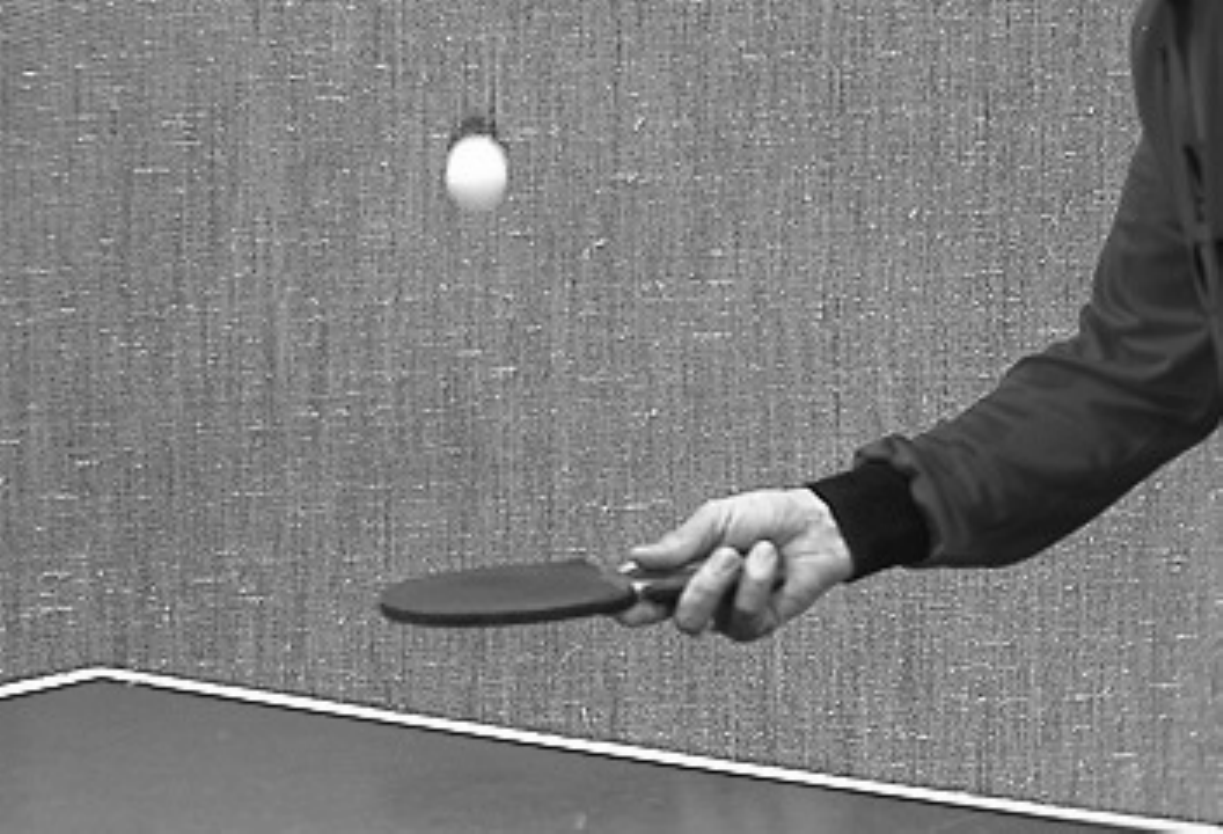}} \hfill
  \subfigure{\includegraphics[width=0.24\linewidth]{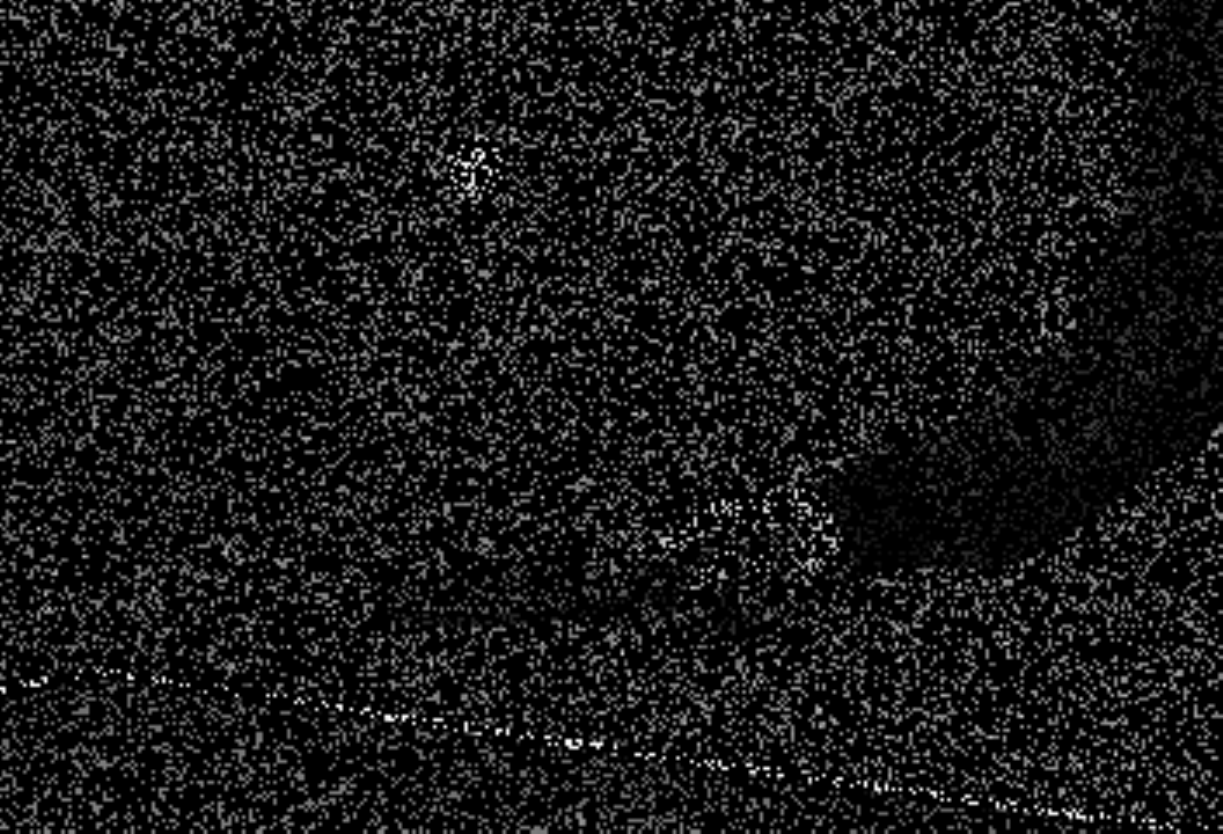}} \hfill
  \subfigure{\includegraphics[width=0.24\linewidth]{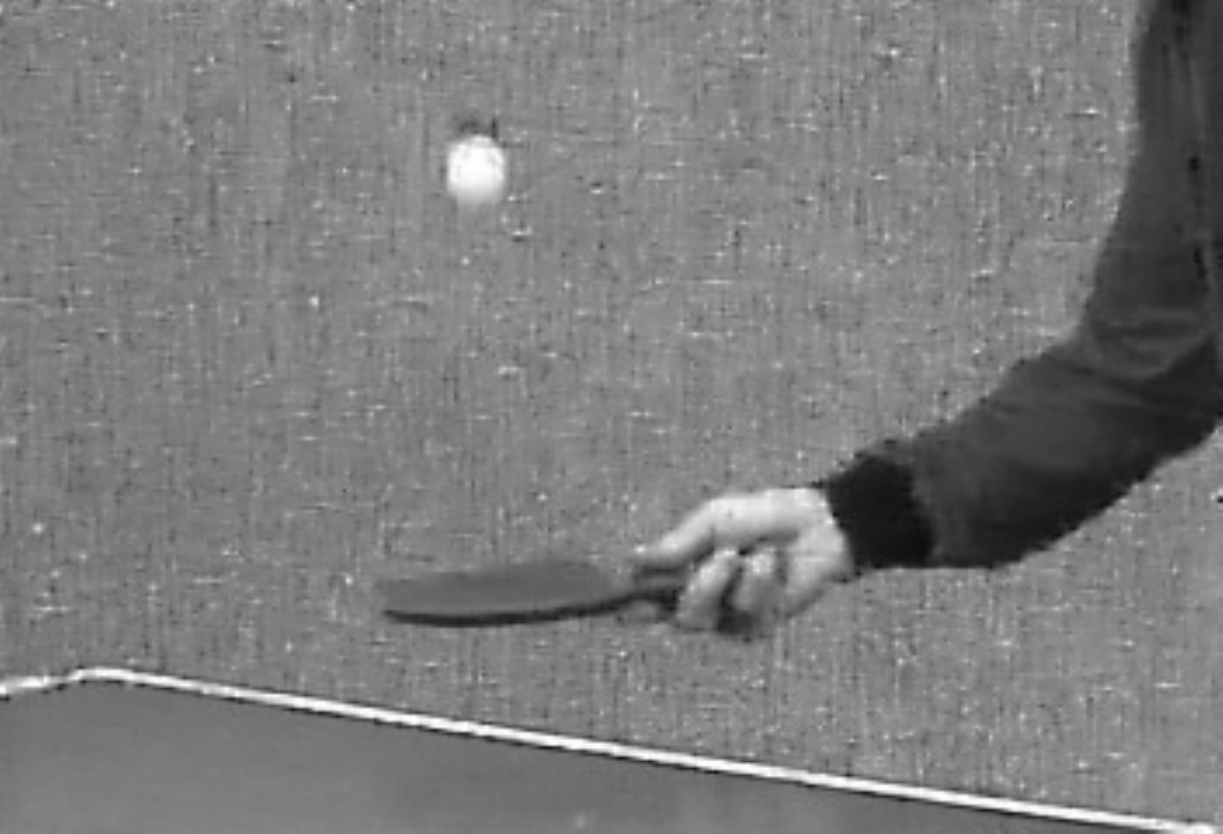}} \hfill
  \subfigure{\includegraphics[width=0.24\linewidth]{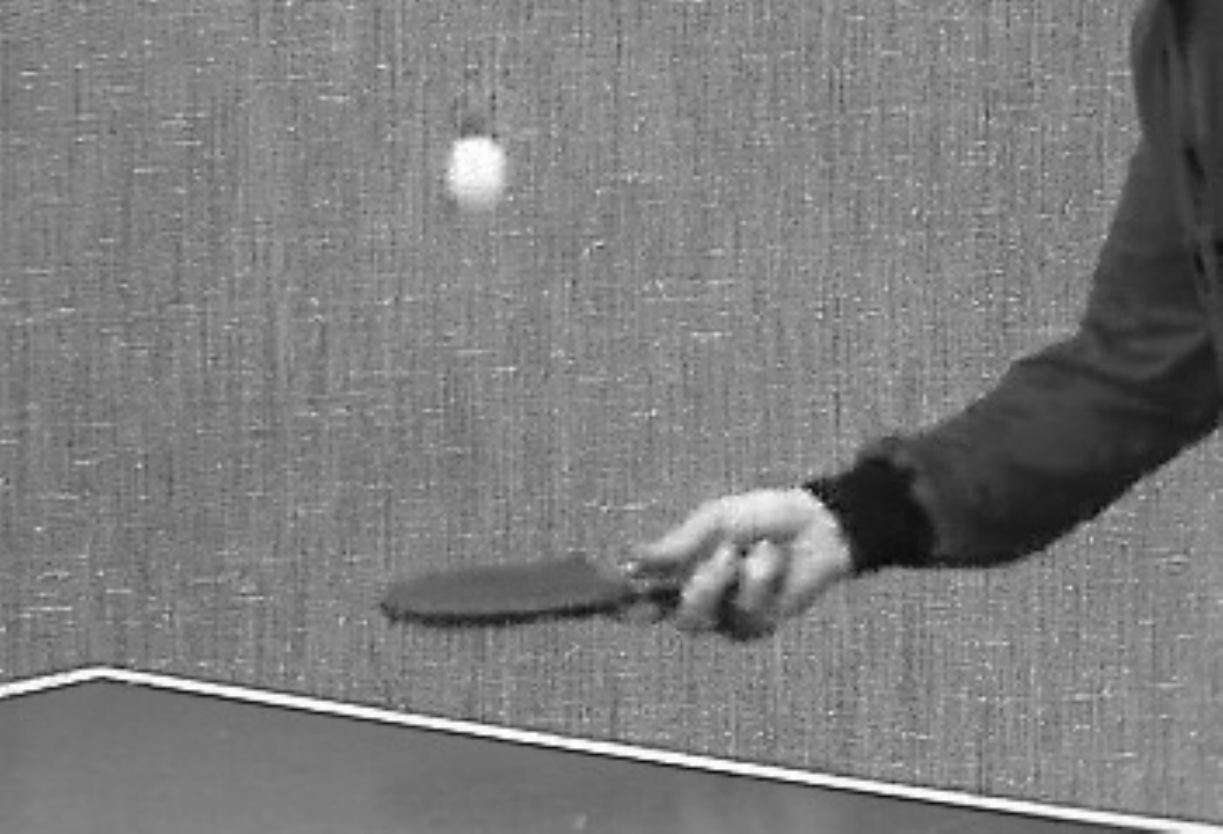}} \\
  \subfigure{\includegraphics[width=0.24\linewidth]{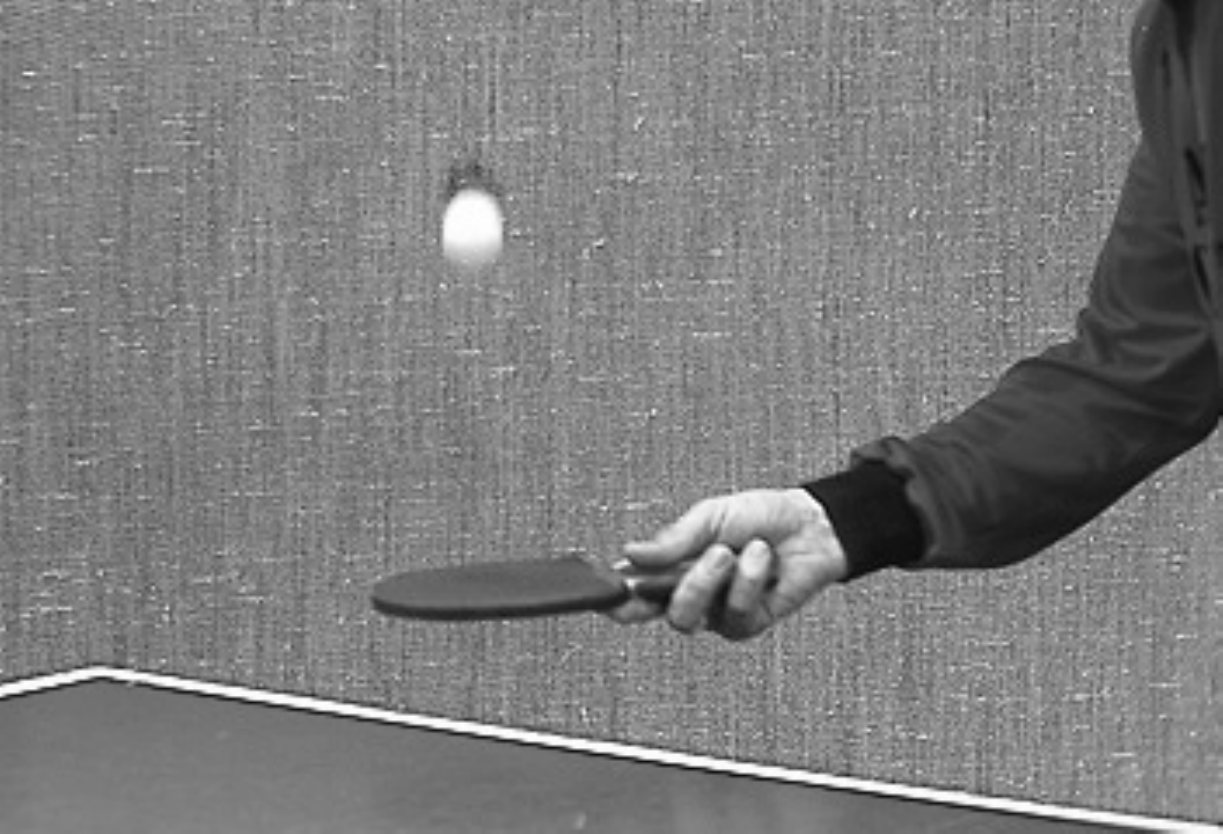}} \hfill
  \subfigure{\includegraphics[width=0.24\linewidth]{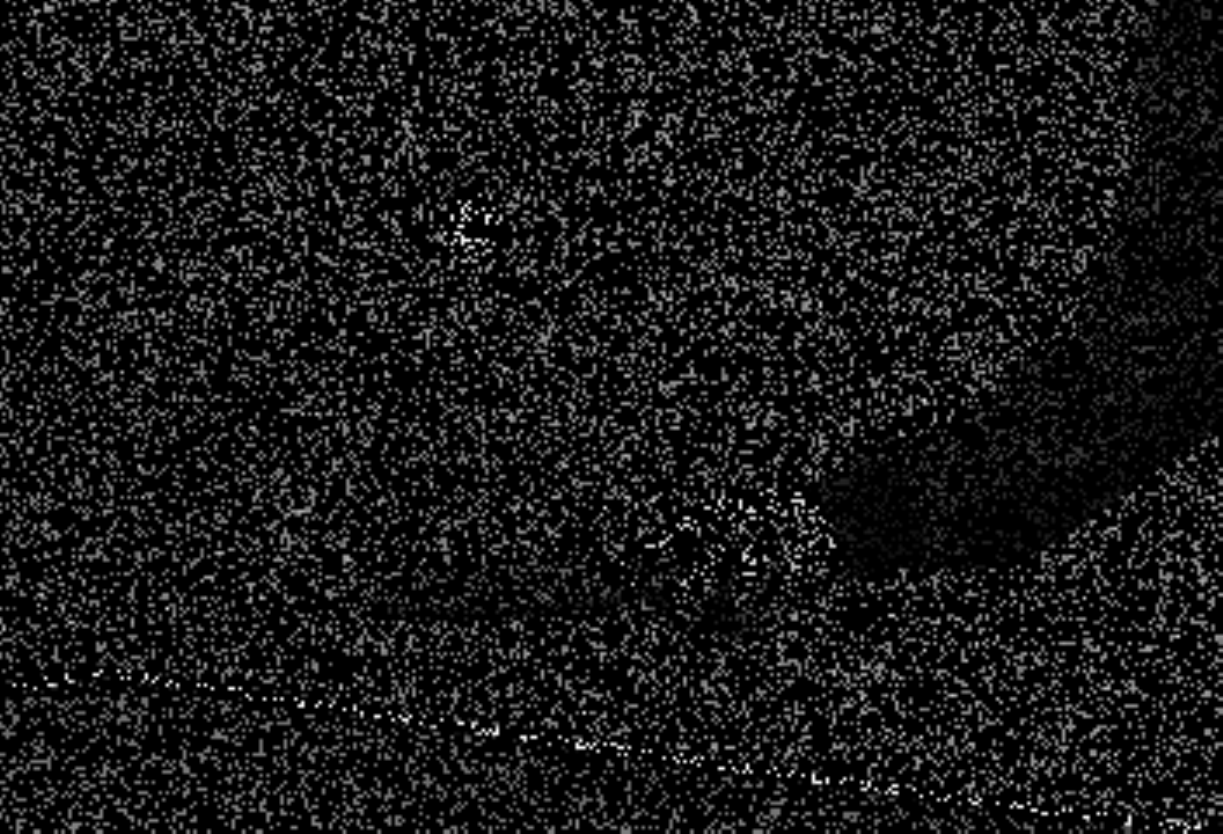}} \hfill
  \subfigure{\includegraphics[width=0.24\linewidth]{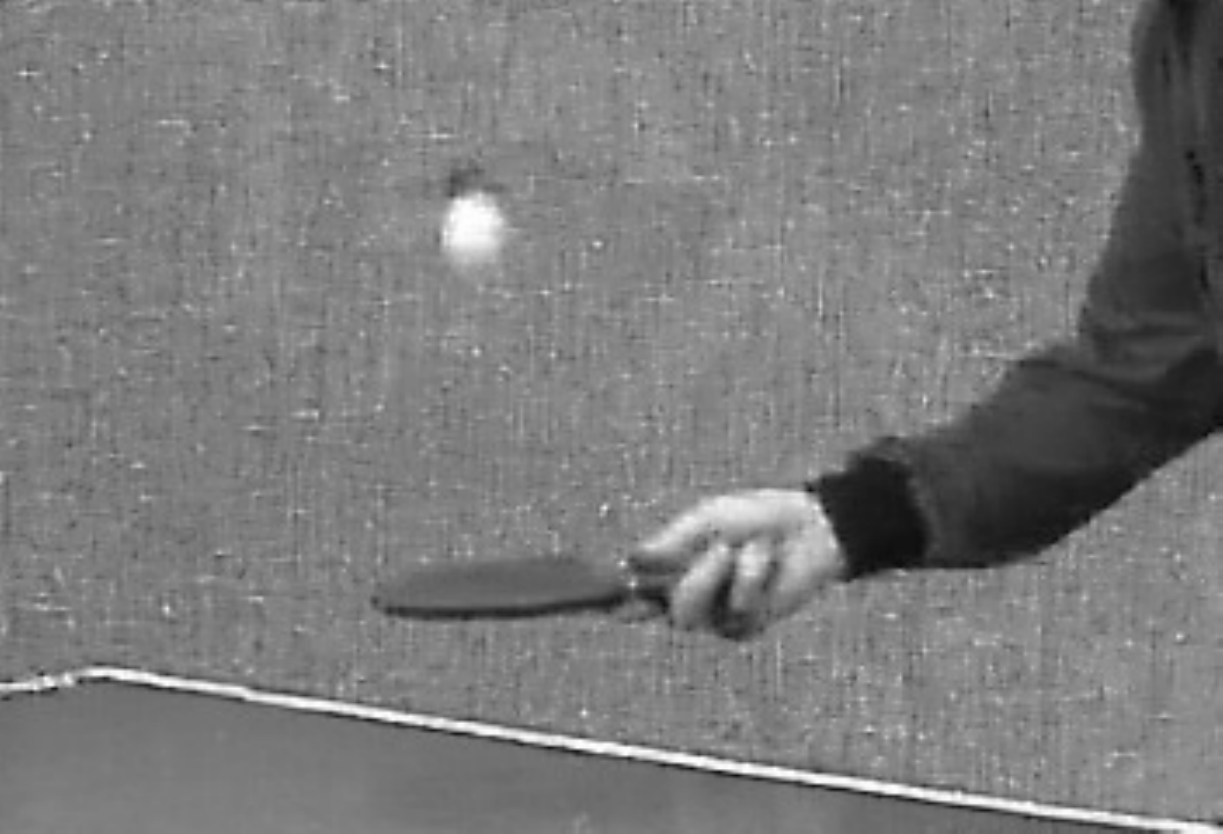}} \hfill
  \subfigure{\includegraphics[width=0.24\linewidth]{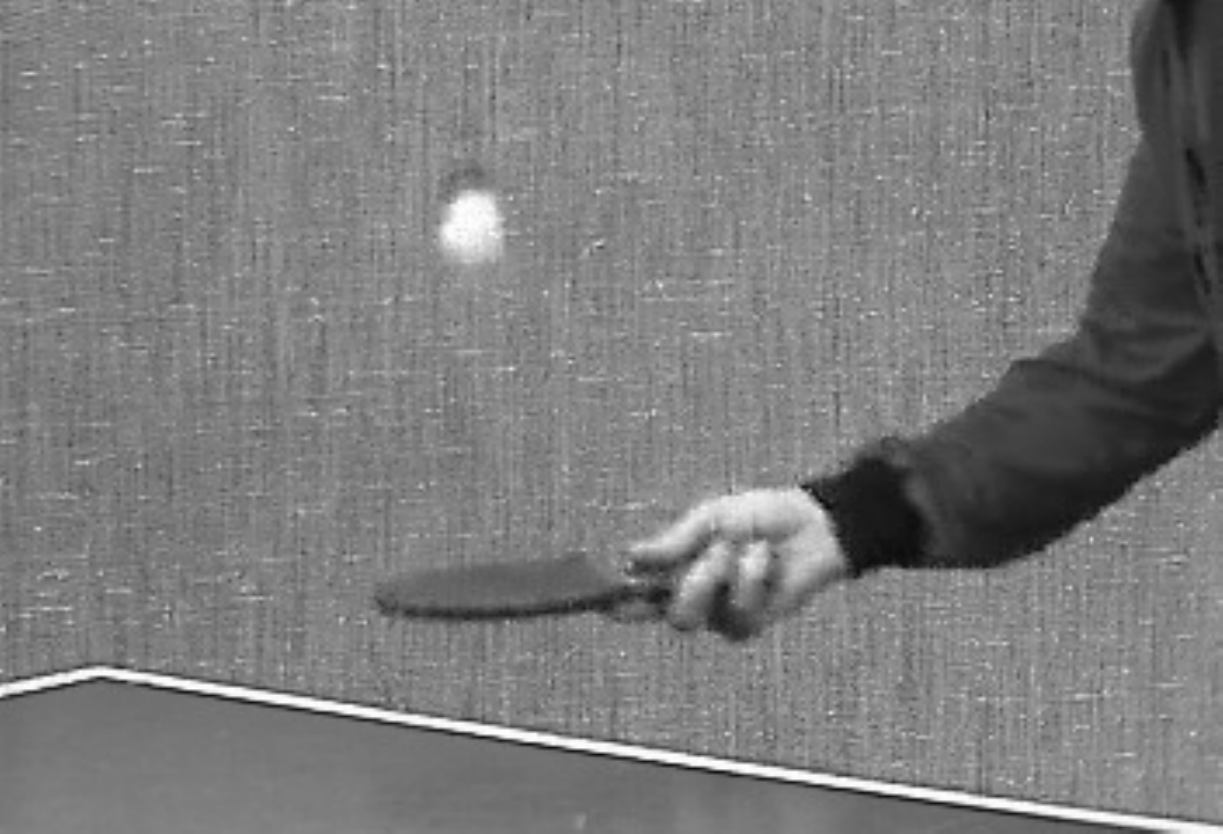}} \\
  \subfigure{\includegraphics[width=0.24\linewidth]{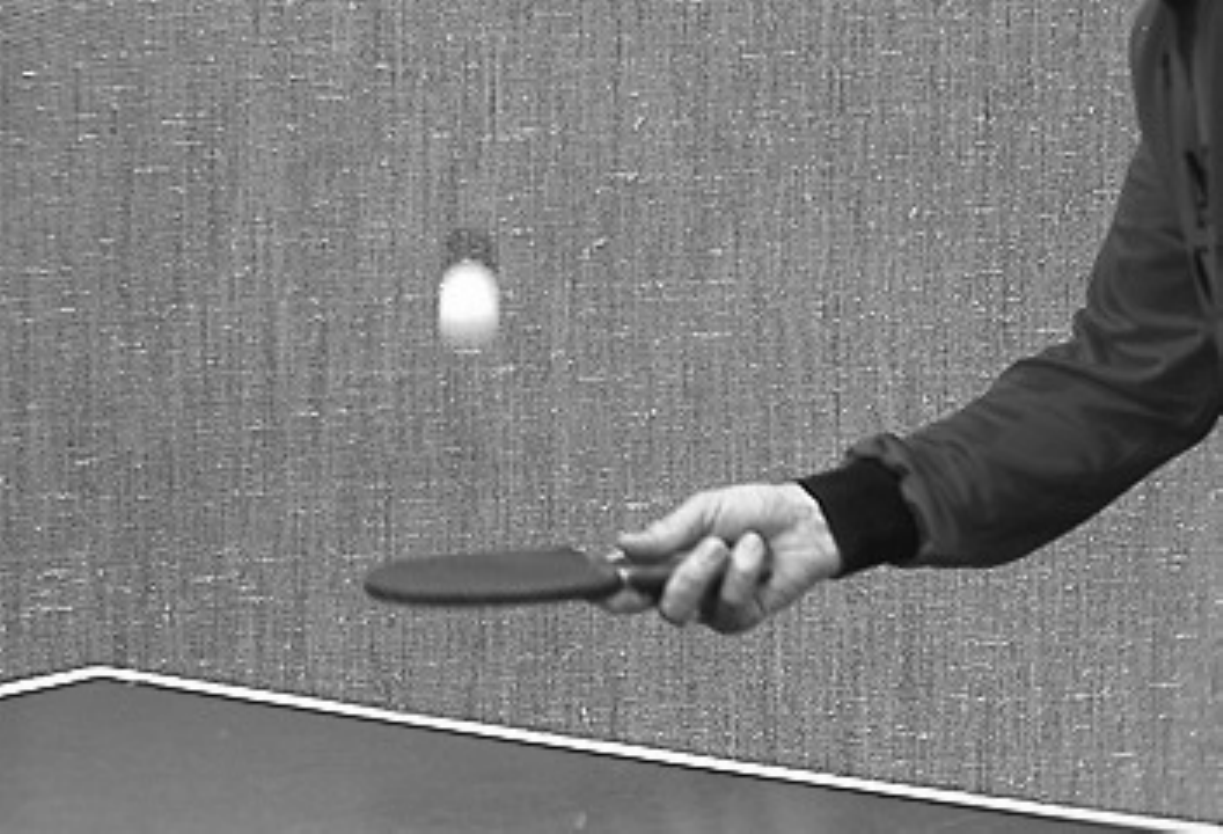}} \hfill
  \subfigure{\includegraphics[width=0.24\linewidth]{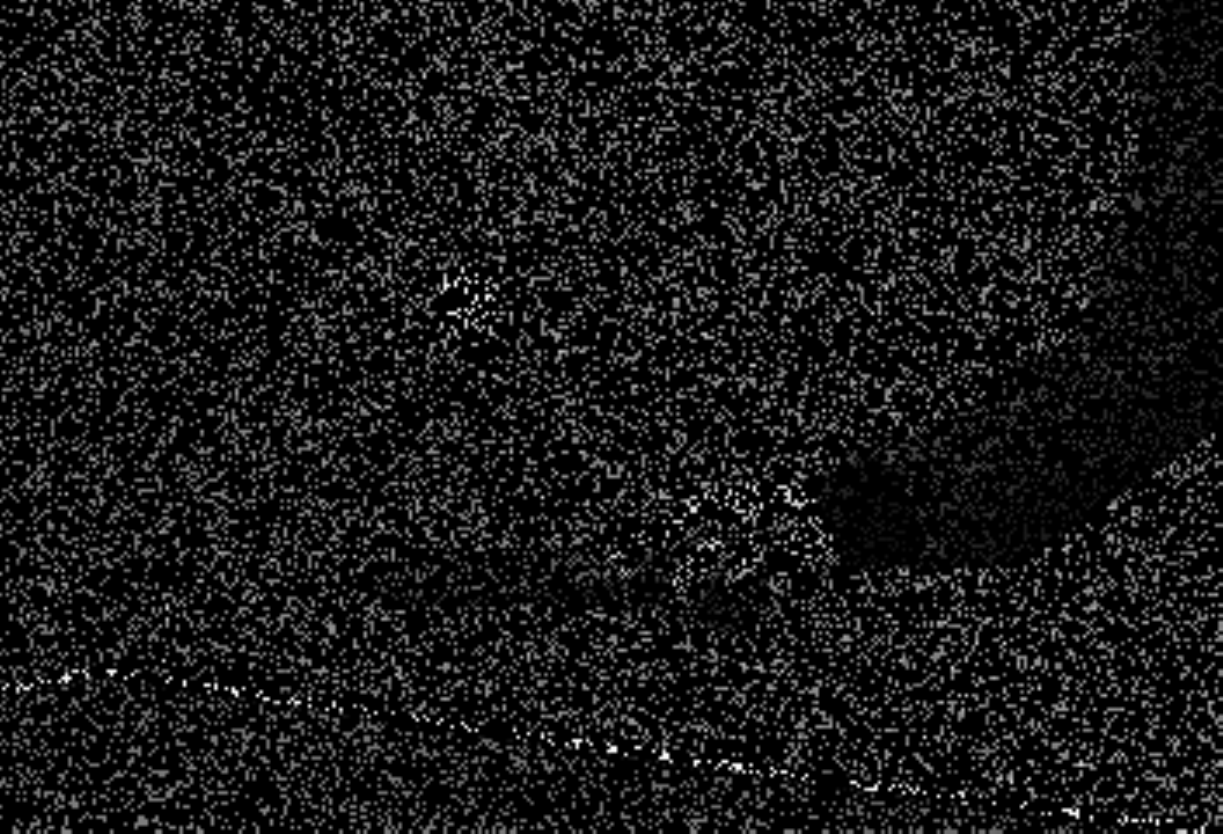}} \hfill
  \subfigure{\includegraphics[width=0.24\linewidth]{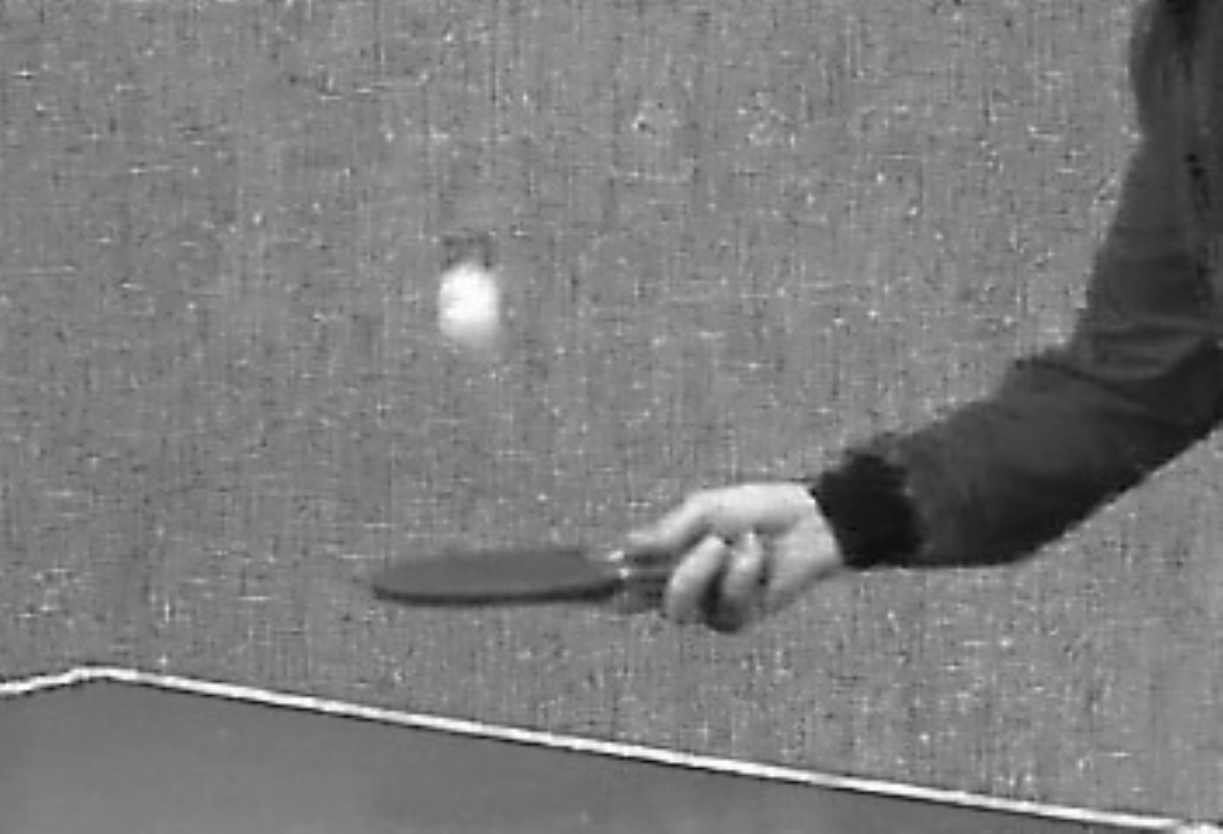}} \hfill
  \subfigure{\includegraphics[width=0.24\linewidth]{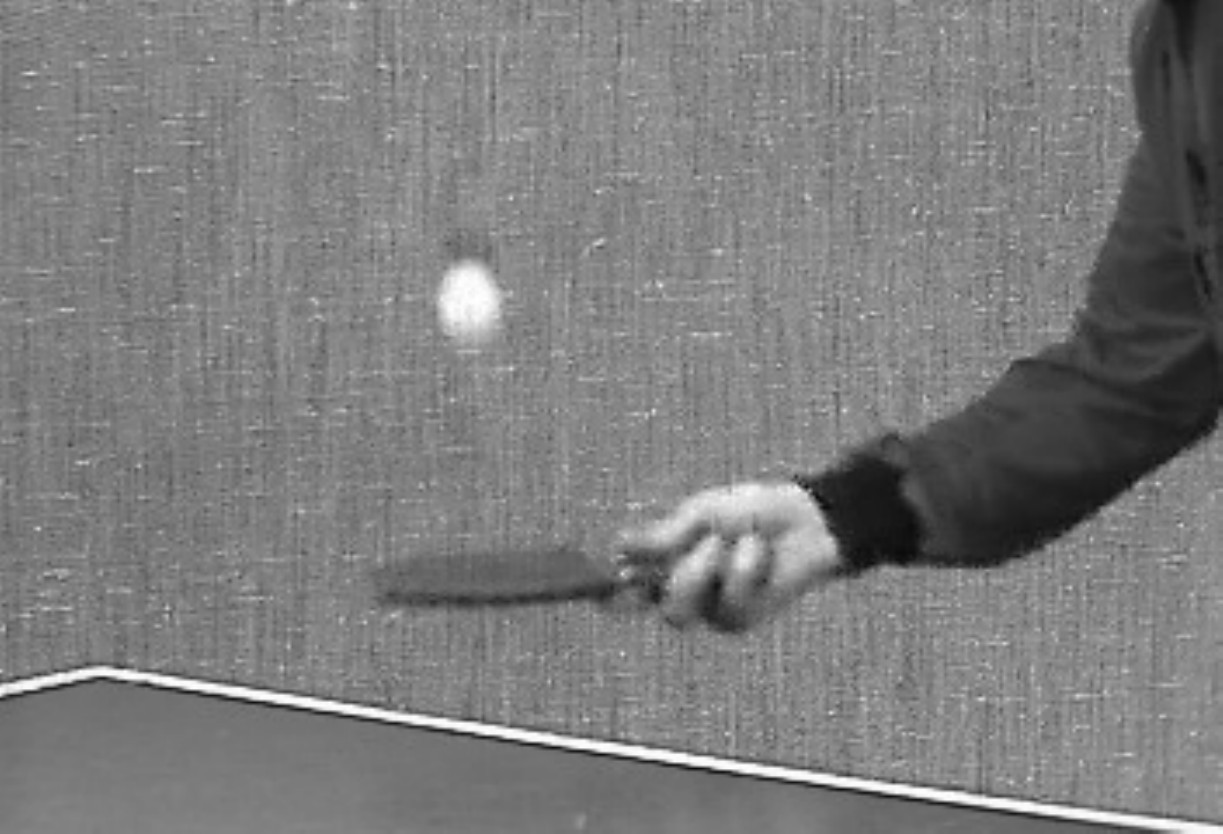}} \\
  \addtocounter{subfigure}{-16}
  \subfigure[{Original}]{\includegraphics[width=0.24\linewidth]{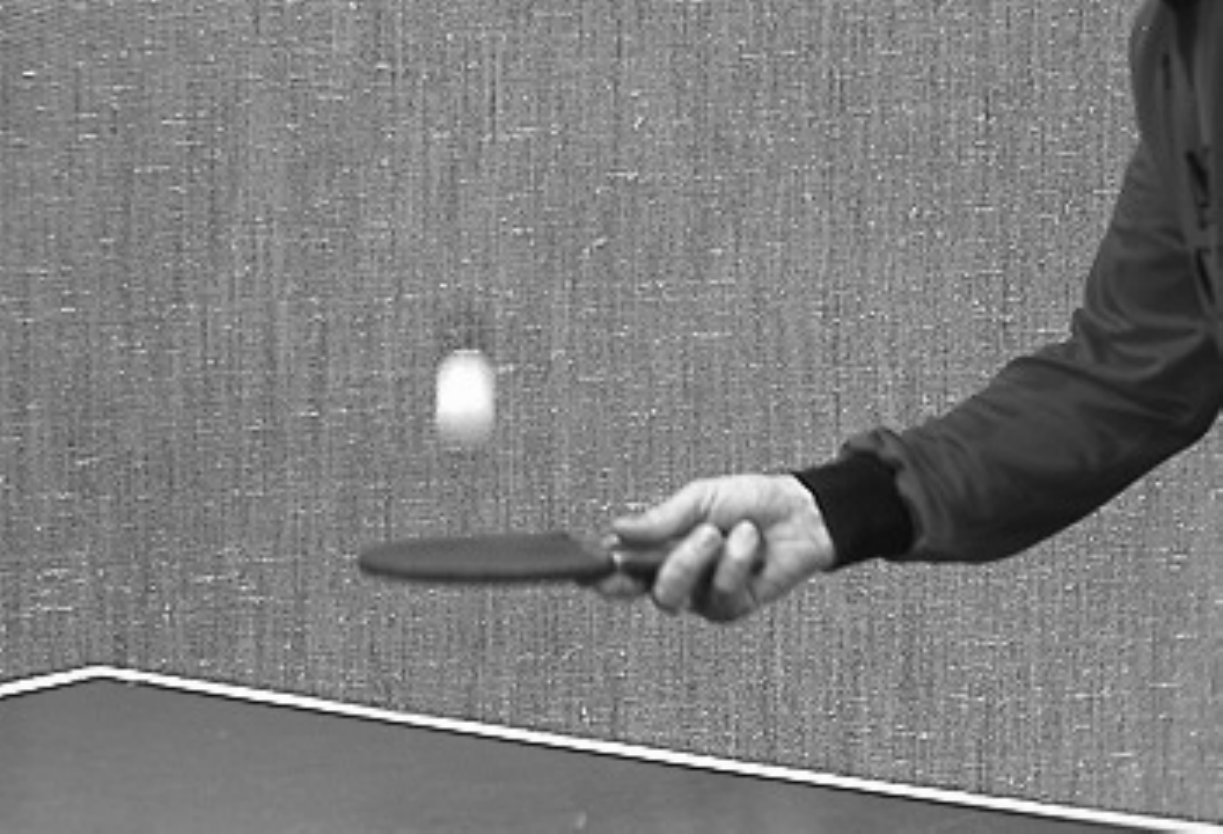}} \hfill
  \subfigure[{Damaged}]{\includegraphics[width=0.24\linewidth]{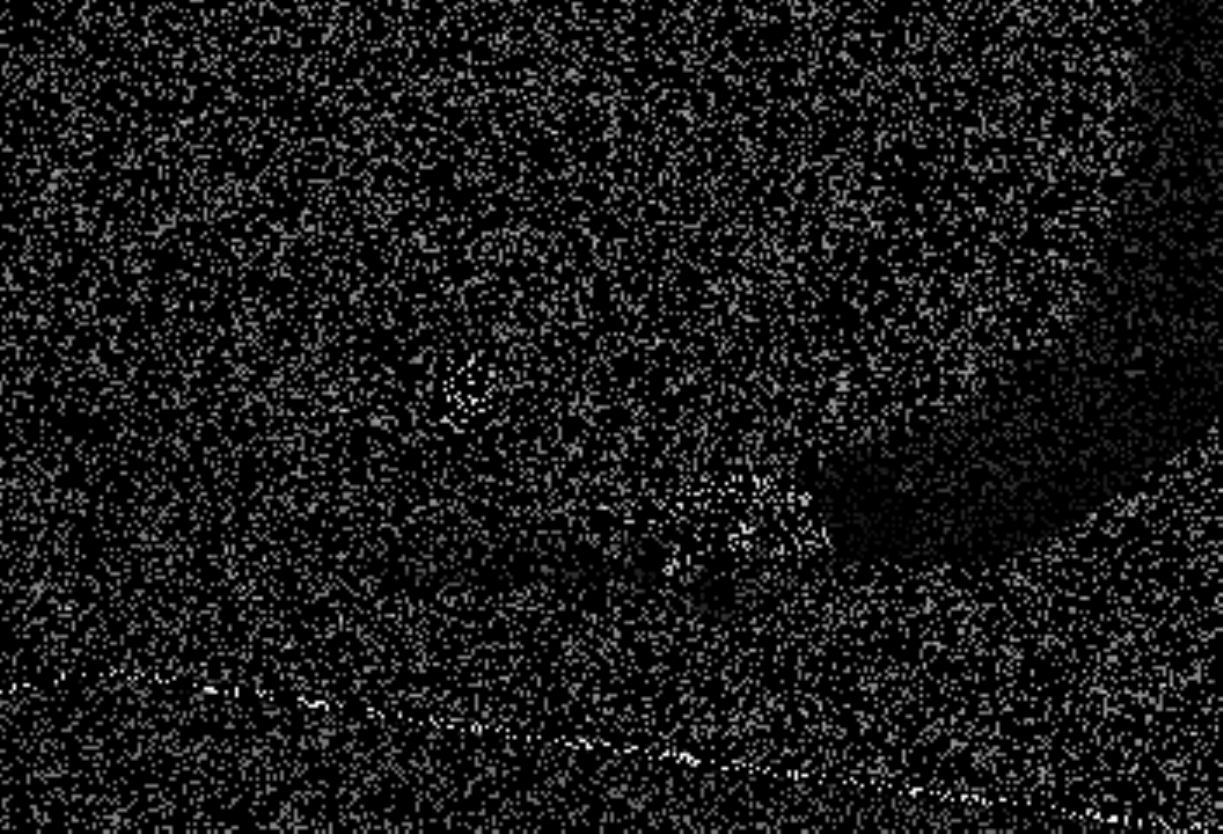}} \hfill
  \subfigure[{Image~inpaint.}]{\includegraphics[width=0.24\linewidth]{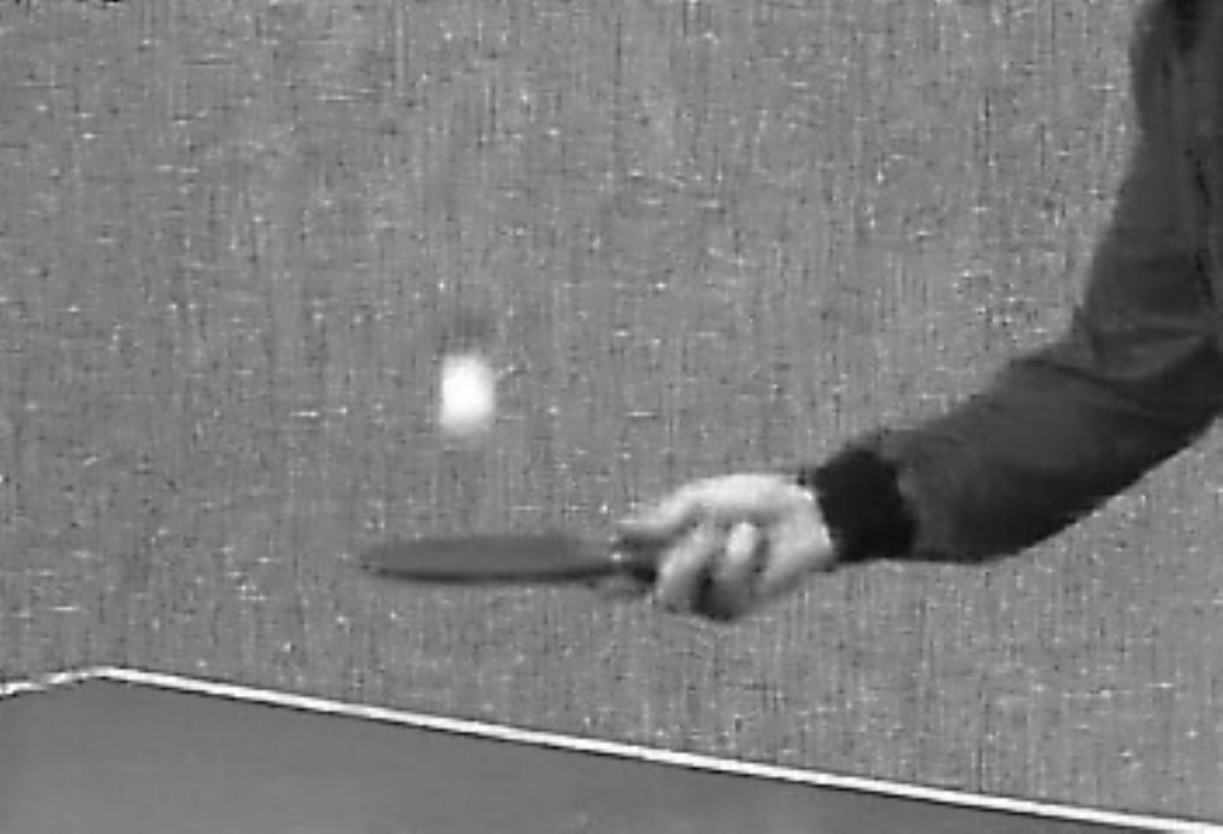}} \hfill
  \subfigure[{Video~inpaint.}]{\includegraphics[width=0.24\linewidth]{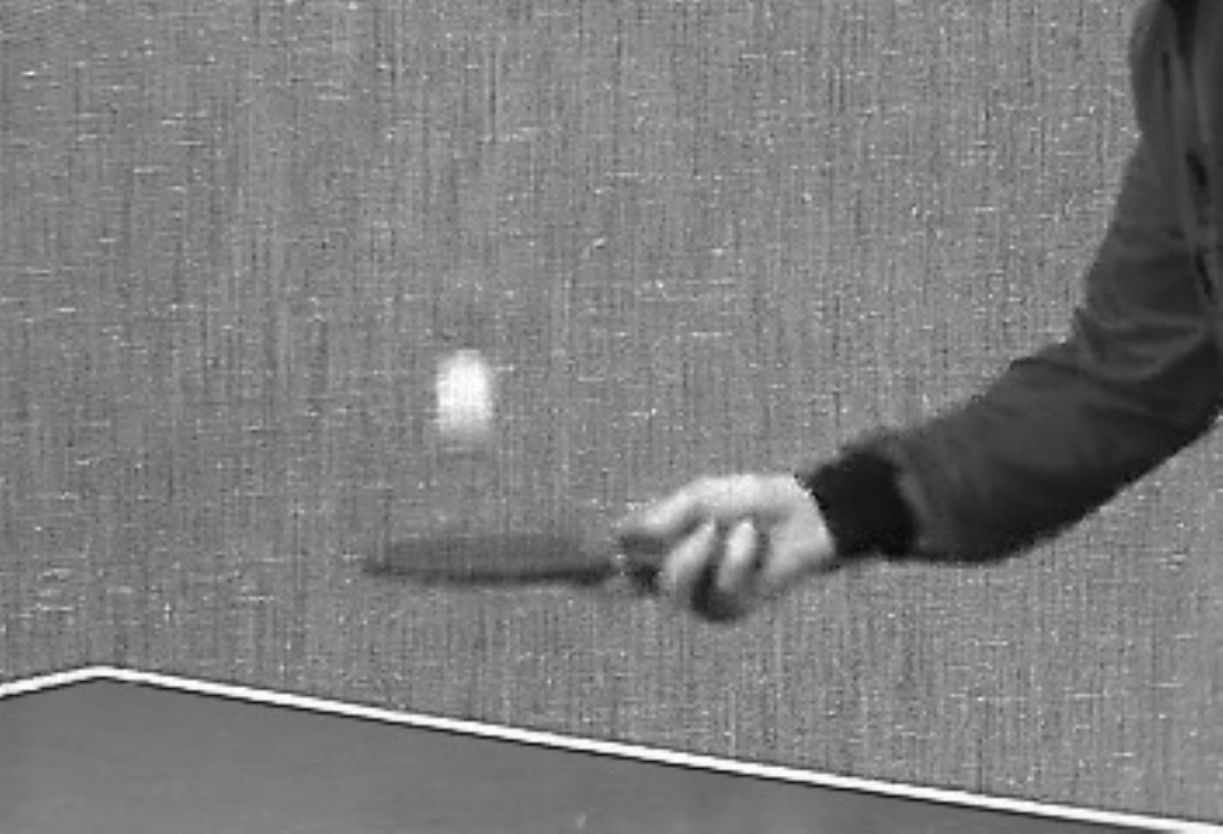}} \\
  \subfigure[{Zoom on (a)}]{\includegraphics[width=0.24\linewidth]{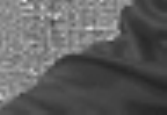}} \hfill
  \subfigure[{Zoom on (b)}]{\includegraphics[width=0.24\linewidth]{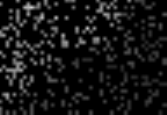}} \hfill
  \subfigure[{Zoom on (c)}]{\includegraphics[width=0.24\linewidth]{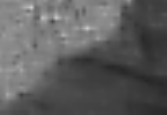}} \hfill
  \subfigure[{Zoom on (d)}]{\includegraphics[width=0.24\linewidth]{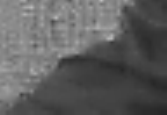}}
  \caption{Video inpainting result from~\citet{mairal2008b}. The third column shows the result when each frame is processed independently from the others. Last column shows the result of the video processing approach. Best seen by zooming on a computer screen. ``Copyright \copyright 2008 Society for Industrial and Applied Mathematics.  Reprinted with permission.  All rights reserved''.}\label{fig:dict_videoB}
\end{figure}

%% file: content_arxiv/image_face.tex
We have shown so far that dictionary learning was appropriate for many
\emph{restoration} tasks, but an intuitive use of sparsity is for
\emph{compression}. Following this insight, \citet{bryt2008} have proposed 
a technique for encoding photographs of human faces. They achieve
impressive compression results, where each image is represented with less
than~$1\,000$ bytes while suffering from minor loss only in visual quality. The algorithm
requires a database of training face images that are processed as follows:
\begin{enumerate}
   \item input images are geometrically aligned and downsampled to the same size.
      As a result, all the main face features are at the same
      position across images. This step is automatically performed by first detecting
      $13$ face features and warping the images to achieve the desired alignment~\citep[see][]{bryt2008}; 
   \item the images are sliced into~$15 \times 15$ \emph{non-overlapping} patches;
   \item \emph{for each patch position}, a dictionary of size~$p=512$ elements
      is learned by using all patches available at the same position from the
      training images. The K-SVD algorithm is chosen by~\citet{bryt2008} for
      that purpose.
\end{enumerate}
After this training phase, a new test image can be encoded as follows:
\begin{enumerate}
   \item the test image goes through the same geometrical alignment process as
      for the training images. Encoding the inverse transformation requires
      $20$ bytes;
   \item the image is sliced into non-overlapping patches that encoded using orthogonal matching pursuit (see Section~\ref{sec:optiml0}) with a few number~$s$ of non-zero coefficients (typically $s=4$);
   \item the indices of the sparse coefficients are encoded with a Huffman table~\citep[see][]{mackay2003}; the weights are quantized into~$7$ bits.
\end{enumerate}
Once encoded, each patch can be decoded with a simple matrix vector
multiplication, and the inverse geometrical warping is applied to obtain an
approximation of the original image.  Visual results are presented in
Figure~\ref{fig:faces} for face images represented by $550$ bytes only. Unlike
generic approaches such as JPEG and JPEG2000, the approach of~\citet{bryt2008}
can exploit the fact that each patch represent a specific information, \eg,
eye, mouth, which can be learned by using adaptive dictionaries.
Since the patches do not overlap, the reconstructed images suffer from 
blocking artifacts, which can be significantly reduced by using a deblocking
post-processing algorithm.

\def\bblock{\hspace*{0.7cm}}
\begin{figure}[hbtp]
   \subfigure[Original]{\includegraphics[width=0.19\linewidth]{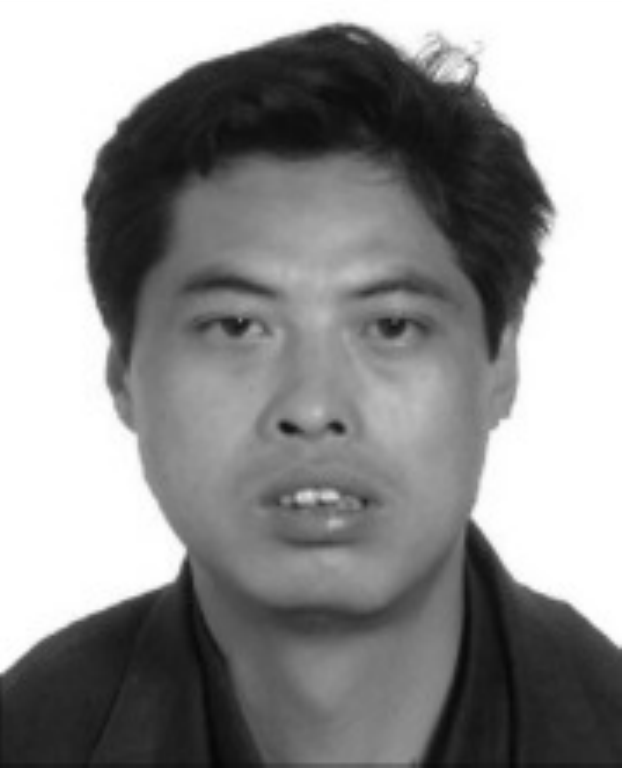}} \hfill
   \subfigure[JPEG \newline \bblock(14.15)]{\includegraphics[width=0.19\linewidth]{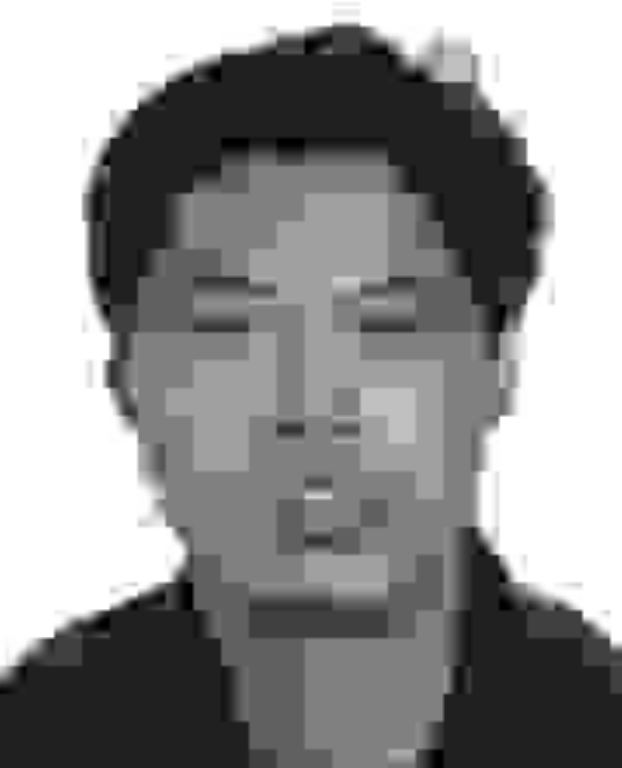}} \hfill
   \subfigure[JPEG2000 \newline \bblock(12.15)]{\includegraphics[width=0.19\linewidth]{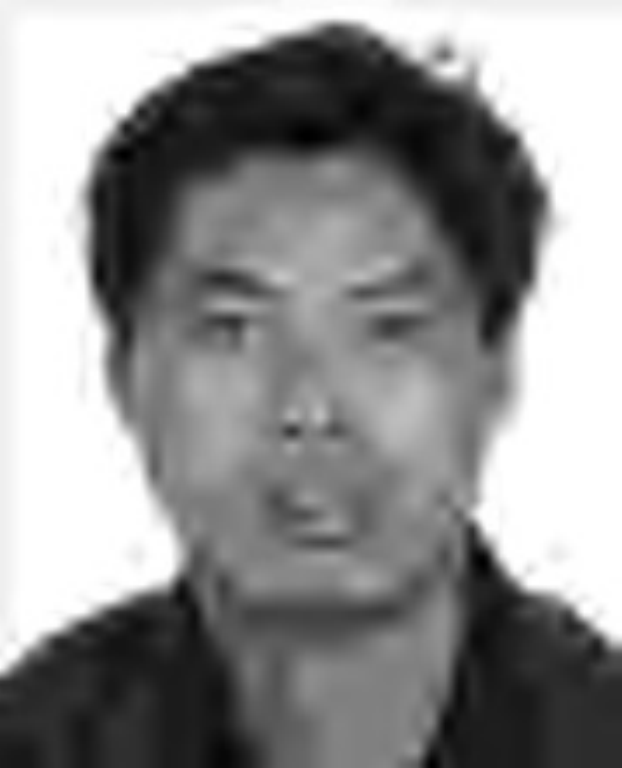}} \hfill
   \subfigure[KSVD \newline \bblock(6.89)]{\includegraphics[width=0.19\linewidth]{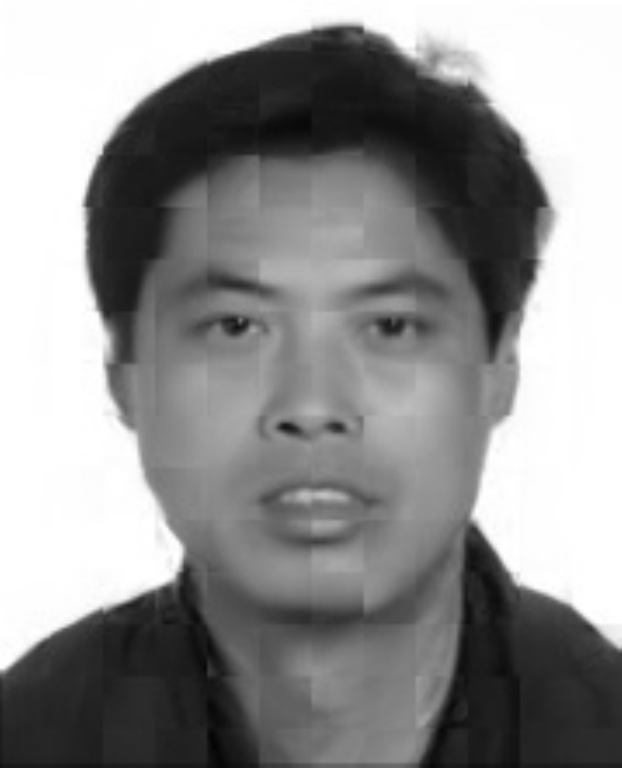}} \hfill
   \subfigure[KSVD+ \newline \bblock(6.55)]{\includegraphics[width=0.19\linewidth]{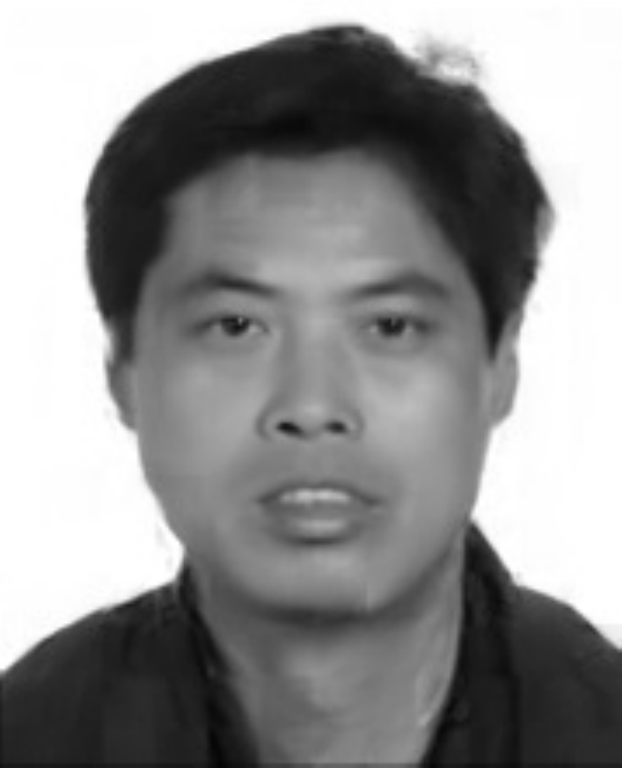}} \\
   \subfigure[Original]{\includegraphics[width=0.19\linewidth]{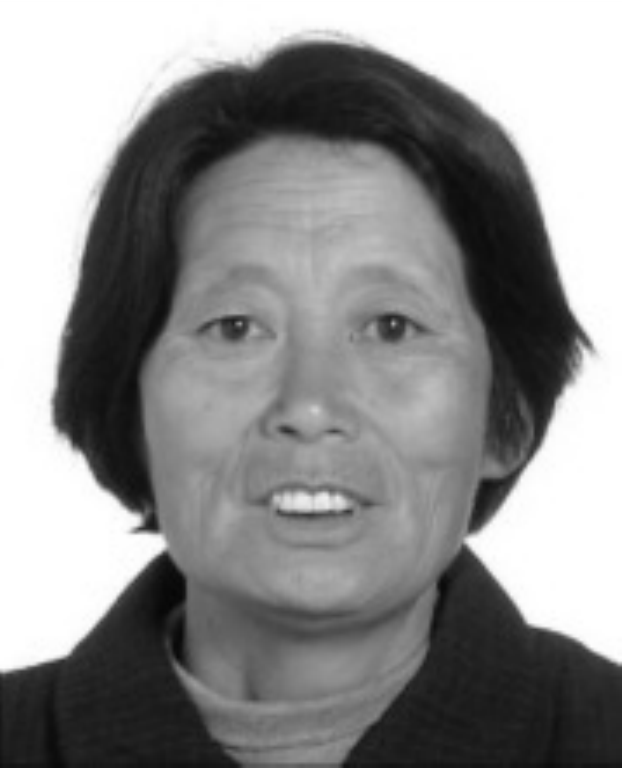}} \hfill
   \subfigure[JPEG \newline \bblock(15.17)]{\includegraphics[width=0.19\linewidth]{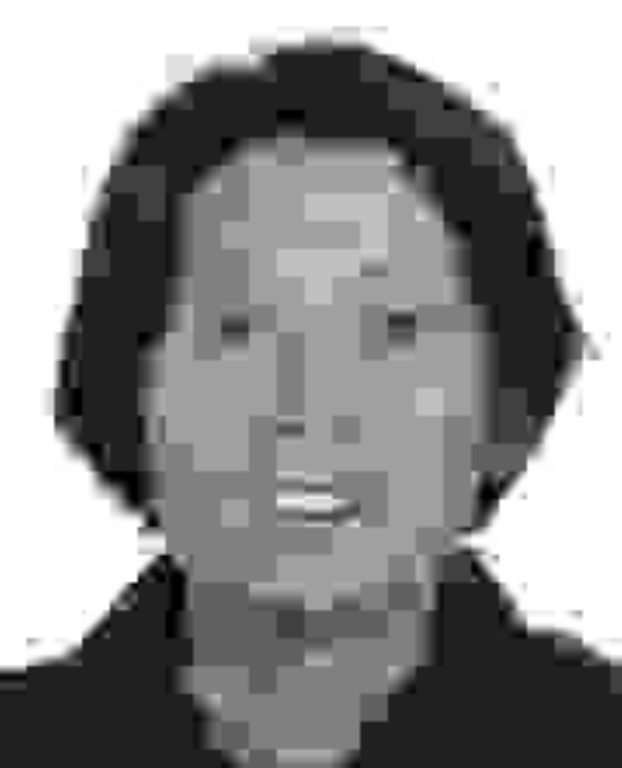}} \hfill
   \subfigure[JPEG2000 \newline \bblock(13.53)]{\includegraphics[width=0.19\linewidth]{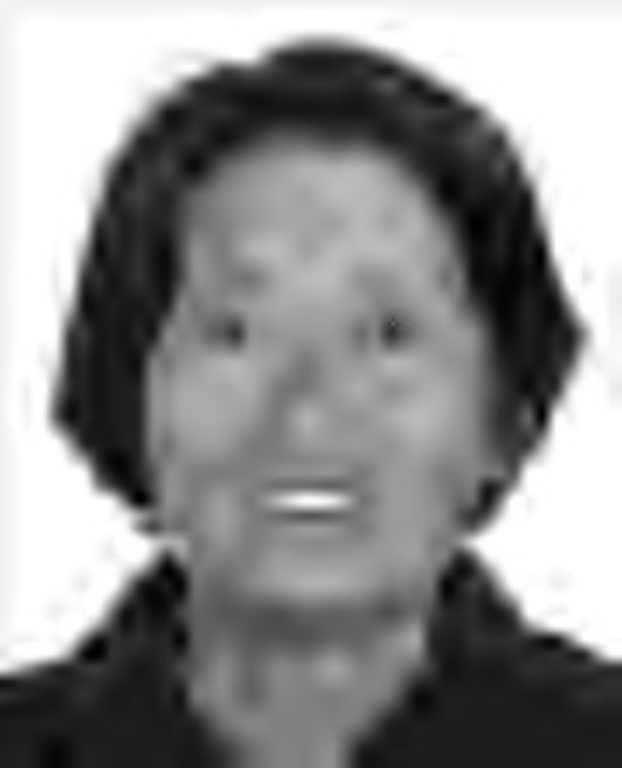}} \hfill
   \subfigure[KSVD \newline \bblock(8.23)]{\includegraphics[width=0.19\linewidth]{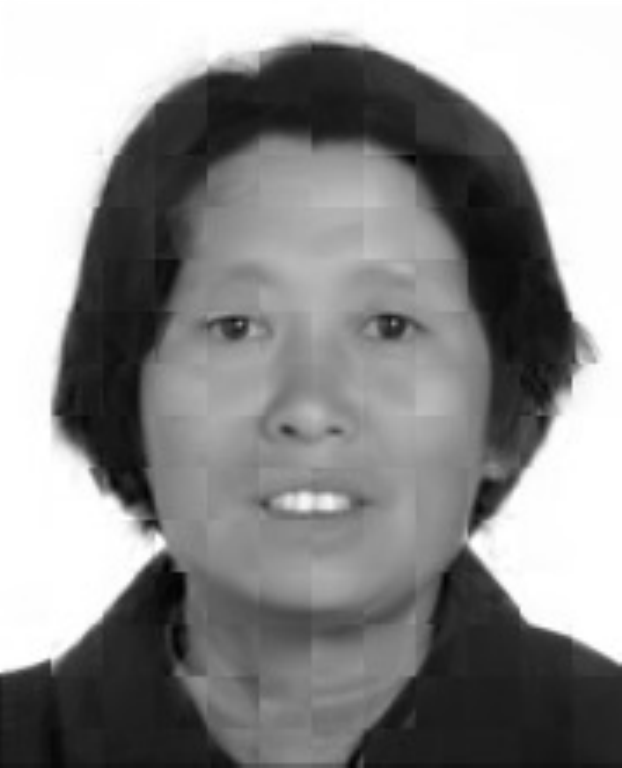}} \hfill
   \subfigure[KSVD+ \newline \bblock(7.83)]{\includegraphics[width=0.19\linewidth]{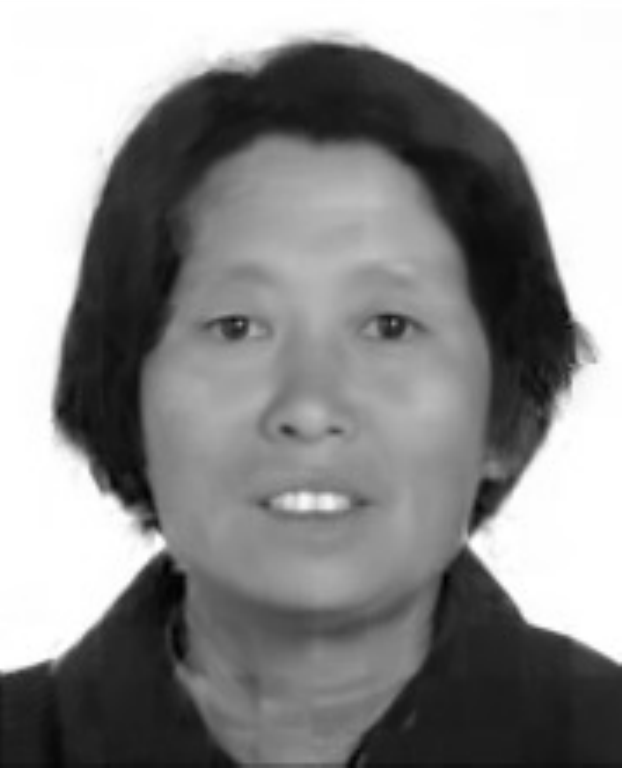}} \\
   \subfigure[Original]{\includegraphics[width=0.19\linewidth]{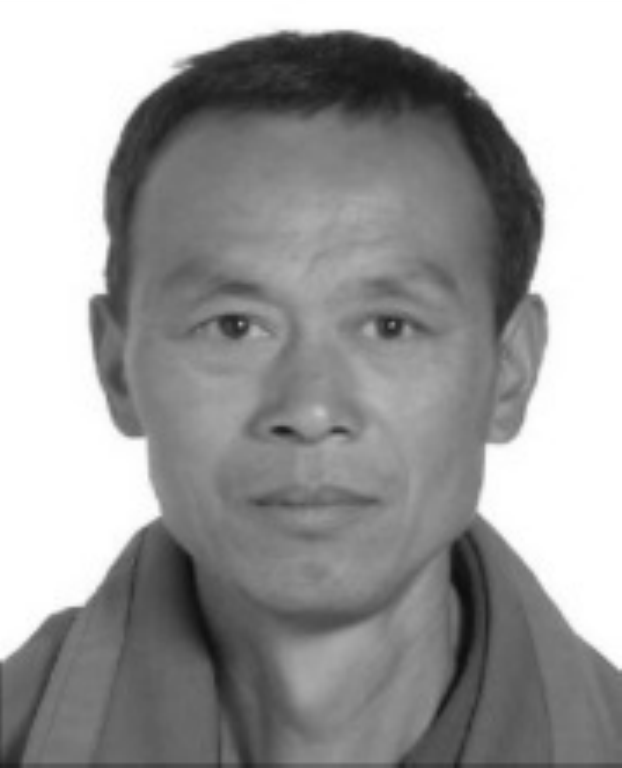}} \hfill
   \subfigure[JPEG \newline \bblock(14.54)]{\includegraphics[width=0.19\linewidth]{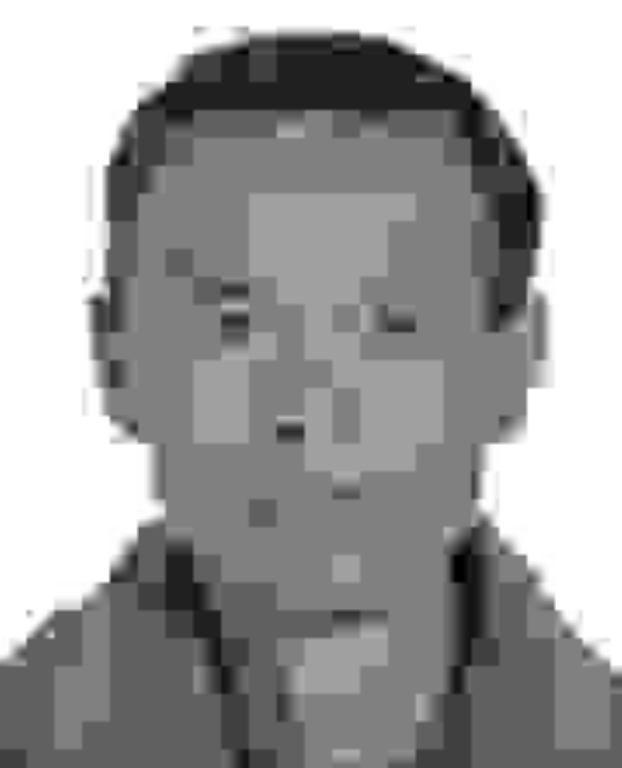}} \hfill
   \subfigure[JPEG2000 \newline \bblock(11.42)]{\includegraphics[width=0.19\linewidth]{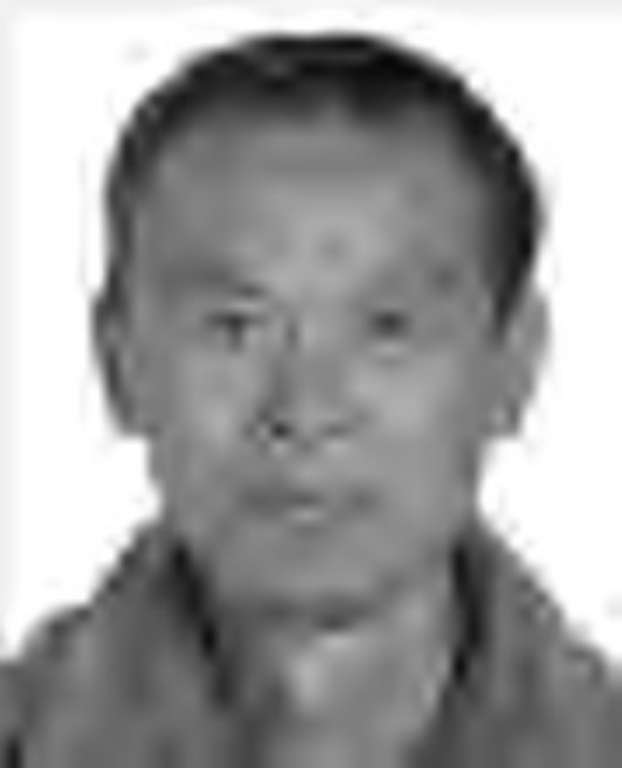}} \hfill
   \subfigure[KSVD \newline \bblock(6.69)]{\includegraphics[width=0.19\linewidth]{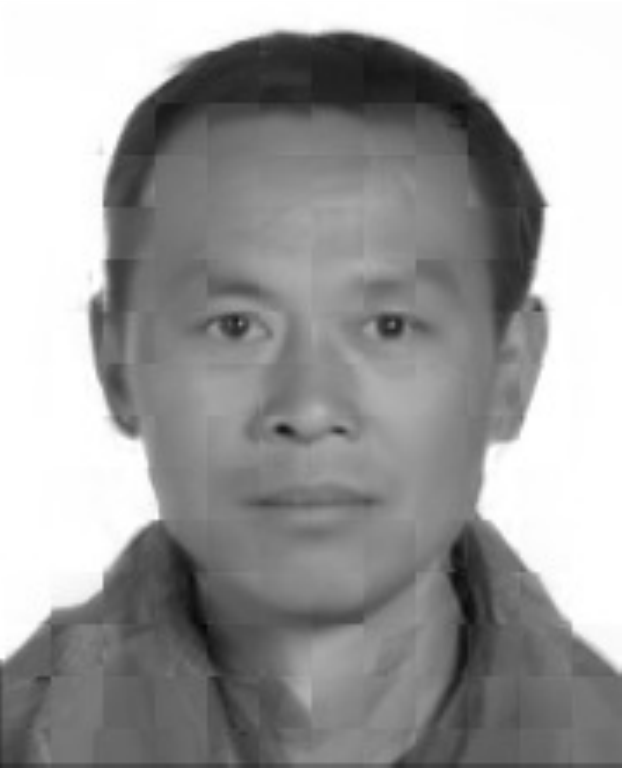}} \hfill
   \subfigure[KSVD+ \newline \bblock(6.30)]{\includegraphics[width=0.19\linewidth]{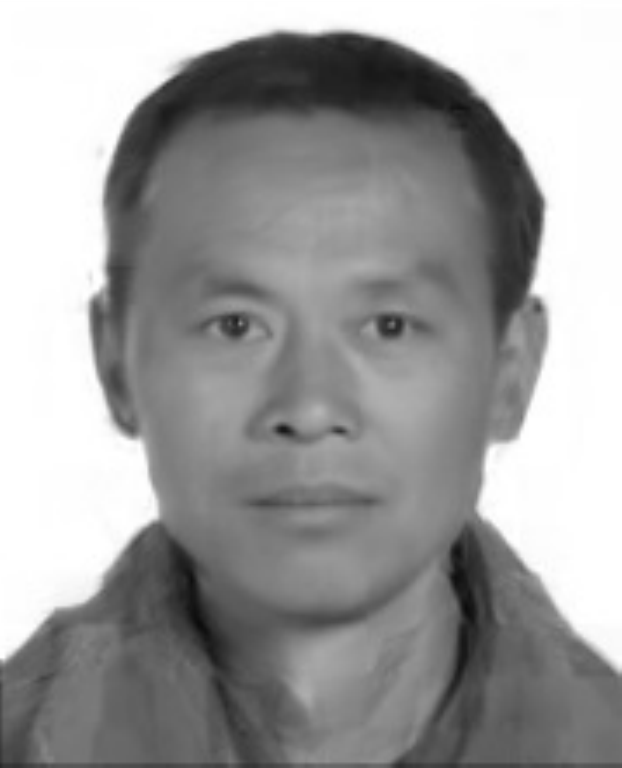}} \\
   \caption{Results of the face compression method of~\citet{bryt2008} compared
      to two classical image compression algorithms. All images are compressed
      into $550$ bytes representations. The raw output of the method, denoted by KSVD, is
      presented in the penultimate column. The last column KSVD+ is a
      post-processed result obtained by processing the output of KSVD with a
      deblocking algorithm that reduces visual artifacts.  The numbers in
      parenthesis are the square root of the mean squared errors (RMSE),
      obtained by comparing the compressed images with the original one (first
      column). Images kindly provided to us by Ori Bryt and Michael Elad. Best
   seen by zooming on a computer screen.}\label{fig:faces}
\end{figure}

%% file: content_arxiv/image_other.tex
Dictionary learning was successful in many image processing tasks because of its
ability to model well natural image patches. Other patch-based methods can be 
used for that purpose, and have also shown impressive results for image
restoration. For the sake of completeness, we briefly present a few of them
that have gained some success, but we refer the reader to the corresponding
papers for more details.

\paragraph{Non-local means and non-parametric approaches.}
\citet{efros} showed that self-similarities inherent to natural images could
be used effectively in texture synthesis tasks. Following their insight,
\citet{buades2005} have introduced the {\em non-local means} approach to
image denoising, where the prominence of self-similarities is used as
a prior on natural images.\footnote{This idea has in fact appeared in
the literature in various guises and under different equivalent
interpretations, \eg, kernel density estimation,
mean-shift iterations~\citep{awate2006}, diffusion processes on graphs \citep{szlam2007},
and long-range random fields~\citep{li2008}.}  
Concretely, let us consider a noisy image
written as a column vector $\y$ in $\Real^n$, and denote by $\y[i]$
the $i$-th pixel and, as usual, by~$\y_i$ the patch of size $m$ centered on this
pixel for some appropriate size~$m$. This approach exploits the
simple but very effective idea that two pixels associated with
similar patches $\y_i$ and $\y_j$ should have similar values $\y[i]$
and $\y[j]$. Using $\y_i$ as an explanatory variable for $\y[i]$ leads
to the non-local means formulation, where the denoised pixel $\hatx[i]$
is obtained by a weighted average:
\begin{equation} 
   \hatx[i] = \sum_{j=1}^n
\frac{K_h(\y_i-\y_j)}{\sum_{l=1}^n K_h(\y_i-\y_l)}
\y[j], \label{eq:nl} 
\end{equation} 
and $K_h$ is a Gaussian kernel of bandwidth $h$. The estimator~(\ref{eq:nl})
is in fact classical in the non-parametric statistics literature, even though
it was never applied to image denoising with natural image patches before; it is
usually called Nadaraya-Watson
estimator~\citep[see][]{nadaraya,watson,wasserman2006}.\footnote{Interestingly,
the same class of non-parametric estimators have also been used
independently by~\citet{takeda2007}, where the explanatory variables are not
full image patches, but neighbor pixels. It results in Nadaraya-Watson
estimators that exploit local information, as opposed to the non-local one
of~\citet{buades2005}.}

Later, several extensions have been proposed along this line of work. First attempts
have focused on improving the computational speed \citep{mahmoudi2005}, or on automatically adapting
the kernel bandwidth~$h$ to the pixel of interest \citep{kervrann2006}.
Extensions of the non-local principle for other tasks than denoising have also been
developed, \eg, for image upscaling~\citep{glasner2009}, or
demosaicking~\citep{buades2009}. Later, other estimators exploiting image
self-similarities and based on nonparametric density estimation have been
proposed and analyzed by~\citet{levin2012,chatterjee2012}.

\paragraph{BM3D.}
\citet{dabov2} proposed a patch-based procedure called BM3D for image denoising.
The method exploits several ideas, including image self-similarities and sparse
wavelet estimation, and provides outstanding results at a reasonable
computational cost. Several years after it was developed, it remains considered
as the state of the art.  BM3D proceeds with the following pipeline for
denoising an image~$\y$:
\begin{itemize}
   \item {\bfseries block matching}: like non-local means, BM3D exploits
      self-similarity; it processes all patches by first matching them with similar
      ones in the noisy image~$\y$, stacking them together into a 3D signal block; 
   \item {\bfseries 3D filtering}: each block is denoised by using hard or
      soft-thresholding with a 3D orthogonal DCT dictionary, following the
      wavelet thresholding tradition;
   \item {\bfseries  patch averaging}: as in Section~\ref{subsec:denoising}, a
      denoised image~$\hatx_0$ is obtained by averaging the estimates of the
      overlapping patches;\footnote{More precisely, a weighted average is
         performed. We omit this detail here for simplicity, and refer
         to~\citep{dabov2} for more details.} 
   \item {\bfseries refinement with Wiener filtering}: the previous steps are refined
      by using the intermediate estimate~$\hatx_0$. First, block matching is
      applied to~$\hatx_0$ instead of~$\y$ in order to obtain a more reliable
      matching of noisy patches from~$\y$; weights are obtained for the
      3D-filtering by using a Wiener filter. After a new patch averaging step,
      a final estimate~$\hatx_1$ is obtained.
\end{itemize}
A few additional heuristics are also used to further boost the denoising
performance~\citep[see][]{dabov2}. Finally, the scheme has proven to be very efficient
and gives substantially better results than regular non-local means.  Later, it
was improved by~\citet{dabov2009} by adding other components: (i)
shape-adaptive patches that are non necessarily rectangular; (ii) 3D filtering
with adaptive dictionaries obtained with PCA inside each block instead of using
simple DCT.

\paragraph{Non-local sparse models.}
Since BM3D was successful in combining image self-similarities and wavelet/DCT
denoising on blocks of similar patches, \citet{mairal2009} have proposed to exploit
further this idea and combine the non-local means principle with dictionary
learning.

One motivation is that non-local means approach has proven to be
effective in general, but it fails in some cases. In the extreme, when a patch
does not look like any other one in the image, it is impossible to exploit
self-similarities to estimate the corresponding pixel value. To some extent, sparse image
models based on dictionary learning can handle such situations when the patch
at hand admits a sparse decomposition, but they suffer from another drawback:
similar patches sometimes admit very different estimates due to the potential
instability of sparsity patterns, which can result in practice in
noticeable reconstruction artifacts. The idea of the non-local sparse
model is that \emph{similar patches should admit
similar sparse decompositions}. By enforcing this principle, one hopes to
obtain more stable decompositions and subsequently a better estimation.

Concretely, let us define for each patch $\y_i$ the set $\SSS_i$:
\begin{displaymath}
\SSS_i \defin \left\{ j=1,\ldots,n \st \|\y_i-\y_j\|_2^2 \leq \xi  \right\},
\end{displaymath}
where $\xi$ is some threshold. The block~$\SSS_i$ is essentially constructed by
following the block matching step of BM3D.  What differs is the
way~$\SSS_i$ is subsequently processed. Whereas BM3D performs some wavelet filtering, the
non-local sparse model jointly decomposes the patches onto a previously learned
dictionary~$\D$ in $\Real^{m \times p}$. 

To encourage the decompositions to be similar, the patches from the set~$\SSS_i$
are forced to share a joint sparsity pattern---that is, they should share a
common set of nonzero coefficients. Fortunately, this can be achieved with a
{\em group-sparsity penalty}. Denoting by~$\A_i\defin [\alphab_{ij}]_{j \in
\SSS_i}$ the matrix of coefficient in~$\Real^{p \times |\SSS_i|}$ corresponding to the group~$\SSS_i$, the regularizer
encourages the matrix~$\A_i$ to have a small number of non-zero rows, yielding
a type of sparsity patterns that is illustrated in Figure~\ref{fig:sparse}. 

Formally, the penalty can be written for a matrix~$\A$ in~$\Real^{p \times k}$ as
\begin{equation*}
   \Psi_q(\A) \defin \sum_{j=1}^p \|\alphab^j\|_2^q,
\end{equation*}
where $\alphab^j$ denotes the $j$-th row of $\A$.  When~$q=1$, we obtain the
group-Lasso norm already presented in Section~\ref{sec:l1}.  When~$q=0$, the
penalty simply counts the number of non-zero rows in~$\A$ and thus plays the
same role as the~$\ell_0$-penalty in the context of group sparsity~\citep{tropp2}.
Then, decomposing the patch $\y_i$ with the penalty~$\Psi_q$ on the set
$\SSS_i$ amounts to solving
\begin{equation*}
   \min_{\A_i \in \Real^{p \times |\SSS_i|}} \Psi_q(\A_i) \st \sum_{j \in \SSS_i} \|\y_j-\D\alphab_{ij}\|_2^2 \leq \varepsilon_i.
   \label{eq:bp4}
\end{equation*}
As for the classical denoising approach based on dictionary learning presented
in Section~\ref{subsec:denoising}, the following ingredients are used to
ultimately reconstruct the denoised image: (i) $\varepsilon_i$ is chosen
according to some effective heuristics; (ii) the non-convex penalty~$\Psi_0$ with greedy algorithms~\citep{tropp} is
preferred to~$\Psi_1$ for the final image reconstruction task once the dictionary has been
learned; (iii) since the patches overlap, the
final step of the algorithm averages the estimates of every pixel.
A few additional heuristics are also used to improve the speed or quality of the algorithm~\citep[see][]{mairal2009}, including
a refinement step as in BM3D.

The non-local sparse principle was also successfully applied
by~\citet{mairal2009} to image demosaicking (see Figure~\ref{fig:demos}), and
more generally to image interpolation~\citep{romano2014}.

\begin{figure}[t]
   \centering
   \includegraphics[width=0.6\linewidth]{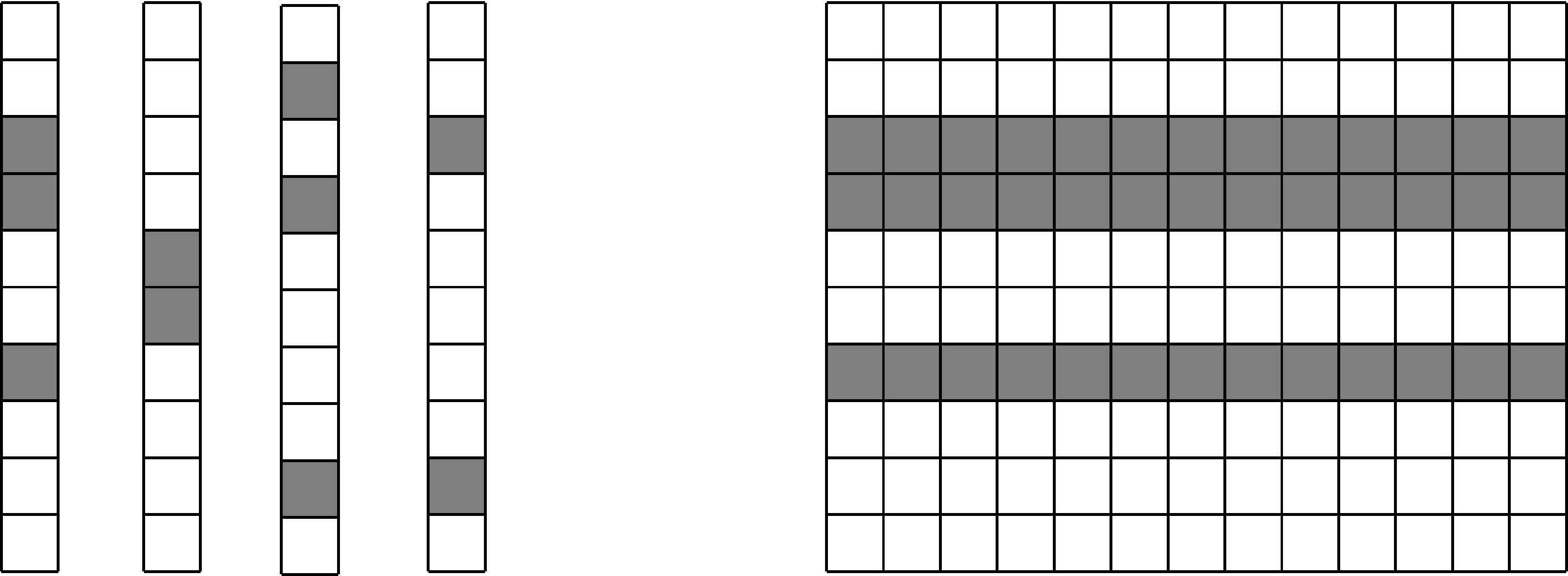}
   \caption{Sparsity vs. joint sparsity: Grey squares represents non-zeros
   values in vectors (left) or matrix (right).} \label{fig:sparse}
\end{figure}

Later, other variants of non-local sparse models have been proposed, notably
the \emph{centralized sparse representation} technique
of~\citet{dong2011b,dong2013}. There, input patches are clustered and an
orthogonal dictionary is learned for each cluster with PCA. Then, patches are
sparsely encoded on the PCA dictionaries, and the decompositions
coefficients inside each cluster are encouraged to be close to a per-cluster mean
representation. As shown in Sections~\ref{subsec:denoising}
and~\ref{subsec:demosaicking}, the performance achieved by non-local sparse
models is in general competitive with the state of the art for different tasks,
\eg, denoising or demosaicking.

\paragraph{Gaussian mixture models.}
We have already mentioned in Section~\ref{subsec:discussion} that Gaussian
mixture models have been successful for modeling natural image
patches~\citep{zoran2011,yu2012}. Not surprisingly, such approaches also lead
to impressive results for image denoising and other image reconstructions tasks
such as image upscaling~\citep{yu2012}. Given a database of image patches, the
model parameters are usually learned by using the EM-algorithm or one of its
variants, and denoising can be achieved with classical estimators for
probabilistic models, \eg, maximum a posteriori estimation.

%% file: content_arxiv/visual_intro.tex
Because of its ability to model well natural image patches and automatically
discover interpretable visual patterns, one may wonder whether dictionary
learning can be useful for visual recognition or not. Patches are indeed often
used as lowest-level features in image analysis pipelines. For instance,
this is the case of popular approaches based on bags of
words~\citep{csurka2004,lazebnik2006}, which  exploit image
descriptors that are nothing else than pre-processed natural image
patches~\citep{lowe2004,dalal2005,mikolajczyk2005,tola2010}. This is also the case of
convolutional neural networks~\citep{lecun1998}, whose first layer performs
local convolutions on raw pixel values.

It is thus tempting to believe that modeling patches is an important step for
automatic image understanding and that dictionary learning can be a key
component of recognition architectures. The different approaches that we review
in this section seem to confirm this fact, even though recognition requires
other properties that are not originally provided by sparse coding models, such
as invariance or stability of image representations to local perturbations.

In Section~\ref{sec:visual_features}, we show how dictionary learning has
been used for image classification as an alternative to clustering techniques
in bag-of-words models, before reviewing in Section~\ref{sec:visual_botany} 
numerous variants that have been proposed in the literature. Then, we present
different approaches where sparse models are used as a classification tool
for face recognition (Section~\ref{sec:visual_faces}), patch classification, and 
edge detection (Section~\ref{sec:visual_edges}). Finally, we draw some links
between dictionary learning and neural networks in
Section~\ref{subsec:backprop} with backpropagation rules and networks for
approximating sparse models, before reviewing a few additional computer vision
applications in Section~\ref{sec:visual_other}.  

We show that dictionary learning has been successfully used in various guises
for different recognition tasks.

%% file: content_arxiv/visual_features.tex
Following \citet{boureau2010}, we show in this
section that interleaved steps of (nonlinear) feature coding and
pooling are a very simple but common approach to image modeling in
visual recognition tasks such as retrieval and categorization. It is,
implicitly or explicitly, used in SIFT \citep{lowe2004}, bags of
features \citep{csurka2004,SiZi03}, HOG \citep{dalal2005}, the pyramid
match kernel \citep{GrDa05}, spatial pyramids \citep{lazebnik2006}, soft
quantization \citep{vGVSG10}, and most importantly for this
presentation, sparse coding approaches to
recognition \citep{yang2009,boureau2010,wang2010locality,yang2010b}.

We also show that the coding/pooling approach can often be intuitively
justified in terms of (approximate) feature matching.  We first consider
\emph{global pooling} approaches that discard all spatial information for every
feature, before moving to \emph{local pooling} that keeps part of it.

\subsection{Global pooling}
\paragraph{Bags of features and sparse coding.}
This paragraph largely follows~\citet{boureau2010}.  Let us consider an image
with its associated local image features, represented by (and identified with) a
finite set $\X=[\x_1,\ldots,\x_n]$ of~$n$ vectors in $\Real^m$. For instance,
each~$\x_i$ may represent a non-rotationally invariant SIFT descriptor computed
at location~$i$ in the image, but it may as well represent another feature type, \eg, local
binary pattern~\citep{ojala2002} or a color histogram.

Let us also assume that we are given a {\em coding operator}, that is a
function $\alphab: \Real^m \to \AA$ mapping features in $\Real^m$ onto the corresponding codes
in some space $\AA$, and a {\em pooling
operator} $\betab$ mapping finite subsets of $\AA$ onto
elements of $\Real^p$. The simplest version of the coding/pooling
approach is to model an image represented by a set of descriptors $\X$
in~$\Real^{m \times n}$ by the vector
$\gammab(\X)=\betab[\alphab(\X)]$, where
$\alphab(\X)$ denotes the set of points~$[\alphab(\x_i)]_{i=1}^n$ in~$\AA^n$. We dub this
approach {\em global} pooling because it summarizes the codes of {\em all}
features in the image by a single vector in~$\Real^p$.

The best known example of global pooling is the bag-of-features
model inherited from text processing and popularized in computer
vision by~\citet{SiZi03} and~\citet{csurka2004}: let us consider a quantizer
$\xi:\Real^m\rightarrow \{1,\ldots,p\}$, constructed (for example) using the
algorithm K-means on the elements of $\X$. We can use $\xi$ to encode each
feature $\x$ as the binary vector $\alphab(\x)$ of dimension $p$ with all zero
entries except for entry number $\xi(\x)$, which is
equal to one. Note that the code space $\AA$ is $\{0,1\}^p$ in this case.
 
We can now use {\em mean pooling} (also called {\em average
pooling}) to summarize the feature codes~$\alphab(\x_i)$ of the whole image by
their average--- that is, a single vector of dimension $p$:\footnote{Note
that the tf-idf renormalization common in image
retrieval \citep{SiZi03} can be achieved by reweighting the components
of $\gammab(\X)$: let $\gammab_i$ denote the $i$-th component of
$\gammab(\X)$, \ie, the frequency of the corresponding bin in
$\X$, replace $\gammab_i$ by $\gammab_i \log (N/n\gammab_i)$, where $N$ is
the total number of features in the database being queried.}
 \begin{equation*}
 \gammab(\X)=\betab[\alphab(\X)],\quad\text{where}\quad
 \betab[\alphab(\X)]=\frac{1}{n}\sum_{i=1}^n\alphab(\x_i).
 \end{equation*}
The vector $\gammab(\X)$ is called the bag of features associated with the
image.

Two images respectively represented by the local descriptors
$\X=[\x_1,\ldots,\x_n]$ in~$\Real^{m \times n}$ and $\Y=[\y_1,\ldots,\y_{n'}]$
in~$\Real^{m \times n'}$ can now be compared by computing the dot product of
their bags of features, that is, their similarity is computed as
$s(\X,\Y)=\gammab(\X)^\top \gammab(\Y)$. This
approach is commonly used in both image retrieval (where $s$ is used
for ranking)~\citep{SiZi03} and categorization (where $s$ is used as
a kernel) \citep{csurka2004}. In both cases, the similarity function
can be computed efficiently due to the fact that bags of features are
typically very sparse for large values of $p$.
 
When the quantizer $\xi$ is obtained using the K-means algorithm, the centroids
$\d_j$ ($j=1,\ldots,p$) of the corresponding clusters can be
seen as the elements (``atoms'', or ``words'') of some visual
 dictionary, and the codes $\alphab(\x)$  tell us what atom
a given feature is associated with.

This approach readily generalizes to other types of feature coding.  In
particular, it is possible to replace the vector-quantized binary codes of
the features $\x$ in $\X$ by their sparse code relative to some
(given or learned) dictionary $\D$~\citep{yang2009,boureau2010}, that is
\begin{equation*}
 \alphab(\x) \in \argmin_{\alphab' \in\Real^p}
 \frac{1}{2}\left\|\x-\D\alphab'\right\|_2^2
 +\lambda \|\alphab'\|_1,
\end{equation*}
then use mean pooling as before.
 
\paragraph{Feature matching interpretation.}
Here, we follow~\citet{JDS10}, and show that, 
in the case of bags of features constructed with a quantifier, the
similarity function $s(\X,\Y)$ can be justified intuitively by the fact
that it measures the proportion of pairs of features in~$\X$ and~$\Y$
that match each other~\citep{JDS10}. Maximizing $s(\X,\Y)$ can thus be
thought of as some (approximate) feature matching process.
Indeed, given a quantizer~$\xi$, let us define the functions $\delta_j:\Real^m\rightarrow
\{0,1\}$ ($j=1,\ldots,p$) and $\delta:\Real^m\times\Real^m\rightarrow
\{0,1\}$ by
\begin{equation}
\delta_j(\x)=\begin{cases}
 1 & \text{when}\,\, \xi(\x)=j,\\
 0 & \text{otherwise},
 \end{cases}\,\,
 \text{and}\,\,\,
 \delta(\x,\y)= 
 \begin{cases}
 1 & \text{when} \,\, \xi(\y)=\xi(\x),\\
 0 & \text{otherwise}.
 \end{cases}
 \label{eq:deltabag}
\end{equation}
We can now rewrite our similarity function as
\begin{equation*}
   s(\X,\Y)=\sum_{j=1}^p \left[\frac{1}{n}\sum_{i=1}^n \delta_j(\x_{i})\right]
   \left[\frac{1}{n'}\sum_{i'=1}^n \delta_j(\y_{i'})\right]
   =\frac{1}{n\, n'}\sum_{i=1}^n \sum_{i'=1}^{n'}
\delta(\x_{i},\y_{i'}),
\label{eq:simbag}
\end{equation*} 
which indeed measures the fraction of pairs of features in the two images
that are deemed to match when they admit the same code.

When the quantizer $\xi$ is obtained by running the K-means algorithm,
$\delta(\x,\y)$ is nonzero when the features $\x$ and $\y$ both lie
in the same cell of the Voronoi diagram of the $p$ centroids of the
corresponding clusters.
Related approaches explicitly phrased in term
of feature correspondences have been used in both image matching and
retrieval using some nearest-neighbor (or NN) decision rule such as
$\varepsilon$-search or $K$-NN. Concretely, one measures again the
fraction of matching features as
\begin{equation*}
   s(\X,\Y)=\frac{1}{n\, n'}\sum_{i=1}^n \sum_{i'=1}^{n'} \delta(\x_i,\y_{i'}),
\label{eq:sim}
\end{equation*}
but, this time the function $\delta$ is defined by
\begin{equation}
 \delta(\x,\y)= 
 \begin{cases}
 1 & \text{when} \,\, \|\y-\x\|_2 <\varepsilon,\\
 0 & \text{otherwise},
 \end{cases}
 \label{eq:deltaeps}
 \end{equation}
in the $\varepsilon$-search case, or, in the $K$-NN case, by
\begin{equation}
 \delta(\x,\y)= 
 \begin{cases}
 1 & \text{when}\,\, \y \,\, \text{is one of the} \,\, $K$ \,\,
 \text{nearest neighbors of} \,\, \x,\\
 1 & \text{when}\,\, \x \,\, \text{is one of the} \,\, $K$ \,\,
 \text{nearest neighbors of} \,\, \y,\\
 0 & \text{otherwise}.
 \end{cases}
 \label{eq:NNS}
 \end{equation}
 
This approach is the basis of the classical image retrieval and image matching
techniques  of~\citet{ScMo97} and~\citet{lowe2004}.  As in the case of bags of
features~\citep{SiZi03}, the matching step if often complemented by some
geometrical verification stage in practice. Similar schemes have also been used
in image categorization~\citep{WCG03}.
 
It should be noted that the coding/pooling models also admits
a feature-matching interpretation when images are encoded using sparse
coding instead of a quantifier: we can write in this case the
similarity $s(\X,\Y)$ of two images, respectively represented by the sets of descriptors $\X$ and $\Y$, as
\begin{equation*}
 s(\X,\Y)=\sum_{j=1}^p \left[\frac{1}{n}\sum_{i=1}^n \alpha_j(\x_i)\right]
 \left[\frac{1}{n'}\sum_{i'=1}^{n'} \alpha_j(\y_{i'})\right]
 =\frac{1}{n\,n'}\sum_{i=1}^n \sum_{i'=1}^{n'}
 \delta(\x_i,\y_{i'}),
 \end{equation*}
 where $\alphab(\x)=[\alpha_1(\x),\ldots,\alpha_p(\x)]$ and~$\delta(\x,\y)=\alphab(\x)^\top\alphab(\y)$. Note
 that by analogy,~$\delta(\x,\y)$ can be rewritten in a similar form to the previous cases~(\ref{eq:deltabag}), (\ref{eq:NNS}), or~(\ref{eq:deltaeps}):
 \begin{equation*}
 \delta(\x,\y)=
 \begin{cases}
 \alphab(\x)^\top\alphab(\y)
 & \text{when}\quad
 \exists j,\,\, \alpha_j(\x)\alpha_j(\y)\neq 0,\\
 0 & \text{otherwise}.
 \end{cases}
 \label{eq:simsparse}
 \end{equation*}
 The function $\delta$ encodes in this case the fact that two feature
 vectors are deemed to match when they both have at least one nonzero 
 code element for some dictionary atom, the contribution of each matching
 feature pair being proportional to the dot product of their codes.

 \paragraph{Max pooling.}
The construction of the vector $\gammab(\X)$ associated with an
image can be seen as pooling the code vectors
$[\alphab(\x_i)]_{i=1}^n$ associated with the image features and
summarizing them by a single vector $\gammab(\X)$ which is their
average. As noted before, this process is thus often dubbed {\em
mean}, or {\em average pooling}. A simple variant is {\em sum
pooling}, where the vectors $\gammab(\X)$ are not normalized and
simply sum up the vectors $\alphab(\x_i)$.
 
It has proven useful in several classification tasks to replace mean
pooling by {\em max pooling}~\citep{serre2005}, where
\begin{equation*}
\gamma_j(\X)=\max_{i=1,\ldots,n} \alpha_j(\x_i),
\end{equation*}
and~$\gammab(\X)=[\gamma_1(\X),\ldots,\gamma_p(\X)]$.
In the case of bags of features, the vectors $\alphab(\x_i)$ are binary, and
computing the similarity $s(\X,\Y)$ amounts to counting the number of bins in the
quantization associated with at least one feature in each
image. Intuitively, this may be justified in an application such as
Video Google~\cite{SiZi03}, where the query image (typically a box
drawn around an object of interest such as a tie or a plate) is
normally much smaller than the database images (typically depicting an
entire scene) so that features occurring often in the scene are not
given too much importance.
 
Max pooling is also used with sparse coding. This is a priori
problematic because of the sign ambiguity of dictionary learning
discussed in Section~\ref{subsec:theory}. Indeed, as noted
by~\citet{MuPe14}, max pooling should only be used (or perhaps more
accurately, is only intuitively justified) when the individual
vector entries encode a strength of association between a descriptor
and a codeword, and are thus all positive. This is the case for bags
of features but not for sparse coding, with max pooling resulting 
potentially in $2^p$ different pooled features. The
construction of Section~\ref{subsec:theory} can be used to
remove this ambiguity, but does not provide an intuitive
justification of max pooling in this setting.
It may then make sense to use a variant of dictionary learning that enforces
non-negativity constraints on the sparse codes, or to duplicate 
the entries of~$\alphab(\x)$ into negative and positive components, resulting
in a vector of size~$2p$, before applying max pooling. Specifically, each
entry~$\alpha_j(\x)$ will be duplicated into two
values~$\max(\alpha_j(\x),0)$ and~$\max(-\alpha_j(\x),0)$.

Finally, other nonlinear pooling technique have also been studied in the
literature~\citep[see][]{koniusz2013}; for simplicity, we restrict our
presentation to the average and max-pooling strategies, which are the most
popular ones.

\subsection{Local pooling}
\paragraph{Spatial pooling: spatial pyramids and their cousins}
 The global coding/pooling approach yields ``orderless'' image
 models \cite{KoVD99}, where all spatial information has been
discarded. While this affords a great deal of robustness (including a
 total invariance to image transformations {\em as long as} they leave
 the local features invariant),\footnote{One should keep in mind that
 being invariant to within-image transformations does not imply
 being invariant to, say, rotations in depth since (at least) some
 features will become occluded, or will be revealed, as the observed
 object rotates relative to the camera.} it seems wasteful to discard
 {\em all} spatial information.
 
 The spatial pyramid of~\citet{lazebnik2006}
 addresses this problem by overlaying a coarse pyramid with $L$ levels
 over the image (the HOG model of~\citet{dalal2005} is
 based on a similar idea using a coarse grid instead of a
 pyramid). There are $2^{2l}$ cells at level $l$ ($l=0,\ldots,L-1$) of
 the pyramid, and the features falling in each cell of the pyramid are
 binned separately. Let us define by $\SSS_{kl}$ the indices of features in~$\{1,\ldots,n\}$
 falling into cell number $k$ at level $l$; the spatial pyramid image
 descriptor $\gammab(\X)$ is obtained by concatenating the
 unnormalized histograms (sum pooling) $\gamma_{kl}(\X)=\sum_{i \in
 \SSS_{kl}}\alphab(\x_i)$ associated with all cells at all
 levels of the pyramid to form a vector in $\Real^{pd}$, where
 $d=1+\ldots+d_{L-1}=(2^{2L}-1)/3$, and $d_l=2^{2l}$ for
 $l=0,\ldots,L-1$.

 Note that the case $L=1$ corresponds to the global pooling model, as
 used in its bag of features or sparse coding instances, except for the
 fact that it uses {\em sum pooling} instead of mean or max pooling,
 that is, the sum of the features is used instead of their mean or max
 value. As explained in the next section, the spatial pyramids
 associated with two images can be compared using their dot product or
 the {\em pyramid matching kernel} of~\citet{GrDa05}, and they can be
 interpreted in terms of feature matching in both cases.
 
\paragraph{Feature matching interpretation.}
Given again two set of descriptors~$\X$ and~$\Y$ and a spatial pyramid
structure leading to respective sets of indices~$\SSS_{kl}$ and~$\SSS'_{kl}$ for the cell~$k$ at level~$l$, we
can compute the similarity 
 \begin{equation*}
 s(X,Y)=\gammab(\X)^\top\gammab(\Y)=
 \sum_{l=0}^{L-1}\sum_{k=1}^{d_l}
 \sum_{i \in \SSS_{kl}}\sum_{ i' \in \SSS'_{kl}}\alphab(\x_i)^\top \alphab(\y_{i'}).
 \end{equation*}
This similarity function measures the number of matches between the
two images, where matches are found  at different spatial scales as
pairs of features falling in the same cell.
 
As shown by~\citet{lazebnik2006}, an alternative is provided by the
{\em pyramid match kernel}, defined as follows. Let us first define the
{\em histogram intersection} function as
\begin{equation*}
I_l(\X,\Y)=\sum_{k=1}^{d_l} \min[\gamma_{kl}(\X),\gamma_{kl}(\Y)],
\end{equation*}
where $\gamma_{kl}(\Y)= \sum_{i \in \SSS'_{kl}}\alphab(\y_i) $, $\gamma_{kl}(\X)= \sum_{i \in \SSS_{kl}}\alphab(\x_i) $, and where the min operator is applied componentwise. The $j$-th
coordinate of $I_l(\X,\Y) $ measures the number of matches between $\X$
and $\Y$ corresponding to the atom $\d_j$ of the dictionary,
measured as the number of points from~$\X$ and~$\Y$ that have nonzero
codes, and fall in the same cell.  These points match in
the code space because they have a nonzero element in the same spot,
and they also match spatially, because they fall into the same cell.
 
By construction, the matches found at level $l$ of the pyramid also
include the matches found at level $l + 1$. The number of new matches
found at level $l$ is thus $I_l(\X,\Y)-I_{l+1}(\X,\Y)$ for $l =
 0,\ldots,L-1$. This suggests the following {\em pyramid match
kernel}~\citep{GrDa05} to compare two images:
 \begin{equation*}
    \begin{split}
       K(\X,\Y) & =I_L(\X,\Y)+\sum_{l=0}^{L-1}\frac{1}{2^{L-l}}[I_l(\X,\Y)-I_{l+1}(\X,\Y)] \\
                  & =\frac{1}{2^L}I_0(\X,\Y)+\sum_{l=0}^{L}\frac{1}{2^{L-l+1}}I_l(\X,\Y),
    \end{split}
 \end{equation*}
 where the weight $\frac{1}{2^{L-l}}$, which is inversely proportional to the cell width at
 level $l$ is used to penalize matches found in larger cells because they
 involve increasingly dissimilar features.  It is easily shown that $K$ is a
 positive-definite kernel, which allows using the machinery of kernel methods
 and reproducing kernel Hilbert spaces~\citep{Shawe-Taylor2004}.

 Like the histogram intersection function, the pyramid match kernel
 can be interpreted as measuring the number of matches between
 $\X$ and~$\Y$ in both feature space and the spatial domain. Contrary
 to the previous cases discussed in this presentation, where all pairs
 of matching features were counted, a  point of $\X$ may
 only match  a unique point of~$\Y$ in this case.
 
 \paragraph{Feature space pooling.}
 The pyramid match kernel was originally proposed by~\citet{GrDa05} as a method for matching images (or equivalently
 counting approximate correspondences) in {\em feature space}, without
 retaining any spatial information. With spatial pyramids, \citet{lazebnik2006} argued that it might be better suited
 to encoding spatial information in the two-dimensional image, using
 instead traditional vector quantization based on the K-means algorithm to
 handle matching in the high-dimensional feature space. 
 
 The two approaches can be thought of as performing a {\em local} form
 of pooling, over bins defined in feature space in~\citet{GrDa05}, or
 over cells defined in image space in~\citet{lazebnik2006}. In practice,
 the latter method usually gives better results for image
 categorization, but it may be interesting to combine the two when the
 feature codes are obtained using sparse coding.  This is precisely
 what~\citet{BRFPL11} have proposed to do \citep[see also][for related work that will be partly detailed  in the next section]{gao2010,gao2013,wang2010locality,YYH10,ZYZTJH10,koniusz2011}. In 
 this approach, a dictionary is learned on the training data, and the
 corresponding sparse codes are then clustered using K-means. 
 Let $\SSS_{klj}$ now denote the index set of features that fall in cell number $k$
 of the image at level $l$ of the pyramid, and in the Voronoi cell
 number $j$ of the sparse code space, the unnormalized histogram
 $\gamma_{kl}(\X)$ of spatial pyramid is simply replaced by 
 $\gamma_{klj}(\X)=\sum_{i \in \SSS_{klj}}
 \alphab(\x_i)$, resulting in a final descriptor
 $\gammab(\X)$ of size $pdK$ where $K$ is the number of centroids
 used by the K-means algorithm.
 Combined with max pooling, this method indeed gives very good results
 for image categorization on standard benchmarks such as Caltech 101.

%% file: content_arxiv/visual_botany.tex
To address the limitations of the original dictionary learning formulation
of~\citet{field1996,olshausen1997} for feature encoding, many different
variants have been proposed in the literature. In this section, we go through a
few of them that have gained a significant amount of attention and have been
reported to improve upon the original codebook learning approach presented in
the previous section.

\paragraph{Local coordinate coding.}
\citet{yu2009nonlinear} have proposed a significantly different view of the
dictionary learning problem than the one presented so far in this monograph,
where the learned dictionary elements are interpreted as anchor points on a
nonlinear smooth manifold.\footnote{Note that the definition of ``manifold'' from~\citet{yu2009nonlinear} slightly differs
   from the traditional one. More precisely, they define a ``manifold'' $\MM$ as a subset of~$\Real^m$ such there exists a constant $C$ and ``dimension'' $p$ such that for all~$\x$ in~$\MM$, there exists $p$ vectors~$\D(\x) = [\d_1(\x),\ldots,\d_p(\x)]$ in~$\Real^{m \times p}$ such that
   for all~$\x'$ in~$\Real^m$,
\begin{displaymath}
   \min_{\alphab \in \Real^p} \left\| \x'- \x - \D(\x)\alphab\right\|_2  \leq C \|\x'-\x\|_2^2.
\end{displaymath}
} More precisely, these anchor points define a local
coordinate system to represent data points, \eg, natural image patches or local
descriptors. To make this interpretation relevant,~\citet{yu2009nonlinear} have introduced the concept of~\emph{locality}
for the dictionary elements, encouraging them to be
close to data points. In practice, the resulting formulation, called ``local
coordinate coding'', consists of optimizing the following cost function:
\begin{equation}
   \min_{ \D \in \Real^{m \times p}, \A \in \Real^{p \times n}}  \frac{1}{n}\sum_{i=1}^n \left(\frac{1}{2}\|\x_i-\D\alphab_i\|_2^2 + \lambda\sum_{j=1}^p \|\x_i-\d_j\|_2^2 \, |\alphab_i[j]|\right), \label{eq:lcc}
\end{equation}
where, as usual, $\X=[\x_1,\ldots,\x_n]$ in~$\Real^{m \times n}$ is the
database of training signals, $\A=[\alphab_1,\ldots,\alphab_n]$ is the matrix
of sparse coefficients, and~$\D=[\d_1,\ldots,\d_p]$ is the dictionary. In contrast
to the classical dictionary learning formulation, the
regularization function is a weighted $\ell_1$-norm, where the quadratic weights 
$\|\x_i-\d_j\|_2^2$ encourage the selection of dictionary elements that are
close to the signal~$\x_i$ in Euclidean norm. Optimizing~(\ref{eq:lcc}),
meaning finding a stationary point since the problem is nonconvex, can be
achieved with alternate minimization between~$\D$ and~$\A$, which is a classical
strategy for dictionary learning (see Section~\ref{chapter:optim}).

Note that the link between~(\ref{eq:lcc}) and the manifold assumption is not
obvious at first sight. In fact, \citet{yu2009nonlinear} show
that~(\ref{eq:lcc}) can be interpreted as a practical heuristic for a slightly
different formulation that has precise theoretical guarantees, namely an
approximation bound for modeling the supposedly existing manifold. Later, the
locality principle was revisited by~\citet{yu2010improved} with an improved
approximation scheme, and by~\citet{wang2010locality} with a practical feature
encoding scheme.

The latter approach, dubbed ``locality-constrained linear coding (LLC)'',
exploits the locality principle in a simple way to select a pre-defined number
$K < p$ of non-zero coefficients for the codes~$\alphab_i$ instead of using
the~$\ell_1$-regularization.  Given a fixed dictionary~$\D$ and a
signal~$\x_i$, the LLC scheme forces the support of~$\alphab_i$ to correspond
to the $K$-nearest dictionary elements to~$\x_i$. The value of the non-zero
coefficients are then computed in analytical form without having to solve a
sparse regression problem.  Recent reviews and benchmarks have shown that such
a coding scheme is competitive for image classification
tasks~\citep{sanchez2013}.

\paragraph{Laplacian sparse coding.}
We have seen in Section~\ref{sec:image_other}  
with the non-local sparse models that a successful idea for image restoration
is to encourage similar signals to share similar sparse decomposition patterns. 
For visual recognition, \citet{gao2010,gao2013} have shown that such a principle can 
be also useful. Given a training set~$\X=[\x_1,\ldots,\x_n]$ of local descriptors,
their formulation called ``Laplacian sparse coding'' consists of minimizing
a regularized dictionary learning objective function
\begin{equation*}
   \min_{ \D \in \CC, \A \in \Real^{p \times n}}  \frac{1}{n}\sum_{i=1}^n \left(\frac{1}{2}\|\x_i-\D\alphab_i\|_2^2 + \lambda\|\alphab_i\|_1 \right) + \mu \sum_{i,j} W_{ij} \|\alphab_i-\alphab_j\|_2^2, \label{eq:laplaciansc}
\end{equation*}
where $\W = [W_{ij}]$ in $\Real^{n \times n}$ is a similarity matrix; $W_{ij}$
should be typically large when~$\x_i$ and~$\x_j$ are similar according to some
appropriate metric, encouraging therefore~$\alphab_i$ to be close to~$\alphab_j$
in Euclidean norm. The terminology of ``Laplacian'' comes from the machine
learning literature about semi-supervised learning where model
variables sit on a graph~\citep{belkin2003,belkin2004} and where
such regularization functions are used.
\citet{gao2010,gao2013} further discuss the
choice of a good similarity matrix~$\W$ when the signals~$\x_i$ are local
descriptors, and experimentally show that the additional regularization yields
better classification results than the traditional sparse coding approach of~\citet{yang2009}.

\paragraph{Convolutional sparse coding.}
In Section~\ref{subsec:matrix_other}, we have presented the convolutional
sparse coding model applied to natural images without any concrete
application. \citet{zeiler2011} have applied such a formulation for
visual recognition with a multilayer scheme inspired from convolutional neural
networks and from other hierarchical models~\citep{serre2005}.

For each layer of the network, the formulation leverages the concept of
``feature maps''---that is, a set of two-dimensional spatial maps that carry
coefficients corresponding to a particular feature type at different locations. A feature
map~$\x$ has typically three dimensions, \eg, $\x$ is in~$\Real^{\sqrt{l}
\times \sqrt{l} \times p}$ for a square feature map of spatial size~$\sqrt{l}
\times \sqrt{l}$ and~$p$ different feature types, but it is sometimes more
convenient to represent~$\x$ as a vector of size~$pl$. It is also convenient
to consider three-dimensional spatial patches in the feature maps, \eg, a patch
of size $\sqrt{e} \times \sqrt{e} \times p$ that contains all information
across feature maps from a spatial window of size $\sqrt{e} \times \sqrt{e}$
centered at a particular location.

Assuming that the parameters of the network have been already learned, one
layer of the hierarchical model processes input feature maps produced by the
previous layer (the first input feature map at the bottom of the hierarchy
being the image itself), and produces an output feature map.  Each layer
performs successively two operations, which are illustrated in
Figure~\ref{fig:hierarchy_convsc}:
\begin{enumerate}
   \item {\bfseries convolutional sparse coding}: the input feature map~$\x$ is
      encoded by using the convolutional sparse coding formulation already
      presented in Section~\ref{subsec:matrix_other}, which we recall now.
      Assuming that a dictionary~$\D$ in~$\Real^{m \times p}$ has been
      previously learned, the feature map $\x$---represented as a vector in~$\Real^l$
      here---is encoded as follows:
      \begin{equation*}
         \min_{\A \in \Real^{p \times l}} \frac{1}{2}\left\|\x - \sum_{k=1}^l \R_k^\top \D \alphab_k\right\|_2^2 + \lambda\sum_{k=1}^l \|\alphab_k\|_1,\label{eq:convolutional2}
      \end{equation*}
      where $\R_k$ is the linear operator that extracts the patch centered at
      the~$k$-th location from~$\x$ and~$\R_k^\top$ positions a small patch at a
      location~$k$ in a larger feature map, using a similar notation as
      in~(\ref{eq:patch_averaging}). This operation produces sparse codes~$\A =
      [\alphab_1,\ldots,\alphab_l]$ in~$\Real^{p \times l}$, with~$p$ different
      coefficients for each position~$k$. Then, the matrix~$\A$ can  
      be interpreted as an intermediate feature map of size~$\sqrt{l} \times
      \sqrt{l}$ with~$p$ channels, as shown in Figure~\ref{fig:hierarchy_convsc}.
   \item {\bfseries three-dimensional max pooling}: The spatial resolution and
      the number of channels of~$\A$ is typically reduced by~$2$ with a
      max-pooling operation already presented in
      Section~\ref{sec:visual_features}  with three-dimensional pooling
      regions.
\end{enumerate}

\pgfdeclarelayer{bottom}\pgfdeclarelayer{middle}\pgfdeclarelayer{top}
\pgfsetlayers{bottom,middle,top}  
\begin{figure}
   \centering
   \begin{tikzpicture}[scale=1.5,every node/.style={minimum size=1cm},on grid]
      \begin{pgfonlayer}{bottom}
         \newcount\mycount
         \foreach \i in {0,1,2,3,4} {
            \mycount=\i
            \multiply\mycount by 3
            \begin{scope}[  
                  yshift=\mycount,every node/.append style={
                  yslant=0.5,xslant=-1,rotate=-10},yslant=0.5,xslant=-1,rotate=-10
               ]
               \coordinate (X\i) at (0.15,0.75);
               \coordinate (G\i) at (1.5,0.45);
               \coordinate (ZA\i) at (1.05,0);
               \coordinate (ZB\i) at (1.35,0);
               \coordinate (ZC\i) at (0.75,0);
               \coordinate (AAA\i) at (0,0);
               \coordinate (AAB\i) at (0,2.4);
               \coordinate (AAC\i) at (2.4,0);
               \coordinate (AAD\i) at (2.4,2.4);
               \coordinate (AA\i) at (0.6,0);
               \coordinate (AB\i) at (0.6,0.9);
               \coordinate (AC\i) at (1.5,0);
               \coordinate (AD\i) at (1.5,0.9);
               \coordinate (EA\i) at (0,0.6);
               \coordinate (EB\i) at (0,0.9);
               \coordinate (EC\i) at (0.3,0.6);
               \coordinate (ED\i) at (0.3,0.9);
               \newcount\prevcount
               \prevcount=\i
               \advance\prevcount by -1
               \ifnum\i>0
               \draw[thick,blue!70] (AA\i) -- (AA\the\prevcount);
               \draw[thick,blue!70] (AB\i) -- (AB\the\prevcount);
               \draw[thick,blue!70] (AC\i) -- (AC\the\prevcount);
               \draw[thick,blue!70] (AD\i) -- (AD\the\prevcount);
               \draw[thick,black] (AAA\i) -- (AAA\the\prevcount);
               \draw[thick,black] (AAB\i) -- (AAB\the\prevcount);
               \draw[thick,black] (AAC\i) -- (AAC\the\prevcount);
               \draw[thick,black] (AAD\i) -- (AAD\the\prevcount);
               \fi
               \fill[white,fill,opacity=.7] (0,0) rectangle (2.4,2.4);
               \draw[step=3mm, gray!70] (0,0) grid (2.4,2.4);
               \draw[black] (0,0) rectangle (2.4,2.4);
               \draw[blue!20,fill] (0.6,0) rectangle (1.5,0.9);
               \draw[blue!90] (0.6,0) rectangle (1.5,0.9);
               \draw[step=3mm, blue!70] (0.6,0) grid (1.5,0.9);
            \end{scope}
         }
         \draw[-latex,thick] (-2.2,1.0) node[below,yshift=2mm]{{\color{black} input maps $\x$}}to[out=90,in=180] (AAB0);
         \draw[-latex,thick] (2.2,0.6) node[right]{{\color{blue!80!black!80} patch $\R_k \x$}}to[out=200,in=-90] (ZA4);
         \draw[-latex,thick] (1.8,2.4) node[right]{{\color{black!80} \small convolutional sparse coding}}to[out=190,in=0] (1.0,1.7);
      \end{pgfonlayer}
      \begin{pgfonlayer}{middle}
         \newcount\mycount
         \foreach \i in {0,1,2,3,4} {
            \mycount=\i
            \multiply\mycount by 3
            \advance\mycount by 55
            \begin{scope}[  
                  yshift=\mycount,every node/.append style={
                  yslant=0.5,xslant=-1,rotate=-10},yslant=0.5,xslant=-1,rotate=-10
               ]
               \coordinate (W\i) at (0.75,0.15);
               \coordinate (BAA\i) at (0,0);
               \coordinate (BAB\i) at (0,1.8);
               \coordinate (BAC\i) at (1.8,0);
               \coordinate (BAD\i) at (1.8,1.8);
               \coordinate (BA\i) at (0.6,0);
               \coordinate (BB\i) at (0.6,0.3);
               \coordinate (BC\i) at (0.9,0);
               \coordinate (BD\i) at (0.9,0.3);
               \coordinate (KA\i) at (1.2,0.6);
               \coordinate (KB\i) at (1.2,1.2);
               \coordinate (KC\i) at (1.8,0.6);
               \coordinate (KD\i) at (1.8,1.2);
               \newcount\prevcount
               \prevcount=\i
               \advance\prevcount by -1
               \ifnum\i>0
               \draw[thick,blue!70] (BA\i) -- (BA\the\prevcount);
               \draw[thick,blue!70] (BB\i) -- (BB\the\prevcount);
               \draw[thick,blue!70] (BC\i) -- (BC\the\prevcount);
               \draw[thick,blue!70] (BD\i) -- (BD\the\prevcount);
               \draw[thick,green!70] (KA\i) -- (KA\the\prevcount);
               \draw[thick,green!70] (KB\i) -- (KB\the\prevcount);
               \draw[thick,green!70] (KC\i) -- (KC\the\prevcount);
               \draw[thick,green!70] (KD\i) -- (KD\the\prevcount);
               \draw[thick,black] (BAA\i) -- (BAA\the\prevcount);
               \draw[thick,black] (BAB\i) -- (BAB\the\prevcount);
               \draw[thick,black] (BAC\i) -- (BAC\the\prevcount);
               \draw[thick,black] (BAD\i) -- (BAD\the\prevcount);
               \fi
               \fill[white,fill,opacity=.7] (0,0) rectangle (1.8,1.8);
               \draw[step=3mm, gray!70] (0,0) grid (1.8,1.8);
               \draw[black] (0,0) rectangle (1.8,1.8);
               \draw[blue!20,fill] (0.6,0) rectangle (0.9,0.3);
               \draw[blue!90] (0.6,0) rectangle (0.9,0.3);
               \draw[green!20,fill] (1.2,0.6) rectangle (1.8,1.2);
               \draw[green!70!black!100] (1.2,0.6) rectangle (1.8,1.2);
               \draw[step=3mm, green!70!black!100] (1.2,0.6) grid (1.8,1.2);
            \end{scope}
         }
         \draw[-latex,thick] (2.2,3.3) node[right]{{\color{blue!80!black!80} sparse vector $\alphab_{k}$}}to[out=180,in=60] (W4);
         \draw[-latex,thick] (-2,2.7) node[below,yshift=2mm]{{\color{black} sparse codes $\A$}}to[out=90,in=180] (BAB0);
         \draw[-latex,thick] (2.2,4.3) node[right]{{\color{black!80} \small 3D max pooling}}to[out=190,in=0] (1.3,4);
      \end{pgfonlayer}
      \begin{pgfonlayer}{bottom}
         \draw[thick,blue!70] (BA0) -- (AA4);
         \draw[thick,blue!70] (BB0) -- (AB4);
         \draw[thick,blue!70] (BC0) -- (AC4);
         \draw[thick,blue!70] (BD0) -- (AD4);
      \end{pgfonlayer}
      \begin{pgfonlayer}{top}
         \foreach \i in {0,1,2,3,4} {
            \mycount=\i
            \multiply\mycount by 3
            \advance\mycount by 105
            \begin{scope}[  
                  yshift=\mycount,every node/.append style={
                  yslant=0.5,xslant=-1,rotate=-10},yslant=0.5,xslant=-1,rotate=-10
               ]
               \coordinate (H\i) at (1.25,0.75);
               \coordinate (CAA\i) at (0,0);
               \coordinate (CAB\i) at (0,1.5);
               \coordinate (CAC\i) at (1.5,0);
               \coordinate (CAD\i) at (1.5,1.5);
               \coordinate (CA\i) at (1.0,0.5);
               \coordinate (CB\i) at (1.0,1.0);
               \coordinate (CC\i) at (1.5,0.5);
               \coordinate (CD\i) at (1.5,1.0);
               \newcount\prevcount
               \prevcount=\i
               \advance\prevcount by -1
               \ifnum\i>0
               \draw[thick,black] (CAA\i) -- (CAA\the\prevcount);
               \draw[thick,black] (CAB\i) -- (CAB\the\prevcount);
               \draw[thick,black] (CAC\i) -- (CAC\the\prevcount);
               \draw[thick,black] (CAD\i) -- (CAD\the\prevcount);
               \draw[thick,green!70!black!100] (CA\i) -- (CA\the\prevcount);
               \draw[thick,green!70!black!100] (CB\i) -- (CB\the\prevcount);
               \draw[thick,green!70!black!100] (CC\i) -- (CC\the\prevcount);
               \draw[thick,green!70!black!100] (CD\i) -- (CD\the\prevcount);
               \fi
               \fill[white,fill,opacity=.7] (0,0) rectangle (1.5,1.5);
               \draw[step=5mm, gray!70] (0,0) grid (1.5,1.5);
               \draw[black] (0,0) rectangle (1.5,1.5);
               \draw[green!20,fill] (1.0,0.5) rectangle (1.5,1.0);
               \draw[green!70!black!100] (1.0,0.5) rectangle (1.5,1.0);
            \end{scope}
         }
         \draw[-latex,thick] (-2,4.3) node[below,yshift=2mm]{{\color{black} output maps $\B$}}to[out=90,in=180] (CAB0);
         \draw[-latex,thick] (2.2,5.2) node[right]{{\color{green!70!black!100} pooled code $\betab_{j}$}}to[out=180,in=60] (H4);
      \end{pgfonlayer}
      \begin{pgfonlayer}{middle}
         \draw[thick,green!70!black!100] (KA4) -- (CA0);
         \draw[thick,green!70!black!100!green!70] (KB4) -- (CB0);
         \draw[thick,green!70!black!100!green!70] (KC4) -- (CC0);
         \draw[thick,green!70!black!100!green!70] (KD4) -- (CD0);
      \end{pgfonlayer}
   \end{tikzpicture}
   \caption{Illustration representing one layer of the hierarchical
   convolutional sparse coding model of~\citet{zeiler2011}. The vector~$\x$ denotes the
input feature maps, which are encoded into sparse codes
$\A=[\alphab_{k}]_{k=1}^l$ with the convolutional sparse coding formulation.
The output feature maps $\B=[\betab_j]_{j=1}^{l'}$ are obtained after a
three-dimensional max pooling step, reducing the spatial resolution and the
number of maps from~$p$ to~$p'$.}\label{fig:hierarchy_convsc}
\end{figure}
\pgfsetlayers{main}  

\paragraph{Separable sparse coding.}
A variant of the convolutional sparse coding model has also been investigated
by~\citet{rigamonti2013}, where the dictionary elements are
spatially separable, or partially separable. More precisely, since a dictionary
element~$\d_j$ in~$\Real^m$ represents a two-dimensional square patch, it is
possible to reorganize its entries as a matrix~$\mat(\d_j)$ in~$\Real^{\sqrt{m} \times
\sqrt{m}}$. Then, a dictionary element~$\d_j$ is said to be separable when there
exist some vectors $\u_j$ and~$\v_j$ in~$\Real^{\sqrt{m}}$ such that
$\mat(\d_j) = \u_j \v_j^\top$. In other words, separability simply means that the
matrix~$\mat(\d_j)$ is rank one. 

The motivation for looking for separable dictionary elements is to accelerate
the computation of the inner products~$\d_j^\top \R_k\x$ for all overlapping patches~$\R_k\x$
from an input image~$\x$ in~$\Real^l$. This step indeed dominates the computational cost
of reconstruction algorithms for the convolutional sparse coding model, and thus
it is critical to make it efficient.
The key observation is that computing all products~$\d_j^\top \R_k\x$ is equivalent
to performing a convolution of the two-dimensional filter~$\mat(\d_j)$ on the
input image~$\x$, which naively requires~$O(ml)$ operations.\footnote{Note that
the fast Fourier transform (FFT) could be used to reduce this theoretical
complexity. Unfortunately, the FFT induces a computational overhead that make
it inefficient in practice when convolving a large image with a small filter,
which is the case here.} In the separable
case, we remark that we have~$\d_j^\top \R_k\x = \trace(\mat(\d_j)^\top \mat(\R_k\x)) =
\trace(\v_j \u_j^\top \mat(\R_k\x)) = \u_j^\top \mat(\R_k\x)\v_j$. It is then easy
to see that the inner
products can be obtained by convolving twice the input image~$\x$: first with
the filter~$\u_j$ along the vertical direction, then with~$\v_j^\top$ along the
horizontal direction.  As a result, the complexity drops to $O(\sqrt{m}l)$
operations.

\citet{rigamonti2013} further study these filters for three-dimensional data,
where the gain in complexity is even more important than in the two-dimensional
case. They also show how to learn partially separable filters by using the following
regularization function for the dictionary:
\begin{displaymath}
   \varphi(\D) \defin \sum_{j=1}^p \|\mat(\d_j)\|_*,
\end{displaymath}
where $\|.\|_*$ is the trace norm presented in Section~\ref{sec:l1} As a
result, the penalty~$\varphi$ encourages the matrices~$\mat(\d_j)$ to be low-rank
instead of exactly rank one, meaning that they can be expressed by a sum of
a few rank-one matrices.

Another variant of ``separable dictionary learning'' has also been
proposed at the same time by~\citet{hawe2013} in a non-convolutional setting,
where the dictionary is assumed to be the Kronecker product of two smaller
dictionaries:
\begin{displaymath}
   \D = \B \otimes \C,
\end{displaymath}
where~$\B$ is in~$\Real^{m_b \times p_b}$ and~$\C$ is in~$\Real^{m_c \times
p_c}$.  In other words, the matrix~$\D$ is in~$\Real^{m_b m_c \times p_b p_c}$
and is a block matrix with~$m_b \times p_b$ blocks, such that the~$i,j$-th
block is $\B[i,j]\C$. Similar to~\citet{rigamonti2013}, the motivation
of~\citet{hawe2013} is to make the matrix-vector multiplications~$\D\alphab$
or~$\D^\top \x$ more efficient by exploiting the properties of Kronecker
products~\citep[see][]{golub2012}. However, it is worth noticing that the
concepts of ``separability'' used by~\citet{rigamonti2013} and~\citet{hawe2013}
are not equivalent to each other.

\paragraph{Multipath sparse coding and hierarchical matching pursuit.}
Finally, another model called hierarchical matching pursuit and introduced
by~\citet{bo2011,bo2013} has recently obtained state-of-the-art results in
various recognition tasks.  Similar to the convolutional approach
of~\citet{zeiler2011}, the scheme of~\citet{bo2011} produces for each image a
sequence of feature maps organized in a multilayer fashion, but the coding
scheme is somewhat simpler than in the convolutional sparse coding model.  In a
nutshell, the main differences are the following:
\begin{itemize}
   \item {\bfseries sparse coding of independent patches:} instead of using the
      convolutional model, all patches from an input feature map are coded
      independently;
   \item {\bfseries orthogonal matching pursuit:} instead of using
      the~$\ell_1$-regularization, the coding algorithm is a greedy method,
      which is presented in Section~\ref{sec:optiml0};
   \item {\bfseries contrast normalization:} after each pooling step, patches
      are contrast-normalized, following the procedure described in
      Section~\ref{subsec:preprocess}.
\end{itemize}
Later, the multilayer scheme was improved by~\citet{bo2013} by combining
several networks that have a different number of layers and thus that have
different invariant properties. A simpler version of this scheme with a single
layer has also been shown to be effective for replacing low-level features such
as histograms of gradients~\citep{dalal2005} for object detection
tasks~\citep{ren2013}.

%% file: content_arxiv/visual_faces.tex
One important success of sparse estimation for classification tasks is face
recognition. \citet{wright2008robust} have indeed proposed a state-of-the-art classifier for
this task; as we will see, however, even though it was first evaluated on face
datasets, the classifier is generic and can be applied to other modalities.
More specifically, let us first consider a set of training signals from~$k$ different
classes, represented by the matrix~$\D=[\D_1,\ldots,\D_k]$ in~$\Real^{m \times
p}$ where the columns of each sub-matrix~$\D_j$ are signals from the~$j$-th
class.  In order to classify a new test signal~$\x$ in~$\Real^m$,
\citet{wright2008robust} have proposed to sparsely decompose~$\x$ onto the
matrix~$\D$, and then measure which one of the sub-matrix~$\D_j$ is the most
``used'' in the decomposition. The precise procedure is given in
Algorithm~\ref{alg:src}. 

Interestingly, this approach called ``sparse-representation based classifier
(SRC)'' is related to other non-parametric machine learning techniques such as
nearest neighbor classifiers that look for the most similar training samples
to~$\x$ before choosing the class label with a voting scheme~\citep[see,
\eg,][]{hastie2009}, or methods looking for the nearest
subspace~\citep{naseem2010linear}, which project the data~$\x$ onto each span
of the submatrices~$\D_j$, before selecting the subspace that best reconstructs
the input data.

\begin{algorithm}[hbtp]
   \caption{Sparse-representation based Classifier (SRC).}\label{alg:src}
   \begin{algorithmic}[1]
      \REQUIRE training instances $\D=[\D_1,\ldots,\D_k]$ in~$\Real^{m \times
      p}$, each~$\D_j$ represents signals from class~$j$; One test signal~$\x$
      in~$\Real^m$.
      \STATE {\bfseries normalization:} make each column of~$\X$ of unit $\ell_2$-norm;
      \STATE {\bfseries sparse decomposition:}
      \begin{equation}
         \alphab^\star \in \argmin_{\alphab \in \Real^p} \left[ \|\alphab\|_1 \st \x=\D\alphab \right]; \label{eq:src}
      \end{equation}
      \STATE {\bfseries classification:}
      \begin{equation}
         \hatj \in \argmin_{j \in \{1,\ldots,k\}} \|\x-\D_j \delta_j(\alphab^\star)\|_2^2, \label{eq:src2}
      \end{equation}
      where~$\delta_j$ selects the entries of~$\alphab^\star$ corresponding to the class~$j$;
      in other words, $\D\alphab^\star = \sum_{j=1}^k \D_j \delta_j(\alphab^\star)$;
      \RETURN the estimated class $\hatj$ for the signal~$\x$.
   \end{algorithmic}
\end{algorithm}

For the problem of face recognition, each matrix~$\D_j$ contains face samples
of the same subject, and the method of~\citet{wright2008robust} has gained a
lot of popularity. However, this approach may suffer from several issues, which
have been addressed by various extensions:
\begin{itemize}
  \item {\bfseries occlusion:} face images in realistic environments often
     contain unwanted sources of variations. Subjects may indeed wear clothes,
     scarves, glasses, which occlude part of the face, and which are not
     necessarily present in the training data. To improve the robustness 
     of the formulation,~\citet{wright2008robust} have introduced an auxiliary
     variable~$\e$ in~$\Real^p$ whose purpose is to model occluded parts. 
     Then, Eq.~(\ref{eq:src}) is replaced by the following
     linear program:
     \begin{displaymath}
        (\alphab^\star,\e^\star) \in \argmin_{\alphab \in \Real^p, \e \in \Real^m} \left[ \|\alphab\|_1 + \|\e\|_1 \st \x=\D\alphab+\e\right],
     \end{displaymath}
     where the~$\ell_1$-norm encourages the variable~$\e$ to be sparse.  Thus,
     one hopes the variable~$\e$ to have non-zero coefficients corresponding to
     occluded areas that are difficult to reconstruct with the training data~$\D$. 
     The selection rule~(\ref{eq:src2}) needs also to be modified, and becomes
     \begin{displaymath}
        \hatj \in \argmin_{j \in \{1,\ldots,k\}} \|\x-\D_j \delta_j(\alphab^\star) - \e^\star\|_2^2.
     \end{displaymath}
  \item {\bfseries image misalignment:} the main assumption of the face
     recognition system of~\citet{wright2008robust} is that a new test image
     can be sparsely decomposed as a linear combination of training images.
     Of course, such an assumption does not hold when the training images are
     not perfectly aligned---that is, when the face features for a subject are
     not localized at the same positions across training samples. To deal with
     that issue, \citet{wagner2012} have proposed dedicated solutions, which
     jointly optimize the variables~$\alphab$ and~$\e$, along with a
     deformation variable whose purpose is to automatically ``realign images''.
\end{itemize}
Among other variants, it is worth mentioning some work using dictionary
learning to build the matrices~$\D_j$ \citep[see,\eg,][]{zhang2010}, and also
the use of random projection features~\citep{wright2008robust}.

%% file: content_arxiv/visual_patches.tex
The capability of dictionary learning for modeling particular types of signals
has been exploited for classification tasks in several ways. Specifically, there
are two main successful strategies; the first one learns one dictionary per
class of signal and compare reconstruction errors obtained by the dictionaries
in the same vein as the SRC classifier of the previous section does.  The other
direction uses sparse codes produced by a single dictionary as a high-level
representation of signals for a subsequent classification procedure.
We successively present these two lines of research, along with some successful
applications.

\paragraph{Finding the dictionary that best represents a class of signals.}
Let us consider a training set of signals~$\X=[\X_1,\ldots,\X_k]$ in~$\Real^{m
\times n}$ where the columns of each sub-matrix~$\X_j$ in~$\Real^{m \times
n_j}$ correspond to signals of the~$j$-th class. Even though dictionary
learning is an unsupervised learning technique, it can be used to model each
class independently by learning a dictionary~$\D_j$ ``adapted'' to the
data~$\X_j$. For instance, it is possible to define~$\D_j$ as the output
of a learning algorithm for optimizing
\begin{displaymath}
   \min_{\D_j \in \CC, \A_j \in \Real^{p \times {n_j}}} \frac{1}{2}\|\X_j - \D_j \A_j\|_{\text{F}}^2 + \lambda \sum_{i=1}^{n_j}\|\alphab_i\|_1,
\end{displaymath}
where $\A_j=[\alphab_1,\ldots,\alphab_{n_j}]$.
Then, a new test signal~$\x$ in~$\Real^m$ is classified according to the rule
\begin{equation}
   \hatj = \argmin_{j \in \{1,\ldots,k\}} \left[\min_{\alphab \in \Real^p}
   \frac{1}{2}\|\x-\D_j\alphab\|_2^2+\lambda\|\alphab\|_1 \right].
   \label{eq:class_aa}
\end{equation}
This simple rule was proposed by~\citet{ramirez2010}, following earlier work
of~\citet{mairal2008c} that focused on the~$\ell_0$-penalty and discriminative
formulations.  Even though the dictionaries are learned in an unsupervised
manner,~(\ref{eq:class_aa}) performs surprisingly well for data that is well
modeled by dictionary learning. For instance, the approach
of~\citet{ramirez2010} essentially builds upon~(\ref{eq:class_aa}) while also
encouraging the dictionaries~$\D_j$ to be mutually incoherent, and obtains
competitive results for digit recognition, namely about~$1.2\%$ error rate
on the MNIST dataset. In contrast, a nonlinear support vector machine (SVM) with
a Gaussian kernel achieves~$1.4\%$ and a $K$-nearest neighbor approach
$3.9\%$~\citep{lecun1998}. Of course, better results have been obtained 
by other techniques on the MNIST dataset, but they are in general not based on
generic classifiers. More precisely, state-of-the-art algorithms typically
encode invariance to image deformations~\citep{ranzato2007,bruna2013}.

Since dictionary learning is supposed to be well adapted to natural image
patches, the classification rule~(\ref{eq:class_aa}) can naturally be applied
to model the local appearance of particular object classes.  For instance, we
show in Figure~\ref{fig:patch_classif} some results obtained with a formulation
that is closely related to~(\ref{eq:class_aa}) for a weakly supervised
learning task~\citep[see][for more details]{mairal2008c}. Because 
the local appearance of bicycles is highly discriminative, the method performs
well at detecting patches that overlap with bicycles in the test images.
\begin{figure}[hbtp]
   \includegraphics[width=0.24\linewidth]{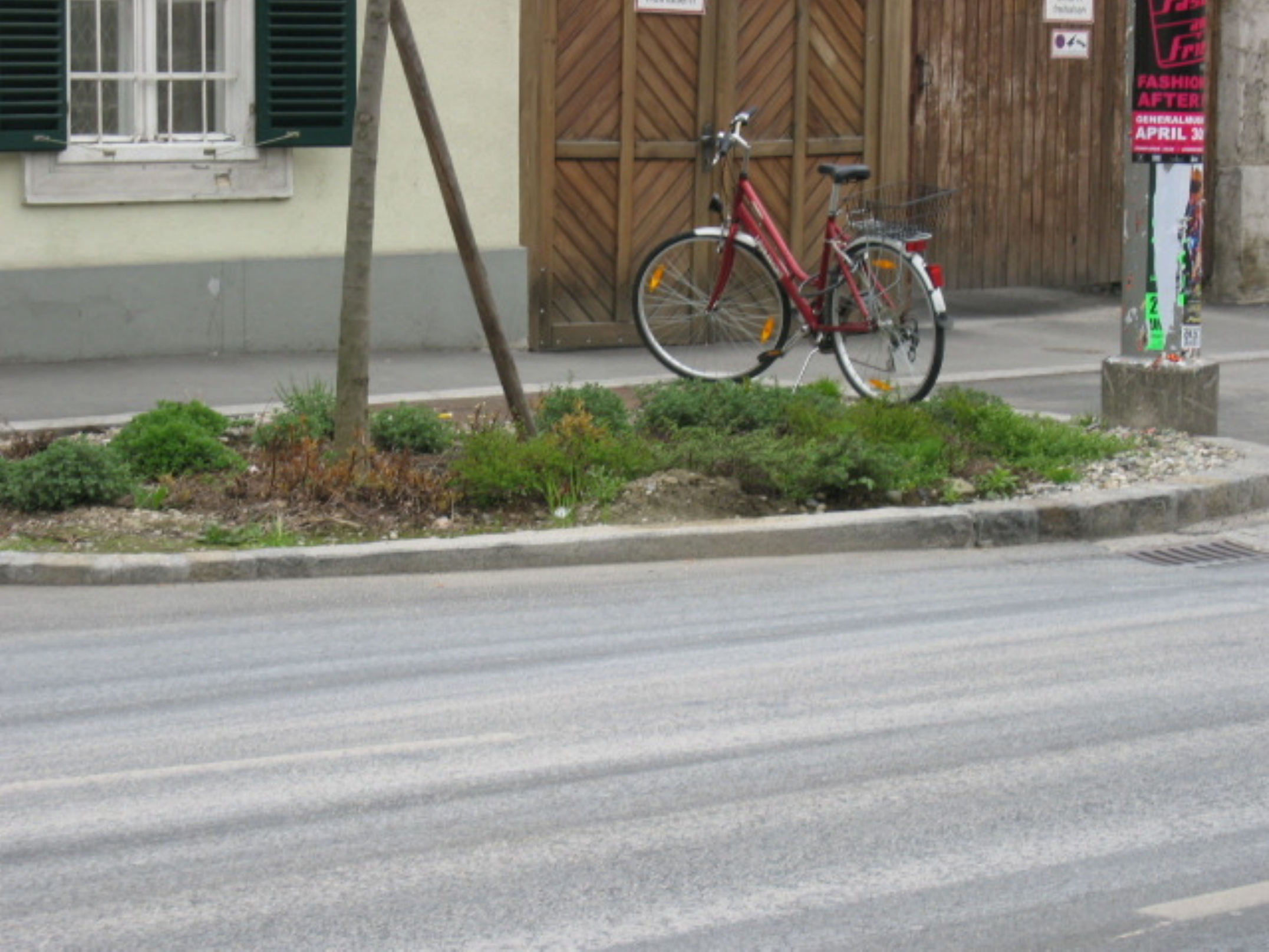} 
   \includegraphics[width=0.24\linewidth]{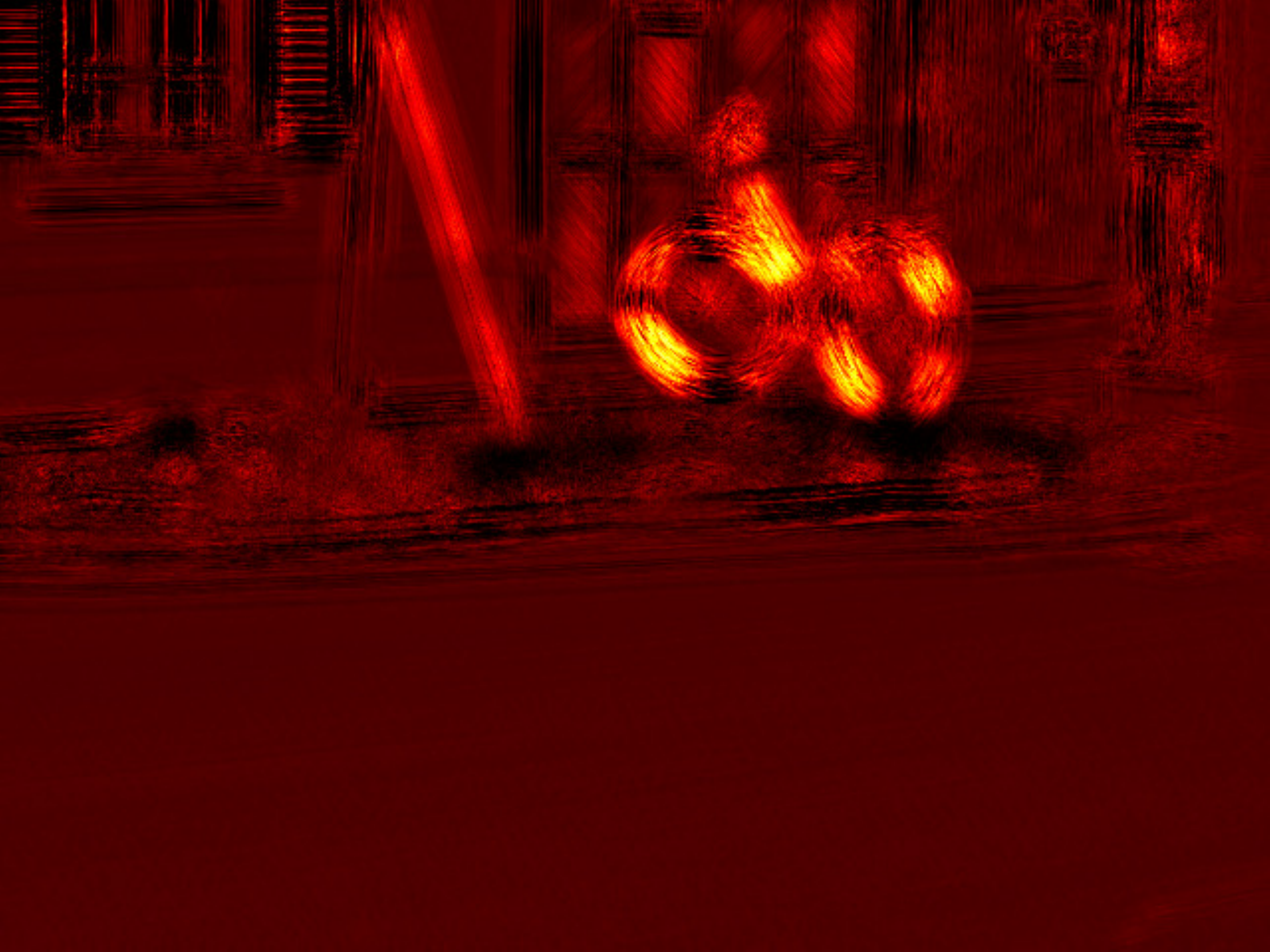} \hfill
   \includegraphics[width=0.24\linewidth]{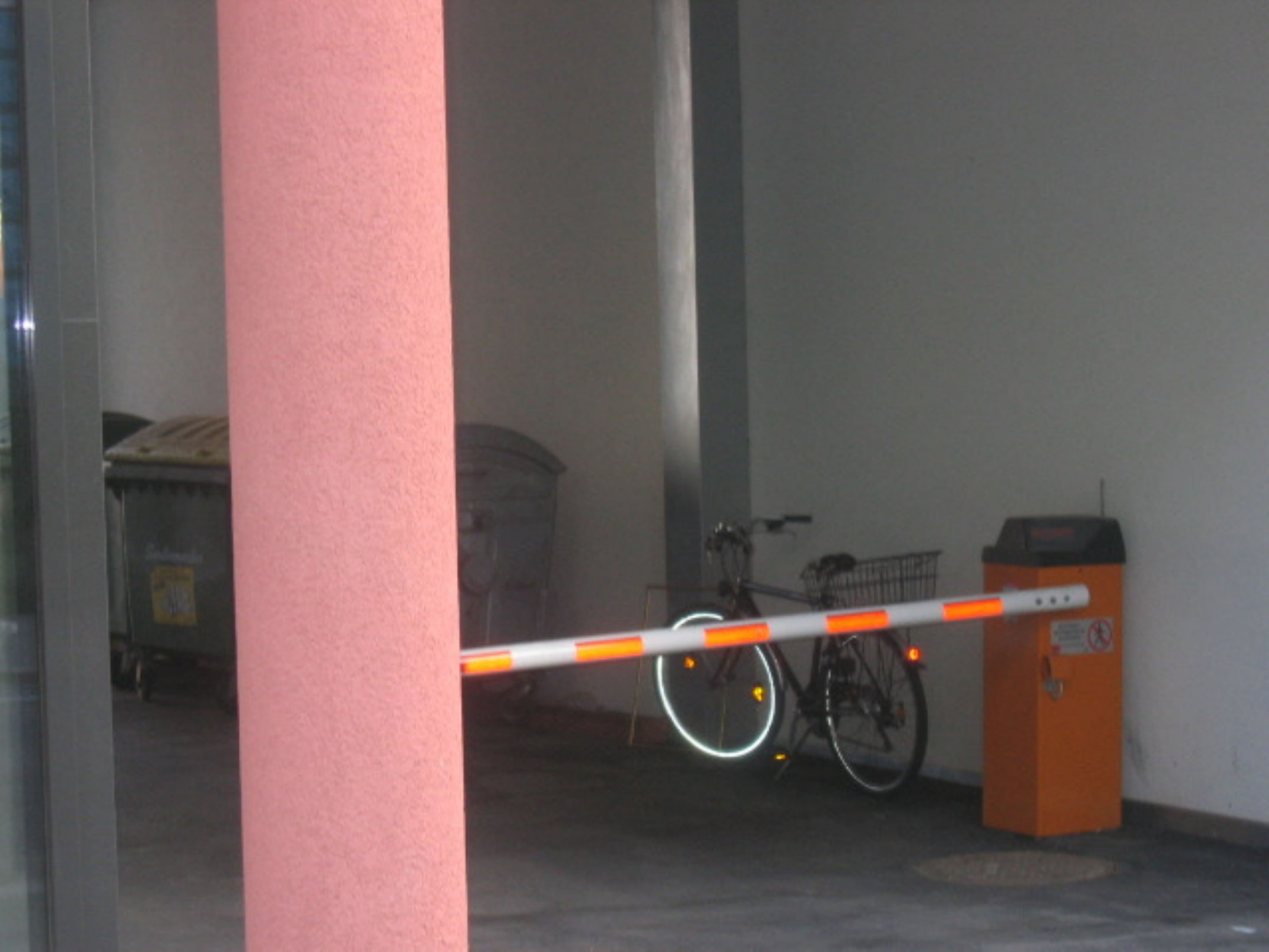} 
   \includegraphics[width=0.24\linewidth]{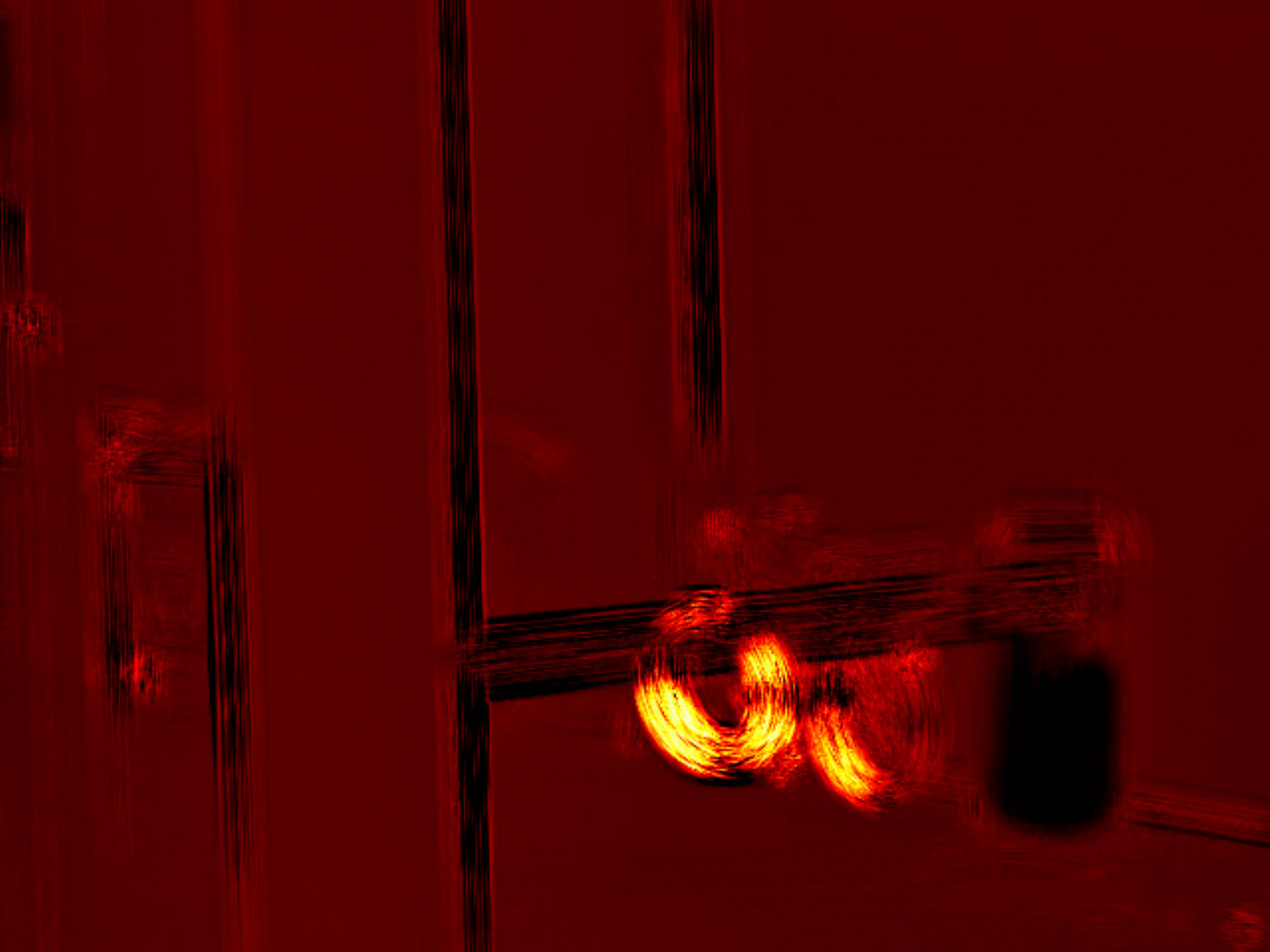}

   \vspace*{0.2cm}
   \includegraphics[width=0.24\linewidth]{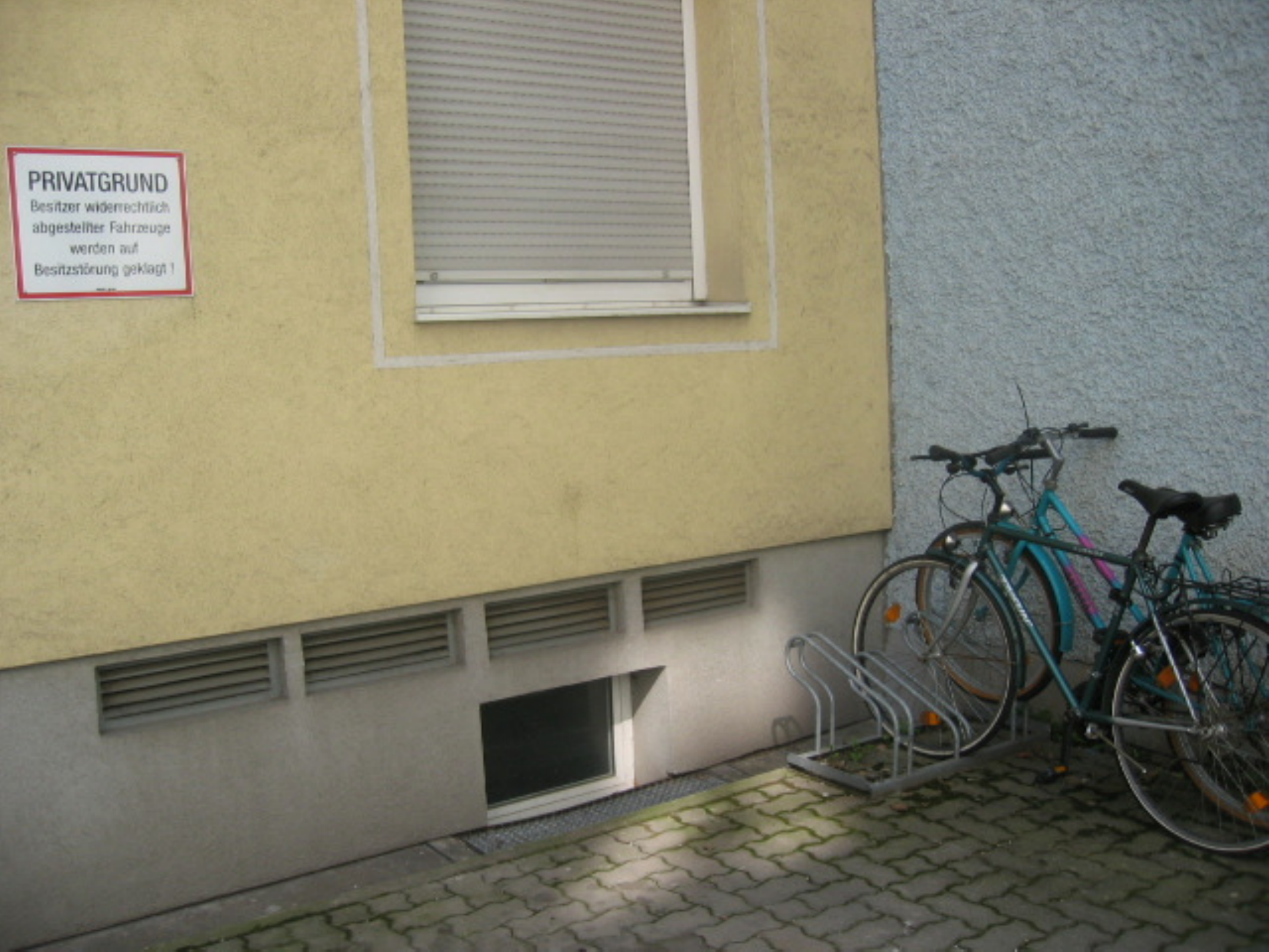} 
   \includegraphics[width=0.24\linewidth]{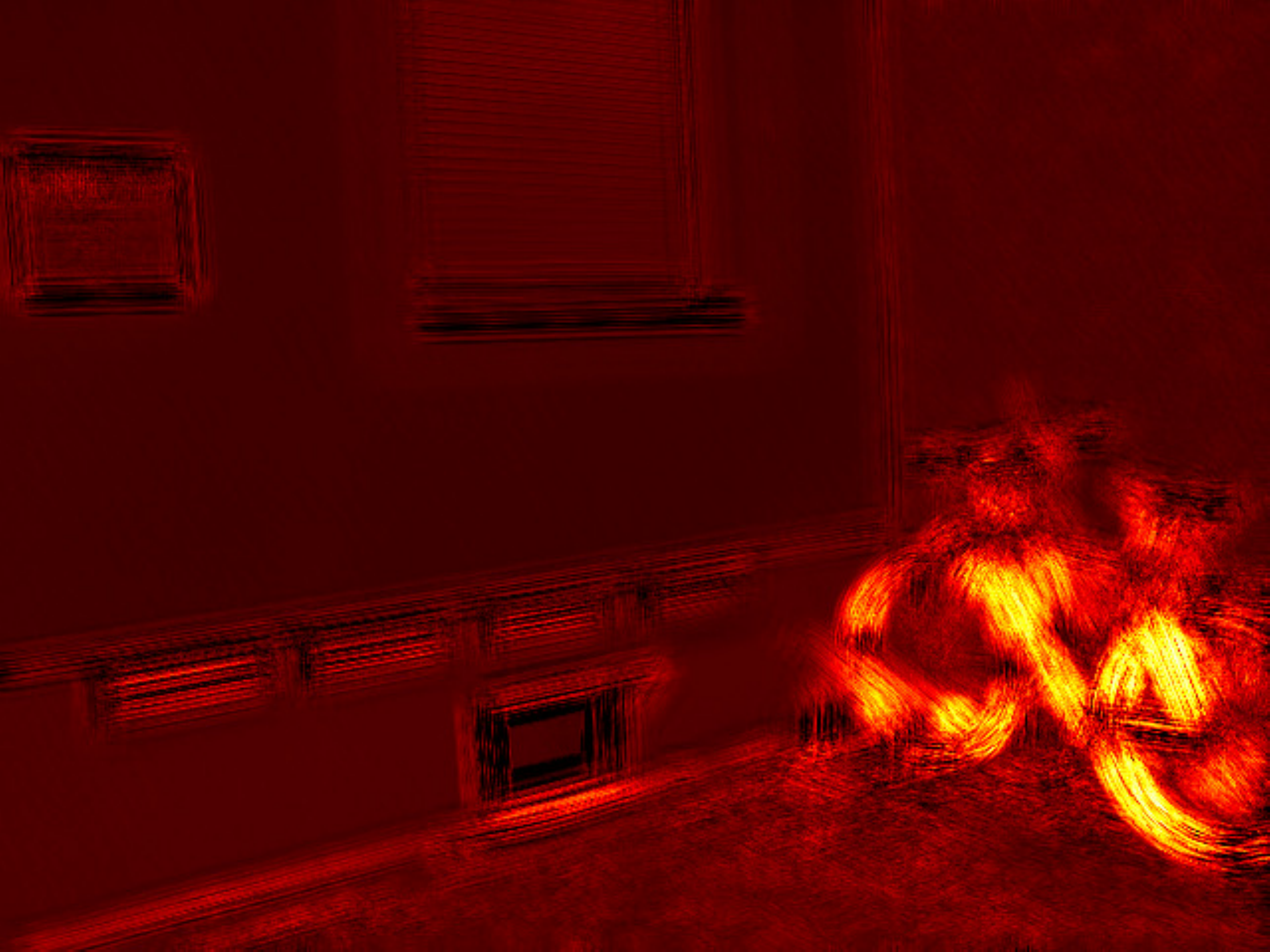} \hfill
   \includegraphics[width=0.24\linewidth]{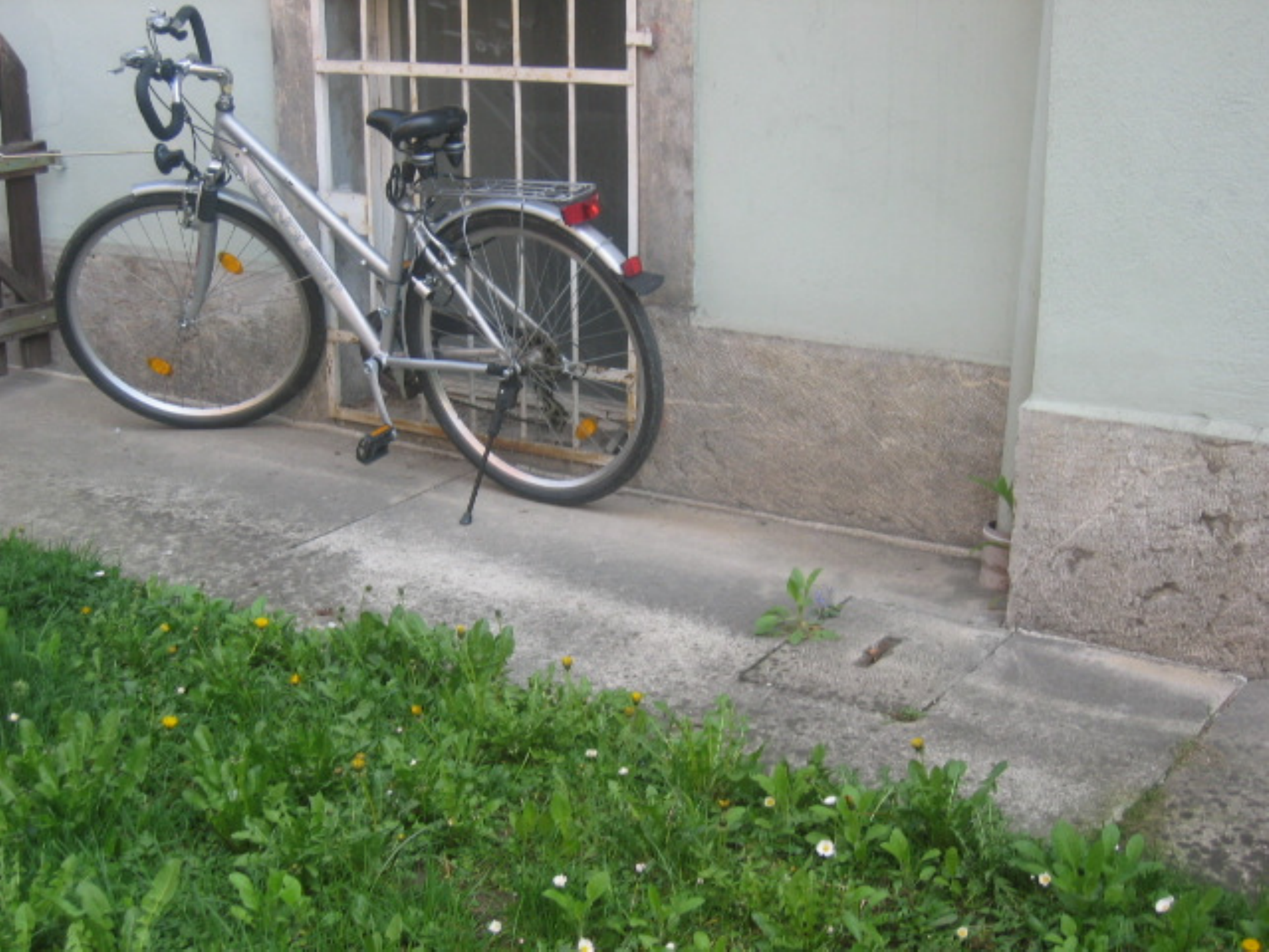} 
   \includegraphics[width=0.24\linewidth]{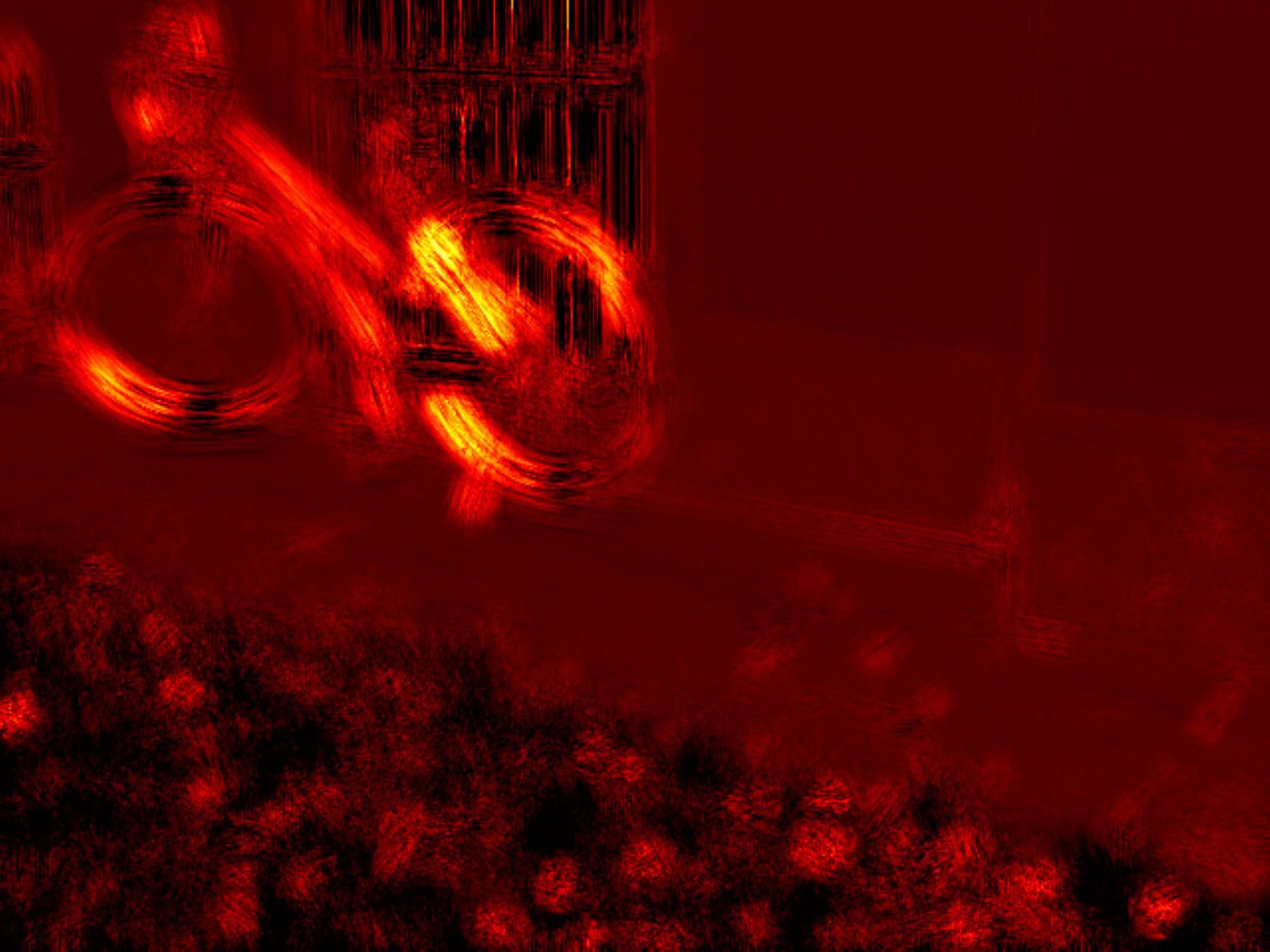} 
   \caption{Results of pixelwise classification of test images
      from~\citet{mairal2008c}. Two dictionaries are learned beforehand, one
      from images containing a bicycle, and one from images corresponding to
      background only. For each of the four test images displayed here, we also
      display a confidence map for the presence of a bicycle. For each pixel, 
      a confidence value is computed for a patch centered at the pixel, by
      comparing the residual of the reconstruction errors obtained with the two
      dictionaries; yellow bright pixel values represent high confidence,
      whereas dark red pixel values represent low confidence. 
   }
\label{fig:patch_classif}
\end{figure}

Another successful application of dictionary learning is edge detection. A
first attempt by~\citet{mairal2008d} is based on the principles described in
this paragraph, and consists of learning two dictionaries, one on patches
centered on edges, and one centered on background. The method was trained on
the manually annotated Berkeley segmentation dataset~\citep{martin2001} and was
shown to perform as well as the dedicated approach called
``Pb''~\citep{martin2004}. Later, state-of-the-art results have been obtained
by~\citet{xiaofeng2012} by essentially combining two main ideas: (i) learning a
single dictionary and exploiting the sparse codes for classification, a
paradigm that will be the focus of the next paragraph, (ii) a dedicated way of
performing multiscale feature pooling, giving more stability and some
invariance to the sparse representations. We refer to~\citet{xiaofeng2012} for all 
details and we present some results obtained with their software package in
Figure~\ref{fig:edges}.\footnote{The software package is available
here: \url{http://homes.cs.washington.edu/~xren}.}
In the next paragraph, we now move to the second paradigm, namely exploiting
the sparse codes as a nonlinear transformation of the input data for a
subsequent linear classifier.

\begin{figure}[hbtp]
   \includegraphics[width=0.24\linewidth]{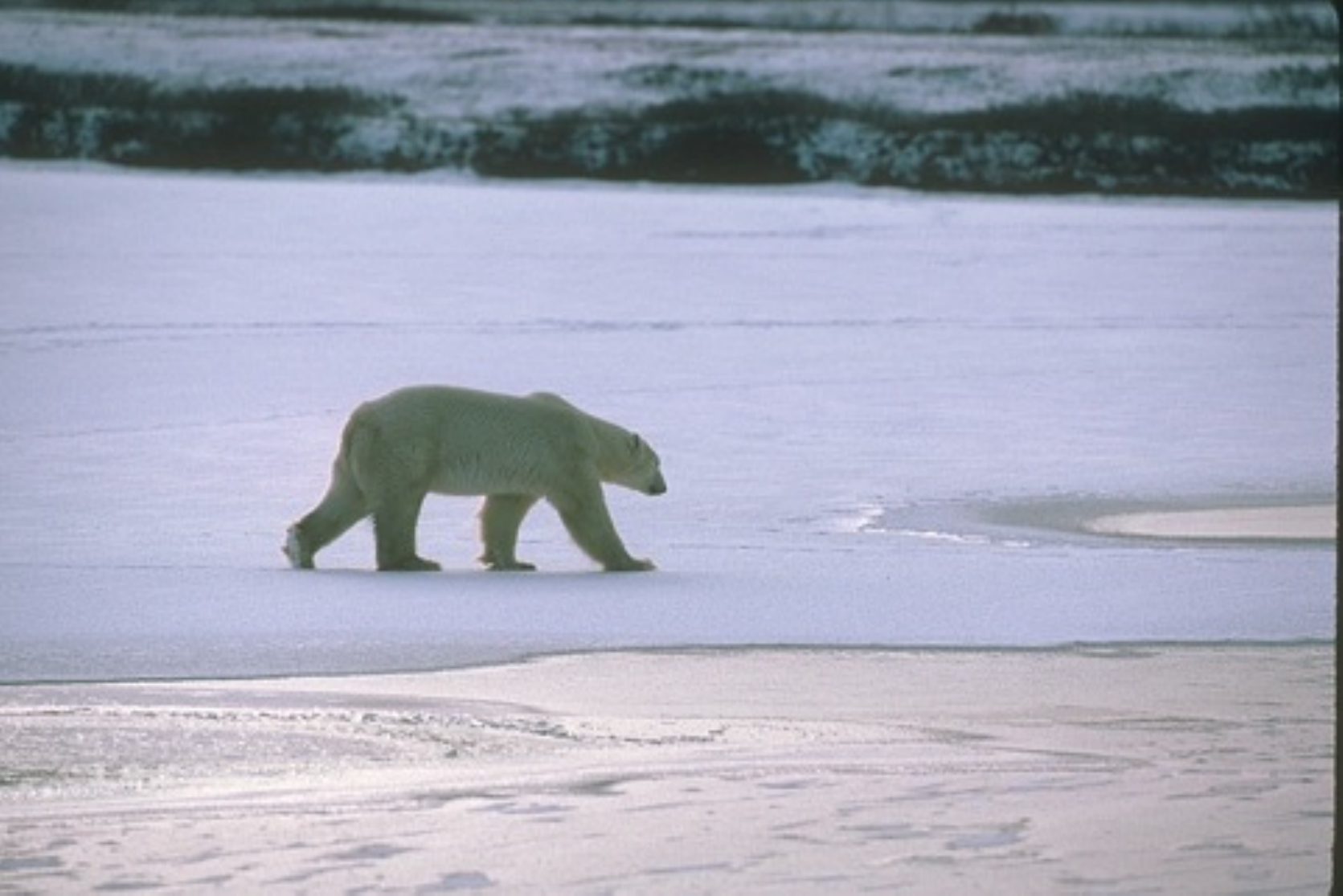} 
   \includegraphics[width=0.24\linewidth]{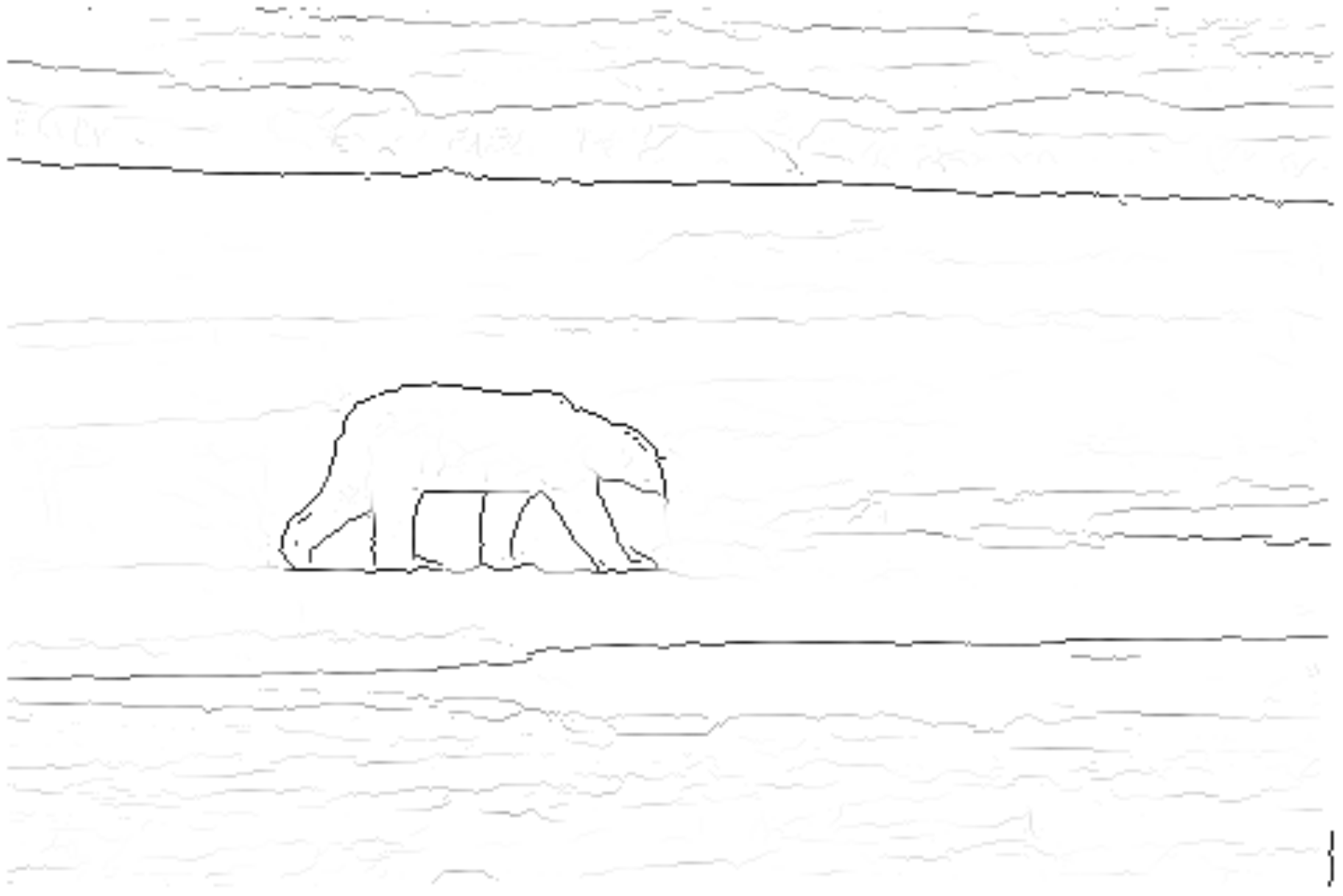} \hfill
   \includegraphics[width=0.24\linewidth]{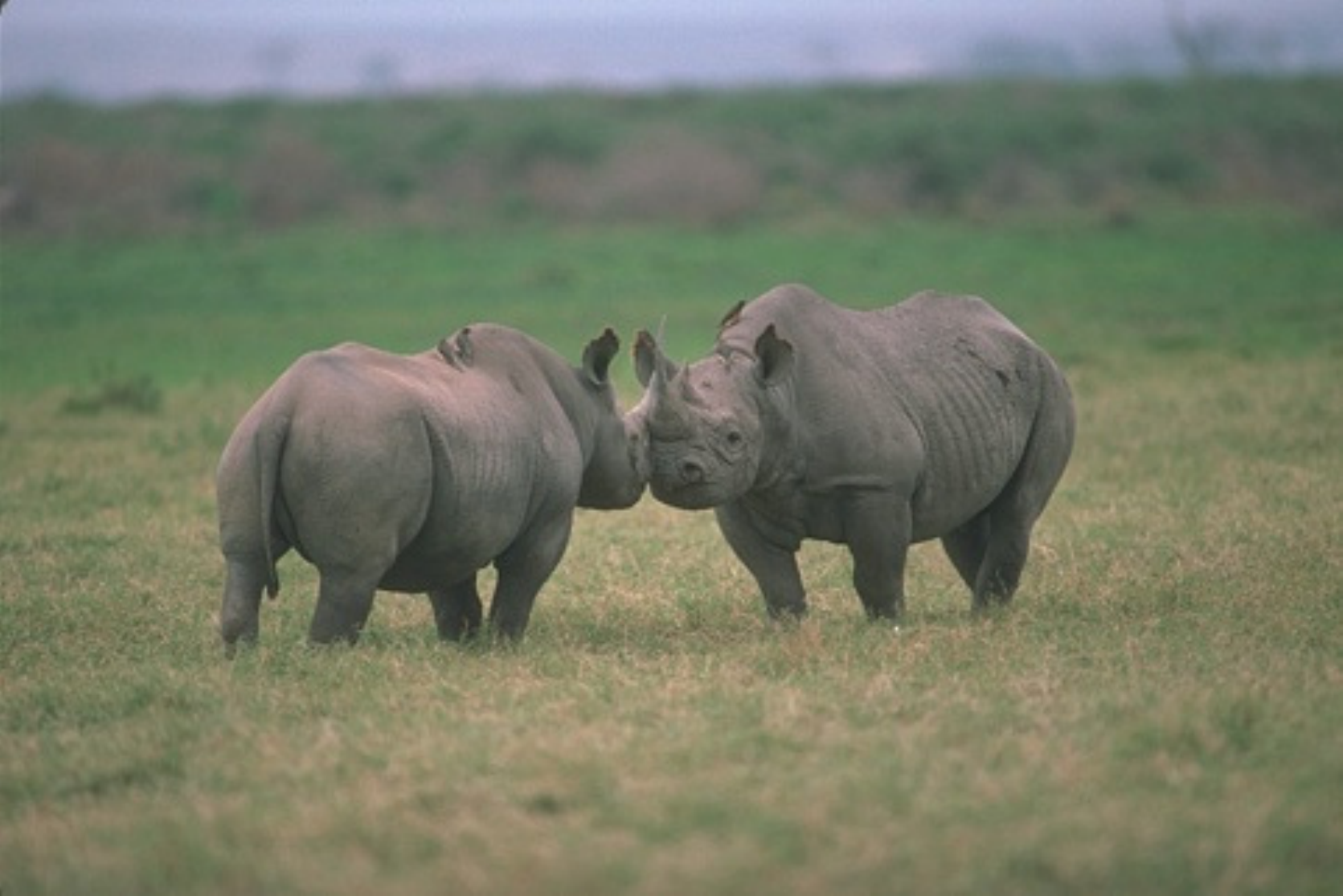} 
   \includegraphics[width=0.24\linewidth]{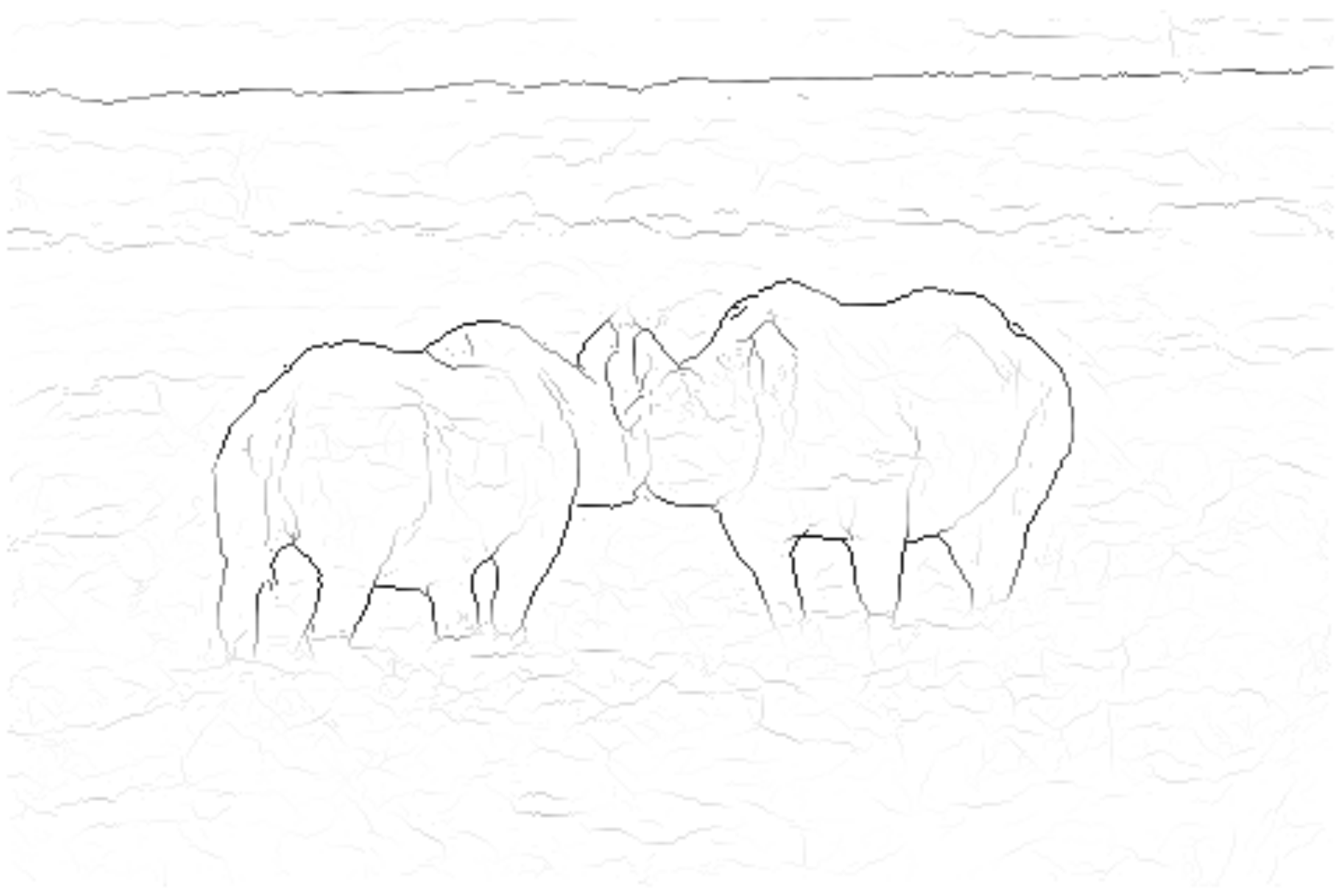}

   \vspace*{0.2cm}
   \includegraphics[width=0.24\linewidth]{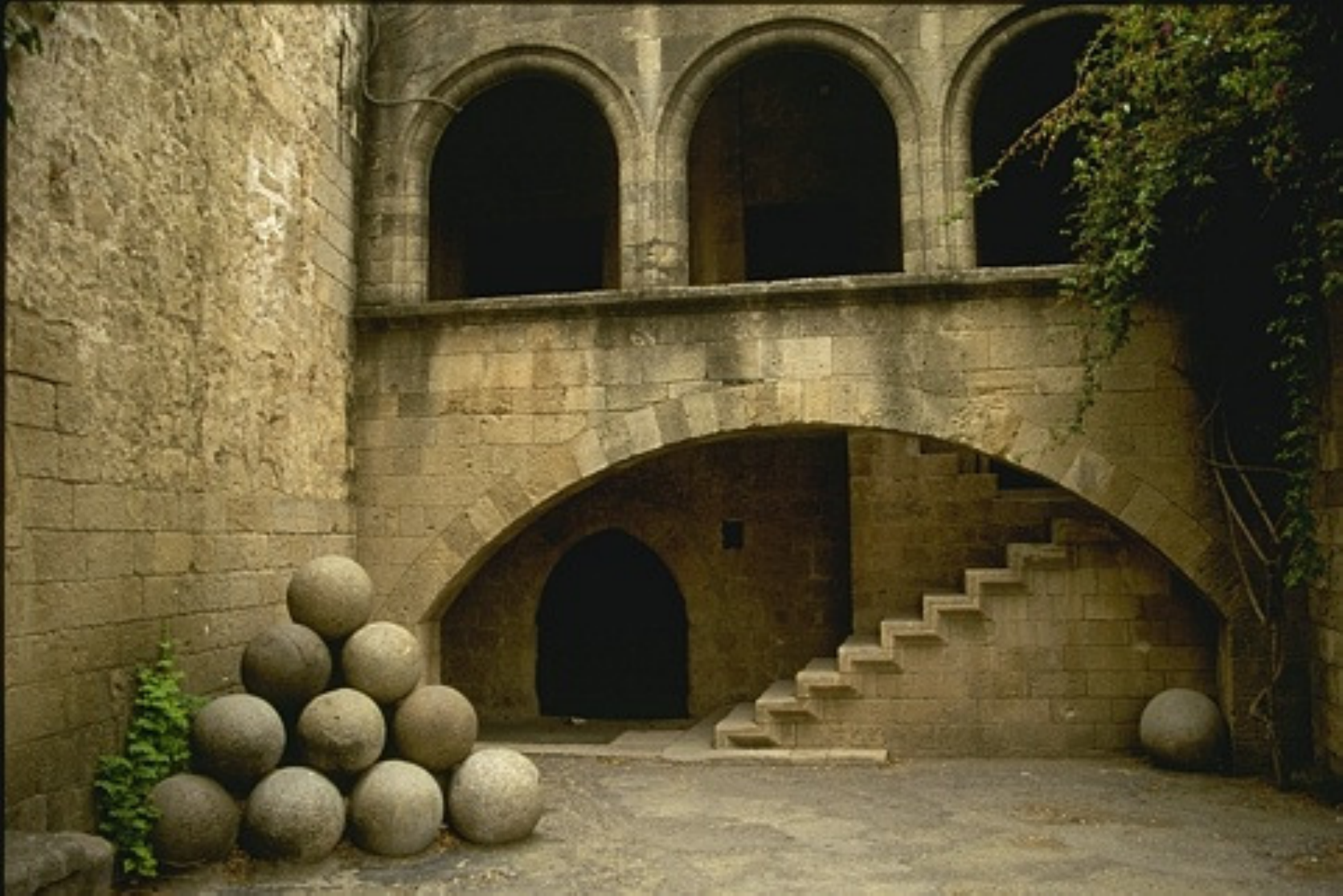} 
   \includegraphics[width=0.24\linewidth]{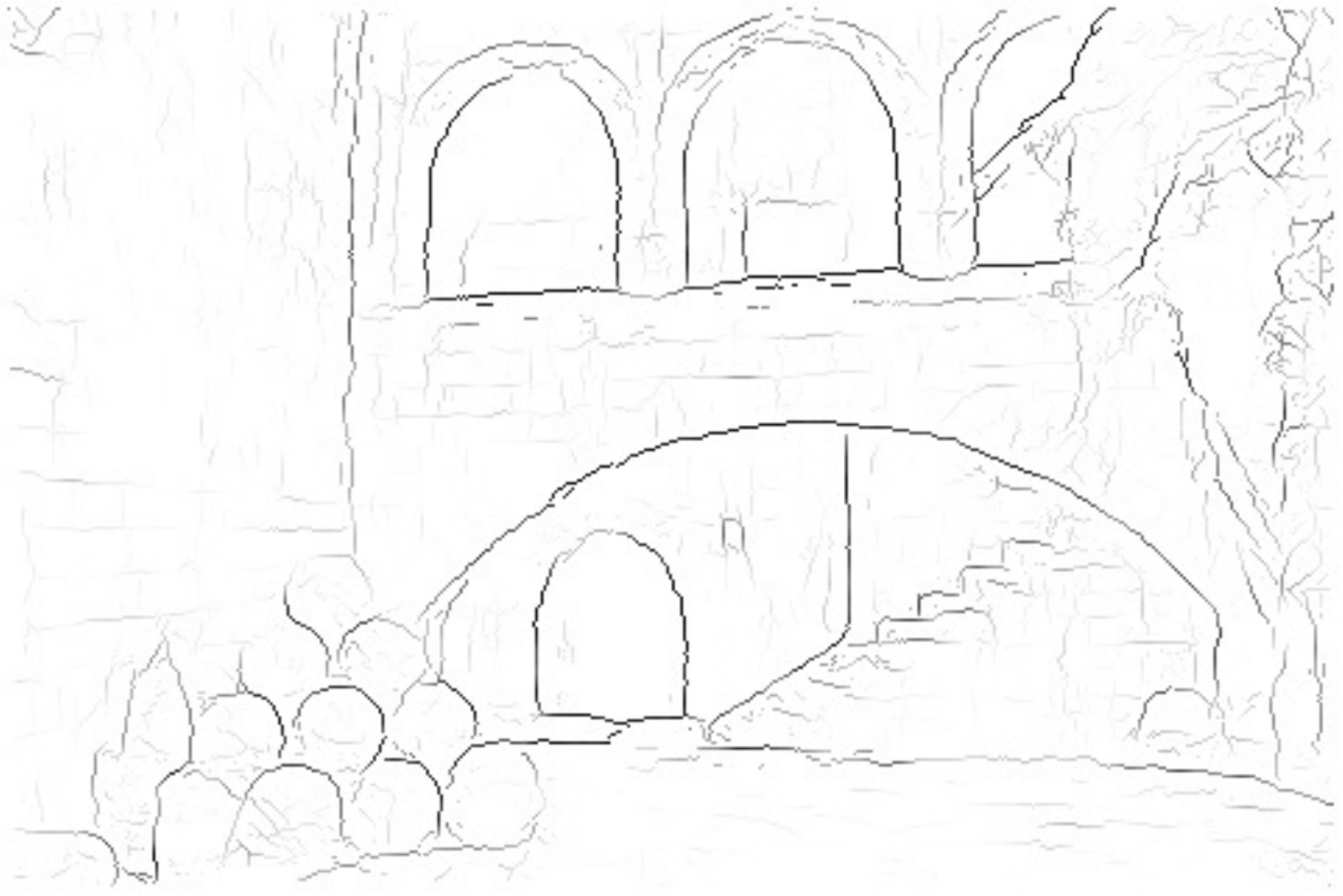} \hfill
   \includegraphics[width=0.24\linewidth]{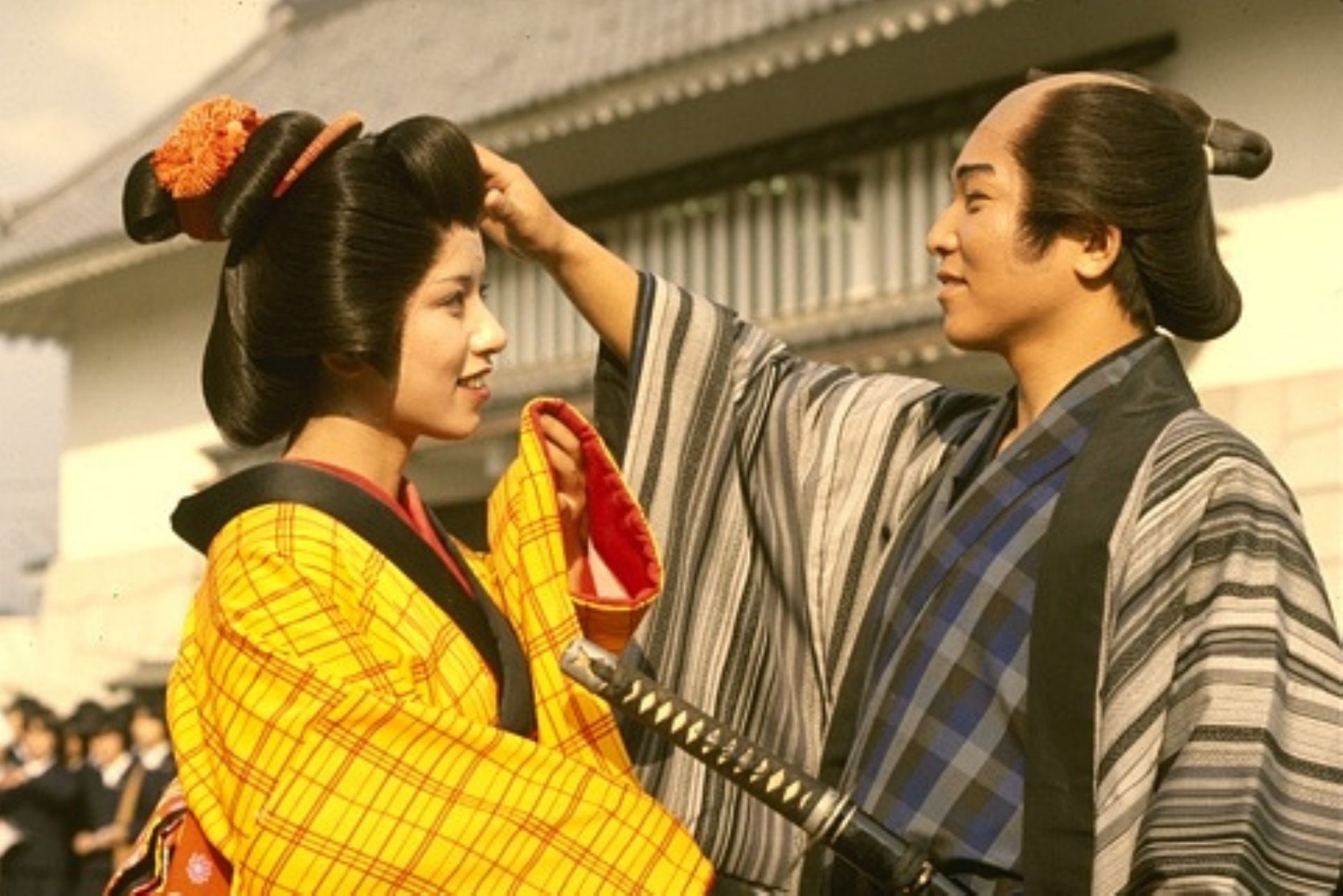} 
   \includegraphics[width=0.24\linewidth]{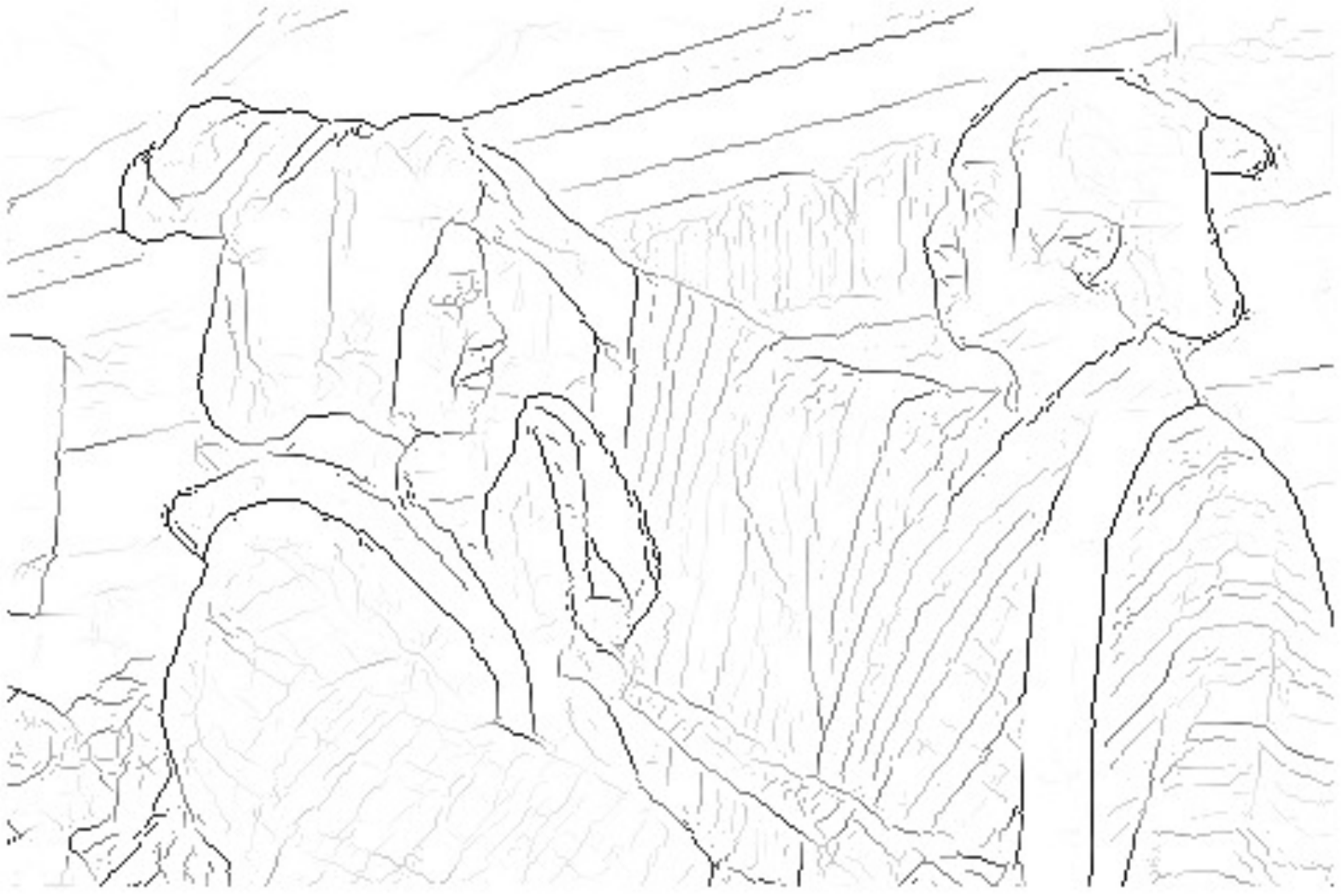} 
   \caption{Edge detection results obtained with the software package
      of \citet{xiaofeng2012} on the Berkeley segmentation dataset
      BSD500 \citep{martin2001}. Four pairs of images are presented. The left
      ones represent input images and the right ones confidence maps for the
      presence of edges. Dark values represent high confidence.
      Best seen by zooming on a computer screen.
   }
\label{fig:edges}
\end{figure}

\paragraph{Using sparse codes for classification.}
Given a signal~$\x$ in~$\Real^m$, the sparse coding principle produces a
sparse vector~$\alphab^\star$ such that $\x \approx \D\alphab^\star$ for some
dictionary~$\D$ in~$\Real^{m \times p}$. The process transforming~$\x$
into~$\alphab^\star$ is highly nonlinear, and the entries of the
vector~$\alphab^\star$ can be interpreted as the ``contributions'' of the
corresponding columns of~$\D$ to the reconstruction of~$\x$.

When the dictionary atoms represent discriminative features for a
classification task, it becomes appealing to use the vectors~$\alphab^\star$ as
new inputs to a classifier, \eg, linear SVM or logistic regression. The
strategy is thus significantly different than the one presented in the previous
paragraph.  The proof of concept was first demonstrated by~\citet{huang2006} with a
pre-defined fixed dictionary and a cost function inspired from linear
discriminant analysis~\citep[see][]{hastie2009}. Later, dictionary
learning was successfully used by~\citet{raina2007} for unsupervised feature
learning as an effective way of exploiting unlabeled data for various
classification tasks with image and text modalities. 

Then, joint cost functions that involve both a discriminative term and a
classical dictionary learning formulation have been proposed several times in
the literature~\citep{rodriguez2008,mairal2008e,pham2008,yang2011fisher}.  Let us
consider a training set~$(\x_i,y_i)_{i=1}^n$, where the vectors $\x_i$
in~$\Real^m$ are input signals and the scalars $y_i$ are associated labels.
The simplest cost functions that combine sparse coding and classification have
the following form
\begin{equation}
   \min_{\substack{\D \in \CC, \W \in \Real^{p \times k} \\  \A \in \Real^{p \times n}}}  \frac{1}{n}\sum_{i=1}^n \left[ \frac{1}{2}\|\x_i-\D\alphab_i\|_2^2 \!+\! \lambda \|\alphab_i\|_1 \!+ \!\gamma L( y_i , \W,\alphab_i) \right]  + \frac{\mu}{2}\|\W\|_{\text{F}}^2, \label{eq:discr_dict}
\end{equation}
where the loss function~$L$ measures how far a prediction based on
$\alphab_i$ is from the correct label~$y_i$. The matrix~$\W$ represents model
parameters for the classifier and the quadratic term
$({\mu}/{2})\|\W\|_{\text{F}}^2$ is a regularizer.  Typically,~$L$ is a convex
function and the prediction model is linear. For example, when the labels~$y_i$
are in~$\{-1,+1\}$ for a binary classification problem, we may want to learn a
linear decision rule parameterized by a vector~$\w$ (meaning the matrix~$\W$ has
a single column), and use one of the following loss functions
\begin{equation}
   L(y,\w,\alphab) \defin \left\{ 
      \begin{array}{ll} 
         \log\left(1+ e^{ -y \w^\top \alphab}\right)  & ~\text{(logistic loss)}, \\
         \max\left(0,1-y \w^\top \alphab\right)  & ~\text{(hinge loss)}, \\
         \frac{1}{2}\left(y - \w^\top \alphab\right)^2  & ~\text{(square loss)}. \\
      \end{array}
      \right. \label{eq:loss}
\end{equation}
For simplicity, we omit the fact that the loss function may involve an
intercept---that is a constant term~$b$ in~$\Real$ such that the linear model
is $y \approx \sign( \w^\top \alphab + b)$ instead of simply~$y \approx
\sign(\w^\top \alphab)$.
Depending on the chosen loss function~$L$, the
formulation~(\ref{eq:discr_dict}) can combine dictionary learning with a
linear SVM, or a logistic regression classifier. 
We present the three loss functions~(\ref{eq:loss}) in Figure~\ref{fig:loss}.
The logistic and hinge loss are very similar; both of them are asymptotically linear and encourage the signs
of the products~$\w^\top \alphab_i$ to be the same as~$y_i$.
In all cases, it is possible
to iteratively decrease the value of the cost function~(\ref{eq:discr_dict})
and obtain a stationary point by alternate minimization between the
variables~$\D,\A,\W$. 
As for the classical dictionary learning formulation, it is in fact
impossible to find the global optimum;
the problem is indeed convex with respect to each variable when
keeping the other variables fixed, but it is not jointly convex. 
\begin{figure}
   \centering
    \includegraphics[page=4]{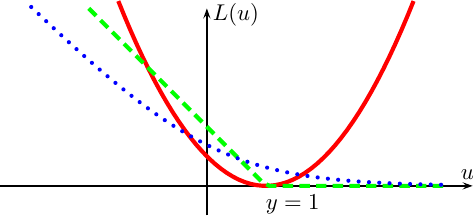}
\caption{Visualization of three loss functions for a positive label~$y=1$. The
   plain red curve represents the square loss $L(u)=(1/2)(y-u)^2$, the dotted blue
one represents the logistic loss $L(u)=\log(1-e^{-y u})$, and the green dotted
one represents the hinge loss $L(u)=\max(0,1-y u)$.}\label{fig:loss}
\end{figure}

The motivation of the formulation~(\ref{eq:discr_dict}) is to encourage the
dictionary to provide sparse decomposition vectors~$\alphab$ that are good for
discrimination. It does so by encouraging them to have a small classification
error on the training set measured by a convex loss function from machine
learning. There are however two possible caveats with such an approach, which
we discuss now.

\subparagraph{The sparse decomposition patterns may be unstable.}
The first caveat is the lack of stability of the estimator provided by sparse
decomposition algorithms. For instance, the solution of the Lasso can
significantly change after a small perturbation of the input data---that is,
the vector 
\begin{equation}
   \alphab^\star(\x,\D) \defin \argmin_{\alphab \in \Real^p} \frac{1}{2}\|\x-\D\alphab\|_2^2 + \lambda\|\alphab\|_1, \label{eq:alphastar}
\end{equation}
assuming the optimization problem has a unique solution, can drastically change
for small variations of~$\x$, or even small variations of~$\lambda$ in pathological
cases~\citep{meinshausen2010,mairal2012b}. For classification tasks, it is
unfortunately often desirable to have stable representations, and thus
improving this aspect may be important~\citep{bruna2013}.
Several strategies have been proposed to deal with such an issue:
\begin{itemize}
   \item {\bfseries promoting incoherent dictionaries:} when~$\D^\top\D$ is close
      to the identity matrix, the dictionary is called ``incoherent''. When it
      is the case, sparse reconstruction algorithms may benefit from support recovery
      guarantees~(see Section~\ref{sec:compressedsensing}).  Following this
      insight, \citet{ramirez2009sparse} have shown that promoting incoherent
      dictionaries may lead to better generalization properties. Later, this
      idea was also found useful in other classification
      schemes~\citep{bo2013}.
   \item {\bfseries pooling features:} when a set of signals that are close to 
      each other are available, \eg, a set patches that significantly overlap with each other,
      feature pooling on the set of corresponding sparse codes
      can lead to more stable, but less sparse representations. For
      instance,~\citet{ren2013,bo2013} adopts such a strategy in an
      effective image representation for object detection.
   \item {\bfseries replacing~$\ell_1$ by the elastic-net:} when replacing
      the~$\ell_1$-norm in~(\ref{eq:alphastar}) by the Elastic-net penalty
      $\|\alphab\|_1+(\gamma/2)\|\alphab\|_2^2$, it is possible to show that
      $\alphab^\star$ is Lipschitz continuous~\citep{mairal2012}, with a
      Lipschitz constant proportional to~$(1/\gamma)^2$. Here, the
      parameter~$\gamma$ controls a trade-off between sparsity and
      stability. As a side effect, it also leads to a unique solution of the
      sparse reconstruction problem due to the strong convexity of the penalty.
\end{itemize}

\subparagraph{The label~$y$ is not available at test time.}
Interestingly, the formulation~(\ref{eq:discr_dict}) is related to 
the super-resolution approach described in Section~\ref{subsec:upscaling}. even though
the relation is only clear from hindsight.
Indeed, the goal of classification is to map an input signal~$\x$ to 
an output label~$y$. In contrast, super-resolution can be seen as a 
multivariate regression problem, whose goal is to learn a mapping between~$\x$
and a high-resolution signal---say, a signal~$\y$. It is then not surprising
that the classification formulation~(\ref{eq:discr_dict}) is very similar to
the regression one~(\ref{eq:upscaling}). In both cases, learning the mapping
involves a dictionary~$\D$ for representing the input signals~$\x$ and a model
that maps sparse codes~$\alphab$ to the variable to predict.
In~(\ref{eq:discr_dict}), the parameters~$\W$ play for instance the same role
as the parameters $\D_h$ in~(\ref{eq:upscaling}).

In Section~\ref{subsec:upscaling}, we have discussed some limitations of the
super-resolution formulation~(\ref{eq:upscaling}), which are in fact also
relevant for the classification task. Given some fixed~$\D$ and~$\W$, we remark
that the sparse codes~$\alphab_i$ for the training data may be obtained as
follows
\begin{equation}
   \alphab_i \in \argmin_{\alphab \in \Real^p} \frac{1}{2}\|\x_i-\D\alphab\|_2^2 \!+\! \lambda \|\alphab\|_1 \!+ \!\gamma L( y_i , \W,\alphab). \label{eq:alphai}
\end{equation}
Unfortunately, given a new test signal~$\x$, the label~$y$ is not available and
the formulation~(\ref{eq:alphastar}) has to be used instead to obtain the
corresponding sparse code~$\alphab$. Therefore, there is a discrepancy in the
way the signals are encoded during the training and test phases.

In Section~\ref{subsec:upscaling}, we have addressed that issue by presenting
an extension involving a bilevel optimization problem. It is also possible to
follow the same strategy for the classification task, but we develop this
method in the next section since it draws a strong connection between
dictionary learning and neural networks.

%% file: content_arxiv/visual_backprop.tex
There are obvious connections between dictionary learning and neural
networks. Some of them have already appeared in this monograph with
multilayer schemes involving feature maps, which are typical architectures of
convolutional neural networks and related work~\citep{lecun1998,serre2005}. In
this section, we present two other important connections. First, we introduce
the concept of ``backpropagation'' for dictionary learning, which is directly
borrowed from neural networks~\citep{lecun2}, and which addresses an issue that
was left opened at the end of the previous section.  Then, we focus on particular
networks whose purpose is to approximate efficiently the solution of sparse
coding problems~\citep{gregor2010}.

\paragraph{Backpropagation rules for dictionary learning.}
In Figure~\ref{fig:backprop}, we present two paradigms for dictionary learning
coupled with a prediction task, \eg, regression or classification. In the
first one, the dictionary is obtained without supervision and the prediction model is
learned in a subsequent step. In the latter paradigm, all parameters, including the
dictionary, are learned jointly for the prediction task. This requires a
similar concept as ``backpropagation''~\citep{lecun2}, which consists of using
the chain rule in neural networks.

Formally, let us consider a training set of
signals~$\X=[\x_1,\ldots,\x_n]$ in~$\Real^{m \times n}$ associated to 
measurements~$\Y=[\y_1,\ldots,\y_n]$ in~$\Real^{k \times n}$ representing
the variables to predict at test time. Such a setting is quite general on
purpose: the vectors~$\y_i$ may represent for instance the high-resolution
patches from Section~\ref{subsec:upscaling} or the half-toned patches from
Section~\ref{subsec:invert}, but they may also represent a label for the classification
tasks of Section~\ref{sec:visual_edges}. In this last case,~$k=1$ and
the~$\y_i$'s are simply scalars.

In the formulation of interest, the signals~$\x_i$ are encoded with~$\D$:
\begin{equation}
   \alphab^\star(\x,\D) = \argmin_{\alphab \in \Real^p} \frac{1}{2}\|\x-\D\alphab\|_2^2+\lambda\|\alphab\|_1 + \frac{\mu}{2}\|\alphab\|_2^2, \label{eq:alphastar2}
\end{equation}
where the use of the elastic-net penalty instead of~$\ell_1$ will be made clear in
the sequel. Then, we consider a similar prediction model as in the previous section;
following the same methodology, we introduce a loss function~$L$, and the
corresponding empirical risk minimization problem 
\begin{equation}
   \min_{\D \in \CC, \W \in \Real^{p \times k}} \frac{1}{n}\sum_{i=1}^n L\left( \y_i, \W, \alphab^\star(\x_i,\D)\right)  + \frac{\gamma}{2}\|\W\|_{\text{F}}^2, \label{eq:taskdriven2}
\end{equation}
Specifically, any loss from the previous section is considered to be
appropriate here.  The advantage of~(\ref{eq:alphastar2})
over~(\ref{eq:alphai}) is that the
encoding scheme remains the same between training and testing, and yet~$\D$
and~$\W$ can be learned jointly for the final prediction task
in~(\ref{eq:taskdriven2}). Moreover~(\ref{eq:taskdriven2}) extends the
super-resolution formulation of~(\ref{eq:taskdriven}) to more general problems
such as classification.

Even though~(\ref{eq:taskdriven2}) seems appealing, it is unfortunately 
particularly difficult to solve. In addition to being nonconvex,
the relation between~$\D$ and the objective function goes through the argmin of
a nonsmooth optimization problem. As a consequence, the cost function is also
nonsmooth with respect to~$\D$ and it is not clear how to obtain a descent direction.

In an asymptotic regime where enough training data is available,
\citet{mairal2012} show that the cost function becomes differentiable, and they
derive a stochastic gradient descent algorithm from this theoretical result.
They call this approach ``task-driven dictionary learning''.
More precisely, consider the expected cost
\begin{displaymath}
   f(\D,\W) \defin \EE_{(\y,\x)}\left[ L\left( \y, \W, \alphab^\star(\x,\D)\right) \right], 
\end{displaymath}
where the expectation is taken to the unknown probability distribution of the
data pairs~$(\y,\x)$. Under mild assumptions on the data distribution and when~$\mu > 0$, it is
possible to show that~$f$ is differentiable and that its gradient admits a
closed form:
\begin{equation}
   \left\{
      \begin{array}{rcl}
         \nabla_{\W} f(\D,\W)  & = & \EE_{(\y,\x)}\left[ \nabla L(\y,\W, \alphab^\star)^\top  \right], \\
         \nabla_{\D} f(\D,\W)  & = & \EE_{(\y,\x)}\left[ -\D \betab^\star \alphab^{\star \top} + (\x-\D\alphab^\star)\betab^{\star \top}  \right], \\
      \end{array}
      \right. \label{eq:backprop}
\end{equation}
where~$\alphab^\star$ is short for~$\alphab^\star(\x,\D)$,
\begin{displaymath}
   \betab^\star[\Gamma^\complement] = 0,~\text{and}~ \betab^\star[\Gamma] \defin \left(\D_\Gamma^\top \D_\Gamma + \mu \I  \right)^{-1}\nabla_{\alphab[\Gamma]} L(\y,\W,\alphab^\star),
\end{displaymath}
where~$\Gamma$ is the set of nonzero entries of~$\alphab^\star$. 
Here,~the role of parameter~$\mu$ is to ensure that the matrix to be inverted is well conditioned.
Since the gradient has the form of an expectation, it is natural to use a
stochastic gradient descent algorithm to optimize the cost
function~\citep{bottou2008tradeoffs}. Even though the nonconvexity makes it
impossible to find the global optimum, it was shown by~\citet{mairal2012}
that following heuristics from the neural network literature~\citep{lecun2},
this strategy yields sufficiently good results for many prediction tasks.

The concept of learning~$\D$ to be good for prediction is similar to the
backpropagation principle in neural networks where all the parameters of the
network are learned for a prediction task using the chain rule~\citep{lecun2}. To the best 
of our knowledge, \citet{bradley2008} were the first to use such a 
terminology for dictionary learning; they proposed indeed a
a different supervised dictionary learning formulation with smoothed
approximations of sparsity-inducing penalties. In a heuristic fashion,
similar rules as~(\ref{eq:backprop}) have also been derived
by~\citet{boureau2010} and~\citet{yang2010b} in the context of classification
and by~\citet{yang2012} for super-resolution (see
Section~\ref{subsec:upscaling}).

\tikzstyle{block_red} = [draw=black!0, fill=red!40, rectangle, rounded corners,text width=2.0cm, text badly centered,node distance=2.5cm]
\tikzstyle{block_green} = [draw=black!0, fill=green!40, rectangle, rounded corners,text width=4cm, text badly centered,node distance=1.5cm]
\tikzstyle{block_blue} = [draw=black!0, fill=blue!40, rectangle, rounded corners,text width=4cm, text badly centered,node distance=1.5cm]
\tikzstyle{block_neutral} = [rectangle, rounded corners,text width=4cm, text badly centered,node distance=1.5cm,text height=0.39cm]
\tikzstyle{image} = [text centered, minimum height=1cm]
\tikzstyle{arrow} = [draw,line width=0.12cm, ->]
\tikzstyle{arrowb} = [draw,line width=0.12cm,color=red, ->]

\begin{figure}
   \subfigure[Without backprogagation.]{\label{subfig:nobackprop}
      \begin{tikzpicture}[node distance=0.6cm]
         \node[block_green](output){\bf Output $\y$};
         \node[block_green,below of=output,node distance=2.1cm](alphas){\bf Sparse codes $\alphab$};
         \node[block_green,left of=alphas,node distance=6cm](dict){\bf Dictionary $\D$};
         \node[block_green,left of=output,node distance=6cm](weights){\bf Weights $\W$};
         \node[block_green,below of=alphas,node distance=2.1cm](input){\bf Input vector $\x$};
         \path[arrow](input) -- (alphas);
         \path[arrow](dict) -- (alphas);
         \path[arrow](alphas) -- (output);
         \path[arrow](weights) -- (output);
         \draw [very thick,color=blue] (-8.3,-0.5) -- (-8.3,0.5) -- (2.3,0.5) -- (2.3,-2.6) -- (-2.4,-2.6) -- (-2.4,-0.5) -- (-8.3,-0.5);
         \draw [very thick,color=magenta,yshift=-1mm] (-8.3,-2.5) -- (-8.3,-1.5) -- (2.4,-1.5) -- (2.4,-4.6) -- (-2.6,-4.6) -- (-2.6,-2.5) -- (-8.3,-2.5);
         \node[block_neutral,below of=dict,node distance=1cm](layer1){\bf \color{magenta} Stage 1: sparse coding layer};
         \node[block_neutral,below of=weights,node distance=1cm](layer1){\bf \color{blue} Stage 2: supervised learning layer};
      \end{tikzpicture}
   }
   \subfigure[With backpropagation.]{\label{subfig:backprop}
      \begin{tikzpicture}[node distance=0.6cm]
         \node[block_green](output){\bf Output $\y$};
         \node[block_neutral,node distance=0.3cm,left of=output](neutral1){~};
         \node[block_green,below of=output,node distance=2.1cm](alphas){\bf Sparse codes $\alphab$};
         \node[block_neutral,node distance=0.3cm,left of=alphas](neutral2){~};
         \node[block_green,left of=alphas,node distance=6cm](dict){\bf Dictionary $\D$};
         \node[block_green,left of=output,node distance=6cm](weights){\bf Weights $\W$};
         \node[block_green,below of=alphas,node distance=2.1cm](input){\bf Input vector $\x$};
         \path[arrow](input) -- (alphas);
         \path[arrow](dict) -- (alphas);
         \path[arrow](alphas) -- (output);
         \path[arrow](weights) -- (output);
         \path[arrowb](neutral1) -- (neutral2);
         \node[block_neutral,below of=weights,node distance=1cm](layer1){\bf \color{red} Supervised Dictionary Learning};
         \draw [very thick,color=red] (-8.3,-2.6) -- (-8.3,0.5) -- (2.4,0.5) -- (2.4,-4.7) -- (-2.4,-4.7) -- (-2.4,-2.6) -- (-8.3,-2.6);
      \end{tikzpicture}
   }
   \caption{Two paradigms for dictionary learning are illustrated in the figure. An input signal~$\x$ is sparsely
      encoded by using a dictionary~$\D$, producing sparse codes~$\alphab$.
      Then, a second model with parameters~$\W$ makes a prediction~$\y$ based
      on~$\alphab$. In~\subref{subfig:nobackprop}, the
      dictionary~$\D$ is learned unsupervised in a first layer---that is,
      without exploiting output information~$\y$. Then, the parameters~$\W$ of
      the prediction model are learned afterwards in the second layer.
      In~\subref{subfig:backprop}, we illustrate the concept of supervised
      dictionary learning, where~$\D$ and~$\W$ are jointly learned for the
      prediction task; this requires ``backpropagating'' information,
   represented by the red arrow, from the top layer to the bottom one.}\label{fig:backprop}
\end{figure}
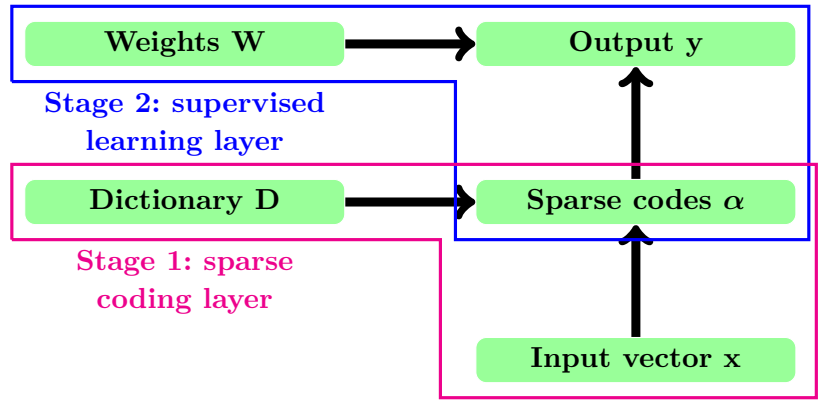
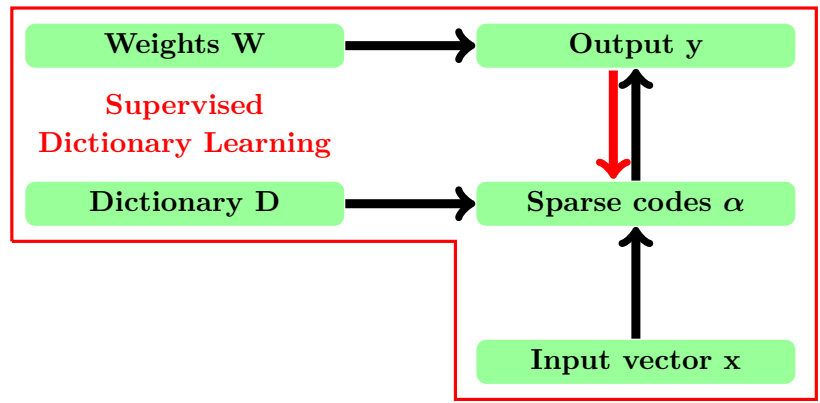

%% file: content_arxiv/visual_fast.tex
\paragraph{Fast approximations of sparse coding.}
Another link between neural networks and dictionary learning has been
established by~\citet{kavukcuoglu2010b} and~\citet{gregor2010}, who have trained
neural networks to approximate sparse codes.  Given some fixed dictionary~$\D$
and a training set~$\X$, their idea is to learn a simple non-linear
function $g(\W,\x)$ that approximates~$\alphab^\star(\x,\D)$ for any
signal~$\x$ of interest.
By using similar notation as in the previous paragraphs, the general
formulation proposed by~\citet{gregor2010} is the following
\begin{displaymath}
   \min_{ \W \in \WW }  \frac{1}{n}\sum_{i=1}^n \|\alphab^\star(\x_i,\D)- g(\W,\x_i)\|_2^2,
\end{displaymath}
where~$\W$ is a set of parameters for the prediction function~$g$ optimized with
stochastic gradient descent and~$\WW$ is an optimization domain. The basic prediction function~$g$ introduced
by~\citet{kavukcuoglu2010b} is the composition of a linear transformation and a
pointwise non-linearity, which can be interpreted as a one-layer neural network:
\begin{equation}
   g(\W,\x)[j] \defin \gamma_j \tanh(\w_j^\top \x + b_j), \label{eq:fastsc}
\end{equation}
where the scalars~$\gamma_j$ and~$b_j$ are some weights that can be also
optimized. The motivation for such an approach is to reduce the cost of computing
sparse codes at test time. Whereas obtaining~$\alphab^\star(\x,\D)$ requires
solving an optimization problem that can be costly (see Section~\ref{chapter:optim}), the complexity of computing $g(\W,\w)$ is the
same as a matrix-vector multiplication, making it possible to develop real-time
applications.

However, there are two major shortcomings to the original approach
of~\citet{kavukcuoglu2010b}, as noted by~\citet{gregor2010}. First, the choice
of the non-linearity $\tanh$ does not yield exactly sparse vectors, but it can
be replaced by more appropriate shrinkage functions. Second, the
scheme~(\ref{eq:fastsc}) encodes each entry of the output vector~$g(\W,\x)$
independently and cannot possibly take into account the fact that there exists
correlation among the dictionary elements. To address this
issue,~\citet{gregor2010} have proposed an iterative approach with a higher
complexity than~(\ref{eq:fastsc}), but with a better approximation (see their
paper for more details).

%% file: content_arxiv/visual_other.tex
Because the research topic has been very active in the last decade, sparse
estimation and dictionary learning have been used in numerous ways in computer
vision; unfortunately, writing an exhaustive review would be an endless
exercise, and we have instead focused so far on a few approaches related to
visual recognition in images. In the remaining paragraphs, we aim to be
slightly more exhaustive and briefly mention a few other successful
applications.

\paragraph{Action recognition in videos.} We start with natural extensions to
videos of visual recognition pipelines originally developed for images.
\citet{castrodad2012} and~\citet{guha2012} have proposed two related approaches
for action classification in movies. We present here the main principles from the
method of~\citet{castrodad2012}, which has shown competitive results in several
standard evaluation benchmarks.

Let us consider~$n$ training video sequences represented by vectors $\x_1,
\x_2, \ldots, \x_n$ in~$\Real^{lT}$, where~$T$ is the number of frames in the
sequence and~$l$ is the number of pixels in each frame. Each sequence contains
a subject performing a particular action and the goal is to predict the
type of action in subsequent test videos. We assume that there are~$k$
different possible actions.

To address this classification problem, \citet{castrodad2012} have
introduced a two-layer dictionary learning scheme, after pre-processing
the videos to replace the raw pixels by temporal gradients. Then, the
two layers of the pipeline can be described as follows:
\begin{itemize}
   \item {\bfseries local model of actions:} it is common in videos to work
      with three-dimensional patches with a time dimension to describe the
      local appearance of actions~\citep{laptev2005}. For each class~$j$, one
      dictionary~$\D_j$ is learned on the training data to represent the
      appearance of such patches within the action class~$j$. This step is particularly 
      useful when some actions are locally discriminative.
   \item {\bfseries encoding of full sequences:}
      once the dictionaries~$\D_j$ are learned, a train or test video is
      processed by extracting its overlapping three-dimensional patches,
      and by encoding them with the concatenated dictionary
      $\D=[\D_1,\ldots,\D_k]$, say with~$p$ dictionary elements.
      This produces a representation~$\alphab_i$ for every patch~$i$ of the
      sequence. These patches are pooled into a single vector~$\betab$
      in~$\Real^p$ for the sequence.
   \item {\bfseries global model of actions:} 
      the first layer provides a high-dimensional vector for each video. The
      second layer is exploiting these nonlinear representations of the
      sequences as inputs to the classification scheme~(\ref{eq:class_aa}).
      As a result, the layer is learning a global model for each action class,
      by learning again one dictionary per class.
      The difference with the first layer is that the dictionaries model the
      appearance of full sequences instead of local three-dimensional
      patches.
\end{itemize}
Note that, to simplify, we have also omitted some useful heuristics such as the
removing of three-dimensional patches with low energy, corresponding to
motionless regions. We have also not provided all mathematical details such as
the chosen dictionary learning formulations and algorithms. We refer
to~\citet{castrodad2012} for a more exhaustive description of the method.

\paragraph{Visual tracking.}
We now move to a technique for visual tracking introduced by~\citet{mei2009},
which relies on the sparsity-inducing effect of the~$\ell_1$-norm and on some
ideas from the robust face recognition system of~\citet{wright2008robust}.
Given a stream of images $\x_1,\x_2,\ldots$, the goal of tracking is to find a
bounding box around a moving object in the sequence. Typically, the bounding
box is annotated for the first image, and the task is to find the boxes in the
subsequent frames by exploiting temporally consistency.

The approach of~\citet{mei2009} is inspired by template matching techniques
\citep{lucas1981,matthews2004}. It assumes that at time~$t$, a set of~$p$
templates~$\D=[\d_1,\ldots,\d_p]$ in~$\Real^{m \times p}$ is available, where
each column~$\d_j$ is a small image filled by the object of interest---in other
words, a bounding box. Typically, each~$\d_j$ has been either obtained by a
ground truth annotation at the beginning of the sequence, or is one of the
estimated boxes in the previous frames.  Then, estimating the position of the object in
frame~$t$ consists of scanning all possible bounding boxes in that frame, and
comparing how well they can be reconstructed by the templates~$\D$. 

Concretely, let us assume that there are~$l$ possible boxes, and let us denote by~$\R_k$ the
linear operator that extracts the $k$-th box from the image and subsamples it to
make a representation with~$m$ parameters. Then, the optimization problem is simply
\begin{equation}
   {\hat k} \in \argmin_{k \in \{1,\ldots,l\}} \left[ \min_{\alphab \in
      \Real_+^p, \e \in \Real^m} \frac{1}{2}\|\R_k \x_t - \D\alphab - \e\|_2^2 +
   \lambda \|\e\|_1 + \sum_{j=1}^p \eta_j | \alphab[j]| \right],
   \label{eq:tracking}
\end{equation}
where~$\e$ models the occlusion as in the face recognition formulation
of~\citet{wright2008robust} presented in Section~\ref{sec:visual_faces}.
The~$\ell_1$-norm encourages~$\e$ to be sparse, whereas the
weighted-$\ell_1$-norm also promotes the selection of a few templates
from~$\D$. The role of the weights~$\eta_j$ is to give more importance to some
templates, especially the ones that have been selected recently in the previous
frames.

Overall, the tracking scheme performs two successive operations for each frame.
It selects the new bounding box with~(\ref{eq:tracking}), and then updates the
dictionary~$\D$ and the weights~$\eta_j$, according to an ad-hoc
procedure~\citep[see][for the details]{mei2009}.  According to a recent
survey~\citep{wu2013}, trackers based on sparse representations, which are
essentially improved versions of~\citet{mei2009}, perform well compared to the
state of the art, and are able to deal with occlusion effectively.

\paragraph{Data visualization.}
Finally, we deviates from visual recognition by presenting
applications of sparse estimation to data visualization in computer
vision~\citep{elhamifar2012,chen2014}.  The first approach proposed
by~\citet{elhamifar2012} can be viewed as an extension of the sparse subspace
clustering method of~\citet{elhamifar2009}, and was successfully applied to the
problem of video summarization.

More precisely, consider a video sequence~$\X=[\x_1,\ldots,\x_T]$ in~$\Real^{l
\times T}$ where each~$\x_t$ in~$\Real^l$ is the~$t$-th frame. The goal of
video summarization is to find a set of~$s \ll T$ frames that best
``represents'' the video. The selection problem is formulated as
follows
\begin{equation}
   \min_{\A \in \Real^{T \times T}}  \|\X-\X\A\|_{\text{F}}^2 \st  \Psi_0(\A) \leq s ~~\text{and}~~\forall i,~ \sum_{j=1}^T \alphab_i[j] = 1,\label{eq:video_summary}
\end{equation}
where~$\A=[\alphab_1,\ldots,\alphab_T]$ and~$\Psi_0(\A)$ counts the number of
rows of~$\A$ that are non-zero. The constraints have two simultaneous effects:
first, they encourage every frame~$\x_t$ to be close to a linear
combination~$\X\alphab_t$ of other frames with coefficients~$\alphab_t$ that
sum to one; second, since~$\A$ has at most~$s$ non-zero rows, only~$s$ frames
from~$\X$ can be used in the approximations~$\X\alphab_t$. The video sequence
is consequently ``summarized'' by these~$s$ frames.

To visualize large image collections,~\citet{chen2014} has used related ideas by 
exploiting an older unsupervised learning technique called archetypal analysis, which
is presented in Section~\ref{subsec:matrix_other}. Specifically, let us
consider a collection of images, denoted again by~$\X=[\x_1,\ldots,\x_n]$,
but which do not necessarily form a video sequence. Archetypal analysis looks
for a dictionary~$\D$ in~$\Real^{m \times p}$ that ``represents'' well the data,
\begin{displaymath}
  \min_{\substack{
        \alphab_i \in \Delta_p~\text{for}~ 1\leq i \leq n \\
  \betab_j \in \Delta_n~\text{for}~ 1\leq j \leq p}}  \sum_{i=1}^n \|\x_i - \D \alphab_i\|_2^2  \st ~\forall j, \d_j = \X\betab_j, 
\end{displaymath}
where~$\Delta_p$ and~$\Delta_n$ are simplex sets---that is, sets of vectors
with non-negative entries that sum to one.  In contrast to the original
dictionary learning formulation, the dictionary elements, called archetypes,
are convex combinations of data points and admit a simple interpretation.

In Figure~\ref{fig:paris_n}, we present a result obtained by~\citet{chen2014}
on a database of images downloaded on the Flickr website with the request
``Paris''. Each image is encoded using the Fisher vector
representation~\citep{perronnin2010} and~$p=256$ archetypes are learned. Each
archetype~$\d_j$ is a convex combination~$\X\betab_j$ of a few input images
with a sparse vector~$\betab_j$ due to the simplicial constraint. For each
archetype, it is the natural to visualize images that are used in the sparse
combination, as well as the corresponding non-zero coefficients~$\betab_j[i]$.
To some extent, the resulting figure ``summarizes'' the image collection.

\begin{figure}[hbtp]
   \subfigure[Expected landmarks.]{\label{subfig:paris1}
      \begin{minipage}{0.305\textwidth}
         \includegraphics[width=\textwidth]{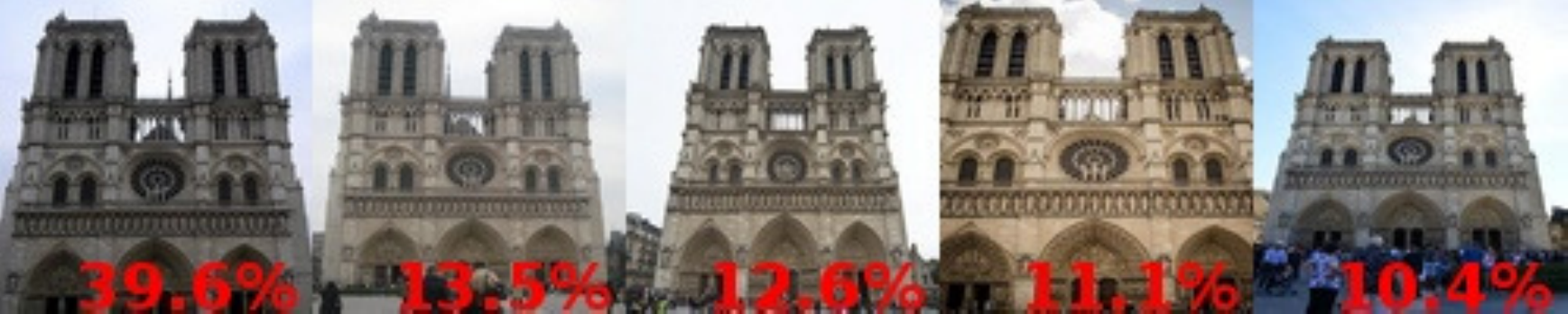} \newline
         \includegraphics[width=\textwidth]{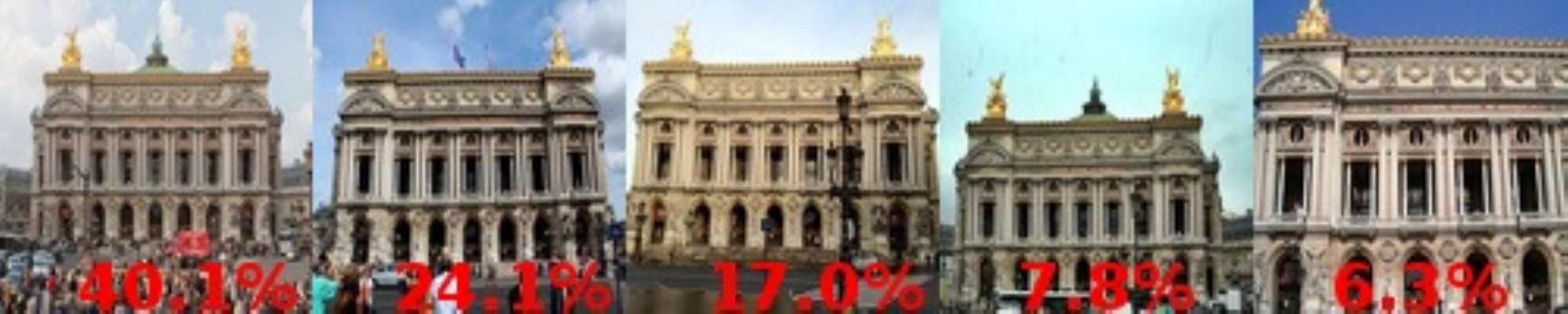} \\
         \includegraphics[width=\textwidth]{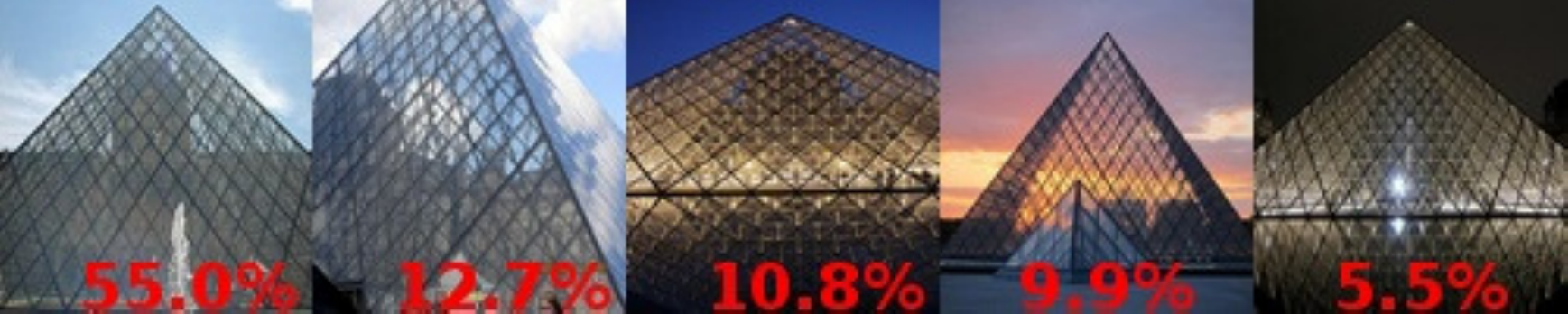} \\
         \includegraphics[width=\textwidth]{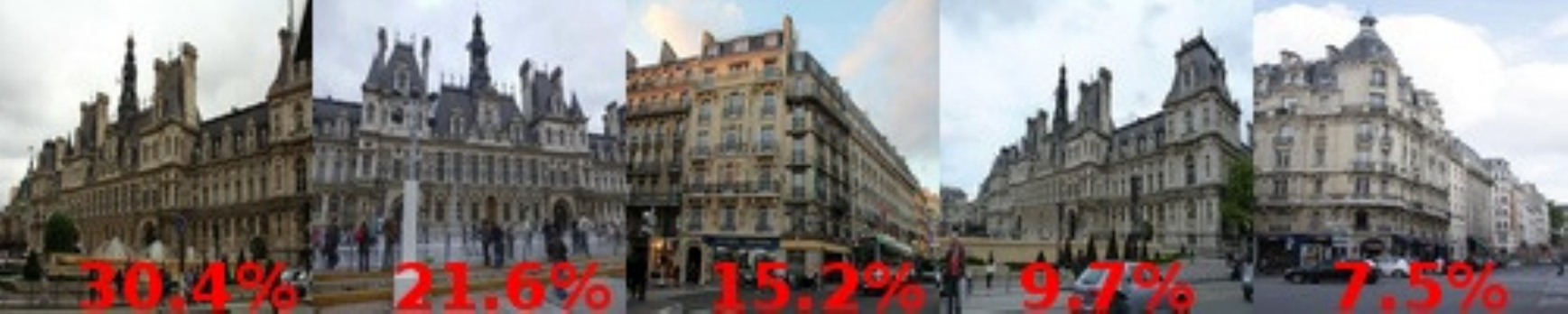} \\
         \includegraphics[width=\textwidth]{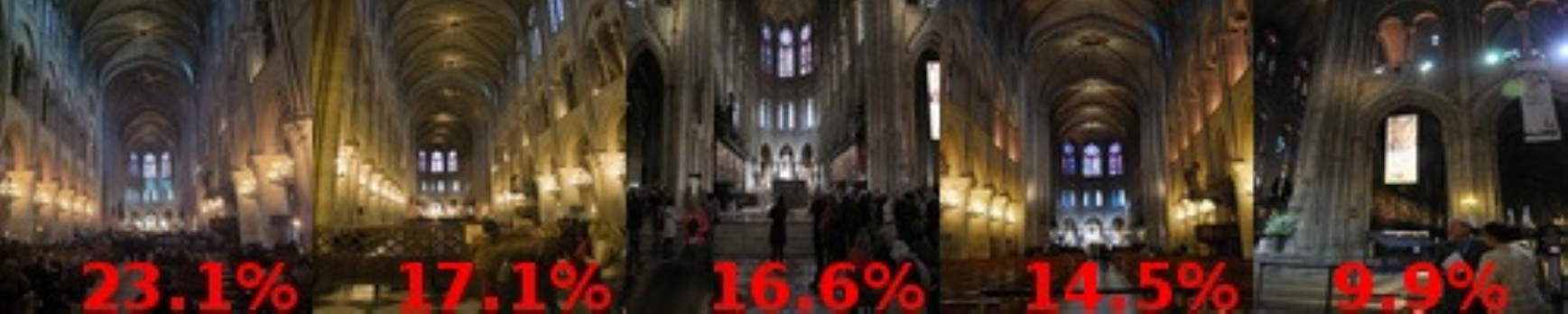} \\
         \includegraphics[width=\textwidth]{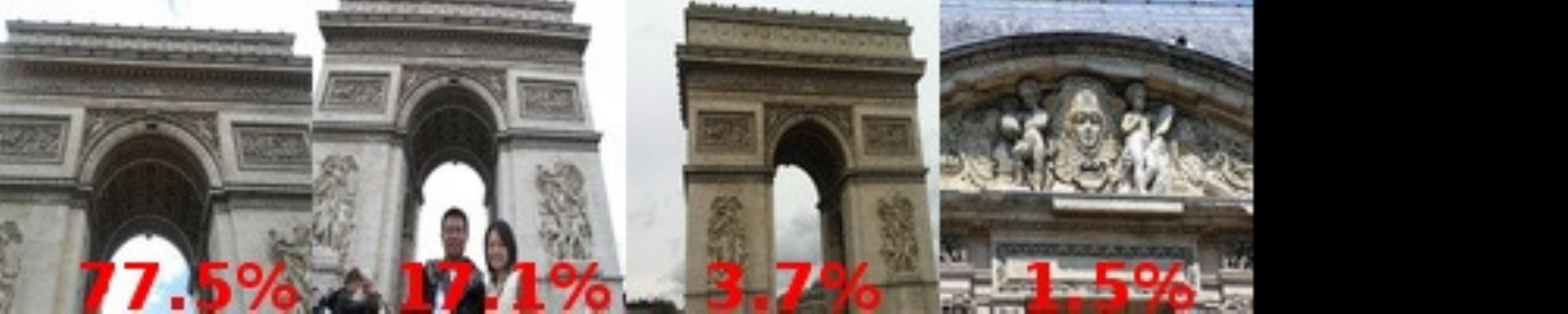} \\
         \includegraphics[width=\textwidth]{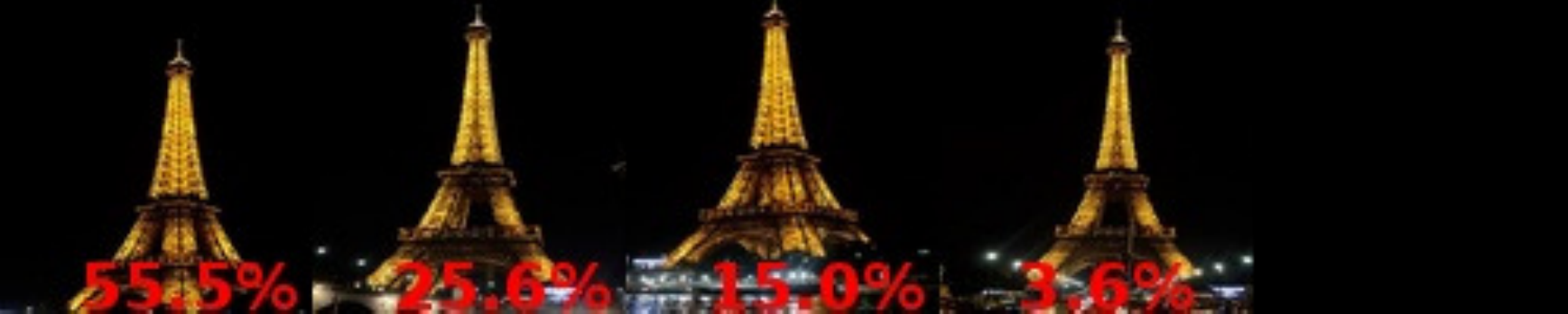}
      \end{minipage}
   }
   \subfigure[Unexpected ones.]{\label{subfig:paris2}
      \begin{minipage}{0.305\textwidth}
         \includegraphics[width=\textwidth]{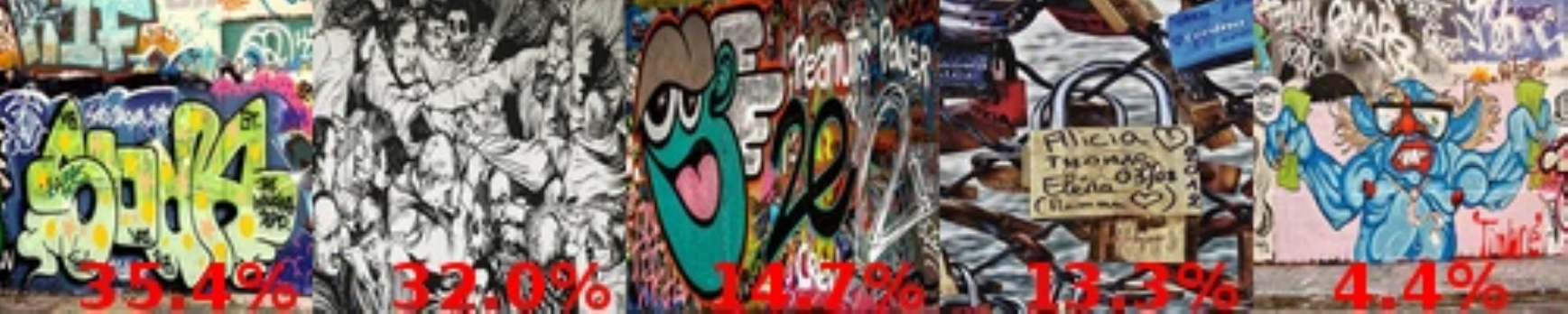} \newline
         \includegraphics[width=\textwidth]{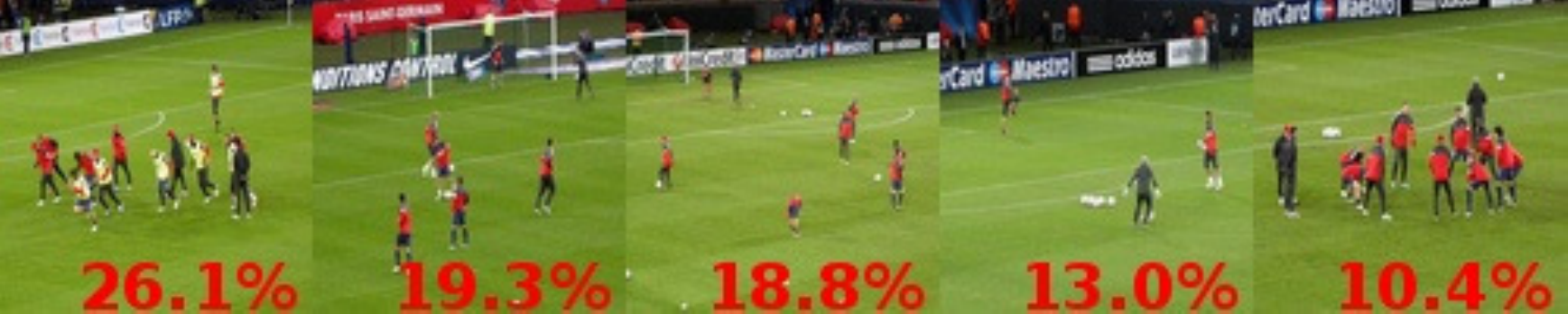} \\
         \includegraphics[width=\textwidth]{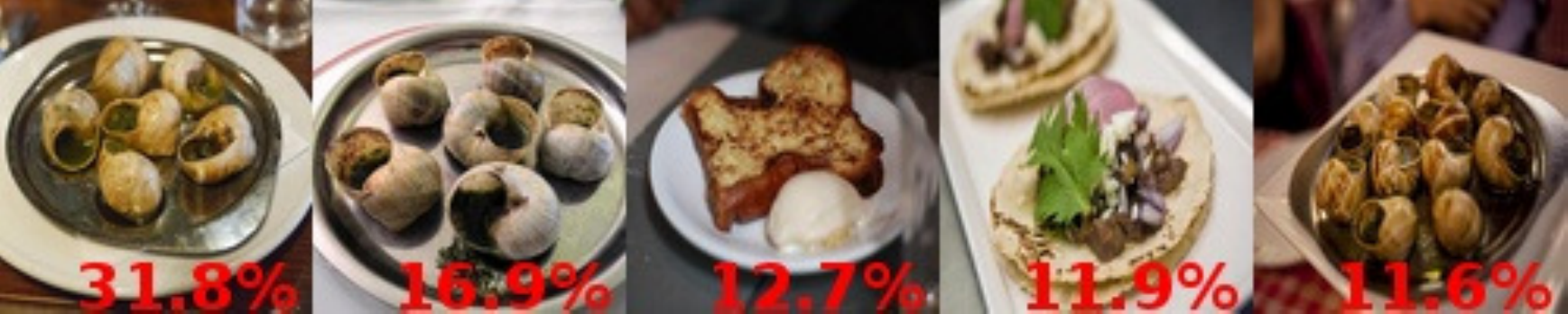} \\
         \includegraphics[width=\textwidth]{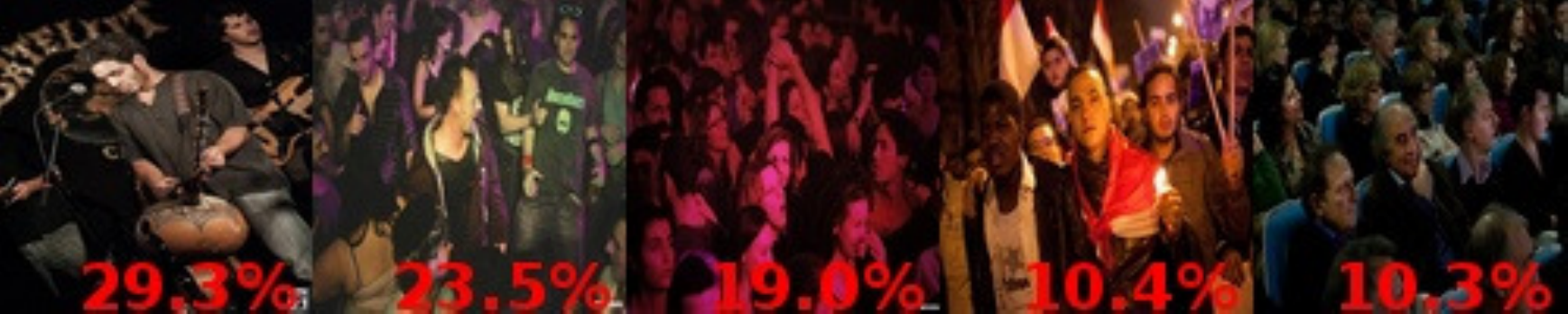} \\
         \includegraphics[width=\textwidth]{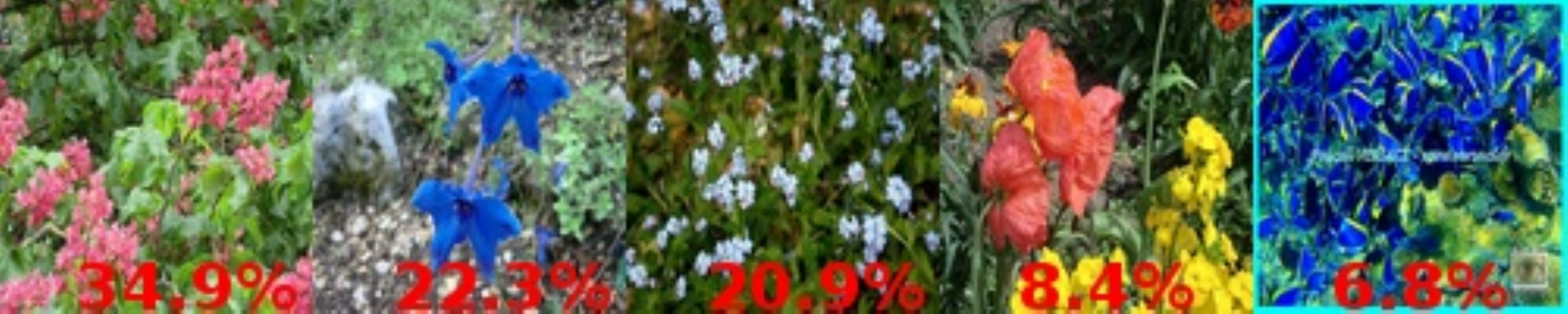} \\
         \includegraphics[width=\textwidth]{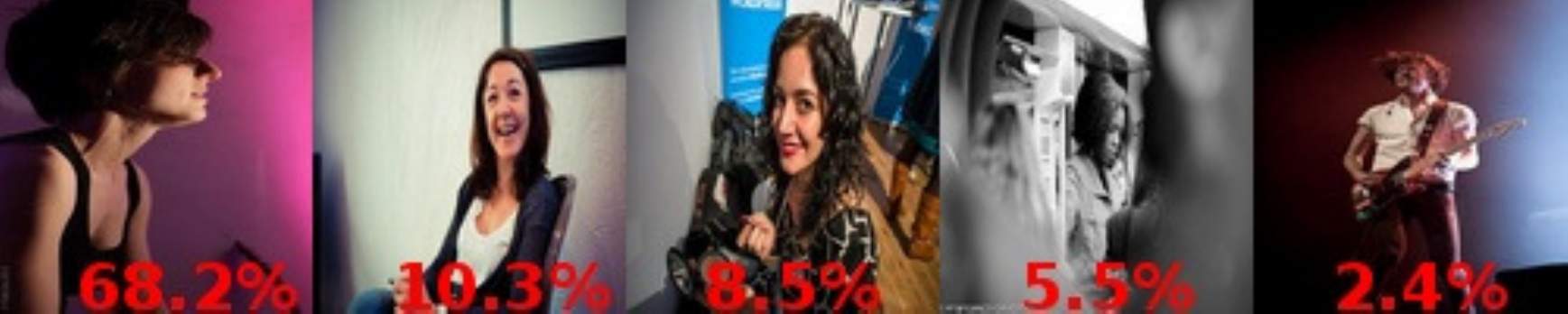} \\
         \includegraphics[width=\textwidth]{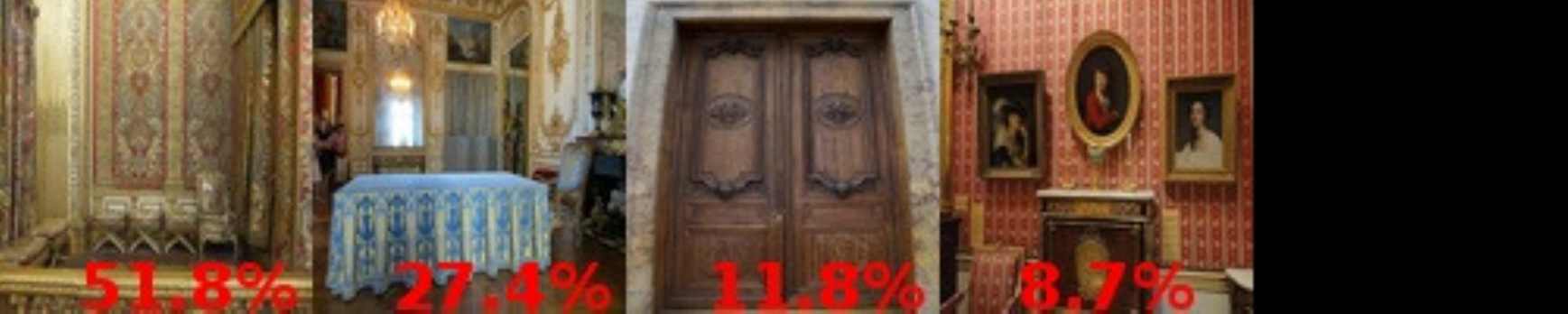}
      \end{minipage}
   }
   \subfigure[Scene composition.]{\label{subfig:paris3}
      \begin{minipage}{0.305\textwidth}
         \includegraphics[width=\textwidth]{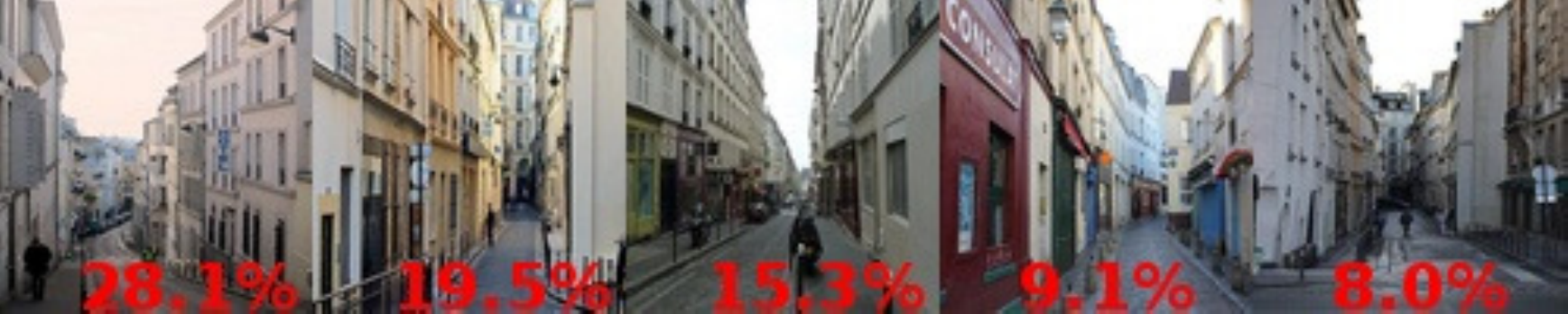} \newline
         \includegraphics[width=\textwidth]{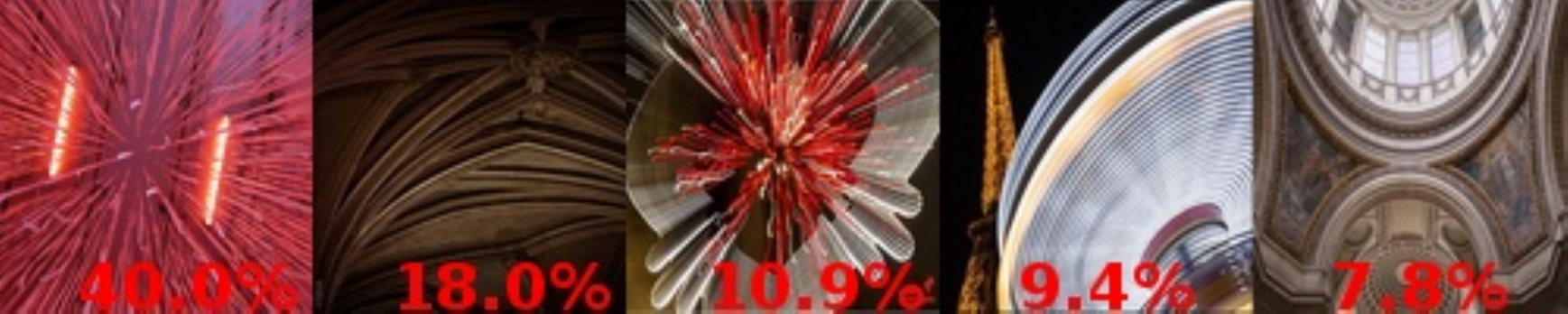} \\
         \includegraphics[width=\textwidth]{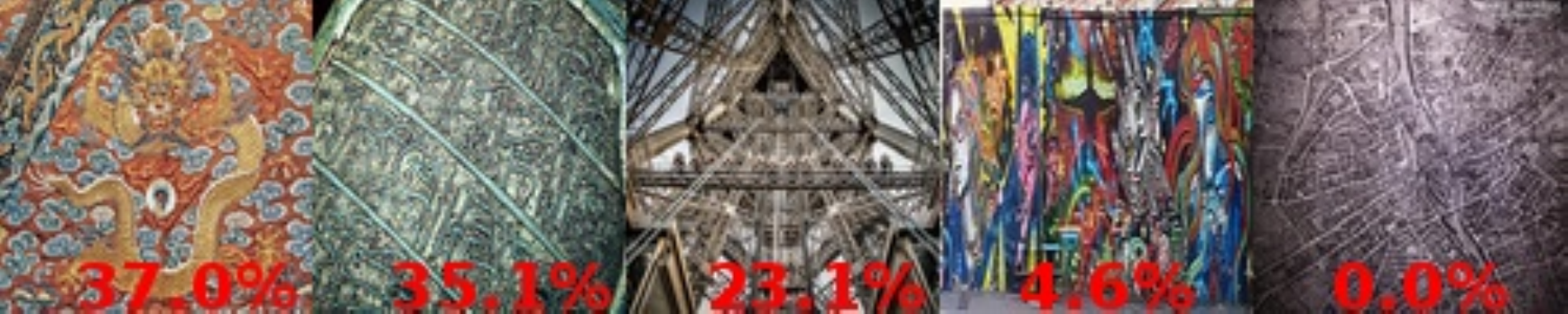} \\
         \includegraphics[width=\textwidth]{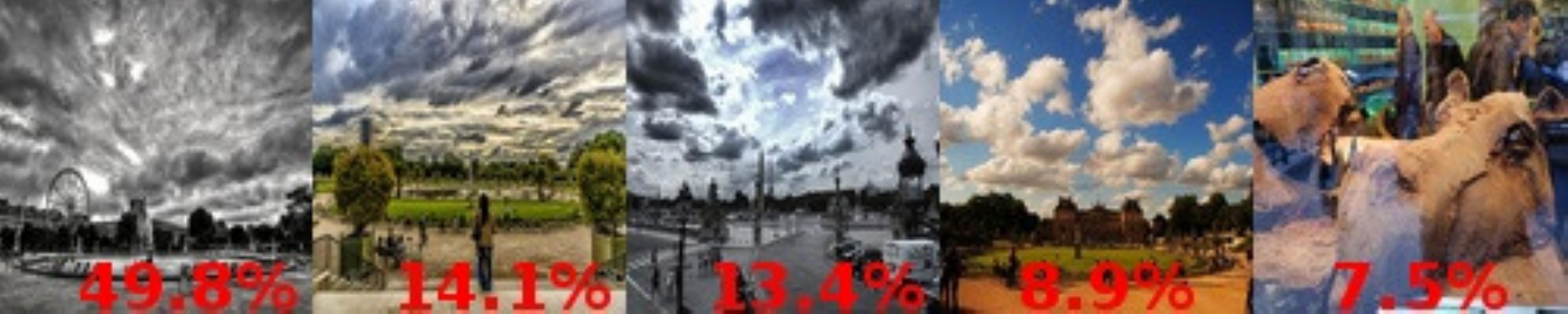} \\
         \includegraphics[width=\textwidth]{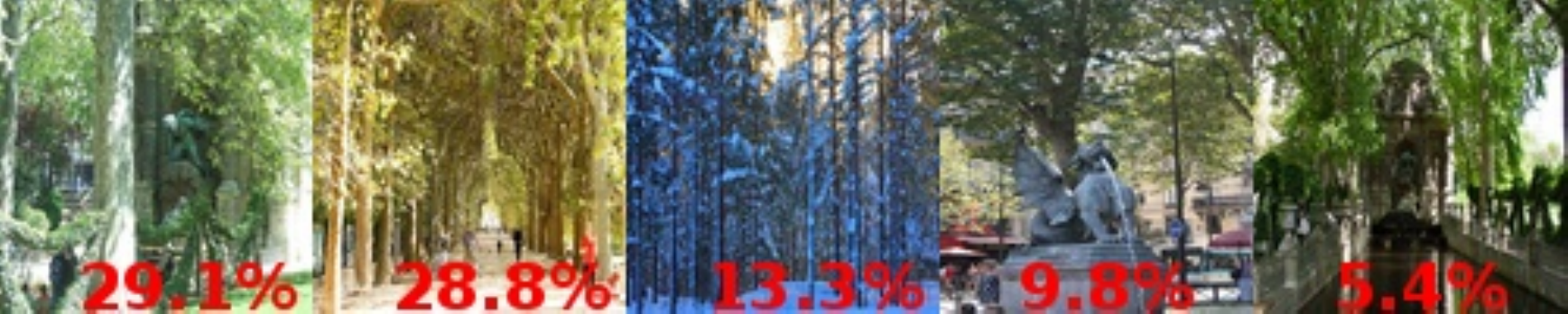} \\
         \includegraphics[width=\textwidth]{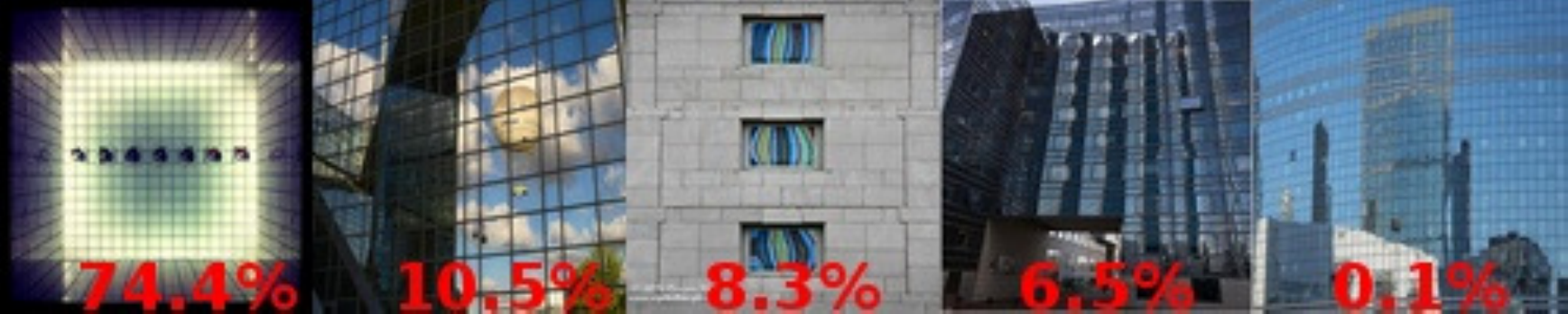}  \\
         \includegraphics[width=\textwidth]{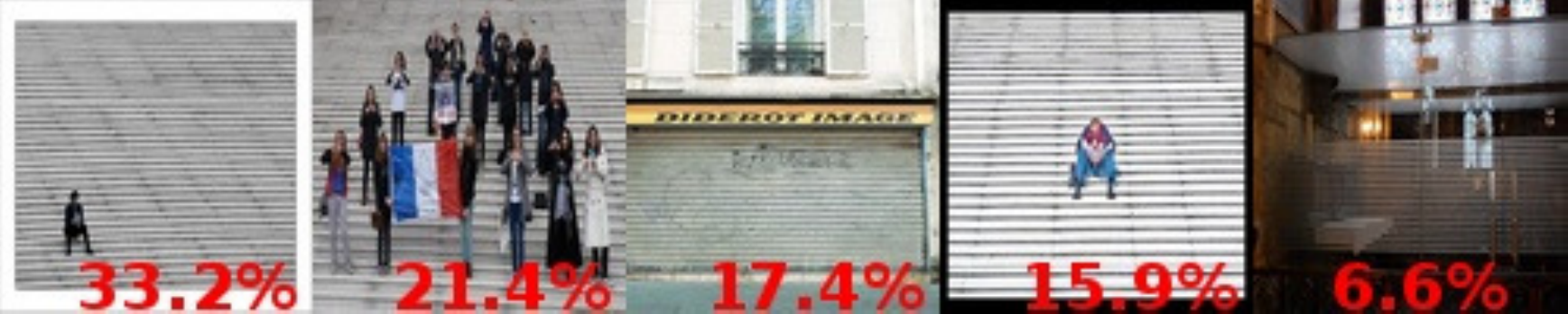}
      \end{minipage}
   }
    \caption{A few archetypes learned from $36\,600$ pictures corresponding to
       the request ``Paris'' downloaded from Flickr and sorted by ``relevance''. The figure is produced by~\citet{chen2014} with the
       software package SPAMS available at \url{http://spams-devel.gforge.inria.fr/} The numbers in red represent
       the values of the non-zero coefficients~$\betab_j[i]$ for
       each image.  In~\subref{subfig:paris1}, we obtain expected landmarks such
       as the Eiffel tower or Notre Dame; in~\subref{subfig:paris2}, we display some
       unexpected archetypes corresponding to soccer, graffitis, food, flowers
       and social gatherings; in~\subref{subfig:paris3}, we see some archetypes
       that do not seem to have some semantic meaning, but that capture some
       scene compositions or texture present in the dataset.  Best seen by
    zooming on a computer screen.} \label{fig:paris_n}
\end{figure}

%% file: content_arxiv/optim_intro.tex
In the previous parts of the monograph, many formulations for dealing with
image processing and computer vision tasks have been introduced, leading in
general to cost functions to optimize. However, we have not presented yet in great
details how the corresponding algorithms should be implemented in practice.
The purpose of this section is to introduce basic optimization tools to
address sparse estimation and dictionary learning problems. Because the
literature is vast~\citep[see][for a review]{bach2012}, we focus on a few
techniques that we consider to be easy to use and efficient enough.

Specifically, we mostly present algorithms that are parameter-free. This means
that in addition to~\emph{model parameters} that are inherent to the
formulation, they do not require tuning any input value to converge;
they should at most require choosing a number of iterations or a stopping
criterion. We also put the emphasis on algorithms that are easy to implement: a
few lines of a high-level language, \eg, Matlab or Python, are in general
sufficient to obtain an effective prototype, even though cutting-edge
implementations may be more involved to develop---that is, they  may require
a low-level language such as C/C++ and optimized third-party libraries for linear
algebra such as BLAS and LAPACK.\footnote{see \url{http://www.netlib.org/blas/}.}

First, we address classical $\ell_0$-regularized problems, before treating
the case of the~$\ell_1$-penalty. Then, we present iterative
reweighted $\ell_1$ and~$\ell_2$ techniques, and conclude with dictionary
learning algorithms. Many of these methods have been used for the experiments
conducted in this monograph, and are available in the SPAMS
toolbox.\footnote{\url{http://spams-devel.gforge.inria.fr/}.}

%% file: content_arxiv/optim_l0.tex
We start our review of optimization techniques with least square problems
that are regularized with the~$\ell_0$-penalty. More precisely, we are
interested in finding approximate solutions of
\begin{equation}
   \min_{\alphab \in \Real^p}  \|\x-\D\alphab\|_2^2 \st \|\alphab\|_0 \leq k,  \label{eq:l0a}
\end{equation}
or
\begin{equation}
   \min_{\alphab \in \Real^p} \|\alphab\|_0 \st \|\x-\D\alphab\|_2^2 \leq \varepsilon,  \label{eq:l0c}
\end{equation}
where~$\x$ is an input signal in~$\Real^m$ and~$\D$ is a dictionary
in~$\Real^{m \times p}$. There are mainly two classes of algorithms for dealing
with such NP-hard problems: greedy algorithms and iterative hard-thresholding methods.
We start with algorithms of the first kind, which are designed for
dealing with both formulations. They are commonly known as \emph{matching pursuit} and
\emph{orthogonal matching pursuit} in signal processing, and \emph{forward selection} techniques in statistics.

\paragraph{Coordinate descent algorithm: matching pursuit (MP).} Introduced by~\citet{mallat}, matching
pursuit is a coordinate descent algorithm that iteratively selects one entry of
the current estimate~$\alphab$ at a time and updates it. It is closely related to projection pursuit
algorithms from the statistics literature~\citep{friedman1981}, and is presented in
Algorithm~\ref{alg:mp}, where~$\D=[\d_1,\ldots,\d_p]$ is a dictionary with unit-norm columns.
\begin{algorithm}[hbtp]
   \caption{Matching pursuit algorithm.}\label{alg:mp}
   \begin{algorithmic}[1]
      \REQUIRE Signal $\x$ in~$\Real^m$, dictionary $\D$ in~$\Real^{m \times p}$ with unit-norm columns, and stopping criterion $k$ or $\varepsilon$. 
      \STATE Initialize $\alphab \leftarrow 0$;
      \WHILE{$\|\alphab\|_0 < k$ if~$k$ is provided or $\|\x-\D\alphab\|_2^2 > \varepsilon$ if~$\varepsilon$ is provided}
      \STATE select the coordinate with maximum partial derivative:
      $$\hatj \in \argmax_{j=1\ldots p} | \d_j^\top (\x-\D\alphab)|;$$
      \STATE update the coordinate
      \begin{equation}
         \alphab[\hatj] \leftarrow \alphab[\hatj] + \d_{\hatj}^\top(\x-\D\alphab);  \label{eq:updatemp}
      \end{equation}
      \ENDWHILE
      \RETURN the sparse decomposition $\alphab$ in~$\Real^p$.
   \end{algorithmic}
\end{algorithm}
Each iteration performs in fact a one dimensional minimization of the residual
with respect to the~$\hatj$-th entry of~$\alphab$ when keeping the other
entries fixed. The update~(\ref{eq:updatemp}) can indeed be interpreted as
\begin{equation}
   \alphab[\hatj] \leftarrow \argmin_{\alpha \in \Real} \left\| \x- \sum_{l \neq \hatj}
   \alphab[l]\d_l  -  \alpha \d_{\hatj}\right\|_2^2. \label{eq:updatemp2}
\end{equation}
Therefore, the residual monotonically decreases during the optimization,
whereas the model sparsity~$\|\alphab\|_0$ increases or remains constant until
the stopping criterion is satisfied. One drawback of matching pursuit is that
the same entry~$\hatj$ can be selected several times during the optimization,
possibly leading to a large number of iterations.  The next algorithm, called
``orthogonal matching pursuit'', does not suffer from this issue, and
usually provides a better sparse approximation with additional computational cost.

\paragraph{Active-set algorithm: orthogonal matching pursuit (OMP).} 
The algorithm is similar in spirit to matching pursuit, but enforces the
residual to be always orthogonal to all previously selected variables,
which is equivalent to saying that the algorithm reoptimizes the value of all
non-zero coefficients once it selects a new variable.  As a result, each
iteration of OMP is more costly than those of MP, but the procedure is
guaranteed to converge in a finite number of iterations. The algorithm appears
in the signal processing literature~\citep{pati1993}, and is called ``forward
selection'' in statistics (see Section~\ref{sec:early}).

Several variants of OMP have been proposed that essentially differ in the way
a new variable is selected at each iteration. In this monograph, we focus on an
effective variant introduced by~\citet{gharavi1998,cotter}, which is called
``order recursive orthogonal matching pursuit''. To simplify, we simply
use the terminology ``OMP'' to denote this variant and omit the term ``order recursive'' in the rest of the monograph. The
procedure is presented in Algorithm~\ref{alg:omp}, where~$\D_\Gamma$ represents
the matrix of size~$\Real^{m \times |\Gamma|}$ whose columns are those of~$\D$
indexed by~$\Gamma$, and $\Gamma^\complement$ denotes the complement set
of~$\Gamma$ in~$\{1,\ldots,p\}$. The classical OMP algorithm (not the order
recursive variant) chooses $\hatj$ exactly as in the MP algorithm.  Through our
experience, we have found that the order recursive variant provides a slightly
better sparse approximation in practice, even though it may seem more
complicated to implement at first sight.
\begin{algorithm}[hbtp]
   \caption{Orthogonal matching pursuit - order recursive variant.}\label{alg:omp}
   \begin{algorithmic}[1]
      \REQUIRE Signal $\x$ in~$\Real^m$, dictionary $\D$ in~$\Real^{m \times p}$, and stopping criterion $k$ or $\varepsilon$. 
      \STATE Initialize the active set $\Gamma = \emptyset$ and $\alphab \leftarrow 0$;
      \WHILE{$|\Gamma| < k$ if~$k$ is provided or $\|\x-\D\alphab\|_2^2 > \varepsilon$ if~$\varepsilon$ is provided}
      \STATE select a new coordinate~$\hatj$ that leads to the smallest residual:
      \begin{equation} 
         (\hatj, \hatbetab) \in \argmin_{j \in \Gamma^\complement, \betab \in \Real^{|\Gamma|+1}} \left\|\x-\D_{\Gamma \cup \{j \}}\betab\right\|_2^2; \label{eq:omp}
   \end{equation}
      \STATE update the active set and the solution~$\alphab$:
      \begin{displaymath}
         \begin{split}
            \Gamma & \leftarrow \Gamma \cup \{\hatj\}; \\
            \alphab_\Gamma & \leftarrow \hatbetab  ~~~\text{and}~~~ \alphab_{\Gamma^\complement} \leftarrow 0;
         \end{split}
      \end{displaymath}
      \ENDWHILE
      \RETURN the sparse decomposition $\alphab$ in~$\Real^p$.
   \end{algorithmic}
\end{algorithm}

Unlike MP, the set~$\Gamma$ of non-zero variables of~$\alphab$ strictly
increases by one unit after each iteration, and it is easy to show that the
residual is always orthogonal to the matrix~$\D_\Gamma$
of previously selected dictionary elements. Indeed, we have at each step 
of the algorithm:
\begin{displaymath}
   \alphab_\Gamma = \left(\D_\Gamma^\top \D_\Gamma\right)^{-1} \D_\Gamma^\top \x,
\end{displaymath}
(assuming $\D_\Gamma$ to be full-rank), and thus
\begin{displaymath}
   \D_\Gamma^\top (\x-\D \alphab) =   \D_\Gamma^\top (\x-\D_\Gamma \alphab_\Gamma) = 0.
\end{displaymath}
It is interesting to note that efficient implementations of
Algorithm~\ref{alg:omp} exist, despite the fact that the algorithm naively requires
to minimize~$|\Gamma^\complement|$ quadratic functions---equivalently
solve~$|\Gamma^\complement|$ linear systems---at each iteration.  The main
ingredients for a fast implementation have been discussed by~\citet{cotter}, and are summarized
below:
\begin{enumerate}
   \item if the Cholesky factorization of~$(\D_\Gamma^\top\D_\Gamma)$ is
      already computed at the beginning of an iteration, finding the
      index~$\hatj$ to select in the update~(\ref{eq:omp}) can be found with
      very few operations by using simple linear algebra relations~\citep{cotter}.
   \item given some index~$\hatj$, we have
      \begin{displaymath}
         \hatbetab =  (\D_\Gamma^{\prime\top} \D_\Gamma')^{-1} \D_\Gamma^{\prime\top} \x,
      \end{displaymath}
      where $\Gamma' = \Gamma \cup \{\hatj\}$. Whereas the
      matrix inversion naively costs~$O(|\Gamma'|^3)$ operations, it can
      be obtained in~$O(|\Gamma|^2)$ operations if the Cholesky factorization $(\D_\Gamma^\top
      \D_\Gamma)^{-1}$ is already computed.
   \item if $\Q \defin \D^\top\D$ is precomputed, a significant gain in terms
      of speed can be obtained by exploiting the fact that the quadratic
      function~$\|\x-\D\alphab\|_2^2$ expands as $\|\x\|_2^2 - 2 \q^\top\alphab +
      \alphab^\top\Q\alphab$. This trivial reformulation leads to an
      implementation that never explicitly computes the
      residual~$\r=\x-\D\alphab$, but instead works with the
      quantity~$\D^\top\r$, initialized with~$\q$, and updated at every
      iteration.  This simple strategy may be useful when 
      large set of signals~$\x_1,\ldots,\x_n$ needs to be encoded in parallel with the same
      dictionary. Then, it is worth computing~$\Q$ once for all, before
      processing all signals. It is then possible to obtain an implementation
      of Algorithm~\ref{alg:omp} whose complexity for encoding these~$n$
      signals is 
      \begin{displaymath}
         \underbrace{O(mp^2)}_{\text{compute}~\D^\top \D~\text{once}} + \underbrace{n O(k^3)}_{\text{Cholesky factorizations}}+ \underbrace{nO(pm+pk^2)}_{\text{update of}~\D^\top \r}.
      \end{displaymath}
      A similar strategy for the classical OMP algorithm that involves a
      simpler selection rule of the variable~$\hatj$ is discussed
      by~\citet{rubinstein2008}.
\end{enumerate}
Even though the algorithm only provides an approximate solution of
the~$\ell_0$-regularized least square problem in general, OMP was shown
by~\citet{tropp2007} to be able to reliably recover a sparse signal from random
measurements. Interestingly, the conditions that are required on the dictionary
are similar to the ones for sparse recovery in the compressed sensing
literature (see Section~\ref{sec:compressedsensing}).

\paragraph{Gradient-descent technique: iterative hard-thresholding.}
A significantly different approach than the previous greedy algorithms consists of using a gradient
descent principle, an approach called \emph{iterative
hard-thresholding}~\citep{starck2003,herrity,blumensath2009}, designed to find an
approximate solution of either~(\ref{eq:l0a}) or 
\begin{equation}
   \min_{\alphab \in \Real^p} \frac{1}{2}\|\x-\D\alphab\|_2^2 + \lambda \|\alphab\|_0, \label{eq:l0b}
\end{equation}
for some penalty parameter~$\lambda$. The method is presented in
Algorithm~\ref{alg:ht} and is an instance of the more general \emph{proximal
gradient descent} algorithms~\citep[see][for a review]{bach2012}. 

\begin{algorithm}[hbtp]
   \caption{Iterative hard thresholding.}\label{alg:ht}
   \begin{algorithmic}[1]
      \REQUIRE Signal $\x$ in~$\Real^m$, dictionary $\D$ in~$\Real^{m \times
      p}$, target sparsity $k$ or penalty parameter~$\lambda$, step size~$\eta$, number of iterations~$T$, initial
      solution~$\alphab_0$ in~$\Real^p$ (with $\|\alphab_0\|_0 \leq k$ if $k$ is provided).
      \STATE Initialize $\alphab \leftarrow \alphab_0$;
      \FOR{$t=1,\ldots,T$}
      \STATE perform one step of gradient descent:
      \begin{displaymath}
         \alphab \leftarrow \alphab + \eta \D^\top (\x-\D\alphab);
      \end{displaymath}
      \STATE choose threshold $\tau$ to be the~$k$-th largest magnitude among the entries of~$\alphab$ if~$k$ is provided, or $\tau=\sqrt{2\lambda}$ if~$\lambda$ is provided;
      \FOR{$j=1,\ldots,p$}
      \STATE hard-thresholding:
      \begin{displaymath}
         \alphab[j] \leftarrow \left\{\begin{array}{cc}
               \alphab[j] & \text{if}~|\alphab[j]| \geq \tau, \\
               0 & \text{otherwise}.
            \end{array}\right.
      \end{displaymath}
      \ENDFOR
      \ENDFOR
      \RETURN the sparse decomposition $\alphab$ in~$\Real^p$.
   \end{algorithmic}
\end{algorithm}

As we will see, when~$\eta$ is smaller than the inverse of the
largest eigenvalue of~$\D^\top\D$, the algorithm monotonically decreases the
value of the objective function. To see this, we can interpret the
iterative hard-thresholding algorithm as a majorization-minimization
scheme~\citep{lange2000}, which consists of minimizing at each iteration a
locally tight upper-bound of the objective. This principle is illustrated in
Figure~\ref{fig:mm}.
\begin{figure}[hbtp]
   \centering
   \includegraphics[page=5]{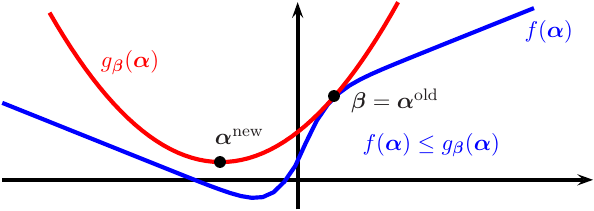}
   \caption{Illustration of the basic majorization-minimization principle. The
   objective function~$f$ (in blue) needs to be minimized. At the current
   iteration, a majorizing surrogate~$g_\betab$ (in red) is computed that is locally tight
at the previous estimate~$\alphab^\text{old}$. After minimizing the surrogate,
we obtain a new estimate~$\alphab^\text{new}$ such that $f(\alphab^{\text{new}}) \leq f(\alphab^{\text{old}})$.}\label{fig:mm}
\end{figure}

The majorizing surrogate can be obtained from the following 
inequality that holds for all~$\alphab$ and~$\betab$ in~$\Real^p$:
\begin{equation}
   \begin{split}
      \frac{1}{2}\|\x-\D\alphab\|_2^2 & = \frac{1}{2}\|\x-\D\betab- \D(\alphab-\betab)\|_2^2  \\
                                      & = \frac{1}{2}\|\x-\D\betab\|_2^2 - (\alphab\!-\!\betab)^\top\D^\top(\x\!-\!\D\betab) +  \frac{1}{2}\|\D(\alphab-\betab)\|_2^2 \\
                                      & \leq \frac{1}{2}\|\x-\D\betab\|_2^2 - (\alphab\!-\!\betab)^\top\D^\top(\x\!-\!\D\betab) +  \frac{1}{2\eta}\|\alphab-\betab\|_2^2 \\
                                      & = \underbrace{C_\betab + \frac{1}{2\eta}\left\| \alphab- \left(\betab + \eta \D^\top(\x\!-\!\D\betab)\right) \right\|_2^2}_{g_\betab(\alphab)}, \label{eq:gbeta}
\end{split}
\end{equation}
where~$C_\betab$ is a term independent of~$\alphab$. All equalities are simple
linear algebra relations and the inequality comes from the fact that~$\z^\top
\D^\top \D\z \leq \frac{1}{\eta}\z^\top\z$ for all~$\z$ in~$\Real^p$,
since~$1/\eta$ is assumed to be larger than the largest eigenvalue
of~$\D^\top\D$. Under this assumption, we have the relation
$\frac{1}{2}\|\x-\D\alphab\|_2^2 \leq g_\betab(\alphab)$ for all~$\alphab$
in~$\Real^p$, with an equality for~$\alphab=\betab$.
Then, we have two possible cases:
\begin{itemize}
   \item if the problem of interest is~(\ref{eq:l0a}), $k$ is provided to
      Algorithm~\ref{alg:ht}, and it is easy to show that the value of~$\alphab$ at
      the end of one iteration is the solution of
      \begin{displaymath}
         \min_{\alphab \in \Real^p} g_\betab(\alphab) \st \|\alphab\|_0 \leq k,
      \end{displaymath}
      where~$\betab$ is the value of~$\alphab$ at the beginning of the iteration.
   \item if one is interested in the penalized problem~(\ref{eq:l0b}), the
      value of~$\alphab$ after each iteration of Algorithm~\ref{alg:ht} is
      instead the solution of
      \begin{displaymath}
         \min_{\alphab \in \Real^p} g_\betab(\alphab) +\lambda \|\alphab\|_0,
      \end{displaymath}
      and the function $\alphab \mapsto g_\betab(\alphab) +\lambda
      \|\alphab\|_0$ is a majorizing surrogate of the objective 
$\frac{1}{2}\|\x-\D\alphab\|_2^2 + \lambda \|\alphab\|_0$.
\end{itemize}
In both cases, Algorithm~\ref{alg:ht} can be interpreted as a
majorization-minimization scheme and the value of the objective function
monotonically decreases.

%% file: content_arxiv/optim_l1.tex
In the previous section, we have studied three optimization
problems~(\ref{eq:l0a}), (\ref{eq:l0c}) and~(\ref{eq:l0b}). Similarly, we
consider three formulations of interest involving the~$\ell_1$-norm:
\begin{equation}
   \min_{\alphab \in \Real^p} \frac{1}{2}\|\x-\D\alphab\|_2^2 + \lambda\|\alphab\|_1,\label{eq:l1a}
\end{equation}
and
\begin{equation}
   \min_{\alphab \in \Real^p} \|\alphab\|_1 \st \|\x-\D\alphab\|_2^2 \leq \varepsilon,\label{eq:l1b}
\end{equation}
and
\begin{equation}
   \min_{\alphab \in \Real^p}  \|\x-\D\alphab\|_2^2 \st \|\alphab\|_1 \leq \mu.\label{eq:l1c}
\end{equation}
We have also presented three optimization strategies to deal with
the~$\ell_0$-penalty, namely: matching pursuit, orthogonal matching pursuit, and
iterative hard-thresholding algorithms. Interestingly, it is possible to 
deal with the~$\ell_1$-penalty by following similar techniques, even though
the links between $\ell_1$-approaches and~$\ell_0$-ones are only clear
from hindsight.  

\paragraph{Coordinate descent.} 
The counterpart of matching pursuit for $\ell_1$-regularized problems is
probably the coordinate descent method. Even though the general scheme is
classical in optimization~\citep[see][]{bertsekas}, it was first proposed
by~\citet{fu_penalized_1998} for minimizing the Lasso~(\ref{eq:l1a}).  Its
convergence to the exact solution of the convex non-smooth problem such as~(\ref{eq:l1a}) was shown
by~\citet{tseng2009coordinate}, and its empirical efficiency was demonstrated 
by~\citet{wu2008coordinate}.  The coordinate descent approach is presented in
Algorithm~\ref{alg:cd}, where~$S_\lambda$ denotes the soft-thresholding
operator already introduced in Section~\ref{sec:wavelets}:
\begin{displaymath}
   S_\lambda: \theta \mapsto \sign(\theta)\max( |\theta|-\lambda,0),
\end{displaymath}
and the selection rule for choosing~$\hatj$ can be on the following strategies:
\begin{itemize}
   \item cycle in~$\{1,\ldots,p\}$ (the most typical choice);
   \item pick up~$\hatj$ uniformly at random in~$\{1,\ldots,p\}$;
   \item choose the entry~$\hatj$ exactly as in the matching pursuit algorithm.
\end{itemize}
\begin{algorithm}[hbtp]
   \caption{Coordinate descent algorithm for the Lasso~(\ref{eq:l1a}).}\label{alg:cd}
   \begin{algorithmic}[1]
      \REQUIRE Signal $\x$ in~$\Real^m$, dictionary $\D$ in~$\Real^{m \times p}$, penalty parameter~$\lambda$,
      number of iterations~$T$, initial solution~$\alphab_0$ in~$\Real^p$.
      \STATE Initialize $\alphab \leftarrow \alphab_0$;
      \FOR{$t=1,\ldots,T$}
      \STATE select one coordinate~$\hatj$, \eg, by cycling or at random;
      \STATE update the coordinate by soft-thresholding:
      \begin{equation}
         \alphab[\hatj] \leftarrow S_\lambda\left(\alphab[\hatj] + \frac{\d_{\hatj}^\top(\x-\D\alphab)}{\|\d_{\hatj}\|_2^2}\right); \label{eq:updatecd}
      \end{equation}
      \ENDFOR
      \RETURN the sparse decomposition $\alphab$ in~$\Real^p$.
   \end{algorithmic}
\end{algorithm}

It is relatively easy to show that the update~(\ref{eq:updatecd}) is in fact
minimizing the objective~(\ref{eq:l1a}) with respect to the~$\hatj$-th entry
of~$\alphab$ when keeping the others fixed. More precisely,~(\ref{eq:updatecd}) is equivalent to
\begin{displaymath}
   \alphab[\hatj] \leftarrow \argmin_{\alpha \in \Real} \left\| \x- \sum_{l \neq \hatj}
   \alphab[l]\d_l  -  \alpha \d_{\hatj}\right\|_2^2 + \lambda|\alpha|,
\end{displaymath}
which is very similar to the matching pursuit update~(\ref{eq:updatemp2}) but
with the~$\ell_1$-penalty in the formulation.  The benefits of the coordinate
descent scheme are three-fold: (i) it is extremely simple to implement; (ii) it
has no parameter; (iii) it can be very efficient in
practice~\citep[see][Section 8, for a benchmark]{bach2012}.

Assuming that one chooses the cycling selection rule, the complexity of a naive
update of~$\alphab[j]$ is $O(mp)$ for computing the matrix-vector
multiplication~$\D\alphab$.  Fortunately, better complexities can be obtained
by adopting one of the two following alternative strategies:
\begin{enumerate}
   \item during the algorithm, we choose to update an auxiliary variable representing the residual~$\r$ such that we
      always have~$\r=\x-\D\alphab$. Given~$\r$, updating~$\alphab[\hatj]$ only cost~$O(m)$
      operations because of the inner-product~$\d_{\hatj}^\top\r$.
      Then, updating~$\r$ either has zero cost if~$\alphab[\hatj]$ does not
      change, or cost $O(m)$ otherwise. Finally, the per-iteration complexity of this strategy
      is always $O(m)$;
   \item assuming that the matrix~$\D^\top\D$ is pre-computed, it is possible to obtain an even
      better per-iteration complexity by updating an auxiliary
      variable~$\z$ instead of~$\r$ such that we always
      have~$\z=\D^\top(\x-\D\alphab)$.  Given~$\z$, updating~$\alphab[j]$ only
      cost $O(1)$ operations. Updating~$\z$ has either zero cost
      if~$\alphab[j]$ does not change, or cost $O(p)$ operations.  As a result,
      the per-iteration complexity is at most~$O(p)$, but more interestingly,
      each time a variable~$\alphab[j]$ is equal to zero before its update and
      remains equal to zero after the update (a typical scenario for very
      sparse solutions), the cost per iteration is $O(1)$.
\end{enumerate}

Extensions of coordinate descent to deal with more general smooth loss
functions and group-Lasso penalties have also been
proposed~\citep{tseng2009coordinate}. These variants are also easy to
implement, and exhibit good performance in practice~\citep[see][]{bach2012}.

\paragraph{Proximal gradient algorithm - iterative soft-thresholding.} The
second approach that we consider is the $\ell_1$-counterpart
of the iterative hard-thresholding algorithm called ``iterative soft-thresholding''~\citep{nowak2001,figueiredo2003,starck2003,daubechies_iterative_2004}. Originally designed to address~(\ref{eq:l1a}),
it can be easily modified to deal with~(\ref{eq:l1c}).  It is similar to the
hard-thresholding technique of the previous section and is presented in
Algorithm~\ref{alg:prox}, where~$\eta$ satisfies the same property as in the
previous section---that is, $1/\eta$ is larger than the largest eigenvalue
of~$\D^\top\D$. Note that the orthogonal projection~(\ref{eq:l1project}) onto
the~$\ell_1$-ball can be done efficiently in linear
time~\citep{brucker1984,duchi2008efficient}.

\begin{algorithm}[hbtp]
   \caption{Proximal gradient algorithm for~(\ref{eq:l1a}) or~(\ref{eq:l1c}).}\label{alg:prox}
   \begin{algorithmic}[1]
      \REQUIRE Signal $\x$ in~$\Real^m$, dictionary $\D$ in~$\Real^{m \times
      p}$, regularization parameter~$\lambda$ or $\mu$, gradient descent step size~$\eta$, number of iterations~$T$, initial
      solution~$\alphab_0$ in~$\Real^p$ (with $\|\alphab_0\|_1 \leq T$ if $T$ is provided).
      \STATE Initialize $\alphab \leftarrow \alphab_0$;
      \FOR{$t=1,\ldots,T$}
      \STATE perform one step of gradient descent:
      \begin{displaymath}
         \alphab \leftarrow \alphab + \eta \D^\top (\x-\D\alphab);
      \end{displaymath}
      \IF{$\lambda$ is provided}
      \STATE soft-thresholding: for all~$j=1,\ldots,p$,
      \begin{displaymath}
         \alphab[j] \leftarrow S_\lambda(\alphab[j]);
      \end{displaymath}
      \ELSE
         \STATE $\mu$ is provided and one performs the orthogonal projection
         \begin{equation}
            \alphab \leftarrow \argmin_{\betab \in \Real^p} \left[\frac{1}{2}\|\alphab-\betab\|_2^2 \st \|\betab\|_1 \leq \mu\right]; \label{eq:l1project}
         \end{equation}
      \ENDIF
      \ENDFOR
      \RETURN the sparse decomposition $\alphab$ in~$\Real^p$.
   \end{algorithmic}
\end{algorithm}

Then, the method can also be interpreted under the majorization-minimization
point of view illustrated in Figure~\ref{fig:mm}.
Indeed, let us consider the two cases:
\begin{itemize}
   \item if the problem of interest is~(\ref{eq:l1a}), the user provides the
      parameter~$\lambda$, and it is easy to show that the value of~$\alphab$
      at the end of one iteration is the solution of
      \begin{displaymath}
         \min_{\alphab \in \Real^p} g_\betab(\alphab) +\lambda \|\alphab\|_1,
      \end{displaymath}
      where~$g_\betab$ is defined in~(\ref{eq:gbeta}), 
      and the function $\alphab \mapsto g_\betab(\alphab) +\lambda
      \|\alphab\|_1$ is a majorizing surrogate of the objective
      $({1}/{2})\|\x-\D\alphab\|_2^2 + \lambda \|\alphab\|_1$.
   \item if one is interested instead in the penalized problem~(\ref{eq:l1c}), the
      value of~$\alphab$ after each iteration of Algorithm~\ref{alg:ht} is
      the solution of
      \begin{displaymath}
         \min_{\alphab \in \Real^p} g_\betab(\alphab) \st \|\alphab\|_1 \leq \mu,
      \end{displaymath}
      where~$\betab$ is the value of~$\alphab$ at the beginning of the
      iteration.
\end{itemize}
The majorization-minimization principle explains why the algorithm
monotonically decreases the value of the objective function, but a more
refined analysis can provide convergence rates. In general, we know
that the objective function value converges to the optimal one with the
convergence rate~$O(1/t)$, and that some variants such as FISTA admit a better
rate~$O(1/t^2)$, as shown by~\citet{beck2009,nesterov2007gradient}.  These
methods are relatively easy to implement, and choosing~$\eta$ is not an issue
in practice.  When an appropriate value for~$\eta$ is not available beforehand, 
it is possible to automatically find a good one with a line-search scheme.

Similarly as coordinate descent, extensions to more general problems than
$\ell_1$-regularized least squares are possible. In general proximal gradient
methods are indeed adapted to solving convex problems of the form
\begin{displaymath}
   \min_{ \alphab \in \AA} f(\alphab) + \psi(\alphab),
\end{displaymath}
where~$\AA$ is a convex set in~$\Real^p$, $f$ is differentiable and its
gradient is uniformly Lipschitz continuous, and~$\psi$ is a convex nonsmooth
function such that the so-called \emph{proximal operator} of~$\psi$
\begin{displaymath}
   \argmin_{\alphab \in \AA} \frac{1}{2}\|\alphab-\betab\|_2^2+\psi(\alphab),
\end{displaymath}
can be efficiently computed for all~$\betab$ in~$\Real^p$.  As a result,
proximal gradient methods have been successfully used with the Group Lasso
penalty including the hierarchical and structured variants of
Section~\ref{sec:l1}~\citep[see][]{Jenatton2010b,mairal10}.

\paragraph{Homotopy and LARS algorithms}
Finally, the homotopy method \citep{osborne2000,efron} is related to orthogonal
matching pursuit, but for solving the~$\ell_1$-regularized
problems~(\ref{eq:l1a}), (\ref{eq:l1b}), and (\ref{eq:l1c}). Similarly to OMP,
the algorithm maintains an active-set of variables~$\Gamma$, initialized with
the empty set, and iteratively updates it by one variable at a time. 

The homotopy method follows in fact the regularization path of the Lasso---that
is, it finds the set of all solutions of~(\ref{eq:l1a}) for all values of the
regularization parameter~$\lambda$, formally defined as
\begin{displaymath}
   \PP \defin \{ \alphab^\star(\lambda) : \lambda > 0 \},
\end{displaymath}
where  $\alphab^\star(\lambda)$ is the solution of~(\ref{eq:l1a}), which is
assumed to be unique under a small technical
assumption~\citep[see][]{osborne2000,mairal2012b}. A remarkable property of the path
is that it can be shown to be piecewise linear. In other words, it can be
described by a finite number~$k$ of linear segments, as illustrated earlier in
Figure~\ref{fig:path}:
\begin{equation}
   \PP = \bigcup_{l=1}^k ~\left\{ \theta \alphab^\star(\lambda_l) + (1-\theta)\alphab^\star(\lambda_{l+1}) :  \theta \in [0,1] \right\} . \label{eq:pathlin}
\end{equation}
The homotopy method finds a point on the path, computes the direction of the
current linear segment, and follows it until the direction of the path changes.
The algorithm either returns the full regularization path---that is, the
sequence of~$\alphab^\star(\lambda_l)$ from Eq.~(\ref{eq:pathlin}), or stops if it
reaches a solution for a target~$\lambda$, a target residual~$\epsilon$,
or~$\ell_1$-norm~$\mu$. Thus, it can be used for solving the three
variants~(\ref{eq:l1a}), (\ref{eq:l1b}), or (\ref{eq:l1c}).

The piecewise linearity of the path was discovered by~\citet{markowitz} in the
context of portfolio selection, and was turned into an effective algorithm
for solving the Lasso by~\citet{osborne2000} and~\citet{efron}. It should be also noted that
solving~(\ref{eq:l1a}) for all values of~$\lambda$ is in fact an instance of
\emph{parametric quadratic programming}, for which path following algorithm
encompassing the homotopy method appear early in the optimization
literature~\citep{ritter}.  The homotopy turns out to be extremely efficient
for solving very sparse
medium-scale problems~\citep[see][]{bach2012}, even though its worst-case
complexity has been shown to be exponential~\citep{mairal2012b}.
Similar to the simplex algorithm for solving linear programs, the empirical
complexity is therefore extremely bad in the worst-case scenario, but it is most
often very good in practical ones.

Following the presentation of~\citet{mairal2012b}, we describe the method in
Algorithm~\ref{alg:homotopy}. The formula for the direction of the path is
given in Eq.~(\ref{eq:path}), where the quantity~$\etab$ is defined as $\etab
\defin \sign(\D^\top(\x-\D\alphab^\star))$ in~$\{-1,0,1\}^p$. This relation can
be derived from optimality conditions of the Lasso presented in Section~\ref{sec:l1},
but the derivation requires
technical developments that are beyond the scope of the monograph. We refer instead
to~\citet{mairal2012b} for more details.  Interestingly, implementing
efficiently Algorithm~\ref{alg:homotopy} raises similar difficulties
as for implementing orthogonal matching pursuit. We need indeed to be able to
update the inverse matrices~$(\D_\Gamma^\top \D_\Gamma)^{-1}$ at
each iteration when the active-set~$\Gamma$ changes. The Cholesky decomposition
or the Woodbury formula~\citep[see][]{golub2012} can be used for that purpose.

For example, let us give a few details on how the Woodbury formula can be useful.
Denoting by~$\Gamma' \defin \Gamma \cup \{ \hatj \}$, we have the classical relation
\begin{displaymath}
   \left(\D_{\Gamma'}^\top \D_{\Gamma'}\right)^{-1} = \left[ \begin{array}{cc}
         (\D_{\Gamma}^\top \D_{\Gamma})^{-1} + \frac{1}{s} \z \z^\top &  -\frac{1}{s} \z \\
         -\frac{1}{s} \z^\top  & \frac{1}{s} 
   \end{array} \right],
\end{displaymath}
where~$s \defin \d_\hatj^\top \d_\hatj - \d_\hatj^\top (\D_{\Gamma}^\top
\D_{\Gamma})^{-1}  \d_\hatj$ is called the \emph{Schur complement}, and~$\z
\defin (\D_{\Gamma}^\top \D_{\Gamma})^{-1}  \d_\hatj$.  Thus, computing
$(\D_{\Gamma'}^\top \D_{\Gamma'})^{-1}$ given $(\D_{\Gamma}^\top
\D_{\Gamma})^{-1}$ from the previous iteration can be done by a few matrix-vector multiplications only.  In the
same vein, obtaining $(\D_{\Gamma}^\top
\D_{\Gamma})^{-1}$ from $(\D_{\Gamma'}^\top \D_{\Gamma'})^{-1}$ is also a
matter of manipulating a few linear algebra relations.

We also remark that the implementation of the homotopy can benefit from the
pre-computation of the matrix~$\D^\top\D$ in the exact same way as for
orthogonal matching pursuit, such that both implementations have in fact the
same complexity (up to a constant factor that is slightly larger for the
homotopy method).

\begin{algorithm}[hbtp]
\caption{Homotopy algorithm for the Lasso.}\label{alg:homotopy}
\begin{algorithmic}[1]
   \REQUIRE Signal $\x$ in~$\Real^m$, dictionary $\D$ in~$\Real^{m \times p}$.
   \STATE $\lambda \leftarrow \|\D^\top\x\|_\infty$; 
   \STATE $\Gamma \leftarrow \{ j \in \{ 1,\ldots,p \} : |\d_j^\top\x|
    = \lambda\}$;
     \WHILE{$\lambda > 0$ or according to another stopping criterion} 
     \STATE compute the current direction of the regularization path: 
     \begin{equation}
        \alphab_\Gamma^\star(\lambda) = (\D_\Gamma^\top\D_\Gamma)^{-1}(\D_\Gamma^\top\x-\lambda \etab_\Gamma)~~\text{and}~~\alphab^\star_{\Gamma^\complement}(\lambda)=0.\label{eq:path}
     \end{equation}
     \STATE find the smallest~$\tau > 0$ such that one of the following events occurs:
         \begin{itemize}
            \item there exists $j$ in $\Gamma^\complement$ such that $|\d_j^\top(\x-\D\alphab^\star(\lambda-\tau))| =\lambda-\tau$; then, add $j$ to~$\Gamma$;
            \item there exists $j$ in $\Gamma$ such that $\alphab^\star(\lambda-\tau)[j]$ hits zero; then, remove $j$ from $\Gamma$.
         \end{itemize}
         It is also easy to show that the value of $\tau$ can be obtained in closed form such that one can ``jump'' from a kink to another;
     \STATE update $\lambda \leftarrow \lambda-\tau$; record the pair $(\lambda,\alphab^\star(\lambda))$; 
     \ENDWHILE
\STATE {\bf{Return:}} sequence of recorded values~$(\lambda,\alphab^\star(\lambda))$.
\end{algorithmic}
\end{algorithm}

%% file: content_arxiv/optim_reweighted.tex
We have mentioned in the introduction of this monograph that non-convex
sparsity-inducing penalties are often used in sparse estimation. These
regularization functions typically have the form
\begin{equation}
   \psi(\alphab)  = \sum_{j=1}^p \varphi\left(|\alphab[j]|\right), \label{eq:dc}
\end{equation}
where the functions~$\varphi$ are concave, non-decreasing and differentiable on~$\Real_+$.  A
natural way of dealing with such penalties when they appear in the objective
function is to use an approach called DC-programming, where DC stands for
``difference of convex''~\citep[see][for a review]{gasso2009}.
Even though the resulting optimization problems are non-convex and impossible
to solve exactly in a reasonable amount of time, DC-programming is a simple
technique to obtain a stationary point by using the majorization-minimization
principle already illustrated in Figure~\ref{fig:mm}.
It consists of exploiting a majorizing surrogate of~$\psi$ obtained by
linearizing the function~$\varphi$ in~(\ref{eq:dc}), which is upper-bounded by
its linear approximations because of concavity, as shown in
Figure~\ref{fig:dc}.
\begin{figure}[hbtp]
   \centering
    \includegraphics[page=6]{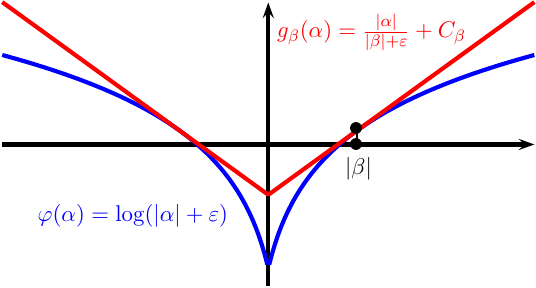}
\caption{Illustration of the DC-programming approach. The function $\varphi$ (in blue) is upper-bounded by its linear approximation (in red), which is tight at the point~$\beta$. The resulting majorizing surrogate is a weighted $\ell_1$-norm plus a constant~$C_\beta$.}
\label{fig:dc}
\end{figure}

For example, for minimizing
\begin{displaymath}
   \min_{\alphab \in \Real^p} \frac{1}{2}\|\x-\D\alphab\|_2^2 + \lambda \sum_{j=1}^p \log(|\alphab[j]|+\varepsilon),
\end{displaymath}
the resulting majorization-minimization algorithm solves a sequence of weighted Lasso problems at each iteration
\begin{equation}
   \alphab^{\text{new}} \leftarrow \argmin_{\alphab \in \Real^p} \left[\frac{1}{2}\|\x-\D\alphab\|_2^2 + \lambda \sum_{j=1}^p \frac{|\alphab[j]|}{|\alphab^\text{old}[j]| + \varepsilon} \right], \label{eq:reweighted}
\end{equation}
where~$\alphab^\text{old}$ denotes the current estimate at the beginning of the
iteration.  The update~(\ref{eq:reweighted}) is a reweighted-$\ell_1$
algorithm, which is known to provide sparser solutions than the regular Lasso.
Such majorization-minimization approaches for non-convex sparse estimation have
been introduced in various contexts. For example, \citet{fazel2002,fazel2003} and
then later~\citet{candes4} use such a principle for the regularization
function~$\varphi: x \mapsto \log(x+\varepsilon)$,
whereas~\citet{figueiredo2005,figueiredo2007} propose majorization-minimization
algorithms for dealing with the~$\ell_q$-penalty with~$q < 1$.

%% file: content_arxiv/optim_group.tex
We now focus on another type of algorithm for dealing with convex regularizers
such as the~$\ell_1$-norm. Similar to the reweighted-$\ell_1$-algorithms of the
previous section, the approaches we present here consist of solving a sequence
of sub-problems, each one involving a simple penalty.  The key idea is that
many non-smooth convex regularizers can be seen as minima of quadratic
functions, leading to reweighted $\ell_2$-algorithms.
 
\paragraph{Variational formulation of the $\ell_1$-norm.}
We start with the simplest case of the $\ell_1$-norm, which admits the 
following variational form~\citep{grandvalet1999,ingrid2010iteratively}:
$$
\| \alphab \|_1 = \inf_ { \boldsymbol{\eta} \in \Real^p_+ } \frac{1}{2}
\sum_{j=1}^p \left\{ \frac{\alphab[j] ^2}{ \boldsymbol{\eta}[j]} +
\boldsymbol{\eta}[j] \right\},
$$
with the optimal $ \boldsymbol{\eta} \in \Real^p_+$ obtained as $ \boldsymbol{\eta}[j] = | \alphab[j]|$. 
The Lasso optimization problem then becomes:
$$
\min_{\alphab \in \Real^p }  \min_{  \boldsymbol{\eta} \in \Real_+^p} \frac{1}{2} \| \x  - \D  \alphab \|_2^2 + \frac{\lambda}{2} \sum_{j=1}^p \left\{ \frac{\alphab[j] ^2}{ \boldsymbol{\eta}[j]} +  \boldsymbol{\eta}[j] \right\} .
$$ 
Fortunately, this problem is jointly convex in $(\alphab,\boldsymbol{\eta})$. The
minimization with respect to $ \boldsymbol{\eta}$ given $\alphab$ can be done
in closed form (by taking absolute values of the components of $\alphab$); the
minimization with respect to~$\alphab$ given $ \boldsymbol{\eta}$ is a
weighted least-squares problem, which can be solved by the solution of a
linear system as:
$$
 \alphab = \big[ \D ^\top \D + \lambda {\rm Diag} (  \boldsymbol{\eta})^{-1} \big]^{-1} \D^\top \x.
$$
Note that since the linear system above is ill-conditioned when some entries of~$\etab$ are very small, it is better to equivalently reformulate it as
$$ \alphab  =  {\rm Diag} (  \boldsymbol{\eta})^{1/2} \big[ 
 {\rm Diag} (  \boldsymbol{\eta})^{1/2}   \D ^\top \D  {\rm Diag} (  \boldsymbol{\eta})^{1/2}  + \lambda \mathbf{I} \big]^{-1}  {\rm Diag} (  \boldsymbol{\eta})^{1/2}  \D^\top \x. $$
 This naturally leads to   alternating minimization procedures where one alternates between minimizing with respect to $ \boldsymbol{\eta}$ as above in closed form, and minimizing with respect to $\alphab$, which is now easier because the regularizer has become a quadratic function. With quadratic losses, this leads to a linear system for which many existing algorithms exist. However, because the function $ (\alphab[j],\boldsymbol{\eta}[j] ) \mapsto  {\alphab[j] ^2}/{ \boldsymbol{\eta}[j]}$ cannot be extended to be continuous at $(0,0)$, alternating minimization does not converge and it is customary to add a penalty term $ \sum_{j=1}^p \frac{\varepsilon}{2 \boldsymbol{\eta}[j]  }$ with $\varepsilon >0$ small, which leads to an update for $\boldsymbol{\eta}[j] $ of the form $\boldsymbol{\eta}[j]  = \sqrt{ | \alphab[j]|^2 + \varepsilon}$ that avoids instabilities. See more details in~\citet[][Section 5]{bach2012}.

\paragraph{Extensions to $\ell_q$-quasi-norms.}
More generally, for any $q \in (0,2)$, and $\| \alphab \|_q = \big( \sum_{j=1}^q |\alphab[j]|^q \big)^{1/q}$ and $r = q/(2-q)$, we have
\citep{jenatton2010sspca}:
$$
\| \alphab \|_q = \inf_ { \boldsymbol{\eta} \in \Real_+^p } \frac{1}{2} \sum_{j=1}^p \frac{\alphab[j] ^2}{ \boldsymbol{\eta}[j]} 
    + \frac{1}{2} \| \boldsymbol{\eta} \|_r, $$
    with the optimal $ \boldsymbol{\eta}$ in $\Real_+^p$ obtained as $ \boldsymbol{\eta}[j] = | \alphab[j]|^{2-q} \| \alphab\|_q^{q-1}$. Note that similar variational formulations may be obtained for the squared quasi-norms~\citep{micchelli2005learning}. This gives an alternative to reweighted $\ell_1$-formulations presented in Section~\ref{sec:reweighted} for $q <1 $ (the non-convex case).
 
\paragraph{Extensions to group norms.}
The various formulations above can be extended to structured sparsity-inducing
norms, such as the ones presented in Section~\ref{sec:l1}. For example, for the
 group-lasso norm in Eq.~(\ref{eq:grouplasso}), we have:
$$
 \sum_{g \in \GG} \|\alphab[g]\|_2
 = \inf_ { \boldsymbol{\eta} \in \Real_+^p } \frac{1}{2} \sum_{g \in \GG} \bigg\{ \frac{\|\alphab[g]\|_2^2}{ \boldsymbol{\eta}[g]} +  \boldsymbol{\eta}[g]
 \bigg\},
$$
which also leads to a quadratic function of $\alphab$ and to simple algorithms
based on solving linear systems when the loss is quadratic. Note that these
quadratic variational formulations  extend to all norms and  lead to multiple
kernel learning formulations when used with positive-definite
kernels~\citep[see more details in][]{bach2012}.

%% file: content_arxiv/optim_dict.tex
Finally, we review a few dictionary learning algorithms that leverage the
optimization algorithms for sparse estimation presented in the
previous sections. We start with the stochastic gradient descent
method~\citep[see][and references therein]{kushner2003,bottou2008tradeoffs}
since it is very close to the original algorithm
of~\citet{olshausen1997}.  Then, we move to other classical approaches such as
alternate minimization~\citep{engan1999,ng-sparsecoding}, block coordinate
descent, K-SVD~\citep{aharon2006}, and then present efficient methods based on
stochastic approximations~\citep{mairal2010,skretting2010}.

We consider two dictionary learning formulations. For both of them, the goal is
to learn a dictionary~$\D$ in~$\CC$, where~$\CC$ is the set of matrices
in~$\Real^{m \times p}$ whose columns are constrained to have less than unit
$\ell_2$-norm. Then, we use a sparsity-inducing penalty~$\psi$ either as a
penalty
\begin{equation}
   \min_{\D \in \CC, \A \in \Real^{p \times n}} \frac{1}{n} \sum_{i=1}^n \frac{1}{2}\|\x_i-\D\alphab_i\|_2^2 + \lambda\psi(\alphab_i), \label{eq:dictpenalty}
\end{equation}
or as a constraint
\begin{equation}
   \min_{\D \in \CC, \A \in \Real^{p \times n}} \frac{1}{n} \sum_{i=1}^n \frac{1}{2}\|\x_i-\D\alphab_i\|_2^2  \st \psi(\alphab_i) \leq \mu,\label{eq:dictconstraint}
\end{equation}
and~$\X=[\x_1,\ldots,\x_n]$ in~$\Real^{m \times n}$ is the training set of
signals. Even though these problems are non-convex, many algorithms have been
developed with empirical good performance for various tasks, \eg,
image denoising. In general, these methods have no other guarantee than
providing a stationary point (in the best case), but they have been used
successfully in practical contexts, as shown in other parts of this
monograph.

\paragraph{Stochastic gradient descent.}
The first algorithm proposed by \citet{field1996,olshausen1997} addresses
the formulation~(\ref{eq:dictpenalty}), where~$\psi$ is
the~$\ell_1$-norm or a smooth approximate sparsity-inducing penalty such as
$\psi(\alphab) = \sum_{j=1}^p \log( \alphab[j]^2+ \varepsilon)$.  The algorithm
can be described as a heuristic stochastic gradient descent, which alternates
between two stages: (i) gradient steps with a fixed step size computed on
mini-batches of the training set and (ii) rescaling heuristic that prevents the
norm of the columns~$\d_j$ to grow out of bounds---thus, taking implicitly the
constraint~$\D \in \CC$ into account.

A less heuristic but very related procedure is the projected stochastic gradient
descent method, which is presented in Algorithm~\ref{alg:sgddict}. This
approach relies on the following equivalent formulation to~(\ref{eq:dictpenalty}):
\begin{equation}
   \min_{\D \in \CC} \frac{1}{n} \sum_{i=1}^n L(\x_i,\D), \label{eq:expectn}
\end{equation}
where
\begin{equation}
   L(\x_i,\D)  \defin \min_{\alphab \in \Real^p} \frac{1}{2}\|\x_i - \D\alphab\|_2^2 + \lambda \|\alphab\|_1. \label{eq:functionL}
\end{equation}
At each iteration, the algorithm selects one signal~$\hati$ and performs one
gradient step $\D \leftarrow \D - \eta_t \nabla_\D L(\x_\hati,\D)$ with a step size~$\eta_t$. 
It can indeed be shown that
the function~$\D \to L(\x_\hati,\D)$ is differentiable according to Danskin's
theorem~\citep[see][Proposition B.25]{bertsekas} and that its gradient admits a closed form $\nabla_\D
L(\x_\hati,\D) = -(\x_\hati -\D\alphab_\hati)\alphab_\hati^\top$,
where~$\alphab_\hati$ is defined in~(\ref{eq:dictalphai}).
Then, the algorithm performs a Euclidean projection~$\Pi_\CC$ onto the convex
set~$\CC$, which corresponds to the renormalization update $\d_j \to
(1/\max(\|\d_j\|_2,1))\d_j$ for all~$j$ in~$\{1,\ldots,p\}$.
A practical variant also consists of using mini-batches instead of drawing a
single data point at each iteration. For simplicity, we omit this variant in
our description of Algorithm~\ref{alg:sgddict}.

The stochastic gradient descent approach is effective \citep[see][for a
benchmark]{mairal2010}, but it raises a few practical difficulties such as
choosing well the step sizes~$\eta_t$ in a data-independent manner. Classical
strategies use for instance step sizes of the
form~$\eta_t=\eta/(t+t_0)^\gamma$, but finding an appropriate set of
hyper-parameters~$\eta,t_0$, and~$\gamma$ may be challenging in practice.

\begin{algorithm}[hbtp]
   \caption{Stochastic gradient descent for~(\ref{eq:dictpenalty}).}\label{alg:sgddict}
   \begin{algorithmic}[1]
      \REQUIRE Signals $\X=[\x_1,\ldots,\x_n]$ in~$\Real^{m \times n}$, initial dictionary $\D_0$ in~$\CC$, penalty parameter~$\lambda$,
      number of iterations~$T$.
      \STATE Initialize $\D \leftarrow \D_0$;
      \FOR{$t=1,\ldots,T$}
      \STATE select one signal~$\x_{\hati}$ at random;
      \STATE compute the sparse code:
      \begin{equation}
         \alphab_\hati \in \argmin_{\alphab \in \Real^p} \frac{1}{2}\|\x_i-\D\alphab\|_2^2 + \lambda\psi(\alphab); \label{eq:dictalphai}
      \end{equation}
      \STATE perform one projected gradient descent step:
      \begin{equation*}
         \D \leftarrow \Pi_{\CC}\left[\D - \eta_t(\D \alphab_\hati - \x_i) \alphab_\hati^\top \right];
      \end{equation*}
      \ENDFOR
      \RETURN the sparse decomposition $\alphab$ in~$\Real^p$.
   \end{algorithmic}
\end{algorithm}

\paragraph{Alternate minimization.}
The most classical approach for dictionary learning is the alternate
minimization scheme, which consists of optimizing~(\ref{eq:dictpenalty})
or~(\ref{eq:dictconstraint}) by alternating between two minimization steps: one
with respect to~$\D$ with the sparse codes~$\A$ fixed, and one with respect
to~$\A$ with~$\D$ fixed. In practice, this approach is not as fast as a
well-tuned stochastic gradient descent algorithm, but it is parameter-free and
we subjectively find it very reliable according to our experience for
image processing and computer vision tasks. 

Alternate minimization was first proposed by~\citet{engan1999} for dealing with the~$\ell_0$-penalty
with the method of optimal directions (MOD), and was revisited later
by \citet{ng-sparsecoding} with the~$\ell_1$-regularization. We present the
MOD approach in Algorithm~\ref{alg:mod} and its~$\ell_1$-counterpart in
Algorithm~\ref{alg:alt}. As we shall see, they slightly differ in the
dictionary update step, which can be simplified for the~$\ell_0$-penalty.

The sparse codes~$\alphab_i$ for MOD are obtained with any algorithm for dealing
approximately with~$\ell_0$, \eg, any technique from Section~\ref{sec:optiml0}.
Since all the vectors~$\alphab_i$ can be computed in parallel, it is worth mentioning
that this sparse decomposition step can benefit from an efficient
implementation of OMP including the heuristics of Section~\ref{sec:optiml0},
notably those consisting of pre-computing~$\D^\top\D$ before updating all the~$\alphab_i$'s in parallel.
For dealing with~$\ell_1$, the homotopy method can be used similarly, with the
same heuristics that apply to OMP. 
\begin{algorithm}[hbtp]
   \caption{Method of optimal directions for~$\psi=\ell_0$.}\label{alg:mod}
   \begin{algorithmic}[1]
      \REQUIRE Signals $\X=[\x_1,\ldots,\x_n]$ in~$\Real^{m \times n}$, initial dictionary $\D_0$ in~$\CC$, regularization parameter~$\mu$,
      number of iterations~$T$.
      \STATE Initialize $\D \leftarrow \D_0$;
      \FOR{$t=1,\ldots,T$}
      \STATE compute the sparse codes, \eg, with OMP:
      \FOR{$i=1,\ldots,n$}
      \STATE
      \begin{equation*}
         \alphab_i \approx \argmin_{\alphab \in \Real^p} \left[\frac{1}{2}\|\x_i-\D\alphab\|_2^2 \st  \|\alphab\|_0 \leq \mu \right]; \label{eq:dictalphaimod}
      \end{equation*}
      \ENDFOR
      \STATE update the dictionary~$\D$:
      \begin{equation}
         \D \leftarrow   \Pi_{\CC}[\X\A^\top (\A\A^\top)^{-1}]; \label{eq:updateDmod}
      \end{equation}
      \ENDFOR
      \RETURN the dictionary $\D$ in~$\CC$.
   \end{algorithmic}
\end{algorithm}
\begin{algorithm}[hbtp]
   \caption{Alternate minimization for~$\psi=\ell_1$.}\label{alg:alt}
   \begin{algorithmic}[1]
      \REQUIRE Signals $\X=[\x_1,\ldots,\x_n]$ in~$\Real^{m \times n}$, initial
      dictionary $\D_0$ in~$\CC$, regularization parameter~$\lambda$ or~$\mu$,
      number of iterations~$T$.
      \STATE Initialize $\D \leftarrow \D_0$;
      \FOR{$t=1,\ldots,T$}
      \STATE compute the sparse codes, \eg, with the homotopy method:
      \FOR{$i=1,\ldots,n$}
      \STATE
      update~$\alphab_i$ by solving~(\ref{eq:l1a}) with~$\x=\x_i$ if~$\lambda$ is provided, or~(\ref{eq:l1c}) if~$\mu$ is provided.
      \ENDFOR
      \STATE update the dictionary~$\D$:
      \begin{equation}
         \D \in \argmin_{\D \in \CC} \frac{1}{n}\sum_{i=1}^n\frac{1}{2}\|\x_i-\D\alphab_i\|_2^2; \label{eq:updateDl1}
      \end{equation}
      \ENDFOR
      \RETURN the dictionary $\D$ in~$\CC$.
   \end{algorithmic}
\end{algorithm}

Moving on now to the dictionary update step, we remark that~(\ref{eq:updateDmod})
and~(\ref{eq:updateDl1}) are not the same and are not equivalent to each other,
even though, at first sight, the optimization problem with respect to~$\D$ when~$\A$ is fixed
seems to be independent of the regularization~$\psi$.  The MOD
update indeed minimizes the least-square function~(\ref{eq:updateDl1}) by
discarding the constraint~$\D$ in~$\CC$, yielding a solution~$\X\A^\top
(\A\A^\top)^{-1}$ (assuming $\A\A^\top$ to be invertible), before projecting it
onto the constraint set~$\CC$. In constrained
optimization~\citep[see][]{bertsekas}, it is usually not appropriate to discard
a constraint, solve the resulting unconstrained formulation, before projecting
back the solution onto the constraint set. In general, such a procedure can be
indeed arbitrarily bad regarding the original constrained problem.  In the
context of~$\ell_0$-regularized dictionary learning, however, such an approach
does fortunately make sense. The objective function~(\ref{eq:dictconstraint})
when~$\psi=\ell_0$ is indeed invariant to a rescaling operation consisting of
multiplying each column of~$\d_j$ by a scalar~$\gamma_j$ and the
corresponding~$j$-th row of~$\A$ by~$(1/\gamma_j)$. As a consequence, the
minimum value of the objective function is independent of the scale of the
columns of~$\D$; thus the constraint~$\D \in \CC$ can always be safely
ignored and the normalization~$\D \leftarrow \Pi_\CC[\D]$ can be applied at any
moment without changing the quality of the dictionary. 

In the case of~$\ell_1$, the objective function is not invariant to the
previous rescaling operation and the MOD update for the dictionary
should not be used anymore. Instead, the objective
function~(\ref{eq:updateDl1}) needs to be minimized with another approach.
Since it is convex,~\citet{ng-sparsecoding} propose to use a Newton method in
the dual of~(\ref{eq:updateDl1}), which requires solving~$p \times p$ linear
systems at every Newton iteration. Another easy-to-implement choice is the
block coordinate descent update of~\citet{mairal2010}, which we present
in Algorithm~\ref{alg:updateD}. Each step of the block coordinate descent
approach is guaranteed to decrease the value of the objective function, while
keeping~$\D$ in the constraint set~$\CC$, and ultimately
solve~(\ref{eq:updateDl1}) since the problem is convex~\citep[see][for
convergence results of block coordinate descent methods]{bertsekas}.
In practice,~(\ref{eq:updateDl1}) does not need to be solved exactly and
performing a few passes, say 10, over the columns of~$\D$ seems to perform well
in practice.

The update~(\ref{eq:updated}) can be justified as follows.
Minimizing~(\ref{eq:updateDl1}) with respect to one column~$\d_j$ when keeping
the other columns fixed can be formulated as
\begin{displaymath}
   \d_j \leftarrow \argmin_{\d \in \Real^m, \|\d\|_2 \leq 1} \left[ \sum_{i=1}^n
      \frac{1}{2}\left\|\x_i - \sum_{\substack{l \neq j}} \alphab_i[l]\d_l - \alphab_i[j]\d\right\|_2^2
\right].
\end{displaymath}
Then, this formulation can be rewritten in a matrix form
\begin{displaymath}
   \d_j \leftarrow \argmin_{\d \in \Real^m, \|\d\|_2 \leq 1} \left[ \frac{1}{2}\left\|\X - \D\A + \d_j\alphab^j -\d \alphab^j\right\|_\text{F}^2
\right],
\end{displaymath}
where~$\d_j$ on the right side is the value of the variable~$\d_j$ before the
update, and~$\alphab^j$ in~$\Real^{1 \times n}$ is the $j$-th row of the
matrix~$\A$. After expanding the Frobenius norm and removing the constant term, we obtain
\begin{displaymath}
   \begin{split}
      \d_j & \leftarrow \argmin_{\d \in \Real^m, \|\d\|_2 \leq 1} \left[- \d^\top(\X - \D\A + \d_j\alphab^j)\alphab^{j\top} +  \frac{1}{2}\left\|\d\alphab^j\right\|_\text{F}^2 \right] \\
& = \argmin_{\d \in \Real^m, \|\d\|_2 \leq 1} \left[- \d^\top (\b_j - \D \c_j + \d_j \C[j,j]) +  \frac{1}{2}\left\|\d\right\|_2^2\C[j,j] 
\right] \\
& = \argmin_{\d \in \Real^m, \|\d\|_2 \leq 1} \left[ \frac{1}{2} \left\| \frac{1}{\C[j,j]} (\b_j - \D \c_j) + \d_j -\d \right\|_2^2 
\right], 
   \end{split}
\end{displaymath}
where the quantities~$\B=[\b_1,\ldots,\b_p]$ in~$\Real^{m \times p}$
and~$\C=[\c_1,\ldots,\c_p]$ in~$\Real^{p \times p}$ are defined in
Algorithm~\ref{alg:updateD}. Finally, we see from the last equation that
updating~$\d_j$ amounts to performing one orthogonal projection of the vector
$({1}/{\C[j,j]})(\b_j - \D \c_j) + \d_j$ onto the unit Euclidean ball, leading to the
update~(\ref{eq:updated}).

\begin{algorithm}[hbtp]
   \caption{Dictionary update with block coordinate descent.}
   \label{alg:updateD}
   \begin{algorithmic}[1]
      \REQUIRE $\D_0 \in \CC$ (input dictionary); $\X \in \Real^{m \times n}$ (dataset); $\A \in \Real^{p \times n}$ (sparse codes); 
      \STATE Initialization: $\D \leftarrow \D_0$; $\B \leftarrow \X\A^\top$; $\C \leftarrow \A\A^\top$;
      \REPEAT
      \FOR {$j = 1,\ldots,p$}
      \STATE update the $j$-th column to optimize for (\ref{eq:updateDl1}):
      \begin{equation} 
         \begin{split}
            \d_j &\leftarrow \frac{1}{\C[j,j]}(\b_j-\D\c_j) + \d_j, \\
            \d_j &\leftarrow \frac{1}{\max(\|\d_j\|_2,1)}\d_j. \label{eq:updated}
         \end{split}
      \end{equation}
      \ENDFOR
      \UNTIL convergence;
      \RETURN $\D$ (updated dictionary).
   \end{algorithmic}
\end{algorithm}

\paragraph{Block coordinate descent.}
We have presented in the previous paragraph the alternate minimization method
where the dictionary is updated via block coordinate descent. We have also
introduced in Section~\ref{sec:optiml1} a coordinate descent algorithm for
dealing with~$\ell_1$-penalized problems. It is then natural to combine the two
approaches into a block coordinate descent dictionary learning method
for~(\ref{eq:dictpenalty}) with~$\psi=\ell_1$. We present it in
Algorithm~\ref{alg:dictbcd}, where~$S_\lambda$ is the soft-thresholding
operator applied element-wise to a vector.  The main asset of the
block-coordinate descent approach is its simplicity, and ease of
implementation. In practice, it can also be very effective.

\begin{algorithm}[hbtp]
   \caption{Block coordinate descent for~(\ref{eq:dictpenalty}) with $\psi=\ell_1$.}\label{alg:dictbcd}
   \begin{algorithmic}[1]
      \REQUIRE Signals $\X=[\x_1,\ldots,\x_n]$ in~$\Real^{m \times n}$, initial
      dictionary $\D_0$ in~$\CC$, initial sparse coefficients~$\A_0$ in~$\Real^{p \times n}$ (possibly equal to zero), regularization parameter~$\lambda$,
      number of iterations~$T$.
      \STATE Initialize $\D \leftarrow \D_0$; $\A \leftarrow \A_0$;
      \FOR{$t=1,\ldots,T$}
      \STATE update the sparse codes, one row of~$\A$ at a time; perform at least one pass (possibly several) of the following iterations:
      \FOR{$j=1,\ldots,p$}
      \STATE 
         \begin{equation}
            \alphab^j \leftarrow S_\lambda\left(\alphab^j + \frac{1}{\|\d_{j}\|_2^2}\d_{j}^\top(\X-\D\A)\right); \label{eq:updatecddict}
         \end{equation}
      \ENDFOR
      \STATE update the dictionary~$\D$: perform one pass or more of the inner loop of Algorithm~\ref{alg:updateD}.
      \ENDFOR
      \RETURN the dictionary $\D$ in~$\CC$.
   \end{algorithmic}
\end{algorithm}

\paragraph{K-SVD.}
For $\ell_0$-regularized dictionary learning problems, one of the most 
popular approach is the K-SVD algorithm of \citet{aharon2006}. The algorithm is related to the
alternate minimization strategy MOD~\citep{engan1999}, but the dictionary
update step updates the non-zero coefficients of the matrix~$\A$ at the
same time. The full approach is detailed in Algorithm~\ref{alg:dictksvd}.  The
sparse encoding approach is the same as in MOD, but the dictionary elements are
updated in a block coordinate fashion, one at a time.

More precisely, when updating the column~$\d_j$, the algorithm finds the subset
of signals~$\Omega$ that use~$\d_j$ in their current sparse decomposition.
Then, the non-zero coefficients of the~$j$-th row of~$\A$, denoted
by~$\alphab_\Omega^j$, are updated at the same time as~$\d_j$
in~(\ref{eq:updatedksvd}). Such an update is in fact equivalent to performing a
rank-one singular value decomposition, which has inspired the name of the
algorithm.  A typical technique for solving~(\ref{eq:updatedksvd}) is the power
method~\citep[see][]{golub2012}, which alternates between the optimization
of~$\d$ with~$\betab$ fixed, and of~$\betab$ with~$\d$ fixed.

\begin{algorithm}[hbtp]
   \caption{K-SVD for~(\ref{eq:dictconstraint}) with $\psi=\ell_0$.}\label{alg:dictksvd}
   \begin{algorithmic}[1]
      \REQUIRE Signals $\X=[\x_1,\ldots,\x_n]$ in~$\Real^{m \times n}$, initial
      dictionary $\D_0$ in~$\CC$, regularization parameter~$\mu$,
      number of iterations~$T$.
      \STATE Initialize $\D \leftarrow \D_0$; 
      \FOR{$t=1,\ldots,T$}
      \STATE compute the sparse codes, \eg, with OMP:
      \FOR{$i=1,\ldots,n$}
      \STATE
      \begin{equation*}
         \alphab_i \approx \argmin_{\alphab \in \Real^p} \frac{1}{2}\|\x_i-\D\alphab\|_2^2 \st  \|\alphab\|_0 \leq \mu; 
      \end{equation*}
      \ENDFOR
      \STATE update the dictionary~$\D$:
      \FOR{$j=1,\ldots,p$}
         \STATE define the set 
           $ \Omega \defin \{ i \in \{1,\ldots,n\} :  \alphab_i[j] \neq 0  \}$;
         \STATE update~$\d_j$ and the non-zero coefficients of~$\alphab^j$:
         \begin{equation}
            (\d_j,\alphab^j_\Omega) \in \argmin_{\substack{\d \in \Real^m, \|\d\|_2 \leq 1 \\ \betab \in \Real^{|\Omega|}}} \left\|\X_\Omega -\sum_{\substack{l \neq j}}\d_l \alphab_\Omega^l - \d\betab^\top \right\|_{\text{F}}^2; \label{eq:updatedksvd}
         \end{equation}
      \ENDFOR
      \ENDFOR
      \RETURN the dictionary $\D$ in~$\CC$.
   \end{algorithmic}
\end{algorithm}

\paragraph{Online dictionary learning.}
Finally, we describe the online dictionary learning method
of~\citet{mairal2010}, which we have used in the experiments of this monograph.
The approach is based on stochastic approximations, where the goal is to minimize the
expectation
\begin{equation}
   \min_{\D \in \CC} \EE_\x[L(\x,\D)], \label{eq:expect2}
\end{equation}
where~$L$ is defined in~(\ref{eq:functionL}), and the expectation is taken with
respect to the unknown probability distribution of the data~$\x$. With a
discrete probability distribution with a finite sample,~(\ref{eq:expect2})
simply amounts to~(\ref{eq:expectn}).

In a nutshell, the algorithm works by using a stochastic version of the
majorization-minimization principle~\citep{mairal2013stochastic}.  It sequentially builds a quadratic
surrogate of the expected cost~(\ref{eq:expect2}), which is minimized at every
iteration. It can be shown that the quality of approximation of the surrogate  
near the current estimate improves during the algorithm iterates, and that the 
approach provides asymptotically a stationary point of~(\ref{eq:expect2}).  The
basic procedure is summarized in Algorithm \ref{alg:odl} and some variants that
significantly improves its practical efficiency are presented
in~\citep{mairal2010}.

Assuming that the training set is composed of i.i.d.~samples
of a data distribution, the inner loop draws one element $\x_t$
at a time, as in the stochastic gradient descent algorithm, and alternates
between classical sparse coding steps for computing the decomposition vector~$\alphab_t$
of~$\x_t$ over the current dictionary $\D_{t-1}$, and dictionary update steps.
The new dictionary $\D_t$ is computed by minimizing over $\CC$ the
quadratic surrogate~(\ref{eq:updateDodl}) where the vectors $\alphab_i$ for $i
< t$ have been computed during the previous steps of the algorithm. 
The motivation of the approach
is based on two ideas:
\begin{itemize}
   \item the quadratic surrogate aggregates the past information with a few
      sufficient statistics, namely the matrices $\B$ and~$\C$. One key
      component of the convergence analysis shows that the difference between
      the quadratic surrogate and the objective function at the current
      estimate converges to zero almost surely.
   \item with warm restart, the dictionary update is very efficient. In
      practice, one pass over the dictionary elements seems to perform
      empirically well.
\end{itemize}
Variants to improve the efficiency of the online dictionary learning algorithm
include: (i) the use of mini-batches instead of drawing a single point at every
iteration, (ii) rescaling the past data by updating the matrices~$\B \leftarrow
\gamma_t \B + \x_t\alphab_t^\top$, with a sequence of parameters~$\gamma_t$
that converges to one.  The implementation of the SPAMS software contains such
variants, which have proven to be important in practice~\citep[see][for
additional discussions]{mairal2010}.
 
\begin{algorithm}[hbtp]
   \caption{Online dictionary learning for~(\ref{eq:dictpenalty} with~$\psi=\ell_1$.}
   \label{alg:odl}
   \begin{algorithmic}[1]
      \REQUIRE i.i.d signals $\x_1,\x_2,\ldots \in \Real^m$, regularization parameter $\lambda$, initial dictionary $\D_0 \in \CC$, number of iterations $T$.
      \STATE Initialization: $\B \in \Real^{m \times p} \leftarrow 0$,
      $\C \in \Real^{p \times p} \leftarrow 0$, $\D \leftarrow \D_0$;
      \FOR {$t= 1,\ldots,T$}
      \STATE draw $\x_t$ and sparsely encode it using homotopy:
      \begin{equation*} 
         \alphab_t  \in  \argmin_{\alphab \in
         \Real^p} \frac{1}{2}\|\x_t-\D\alphab\|_2^2 +
         \lambda\|\alphab\|_1; 
      \end{equation*}
      \STATE $\B \leftarrow \B + \x_t \alphab_t^\top$;
      \STATE $\C \leftarrow \C + \alphab_t \alphab_t^\top$;
      \STATE update $\D$ using Algorithm~\ref{alg:updateD} (without recomputing the matrices~$\B$ and~$\C$) such that
      \begin{eqnarray}
         \!\!\!\!\D & \leftarrow &\argmin_{\D \in \CC} \frac{1}{t}\sum_{i=1}^t
         \Big(\frac{1}{2}\|\x_i-\D\alphab_i\|_2^2 + \lambda\|\alphab_i\|_1\Big),\nonumber \\
         & = & \argmin_{\D \in \CC} \frac{1}{t}\Big( \frac{1}{2}\trace(\D^\top\D\C) - \trace(\D^\top\B)\Big); \label{eq:updateDodl}
      \end{eqnarray}
      \ENDFOR
      \RETURN learned dictionary $\D$ in~$\CC$.
   \end{algorithmic}
\end{algorithm}

Note that along the same lines, a similar algorithm called ``recursive least square
dictionary learning'' was independently developed by \citet{skretting2010}
focusing on the~$\ell_0$-penalty.

\paragraph{Other extensions.}
Because the literature on dictionary learning algorithms is vast, we have
presented a few ones that address the formulations~(\ref{eq:dictpenalty})
or~(\ref{eq:dictconstraint}). Extending these approaches to other settings,
such as other loss functions than the square loss, or other regularization
functions, is easy in general by replacing the sparse encoding and dictionary
update steps by ad hoc optimization techniques.

%% file: content_arxiv/optim_other.tex
For simplicity, we have focused on sparse estimation problems involving
least-square cost functions, and have left aside structured sparsity penalties
such as the group Lasso. Some of the methods we have described can be easily
extended to these settings, in particular the proximal gradient and coordinate
descent algorithms are relatively flexible. We leave these extensions to more
exhaustive reviews~\citep[see][for instance]{bach2012}.

Other classes of algorithms are also popular for solving convex sparse
estimation problems. For example, proximal splitting
techniques~\citep{combettes2010proximal,chambolle2011,condat2013,raguet2013}, such as the alternating
direction method of multipliers~\citep[see][]{boyd2011}, are often used in signal
processing. These methods are very flexible, simple to implement, and can be applied to
a large range of problems. However, they require tuning (at least) one parameter
to obtain fast convergence in practice, which, to the best of our knowledge,
cannot be done automatically.  In the context of machine learning, stochastic
optimization techniques have also been
developed~\citep[\eg][]{Duchi2010,xiao2010dual}. These methods are adapted to
large-scale problems; even though, they may seem simple at first sight, they
require a careful choice of a step size parameter and thus remain a bit
difficult to use.

Finally, we have also omitted quasi-Newton methods such as variants of
L-BFGS~\citep{schmidt2012}, which can exploit the curvature of the objective
function to achieve fast convergence.

%% file: content_arxiv/ccl.tex
In this monograph, we have reviewed a large number of applications of
dictionary learning in image processing and computer vision and presented basic
sparse estimation tools. We started with a historical tour of sparse estimation
in signal processing and statistics, before moving to more recent concepts such
as sparse recovery and dictionary learning. Subsequently, we have shown that
dictionary learning is related to matrix factorization techniques, and
that it is particularly effective for modeling natural image patches. As a
consequence, it has been used for tackling several image processing problems
and is a key component of many state-of-the-art methods in visual recognition.
We have concluded the monograph with a presentation of optimization techniques
that should make dictionary learning easy to use for researchers that are not
experts of the field.

With efficient software available, it is remarkable to see that the
interest for dictionary learning and sparse estimation is still acute today,
and seems to increase in other scientific communities than image processing or
computer vision. In particular, new applications of sparse matrix factorization
are found in neuroscience~\citep{varoquaux2011}, bioinformatics for processing
gene expression data~\citep[see][Section 5.4]{barillot2012}, audio
processing~\citep{baby2014}, and in astrophysics~\citep{vinci2014}.  We have
also put the emphasis on natural images, and we have omitted the application of
dictionary learning to other domains, such as hyperspectral data, where
significant success have been obtained~\citep{charles2011,xing2012}, or to
disparity maps for stereo imaging~\citep{tosic2011}.

Exploring the possible use of dictionary learning in these fields would be of
great interest and high importance, but unfortunately this goal falls beyond
the scope of this monograph.